%% file: phd_main.tex
\documentclass[
11pt, 
oneside, 
english, 
singlespacing, 
nohyperref, 
headsepline, 
]{phd_style} 

\usepackage[utf8]{inputenc} 
\usepackage[T1]{fontenc} 
\usepackage{mathpazo} 

\usepackage{booktabs}       
\usepackage{multirow}
\usepackage{microtype}      
\usepackage{xcolor}         
\usepackage{footmisc}
\usepackage{tabularx}
\usepackage{thmtools}
\usepackage{thm-restate}

\usepackage{hyperref}
\definecolor{mydarkblue}{rgb}{0,0.08,0.5}
\hypersetup{ %
	pdftitle={},
	pdfauthor={},
	pdfsubject={},
	pdfkeywords={},
	colorlinks=true,
	linkcolor=mydarkblue,
	citecolor=mydarkblue,
	filecolor=mydarkblue,
	urlcolor=mydarkblue
}

\usepackage[numbers]{natbib}
\renewcommand{\citet}[1]{\cite{#1}}

\usepackage[autostyle=true]{csquotes} 

\usepackage{color}
\usepackage{amsfonts}
\usepackage{algorithm}
\usepackage{algorithmic}
\usepackage{graphicx}
\usepackage{nicefrac}
\usepackage{bbm}
\usepackage{wrapfig}
\usepackage{tikz}
\usepackage[titletoc,title]{appendix}
\usepackage{stmaryrd}
\usepackage{comment}
\usepackage{enumerate}
\usepackage{enumitem}
\usepackage[normalem]{ulem}

\usepackage{titlesec}
\usepackage{amsmath,amsthm,amssymb,mathtools}
\usepackage{caption}
\DeclareCaptionFont{captionfont}{\fontsize{9pt}{12pt}\selectfont}
\usepackage[caption = false]{subfig}

\usepackage[extdef=true]{delimset}
\usepackage[capitalize,noabbrev]{cleveref}
\usepackage[most,listings]{tcolorbox}

\newtheorem{lemma}{Lemma}
\newtheorem{corollary}{Corollary}
\newtheorem{theorem}{Theorem}
\newtheorem*{theorem*}{Theorem}
\newtheorem{proposition}{Proposition}
\newtheorem*{proposition*}{Proposition}
\newtheorem{assumption}{Assumption}

\newtheorem{conjecture}{Conjecture}		

\theoremstyle{definition}
\newtheorem{definition}{Definition}

\newcommand{\eg}{{\it e.g.}}
\newcommand{\ie}{{\it i.e.}}
\newcommand{\cf}{{\it cf.}}

\newcommand{\etal}{{\it et al.}}

\newcommand{\lem}{Lemma}
\newcommand{\lems}{Lemmas}

\newcommand{\thm}{Theorem}
\newcommand{\prop}{Proposition}

\newcommand{\defin}{Definition}
\newcommand{\eq}{Equation}
\newcommand{\eqs}{Equations}
\newcommand{\fig}{Figure}
\newcommand{\figs}{Figures}
\newcommand{\tab}{Table}
	
\newcommand{\sect}{Section}
\newcommand{\sects}{Sections}
\newcommand{\subsect}{Section}
\newcommand{\subsects}{Sections}	
\newcommand{\app}{Appendix}
	
\newcommand{\subapp}{Appendix}

\definecolor{darkgray}{rgb}{0.5, 0.5, 0.5}
\newcommand\darkgray[1]{{\color{darkgray}#1}}

\definecolor{green}{rgb}{0.0, 0.5, 0.0}

\include{notation}

\thesistitle{Understanding Deep Learning \\ via Notions of Rank} 
\supervisor{Dr. Nadav Cohen} 
\degree{Doctor of Philosophy in Computer Science} 
\author{Noam Razin} 


\AtBeginDocument{
\hypersetup{pdftitle=\ttitle} 
\hypersetup{pdfauthor=\authorname} 
\hypersetup{pdfkeywords=\keywordnames} 
}

\begin{document}

\frontmatter 

\pagestyle{plain} 


\begin{titlepage}
\begin{center}

\vspace*{.06\textheight}
{
	\includegraphics[width=0.3\textwidth]{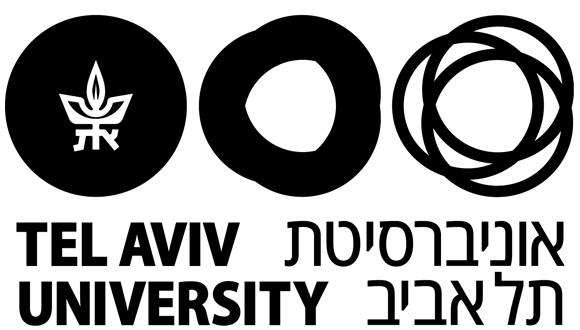} \\[0.5cm]
	\fontsize{12}{12}\selectfont Raymond and Beverly Sackler Faculty of Exact Sciences \\[1mm] Blavatnik School of Computer Science 
}\vspace{1.5cm} 

\HRule \\[0.4cm] 
{\huge \bfseries \ttitle\par}\vspace{0.4cm} 
\HRule \\[1.5cm] 

\begin{center} \large
\fontsize{13}{16}\selectfont 
Author:\\
\textbf{\authorname} 
\end{center}
\vspace{-0.5mm}
\begin{center} \large
\fontsize{13}{16}\selectfont 
Supervisor:\\
\textbf{\supname} 
\vspace{2mm}
\end{center}

\vfill

\fontsize{13}{16}\selectfont  \textit{A dissertation submitted for the degree of \\ \degreename}\\[0cm] 
 
\vfill
\vspace{3mm}
{\large August 2024}\\[3.5cm] 

\vfill
\end{center}
\end{titlepage}

\clearpage


\begin{acknowledgements}
\parindent 0pt
\topsep 4pt plus 1pt minus 2pt
\partopsep 1pt plus 0.5pt minus 0.5pt
\itemsep 2pt plus 1pt minus 0.5pt
\parsep 2pt plus 1pt minus 0.5pt
\parskip .5pc

\addchaptertocentry{\acknowledgementname} 
\vspace{-3mm}
\input{ack}
\end{acknowledgements}

\clearpage


\begin{abstract}
\parindent 0pt
\topsep 4pt plus 1pt minus 2pt
\partopsep 1pt plus 0.5pt minus 0.5pt
\itemsep 2pt plus 1pt minus 0.5pt
\parsep 2pt plus 1pt minus 0.5pt
\parskip .5pc

	\addchaptertocentry{\abstractname}
	\vspace{-3mm}
	\input{abstract}
\end{abstract}


\setcounter{tocdepth}{1}
\tableofcontents 




%
%


%
%
%
%


\mainmatter 


{
\parindent 0pt
\topsep 4pt plus 1pt minus 2pt
\partopsep 1pt plus 0.5pt minus 0.5pt
\itemsep 2pt plus 1pt minus 0.5pt
\parsep 2pt plus 1pt minus 0.5pt
\parskip .5pc

\part{Introduction}
\label{part:intro}

\include{Chapters/intro}

\part{Generalization via Implicit Rank Minimization}
\label{part:gen}

\include{Chapters/imp_reg_not_norms}
\include{Chapters/imp_reg_tf}

\include{Chapters/imp_reg_htf}

\include{Chapters/imp_reg_related}

\part{Expressiveness of Graph Neural Networks via Separation Rank}
\label{part:gnns}

\include{Chapters/gnn_interactions}
\include{Chapters/gnn_related}

\part{Conclusion}
\label{part:conclusion}

\include{Chapters/conclusion}

}


\clearpage
{\small
	\bibliographystyle{plainnat}
	\bibliography{refs}
}


\appendix 

\part{Appendix}

{
\parindent 0pt
\topsep 4pt plus 1pt minus 2pt
\partopsep 1pt plus 0.5pt minus 0.5pt
\itemsep 2pt plus 1pt minus 0.5pt
\parsep 2pt plus 1pt minus 0.5pt
\parskip .5pc

\include{Appendices/imp_reg_not_norms}

\include{Appendices/imp_reg_tf}

\include{Appendices/imp_reg_htf}

\include{Appendices/gnn_interactions}

}

\end{document}

%% file: notation.tex
\def\be{\begin{equation}}
	\def\ee{\end{equation}}
\def\beas{\begin{eqnarray*}}
	\def\eeas{\end{eqnarray*}}
\def\bea{\begin{eqnarray}}
	\def\eea{\end{eqnarray}}
\newcommand{\hbf}{{\mathbf h}}
\newcommand{\xbf}{{\mathbf x}}
\newcommand{\x}{{\mathbf x}}
\newcommand{\ybf}{{\mathbf y}}
\newcommand{\zbf}{{\mathbf z}}
\newcommand{\ubf}{{\mathbf u}}
\newcommand{\vbf}{{\mathbf v}}

\newcommand{\wbf}{{\mathbf w}}
\newcommand{\w}{{\mathbf w}}
\newcommand{\sbf}{{\mathbf s}}
\newcommand{\ebf}{{\mathbf e}}
\newcommand{\abf}{{\mathbf a}}
\newcommand{\bbf}{{\mathbf b}}

\newcommand{\obf}{{\mathbf o}}

\newcommand{\qbf}{{\mathbf q}}
\newcommand{\K}{{\mathcal K}}

\newcommand{\1}{{\mathbf 1}}

\newcommand{\Abf}{{\mathbf A}}
\newcommand{\Bbf}{{\mathbf B}}

\newcommand{\Dbf}{{\mathbf D}}

\newcommand{\Gbf}{{\mathbf G}}
\newcommand{\Ibf}{{\mathbf I}}

\newcommand{\Obf}{{\mathbf O}}

\newcommand{\Wbf}{{\mathbf W}}
\newcommand{\Xbf}{{\mathbf X}}

\newcommand{\Qbf}{{\mathbf Q}}

\newcommand{\Ubf}{{\mathbf U}}
\newcommand{\Vbf}{{\mathbf V}}

\newcommand{\Sbf}{{\mathbf S}}
\newcommand{\Mbf}{{\mathbf M}}

\newcommand{\Zbf}{{\mathbf Z}}
\newcommand{\A}{{\mathcal A}}
\newcommand{\B}{{\mathcal B}}

\newcommand{\D}{{\mathcal D}}

\newcommand{\M}{{\mathcal M}}

\newcommand{\RR}{{\mathcal R}}
\renewcommand{\S}{{\mathcal S}}
\newcommand{\T}{{\mathcal T}}
\newcommand{\U}{{\mathcal U}}
\newcommand{\V}{{\mathcal V}}
\newcommand{\W}{{\mathcal W}}
\newcommand{\X}{{\mathcal X}}
\newcommand{\XX}{{\mathcal X}}
\newcommand{\Y}{{\mathcal Y}}
\newcommand{\OO}{{\mathcal O}}
\newcommand{\I}{{\mathcal I}}
\newcommand{\J}{{\mathcal J}}
\renewcommand{\L}{\mathcal{L}}

\newcommand{\R}{{\mathbb R}}
\newcommand{\N}{{\mathbb N}}

\newcommand{\normbig}[1]{\big \| #1 \big \|}
\newcommand{\normnoflex}[1]{\| #1 \|}

\newcommand{\inprod}[2]  {\left\langle{#1},{#2}\right\rangle}
\newcommand{\inprodbig}[2]  {\big \langle{#1},{#2} \big \rangle}
\newcommand{\inprodBig}[2]  {\Big \langle{#1},{#2} \Big \rangle}
\newcommand{\inprodbigg}[2]  {\bigg \langle{#1},{#2} \bigg \rangle}
\newcommand{\inprodnoflex}[2]{\langle{#1},{#2}\rangle}

\DeclareMathOperator*{\argmax}{argmax} 
\DeclareMathOperator*{\argmin}{argmin}

\newcommand{\vectbig}[1]{\text{vec}\brk1{#1}} 
\newcommand{\vectnoflex}[1]{\text{vec}\brk{#1}} 
\newcommand{\tfmat}[1]{\llbracket#1\rrbracket}
\newcommand{\tfmatflex}[1]{\left\llbracket#1\right\rrbracket}
\newcommand{\mat}[2]{\delim{\llbracket}{\rrbracket}*{#1 ; #2}}

\newcommand{\matbig}[2]{\bigl\llbracket#1 ; #2 \bigr\rrbracket}

\newcommand{\tenp}{\otimes}
\newcommand{\kronp}{\circ}

\newcommand{\hadmp}{\odot}

\newcommand{\sign}{\mathrm{sign}}

\DeclareMathOperator{\Tr}{Tr}

\newcommand{\rank}{\mathrm{rank}}
\newcommand{\erank}{\mathrm{erank}}
\newcommand{\irank}{\mathrm{irank}}
\newcommand{\seprank}[2]{\mathrm{sep}\brk*{#1 ; #2}}
\newcommand{\sepranknoflex}[2]{\mathrm{sep}\brk{#1 ; #2}}
\newcommand{\seprankbig}[2]{\mathrm{sep}\brk1{#1 ; #2}}

\newcommand{\htfseprank}{\mathrm{sep}}

\DeclareFontFamily{U}{mathx}{\hyphenchar\font45}
\DeclareFontShape{U}{mathx}{m}{n}{<-> mathx10}{}
\DeclareSymbolFont{mathx}{U}{mathx}{m}{n}
\DeclareMathAccent{\widebar}{0}{mathx}{"73}

\newcommand{\frodist}{\D_{Fro}}

\newcommand{\matrixend}{\Wbf_{\text{M}}}
\newcommand{\matrixenddot}{\dot{\Wbf}_{\text{M}}}
\newcommand{\matrixendbar}{\widebar{\Wbf}_{\text{M}}}
\newcommand{\tftensorend}{\W_{\text{T}}}
\newcommand{\tftensorendbar}{\widebar{\W}_{\text{T}}}
\newcommand{\tensorend}{\W_{\text{H}}}
\newcommand{\tensorendmap}{\mathcal{H}}
\newcommand{\htmodetree}{\T}
\newcommand{\mfendloss}{\L_{\text{M}}}
\newcommand{\mfobj}{\phi_{\text{M}}}
\newcommand{\tfendloss}{\L_{\text{T}}}
\newcommand{\tfobj}{\phi_{\text{T}}}
\newcommand{\htfendloss}{\L_{\text{H}}}
\newcommand{\htfobj}{\phi_{\text{H}}}
\newcommand{\mfcomp}[1]{\mathcal{C}^{ ( #1 ) }_{\text{M}}}
\newcommand{\tfcomp}[1]{\mathcal{C}^{ ( #1 ) }_{\text{T}}}
\newcommand{\htfcomp}[2]{\mathcal{C}^{ ( #1, #2 ) }_{\text{H}}}

\newcommand{\mfsing}[1]{\sigma^{ ( #1 ) }_{\text{M}}}
\newcommand{\tfcompnorm}[1]{\sigma^{ ( #1 ) }_{\text{T}}}
\newcommand{\htfcompnorm}[2]{\sigma^{ ( #1, #2 ) }_{\text{H}}}

\newcommand{\htftensorpart}[2]{\W^{(#1, #2)}}

\newcommand{\weightmat}[1]{\Wbf^{(#1)}}
\newcommand{\padcol}{\mathrm{PadC}}
\newcommand{\padrow}{\mathrm{PadR}}
\newcommand{\localcomp}{\mathrm{LC}}

\newcommand{\children}{C}
\newcommand{\subtree}{\htmodetree}
\newcommand{\parent}{Pa}
\newcommand{\interior}{\mathrm{int}}
\newcommand{\lsib}{\mathrm{\overleftarrow{S}}}
\newcommand{\rsib}{\mathrm{\overrightarrow{S}}}

\newcommand{\graph}{\mathcal{G}}
\newcommand{\vertices}{\mathcal{V}}
\newcommand{\edges}{\mathcal{E}}
\newcommand{\neigh}{\mathcal{N}}
\newcommand{\neighin}{\mathcal{N}_{in}}
\newcommand{\neighout}{\mathcal{N}_{out}}
\newcommand{\cut}{\mathcal{C}}
\newcommand{\cutset}{\mathcal{S}}
\newcommand{\cutdir}{\mathcal{C}^{\to}}
\newcommand{\cutsetdir}{\mathcal{S}^{\to}}

\newcommand{\agg}{\text{\sc{aggregate}}}

\newcommand{\fmat}{\Xbf}
\newcommand{\fvec}[1]{\xbf^{(#1)}}
\newcommand{\hidvec}[2]{\hbf^{(#1, #2)}}

\newcommand{\mindim}{D}
\newcommand{\indim}{D_{x}}
\newcommand{\hdim}{D_{h}}

\newcommand{\params}{\theta}
\newcommand{\funcvert}[3]{f^{(#1, #2, #3)}}
\newcommand{\funcgraph}[2]{f^{(#1, #2)}}
\newcommand{\nwalk}[3]{\rho_{#1} \brk{#2, #3}}

\newcommand{\deltatensor}[1]{\boldsymbol{\delta}^{(#1)}}

\newcommand{\tensorendgraph}{\Ftensor}

\newcommand{\fvecdup}[2]{\xbf^{(#1, #2)}}
\newcommand{\deltatensordup}[3]{\boldsymbol{\delta}^{(#1, #2, #3)}}
\newcommand{\weightmatdup}[3]{\Wbf^{(#1, #2, #3)}}
\newcommand{\contract}[1]{*_{#1}}
\newcommand{\contractlast}{*}
\newcommand{\modemapgraph}{\mu}

\newcommand{\dupmap}[3]{\phi_{#1, #2, #3}}
\newcommand{\dupmapvertex}[3]{\phi^{(t)}_{#1, #2, #3}}

\newcommand{\multiset}[1]{\brk[c]*{ \!\! \brk[c]*{ #1 } \!\! } }
\newcommand{\multisetbig}[1]{\brk[c]!{ \!\! \brk[c]!{ #1 } \!\! } }

\newcommand{\tngraph}{\TT}
\newcommand{\tnvertices}[1]{\vertices_{#1}}
\newcommand{\tnedges}[1]{\edges_{#1}}
\newcommand{\tnedgeweights}[1]{w_{#1}}
\newcommand{\mulcutweight}[1]{w_{#1}^{\Pi}}
\newcommand{\modcutweight}[1]{\widetilde{w}_{#1}^{\Pi}}
\newcommand{\modcutedges}[1]{ \widetilde{\edges}_{#1} }
\newcommand{\tnverticesleaf}[2]{\vertices_{#1} \brk[s]{ #2 }}

\newcommand{\gridmatnoflex}[1]{\Bbf \brk{#1}}
\newcommand{\gridtensor}[1]{\Btensor \brk*{#1}}
\newcommand{\gridtensorbig}[1]{\Btensor \brk1{#1}}
\newcommand{\gridtensornoflex}[1]{\Btensor \brk{#1}}
\newcommand{\walkin}[2]{ \mathrm{WI}_{#1} \brk{ #2 } }
\newcommand{\walkinvert}[3]{ \mathrm{WI}_{#1, #2} \brk{ #3 } }
\newcommand{\edgetypemap}{\tau}

\def\multisetcoeff#1#2{\ensuremath{\brk*{ \kern-.3em\brk*{\genfrac{}{}{0pt}{}{#1}{#2} }\kern-.3em}}}
\def\multisetcoeffnoflex#1#2{\ensuremath{\brk1{ \kern-.3em\brk1{\genfrac{}{}{0pt}{}{#1}{#2} }\kern-.3em}}}

\newcommand{\Atensor}{\boldsymbol{\mathcal{A}}}
\newcommand{\Btensor}{\boldsymbol{\mathcal{B}}}
\newcommand{\Ftensor}{\boldsymbol{\mathcal{F}}}
\newcommand{\Ttensor}{\boldsymbol{\mathcal{T}}}

\newcommand{\TT}{{\mathcal T}}

\newcommand{\WW}{{\mathcal W}}

%% file: ack.tex
It is hard to imagine how my PhD would have looked like without the support and guidance of the people around me.

I had the privilege of being the first graduate student of my advisor, Nadav Cohen.
This meant a unique opportunity to learn from him the ins and outs of being an academic: from acquiring research taste to technical and presentation skills. 
Though I see myself as somewhat of a perfectionist, when working with Nadav you quickly learn that you can always strive for better.
His guidance has undoubtedly left a strong imprint on me.
Above all, I greatly appreciate Nadav for consistently pushing me forward while having my professional and personal benefits in mind.
I could not have wished for a better advisor.

I want to thank my collaborators, friends, and mentors from Tel Aviv University and Apple, for making this experience such an enjoyable and productive one, including: Asaf Maman, Tom Verbin, Yotam Alexander, Nimrod De La Vega, Yoni Slutzky, Yuval Milo, Raja Giryes, Amir Globerson, Hattie Zhou, Arwen Bradley, Omid Saremi, Vimal Thilak, Preetum Nakkiran, Joshua Susskind, and Etai Littwin.

I am deeply grateful to my family for their unconditional love and everything that they have done to support me.
In particular, my parents have always been models of excellence for me.
While they never pushed me towards any specific direction, I believe that they have successfully instilled in me a drive to fully pursue my interests, to which I owe much of the ability to complete my PhD.
I also want to thank Prof. Rivka Dresner-Pollak for her close care, for which I will forever be indebted.

Lastly, during my studies I was fortunate to meet the love of my life, Eshbal.
Besides making my life more complete, she is largely responsible for any especially well-designed figure, slide, or poster that I have presented (you might be able to guess when we met by the style of figures in this thesis).
Even if it was just for meeting Eshbal, doing a PhD was well worth it.

%% file: abstract.tex
Despite the extreme popularity of deep learning in science and industry, its formal understanding is limited.
Common practices are based primarily on trial-and-error and intuition, often leading to suboptimal results.
Consequently, there is significant interest in developing a formal theory of deep learning, with the hope that it will shed light on empirical findings and lead to principled methods for improving the efficiency, reliability, and performance of neural networks.

This thesis puts forth notions of \emph{rank} as key for developing a theory of deep learning.
Specifically, building on a connection between certain neural network architectures and tensor factorizations, we employ notions of rank for studying the fundamental aspects of \emph{generalization} and \emph{expressiveness}.

With regard to generalization, the mysterious ability of neural networks to generalize is widely believed to stem from an \emph{implicit regularization}~---~a tendency of gradient-based training towards predictors of low complexity, for some yet unknown measure of complexity.
Through dynamical analyses, we establish an implicit regularization towards low rank in several types of neural network architectures (for corresponding notions of rank).
Notably, this implicit rank minimization differs from any type of norm minimization, in contrast to prior beliefs.
Implications of this finding for explaining generalization over natural data (\eg, audio, images, and text), as well as practical applications (novel regularization schemes), are presented.

With regard to expressiveness, we theoretically characterize the ability of graph neural networks to model interactions via \emph{separation rank}~---~a measure commonly used for quantifying entanglement in quantum physics.
As a practical application of our theory, we design an edge sparsification algorithm that preserves the ability of graph neural networks to model interactions.
Empirical evaluations demonstrate that it markedly outperforms alternative methods.

%% file: Chapters/intro.tex
In the past decade, \emph{deep learning} has been experiencing unprecedented success, and is largely responsible for the technological breakthroughs referred to in the public as “artificial intelligence'' (see, \eg,~\cite{krizhevsky2012imagenet,mikolov2013distributed,silver2016mastering,goodfellow2016deep,brown2020language,achiam2023gpt}).
However, despite the extreme popularity of deep learning in science and industry, its formal understanding is limited.
Common practices are based primarily on trial-and-error and intuition, often leading to suboptimal results, as well as compromise in important aspects including safety and robustness~\cite{szegedy2014intriguing,liu2020privacy}.
As a result, there is significant interest in developing a formal theory of deep learning, with the hope that it will shed light on empirical phenomena, and lead to principled methods for improving the efficiency, reliability, and performance of neural networks.

From the perspective of learning theory, understanding deep learning requires addressing the fundamental questions of \emph{optimization}, \emph{generalization}, and \emph{expressiveness}.
Optimization concerns the effectiveness of gradient-based methods in minimizing neural network training objectives that are non-convex.
Generalization treats the performance of a neural network beyond its training data.
Lastly, expressiveness refers to the ability of practically sized neural networks to represent rich classes of functions.

This thesis focuses on two of the fundamental questions~---~generalization and expressiveness.
It puts forth notions of rank as key for developing a theory of deep learning.
Our approach adopts tools from dynamical systems theory and tensor analysis, building on a recent connection between certain neural network architectures and \emph{tensor factorizations}~\cite{cohen2016expressive,cohen2016convolutional,cohen2017analysis,levine2018benefits,khrulkov2018expressive,khrulkov2019generalized}.\footnote{
	For the sake of this thesis, \emph{tensors} can be thought of as $N$-dimensional arrays, with $N \in \N$ arbitrary.
	For example, matrices correspond to the special case~$N = 2$ and vectors to $N = 1$. 
}
The main theoretical contributions and their practical implications are summarized below.

\subsection*{Generalization via Implicit Rank Minimization (\cref{part:gen})}

One of the central mysteries in deep learning is the ability of neural networks to generalize over natural data (\eg, audio, images, and text) when trained via gradient-based methods, despite having far more learnable parameters than training examples. 
This generalization takes place even in the absence of any explicit regularization~\cite{zhang2021understanding}.
Thus, conventional wisdom is that gradient-based training induces an \emph{implicit regularization}~---~a tendency to fit training examples with predictors of minimal complexity, for some measure of complexity~\cite{neyshabur2015search,neyshabur2017implicit}.
The fact that natural data gives rise to generalization is accordingly understood to result from an agreement between the complexity measure implicitly minimized during training and the complexity of the data (\ie, from the amenability of natural data to be fit with predictors of low complexity).

Mathematically formalizing the above intuition is regarded as a major open problem in the theory of deep learning.
A significant challenge towards doing so is that we lack definitions for predictor complexity that are both implicitly minimized during training of neural networks and capture the essence of natural data (in the sense of natural data being fittable with low complexity).
One widespread hope, initially articulated in~\cite{neyshabur2015search}, is that a characterization based on minimization of norms may apply.
Namely, it is known that for linear regression gradient-based methods converge to the solution with minimal Euclidean norm (see, \eg, Section~5 in~\cite{zhang2021understanding}), and the hope is that this result can carry over to neural networks if we allow replacing the Euclidean norm with a different (possibly architecture-dependent) norm~\cite{gunasekar2017implicit,soudry2018implicit,gunasekar2018implicit,gunasekar2018characterizing,li2018algorithmic,ji2019gradient,wu2020implicit,woodworth2020kernel}.

A standard testbed for studying this prospect is \emph{matrix factorization}~---~a model equivalent to \emph{linear neural networks}, \ie~fully-connected neural networks with no non-linearity~\cite{saxe2014exact}.
In \cref{chap:imp_reg_not_norms} (based on~\cite{razin2020implicit}) we prove that, in contrast to prior belief~\cite{gunasekar2017implicit}, the implicit regularization in matrix factorization cannot be captured by norms.
Specifically, we show that there exist settings in which it drives \emph{all} norms towards \emph{infinity} in favor of minimizing \emph{rank}. 
This indicates that, rather than perceiving the implicit regularization via norms, a potentially more useful interpretation is minimization of rank.

Capitalizing on this interpretation, in~\cref{chap:imp_reg_tf,chap:imp_reg_htf} (based on \cite{razin2021implicit} and \cite{razin2022implicit}, respectively), we establish that the tendency towards low rank extends from linear neural networks to more practical \emph{non-linear neural networks} (with polynomial non-linearity), which are equivalent to \emph{tensor factorizations}.
By characterizing the dynamics that gradient-based methods induce on such networks, we show that these result in a bias towards low rank, for architecture-dependent notions of rank defined over tensors.
To the best of my knowledge, our results constituted the first evidence for implicit regularization minimizing a notion of rank in non-linear neural networks.
Subsequent works have demonstrated that an analogous phenomenon occurs in other types of neural networks as well~\cite{jacot2022implicit,wang2023implicit,timor2023implicit,frei2023implicit}.

Motivated by the fact that notions of rank capture implicit regularization in certain non-linear neural networks, we empirically explore their potential as measures of complexity for explaining generalization over natural data.
We find that it is possible to fit standard image recognition datasets with predictors of extremely low rank. 
This leads us to believe that notions of rank may pave way to explaining both implicit regularization in deep learning and properties of natural data translating it to generalization. 

In terms of practical impact, based on our theory we develop an explicit regularization scheme for improving the performance of convolutional neural networks over tasks involving non-local interactions.
Other research groups have also built upon our analyses of implicit rank minimization for designing practical deep learning systems~\cite{jing2020implicit,huh2021low}.

\subsection*{Expressiveness of Graph Neural Networks via Separation Rank (\cref{part:gnns})}

In~\cref{chap:gnn_interactions} (based on~\cite{razin2023ability}), we extend the aforementioned connection between neural networks and tensor factorizations for studying the expressiveness of graph neural networks (GNNs)~\cite{hamilton2020graph}.
GNNs are widely used for modeling complex interactions between entities represented as vertices of a graph~\cite{duvenaud2015convolutional,kipf2017semi,gilmer2017neural,ying2018graph,wu2020comprehensive}.
Yet, a formal characterization of their ability to model interactions is lacking.
We address this gap by formalizing strength of interactions via \emph{separation rank}~\cite{beylkin2009multivariate,cohen2017inductive}~---~a measure widely used for quantifying entanglement in quantum physics.
Through this notion of rank, we characterize the ability of certain GNNs to model interaction between a given subset of vertices and its complement, \ie~between the sides of a given partition of input vertices.

Our analysis reveals that the ability of a GNN to model interaction is primarily determined by the partition's \emph{walk index}~---~a graph-theoretical characteristic defined by the number of walks originating from the boundary of the partition.
This formalizes conventional wisdom by which GNNs can model stronger interaction between regions of the input graph that are more interconnected.
Experiments corroborate the result by demonstrating that GNNs perform better on tasks requiring interactions across partitions with a higher walk index.

As a practical application of our theory, we design an \emph{edge sparsification} algorithm.
Edge sparsification concerns removal of edges from a graph for reducing computational and/or memory costs, while attempting to maintain selected properties of the graph (\cf~\cite{baswana2007simple,spielman2011graph,hamann2016structure,chakeri2016spectral,sadhanala2016graph,voudigari2016rank,li2020sgcn,chen2021unified}).
In the context of GNNs, our interest lies in maintaining the prediction accuracy as the number of edges removed increases.

We propose an algorithm for removing edges, called \emph{Walk Index Sparsification} (\emph{WIS}), which preserves the ability of a GNN to model interactions.
WIS is simple, computationally efficient, and in our experiments has markedly outperformed alternative methods in terms of attainable prediction accuracies across edge sparsity levels.
More broadly, it showcases the potential of improving GNNs by theoretically analyzing the interactions they can model via separation rank.

\subsection*{Included Work}

To recap, this thesis is based on the contents of the following papers.
\begin{enumerate}
	\item \cref{chap:imp_reg_not_norms} is based on~\cite{razin2020implicit} (published at NeurIPS 2020).
	
	\item \cref{chap:imp_reg_tf} is based on~\cite{razin2021implicit} (published at ICML 2021).
	
	\item \cref{chap:imp_reg_htf} is based on~\cite{razin2022implicit} (published at ICML 2022).
	
	\item \cref{chap:gnn_interactions} is based on~\cite{razin2023ability} (published at NeurIPS 2023).
\end{enumerate}

\subsection*{Excluded Work}
\label{chap:intro:excluded}

Aside from the research included in the thesis, during my doctoral studies I led or contributed to several other works, listed below.

\begin{enumerate}
	\item \cite{alexander2023what} (published at NeurIPS 2023) uses tools from quantum physics to characterize which properties of a data distribution make it suitable for locally connected neural networks.
	
	\item \cite{zhou2024what} (published at ICLR 2024) studies the length generalization ability of Transformer neural networks.
	
	\item \cite{razin2024vanishing} (published at ICLR 2024) identifies a vanishing gradients problem that occurs when finetuning language models via reinforcement learning.
	
	\item \cite{razin2024implicit} (published at ICML 2024) characterizes how the implicit regularization of policy gradient affects extrapolation to unseen initial states, focusing on linear quadratic control.
\end{enumerate}

%% file: Chapters/imp_reg_not_norms.tex
\chapter{Implicit Regularization in Deep Learning May Not Be Explainable \\ by Norms} 
\label{chap:imp_reg_not_norms}

This chapter covers the results of~\cite{razin2020implicit}.

\section{Background and Overview}
\label{mf:sec:overview}

As discussed in~\cref{part:intro}, the ability of neural networks to generalize is widely believed to stem from an implicit regularization of gradient-based training towards predictors of low complexity.
A prominent test-bed for studying implicit regularization in deep learning is \emph{matrix completion} (\cf~\cite{gunasekar2017implicit,arora2019implicit}): given a randomly chosen subset of entries from an unknown matrix~$\Wbf^*$, the task is to recover the unseen entries. 
This may be viewed as a prediction problem, where each entry in~$\Wbf^*$ stands for a data point: observed entries constitute the training set, and the average reconstruction error over the unobserved entries is the test error, quantifying generalization.

Fitting the observed entries in matrix completion is obviously an underdetermined problem with multiple solutions.
However, an extensive body of work (see~\cite{davenport2016overview} for a survey) has shown that if~$\Wbf^*$ is low-rank, certain technical assumptions (\eg~``incoherence'') are satisfied and sufficiently many entries are observed, then various algorithms can achieve approximate or even exact recovery.
Of these, a well-known method based upon convex optimization finds the minimal nuclear norm\footnote{
	The \emph{nuclear norm} of a matrix is the sum of its singular values.
}
matrix among those fitting observations (see~\cite{candes2009exact}).

One may try to solve matrix completion using shallow neural networks. A natural approach, \emph{matrix factorization}, boils down to parameterizing the solution as a product of two matrices~---~$\Wbf = \Wbf_2 \Wbf_1$~---~and optimizing the resulting (non-convex) objective for fitting observations.
Formally, this can be viewed as training a depth two linear neural network. 
It is possible to explicitly constrain the rank of the produced solution by limiting the shared dimension of~$\Wbf_1$ and~$\Wbf_2$.
However,~\cite{gunasekar2017implicit} showed that in practice, even when the rank is unconstrained, running gradient descent with small learning rate (step size) and initialization close to the origin (zero) tends to produce low-rank solutions, and thus allows accurate recovery if~$\Wbf^*$ is low-rank. 
Accordingly, they conjectured that the implicit regularization in matrix factorization boils down to minimization of nuclear norm:
\begin{conjecture}[from~\cite{gunasekar2017implicit}, informally stated] \label{conj:nuclear_norm}
	With small enough learning rate and initialization close enough to the origin, gradient descent on a full-dimensional matrix factorization converges to a minimal nuclear norm solution.
\end{conjecture}

In a subsequent work,~\cite{arora2019implicit} considered \emph{deep matrix factorization}, obtained by adding depth to the setting studied in~\cite{gunasekar2017implicit}.
Namely, they considered solving matrix completion by training a depth~$L$ linear neural network, \ie~by running gradient descent on the parameterization $\Wbf = \Wbf_L \cdots \Wbf_1$, with $L \in \N$ arbitrary (and the dimensions of $\{ \Wbf_l \}_{l = 1}^L$ set such that rank is unconstrained).
It was empirically shown that deeper matrix factorizations (larger~$L$) yield more accurate recovery when~$\Wbf^*$ is low-rank.
Moreover, it was conjectured that the implicit regularization, for any depth $L \geq 2$, can \emph{not} be described as minimization of a mathematical norm (or quasi-norm\footnote{
A \emph{quasi-norm}~$\norm{\cdot}$ on a vector space~$\V$ is a function from~$\V$ to~$\R_{\geq 0}$ that satisfies the same axioms as a norm, except for the triangle inequality, which is replaced by the weaker requirement: there exists $c \geq 1$ such that for all $\vbf_1, \vbf_2 \in \V$ it holds that $\norm{ \vbf_1 + \vbf_2} \leq c \cdot ( \norm{\vbf_1} + \norm{\vbf_2} )$.
\label{note:qnorm}
}).
\begin{conjecture}[based on~\cite{arora2019implicit}, informally stated] \label{conj:no_norm}
	Given a (shallow or deep) matrix factorization, for any norm (or quasi-norm)~$\norm{\cdot}$, there exists a set of observed entries with which small learning rate and initialization close to the origin can \emph{not} ensure convergence of gradient descent to a minimal (in terms of~$\norm{\cdot}$) solution.
\end{conjecture}

Conjectures~\ref{conj:nuclear_norm} and~\ref{conj:no_norm} contrast each other, and more broadly, represent opposing perspectives on the question of whether norms may be able to explain implicit regularization in deep learning.
In this chapter, we resolve the tension between the two conjectures by affirming the latter.
In particular, we prove that there exist natural matrix completion problems where fitting observations via gradient descent on a depth~$L \geq 2$ matrix factorization leads~---~with probability~$0.5$ or more over (arbitrarily small) random initialization~---~\emph{all} norms (and quasi-norms) to \emph{grow towards infinity}, while the rank essentially decreases towards its minimum.
This result is in fact stronger than the one suggested by Conjecture~\ref{conj:no_norm}, in the sense that:
\emph{(i)}~not only is each norm (or quasi-norm) disqualified by some setting, but there are actually settings that jointly disqualify all norms (and quasi-norms);
and
\emph{(ii)}~not only are norms (and quasi-norms) not necessarily minimized, but they can grow towards infinity.
We corroborate the analysis with empirical demonstrations.

Our findings imply that, rather than viewing implicit regularization in (shallow or deep) matrix factorization as minimizing a norm, a potentially more useful interpretation is \emph{minimization of rank}.
As a step towards assessing the generality of this interpretation, we empirically explore an extension of matrix factorization to \emph{tensor factorization}.
Our experiments show that in analogy with matrix factorization, gradient descent on a tensor factorization tends to produce solutions with low rank, where rank is defined in the context of tensors (we will theoretically establish this tendency towards low rank in \cref{chap:imp_reg_tf}).

Similarly to how matrix factorization corresponds to a linear neural network whose input-output mapping is represented by a matrix, it is known (see~\cite{cohen2016expressive}) that tensor factorization corresponds to a certain \emph{non-linear convolutional} neural network whose input-output mapping is represented by a tensor.
We thus obtain a second exemplar of a neural network architecture whose implicit regularization strives to lower a notion of rank for its input-output mapping.
This indicates that the phenomenon may be general, and formalizing notions of rank for input-output mappings of contemporary models may be key to explaining generalization in deep learning.
In \cref{chap:imp_reg_tf,chap:imp_reg_htf} we theoretically support this hypothesis by employing the equivalence between tensor factorizations and certain neural networks.

\medskip

The remainder of the chapter is organized as follows.
\cref{mf:sec:model} presents the deep matrix factorization model.
\cref{mf:sec:analysis} delivers our analysis, showing that its implicit regularization can drive all norms to infinity.
Lastly, experiments with both the analyzed setting and tensor factorization are given in \cref{mf:sec:experiments}.

\section{Deep Matrix Factorization} 
\label{mf:sec:model}

Suppose we would like to complete a $D$-by-$D'$ matrix based on a set of observations $\{ y_{i , j} \in \R \}_{( i , j ) \in \Omega}$, where $\Omega \subset \{ 1 , 2 , \ldots , D \} \times \{ 1 , 2 , \ldots , D' \}$.
A standard (underdetermined) loss function for the task is:
\be
\mfendloss : \R^{D \times D'} \to \R_{\geq 0}
\quad , \quad
\mfendloss ( \Wbf ) = \frac{1}{2} \sum\nolimits_{( i , j ) \in \Omega} \big( ( \Wbf )_{i , j} - y_{i , j} \big)^2
\label{mf:eq:loss}
\text{\,.}
\ee
Employing a depth~$L$ matrix factorization, with hidden dimensions $D_1 , \ldots, D_{L - 1} \in \N$, amounts to optimizing the \emph{overparameterized objective}:
\be
\mfobj ( \Wbf_1, \ldots , \Wbf_L ) := \mfendloss ( \matrixend ) = \frac{1}{2} \sum\nolimits_{( i , j ) \in \Omega} \big( ( \matrixend )_{i , j} - y_{i , j} \big)^2
\label{mf:eq:oprm_obj}
\text{\,,}
\ee
where $\Wbf_l \in \R^{D_l \times D_{l - 1}}$, $l = 1 , \ldots , L$, with $D_L := D , D_0 := D'$, and:
\be
\matrixend := \Wbf_L \cdots \Wbf_1
\label{mf:eq:prod_mat}
\text{\,,}
\ee
referred to as the \emph{end matrix} of the factorization.
Our interest lies on the implicit regularization of gradient descent, \ie~on the type of end matrices (Equation~\eqref{mf:eq:prod_mat}) it will find when applied to the overparameterized objective (Equation~\eqref{mf:eq:oprm_obj}).
Accordingly, and in line with prior work (\cf~\cite{gunasekar2017implicit,arora2019implicit}), we focus on the case in which the search space is unconstrained, meaning $\min \{ D_l \}_{l=0}^L = \min \{ D_0 , D_L \}$ (rank is not limited by the parameterization).

As a theoretical surrogate for gradient descent with small learning rate and near-zero initialization, similarly to~\cite{gunasekar2017implicit} and~\cite{arora2019implicit} (as well as other works analyzing linear neural networks, \eg~\cite{saxe2014exact,arora2018optimization,lampinen2019analytic,arora2019convergence}), we study \emph{gradient flow} (gradient descent with infinitesimally small learning rate):\footnote{
	A technical subtlety of optimization in continuous time is that in principle, it is possible to asymptote (diverge to infinity) after finite time.
	In such a case, the asymptote is regarded as the end of optimization, and time tending to infinity ($t \to \infty$) is to be interpreted as tending towards that point.
}
\be
\dot{\Wbf}_l ( t ) := \tfrac{d}{dt} \Wbf_l ( t ) = - \tfrac{\partial}{\partial \Wbf_l} \phi ( \Wbf_1 ( t ) , \ldots , \Wbf_L ( t ) )
\quad , ~ t \geq 0 ~ , ~ l = 1, \ldots , L 
\label{mf:eq:gf}
\text{\,,}
\ee
and assume \emph{balancedness} at initialization,~\ie:
\be
\Wbf_{l + 1} ( 0 )^\top \Wbf_{l + 1} ( 0 ) = \Wbf_l ( 0 ) \Wbf_l ( 0 )^\top
\quad , ~ l = 1, \ldots , L - 1
\label{mf:eq:balance}
\text{\,.}
\ee
In particular, when considering random initialization, we assume that $\{ \Wbf_l ( 0 ) \}_{l = 1}^L$ are drawn from a joint probability distribution by which Equation~\eqref{mf:eq:balance} holds almost surely.
This is an idealization of standard random near-zero initializations, \eg~Xavier~\cite{glorot2010understanding} and He~\cite{he2015delving}, by which Equation~\eqref{mf:eq:balance} holds approximately with high probability.
The condition of balanced initialization (Equation~\eqref{mf:eq:balance}) played an important role in the analysis of~\cite{arora2018optimization}, facilitating derivation of a differential equation governing the end matrix of a linear neural network (see Lemma~\ref{mf:lem:prod_mat_dyn} in Appendix~\ref{mf:app:lemma:dmf}).
It was shown in~\cite{arora2018optimization} empirically (and will be demonstrated again in Section~\ref{mf:sec:experiments}) that there is an excellent match between the theoretical predictions of gradient flow with balanced initialization, and its practical realization via gradient descent with small learning rate and near-zero initialization.
Other works (\eg,~\cite{arora2019convergence,ji2019gradient}) have supported this match theoretically, and we provided additional support in Appendix~A of \cite{razin2020implicit} by extending our theory to the case of unbalanced initialization (Equation~\eqref{mf:eq:balance}~holding~approximately).

Formally stated, Conjecture~\ref{conj:nuclear_norm} from~\cite{gunasekar2017implicit} treats the case $L = 2$, where the end matrix~$\matrixend$ (Equation~\eqref{mf:eq:prod_mat}) holds~$\alpha \cdot \Wbf_{init}$ at initialization, $\Wbf_{init}$ being a fixed arbitrary full-rank matrix, and $\alpha$ a varying positive scalar.\footnote{
	The formal statement in~\cite{gunasekar2017implicit} applies to symmetric matrix factorization and positive definite~$\Wbf_{init}$, but it is claimed thereafter that affirming the conjecture would imply the same for the asymmetric setting considered here.
	We also note that the conjecture is stated in the context of matrix sensing, thus in particular applies to matrix completion (a special case).
}
Taking time to infinity ($t \to \infty$) and then initialization size to zero ($\alpha \to 0^+$), the conjecture postulates that if the limit end matrix $\matrixendbar := \lim_{\alpha \to 0^+} \lim_{t \to \infty} \matrixend$ exists and is a global optimum for the loss~$\mfendloss ( \cdot )$ (Equation~\eqref{mf:eq:loss}), \ie~$\mfendloss (\matrixendbar ) = 0$, then it will be a global optimum with minimal nuclear norm, meaning $\matrixendbar \in \argmin_{\Wbf : \mfendloss ( \Wbf ) = 0} \norm{\Wbf}_{nuclear}$.

In contrast to Conjecture~\ref{conj:nuclear_norm}, Conjecture~\ref{conj:no_norm} from~\cite{arora2019implicit} can be interpreted as saying that for any depth $L \geq 2$ and any norm or quasi-norm~$\norm{\cdot}$, there exist observations $\{ y_{i , j} \}_{( i , j ) \in \Omega}$ for which global optimization of loss ($\lim_{\alpha \to 0^+} \lim_{t \to \infty} \mfendloss ( \matrixend ) = 0$) does not imply minimization of~$\norm{\cdot}$ (\ie~we may have $\lim_{\alpha \to 0^+} \lim_{t \to \infty} \norm{ \matrixend } \neq \min_{\Wbf : \mfendloss ( \Wbf ) = 0} \norm{\Wbf}$).
Due to technical subtleties (for example the requirement of Conjecture~\ref{conj:nuclear_norm} that a double limit of the end matrix with respect to time and initialization size exists), Conjectures~\ref{conj:nuclear_norm} and~\ref{conj:no_norm} are not necessarily contradictory.
However, they are in direct opposition in terms of the stances they represent~---~one supports the prospect of norms being able to explain implicit regularization in matrix factorization, and the other does not.
The current chapter seeks a resolution.

\section{Implicit Regularization Can Drive All Norms to Infinity}
\label{mf:sec:analysis}

In this section we prove that for matrix factorization of depth~$L \geq 2$, there exist observations $\{ y_{i , j} \}_{( i , j ) \in \Omega}$ with which optimizing the overparameterized objective (Equation~\eqref{mf:eq:oprm_obj}) via gradient flow (Equations \eqref{mf:eq:gf} and~\eqref{mf:eq:balance}) leads~---~with probability~$0.5$ or more over random (``symmetric'') initialization~---~\emph{all} norms and quasi-norms of the end matrix (Equation~\eqref{mf:eq:prod_mat}) to \emph{grow towards infinity}, while its rank essentially decreases towards minimum.
By this we not only affirm Conjecture~\ref{conj:no_norm}, but in fact go beyond it in the following sense:
\emph{(i)}~the conjecture allows chosen observations to depend on the norm or quasi-norm under consideration, while we show that the same set of observations can apply jointly to all norms and quasi-norms;
and
\emph{(ii)}~the conjecture requires norms and quasi-norms to be larger than minimal, while we establish growth towards infinity.

For simplicity of presentation, the current section delivers our construction and analysis in the setting $D = D' = 2$ (\ie~$2$-by-$2$ matrix completion)~---~extension to different dimensions is straightforward (see Appendix~\ref{mf:app:dimensions}).
We begin (\cref{mf:sec:analysis:setting}) by introducing our chosen observations $\{ y_{i , j} \}_{( i , j ) \in \Omega}$ and discussing their properties.
Subsequently (\cref{mf:sec:analysis:norms_up}), we show that with these observations, decreasing loss often increases all norms and quasi-norms while lowering rank.
Minimization of loss is treated thereafter (\cref{mf:sec:analysis:loss_down}).
Finally (\cref{mf:sec:analysis:robust}), robustness of our construction to perturbations is established.

\subsection{A Simple Matrix Completion Problem} \label{mf:sec:analysis:setting}

Consider the problem of completing a $2$-by-$2$ matrix based on the following observations:
\be
\Omega = \{ ( 1 , 2 ) , ( 2 , 1 ) , ( 2 , 2 ) \}
\quad , \quad
y_{1 , 2} = 1
~ , ~
y_{2 , 1} = 1
~ , ~
y_{2 , 2} = 0
\label{mf:eq:obs}
\text{\,.}
\ee
The solution set for this problem (\ie~the set of matrices obtaining zero loss) is:
\be
\S = \left\{ \Wbf \in \R^{2 \times 2} : ( \Wbf )_{1 , 2} = 1 ,  ( \Wbf )_{2 , 1} = 1 , ( \Wbf )_{2 , 2} = 0 \right\}
\label{mf:eq:sol_set}
\text{\,.}
\ee
Proposition~\ref{mf:prop:sol_set_norms} below states that minimizing a norm or quasi-norm along $\Wbf \in \S$ requires confining $( \Wbf )_{1 , 1}$ to a bounded interval, which for Schatten-$p$ (quasi-)norms (in particular for nuclear, Frobenius, and spectral norms)\footnote{
	For $p \,{\in}\, ( 0 , \infty ]$, the \emph{Schatten-$p$ (quasi-)norm} of a matrix $\Wbf \in \R^{D \times D'}$ with singular values $\{ \sigma_r ( \Wbf ) \}_{r = 1}^{\min \{ D , D' \}}$ is defined as $\big( \sum_{r = 1}^{\min \{ D , D' \}} \sigma_r^p ( \Wbf ) \big)^{1 / p}$ if $p < \infty$ and as $\max \{ \sigma_r ( \Wbf ) \}_{r = 1}^{\min \{ D , D' \}}$ if $p = \infty$. 
	It is a norm if $p \geq 1$ and a quasi-norm if $p < 1$.
	Notable special cases are nuclear (trace), Frobenius and spectral norms, corresponding to $p = 1$, $2$, and $\infty$, respectively.
}
is simply the singleton~$\{ 0 \}$.
\begin{proposition} \label{mf:prop:sol_set_norms}
	For any norm or quasi-norm over matrices~$\norm{\cdot}$ and any~$\epsilon > 0$, there exists a bounded interval $I_{\norm{\cdot} , \epsilon} \subset \R$ such that if $\Wbf \in \S$ is an $\epsilon$-minimizer of~$\norm{\cdot}$ (\ie~$\norm{\Wbf} \leq \inf_{\Wbf' \in \S} \norm{\Wbf'} + \epsilon$) then necessarily~$( \Wbf )_{1 , 1} \in I_{\norm{\cdot} , \epsilon}$.
	If~$\norm{\cdot}$ is a Schatten-$p$ (quasi-)norm, then in addition $\Wbf \in \S$ minimizes~$\norm{\cdot}$ (\ie~$\norm{\Wbf} = \inf_{\Wbf' \in \S} \norm{\Wbf'}$) if and only if~$( \Wbf )_{1 , 1} = 0$.
\end{proposition}
\begin{proof}[Proof sketch (proof in Appendix~\ref{mf:app:proofs:sol_set_norms})]
The (weakened) triangle inequality allows lower bounding $\norm{\cdot}$ by $\abs{( \Wbf )_{1 , 1}}$ (up to multiplicative and additive constants).
Thus, the set of $( \Wbf )_{1 , 1}$ values corresponding to $\epsilon$-minimizers must be bounded.
	If $\norm{\cdot}$ is a Schatten-$p$ (quasi-)norm, a straightforward analysis shows it is monotonically increasing with respect to $\abs{( \Wbf )_{1 , 1}}$, implying it is minimized if and only if $( \Wbf )_{1 , 1} = 0$.
\end{proof}

In addition to norms and quasi-norms, we are also interested in the evolution of rank throughout optimization of a deep matrix factorization.
More specifically, we are interested in the prospect of rank being implicitly minimized, as demonstrated empirically in~\cite{gunasekar2017implicit,arora2019implicit}.
The discrete nature of rank renders its direct analysis unfavorable from a dynamical perspective (the rank of a matrix implies little about its proximity to low-rank), thus we consider the following surrogate measures:
\emph{(i)}~\emph{effective rank} (Definition~\ref{def:erank} below; from~\cite{roy2007effective})~---~a continuous extension of rank used for numerical analyses;
and
\emph{(ii)}~\emph{distance from infimal rank} (Definition~\ref{def:irank_dist} below)~---~(Frobenius) distance from the minimal rank that a given set of matrices may approach.

According to Proposition~\ref{mf:prop:sol_set_rank} below, these measures independently imply that, although all solutions to our matrix completion problem~---~\ie~all $\Wbf \in \S$ (see Equation~\eqref{mf:eq:sol_set})~---~have rank~$2$, it is possible to essentially minimize the rank to~$1$ by taking $\abs{( \Wbf )_{1 , 1}} \to \infty$.
Recalling Proposition~\ref{mf:prop:sol_set_norms}, we conclude that in our setting, there is a direct contradiction between minimizing norms or quasi-norms and minimizing rank~---~the former requires confinement to some bounded interval, whereas the latter demands divergence towards infinity.
This is the critical feature of our construction, allowing us to deem whether the implicit regularization in deep matrix factorization favors norms (or quasi-norms) over rank or vice versa.
\begin{definition}[from~\cite{roy2007effective}] \label{def:erank}
	The \emph{effective rank of a matrix $0 \neq \Wbf  \,{\in}\, \R^{D \times D'}$} with singular values $\{ \sigma_r ( \Wbf ) \}_{r = 1}^{\min \{ D , D' \}}$ is defined to be:
	\[
	\erank ( \Wbf ) :=  \exp \brk*{  H \brk*{ \rho_1 ( \Wbf ) , \ldots , \rho_{\min \{ D , D' \}} ( \Wbf ) } }
	\text{\,,}
	\]
	where $\{ \rho_r ( \Wbf ) := \nicefrac{\sigma_r ( \Wbf )}{\sum_{r' = 1}^{\min \{ D , D' \}} \sigma_{r'} ( \Wbf )} \}_{r = 1}^{\min \{ D , D' \}}$ is a distribution induced by the singular values, and $H ( \rho_1 ( \Wbf ), \ldots , \rho_{\min \{ D , D' \}} ( \Wbf ) ):= - \sum\nolimits_{r = 1}^{\min \{ D , D' \}} \rho_r ( \Wbf ) \cdot \ln \rho_r ( \Wbf )$ is its Shannon entropy (by convention $0 \cdot \ln 0 = 0$).
\end{definition}

\begin{definition} \label{def:irank_dist}
	We denote by $\frodist ( \S , \S' ) := \inf \{ \norm{\Wbf - \Wbf'}_{Fro} : \Wbf \in \S , \Wbf' \in \S' \}$ the (Frobenius) distance between $\S , \S' \subset \R^{D \times D'}$, by $\frodist ( \Wbf , \S' )  := \inf \{ \norm{W - \Wbf'}_{Fro} : \Wbf' \in \S' \}$ the distance between $\Wbf \in \R^{D \times D'}$ and~$\S'$, and by $\M_R  := \{ \Wbf \in \R^{D \times D'} : \rank ( \Wbf ) \leq R \}$, for $R = 0, \ldots , \min \{ D , D' \}$, the set of matrices with rank~$R$ or less.
	The \emph{infimal rank of the set~$\S$}, denoted $\irank ( \S )$, is defined to be the minimal~$R$ such that $\frodist ( \S , \M_R ) = 0$.
	The \emph{distance of a matrix $\Wbf \in \R^{D \times D'}$ from the infimal rank of~$\S$} is defined to be $\frodist ( \Wbf , \M_{\irank ( \S )} )$.
\end{definition}

\begin{proposition} \label{mf:prop:sol_set_rank}
	The effective rank (Definition~\ref{def:erank}) takes the values~$( 1 , 2 ]$ along~$\S$ (Equation~\eqref{mf:eq:sol_set}).
	For $\Wbf \in \S$, it is maximized when $( \Wbf )_{1 , 1} = 0$, and monotonically decreases to~$1$ as $\abs{( \Wbf )_{1 , 1}}$ grows.
	Correspondingly, the infimal rank (Definition~\ref{def:irank_dist}) of~$\S$ is~$1$, and the distance of $\Wbf \in \S$ from this infimal rank is maximized when $( \Wbf )_{1 , 1} = 0$, monotonically decreasing to~$0$ as $\abs{( \Wbf )_{1 , 1}}$ grows.
\end{proposition}
\begin{proof}[Proof sketch (proof in Appendix~\ref{mf:app:proofs:sol_set_rank})]
Analyzing the singular values of $\Wbf \in \S$~---~$\sigma_1 (\Wbf) \geq \sigma_2 (\Wbf) \geq 0$~---~reveals that:
\emph{(i)}~$\sigma_1 (\Wbf)$ attains a minimal value of~$1$ when $(\Wbf)_{1 , 1} = 0$, monotonically increasing to~$\infty$ as $\abs{( \Wbf )_{1 , 1}}$ grows;
and	\emph{(ii)}~$\sigma_2 (\Wbf)$ attains a maximal value of~$1$ when $(\Wbf)_{1 , 1} = 0$, monotonically decreasing to~$0$ as $\abs{( \Wbf )_{1 , 1}}$ grows.
The results for effective rank, infimal rank and distance from infimal rank readily follow from this characterization.
\end{proof}

\subsection{Decreasing Loss Increases Norms} \label{mf:sec:analysis:norms_up}

Consider the process of solving our matrix completion problem (\cref{mf:sec:analysis:setting}) with gradient flow over a depth $L \geq 2$ matrix factorization (\cref{mf:sec:model}).
Theorem~\ref{mf:thm:norms_up_finite} below states that if the end matrix (Equation~\eqref{mf:eq:prod_mat}) has positive determinant at initialization, lowering the loss leads norms and quasi-norms to increase, while the rank essentially decreases.
\begin{theorem} \label{mf:thm:norms_up_finite}
	Suppose we complete the observations in Equation~\eqref{mf:eq:obs} by employing a depth $L \geq 2$ matrix factorization, \ie~by minimizing the overparameterized objective (Equation~\eqref{mf:eq:oprm_obj}) via gradient flow (Equations~\eqref{mf:eq:gf} and~\eqref{mf:eq:balance}).
	Denote by~$\matrixend ( t )$ the end matrix (Equation~\eqref{mf:eq:prod_mat}) at time~$t \geq 0$ of optimization, and by $\mfendloss ( t ) := \mfendloss ( \matrixend ( t ) )$ the corresponding loss (Equation~\eqref{mf:eq:loss}).
	Assume that $\det (\matrixend ( 0 ) ) > 0$.
	Then, for any norm or quasi-norm over matrices~$\norm{\cdot}$:
	\be
	\norm{ \matrixend ( t )} \geq a_{\norm{\cdot}} \cdot \frac{1}{\sqrt{\mfendloss (t)}} -  b_{\norm{ \cdot }}
	\quad , ~ t \geq 0
	\label{mf:eq:norms_lb}
	\text{\,,}
	\ee
	where $b_{\norm{\cdot}} := \max \{ \sqrt{2} a_{\norm{\cdot}} , 8c_{\norm{\cdot}}^2 \max_{i,j \in \{1,2\}} \norm{\ebf_i \ebf_j^\top} \}$, $a_{\norm{\cdot}} := \norm{\ebf_1 \ebf_1^\top} / ( \sqrt{2} c_{\norm{\cdot}} )$, the vectors $\ebf_1, \ebf_2 \in \R^2$ form the standard basis, and $c_{\norm{\cdot}} \geq 1$ is a constant with which $\norm{\cdot}$ satisfies the weakened triangle inequality (see Footnote~\ref{note:qnorm}).
	On the other hand:
	\bea
	& \erank ( \matrixend ( t ) ) \leq \inf_{\Wbf' \in \S} \erank (\Wbf') + \frac{2\sqrt{12}}{\ln (2)} \cdot \sqrt{\mfendloss (t)}
	& \quad , ~ t \geq 0
	\label{mf:eq:erank_ub}
	\text{\,,}
	\\[1.5mm]
	& \frodist ( \matrixend ( t ) , \M_{\irank ( \S )} ) \leq 3\sqrt{2} \cdot \sqrt{ \mfendloss (t)}
	& \quad , ~ t \geq 0
	\label{mf:eq:irank_dist_ub}
	\text{\,,}
	\eea
	where $\erank ( \cdot )$ stands for effective rank (Definition~\ref{def:erank}), and $D ( \cdot \hspace{0.5mm} , \M_{\irank ( \S )} )$ represents distance from the infimal rank (Definition~\ref{def:irank_dist}) of the solution set~$\S$ (Equation~\eqref{mf:eq:sol_set}).
\end{theorem}
\begin{proof}[Proof sketch (proof in Appendix~\ref{mf:app:proofs:norms_up_finite})]
Using a dynamical characterization for the singular values of the end matrix, developed in \cite{arora2019implicit} (restated in Appendix~\ref{mf:app:lemma:dmf} as Lemma~\ref{mf:lem:prod_mat_sing_dyn}), we show that the latter's determinant does not change sign, \ie~it remains positive.
This allows us to lower bound $\abs{( \matrixend )_{1, 1} (t)}$ by $1 / \sqrt{\mfendloss (t)}$ (up to multiplicative and additive constants).
Relating $\abs{ (\matrixend )_{1, 1} (t)}$ to (quasi-)norms, effective rank and distance from infimal rank then leads to the desired bounds.
\end{proof}

An immediate consequence of Theorem~\ref{mf:thm:norms_up_finite} is that, if the end matrix (Equation~\eqref{mf:eq:prod_mat}) has positive determinant at initialization, convergence to zero loss leads \emph{all} norms and quasi-norms to \emph{grow to infinity}, while the rank is essentially minimized.
This is formalized in Corollary~\ref{mf:cor:norms_up_asymp} below.
\begin{corollary} \label{mf:cor:norms_up_asymp}
	Under the conditions of Theorem~\ref{mf:thm:norms_up_finite}, if the loss is globally optimized, \ie~if $\lim_{t \to \infty} \mfendloss ( t ) = 0$, then for any norm or quasi-norm over matrices~$\norm{\cdot}$:
	\[
	\lim\nolimits_{t \to \infty} \norm{\matrixend ( t )} = \infty
	\text{\,,}
	\]
	where $\matrixend ( t )$ is the end matrix of the deep factorization (Equation~\eqref{mf:eq:prod_mat}) at time~$t$ of optimization.
	On the other hand:
	\[
	\lim\nolimits_{t \to \infty} \erank ( \matrixend ( t ) ) = \inf\nolimits_{\Wbf' \in \S} \erank (\Wbf')
	~~~ , ~~~
	\lim\nolimits_{t \to \infty} \frodist ( \matrixend ( t ) , \M_{\irank ( \S )} ) = 0
	\text{\,,}
	\]
	where $\erank ( \cdot )$ stands for effective rank (Definition~\ref{def:erank}), and $D ( \cdot \hspace{0.5mm} , \M_{\irank ( \S )} )$ represents distance from the infimal rank (Definition~\ref{def:irank_dist}) of the solution set~$\S$ (Equation~\eqref{mf:eq:sol_set}).
\end{corollary}
\begin{proof}
	Taking the limit $\mfendloss ( t ) \to 0$ in the bounds given by Theorem~\ref{mf:thm:norms_up_finite} establishes the results.
\end{proof}

Theorem~\ref{mf:thm:norms_up_finite} and Corollary~\ref{mf:cor:norms_up_asymp} imply that in our setting (\cref{mf:sec:analysis:setting}), where minimizing norms (or quasi-norms) and minimizing rank contradict each other, the implicit regularization of deep matrix factorization is willing to completely give up on the former in favor of the latter, at least on the condition that the end matrix (Equation~\eqref{mf:eq:prod_mat}) has positive determinant at initialization.
How probable is this condition?
By Proposition~\ref{mf:prop:det_pos} below, it holds with probability~$0.5$ if the end matrix is initialized by any one of a wide array of common distributions, including matrix Gaussian distribution with zero mean and independent entries, and a product of such.
We note that rescaling (multiplying by~$\alpha > 0$) initialization does not change the sign of end matrix's determinant, therefore as postulated by Conjecture~\ref{conj:no_norm}, initialization close to the origin does \emph{not} ensure convergence to solution with minimal norm or quasi-norm.
\begin{proposition} \label{mf:prop:det_pos}
	If $\Wbf \in \R^{D \times D}$ is a random matrix whose entries are drawn independently from continuous distributions, each symmetric about the origin, then $\Pr ( \det ( \Wbf ) > 0 ) = \Pr ( \det ( \Wbf ) < 0 ) = 0.5$.
	Furthermore, for $L \in \N$, if $\Wbf_1 , \ldots , \Wbf_L \in \R^{D \times D}$ are random matrices drawn independently from continuous distributions, and there exists $l \in \{ 1 , \ldots , L \}$ with $\Pr ( \det ( \Wbf_l ) > 0 ) = 0.5$, then $\Pr ( \det ( \Wbf_L \cdots \Wbf_1 ) > 0 ) = \Pr ( \det ( \Wbf_L \cdots \Wbf_1 ) < 0 ) = 0.5$.
\end{proposition}
\begin{proof}[Proof sketch (proof in~\cref{mf:app:proofs:det_pos})]
Multiplying a row of $\Wbf$ by $-1$ keeps its distribution intact while flipping the sign of its determinant.
This implies $\Pr ( \det ( \Wbf ) > 0 ) = \Pr ( \det ( \Wbf ) < 0 )$.
The first result then follows from the fact that a matrix drawn from a continuous distribution is almost surely non-singular.
The second result is an outcome of the same fact, as well as the multiplicativity of determinant and the law of total probability.
\end{proof}

\subsection{Convergence to Zero Loss} \label{mf:sec:analysis:loss_down}

It is customary in the theory of deep learning (\cf~\cite{gunasekar2017implicit,gunasekar2018implicit,arora2019implicit}) to distinguish between implicit regularization~---~which concerns the type of solutions found in training~---~and the complementary question of whether training loss is globally optimized.
We supplement our implicit regularization analysis (\cref{mf:sec:analysis:norms_up}) by addressing this complementary question in two ways:
\emph{(i)}~in \cref{mf:sec:experiments} we empirically demonstrate that on the matrix completion problem we analyze (\cref{mf:sec:analysis:setting}), gradient descent over deep matrix factorizations (\cref{mf:sec:model}) indeed drives training loss towards global optimum, \ie~towards zero;
and
\emph{(ii)}~in Proposition~\ref{mf:prop:loss_down} below we theoretically establish convergence to zero loss for the special case of depth~$2$ and scaled identity initialization (treatment of additional depths and initialization schemes is left for future work).
We note that when combined with Corollary~\ref{mf:cor:norms_up_asymp}, Proposition~\ref{mf:prop:loss_down} affirms that in the latter special case, all norms and quasi-norms indeed grow to infinity while rank is essentially minimized.\footnote{
	Notice that under (positively) scaled identity initialization the determinant of the end matrix (Equation~\eqref{mf:eq:prod_mat}) is positive, as required by Corollary~\ref{mf:cor:norms_up_asymp}.
}
\begin{proposition} \label{mf:prop:loss_down}
	Consider the setting of Theorem~\ref{mf:thm:norms_up_finite} in the special case of depth $L = 2$ and initial end matrix (Equation~\eqref{mf:eq:prod_mat}) $\matrixend ( 0 ) = \alpha \cdot \Ibf$, where $\Ibf$ stands for the identity matrix and $\alpha \in ( 0 , 1 ]$.
	Under these conditions $\lim_{t \to \infty} \mfendloss ( t ) = 0$, \ie~the training loss is globally optimized.
\end{proposition}
\begin{proof}[Proof sketch (proof in~\cref{mf:app:proofs:loss_down})]
We first establish that the end matrix is positive definite for all $t$.
This simplifies a dynamical characterization from~\cite{arora2018optimization} (restated as Lemma~\ref{mf:lem:prod_mat_dyn} in \cref{mf:app:lemmas}), yielding lucid differential equations governing the entries of the end matrix.
Careful analysis of these equations then completes the proof.
\end{proof}

\subsection{Robustness to Perturbations} \label{mf:sec:analysis:robust}

Our analysis (\cref{mf:sec:analysis:norms_up}) has shown that when applying a deep matrix factorization (\cref{mf:sec:model}) to the matrix completion problem defined in \cref{mf:sec:analysis:setting}, if the end matrix (Equation~\eqref{mf:eq:prod_mat}) has positive determinant at initialization~---~a condition that holds with probability~$0.5$ under the wide variety of random distributions specified by Proposition~\ref{mf:prop:det_pos}~---~then the implicit regularization drives \emph{all} norms and quasi-norms \emph{towards infinity}, while rank is essentially driven towards its minimum.
A natural question is how common this phenomenon is, and in particular, to what extent does it persist if the observed entries we defined (Equation~\eqref{mf:eq:obs}) are perturbed.

Theorem~\ref{mf:thm:norms_up_finite_perturb} below generalizes Theorem~\ref{mf:thm:norms_up_finite} (from \cref{mf:sec:analysis:norms_up}) to the case of arbitrary non-zero values for the off-diagonal observations~$y_{1 , 2} , y_{2 , 1}$, and an arbitrary value for the diagonal observation~$y_{2 , 2}$.
In this generalization, the assumption (from Theorem~\ref{mf:thm:norms_up_finite}) of the end matrix's determinant at initialization being positive is modified to an assumption of it having the same sign as~$y_{1 , 2} \cdot y_{2 , 1}$ (the probability of which is also~$0.5$ under the random distributions covered by Proposition~\ref{mf:prop:det_pos}).
Conditioned on the modified assumption, the smaller $\abs{y_{2 , 2}}$ is compared to~$\abs{y_{1 , 2} \cdot y_{2 , 1}}$, the higher the implicit regularization is guaranteed to drive norms and quasi-norms, and the lower it is guaranteed to essentially drive the rank.
Two immediate implications of Theorem~\ref{mf:thm:norms_up_finite_perturb} are:
\emph{(i)}~if the diagonal observation is unperturbed ($y_{2 , 2} = 0$), the off-diagonal ones ($y_{1 , 2} , y_{2 , 1}$) can take on \emph{any} non-zero values, and the phenomenon of implicit regularization driving norms and quasi-norms towards infinity (while essentially driving rank towards its minimum) will persist;
and
\emph{(ii)}~this phenomenon gracefully recedes as the diagonal observation is perturbed away from zero.
We note that Theorem~\ref{mf:thm:norms_up_finite_perturb} applies even if the unobserved entry is repositioned, thus our construction is robust not only to perturbations in observed values, but also to an arbitrary change in the observed locations.
See \cref{mf:app:experiments:further} for empirical demonstrations.
\begin{theorem} \label{mf:thm:norms_up_finite_perturb}
	Consider the setting of Theorem~\ref{mf:thm:norms_up_finite} subject to the following changes:
	\emph{(i)}~the observations from Equation~\eqref{mf:eq:obs} are generalized to:
	\be
	\Omega = \{ ( 1 , 2 ) , ( 2 , 1 ) , ( 2 , 2 ) \}
	~~~ , ~~~
	y_{1 , 2} = z \in \R {\setminus} \{ 0 \}
	~ , ~
	y_{2 , 1} = z' \in \R {\setminus} \{ 0 \}
	~ , ~
	y_{2 , 2} = \epsilon \in \R
	\label{mf:eq:obs_perturb}
	\text{\,,}
	\ee
	leading to the following solution set in place of that from Equation~\eqref{mf:eq:sol_set}:
	\be
	\widetilde{\S} = \left\{ \Wbf \in \R^{2 \times 2} : ( \Wbf )_{1 , 2} = z ,  ( \Wbf )_{2 , 1} = z' , ( \Wbf )_{2 , 2} = \epsilon \right\}
	\label{mf:eq:sol_set_perturb}
	\text{\,;}
	\ee
	and
	\emph{(ii)}~the assumption $\det ( \matrixend ( 0 ) ) > 0$ is generalized to $\sign ( \det ( \matrixend ( 0 ) ) ) = \sign ( z \cdot z' )$, where $\matrixend ( t )$ denotes the end matrix (Equation~\eqref{mf:eq:prod_mat}) at time~$t \geq 0$ of optimization.
	Under these conditions, for any norm or quasi-norm over matrices~$\norm{\cdot}$:
	\be
	\norm{ \matrixend ( t )} \geq a_{\norm{\cdot}} \cdot \frac{\abs{z} \cdot |z'|}{\abs{ \epsilon } + \sqrt{2 \mfendloss (t) } } - b_{\norm{\cdot}}
	\quad , ~ t \geq 0
	\label{mf:eq:norms_lb_perturb}
	\text{\,,}
	\ee
	where $b_{\norm{\cdot}} \! := \! \max \! \big \{ \nicefrac{ a_{\norm{\cdot}} \cdot \abs{z} \cdot |z'|}{\left ( \abs{ \epsilon } + \min \{ \abs{z}, |z'| \} \right )} \, , \, 8 c_{\norm{\cdot}}^2 \max \{ \abs{z}, |z'|, \abs{\epsilon} \} \max_{ i,j \in \{1,2\}  } \norm{ \ebf_i \ebf_j^\top } \big \}$, $a_{\norm{\cdot}} := \norm{\ebf_1 \ebf_1^\top} / c_{\norm{\cdot}}$,
	the vectors $\ebf_1, \ebf_2 \in \R^2$ form the standard basis, and $c_{\norm{\cdot}} \geq 1$ is a constant with which $\norm{\cdot}$ satisfies the weakened triangle inequality (see Footnote~\ref{note:qnorm}).
	On the other hand:
	\bea
	&\hspace{-13mm}  \erank ( \matrixend ( t ) ) \leq \inf_{\Wbf' \in \widetilde{S}} \erank(\Wbf') + \frac{16}{\min \{ \abs{z} , |z'| \} } \Big ( \abs{\epsilon} + \sqrt{2\mfendloss (t)} \Big )
	& , ~ t \geq 0
	\label{mf:eq:erank_ub_perturb}
	\text{\,,}
	\\[1mm]
	& \hspace{-13mm} \frodist ( \matrixend ( t ) , \M_{\irank ( \widetilde{\S} )} ) \leq 4 \abs{\epsilon} + \Big (4 + \frac{ \sqrt{ \abs{z} \cdot |z'| } }{\min \{\abs{z} , |z'| \}} \Big ) \sqrt{2 \mfendloss (t)}
	&  , ~ t \geq 0
	\label{mf:eq:irank_dist_ub_perturb}
	\text{\,,}
	\eea
	where $\erank ( \cdot )$ stands for effective rank (Definition~\ref{def:erank}), and $D ( \cdot \hspace{0.5mm} , \M_{\irank ( \widetilde{\S} )} )$ represents distance from the infimal rank (Definition~\ref{def:irank_dist}) of the solution set~\smash{$\widetilde{\S}$}.
	Moreover, Equations~\eqref{mf:eq:norms_lb_perturb}, \eqref{mf:eq:erank_ub_perturb} and~\eqref{mf:eq:irank_dist_ub_perturb} hold even if the above setting is further generalized as follows:
	\emph{(i)}~the unobserved entry resides in location $( i , j ) \in \{ 1, 2 \} \times \{ 1, 2 \}$, with $z , z' \in \R {\setminus} \{ 0 \}$ observed in the adjacent locations and $\epsilon \in \R$ in the diagonally-opposite one;
	and
	\emph{(ii)}~the sign of $\det ( \matrixend ( 0 ) )$ is equal to that of $z \cdot z'$ if $i = j$, and opposite to it otherwise.
\end{theorem}
\begin{proof}[Proof sketch (proof in~\cref{mf:app:proofs:norms_up_finite_perturb})]
	The proof follows a line similar to that of Theorem~\ref{mf:thm:norms_up_finite}, with slightly more involved derivations.
\end{proof}

\textbf{Disqualifying implicit minimization of norms in finite settings.}~~Theorem~\ref{mf:thm:norms_up_finite} (in \cref{mf:sec:analysis:norms_up}), which applies to our original (unperturbed) matrix completion problem (\cref{mf:sec:analysis:setting}), has shown that the implicit regularization of deep matrix factorization does not minimize norms or quasi-norms, by establishing lower bounds (Equation~\eqref{mf:eq:norms_lb} with~$\norm{\cdot}$ ranging over all possible norms and quasi-norms) that tend to infinity as the training loss~$\mfendloss ( t )$ converges to zero.
The more general Theorem~\ref{mf:thm:norms_up_finite_perturb} allows disqualifying implicit minimization of norms or quasi-norms without requiring divergence to infinity.
To see this, consider the lower bounds it establishes (Equation~\eqref{mf:eq:norms_lb_perturb} with~$\norm{\cdot}$ ranging over all norms and quasi-norms), and in particular their limits as $\mfendloss ( t ) \to 0$. 
If the observed value~$\epsilon$ is different from zero, these limits are all finite.
Moreover, given a particular norm or quasi-norm~$\norm{\cdot}$, we may choose~$\epsilon$ different from zero, yet small enough such that the lower bound for~$\norm{\cdot}$ has limit arbitrarily larger than the infimum of~$\norm{\cdot}$ over the solution set.

\section{Experiments}
\label{mf:sec:experiments}

This section presents our empirical evaluations.
We begin in Section~\ref{mf:sec:experiments:analyzed} with deep matrix factorization (Section~\ref{mf:sec:model}) applied to the settings we analyzed (Section~\ref{mf:sec:analysis}).
Then, we turn to Section~\ref{mf:sec:experiments:tensor} and experiment with an extension to tensor (multi-dimensional array) factorization.
For brevity, many details behind our implementation, as well as some experiments, are deferred to Appendix~\ref{mf:app:experiments}.

\subsection{Analyzed Settings} \label{mf:sec:experiments:analyzed}

\begin{figure}[t]
	\begin{center}
		\hspace*{-2mm}
		\subfloat{
			\includegraphics[width=0.31\textwidth]{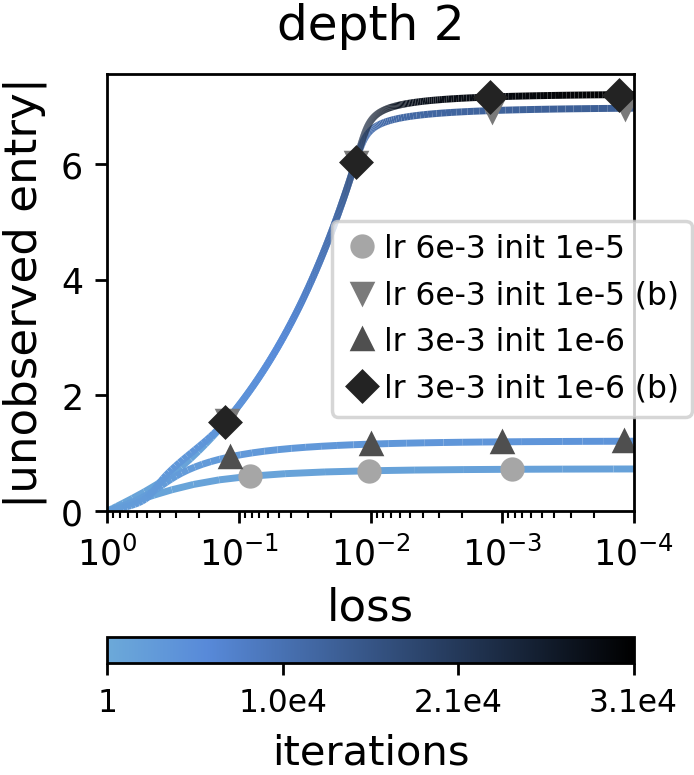}
		}
		~
		\subfloat{
			\includegraphics[width=0.31\textwidth]{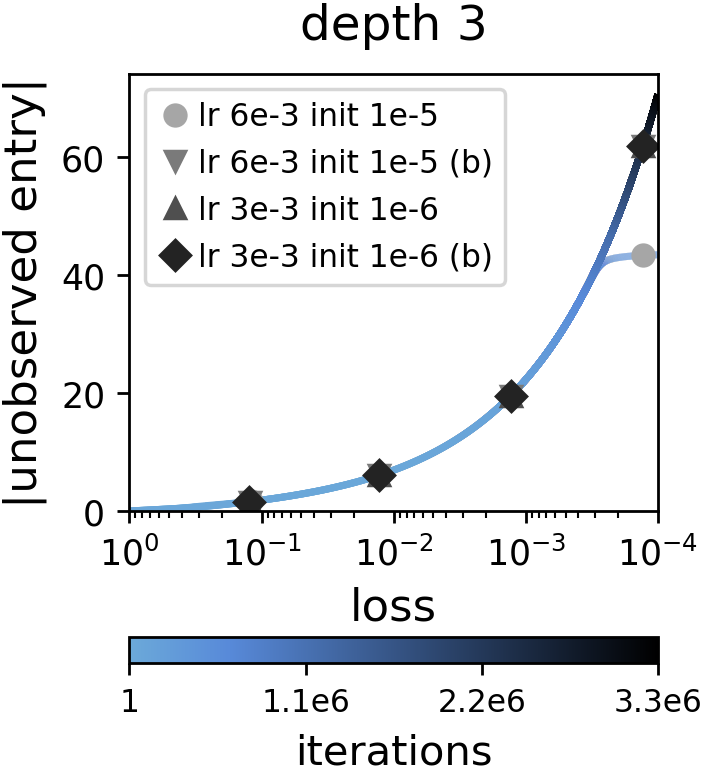}
		}
		~
		\subfloat{
			\includegraphics[width=0.31\textwidth]{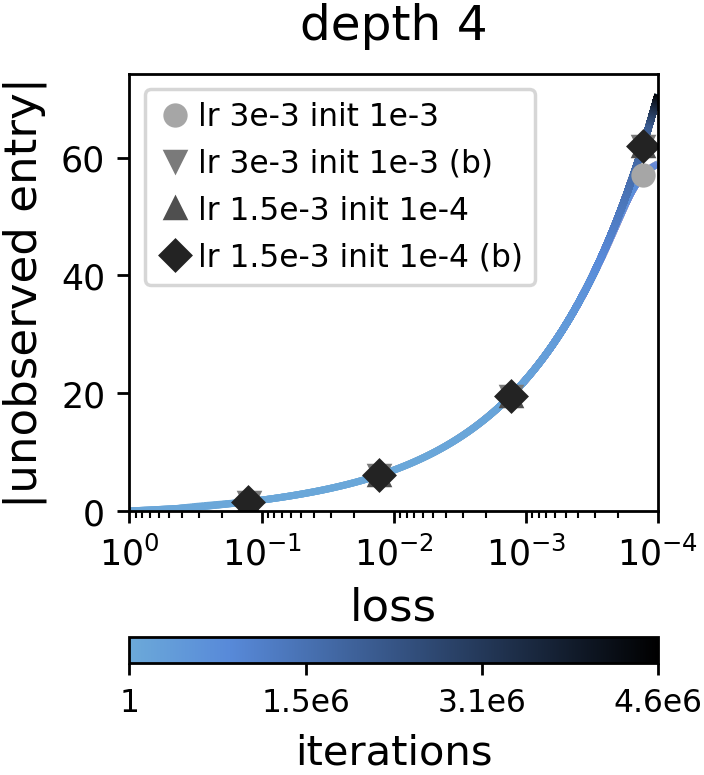}
		}
	\end{center}
	\vspace{-2mm}
	\caption{
		Implicit regularization in matrix factorization can drive \emph{all} norms (and quasi-norms) \emph{towards infinity}.
		For the matrix completion problem defined in Section~\ref{mf:sec:analysis:setting}, our analysis (Section~\ref{mf:sec:analysis:norms_up}) implies that with small learning rate and initialization close to the origin, when the product matrix (Equation~\eqref{mf:eq:prod_mat}) is initialized to have positive determinant, gradient descent on a matrix factorization leads absolute value of unobserved entry to increase (which in turn means norms and quasi-norms increase) as loss decreases, \ie~as observations are fit.
		This is demonstrated in the plots above, which for representative runs, show absolute value of unobserved entry as a function of the loss (Equation~\ref{mf:eq:loss}), with iteration number encoded by color.
		Each plot corresponds to a different depth for the matrix factorization, and presents runs with varying configurations of learning rate and initialization (abbreviated as ``lr'' and ``init'', respectively).
		Both balanced (Equation~\ref{mf:eq:balance}) and unbalanced (layer-wise independent) random initializations were evaluated (former is marked by~``(b)'').
		Independently for each depth, runs were iteratively carried out, with both learning rate and standard deviation for initialization decreased after each run, until the point where further reduction did not yield a noticeable change (presented runs are those from the last iterations of this process).
		Notice that depth, balancedness, and small learning rate and initialization, all contribute to the examined effect (absolute value of unobserved entry increasing as loss decreases), with the transition from depth~$2$ to $3$~or~more being most significant.
		Notice also that all runs initially follow the same curve, differing from one another in the point at which they divert (enter a phase where examined effect is lesser).
		A complete investigation of these phenomena is left for future work.
		For further implementation details, and similar experiments with different matrix dimensions, as well as perturbed and repositioned observations, see Appendix~\ref{mf:app:experiments}.
	}
	\label{mf:fig:experiment_dmf}
\end{figure}

In~\cite{gunasekar2017implicit}, Gunasekar~\etal~experimented with matrix factorization, arriving at Conjecture~\ref{conj:nuclear_norm}.
In the following work~\cite{arora2019implicit}, Arora~\etal~empirically evaluated additional settings, ultimately arguing against Conjecture~\ref{conj:nuclear_norm}, and raising Conjecture~\ref{conj:no_norm}.
Our analysis (Section~\ref{mf:sec:analysis}) affirmed Conjecture~\ref{conj:no_norm}, by providing a setting in which gradient descent (with infinitesimally small learning rate and initialization arbitrarily close to the origin) over (shallow or deep) matrix factorization provably drives \emph{all} norms (and quasi-norms) \emph{towards infinity}.
Specifically, we established that running gradient descent on the overparameterized matrix completion objective in Equation~\eqref{mf:eq:oprm_obj}, where the observed entries are those defined in Equation~\eqref{mf:eq:obs}, leads the unobserved entry to diverge to infinity as loss converges to zero.
Figure~\ref{mf:fig:experiment_dmf} demonstrates this phenomenon empirically.
Figures~\ref{mf:fig:experiment_dmf_diff_dimensions} and~\ref{mf:fig:experiment_dmf_perturb} in Appendix~\ref{mf:app:experiments:further} extend the experiment by considering, respectively:
different matrix dimensions (see Appendix~\ref{mf:app:dimensions});
and
perturbations and repositionings applied to observations (\cf~Section~\ref{mf:sec:analysis:robust}).
The figures confirm that the inability of norms (and quasi-norms) to explain implicit regularization in matrix factorization translates from theory to practice.

\subsection{From Matrix to Tensor Factorization} \label{mf:sec:experiments:tensor}

At the heart of our analysis (Section~\ref{mf:sec:analysis}) lies a matrix completion problem whose solution set (Equation~\eqref{mf:eq:sol_set}) entails a direct contradiction between minimizing norms (or quasi-norms) and minimizing rank.
We have shown that on this problem, gradient descent over (shallow or deep) matrix factorization is willing to completely give up on the former in favor of the latter.
This suggests that, rather than viewing implicit regularization in matrix factorization through the lens of norms (or quasi-norms), a potentially more useful interpretation is \emph{minimization of rank}.
Indeed, while global minimization of rank is in the worst case computationally hard (\cf~\cite{recht2011null}), it has been shown in~\cite{arora2019implicit} (theoretically as well as empirically) that the dynamics of gradient descent over matrix factorization promote sparsity of singular values, and thus they may be interpreted as searching for low rank locally.
As a step towards assessing the generality of this interpretation, we empirically explore an extension of matrix factorization to \emph{tensor factorization}.\footnote{
	There exist many types of tensor factorizations (\cf~\cite{kolda2009tensor,hackbusch2012tensor}).
	We treat here the classic and most basic one, known as \emph{CANDECOMP/PARAFAC (CP)}.
	\label{note:cp_decomp}
}

In the context of matrix completion, (depth~$2$) matrix factorization amounts to optimizing the loss in Equation~\eqref{mf:eq:loss} by applying gradient descent to the parameterization $\Wbf = \sum_{r = 1}^R \wbf_r \otimes \wbf'_r$, where $R \in \N$ is a predetermined constant, $\otimes$ stands for tensor product (\ie~outer product), and $\{ \wbf_r \in \R^D \}_{r = 1}^R , \{ \wbf'_r \in \R^{D'} \}_{r = 1}^R$ are the optimized parameters.\footnote{
	To see that this parameterization is equivalent to the usual form $\Wbf = \Wbf_2 \Wbf_1$, simply view $R$ as the dimension shared between $\Wbf_1$ and~$\Wbf_2$, $\{ \wbf_r \}_{r = 1}^R$ as the columns of~$\Wbf_2$, and $\{ \wbf'_r \}_{r = 1}^R$ as the rows of~$\Wbf_1$.
}
The minimal~$R$ required for this parameterization to be able to express a given~$\widebar{\Wbf} \in \R^{D \times D'}$ is precisely the latter's rank.
Implicit regularization towards low rank means that even when $R$ is large enough for expressing any matrix (\ie~$R \geq \min \{ D , D' \}$), solutions expressible (or approximable) with small~$R$ tend to be learned.

\begin{figure}[t]
	\begin{center}
		\hspace*{-3.5mm}
		\subfloat{
			\includegraphics[width=0.48\textwidth]{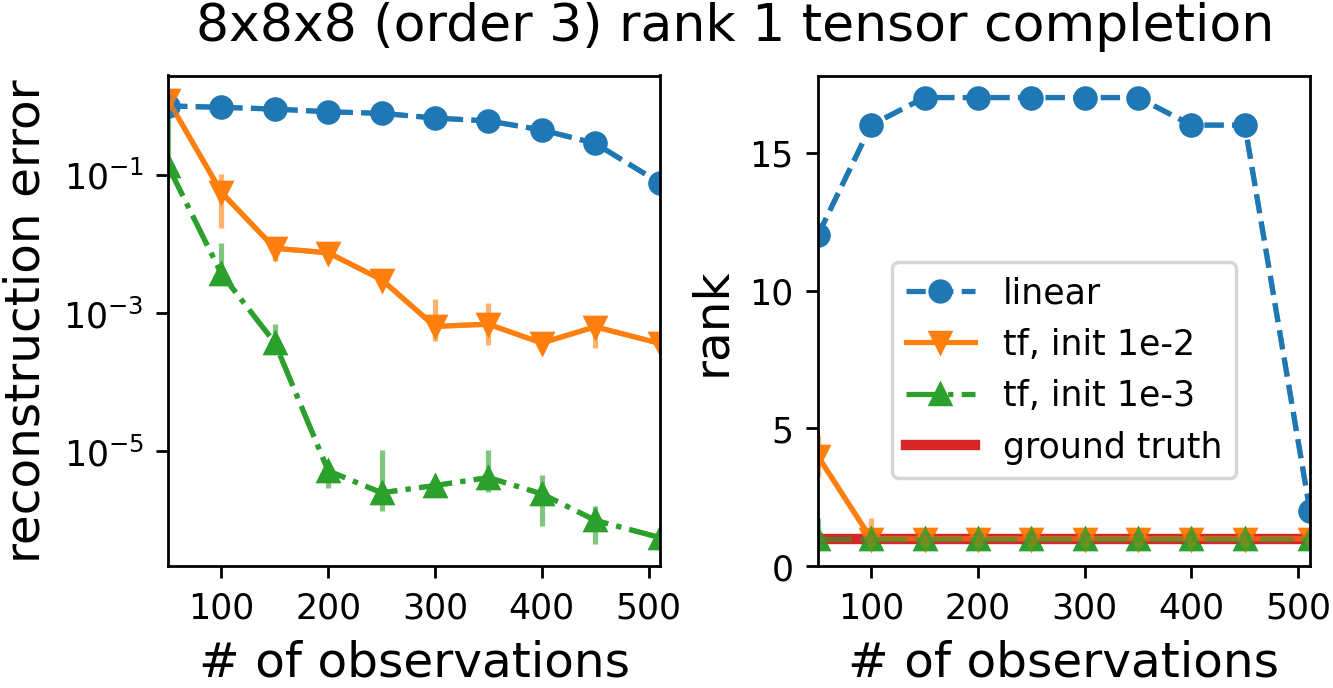}
		} ~~
		\subfloat{
			\includegraphics[width=0.48\textwidth]{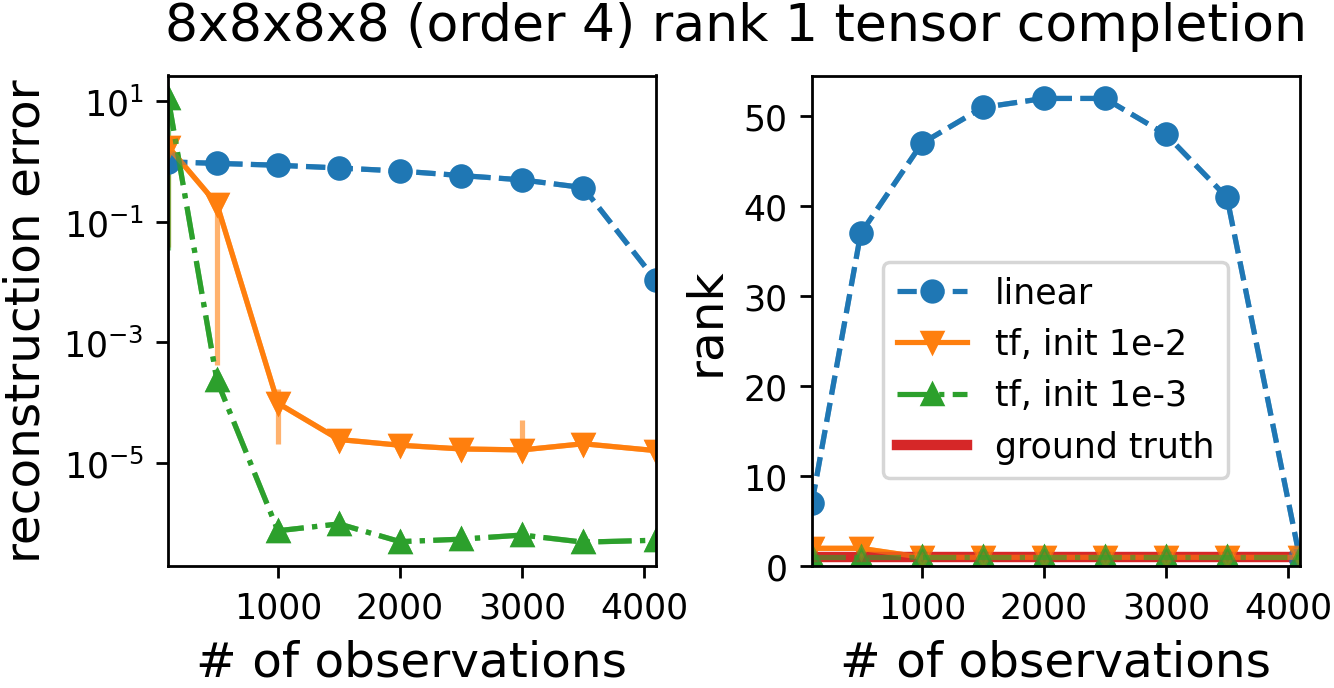}
		}
	\end{center}
	\caption{
		Gradient descent over tensor factorization exhibits an implicit regularization towards low tensor rank.
		Plots above report results of tensor completion experiments, comparing:
		\emph{(i)}~minimization of loss (Equation~\eqref{mf:eq:loss_tensor}) via gradient descent over tensor factorization (Equation~\eqref{mf:eq:tf} with~$R$ large enough for expressing any tensor) starting from (small) random initialization (method is abbreviated as~``tf'');
		against
		\emph{(ii)}~trivial baseline that matches observations while holding zeros in unobserved locations~---~equivalent to minimizing loss via gradient descent over linear parameterization (\ie~directly over~$\W$) starting from zero initialization (hence this method is referred to as ``linear'').
		Each pair of plots corresponds to a randomly drawn low-rank ground truth tensor, from which multiple sets of observations varying in size were randomly chosen.
		The ground truth tensors corresponding to left and right pairs both have rank~$1$ (for results obtained with additional ground truth ranks see Figure~\ref{mf:fig:experiment_tf_r3} in Appendix~\ref{mf:app:experiments:further}), with sizes $8$-by-$8$-by-$8$ (order~$3$) and $8$-by-$8$-by-$8$-by-$8$ (order~$4$) respectively.
		The plots in each pair show reconstruction errors (Frobenius distance from ground truth) and ranks (numerically estimated) of final solutions as a function of the number of observations in the task, with error bars spanning interquartile range ($25$'th to $75$'th percentiles) over multiple trials (differing in random seed for initialization), and markers showing median.
		For gradient descent over tensor factorization, we employed an adaptive learning rate scheme to reduce run times (see Appendix~\ref{mf:app:experiments:details} for details), and iteratively ran with decreasing standard deviation for initialization, until the point at which further reduction did not yield a noticeable change (presented results are those from the last iterations of this process, with the corresponding standard deviations annotated by ``init'').
		Notice that gradient descent over tensor factorization indeed exhibits an implicit tendency towards low rank (leading to accurate reconstruction of low-rank ground truth tensors), and that this tendency is stronger with smaller initialization.
		For further details and experiments see Appendix~\ref{mf:app:experiments}.
	}
	\label{mf:fig:experiment_tf}
\end{figure}

A generalization of the above is obtained by switching from matrices (tensors of \emph{order}~$2$) to tensors of arbitrary order~$N \in \N$.
This gives rise to a \emph{tensor completion} problem, with corresponding loss:
\be
\tfendloss : \R^{D_1 \times \cdots \times D_N} \to \R_{\geq 0}
\quad , \quad
\tfendloss ( \W ) = \frac{1}{2} \sum\nolimits_{( i_1 , \ldots , i_N ) \in \Omega} \big( ( \W )_{i_1 , \ldots , i_N} - y_{i_1  , \ldots , i_N} \big)^2
\label{mf:eq:loss_tensor}
\text{\,,}
\ee
where $\{ y_{i_1 , \ldots , i_N} \in \R \}_{( i_1 , \ldots , i_N ) \in \Omega}$, $\Omega \subset \{ 1 , \ldots , D_1 \} \times \cdots \times \{ 1 , 2 , \ldots , D_N \}$, stands for the set of observed entries.
One may employ a tensor factorization by minimizing the loss in Equation~\eqref{mf:eq:loss_tensor} via gradient descent over the parameterization:
\be
\W = \sum\nolimits_{r = 1}^R \wbf^{ 1 }_r  \otimes \cdots \otimes \wbf^{ N }_r
\quad , ~
\wbf^{ n }_r \in \R^{D_n}
~ , ~
r = 1 , \ldots , R
~ , ~
n = 1 , \ldots , N
\label{mf:eq:tf}
\text{\,,}
\ee
where again, $R \in \N$ is a predetermined constant, $\otimes$~stands for tensor product (\ie~outer product), and $\{ \wbf^{ n }_r \}_{r = 1}^{R} \hspace{0mm}_{n = 1}^{N}$ are the optimized parameters.
In analogy with the matrix case, the minimal~$R$ required for this parameterization to be able to express a given~$\widebar{\W} \in \R^{D_1 \times \cdots \times D_N}$ is defined to be the latter's \emph{tensor rank}.\footnote{
When referring to tensor rank, we mean the classic \emph{CP-rank} (see~\cite{kolda2009tensor}).
}
An implicit regularization towards low rank here would mean that even when $R$ is large enough for expressing any tensor, solutions expressible (or approximable) with small~$R$ tend to be learned. 

Figure~\ref{mf:fig:experiment_tf} displays results of tensor completion experiments, in which tensor factorization (optimization of loss in Equation~\eqref{mf:eq:loss_tensor} via gradient descent over parameterization in Equation~\eqref{mf:eq:tf}) is applied to observations drawn from a low-rank ground truth tensor.
As can be seen in terms of both reconstruction error (distance from ground truth tensor) and tensor rank of the produced solutions, tensor factorizations indeed exhibit an implicit regularization towards low rank.
The phenomenon thus goes beyond the special case of matrix (order~$2$ tensor) factorization.
In \cref{chap:imp_reg_tf} we will theoretically support this finding.

\medskip

\begin{figure}[t]
	\begin{center}
		\includegraphics[width=0.9\textwidth]{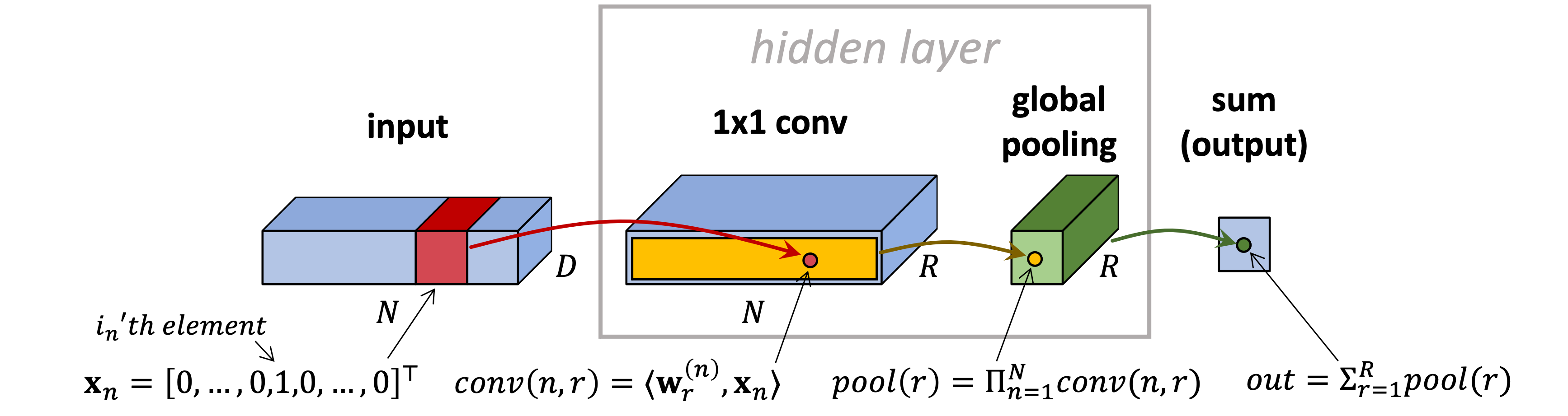}
	\end{center}
	\vspace{-2mm}
	\caption{
		Tensor factorization corresponds to a \emph{non-linear} convolutional neural network (with polynomial non-linearity), analogously to how matrix factorization corresponds to a linear neural network.
		The input to the network is a tuple $( i_1, \ldots , i_N ) \in \{ 1 , \ldots , D_1 \} \times \cdots \times \{ 1 , \ldots , D_N \}$, represented via one-hot vectors $( \xbf_1, \ldots , \xbf_N ) \in \R^{D_1} \times \cdots \times \R^{D_N}$ (illustration assumes $D_1 = \cdots = D_N = D$ to avoid clutter).
		These vectors are processed by a hidden layer comprising:
		\emph{(i)}~locally connected linear operator with $R$~channels, the $r$'th one computing inner products against filters $( \wbf^{1}_r , \ldots , \wbf^{N}_r ) \in \R^{D_1} \times \cdots \times \R^{D_N}$ (this operator is referred to as ``$1 {\times} 1$~conv'', appealing to the case of weight sharing, \ie~$\wbf^{1}_r  = \cdots = \wbf^{N}_r$);
		followed by
		\emph{(ii)}~global pooling computing products of all activations in each channel (which induces polynomial non-linearity).
		The result of the hidden layer is then reduced through summation to a scalar~---~output of the network.
		Overall, given input tuple $( i_1 , \ldots , i_N )$, the network outputs $( \W )_{i_1 , \ldots , i_N}$, where $\W \in \R^{D_1 \times \cdots \times D_N}$ is given by the tensor factorization in Equation~\eqref{mf:eq:tf}.
		Notice that the number of terms~($R$) and the tunable parameters~($\{ \wbf^{n}_r \}_{r , n}$) in the factorization respectively correspond to the width and the learnable filters of the network.
		Our tensor factorization (Equation~\eqref{mf:eq:tf}) was derived as an extension of a shallow (depth~$2$) matrix factorization, and accordingly, the convolutional neural network it corresponds to is shallow (has a single hidden layer).
		Endowing the factorization with hierarchical structures would render it equivalent to a \emph{deep} convolutional neural network (see~\cite{cohen2016expressive} for details).
		We will investigate the implicit regularization of these models in~\cref{chap:imp_reg_htf}.
	}
	\label{mf:fig:convac}
\end{figure}

As discussed in Section~\ref{mf:sec:overview}, matrix completion can be seen as a prediction problem, and matrix factorization as its solution with a \emph{linear neural network}.
In a similar vein, tensor completion may be viewed as a prediction problem, and tensor factorization as its solution with a certain (depth~$2$) \emph{non-linear} convolutional neural network~---~see Figure~\ref{mf:fig:convac}.
The non-linearity of this neural network is polynomial and stems from product pooling layers.
Analogously to how the input-output mapping of a linear neural network is naturally represented by a matrix, that of the neural network equivalent to tensor factorization admits a natural representation as a tensor.
Our experiments (Figure~\ref{mf:fig:experiment_tf} and Figure~\ref{mf:fig:experiment_tf_r3} in Appendix~\ref{mf:app:experiments:further}) show that when learned via gradient descent, this tensor tends to have low tensor rank.
We thus obtain a second exemplar of a neural network architecture whose implicit regularization strives to lower a notion of rank for its input-output mapping.
This indicates that the phenomenon may be general, and formalizing notions of rank for input-output mappings of contemporary models may be key to explaining generalization in deep learning.
In \cref{chap:imp_reg_tf,chap:imp_reg_htf} we theoretically support this hypothesis by employing the equivalence between tensor factorizations and certain neural networks.

%% file: Chapters/imp_reg_tf.tex
\chapter{Implicit Regularization in Tensor Factorization} 
\label{chap:imp_reg_tf}

The contents of this chapter are based on~\cite{razin2021implicit}.

\section{Background and Overview}
\label{tf:sec:overview}

\cref{chap:imp_reg_not_norms} considered the implicit regularization in matrix factorization, under the context of matrix completion problems.
Recall that in matrix completion, we are given a randomly chosen subset of entries from an unknown matrix~$\Wbf^* \in \R^{D \times D'}$, and our goal is to recover unseen entries.
This can be viewed as a prediction problem, where the set of possible inputs is $\X = \{ 1 , ...\, , D \} {\times} \{ 1 , ...\, , D' \}$, the possible labels are $\Y = \R$, and the label of $( i , j ) \in \X$ is~$( \Wbf^* )_{i , j}$.
Under this viewpoint, observed entries constitute the training set, and the average reconstruction error over unobserved entries is the test error, quantifying generalization.
A predictor, \ie~a function from $\X$ to~$\Y$, can then be seen as a matrix.

Although it was initially conjectured that the implicit regularization in matrix factorization minimizes some norm \cite{gunasekar2017implicit}, as we showed in \cref{chap:imp_reg_not_norms}, there exist cases in which no norm is being minimized.
Specifically, there exist matrix completion problems in which fitting the observed entries leads \emph{all norms to grow towards infinity} in favor of \emph{minimizing rank}.
Recent studies \cite{arora2019implicit,li2021towards} further suggest that gradient descent with small learning rate and near-zero initialization induces an incremental rank learning process, which results in low rank solutions.

A central question that arises is to what extent is the study of implicit regularization in matrix factorization relevant to more practical settings.
The experiments in~\cref{mf:sec:experiments:tensor} have shown that the tendency towards low rank extends from matrices (two-dimensional arrays) to \emph{tensors} (multi-dimensional arrays).
Namely, in the task of $N$-dimensional \emph{tensor completion}, which (analogously to matrix completion) can be viewed as a prediction problem over $N$~input variables, training a \emph{tensor factorization}\footnote{
	Recall that by ``tensor factorization'' we refer throughout to the classic CP factorization~\cite{kolda2009tensor}.
}
via gradient descent with small learning rate and near-zero initialization tends to produce tensors (predictors) with low \emph{tensor rank}.
Analogously to how matrix factorization may be viewed as a linear neural network, tensor factorization can be seen as a certain \emph{non-linear} convolutional neural network (two-layer network with polynomial non-linearity, \cf~\cite{cohen2016expressive}), and so it represents a setting much closer to practical deep learning.

In this chapter we theoretically analyze the implicit regularization in tensor factorization.
We circumvent the notorious difficulty of tensor problems \cite{hillar2013most} by adopting a dynamical systems perspective.
Characterizing the evolution that gradient descent with small learning rate and near-zero initialization induces on the components of a factorization, we show that their norms are subject to a momentum-like effect, in the sense that they move slower when small and faster when large.
This implies a form of greedy low tensor rank search, generalizing phenomena known for the case of matrices.
We employ the finding to prove that, with the classic Huber loss from robust statistics~\cite{huber1964robust}, arbitrarily small initialization leads tensor factorization to follow a trajectory of rank one tensors for an arbitrary amount of time or distance.
Experiments validate our analysis, demonstrating implicit regularization towards low tensor rank in a wide array of configurations.

Recall that, as discussed in~\cref{part:intro}, a major challenge towards understanding generalization in deep learning is that we lack definitions for predictor complexity that are both implicitly minimized during training of neural networks and capture the essence of natural data, in the sense of it being fittable with low complexity.
Motivated by the fact that tensor rank captures the implicit regularization of a non-linear neural network, we empirically explore its potential to serve as a measure of complexity for multivariable predictors.
We find that it is possible to fit standard image recognition datasets~---~MNIST~\cite{lecun1998mnist} and Fashion-MNIST~\cite{xiao2017fashion}~---~with predictors of extremely low tensor rank, far beneath what is required for fitting random data.
This leads us to believe that tensor rank (or more advanced notions such as hierarchical tensor ranks) may pave way to explaining both implicit regularization of contemporary deep neural networks, and the properties of natural data translating this implicit regularization to generalization.

\medskip

The remainder of the chapter is organized as follows.
Section~\ref{tf:sec:tf} presents the tensor factorization model, as well as its interpretation as a neural network.
Section~\ref{tf:sec:dynamic} characterizes its dynamics, followed by Section~\ref{tf:sec:rank} which employs the characterization to establish (under certain conditions) implicit tensor rank minimization.
Lastly, experiments demonstrating both the dynamics of learning and the ability of tensor rank to capture the essence of standard datasets are given in Section~\ref{tf:sec:experiments}.
Extension of our results to tensor sensing (more general setting than tensor completion) is discussed in Appendix~\ref{tf:app:sensing}.

\section{Tensor Factorization} \label{tf:sec:tf}

Consider the task of completing an $N$-dimensional tensor ($N \geq 3$) with axis lengths $D_1, \ldots, D_N \in \N$, or, in standard tensor analysis terminology, an \emph{order}~$N$ tensor with \emph{modes} of \emph{dimensions} $D_1, \ldots, D_N$.
Given a set of observations $\{ y_{i_1, \ldots, i_N} \in \R \}_{(i_1, \ldots, i_N) \in \Omega }$, where $\Omega$ is a subset of all possible index tuples, a standard (underdetermined) loss function for the task is:
\be
\tfendloss : \R^{D_1 \times \cdots \times D_N} \to \R_{\geq 0} ~~~ , ~~~ \tfendloss ( \W ) \,{=}\, \frac{1}{\abs{\Omega}} \hspace{-0.5mm} \sum\nolimits_{(i_1, \ldots, i_N) \in \Omega} \hspace{-0.5mm} \ell \left ( ( \W )_{i_1, \ldots, i_N} \,{-}\, y_{i_1, \ldots, i_N} \right )
\text{\,,}
\label{tf:eq:tc_loss}
\ee
where $\ell : \R \to \R_{\geq 0}$ is differentiable and locally smooth.
A typical choice for~$\ell ( \cdot )$ is $\ell ( z ) = \frac{1}{2} z^2$, corresponding to $\ell_2$~loss.
Other options are also common, for example the Huber loss from robust statistics~\cite{huber1964robust}~---~a differentiable surrogate for $\ell_1$~loss.

Performing tensor completion with an $R$-component tensor factorization amounts to optimizing the following (non-convex) objective:
\be
\tfobj \brk1{ \{ \wbf_r^n \}_{r = 1}^R\hspace{0mm}_{n = 1}^N } := \tfendloss \left ( \tftensorend \right )
\text{\,,}
\label{tf:eq:cp_objective}
\ee
defined over \emph{weight vectors}  $\{ \wbf_r^n \in \R^{D_n} \}_{r = 1}^R\hspace{0mm}_{n = 1}^N$, where:
\be
\tftensorend := \sum\nolimits_{r = 1}^R \wbf_r^1 \tenp \cdots \tenp \wbf_r^N
\label{tf:eq:end_tensor}
\ee
is referred to as the \emph{end tensor} of the factorization, with $\tenp$ representing tensor product (\ie~outer product).
The minimal number of components $R$ required in order for $\tftensorend$ to be able to express a given tensor $\W \in \R^{D_1 \times \cdots \times D_N}$, is defined to be the \emph{tensor rank} of~$\W$.
One may explicitly restrict the tensor rank of solutions produced by the tensor factorization via limiting~$R$.
However, since our interest lies in the implicit regularization induced by gradient descent, \ie~in the type of end tensors (Equation~\eqref{tf:eq:end_tensor}) it will find when applied to the objective $\tfobj (\cdot)$ (Equation~\eqref{tf:eq:cp_objective}) with no explicit constraints, we treat the case where $R$ can be arbitrarily large.

In line with analyses of matrix factorization (\eg~\cite{gunasekar2017implicit,arora2018optimization,arora2019implicit,eftekhari2020implicit,li2021towards}), we model small learning rate for gradient descent through the infinitesimal limit, \ie~through \emph{gradient flow}:
\be
\frac{d}{dt} \wbf_r^{n} (t) := - \frac{\partial}{\partial \wbf_r^{n}} \tfobj  \brk1{ \{ \wbf_{r'}^{n'} (t) \}_{r' = 1}^R\hspace{0mm}_{n' = 1}^N } ~~ , ~ t \geq 0 ~ , ~ r = 1, \ldots, R ~ , ~ n = 1, \ldots, N 
\label{tf:eq:cp_gf}
\text{\,,}
\ee
where $\{ \wbf_r^n (t) \}_{r = 1}^R\hspace{0mm}_{n = 1}^N$ denote the weight vectors at time~$t$ of optimization.

Our aim is to theoretically investigate the prospect of implicit regularization towards low tensor rank, \ie~of gradient flow with near-zero initialization learning a solution that can be represented with a small number of components.

\subsection{Interpretation as Neural Network} \label{tf:sec:tf:nn}

\begin{figure*}[t]
	\begin{center}
		\includegraphics[width=0.67\textwidth]{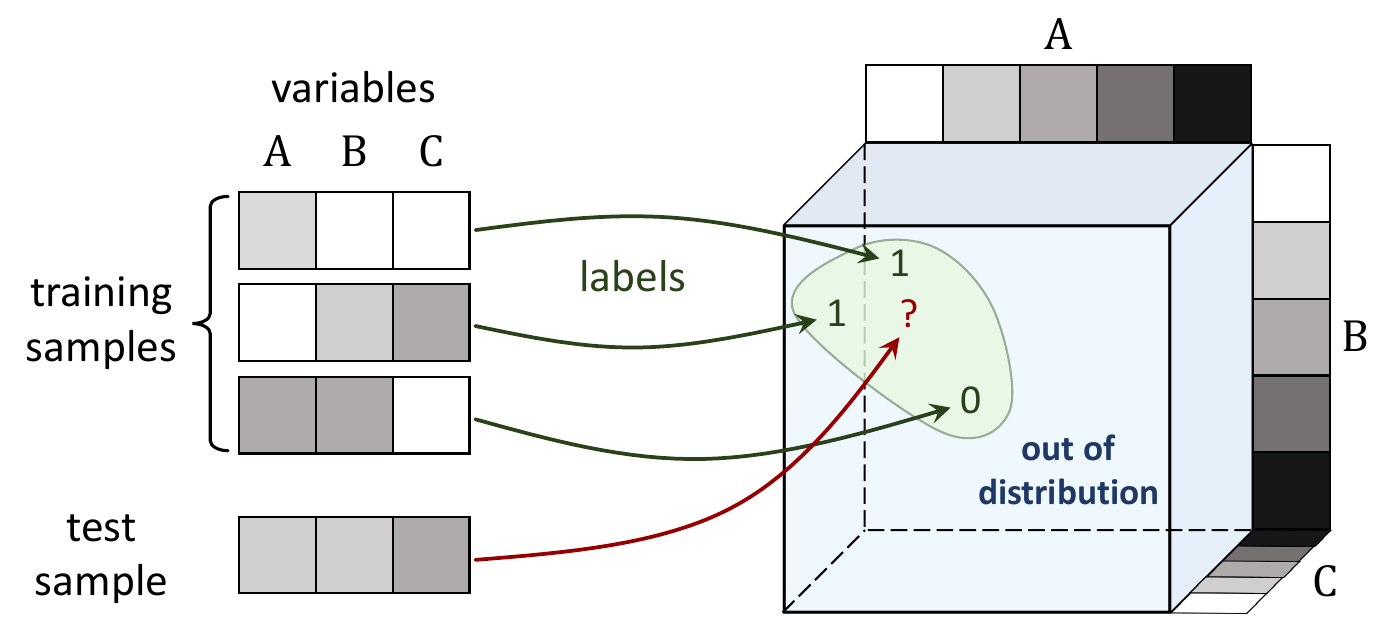}
	\end{center}
	\vspace{-1mm}
	\caption{
		Prediction tasks over discrete variables can be viewed as tensor completion problems.		
		Consider the task of learning a predictor from domain $\X = \{ 1 , \ldots , D_1 \} \times \cdots \times \{ 1 , \ldots , D_N \}$ to range $\Y = \R$ (figure assumes $N = 3$ and $D_1 = \cdots = D_N = 5$ for the sake of illustration).
		Each input sample is associated with a location in an order~$N$ tensor with mode (axis) dimensions $D_1, \ldots, D_N$, where the value of a variable (depicted as a shade of gray) determines the index of the corresponding mode (marked by ``A", ``B" or ``C").
		The associated location stores the label of the sample.
		Under this viewpoint, training samples are observed entries, drawn according to an unknown distribution from a ground truth tensor.
		Learning a predictor amounts to completing the unobserved entries, with test error measured by (weighted) average reconstruction error.
		In many standard prediction tasks (\eg~image recognition), only a small subset of the input domain has non-negligible probability.
		From the tensor completion perspective this means that observed entries reside in a restricted part of the tensor, and reconstruction error is weighted accordingly (entries outside the support of the distribution are neglected).
	}
	\label{tf:fig:pred_prob_as_tc}
\end{figure*}

Tensor completion can be viewed as a prediction problem, where each mode corresponds to a discrete input variable.
For an unknown tensor $\W^* \in \R^{D_1 \times \cdots \times D_N}$, inputs are index tuples of the form $(i_1, \ldots, i_N)$, and the label associated with such an input is $( \W^*)_{i_1, \ldots, i_N}$.
Under this perspective, the training set consists of the observed entries, and the average reconstruction error over unseen entries measures test error.
The standard case, in which observations are drawn uniformly across the tensor and reconstruction error weighs all entries equally, corresponds to a data distribution that is uniform, but other distributions are also viable.

Consider for example the task of predicting a continuous label for a $100$-by-$100$ binary image.
This can be formulated as an order $10000$ tensor completion problem, where all modes are of dimension $2$.
Each input image corresponds to a location (entry) in the tensor~$\W^*$, holding its continuous label.
As image pixels are (typically) not distributed independently and uniformly, locations in the tensor are not drawn uniformly when observations are generated, and are not weighted equally when reconstruction error is computed.
See Figure~\ref{tf:fig:pred_prob_as_tc} for further illustration of how a general prediction task (with discrete inputs and scalar output) can be formulated as a tensor completion problem.

As discussed in \cref{mf:sec:experiments:tensor} of \cref{chap:imp_reg_not_norms}, under the above formulation, tensor factorization can be viewed as a two-layer convolutional neural network with polynomial non-linearity, where the non-linearity stems from product pooling layers.
Given an input, \ie~a location in the tensor, the network produces an output equal to the value that the factorization holds at the given location.
This equivalence between tensor factorization and a non-linear convolutional neural network was illustrated in Figure~\ref{mf:fig:convac} of \cref{chap:imp_reg_not_norms}.
A major drawback of matrix factorization as a theoretical surrogate for modern deep learning is that it misses the critical aspect of non-linearity.
Tensor factorization goes beyond the realm of linear predictors~---~a significant step towards practical neural networks.

\section{Dynamical Characterization} \label{tf:sec:dynamic}

In this section we derive a dynamical characterization for the norms of individual components in the tensor factorization.
The characterization implies that with small learning rate and near-zero initialization, components tend to be learned incrementally, giving rise to a bias towards low tensor rank solutions.
This finding is used in Section~\ref{tf:sec:rank} to prove (under certain conditions) implicit tensor rank minimization, and is demonstrated empirically in Section~\ref{tf:sec:experiments}.\footnote{
	We note that all results in this section apply even if the tensor completion loss~$\tfendloss (\cdot)$ (Equation~\eqref{tf:eq:tc_loss}) is replaced by any differentiable and locally smooth function.
	The proofs in Appendix~\ref{tf:app:proofs} already account for this more general setting.
	\label{note:dyn_holds_for_diff_local_smooth_loss}
}

For the rest of the chapter, unless specified otherwise, when referring to a norm we mean the standard Frobenius (Euclidean) norm, denoted by~$\norm{\cdot}$.

The following lemma establishes an invariant of the dynamics, showing that the differences between squared norms of vectors in the same component are constant through time.

\begin{lemma}
	\label{tf:lem:balancedness_conservation_body}
	For all $r \in \{ 1, \ldots, R \}$ and $n, \bar{n} \in \{ 1 , \ldots, N \}$:
	\[
	\norm*{ \w_r^{n} (t) }^2 - \norm*{ \w_r^{\bar{n}} (t) }^2 = \norm*{ \w_r^{n} (0) }^2 - \norm*{ \w_r^{\bar{n}} (0) }^2 ~~, ~t \geq 0
	\text{\,.}
	\]
\end{lemma}

\begin{proof}[Proof sketch (for proof see Lemma~\ref{tf:lem:balancedness_conservation} in Appendix~\ref{tf:app:proofs:useful_lemmas:tf})]
	The claim readily follows by showing that under gradient flow $\frac{d}{dt}  \normnoflex{ \w_r^{n} (t) }^2 = \frac{d}{dt}  \normnoflex{ \w_r^{\bar{n}} (t) }^2$ for all $t \geq 0$.
\end{proof}

Lemma~\ref{tf:lem:balancedness_conservation_body} naturally leads to the definition below.

\begin{definition}
	\label{def:unbalancedness_magnitude}
	The \emph{unbalancedness magnitude} of the weight vectors $\{ \w_r^n \in \R^{D_n} \}_{r = 1}^R\hspace{0mm}_{n = 1}^N$ is defined to be:
	\[
	\max\nolimits_{r \in \{ 1, \ldots, R\}, ~n, \bar{n} \in \{ 1, \ldots, N\}} \abs*{ \norm*{ \w_r^n }^2 - \norm*{ \w_r^{\bar{n}} }^2 }
	\text{\,.}
	\]
\end{definition}

By Lemma~\ref{tf:lem:balancedness_conservation_body}, the unbalancedness magnitude is constant during optimization, and thus, is determined at initialization.
When weight vectors are initialized near the origin~---~regime of interest~---~the unbalancedness magnitude is small, approaching zero as initialization scale decreases.

Theorem~\ref{tf:thm:dyn_fac_comp_norm_unbal} below provides a dynamical characterization for norms of individual components in the tensor factorization.

\begin{theorem}
	\label{tf:thm:dyn_fac_comp_norm_unbal}
	Assume unbalancedness magnitude $\epsilon \geq 0$ at initialization, and denote by~$\tftensorend (t)$ the end tensor (Equation~\eqref{tf:eq:end_tensor}) at time $t \geq 0$ of optimization.
	Then, for any $r \in \{ 1, \ldots, R \}$ and time $t \geq 0$ at which $\normnoflex{ \tenp_{n = 1}^N \w_r^{n} (t) } > 0$:\footnote{
		When $\| \tenp_{n = 1}^N \w_r^{n} (t) \|$ is zero it may not be differentiable.
	}
	\begin{itemize}[leftmargin=3.5mm]
		\item If $\gamma_r (t) := \inprodnoflex{ - \nabla \tfendloss ( \tftensorend (t) ) }{ \tenp_{n = 1}^N \widehat{\w}_r^{n} (t) } \geq 0$, then:
		\be
		\begin{split}
			& \hspace{-3.2mm} \frac{d}{dt} \normnoflex{ \tenp_{n = 1}^N \w_r^{n} (t) } \leq N \gamma_r (t) ( \normnoflex{ \tenp_{n = 1}^N \w_r^{n} (t) }^{\frac{2}{N}} + \epsilon )^{N - 1} \text{\,,} \\[1mm]
			& \hspace{-3.2mm} \frac{d}{dt} \normnoflex{ \tenp_{n = 1}^N \w_r^{n} (t) } \geq N \gamma_r (t) \cdot \frac{ \normnoflex{ \tenp_{n = 1}^N \w_r^{n} (t) }^2 }{  \normnoflex{ \tenp_{n = 1}^N \w_r^{n} (t) }^{\frac{2}{N}} + \epsilon }
			\text{\,,}
		\end{split}
		\label{tf:eq:dyn_fac_comp_norm_unbal_pos}
		\ee
		
		\item otherwise, if $\gamma_r (t)  < 0$, then:
		\be
		\begin{split}
			& \hspace{-3.2mm} \frac{d}{dt} \normnoflex{ \tenp_{n = 1}^N \w_r^{n} (t) } \geq N \gamma_r (t) ( \normnoflex{ \tenp_{n = 1}^N \w_r^{n} (t) }^{\frac{2}{N}} + \epsilon )^{N - 1} \text{\,,} \\[1mm]
			& \hspace{-3.2mm} \frac{d}{dt} \norm{ \tenp_{n = 1}^N \w_r^{n} (t) } \leq N \gamma_r (t) \cdot \frac{ \norm{ \tenp_{n = 1}^N \w_r^{n} (t) }^2 }{  \norm{ \tenp_{n = 1}^N \w_r^{n} (t) }^{\frac{2}{N}} + \epsilon }
			\text{\,,}
		\end{split}
		\label{tf:eq:dyn_fac_comp_norm_unbal_neg}
		\ee
	\end{itemize}
	where $\widehat{\w}_r^{n} (t) := \w_r^{n} (t) / \normnoflex{ \w_r^{n} (t) }$ for $n = 1, \ldots, N$.
\end{theorem}

\begin{proof}[Proof sketch (proof in Appendix~\ref{tf:app:proofs:dyn_fac_comp_norm_unbal})]
Differentiating a component's norm with respect to time, we obtain $\frac{d}{dt} \normnoflex{ \tenp_{n = 1}^N \w_r^{n} (t) } = \gamma_r (t) \cdot \sum_{n = 1}^N \prod_{n' \neq n} \normnoflex{ \w_r^{n'} (t) }^2$.
The desired bounds then follow from using conservation of unbalancedness magnitude (as implied by Lemma~\ref{tf:lem:balancedness_conservation_body}), and showing that $\normnoflex{ \w_r^{n'} (t) }^2 \leq \normnoflex{ \tenp_{n = 1}^N \w_r^n (t) }^{ 2 / N } + \epsilon$ for all $t \geq 0$ and $n' \in  \{ 1, \ldots, N \}$.
\end{proof}

Theorem~\ref{tf:thm:dyn_fac_comp_norm_unbal} shows that when unbalancedness magnitude at initialization (denoted~$\epsilon$) is small, the evolution rates of component norms are roughly proportional to their size exponentiated by $2 - 2 / N$, where $N$ is the order of the tensor factorization.
Consequently, component norms are subject to a momentum-like effect, by which they move slower when small and faster when large.
This suggests that when initialized near zero, components tend to remain close to the origin, and then, upon reaching a critical threshold, quickly grow until convergence, creating an incremental learning effect that yields implicit regularization towards low tensor rank.
This phenomenon is used in Section~\ref{tf:sec:rank} to formally prove (under certain conditions) implicit tensor rank minimization, and is demonstrated empirically in Section~\ref{tf:sec:experiments}.

When the unbalancedness magnitude at initialization is exactly zero, our dynamical characterization takes on a particularly lucid form.

\begin{corollary}
	\label{tf:cor:dyn_fac_comp_norm_balanced}
	Assume unbalancedness magnitude zero at initialization.
	Then, with notations of Theorem~\ref{tf:thm:dyn_fac_comp_norm_unbal}, for any $r \in \{ 1, \ldots, R \}$, the norm of the $r$'th component evolves by:
	\vspace{-3.5mm}
	\be
	\frac{d}{dt} \norm{ \tenp_{n = 1}^N \w_r^{n} (t) } = N \gamma_r (t) \cdot \norm{ \tenp_{n = 1}^N \w_r^{n} (t) }^{2 - \frac{2}{N}} 
	\text{\,,}
	\label{tf:eq:dyn_fac_comp_norm}
	\ee
	where by convention $\widehat{\w}_r^{n} (t) = 0$ if $\w_r^{n} (t) = 0$.
\end{corollary}

\begin{proof}[Proof sketch (proof in Appendix~\ref{tf:app:proofs:dyn_fac_comp_norm_balanced})]
	If the time~$t$ is such that $\normnoflex{ \tenp_{n = 1}^N \w_r^{n} (t) } > 0$, Equation~\eqref{tf:eq:dyn_fac_comp_norm} readily follows from applying Theorem~\ref{tf:thm:dyn_fac_comp_norm_unbal} with $\epsilon = 0$.
	For the case where $\normnoflex{ \tenp_{n = 1}^N \w_r^{n} (t) } = 0$, we show that the component $\tenp_{n = 1}^N \w_r^{n} (t)$ must be identically zero throughout, hence both sides of Equation~\eqref{tf:eq:dyn_fac_comp_norm} are equal to zero.
\end{proof}

It is worthwhile highlighting the relation to matrix factorization.
There, an implicit bias towards low rank emerges from incremental learning dynamics similar to above, with singular values standing in place of component norms.
In fact, the dynamical characterization given in Corollary~\ref{tf:cor:dyn_fac_comp_norm_balanced} is structurally identical to the one provided by Theorem~3 in \cite{arora2019implicit} for singular values of a matrix factorization.
We thus obtained a generalization from matrices to tensors, notwithstanding the notorious difficulty often associated with the latter (\cf~\cite{hillar2013most}).


\section{Implicit Tensor Rank Minimization} \label{tf:sec:rank}

In this section we employ the dynamical characterization derived in Section~\ref{tf:sec:dynamic} to theoretically establish implicit regularization towards low tensor rank.
Specifically, we prove that under certain technical conditions, arbitrarily small initialization leads tensor factorization to follow a trajectory of rank one tensors for an arbitrary amount of time or distance.
As a corollary, we obtain that if the tensor completion problem admits a rank one solution, and all rank one trajectories uniformly converge to it, tensor factorization with infinitesimal initialization will converge to it as well.
Our analysis generalizes to tensor factorization recent results developed in \cite{li2021towards} for matrix factorization.
As typical in transitioning from matrices to tensors, this generalization entails significant challenges necessitating use of fundamentally different techniques.

For technical reasons, our focus in this section lies on the Huber loss from robust statistics~\cite{huber1964robust}, given by:
\be
\hspace{-0.125mm}
\ell_h : \R \to \R_{\geq 0}
~~~ , ~~~
\ell_h ( z ) \,{:=} \begin{cases}
	\frac{1}{2} z^2 &  , \abs{z} < \delta_h \\
	\delta_h (\abs{z} - \frac{1}{2} \delta_h) &  , \text{otherwise}
\end{cases}
\text{\,,}
\label{tf:eq:huber_loss}
\ee
where $\delta_h > 0$, referred to as the transition point of the loss, is predetermined.
Huber loss is often used as a differentiable surrogate for $\ell_1$~loss, in which case~$\delta_h$ is chosen to be small.
We will assume it is smaller than observed tensor entries:\footnote{
	Note that this entails assumption of non-zero observations.
}
\begin{assumption}
	\label{tf:assump:delta_h}
	$\delta_h < | y_{i_1, \ldots, i_N} | ~ , \forall (i_1, \ldots, i_N) \in \Omega$.
\end{assumption}

We will consider an initialization $\{ \abf_r^n \in \R^{D_n} \}_{r = 1}^R\hspace{0mm}_{n = 1}^N$ for the weight vectors of the tensor factorization, and will scale this initialization towards zero.
In line with infinitesimal initializations being captured by unbalancedness magnitude zero (\cf~Section~\ref{tf:sec:dynamic}), we assume that this is the case:
\begin{assumption}
	\label{tf:assump:a_balance}
	The initialization $\{ \abf_r^n \}_{r = 1}^R\hspace{0mm}_{n = 1}^N$ has unbalancedness magnitude zero.
\end{assumption}
We further assume that within $\{ \abf_r^n \}_{r , n}$ there exists a leading component (subset $\{ \abf_{\bar{r}}^n \}_n$), in the sense that it is larger than others, while having positive projection on the attracting force at the origin, \ie~on minus the gradient of the loss~$\tfendloss ( \cdot )$ (Equation~\eqref{tf:eq:tc_loss}) at zero:
\begin{assumption}
	\label{tf:assump:a_lead_comp}
	There exists $\bar{r} \in \{ 1, \ldots, R \}$ such~that:
	\be
	\hspace{-6mm}
	\begin{split}
		& \inprod{ - \nabla \tfendloss ( 0 ) }{ \tenp_{n = 1}^N \widehat{\abf}_{\bar{r}}^{n} } > 0 \text{\,,} \\
		& \normnoflex{ \abf_{\bar{r}}^n } > \normnoflex{ \abf_{r}^n } {\cdot} \hspace{-1mm} \left ( \tfrac{ \norm{ \nabla \tfendloss ( 0 ) } }{  \inprod{ - \nabla \tfendloss ( 0 ) }{  \tenp_{n = 1}^N \widehat{\abf}_{ \bar{r} }^{n} }} \right )^{ 1 / (N - 2) } ~~, ~\forall r \neq \bar{r} \text{\,,}
	\end{split}
	\label{tf:eq:assump_components_sep_at_init}
	\ee
	where $\widehat{\abf}_{\bar{r}}^{n} := \abf_{\bar{r}}^{n} / \normnoflex{ \abf_{\bar{r}}^{n} }$ for $n = 1, \ldots, N$.
\end{assumption}

Let $\alpha > 0$, and suppose we run gradient flow on the tensor factorization (see Section~\ref{tf:sec:tf}) starting from the initialization $\{ \abf_r^n \}_{r , n}$ scaled by~$\alpha$.
That is, we set:
\[
\w_r^n (0) = \alpha \cdot \abf_r^n \quad , ~r = 1, \ldots, R ~,n = 1, \ldots, N
\text{\,,}
\]
and let $\{ \w_r^n ( t ) \}_{r , n}$ evolve per Equation~\eqref{tf:eq:cp_gf}.
Denote by $\tftensorend (t)$, $t \geq 0$, the trajectory induced on the end tensor (Equation~\eqref{tf:eq:end_tensor}).
We will study the evolution of this trajectory through time.
A hurdle that immediately arises is that, by the dynamical characterization of Section~\ref{tf:sec:dynamic}, when the initialization scale $\alpha$ tends to zero (regime of interest), the time it takes $\tftensorend (t)$ to escape the origin grows to infinity.\footnote{
	To see this, divide both sides of Equation~\eqref{tf:eq:dyn_fac_comp_norm} from Corollary~\ref{tf:cor:dyn_fac_comp_norm_balanced} by $\normnoflex{ \tenp_{n = 1}^N \w_r^{n} (t) }^{2 - 2 / N}$, and integrate with respect to~$t$.
	It follows that the norm of a component at any fixed time tends to zero as initialization scale $\alpha$ decreases.
	This implies that for any $D > 0$, when taking $\alpha \to 0$, the time required for a component to reach norm~$D$ grows to infinity.
}
We overcome this hurdle by considering a \emph{reference sphere}~---~a sphere around the origin with sufficiently small radius:
\be
\S := \{ \W \in \R^{d_1 \times \cdots \times d_N} : \| \W \| = \rho \}
\text{\,,}
\label{tf:eq:ref_sphere}
\ee
where $\rho \in ( 0 , \min_{(i_1, \ldots, i_N) \in \Omega} \abs{ y_{i_1, \ldots, i_N} } - \delta_h )$ can be chosen arbitrarily.
With the reference sphere $\S$ at hand, we define a time-shifted version of the trajectory $\tftensorend (t)$, aligning $t = 0$ with the moment at which $\S$ is reached:
\be
\tftensorendbar ( t ) := \tftensorend \brk*{ t + \inf \{ t' \geq 0 : \tftensorend ( t' ) \in \S \} }
\label{tf:eq:end_tensor_time_shift}
\text{\,,}
\ee
where by definition $\inf \{ t' \geq 0 : \tftensorend ( t' ) \in \S \} = 0$ if $\tftensorend (t)$ does not reach~$\S$.
Unlike the original trajectory~$\tftensorend (t)$, the shifted one $\tftensorendbar ( t )$ disregards the process of escaping the origin, and thus admits a concrete meaning to the time elapsing from optimization commencement.

We will establish proximity of $\tftensorendbar ( t )$ to trajectories of rank one tensors.
We say that $\W_1 ( t ) \in \R^{D_1 \times \cdots \times D_N}$, $t \geq 0$, is a \emph{rank one trajectory}, if it coincides with some trajectory of an end tensor in a one-component factorization, \ie~if there exists an initialization for gradient flow over a tensor factorization with $R = 1$ components, leading the induced end tensor to evolve by~$\W_1 ( t )$.
If the latter initialization has unbalancedness magnitude zero (\cf~Definition~\ref{def:unbalancedness_magnitude}), we further say that $\W_1 ( t )$ is a \emph{balanced rank one trajectory}.\footnote{
	Note that the definitions of rank one trajectory and balanced rank one trajectory allow for $\W_1 ( t )$ to have rank zero (\ie~to be equal to zero) at some or all times~$t \geq 0$.
}

We are now in a position to state our main result, by which arbitrarily small initialization leads tensor factorization to follow a (balanced) rank one trajectory for an arbitrary amount of time or distance.
\begin{theorem}
	\label{tf:thm:approx_rank_1}
	Under Assumptions \ref{tf:assump:delta_h}, \ref{tf:assump:a_balance} and~\ref{tf:assump:a_lead_comp}, for any distance from origin $B > 0$, time duration $T > 0$, and degree of approximation $\epsilon \in ( 0 , 1 )$, if initialization scale $\alpha$ is sufficiently small,\footnote{
		Hiding problem-dependent constants, an initialization scale of $\epsilon B^{-1}  \exp ( - \OO (B^{2} T) )$ suffices.
		Exact constants are specified at the beginning of the proof in Appendix~\ref{tf:app:proofs:approx_rank_1}.
	}
	then:
	\emph{(i)}~$\tftensorend ( t )$ reaches the reference sphere~$\S$;
	and
	\emph{(ii)}~there exists a balanced rank one trajectory $\W_1 ( t )$ emanating from~$\S$, such that $\| \tftensorendbar ( t ) - \W_1 ( t ) \| \leq \epsilon$ at least until $t \geq T$ or $\| \tftensorendbar ( t ) \| \geq B$.
\end{theorem}
\begin{proof}[Proof sketch (proof in Appendix~\ref{tf:app:proofs:approx_rank_1})]
	Using the dynamical characterization from Section~\ref{tf:sec:dynamic} (Lemma~\ref{tf:lem:balancedness_conservation_body} and Corollary~\ref{tf:cor:dyn_fac_comp_norm_balanced}), and the fact that $\nabla \tfendloss (\cdot)$ is locally constant around the origin, we establish that 
	\emph{(i)} $\tftensorend (t)$ reaches the reference sphere $\S$; 
	and 
	\emph{(ii)} at that time, the norm of the $\bar{r}$'th component is of constant scale (independent of $\alpha$), while the norms of all other components are $\OO (\alpha^N)$.
	Thus, taking $\alpha$ towards zero leads $\tftensorend (t)$ to arrive at~$\S$ while being arbitrarily close to the initialization of a balanced rank one trajectory~---~$\W_1 (t)$.
	Since the objective is locally smooth, this ensures $\tftensorendbar ( t )$ is within distance~$\epsilon$ from $\W_1 (t)$ for an arbitrary amount of time or distance.
	That is, if $\alpha$ is sufficiently small, $\| \tftensorendbar ( t ) - \W_1 ( t ) \| \leq \epsilon$ at least until $t \geq T$ or $\| \tftensorendbar ( t ) \| \geq B$.
\end{proof}

As an immediate corollary of Theorem~\ref{tf:thm:approx_rank_1}, we obtain that if all balanced rank one trajectories uniformly converge to a global minimum, tensor factorization with infinitesimal initialization will do so too.
In particular, its implicit regularization will direct it towards a solution with tensor~rank~one.
\begin{corollary}
	\label{corollary:converge_rank_1}
	Assume the conditions of Theorem~\ref{tf:thm:approx_rank_1} (Assumptions \ref{tf:assump:delta_h}, \ref{tf:assump:a_balance} and~\ref{tf:assump:a_lead_comp}), and in addition, that all balanced rank one trajectories emanating from~$\S$ converge to a tensor $\W^* \in \R^{D_1 \times \cdots \times D_N}$ uniformly, in the sense that they are all confined to some bounded domain, and for any $\epsilon > 0$, there exists a time~$T$ after which they are all within distance $\epsilon$ from~$\W^*$.
	Then, for any $\epsilon > 0$, if initialization scale $\alpha$ is sufficiently small, there exists a time~$T$ for which $\| \tftensorend ( T ) - \W^* \| \leq \epsilon$.
\end{corollary}
\begin{proof}[Proof sketch (proof in Appendix~\ref{tf:app:proofs:converge_rank_1})]
	Let $T' > 0$ be a time at which all balanced rank one trajectories that emanated from $\S$ are within distance~$\epsilon / 2$ from $\W^*$.
	By Theorem~\ref{tf:thm:approx_rank_1}, if $\alpha$ is sufficiently small, $\tftensorendbar ( t )$~is guaranteed to be within distance~$\epsilon / 2$ from a balanced rank one trajectory that emanated from~$\S$, at least until time~$T'$.
	Recalling that $\tftensorendbar ( t )$ is a time-shifted version of~$\tftensorend ( t )$, the desired result follows from the triangle inequality.
\end{proof}

\section{Experiments} \label{tf:sec:experiments}

In this section we present our experiments.
Section~\ref{tf:sec:experiments:dyn} corroborates our theoretical analyses (Sections~\ref{tf:sec:dynamic} and~\ref{tf:sec:rank}), evaluating tensor factorization (Section~\ref{tf:sec:tf}) on synthetic low (tensor) rank tensor completion problems.
Section~\ref{tf:sec:experiments:tensor_rank_complexity} explores tensor rank as a measure of complexity, examining its ability to capture the essence of standard datasets.
For brevity, we defer a description of implementation details, as well as some experiments, to Appendix~\ref{tf:app:experiments}.

\subsection{Dynamics of Learning}
\label{tf:sec:experiments:dyn}

\begin{figure*}
	\begin{center}
		\hspace*{-2mm}
		\subfloat{
			\includegraphics[width=0.24\textwidth]{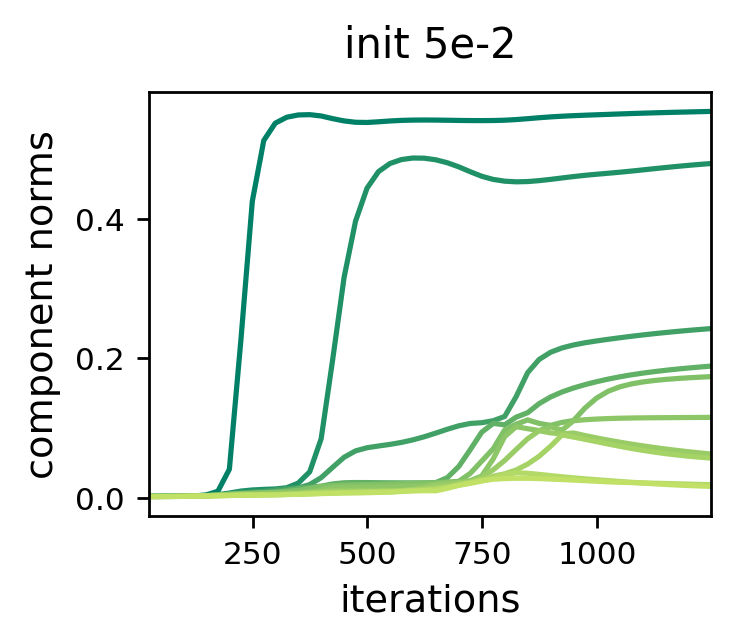}
		}
		\subfloat{
			\includegraphics[width=0.24\textwidth]{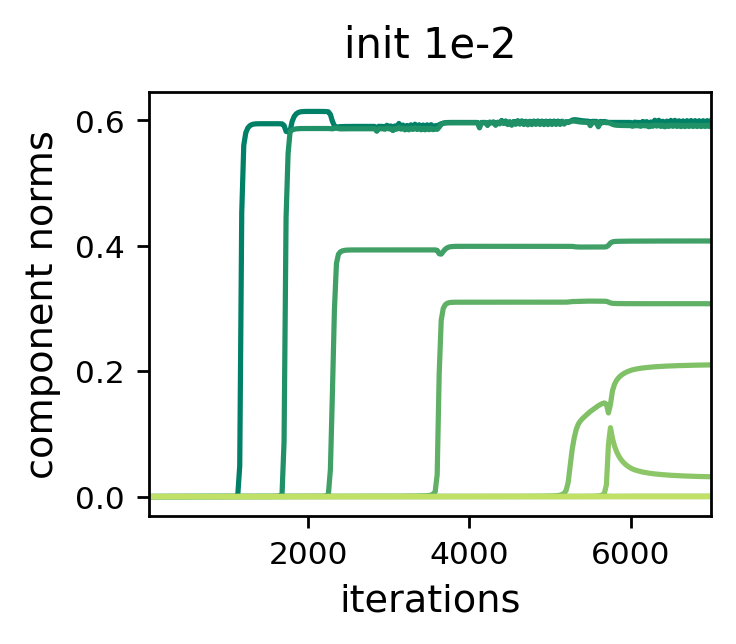}
		}
		\subfloat{
			\includegraphics[width=0.24\textwidth]{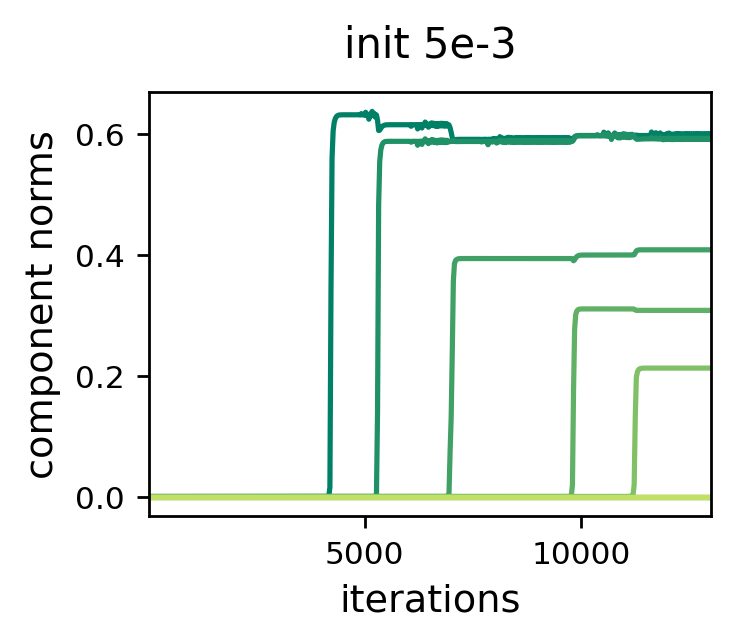}
		}
		\subfloat{
			\includegraphics[width=0.24\textwidth]{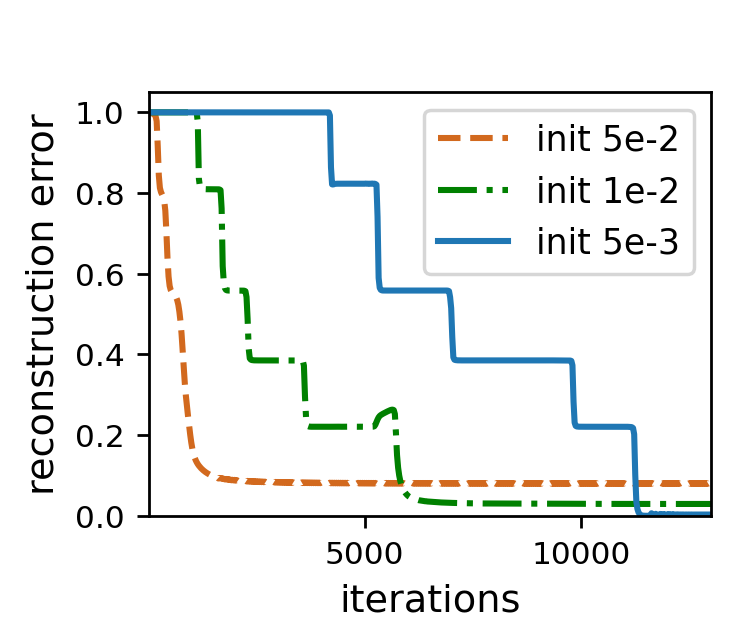}
		}
	\end{center}
	\vspace{-2mm}
	\caption{
		Dynamics of gradient descent over tensor factorization --- incremental learning of components yields low tensor rank solutions.
		Presented plots correspond to the task of completing a (tensor) rank $5$ ground truth tensor of size $10$-by-$10$-by-$10$-by-$10$ (order $4$) based on $2000$ observed entries chosen uniformly at random without repetition (smaller sample sizes led to solutions with tensor rank lower than that of the ground truth tensor).
		In each experiment, the $\ell_2$ loss (more precisely, Equation~\eqref{tf:eq:tc_loss} with $\ell ( z ) := z^2$) was minimized via gradient descent over a tensor factorization with $R = 1000$ components (large enough to express any tensor), starting from (small) random initialization.
		First (left) three plots show (Frobenius) norms of the ten largest components under three standard deviations for initialization~---~$0.05, 0.01,$ and $0.005$.
		Further reduction of initialization scale yielded no noticeable change.
		The rightmost plot compares reconstruction errors (Frobenius distance from ground truth) from the three runs.
		To facilitate more efficient experimentation, we employed an adaptive learning rate scheme (see Appendix~\ref{tf:app:experiments:details} for details).
		Notice that, in accordance with the theoretical analysis of Section~\ref{tf:sec:dynamic}, component norms move slower when small and faster when large, creating an incremental process in which components are learned one after the other.
		This effect is enhanced as initialization scale is decreased, producing low tensor rank solutions that accurately reconstruct the low (tensor) rank ground truth tensor.
		In particular, even though the factorization consists of $1000$ components, when initialization is sufficiently small, only five (tensor rank of the ground truth tensor) substantially depart from zero.
		Appendix~\ref{tf:app:experiments} provides further implementation details, as well as similar experiments with: \emph{(i)} Huber loss (see Equation~\eqref{tf:eq:huber_loss}) instead of $\ell_2$~loss; \emph{(ii)} ground truth tensors of different orders and (tensor) ranks; and \emph{(iii)} tensor sensing (see Appendix~\ref{tf:app:sensing}).
		\vspace{-0.5mm}
	}
	\label{tf:fig:tc_mse_ord4}
\end{figure*}

In \cref{mf:sec:experiments:tensor} of \cref{chap:imp_reg_not_norms} we empirically showed that, with small learning rate and near-zero initialization, gradient descent over tensor factorization exhibits an implicit regularization towards low tensor rank.
The theory in Sections~\ref{tf:sec:dynamic} and~\ref{tf:sec:rank} explains this implicit regularization through a dynamical analysis~---~we prove that the movement of component norms is attenuated when small and enhanced when large, thus creating an incremental learning effect which becomes more potent as initialization scale decreases.
Figure~\ref{tf:fig:tc_mse_ord4} demonstrates this phenomenon empirically on synthetic low (tensor) rank tensor completion problems.
Figures~\ref{tf:fig:tc_huber_ord4},~\ref{tf:fig:tc_mse_ord3} and~\ref{tf:fig:ts_mse_ord4} in Appendix~\ref{tf:app:experiments:further} extend the experiment, corroborating our analyses in a wide array of settings.

\subsection{Tensor Rank as Measure of Complexity}
\label{tf:sec:experiments:tensor_rank_complexity}

Implicit regularization in deep learning is typically viewed as a tendency of gradient-based optimization to fit training examples with predictors whose ``complexity'' is as low as possible.
The fact that ``natural'' data gives rise to generalization while other types of data (\eg~random) do not, is understood to result from the former being amenable to fitting by predictors of lower complexity.
A major challenge in formalizing this intuition is that we lack definitions for predictor complexity that are both quantitative (\ie~admit quantitative generalization bounds) and capture the essence of natural data (types of data on which neural networks generalize in practice), in the sense of it being fittable~with~low~complexity.

As discussed in Section~\ref{tf:sec:tf:nn}, learning a predictor with multiple discrete input variables and a continuous output can be viewed as a tensor completion problem.
Specifically, with $N \in \N$, $D_1 , \ldots , D_N \in \N$, learning a predictor from domain $\X = \{ 1 , \ldots , D_1 \} \times \cdots \times \{ 1 , \ldots , D_N \}$ to range $\Y = \R$ corresponds to completion of an order~$N$ tensor with mode (axis) dimensions $D_1 , \ldots , D_N$. 
Under this correspondence, any predictor can simply be thought of as a tensor, and vice versa.
We have shown that solving tensor completion via tensor factorization amounts to learning a predictor through a certain neural network (Section~\ref{tf:sec:tf:nn}), whose implicit regularization favors solutions with low tensor rank (Sections \ref{tf:sec:dynamic} and~\ref{tf:sec:rank}).
Motivated by these connections, the current subsection empirically explores tensor rank as a measure of complexity for predictors, by evaluating the extent to which it captures natural data, \ie~allows the latter to be fit with low complexity predictors.

As representatives of natural data, we chose the classic MNIST dataset \cite{lecun1998mnist}~---~perhaps the most common benchmark for demonstrating ideas in deep learning~---~and its more modern counterpart Fashion-MNIST \cite{xiao2017fashion}.
A hurdle posed by these datasets is that they involve classification into multiple categories, whereas the equivalence to tensors applies to predictors whose output is a scalar.
It is possible to extend the equivalence by equating a multi-output predictor with multiple tensors, in which case the predictor is associated with multiple tensor ranks.
However, to facilitate a simple presentation, we avoid this extension and simply map each dataset into multiple one-vs-all binary classification problems.
For each problem, we associate the label~$1$ with the active category and $0$ with all the rest, and then attempt to fit training examples with predictors of low tensor rank, reporting the resulting mean squared error, \ie~the residual of the fit.
This is compared against residuals obtained when fitting two types of random data: one generated via shuffling labels, and the other by replacing inputs with noise.

\begin{figure*}
	\begin{center}
		\hspace*{-1mm}
		\includegraphics[width=1\textwidth]{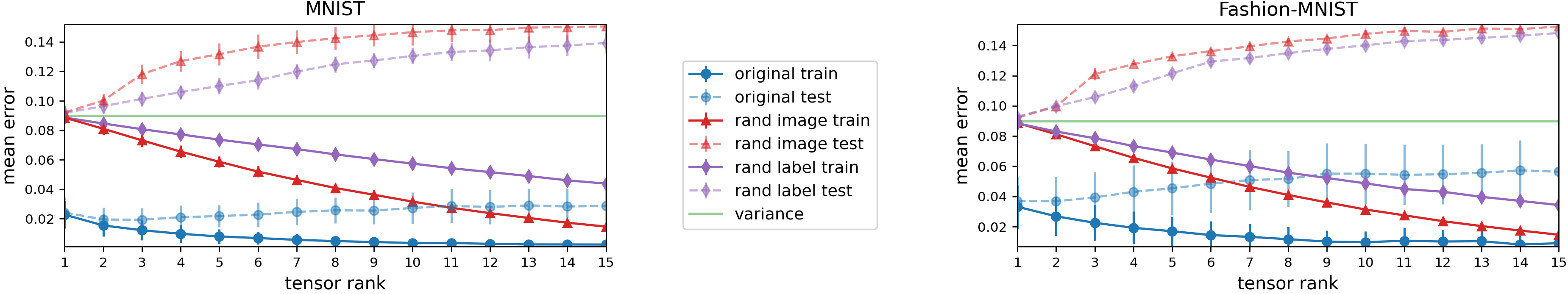}
	\end{center}
	\vspace{-2mm}
	\caption{
		Evaluation of tensor rank as measure of complexity~---~standard datasets can be fit accurately with predictors of extremely low tensor rank (far beneath what is required by random datasets), suggesting it may capture the essence of natural data.
		Left and right plots show results of fitting MNIST and Fashion-MNIST datasets, respectively, with predictors of increasing tensor rank.
		Original datasets are compared against two random variants: one generated by replacing images with noise (``rand image''), and the other via shuffling labels (``rand label'').
		As described in the text (Section~\ref{tf:sec:experiments:tensor_rank_complexity}), for simplicity of presentation, each dataset was mapped into multiple (ten) one-vs-all prediction tasks (label~$1$ for active category, $0$~for the rest), with fit measured via mean squared error.
		Separately for each one-vs-all prediction task and each value $k \in \{ 1 , \ldots , 15 \}$ for the tensor rank, we applied an approximate numerical method (see Appendix~\ref{tf:app:experiments:details:natural_data} for details) to find the predictor of tensor rank~$k$ (or less) with which the mean squared error over training examples is minimal.
		We report this mean squared error, as well as that obtained by the predictor on the test set (to mitigate impact of outliers, large squared errors over test samples were clipped~---~see Appendix~\ref{tf:app:experiments:details:natural_data} for details).
		Plots show, for each value of~$k$, mean (as marker) and standard deviation (as error bar) of these errors taken over the different one-vs-all prediction tasks.
		Notice that the original datasets are fit accurately (low train error) by predictors of tensor rank as low as one, whereas random datasets are not (with tensor rank one, residuals of their fit are close to trivial, \ie~to the variance of the label).
		This suggests that tensor rank as a measure of complexity for predictors has potential to capture the essence of natural data.
		Notice also that, as expected, accurate fit with low tensor rank coincides with accurate prediction on test set, \ie~with generalization.
		For further details, as well as an experiment showing that linear predictors are incapable of accurately fitting the datasets, see Appendix~\ref{tf:app:experiments}.
	}
	\label{tf:fig:mnist_fmnist_rank}
\end{figure*}

Both MNIST and Fashion-MNIST comprise $28$-by-$28$ grayscale images, with each pixel taking one of $256$ possible values.
Tensors associated with predictors are thus of order~$784$, with dimension $256$ in each mode (axis).\footnote{
	In practice, when associating predictors with tensors, it is often beneficial to modify the representation of the input (\cf~\cite{cohen2016expressive}).
	For example, in the context under discussion, rather than having the discrete input variables hold pixel intensities, they may correspond to small image patches, where each patch is represented by the index of a centroid it is closest to, with centroids determined via clustering applied to all patches across all images in the dataset.
	For simplicity, we did not transform representations in our experiments, and simply operated over raw image pixels.
}
A general rank one tensor can then be expressed as a tensor product (\ie~outer product) between $784$ vectors of dimension $256$ each, and accordingly has roughly $784 \cdot 256$ degrees of freedom.
This significantly exceeds the number of training examples in the datasets ($60000$), hence it is no surprise that we could easily fit them, as well as their random variants, with a predictor whose tensor rank is one.
To account for the comparatively small training sets, and render their fit more challenging, we quantized pixels to hold one of two values, \ie~we reduced images from grayscale to black and white.
Following the quantization, tensors associated with predictors have dimension two in each mode, and the number of degrees of freedom in a general rank one tensor is roughly $784 \cdot 2$~---~well below the number of training examples.
We may thus expect to see a difference between the tensor ranks needed for fitting original datasets and those required by the random ones.
This is confirmed by Figure~\ref{tf:fig:mnist_fmnist_rank}, displaying the results of the experiment. 

Figure~\ref{tf:fig:mnist_fmnist_rank} shows that with predictors of low tensor rank, MNIST and Fashion-MNIST can be fit much more accurately than the random datasets.
Moreover, as one would presume, accurate fit with low tensor rank coincides with accurate prediction on unseen data (test set), \ie~with generalization.
Combined with the rest of our results, we interpret this finding as an indication that tensor rank may shed light on both implicit regularization of neural networks, and the properties of natural data translating this implicit regularization to generalization.

%% file: Chapters/imp_reg_htf.tex
\chapter{Implicit Regularization in Hierarchical Tensor Factorization \\ and Deep Convolutional Neural Networks} 
\label{chap:imp_reg_htf}

The contents of this chapter are based on~\cite{razin2022implicit}.

\section{Background and Overview}
\label{htf:sec:overview}

\cref{chap:imp_reg_not_norms,chap:imp_reg_tf} focused on the implicit regularization in matrix and tensor factorization.
Matrix factorization refers to minimizing a given loss (over matrices) by parameterizing the solution as a product of matrices, and optimizing the resulting objective via gradient descent.
Tensor factorization is a generalization of this procedure to multi-dimensional arrays.
There, a tensor is learned through gradient descent over a sum-of-outer-products parameterization (see~\cref{tf:sec:tf}).
By adopting a dynamical viewpoint, in \cref{chap:imp_reg_tf} we established that gradient descent (with small learning rate and near-zero initialization) over tensor factorization induces a momentum-like effect on the components of the factorization, leading them to move slowly when small and quickly when large.
This implies a form of incremental learning that results in low tensor rank solutions, analogous to the incremental rank learning phenomenon identified by \cite{arora2019implicit} for matrix factorization.

From a deep learning perspective, matrix factorization can be seen as a linear neural network, and, in a similar vein, tensor factorization corresponds to a certain shallow (depth two) non-linear convolutional neural network (see~\cref{tf:sec:tf:nn}).
As theoretical surrogates for deep learning, the practical relevance of these models is limited.
The former lacks non-linearity, while the latter misses depth~---~both crucial features of modern neural networks.
A natural extension of matrix and tensor factorizations that accounts for both non-linearity and depth is \emph{hierarchical tensor factorization},\footnote{
	The term “hierarchical tensor factorization'' refers throughout to a variant of the Hierarchical Tucker factorization \cite{hackbusch2009new}, presented in Section~\ref{htf:sec:htf}.
}
which corresponds to a class of \emph{deep non-linear} convolutional neural networks~\cite{cohen2016expressive} (with polynomial non-linearity) that have demonstrated promising performance in practice~\cite{cohen2014simnets,cohen2016deep,sharir2016tensorial,stoudenmire2018learning,grant2018hierarchical,felser2021quantum}, and have been key to the study of expressiveness in deep learning~\cite{cohen2016expressive,cohen2016convolutional,cohen2017inductive,cohen2017analysis,cohen2018boosting,sharir2018expressive,levine2018benefits,levine2018deep,balda2018tensor,khrulkov2018expressive,khrulkov2019generalized,levine2019quantum}.

In this chapter, we provide the first analysis of implicit regularization in hierarchical tensor factorization.
As opposed to tensor factorization, which is a simple construct dating back to at least the early 20'th century~\cite{hitchcock1927expression}, hierarchical tensor factorization was formally introduced only recently~\cite{hackbusch2009new}, and is much more elaborate.
We circumvent the challenges brought forth by the added hierarchy through identification of \emph{local components}, and characterization of their evolution under gradient descent (with small learning rate and near-zero initialization).
The characterization reveals that they are subject to a momentum-like effect, identical to that in matrix and tensor factorizations.
Accordingly, local components are learned incrementally, leading to solutions with low \emph{hierarchical tensor rank}~---~a central concept in tensor analysis~\cite{grasedyck2010hierarchical,grasedyck2013literature}. 
Theoretical and empirical demonstrations validate our analysis.

For the deep convolutional networks corresponding to hierarchical tensor factorization, hierarchical tensor rank is known to measure the strength of dependencies modeled between spatially distant input regions (patches of pixels in the context of image classification)~---~see~\cite{cohen2017inductive,levine2018benefits,levine2018deep}.
The established tendency towards low hierarchical tensor rank therefore implies a bias towards local (short-range) dependencies, in accordance with the fact that convolutional networks often struggle or completely fail to learn tasks entailing long-range dependencies (see, \eg,~\cite{wang2016temporal,linsley2018learning,mlynarski2019convolutional,hong2020graph, kim2020disentangling}).
However, while this failure is typically attributed solely to a limitation in expressive capability (\ie~to an inability of convolutional networks to represent functions modeling long-range dependencies~---~see~\cite{cohen2017inductive,linsley2018learning,kim2020disentangling}), our analysis reveals that it also originates from implicit regularization.
This suggests that the difficulty in learning long-range dependencies may be countered via \emph{explicit} regularization, in contrast to conventional wisdom by which architectural modifications are needed.
Through a series of controlled experiments we confirm this prospect, demonstrating that explicit regularization designed to promote high hierarchical tensor rank can significantly improve the performance of modern convolutional networks (\eg~ResNet18 and ResNet34 from~\cite{he2016deep}) on tasks involving long-range dependencies.

Our results bring forth the possibility that deep learning architectures considered suboptimal for certain tasks (\eg~convolutional networks for natural language processing tasks) may be greatly improved through a right choice of explicit regularization.
Theoretical understanding of implicit regularization may be key to discovering such regularizers.

\medskip

The remainder of the chapter is organized as follows.
For completeness, in Section~\ref{htf:sec:prelim} we outline relevant dynamical characterizations of implicit regularization in matrix and tensor factorizations (from \cite{arora2019implicit} and \cref{chap:imp_reg_tf}, respectively).
Section~\ref{htf:sec:htf} presents the hierarchical tensor factorization model, as well as its interpretation as a deep non-linear convolutional network.
In Section~\ref{htf:sec:inc_rank_lrn} we characterize the dynamics of gradient descent over hierarchical tensor factorization, establishing that they lead to low hierarchical tensor rank.
Section~\ref{htf:sec:low_htr_implies_locality} explains why low hierarchical tensor rank means locality for the corresponding convolutional network.
Lastly, Section~\ref{htf:sec:countering_locality} demonstrates that the locality of modern convolutional networks can be countered using dedicated explicit regularization.


\section{Preliminaries: Matrix and Tensor Factorizations}
\label{htf:sec:prelim}

The dynamical analysis delivered in the current chapter for hierarchical tensor factorization is analogous to prior dynamical analyses for matrix and tensor factorization, carried out in \cite{arora2019implicit} and \cref{chap:imp_reg_tf}, respectively.
For completeness, this section overviews these analyses under a unified notation.

Throughout the chapter, when referring to a norm we mean the standard Frobenius (Euclidean) norm, denoted $\norm{\cdot}$.
For $N \in \N$, we let $[N] := \{ 1, \ldots, N \}$.
For vectors, matrices, or tensors, parenthesized superscripts denote elements in a collection, \eg~$( \wbf^{(n)} \in \R^D )_{ n = 1}^N$, while subscripts refer to entries, \eg~$\Wbf_{i,j} \in \R$ is the $(i,j)$'th entry of $\Wbf \in \R^{D \times D'}$.
A colon indicates all entries in an axis, \eg~$\Wbf_{i, :} \in \R^{D'}$ is the $i$'th  row and $\Wbf_{:, j} \in \R^{D}$ is the $j$'th column of~$\Wbf$.

\begin{figure*}[t]
	\begin{center}
		\hspace{-3mm}
		\includegraphics[width=1\textwidth]{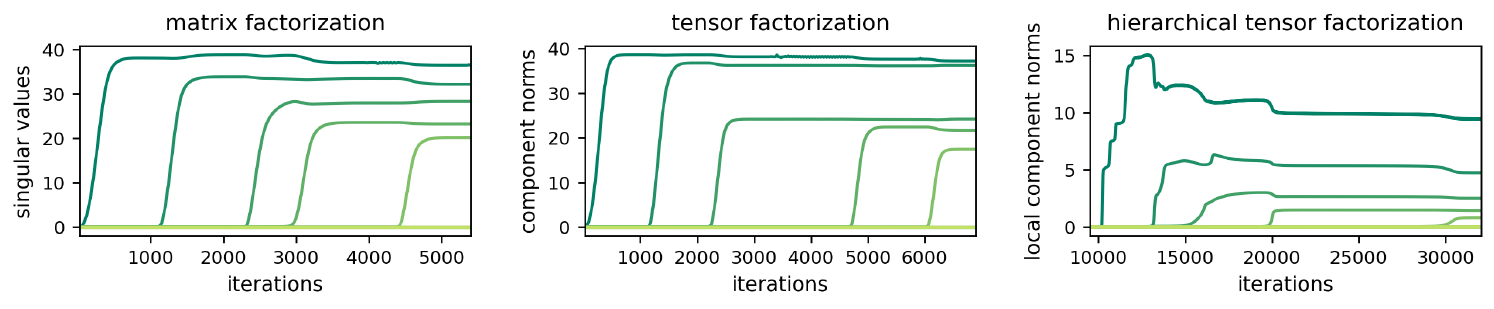}
	\end{center}
	\vspace{-3mm}
	\caption{
		Dynamics of gradient descent over matrix, tensor, and hierarchical tensor factorizations~---~incremental learning leads to low matrix, tensor, and hierarchical tensor ranks, respectively.
		\textbf{Left:} top $10$ singular values of the end matrix in a depth $3$ matrix factorization when minimizing the mean squared error over observed entries from a matrix rank $5$ ground truth (matrix completion loss).
		\textbf{Middle:} top $10$ component norms of an order $3$ tensor factorization when minimizing the mean squared error over observed entries from a tensor rank $5$ ground truth (tensor completion loss).
		\textbf{Right:} top $10$ local component norms at node $\{1, 2, 3, 4 \}$ of an order $4$ hierarchical tensor factorization induced by a perfect binary mode tree (Definition~\ref{htf:def:mode_tree}), when minimizing the mean squared error over observed entries from a hierarchical tensor rank $(5, 5, 5, 5, 5, 5)$ (Definition~\ref{htf:def:ht_rank}) ground truth (tensor completion loss).
		\textbf{All:} initial factorization weights were sampled independently from a zero-mean Gaussian distribution. 
		Notice that, in accordance with existing analyses for matrix and tensor factorizations (Section~\ref{htf:sec:prelim}) and our analysis for hierarchical tensor factorization (Section~\ref{htf:sec:inc_rank_lrn}), the singular values, component norms, and local component norms move slowly when small and quickly when large, creating an incremental learning process that results in effectively low matrix, tensor, and hierarchical tensor rank solutions, respectively.
		In all factorizations this implicit regularization led to accurate reconstruction of the low rank ground truth (reconstruction errors were $0.001$, $0.001$, and $0.005$, respectively).
		For further details such as loss definitions and factorization dimensions, as well as additional experiments for hierarchical tensor factorization, see	Appendix~\ref{htf:app:experiments}.
	}
	\label{htf:fig:mf_tf_htf_dynamics}
\end{figure*}

\subsection{Matrix Factorization: Incremental Matrix Rank Learning}
\label{htf:sec:prelim:mf}

Consider the task of minimizing a differentiable and locally smooth\footnote{
	A differentiable function $g: \R^D \to \R$ is \emph{locally smooth} if for any compact subset $\B \subset \R^D$ there exists $\beta \in \R_{\geq 0}$ such that $\normnoflex{ \nabla g(\xbf) - \nabla g(\ybf) } \leq \beta \cdot \norm{\xbf - \ybf}$ for all $\xbf, \ybf \in \B$.
}
loss $\mfendloss : \R^{D \times D'} \,{\to}\, \R_{\geq 0}$ ($D, D' \in \N$).
For example, $\mfendloss$ can be a matrix completion loss~---~mean squared error over observed entries from a ground truth matrix.
Matrix factorization with hidden dimensions $D_2, \ldots, D_{L} \in \N$ refers to parameterizing the solution $\matrixend \in \R^{D \times D'}$ as a product of $L$ matrices, \ie~as $\matrixend = \Wbf^{(L)} \cdots \Wbf^{(1)}$, where $\Wbf^{(l)} \in \R^{D_{l} \times  D_{l - 1}}$ for $l = 1, \ldots, L$, $D_0 := D'$, and $D_{L} := D$, and minimizing the resulting objective $\mfobj \big ( \Wbf^{(1)}, \ldots, \Wbf^{(L)} \big ) := \mfendloss (\matrixend)$ using gradient descent.
We call $\matrixend$ the \emph{end matrix} of the factorization.
It is possible to explicitly constrain the values that $\matrixend$ can take by limiting the hidden dimensions $D_2, \ldots, D_L$.
However, from an implicit regularization perspective, the case of interest is where the search space is unconstrained, thus we consider $D_2, \ldots, D_L \geq \min \{ D, D' \}$.
Matrix factorization can be viewed as applying a linear neural network for minimizing $\mfendloss$, and as such, serves a prominent theoretical model in deep learning (see \eg~\cite{gunasekar2017implicit,du2018algorithmic,li2018algorithmic,arora2019implicit,gidel2019implicit,mulayoff2020unique,blanc2020implicit,gissin2020implicit,razin2020implicit,chou2020gradient,yun2021unifying,li2021towards}).

Several characterizations of implicit regularization in matrix factorization have suggested that gradient descent, with small learning rate and near-zero initialization, induces a form of incremental matrix rank learning~\cite{gidel2019implicit,gissin2020implicit,chou2020gradient,li2021towards}.
Below we follow the presentation of~\cite{arora2019implicit}, which in line with other analyses, modeled small learning rate through the infinitesimal limit, \ie~via \emph{gradient flow}:
\[
\begin{split}
	\frac{d}{dt} \Wbf^{(l)} (t) = - \frac{\partial}{\partial \Wbf^{(l) } } \mfobj \big ( \Wbf^{(1)} (t), \ldots, \Wbf^{(L)} (t) \big )
\end{split}
\]
for all $t \geq 0$ and $l \in [L]$.
Under gradient flow, the difference $\Wbf^{(l)} (t) \Wbf^{(l)} (t)^\top - \Wbf^{(l + 1)} (t)^\top \Wbf^{(l + 1)} (t)$ remains constant through time for any $l \in [L - 1]$ (see~\cite{arora2018optimization}).
This implies that the \emph{unbalancedness magnitude}, defined as $\max\nolimits_{l} \normnoflex{  \Wbf^{(l)} (t) \Wbf^{(l)} (t)^\top - \Wbf^{(l + 1)} (t)^\top \Wbf^{(l + 1)} (t) }$, does not change through time, thus becomes relatively small as optimization moves away from the origin, more so the closer initialization is to zero.
Accordingly, it is common practice to treat the case of unbalancedness magnitude zero as an idealization of standard near-zero initializations (see, \eg,~\cite{saxe2014exact,arora2018optimization,bartlett2018gradient,lampinen2019analytic,arora2019implicit,elkabetz2021continuous,bah2022learning}).

With unbalancedness magnitude zero, the $r$'th singular value of the end matrix $\matrixend (t) = \Wbf^{(L)} ( t ) \cdots \Wbf^{(1)} ( t )$ ($r \in \left [ \min \{ D, D' \} \right ]$), denoted $\mfsing{r} (t) \in \R$, evolves by (\cf~\cite{arora2019implicit}):\footnote{
	The dynamical characterization of singular values in Equation~\eqref{htf:eq:sing_val_mf_dyn} requires $\mfendloss$ to be analytic, a property met by standard loss functions such as the square and cross-entropy losses.
}
\be
\vspace{0.5mm}
\frac{d}{dt} \mfsing{r} (t) = \mfsing{r} (t)^{2 - \frac{2}{L}} L \inprodbig{ - \nabla \mfendloss ( \matrixend (t) ) }{ \mfcomp{r} (t)  }
\text{\,,}
\label{htf:eq:sing_val_mf_dyn}
\ee
where $\mfcomp{r} (t) := \ubf^{(r)} (t) \vbf^{(r)} (t)^\top \in \R^{D \times D'}$ is the $r$'th singular component of $\matrixend (t)$, meaning $\ubf^{(r)} (t) \in \R^D$ and $\vbf^{(r)} (t) \in \R^{D'}$ are, respectively, left and right singular vectors of $\matrixend (t)$ corresponding to~$\mfsing{r} (t)$.
As evident from Equation~\eqref{htf:eq:sing_val_mf_dyn}, 
two factors govern the evolution rate of a singular value $ \mfsing{r} (t)$.
The first factor, $\inprodnoflex{ - \nabla \mfendloss ( \matrixend (t) ) }{ \mfcomp{r} (t)  }$, is a projection of the singular component $\mfcomp{r} (t)$ onto $- \nabla \mfendloss ( \matrixend (t) )$, the direction of steepest descent with respect to the end matrix.
The more the singular component is aligned with $- \nabla \mfendloss ( \matrixend (t) )$, the faster the singular value grows.
The second, more critical factor, is $\mfsing{r} (t)^{2 - \frac{2}{L } } L$, which implies that the rate of change of the singular value is proportional to its size exponentiated by $2 - 2 / L$ (recall that $L$ is the depth of the matrix factorization).
This gives rise to a momentum-like effect, which attenuates the movement of small singular values and accelerates the movement of large ones.
We may thus expect that if the matrix factorization is initialized near the origin, singular values progress slowly at first, and then, one after the other they reach a critical threshold and quickly rise, until convergence is attained.
Such incremental learning phenomenon leads to low matrix rank solutions.
It is demonstrated empirically in Figure~\ref{htf:fig:mf_tf_htf_dynamics} (left), which reproduces an experiment from~\cite{arora2019implicit}.
We note that under certain technical conditions, the incremental matrix rank learning phenomenon can be used to prove exact matrix rank minimization~\cite{li2021towards}.

\subsection{Tensor Factorization: Incremental Tensor Rank Learning}
\label{htf:sec:prelim:tf}

A depth two matrix factorization boils down to parameterizing a sought-after solution as a sum of tensor (outer) products between column vectors of $\Wbf^{(2)}$ and row vectors of $\Wbf^{(1)}$. 
Namely, since $\matrixend = \Wbf^{(2)} \Wbf^{(1)}$ we may write $\matrixend = \sum_{r = 1}^{R} \Wbf^{(2)}_{:, r} \tenp \Wbf^{(1)}_{r, :}$, where $R$ is the dimension shared between $\Wbf^{(1)}$ and $\Wbf^{(2)}$, and $\tenp$ stands for the tensor product.
Note that the minimal number of summands $R$ required for $\matrixend$ to express a given matrix $\Wbf$ is precisely the latter's matrix rank.

By allowing each summand to be a tensor product of more than two vectors, we may transition from a factorization for matrices to a factorization for tensors.
In tensor factorization, a sought-after solution $\tftensorend \in \R^{D_1 \times \cdots \times D_N}$~---~an \emph{order} $N \geq 3$ tensor with \emph{modes} (axes) of \emph{dimensions} $D_1, \ldots, D_N \in \N$~---~is parameterized as $\tftensorend = \sum_{r = 1}^R \Wbf^{(1)}_{:,r} \tenp \cdots \tenp \Wbf^{(N)}_{:, r}$, where $\Wbf^{(n)} \in \R^{D_n \times R}$ for $n \in [N]$.
Each term $ \Wbf^{(1)}_{:,r} \tenp \cdots \tenp \Wbf^{(N)}_{:, r}$ in this sum is called a \emph{component}, and $\tftensorend$ is referred to as the \emph{end tensor} of the factorization.
Given a differentiable and locally smooth loss $\tfendloss : \R^{D_1 \times \cdots \times D_N} \to \R_{\geq 0}$, \eg~mean squared error over observed entries from a ground truth tensor (\ie~a tensor completion loss), the goal is to minimize the objective $\tfobj \big (\Wbf^{(1)}, \ldots, \Wbf^{(N)} \big ) := \tfendloss ( \tftensorend )$.
In analogy with matrix factorization, the minimal number of components $R$ required for $\tftensorend$ to express a given tensor $\W \in \R^{D_1 \times \cdots \times D_N}$ is defined to be the latter's \emph{tensor rank}, and the case of interest is when $R$ is sufficiently large to not restrict tensor rank (\ie~to admit an unconstrained search space).

Similarly to how matrix factorization corresponds to a linear neural network, tensor factorization is known (see \cref{tf:sec:tf:nn}) to be equivalent to a certain shallow (depth two) non-linear convolutional network (with polynomial non-linearity).
By virtue of this equivalence, illustrated in Figure~\ref{htf:fig:tf_htf_as_convnet} (top), tensor factorization is considered closer to practical deep learning than matrix factorization.

\begin{figure*}[t]
	\begin{center}
		\includegraphics[width=1\textwidth]{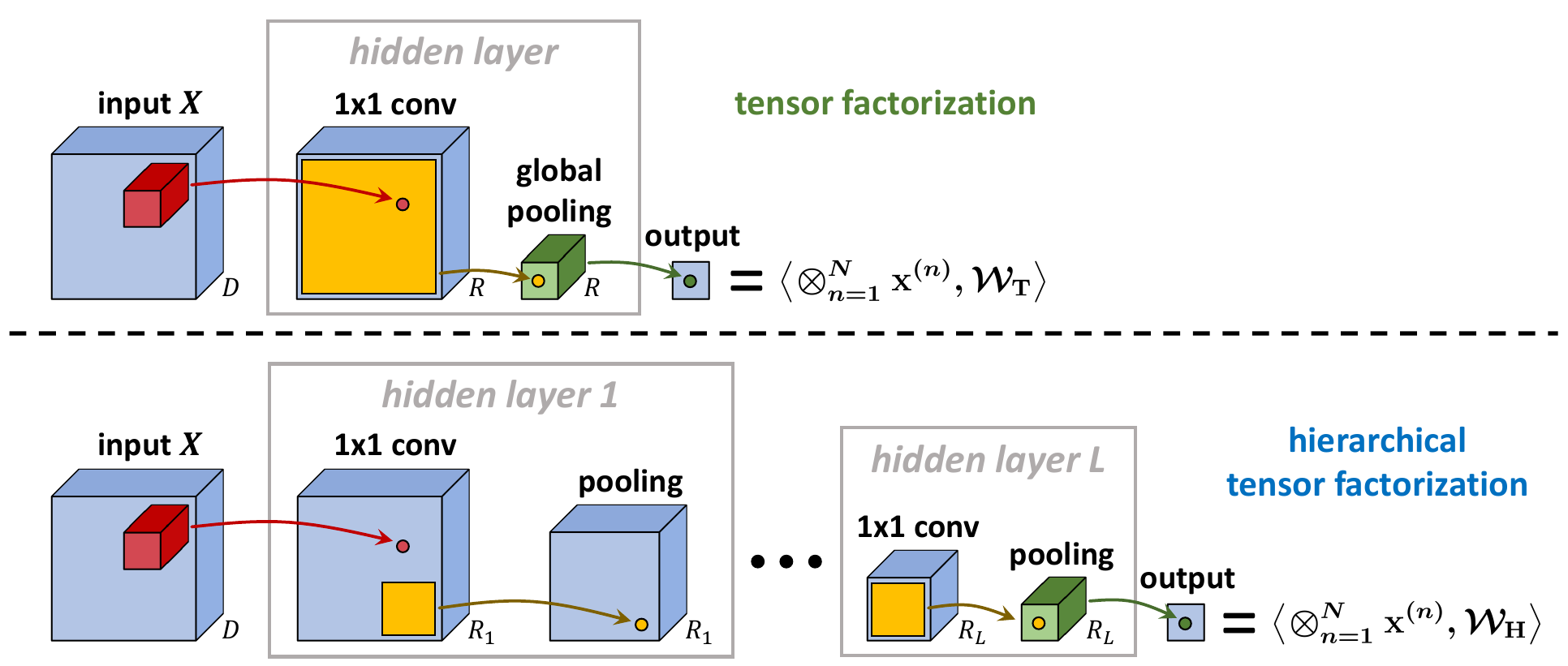}
	\end{center}
	\vspace{-2mm}
	\caption{
		Tensor factorization corresponds to a class of shallow (depth two) non-linear convolutional networks, while hierarchical tensor factorization corresponds to a class of \emph{deep} non-linear convolutional networks (with polynomial non-linearity).
		\textbf{Top:} illustration of the shallow network equivalent to tensor factorization processes.
		The illustration is analogous to \cref{mf:fig:convac} from \cref{chap:imp_reg_not_norms}, and is repeated for ease of comparison with the deep non-linear convolutional network corresponding to hierarchical tensor factorization. 
		Given an input $( \xbf^{(1)}, \ldots, \xbf^{(N)} ) \in \R^{D_1} \times \cdots \times \R^{D_N}$ (illustration assumes $D_1 = \cdots = D_N = D$ to avoid clutter), the network processes it using a single hidden layer, which consists of: \emph{(i)} locally connected linear operator with $R$ channels, computing \smash{$( \Wbf^{(1)} )^\top \xbf^{(1)}, \ldots, ( \Wbf^{(N)} )^\top \xbf^{(N)}$} with learnable weights $\Wbf^{(1)}, \ldots, \Wbf^{(N)}$ (this operator is referred to as ‘‘$1 \times 1$ conv'' in appeal to the common case of weight sharing, \ie~$\Wbf^{(1)} = \cdots = \Wbf^{(N)}$); and \emph{(ii)} channel-wise global product pooling (which induces polynomial non-linearity).
		Summing over the resulting activations then yields the scalar output $\inprodbig{ \tenp_{n = 1}^N \xbf^{(n)} }{ \sum_{r = 1}^R \tenp_{n = 1}^N \Wbf^{(n)}_{:, r} } = \inprodbig{  \tenp_{n = 1}^N \xbf^{(n)}  }{ \tftensorend }$.
		Hence, functions realized by this class of networks are naturally represented via tensor factorization, where the number of components $R$ and the weight matrices $\Wbf^{(1)}, \ldots, \Wbf^{(N)}$ of the factorization correspond to the width and learnable weights of the network, respectively.
		\textbf{Bottom:} for a hierarchical tensor factorization induced by a perfect $P$-ary mode tree (Definition~\ref{htf:def:mode_tree}), the equivalent network is a deep variant of that associated with tensor factorization.
		It has $L = \log_P N$ hidden layers instead of just one, with channel-wise product pooling operating over windows of size $P$ as opposed to globally.
		After passing an input \smash{$( \xbf^{(1)}, \ldots, \xbf^{(N)} ) \in \R^{D_1} \times \cdots \times \R^{D_N}$} through all hidden layers, a final linear layer produces the network's scalar output $\inprodbig{  \tenp_{n = 1}^N \xbf^{(n)}  }{ \tensorend }$, where $\tensorend$ is the end tensor of the hierarchical tensor factorization (Equation~\eqref{htf:eq:ht_end_tensor}), whose weight matrices are equal to the network's learnable weights.
		Thus, functions realized by this class of networks are naturally represented via hierarchical tensor factorization.
		We note that, as shown in~\cite{cohen2016convolutional}, by considering \emph{generalized hierarchical tensor factorizations} it is possible to account for various non-linearities beyond polynomial, \ie~for product pooling being converted to a different pooling operator (\eg~max or average), optionally preceded by a non-linear activation (\eg~rectified linear unit).
	}
	\label{htf:fig:tf_htf_as_convnet}
	\vspace{-2mm}
\end{figure*}

As in matrix factorization, gradient flow over tensor factorization induces invariants of optimization.
In particular, the differences between squared norms of vectors in the same component, \ie~\smash{$\normnoflex{ \Wbf_{:,r}^{ (n) } (t) }^2 - \normnoflex{ \Wbf_{:, r}^{ (n') } (t) }^2$} for $n, n' \in [N]$ and $r \in [R]$, are constant through time (\cf~\cref{tf:lem:balancedness_conservation_body} in \cref{tf:sec:dynamic}).
This leads to the following definition of unbalancedness magnitude: \smash{$\max_{n, n', r} \big | \normnoflex{ \Wbf_{:,r}^{ (n) } (t) }^2 - \normnoflex{ \Wbf_{:, r}^{ (n') } (t) }^2  \big |$}, which does not change during optimization, therefore remains small throughout if initialization is close to the origin.
Under the idealized assumption of unbalancedness magnitude zero (corresponding to infinitesimally small initialization), the norm of the $r$'th component in the factorization ($r \in [R]$), \ie~\smash{$\tfcompnorm{r} (t) := \normnoflex{ \tenp_{n = 1}^N \Wbf^{(n)}_{:, r} (t) }$}, evolves by (\cf~\cref{tf:cor:dyn_fac_comp_norm_balanced} in \cref{tf:sec:dynamic}): 
\be
\frac{d}{dt} \tfcompnorm{r} (t) = \tfcompnorm{r} (t)^{2 - \frac{2}{N}} N \inprodbig{ - \nabla \tfendloss ( \tftensorend (t) ) }{ \tfcomp{r} (t) }
\text{\,,}
\label{htf:eq:comp_norm_tf_dyn}
\ee
where $\tfcomp{r} (t) :=  \tenp_{n = 1}^N \widebar{\Wbf}^{(n)}_{:, r} (t)$, with $\widebar{\Wbf}^{(n)}_{:, r} (t)$ defined as $\Wbf^{(n)}_{:, r} (t) / \normnoflex{\Wbf^{(n)}_{:, r} (t)}$ for all $n \in [N]$ (by convention, if $\Wbf^{(n)}_{:, r} (t) = 0$ then $\widebar{\Wbf}^{(n)}_{:, r} (t) = 0$), denotes the $r$'th normalized component.
Comparing Equation~\eqref{htf:eq:comp_norm_tf_dyn} to Equation~\eqref{htf:eq:sing_val_mf_dyn} reveals that the evolution rate of a component norm in tensor factorization is \emph{structurally identical} to that of a singular value in matrix factorization.
Specifically, it is determined by two factors, analogous to those in Equation~\eqref{htf:eq:sing_val_mf_dyn}:
\emph{(i)}~a projection of the normalized component \smash{$\tfcomp{r} (t)$ onto $ - \nabla \tfendloss ( \tftensorend (t) )$}, which encourages growth of components that are aligned with the direction of steepest descent with respect to the end tensor; 
and
\emph{(ii)}~\smash{$\tfcompnorm{r} (t)^{2 - \frac{2}{N}} N$}, which induces a momentum-like effect, leading component norms to move slower when small and faster when large.
This suggests that, in analogy with matrix factorization, components tend to be learned incrementally, yielding a bias towards low tensor rank.
For completeness, Figure~\ref{htf:fig:mf_tf_htf_dynamics} (middle) demonstrates the phenomenon empirically, reproducing the experiment from~\cref{tf:sec:experiments:dyn}.
Recall that, similarly to the case of matrix factorization, under certain technical conditions we used in \cref{tf:sec:rank} the incremental tensor rank learning phenomenon to prove exact tensor rank minimization.

\section{Hierarchical Tensor Factorization}
\label{htf:sec:htf}

In this section we present the hierarchical tensor factorization model.
We begin by informally introducing the core concepts (\subsect~\ref{htf:sec:htf:informal}), after which we delve into the formal definitions (\subsect~\ref{htf:sec:htf:formal}).

\subsection{Informal Overview and Interpretation as Deep Non-Linear Convolutional Network}
\label{htf:sec:htf:informal}

As discussed in \subsect~\ref{htf:sec:prelim:tf}, tensor factorization produces an order $N$ end tensor through a sum of components, each combining $N$ vectors using the tensor product operator.
It is customary to represent this computation through a shallow tree structure with $N$ leaves, corresponding to the weight matrices $\weightmat{1}, \ldots, \weightmat{N}$, that are directly connected to the root, which computes the end tensor $\tftensorend = \sum_{r = 1}^R \Wbf^{(1)}_{:,r} \tenp \cdots \tenp \Wbf^{(N)}_{:, r}$.

Generalizing this scheme to an arbitrary tree gives rise to hierarchical tensor factorization.
Given a tree, or formally, a \emph{mode tree} of the hierarchical tensor factorization (\defin~\ref{htf:def:mode_tree}), the scheme progresses from leaves to root.
Each internal node combines tensors produced by its children to form higher-order tensors, until finally the root outputs an order $N$ end tensor.
Different mode trees bring about different hierarchical tensor factorizations, which are essentially a composition of many local tensor factorizations, each corresponding to a different location in the mode tree.
We refer to the components of these local tensor factorizations as the \emph{local components} (\defin~\ref{htf:def:local_comp}) of the hierarchical factorization~---~see \fig~\ref{htf:fig:tf_htf_comp_tree} for an illustration.

\begin{figure*}
	\begin{center}
		\includegraphics[width=1\textwidth]{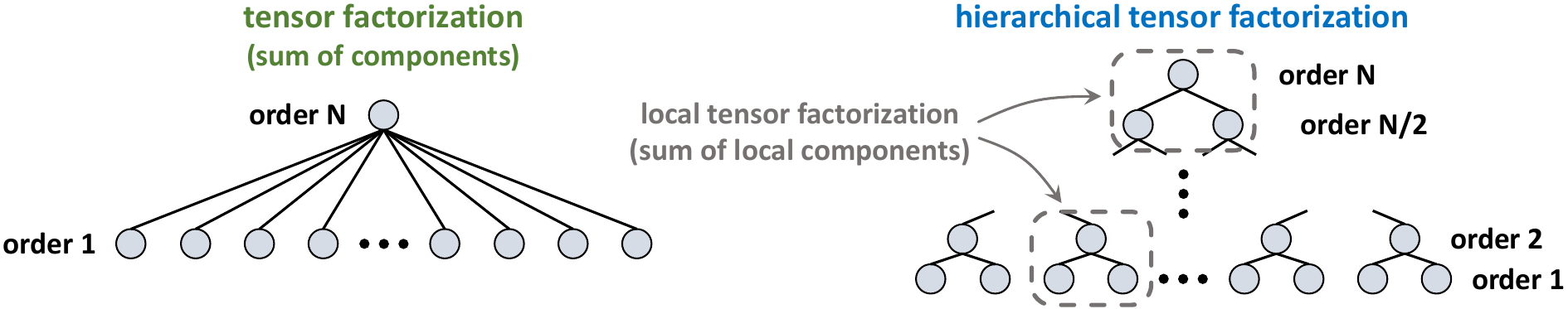}
	\end{center}
	\vspace{-2mm}
	\caption{
		Hierarchical tensor factorization consists of multiple local tensor factorizations.
		\textbf{Left:} tensor factorization represents an order $N$ tensor as a sum of components, each combining $N$ vectors through the tensor product operator.
		Accordingly, it is represented by a shallow tree where all leaves are directly connected to the root.
		\textbf{Right:} hierarchical tensor factorization adheres to an arbitrary tree structure (figure depicts a perfect binary tree), producing an order $N$ tensor by iteratively combining multiple local tensor factorizations.
		The components of the local tensor factorizations constituting the hierarchical tensor factorization are defined to be its local components.
		For a formal description of hierarchical tensor factorization see \subsect~\ref{htf:sec:htf:formal}.
	}
	\label{htf:fig:tf_htf_comp_tree}
	\vspace{-1.5mm}
\end{figure*}

A mode tree of a hierarchical tensor factorization induces a notion of rank called \emph{hierarchical tensor rank} (\defin~\ref{htf:def:ht_rank}).
The hierarchical tensor rank is a tuple whose entries correspond to locations in the mode tree.
The value held by an entry is characterized by the number of local components at the corresponding location, similarly to how tensor rank is characterized by the number of components in a tensor factorization (see \subsect~\ref{htf:sec:prelim:tf}).
Motivated by matrix and tensor ranks being implicitly minimized in matrix and tensor factorizations, respectively (\sect~\ref{htf:sec:prelim}), in \sect~\ref{htf:sec:inc_rank_lrn} we explore the possibility of hierarchical tensor rank being implicitly minimized in hierarchical tensor factorization.
That is, we investigate the prospect of gradient descent (with small learning rate and near-zero initialization) over hierarchical tensor factorization learning solutions that can be represented with few local components at all locations of the mode tree.

\textbf{Equivalence to a class of deep non-linear convolutional networks.}~~As discussed in \sect~\ref{htf:sec:prelim}, matrix factorization can be seen as a linear neural network, and, in a similar vein, tensor factorization corresponds to a certain shallow (depth two) non-linear convolutional network (with polynomial non-linearity).
A drawback of these models as theoretical surrogates for deep learning is that the former lacks non-linearity, while the latter misses depth.
Hierarchical tensor factorization accounts for both of these limitations: for appropriate mode trees, it is known (see~\cite{cohen2016expressive}) to be equivalent to a class of deep non-linear convolutional networks (with polynomial non-linearity).
These networks have demonstrated promising performance in practice~\cite{cohen2014simnets,cohen2016deep,sharir2016tensorial,stoudenmire2018learning,grant2018hierarchical,felser2021quantum}, and their equivalence to hierarchical tensor factorization has been key to the study of expressiveness in deep learning~\cite{cohen2016expressive,cohen2016convolutional,cohen2017inductive,cohen2017analysis,cohen2018boosting,sharir2018expressive,levine2018benefits,levine2018deep,balda2018tensor,khrulkov2018expressive,khrulkov2019generalized,levine2019quantum}.
The equivalence is illustrated in \fig~\ref{htf:fig:tf_htf_as_convnet} (bottom) and rigorously proven in \app~\ref{htf:app:htf_cnn}.

\subsection{Formal Presentation}
\label{htf:sec:htf:formal}
The structure of a hierarchical tensor factorization is determined by a mode tree.
\begin{definition}
	\label{htf:def:mode_tree}
	Let $N \in \N$.
	A \emph{mode tree} $\T$ over $[N]$ is a rooted tree in which:
	\begin{itemize}[topsep=0pt, partopsep=0pt,itemsep=1pt]
		\item every node is labeled by a subset of $[N]$;
		\item there are exactly $N$ leaves, labeled $\{ 1 \}, \ldots, \{ N \}$; and
		\item the label of an interior (non-leaf) node is the union of the labels of its children.
	\end{itemize}
\end{definition}

\noindent
We identify nodes with their labels, \ie~with the corresponding subsets of~$[N]$, and accordingly treat~$\T$ as a subset of~$2^{[N]}$.
Furthermore, we denote the set of all interior nodes by $\interior (\T) \subset \T$, the parent of a non-root node $\nu \in \T \setminus \{ [N] \}$ by $\parent (\nu) \in \T$, and the children of $\nu \in \interior (\T)$ by $\children (\nu) \subset \T$.
When enumerating over children of a node, \ie~over $\children (\nu)$ for $\nu \in \interior (\htmodetree)$, an arbitrary fixed ordering is assumed.

One may consider various mode trees, each leading to a different hierarchical tensor factorization.
Notable choices include: \emph{(i)} a shallow tree (comprising only leaves and root), which reduces the hierarchical tensor factorization to a tensor factorization (\subsect~\ref{htf:sec:prelim:tf}); and \emph{(ii)} a perfect binary tree (applicable if $N$ is a power of $2$) whose corresponding hierarchical tensor factorization is perhaps the most extensively studied.
\fig~\ref{htf:fig:tf_htf_comp_loc_comp}(a) illustrates these two choices.

\begin{figure*}
	\begin{center}
		\includegraphics[width=1\textwidth]{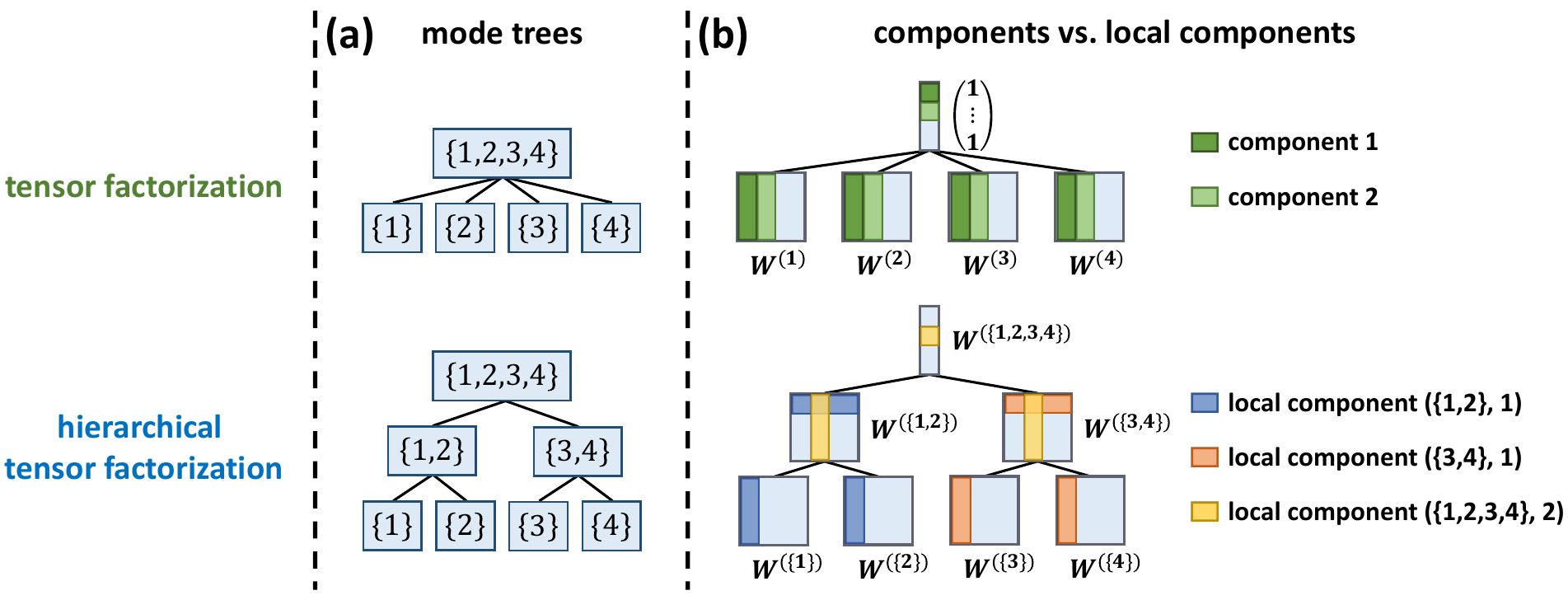}
	\end{center}
	\vspace{-2mm}
	\caption{
		\textbf{(a)} Exemplar mode trees (\defin~\ref{htf:def:mode_tree}) for order $N = 4$ hierarchical tensor factorization.
		Top corresponds to the degenerate case of tensor factorization, while bottom represents the most common choice (perfect binary tree).
		\textbf{(b)} Components of a tensor factorization (top) vs. local components (\defin~\ref{htf:def:local_comp}) of a hierarchical tensor factorization (bottom).
		The $r$'th component of a tensor factorization can be seen as the tensor product between a linear coefficient, which is set to~$1$, and the $r$'th columns of $\Wbf^{(1)}, \ldots, \Wbf^{(N)}$.
		The local components of a hierarchical tensor factorization are the components of the local tensor factorizations forming it. For example, the $r$'th local component at node $\{1, 2\}$ in the hierarchical tensor factorization illustrated above is the tensor product between the $r$'th row of $\weightmat{\{1, 2\}}$ and the $r$'th columns of its children's weight matrices $\weightmat{\{ 1 \}}$ and $\weightmat{\{ 2\}}$.
	}
	\label{htf:fig:tf_htf_comp_loc_comp}
	\vspace{-1.5mm}
\end{figure*}

Along with a mode tree $\htmodetree$, what defines a hierarchical tensor factorization are the number of local components at each interior node, denoted $( R_{\nu} \in \N )_{\nu \in \interior ( \htmodetree )}$.
The induced hierarchical tensor factorization is parameterized by weight matrices $( \weightmat{\nu} \in \R^{R_\nu \times R_{\parent (\nu)} } )_{ \nu \in \htmodetree}$, where $R_{\parent ( [N] )} := 1$ and $R_{\{1\}} := D_1, \ldots, R_{\{N\}} := D_N$.
It creates the \emph{end tensor} $\tensorend \in \R^{D_1 \times \cdots \times D_N}$ by constructing intermediate tensors of increasing order while traversing $\htmodetree$ from leaves to root as follows:
\be
\begin{split}
	&\text{for all $\nu \in \left \{ \{1\}, \ldots, \{ N\} \right \}$ and $r \in [ R_{ \parent (\nu ) } ]$:} \\
	& \quad\underbrace{ \htftensorpart{\nu}{r} }_{\text{order $1$}} := \weightmat{\nu}_{:, r} \text{\,,} \\[1mm]
	&\text{for all $\nu \in \interior (\htmodetree) \setminus \left \{ [N] \right \}$ and $r \in [ R_{\parent (\nu )} ]$ (traverse interior nodes of $\htmodetree$ } \\[-1.2mm]
	& \text{from leaves to root, non-inclusive):} \\
	&\quad \underbrace{\htftensorpart{\nu}{r}}_{\text{order $\abs{\nu}$}} := \pi_\nu \biggl ( \sum\nolimits_{r' = 1}^{R_\nu} \weightmat{\nu}_{r', r} \Bigl [ \tenp_{\nu_c \in \children ( \nu )} \htftensorpart{ \nu_c }{ r' } \Bigr ] \biggr ) \text{\,,} \\[0mm]
	&\underbrace{\tensorend}_{\text{order $N$}} := \pi_{ [N] } \biggl ( \sum\nolimits_{r' = 1}^{R_{ [N] }} \weightmat{ [N] }_{r', 1} \Bigl [ \tenp_{ \nu_c \in \children ( [N] )} \htftensorpart{\nu_c}{r'} \Bigr ] \biggr ) 
	\text{\,,}
\end{split}
\label{htf:eq:ht_end_tensor}
\ee
where $\pi_\nu$, for $\nu \in \htmodetree$, is a \emph{mode permutation} operator which arranges the modes (axes) of its input such that they comply with an ascending order of $\nu$.\footnote{
	For $\nu \in \interior (\htmodetree)$, denote its $K := \abs{ \children (\nu) }$ children by $\nu_1, \ldots, \nu_K$, and the elements of $\nu_k$ by $j^k_1 < \cdots < j^k_{ \abs{\nu_k}}$ for $k \in [K]$.
	Let $h : [\abs{\nu}] \to [\abs{\nu}]$ be the permutation sorting $\big ( j^1_1 , \ldots,  j^1_{ \abs{\nu_1} } ,  \ldots, j^{K}_1, \ldots, j^K_{ \abs{\nu_K} } \big )$ in ascending order.
	Then, the mode permutation operator for $\nu$ is defined by: $\pi_\nu (\W)_{d_1, \ldots, d_{ \abs{\nu} } } = \W_{d_{h(1)}, \ldots, d_{h ( \abs{\nu} ) }}$, where $\W$ is an order $\abs{ \nu }$ tensor.
}

Hierarchical tensor factorization can be viewed as a composition of multiple local tensor factorizations, one for each interior node in the mode tree.
The local tensor factorization for $\nu \in \interior (\htmodetree)$ comprises $R_\nu$ components, referred to as the local components at node $\nu$ of the hierarchical tensor factorization~---~see \fig~\ref{htf:fig:tf_htf_comp_loc_comp}(b) for an illustration, and \defin~\ref{htf:def:local_comp} below.

\begin{definition}
	\label{htf:def:local_comp}
	For $\nu \in \interior (\htmodetree)$ and $ r \in [R_\nu]$, the $(\nu, r)$'th \emph{local component} of the hierarchical tensor factorization is $\weightmat{\nu}_{r, :} \! \tenp \! \bigl ( \tenp_{\nu_c \in \children (\nu)} \weightmat{\nu_c}_{:, r} \bigr )$.
	We use $\localcomp (\nu, r)$ to denote the set comprising $\weightmat{\nu}_{r, :}$ and $\big ( \weightmat{\nu_c}_{:, r} \big )_{\nu_c \in \children ( \nu ) }$, and \smash{$\htfcompnorm{\nu}{r} := \normbig{ \tenp_{\wbf \in \localcomp (\nu, r)} \wbf }$} to denote the norm of the $(\nu, r)$'th local component.
\end{definition}

Mode trees of hierarchical tensor factorizations give rise to the notion of hierarchical tensor rank (\cf~\cite{grasedyck2013literature}), which is based on matrix ranks of specific \emph{matricizations} of a tensor (\cf~\subsect~3.4 in~\cite{kolda2006multilinear}).

\begin{definition}
	\label{htf:def:matricization}
	The \emph{matricization} of $\W \in \R^{D_1 \times \cdots \times D_N}$ with respect to $I \subset [N]$, denoted $\mat{\W}{I} \in \R^{\prod_{i \in I} D_i \times \prod_{j \in [N] \setminus I} D_j}$, is its arrangement as a matrix where rows correspond to modes indexed by $I$ and columns correspond to the remaining modes.\footnote{
		Denoting the elements in $I$ by $i_1 < \cdots < i_{\abs{I}}$ and those in $[N] \setminus I$ by $j_1 < \cdots < j_{ N - \abs{I} }$, the matricization $\mat{\W}{I}$ holds the entries of $\W$ such that $\W_{d_1, \ldots, d_N}$ is placed in row index $1 + \sum_{l = 1}^{\abs{I}} (d_{i_{l}} - 1) \prod_{l' = 1}^{l - 1} D_{i_{l'}}$ and column index $1 + \sum_{l = 1}^{  N - \abs{I} } ( d_{ j_{l} } - 1 ) \prod_{l' = 1}^{l - 1} D_{ j_{l'} }$.
	}
\end{definition}

\begin{definition}
	\label{htf:def:ht_rank}
	The \emph{hierarchical tensor rank} of $\W \in \R^{D_1 \times \cdots \times D_N}$ with respect to mode tree $\T$ is the tuple comprising the matrix ranks of $\W$'s matricizations according to all nodes in $\htmodetree$ except for the root, \ie~$( \rank \mat{ \W }{ \nu  } )_{\nu \in {\T \setminus \{ [N] \} }}$.
	The order of entries in the tuple does not matter as long as it is consistent.
\end{definition}
\noindent
Unless stated otherwise, when referring to the hierarchical tensor rank of a hierarchical tensor factorization's end tensor, the rank is with respect to the mode tree of the factorization. 
Hierarchical tensor rank differs markedly from tensor rank.
Specifically, even when the hierarchical tensor rank is low, \ie~the matrix ranks of matricizations according to all nodes in the mode tree are low, the tensor rank is typically extremely high (exponential in the order of the tensor~---~see~\cite{cohen2016deep}).

Lemma~\ref{htf:lem:loc_comp_ht_rank_bound} below states that the number of local components in a hierarchical tensor factorization controls the hierarchical tensor rank of its end tensor.
More precisely, $R_\nu$~---~the number of local components at $\nu \in \interior (\htmodetree)$~---~upper bounds the matrix rank of matricizations according to the children of $\nu$.

\begin{lemma}[adaptation of \thm~7 in~\cite{cohen2018boosting}]
	\label{htf:lem:loc_comp_ht_rank_bound}
	For any interior node $\nu \in \interior (\htmodetree)$ and child $\nu_c \in \children (\nu)$, it holds that $\rank \mat{\tensorend }{ \nu_c} \leq R_{\nu}$.
\end{lemma}
\begin{proof}
	Deferred to \subapp~\ref{htf:app:proofs:loc_comp_ht_rank_bound}.
\end{proof}
\vspace{-1mm}
\noindent
We may explicitly restrict the hierarchical tensor rank of end tensors $\tensorend$ (\eq~\eqref{htf:eq:ht_end_tensor}) by limiting $( R_\nu )_{\nu \in \interior (\htmodetree)}$.
However, since our interest lies in the implicit regularization of gradient descent, \ie~in the types of end tensors it will find without explicit constraints, we treat the case where $( R_\nu )_{\nu \in \interior (\htmodetree)}$ can be arbitrarily large.

Given a differentiable and locally smooth loss $\htfendloss : \R^{D_1 \times \cdots \times D_N} \to \R_{\geq 0}$, we consider parameterizing the solution as a hierarchical tensor factorization (\eq~\eqref{htf:eq:ht_end_tensor}), and optimizing the resulting (non-convex) objective:
\be
\htfobj \big ( \big ( \weightmat{\nu} \big )_{\nu \in \htmodetree} \big ) := \htfendloss ( \tensorend )
\text{\,.}
\label{htf:eq:htf_obj}
\ee
In line with analyses of implicit regularization in matrix and tensor factorizations (see \sect~\ref{htf:sec:prelim}), we model small learning rate for gradient descent via gradient flow:
\be
\frac{d}{dt} \weightmat{\nu} (t) = - \frac{\partial}{ \partial \weightmat{\nu} } \htfobj \big (  \big ( \weightmat{\nu'} (t) \big )_{\nu' \in \htmodetree} \big )
\label{htf:eq:gf_htf}
\ee
for all $t \geq 0$ and $\nu \in \htmodetree$, where $( \weightmat{\nu} (t) )_{\nu \in \htmodetree}$ denote the weight matrices at time $t$ of optimization.

Over matrix and tensor factorizations, gradient flow initialized near zero is known to minimize matrix and tensor ranks, respectively (\sect~\ref{htf:sec:prelim}).
In particular, it leads to solutions that can be represented using few components.
A natural question that arises is whether a similar phenomenon takes place in hierarchical tensor factorization: does gradient flow with small initialization learn solutions that can be represented with few local components at all locations of the mode tree?
That is, does it learn solutions of low hierarchical tensor rank?
In \sect~\ref{htf:sec:inc_rank_lrn} we answer this question affirmatively.

\section{Incremental Hierarchical Tensor Rank Learning}
\label{htf:sec:inc_rank_lrn}

In this section we theoretically analyze the implicit regularization in hierarchical tensor factorization.
Our analysis extends known results for matrix and tensor factorizations outlined in \sect~\ref{htf:sec:prelim}.
In particular, we show that the implicit regularization in hierarchical tensor factorization induces an incremental learning process that results in low hierarchical tensor rank, similarly to how matrix and tensor factorizations incrementally learn solutions with low matrix and tensor ranks, respectively.
To facilitate this extension, while overcoming the challenges arising from the complexity of the hierarchical tensor factorization model, we characterize the evolution of the local components introduced in \sect~\ref{htf:sec:htf}.
Our analysis is delivered in \subsects~\ref{htf:sec:inc_rank_lrn:evolution},~\ref{htf:sec:inc_rank_lrn:imp_htr_min}, and~\ref{htf:sec:inc_rank_lrn:partial_order}.
For the convenience of the reader, \subsect~\ref{htf:sec:inc_rank_lrn:informal} provides an informal overview.

\subsection{Informal Overview}
\label{htf:sec:inc_rank_lrn:informal}

As discussed in \sect~\ref{htf:sec:prelim}, for both matrix and tensor factorizations, there exists an invariant of optimization whose deviation from zero is referred to as unbalancedness magnitude, and it is common to treat the case of unbalancedness magnitude zero as an idealization of standard near-zero initializations.
With unbalancedness magnitude zero, singular values in a matrix factorization evolve by \eq~\eqref{htf:eq:sing_val_mf_dyn}, and component norms in a tensor factorization move per \eq~\eqref{htf:eq:comp_norm_tf_dyn}.
Equations~\eqref{htf:eq:sing_val_mf_dyn} and~\eqref{htf:eq:comp_norm_tf_dyn} are structurally identical, and are interpreted as implying incremental learning of singular values and component norms, respectively, \ie~of matrix and tensor ranks, respectively.
This interpretation was initially supported by experiments (such as those reported in \fig~\ref{htf:fig:mf_tf_htf_dynamics} (left and middle)), and later via proofs of exact matrix and tensor rank minimization under certain technical conditions.

In \subsect~\ref{htf:sec:inc_rank_lrn:evolution} we show that in analogy with matrix and tensor factorizations, hierarchical tensor factorization entails an invariant of optimization (\lem~\ref{htf:lem:loc_comp_sq_norm_diff_invariant}), which leads to a corresponding notion of unbalancedness magnitude (\defin~\ref{htf:def:unbal_mag}).
For the canonical case of unbalancedness magnitude zero (corresponding to standard near-zero initializations), we prove that the norm of the $r$'th local component associated with node $\nu$ in the mode tree, denoted $\htfcompnorm{\nu}{r} (t)$, evolves by (\thm~\ref{htf:thm:loc_comp_norm_bal_dyn}):
\be
\vspace{3.5mm}
\frac{d}{dt} \htfcompnorm{\nu}{r} \! (t) \! = \!  \htfcompnorm{\nu}{r} \! (t)^{2 - \frac{2}{ L_{\nu} } } L_\nu \inprodbig{- \nabla \htfendloss ( \tensorend (t) ) }{ \htfcomp{\nu}{r} (t) }
\text{,}
\label{htf:eq:informal_loc_comp_norm_bal_dyn}
\ee
where $L_\nu$ is the number of weight vectors in the local component and $\htfcomp{\nu}{r} (t)$ is the direction it imposes on the end tensor $\tensorend (t)$.
\app~\ref{htf:app:dyn_arbitrary} generalizes the above theorem by relieving the assumption of unbalancedness magnitude zero.
Namely, it establishes that \eq~\eqref{htf:eq:informal_loc_comp_norm_bal_dyn} holds approximately when unbalancedness magnitude at initialization is small.
\eq~\eqref{htf:eq:informal_loc_comp_norm_bal_dyn} is structurally identical to Equations~\eqref{htf:eq:sing_val_mf_dyn} and~\eqref{htf:eq:comp_norm_tf_dyn}, therefore the evolution rate of a local component norm in hierarchical tensor factorization mirrors the evolution rates of a singular value in matrix factorization and a component norm in tensor factorization.
One is thus led to interpret \eq~\eqref{htf:eq:informal_loc_comp_norm_bal_dyn} as implying incremental learning of local component norms, \ie~of hierarchical tensor rank (see \sect~\ref{htf:sec:htf}).
We support this interpretation through experiments analogous to those typically conducted for supporting the interpretation of Equations~\eqref{htf:eq:sing_val_mf_dyn} and~\eqref{htf:eq:comp_norm_tf_dyn} as implying incremental learning of matrix and tensor ranks, respectively~---~see \fig~\ref{htf:fig:mf_tf_htf_dynamics} (right) as well as \app~\ref{htf:app:experiments}.
Moreover, we consider technical conditions similar to those assumed for proving exact matrix and tensor rank minimization by matrix and tensor factorizations, respectively, and establish theoretical results aimed at facilitating a proof of exact hierarchical tensor rank minimization~---~see \subsect~\ref{htf:sec:inc_rank_lrn:imp_htr_min}.
Completing the missing steps for deriving such a proof is regarded as a promising direction for future work.

Lastly, we discuss the fact that hierarchical tensor rank does not adhere to a natural total ordering, and the potential of partially ordered complexity measures to further our understanding of implicit regularization in deep learning.
See \subsect~\ref{htf:sec:inc_rank_lrn:partial_order} for details.

\subsection{Evolution of Local Component Norms}
\label{htf:sec:inc_rank_lrn:evolution}

\lem~\ref{htf:lem:loc_comp_sq_norm_diff_invariant} below establishes an invariant of optimization: the differences between squared norms of weight vectors in the same local component are constant through time.
\begin{lemma}
	\label{htf:lem:loc_comp_sq_norm_diff_invariant}
	For all $\nu \in \interior (\htmodetree)$, $r \in [R_\nu]$, and $\wbf, \wbf' \in \localcomp (\nu, r)$:
	\[
	\norm{ \wbf (t) }^2 - \norm{ \wbf' (t) }^2 = \norm{ \wbf (0) }^2 - \norm{ \wbf' (0) }^2 \quad , t \geq 0
	\text{\,.}
	\]
\end{lemma}
\begin{proof}[Proof sketch (proof in \subapp~\ref{htf:app:proofs:loc_comp_sq_norm_diff_invariant})]
We show that $\tfrac{d}{dt} \normnoflex{ \wbf (t) }^2 = \tfrac{d}{dt} \normnoflex{\wbf' (t) }^2$ for all $t \geq 0$.
Then, integrating both sides with respect to time completes the proof.
\end{proof}
\noindent
The above invariant leads to the following definition of unbalancedness magnitude.
\begin{definition}
	\label{htf:def:unbal_mag}
	The \emph{unbalancedness magnitude} of a hierarchical tensor factorization (\eq~\eqref{htf:eq:ht_end_tensor}) is:
	\[
	\max\nolimits_{ \nu \in \interior (\htmodetree) , r \in [R_\nu] , \wbf , \wbf' \in \localcomp (\nu, r) } \abs{ \norm{\wbf }^2 - \norm{\wbf'}^2 }
	\text{\,.}
	\]
\end{definition}
\noindent
\lem~\ref{htf:lem:loc_comp_sq_norm_diff_invariant} implies that the unbalancedness magnitude remains constant throughout optimization.
In the common regime of near-zero initialization, it will start off small, and stay small throughout.
The closer initialization is to zero, the smaller the unbalancedness magnitude is.
In accordance with analyses for matrix and tensor factorizations (see \sect~\ref{htf:sec:prelim}), we treat the case of unbalancedness magnitude zero as an idealization of standard near-zero initializations.
\thm~\ref{htf:thm:loc_comp_norm_bal_dyn} analyzes this case, characterizing the dynamics for norms of local components.

\begin{theorem}
	\label{htf:thm:loc_comp_norm_bal_dyn}
	Assume unbalancedness magnitude zero at initialization.
	Let $\tensorend (t)$ denote the end tensor (\eq~\eqref{htf:eq:ht_end_tensor}) and $\big ( \htfcompnorm{\nu}{r} (t) \big )_{\nu \in \interior (\htmodetree), r \in [R_\nu] }$ denote the norms of local components (\defin~\ref{htf:def:local_comp}) at time $t \geq 0$ of optimization.
	Then, for any $\nu \in \interior (\htmodetree)$ and $r \in[R_\nu]$:
	\be
	\frac{d}{dt} \htfcompnorm{\nu}{r} \! (t) \! = \!  \htfcompnorm{\nu}{r} \! (t)^{2 - \frac{2}{ L_{\nu} } } L_\nu \inprodbig{- \nabla \htfendloss ( \tensorend (t) ) }{ \htfcomp{\nu}{r} (t) }
	\text{\,,}
	\label{htf:eq:loc_comp_norm_bal_dyn}
	\ee
	where $L_\nu := \abs{ \children (\nu) } + 1$ is the number of weight vectors in a local component at node $\nu$, and $\htfcomp{\nu}{r} (t) \in \R^{D_1 \times \cdots \times D_N}$ is the end tensor obtained by normalizing the $r$'th local component at node $\nu$ and setting all other local components at node $\nu$ to zero, \ie~by replacing  in \eq~\eqref{htf:eq:ht_end_tensor} $\htftensorpart{\nu}{r'}$ with $\pi_{\nu} \big ( \big ( \htfcompnorm{\nu}{r} \big )^{-1} \weightmat{\nu}_{r, r'} \big [  \tenp_{\nu_c \in \children (\nu)} \htftensorpart{\nu_c}{r}  \big ] \big )$ for all $r' \in [ R_{\parent ( \nu )}]$.
	By convention, $\htfcomp{\nu}{r} (t) = 0$ if $\htfcompnorm{\nu}{r} (t) = 0$.
\end{theorem}
\begin{proof}[Proof sketch (proof in \subapp~\ref{htf:app:proofs:loc_comp_norm_bal_dyn})]
	If $\htfcompnorm{\nu}{r} (t)$ is zero at some $t \geq 0$, then we show that it must be identically zero through time, leading both sides of \eq~\eqref{htf:eq:loc_comp_norm_bal_dyn} to be equal to zero.
	Otherwise, differentiating the local component's norm with respect to time, we obtain:
	\[
	\begin{split}
		\tfrac{d}{dt} \htfcompnorm{\nu}{r} (t) =
		\inprodbig{ - \nabla \htfendloss ( 	\tensorend (t)) }{ \htfcomp{\nu}{r} (t) }  \cdot \sum\nolimits_{ \wbf \in \localcomp (\nu, r) } \prod\nolimits_{ \wbf' \in \localcomp (\nu, r) \setminus \{ \wbf \} } \norm{ \wbf' (t) }^2
		\text{\,.}
	\end{split}
	\]
	Since the unbalancedness magnitude is zero at initialization, \lem~\ref{htf:lem:loc_comp_sq_norm_diff_invariant} implies that $\normnoflex{ \wbf (t) }^2 = \normnoflex{ \wbf' (t) }^2 = \htfcompnorm{\nu}{r} (t)^{2 / L_\nu}$ for all $\wbf, \wbf' \in \localcomp (\nu, r)$, which together with the expression above for $\tfrac{d}{dt} \htfcompnorm{\nu}{r} (t)$ establishes \eq~\eqref{htf:eq:loc_comp_norm_bal_dyn}.
\end{proof}
\vspace{-0.5mm}

As can be seen from \eq~\eqref{htf:eq:loc_comp_norm_bal_dyn}, the evolution of local component norms in a hierarchical tensor factorization is structurally identical to the evolution of singular values in matrix factorization (\eq~\eqref{htf:eq:sing_val_mf_dyn}) and component norms in tensor factorization (\eq~\eqref{htf:eq:comp_norm_tf_dyn}).
Specifically, it is dictated by two factors: a projection term, $\inprodbig{- \nabla \htfendloss ( \tensorend (t) ) }{ \htfcomp{\nu}{r} (t) }$, and a self-dependence term, $\htfcompnorm{\nu}{r} (t)^{2 - \frac{2}{ L_{\nu} } }  L_\nu$.
Analogous to a singular component $\mfcomp{r} (t)$ in matrix factorization and a normalized component $\tfcomp{r} (t)$ in tensor factorization, $\htfcomp{\nu}{r} (t)$ is the direction that the $(\nu, r)$'th local component imposes on $\tensorend (t)$.\footnote{
	Indeed, just as in matrix factorization $\matrixend = \sum_{r} \mfsing{r} \cdot \mfcomp{r}$, and in tensor factorization $\tftensorend = \sum_{r} \tfcompnorm{r} \cdot \tfcomp{r}$, the end tensor of a hierarchical tensor factorization decomposes as $\tensorend  = \sum_{r = 1}^{R_\nu} \htfcompnorm{\nu}{r} \cdot \htfcomp{\nu}{r}$ (implied by \lems~\ref{htf:lem:ht_multilinear} and~\ref{htf:lem:tensorend_comp_eq_zeroing_cols_or_rows} in \app~\ref{htf:app:proofs}).
}
The projection of $\htfcomp{\nu}{r} (t)$ onto $- \nabla \htfendloss ( \tensorend (t) )$ therefore promotes growth of local components that align $\tensorend (t)$ with $- \nabla \htfendloss ( \tensorend (t) )$, the direction of steepest descent.
More critical is the self-dependence term, $\htfcompnorm{\nu}{r} (t)^{2 - \frac{2}{ L_{\nu} } } L_\nu$, which induces a momentum-like effect that attenuates the movement of small local components and accelerates the movement of large ones.
It suggests that, in analogy with matrix and tensor factorizations, local components tend to be learned incrementally, yielding a bias towards low hierarchical tensor rank.
This prospect is affirmed empirically in \fig~\ref{htf:fig:mf_tf_htf_dynamics} (right) as well as \app~\ref{htf:app:experiments}, and is supported theoretically in \subsect~\ref{htf:sec:inc_rank_lrn:imp_htr_min}.

\vspace{-1.25mm}

\textbf{Evolution of local component norms under arbitrary initialization.}~~\thm~\ref{htf:thm:loc_comp_norm_bal_dyn} can be extended to account for arbitrary initialization, \ie~for initialization with unbalancedness magnitude different from zero.
For conciseness we defer this extension to \app~\ref{htf:app:dyn_arbitrary}, while noting that if initialization has small unbalancedness magnitude~---~as is the case with any near-zero initialization~---~then local component norms approximately evolve per \eq~\eqref{htf:eq:loc_comp_norm_bal_dyn}, \ie~the result of \thm~\ref{htf:thm:loc_comp_norm_bal_dyn} approximately holds.

\subsection{Implicit Hierarchical Tensor Rank Minimization}
\label{htf:sec:inc_rank_lrn:imp_htr_min}

As discussed in \sect~\ref{htf:sec:prelim}, under certain technical conditions, the incremental matrix and tensor rank learning phenomena, induced by the implicit regularization in matrix and tensor factorizations, can be used to prove exact matrix and tensor rank minimization, respectively.
Below we consider similar technical conditions, and provide theoretical results aimed at facilitating an analogous proof for hierarchical tensor factorization, \ie~a proof that its implicit regularization leads to exact hierarchical tensor rank minimization.
We begin by illustrating how, under said conditions, the incremental hierarchical tensor rank learning phenomenon established in \thm~\ref{htf:thm:loc_comp_norm_bal_dyn} leads to solutions with many small local components (\subsect~\ref{htf:sec:inc_rank_lrn:imp_htr_min:illustration}).
We then show that this implies proximity to low hierarchical tensor rank (\subsect~\ref{htf:sec:inc_rank_lrn:imp_htr_min:proximity}).
Throughout the above, the main step missing in order to derive a complete proof of exact hierarchical tensor rank minimization, is confirmation that a certain alignment inequality (\eq~\eqref{htf:eq:alignment_assump}) holds throughout optimization.
We regard this as an important direction for future work.

\subsubsection{Illustrative Demonstration of Small Local Components}
\label{htf:sec:inc_rank_lrn:imp_htr_min:illustration}

Below we qualitatively demonstrate how the dynamical characterization derived in \subsect~\ref{htf:sec:inc_rank_lrn:evolution} implies that the implicit regularization in hierarchical tensor factorization can lead to solutions with small local components.
Under the setting and notation of \thm~\ref{htf:thm:loc_comp_norm_bal_dyn}, consider an initialization $\big (  \Ubf^{(\nu)}  \in \R^{R_\nu \times R_{\parent (\nu)} } \big )_{\nu \in \htmodetree}$ for the weight matrices of the hierarchical tensor factorization, scaled by $\alpha \in \R_{ > 0}$.
That is, $\weightmat{\nu} (0) = \alpha \cdot \Ubf^{(\nu)}$ for all $\nu \in \htmodetree$.
Focusing on some interior node $\nu \in \interior (\htmodetree)$, let $r, \bar{r} \in [R_\nu]$, and assume for simplicity that $\nu$ is not degenerate, in the sense that~it has more than one child.
Suppose also that at initialization the norm of the $(\nu, r)$'th local component is greater than the norm of the $(\nu, \bar{r})$'th local component, \ie~$\htfcompnorm{\nu}{r} (0) > \htfcompnorm{\nu}{\bar{r}} (0)$, and that $\htfcomp{\nu}{r} (t)$ is at least as aligned as $\htfcomp{\nu}{\bar{r}} (t)$ with the direction of steepest descent up to a time $T > 0$, \ie~for all $t \in [0, T]$: 
\be
\inprodbig{- \nabla \htfendloss ( \tensorend (t) ) }{ \htfcomp{\nu}{r} (t) } \geq \inprodbig{- \nabla \htfendloss ( \tensorend (t) ) }{ \htfcomp{\nu}{\bar{r}} (t) }
\text{.}
\label{htf:eq:alignment_assump}
\ee
Then, by \thm~\ref{htf:thm:loc_comp_norm_bal_dyn} for all $t \in [0, T]$:\footnote{
	A local component cannot reach the origin unless it was initialized there (implied by \lem~\ref{htf:lem:bal_zero_stays_zero} in \app~\ref{htf:app:proofs}).
	Accordingly, we disregard the trivial case where $\htfcompnorm{\nu}{\bar{r}} (t) = 0$ for some $t \in [0, T]$.
}
\[
\htfcompnorm{\nu}{\bar{r}} (t)^{ - 2 + \frac{2}{L_\nu}} \frac{d}{dt} \htfcompnorm{\nu}{\bar{r}} (t) \leq  \htfcompnorm{\nu}{r} (t)^{- 2 + \frac{2}{L_\nu} } \frac{d}{dt} \htfcompnorm{\nu}{r} (t)
\text{\,.}
\]
Integrating both sides with respect to time, we may upper bound $\htfcompnorm{\nu}{\bar{r}} (t)$ with a function of $\htfcompnorm{\nu}{r} (t)$:
\be
\htfcompnorm{\nu}{\bar{r}} (t) \leq \Big [ \htfcompnorm{\nu}{r} (t)^{ - \frac{L_\nu - 2}{L_\nu}} + \alpha^{ - (L_\nu - 2) } \cdot const \Big ]^{ - \frac{L_\nu}{L_\nu - 2} }
\text{\,,}
\label{htf:eq:illustrative_upper_bound_comp_norms}
\ee
where $const$ stands for a positive value that does not depend on $t$ and $\alpha$.
\eq~\eqref{htf:eq:illustrative_upper_bound_comp_norms} reveals a gap between \smash{$\htfcompnorm{\nu}{r} (t)$} and \smash{$\htfcompnorm{\nu}{\bar{r}} (t)$} that is more significant the smaller the initialization scale $\alpha$ is.
In particular, regardless of how large $\htfcompnorm{\nu}{r} (t)$ is, $\htfcompnorm{\nu}{\bar{r}} (t)$ is upper bounded by a value that approaches zero as $\alpha \to 0$.
Hence, initializing near zero produces solutions with small local components.

\subsubsection{Small Local Components Imply Proximity to Low Hierarchical Tensor Rank}
\label{htf:sec:inc_rank_lrn:imp_htr_min:proximity}

The following proposition establishes that small local components in a hierarchical tensor factorization imply that its end tensor can be well approximated with low hierarchical tensor rank.

\begin{proposition}
	\label{htf:prop:low_rank_dist_bound}
	Consider an assignment for the weight matrices \smash{$\big ( \weightmat{\nu} \in \R^{R_\nu \times R_{\parent (\nu)} }  \big )_{ \nu \in \htmodetree}$} of a hierarchical tensor factorization, and let $B := \max_{\nu \in \htmodetree} \normnoflex{ \weightmat{\nu}}$.
	Assume without loss of generality that at each $\nu \in \interior (\htmodetree)$, local components are ordered by their norms, \ie~$\htfcompnorm{\nu}{1} \geq \cdots \geq \htfcompnorm{\nu}{R_\nu}$.
	Then, for any $\epsilon \geq 0$ and $( R'_{\nu} \in \N )_{\nu \in \interior (\htmodetree)}$, if \,$\sum\nolimits_{r = R'_\nu + 1}^{R_\nu} \htfcompnorm{\nu}{r} \leq \epsilon \cdot (\abs{\htmodetree} - N)^{-1} B^{\abs{\children(\nu)} + 1 - \abs{\htmodetree}}$ for all $\nu \in \interior (\htmodetree)$, it holds that:
	\[
	\inf_{ \substack{ \W \in \R^{D_1 \times \cdots \times D_N} \text{ s.t. } \\ \forall \nu \in \htmodetree \setminus \{ [N] \} :~\rank \mat{\W }{ \nu} \leq R'_{\parent (\nu) } } } \norm{ \tensorend - \W } \leq \epsilon
	\text{\,,}
	\]
	\ie~$\tensorend$ is within $\epsilon$-distance from the set of tensors whose hierarchical tensor rank is no greater (element-wise) than $\big ( R'_{\parent (\nu)} \big )_{\nu \in \htmodetree \setminus \{ [N] \}}$.
\end{proposition}
\begin{proof}[Proof sketch (proof in \subapp~\ref{htf:app:proofs:low_rank_dist_bound})]
	Let $\widebar{\W}_{\mathrm{H}}^{\S}$ be the end tensor obtained after pruning all local components indexed by $\S := \{ (\nu, r) : \nu \in \interior (\htmodetree) , r \in \{ R'_{\nu} + 1, \ldots, R_\nu \} \}$, \ie~after setting to zero the $r$'th row of $\weightmat{\nu}$ and the $r$'th column of $\weightmat{\nu_c}$ for all $(\nu, r) \in \S$ and $\nu_c \in \children (\nu)$.
	The desired result follows by showing that $\rank \mat{ \widebar{\W}_{\mathrm{H}}^{\S} }{ \nu} \leq R'_{ \parent (\nu) }$ for all $\nu \in \htmodetree \setminus \{ [N] \}$, and upper bounding $\norm{ \tensorend - \widebar{\W}_{\mathrm{H}}^{\S} }$ by $\epsilon$.
\end{proof}

\subsection{Partially Ordered Complexity Measure}
\label{htf:sec:inc_rank_lrn:partial_order}

Existing attempts to explain implicit regularization in deep learning typically argue for reduction of some complexity measure that is \emph{totally ordered} (meaning that within any two values for this measure, there must be one smaller than or equal to the other), for example a norm~\cite{gunasekar2017implicit,soudry2018implicit,li2018algorithmic,woodworth2020kernel,lyu2021gradient}.
Recent evidence suggests that obtaining a complete explanation through such complexity measures may not be possible \cite{razin2020implicit,vardi2021implicit}.
Hierarchical tensor rank (which we have shown to be implicitly reduced by a class of deep non-linear convolutional networks) represents a new type of complexity measure, in the sense that it is \emph{partially ordered}.  Specifically, while it entails a standard (product) partial order~---~$(r_1, \ldots, r_K) \leq (r'_1, \ldots, r'_K)$ if and only if $r_i \leq r'_i$ for all $i \in [K]$~---~it does not admit a natural total order.  
Indeed, \prop~\ref{htf:prop:htr_multiple_minima} below shows that there exist simple learning problems in which, among the data-fitting solutions, there are multiple minimal hierarchical tensor ranks, none smaller than or equal to the other.  
We believe the notion of a partially ordered complexity measure may pave the way to furthering our understanding of implicit regularization in deep learning.

\begin{proposition}
	\label{htf:prop:htr_multiple_minima}
	For every order $N \in \N_{\geq 3}$ and mode dimensions $D_1, \ldots, D_N \in \N_{\geq 2}$, there exists a tensor completion problem (\ie~a loss $\L ( \W ) = \frac{1}{\abs{\Omega}} \sum_{(d_1, \ldots, d_N) \in \Omega} ( \W_{d_1, \ldots, d_N} - \W^*_{d_1, \ldots, d_N})^2$ with ground truth $\W^* \in \R^{D_1 \times \cdots \times D_N}$ and set of observed entries $\Omega \subset [D_1] \times \cdots \times [D_N]$) in which, for every mode tree $\htmodetree$ over $[N]$ (\defin~\ref{htf:def:mode_tree}), the set of hierarchical tensor ranks for tensors fitting the observations includes multiple minimal elements (under the standard product partial order), none smaller than or equal to the other.
	That is, the set $\RR_\htmodetree := \left \{ (\rank \mat{\W }{ \nu})_{\nu \in \htmodetree \setminus \{ [N] \}} : \W \in \R^{D_1 \times \cdots \times D_N} , \L ( \W ) = 0 \right \}$ includes elements $(R_\nu )_{\nu \in \htmodetree \setminus \{ [N]\}}$ and $(R'_\nu )_{\nu \in \htmodetree \setminus \{ [N]\}}$ for which the following hold: \emph{(i)} there exists no $(R''_\nu )_{\nu \in \htmodetree \setminus \{ [N]\}} \in \RR_\htmodetree \setminus \{ (R_{\nu} )_{\nu  \in \htmodetree \setminus \{ [N] \}} , (R'_{\nu } )_{\nu \in \htmodetree \setminus \{ [N]\} } \}$ satisfying $(R''_\nu )_{\nu} \leq (R_\nu )_{\nu}$ or $(R''_\nu )_{\nu} \leq (R'_\nu )_{\nu}$; and \emph{(ii)} neither $(R_\nu )_{\nu} \leq (R'_\nu )_{\nu}$ nor $(R'_\nu )_{\nu} \leq (R_\nu )_{\nu}$.
\end{proposition}
\begin{proof}[Proof sketch (proof in \subapp~\ref{htf:app:proofs:htr_multiple_minima})]
	We construct a tensor completion problem and two solutions $\W$ and $\W'$ (tensors fitting observed entries) such that, for every mode tree $\htmodetree$, the hierarchical tensor ranks with respect to $\htmodetree$ of $\W$ and $\W'$ are two different minimal elements of $\RR_\htmodetree$.
\end{proof}

\section{Low Hierarchical Tensor Rank Implies Locality} 
\label{htf:sec:low_htr_implies_locality}

In \sect~\ref{htf:sec:inc_rank_lrn} we established that the implicit regularization in hierarchical tensor factorization favors solutions with low hierarchical tensor rank.
A natural question that arises is what are the implications of this tendency for the class of deep convolutional networks equivalent to hierarchical tensor factorization (illustrated in \fig~\ref{htf:fig:tf_htf_as_convnet} (bottom)).
It is known~\cite{cohen2017inductive,levine2018benefits,levine2018deep} that for this class of networks, hierarchical tensor rank measures the strength of dependencies modeled between spatially distant input regions (patches of pixels in the context of image classification)~---~see brief explanation in \subsect~\ref{htf:sec:low_htr_implies_locality:sep_rank} below, and formal derivation in \app~\ref{htf:app:ht_sep_rank}.
An implicit regularization towards low hierarchical tensor rank thus implies a bias towards local (short-range) dependencies.
While seemingly benign, this observation is shown in \sect~\ref{htf:sec:countering_locality} to bring forth a practical method for improving performance of contemporary convolutional networks (\eg~ResNet18 and ResNet34 from~\cite{he2016deep}) on tasks with long-range dependencies.

\subsection{Locality via Separation Rank}
\label{htf:sec:low_htr_implies_locality:sep_rank}

Given a multivariate function $f$ with scalar output, a popular measure of dependencies between a set of input variables and its complement is known as \emph{separation rank}.  
The separation rank, formally presented in \defin~\ref{htf:def:sep_rank} below, was originally introduced in~\cite{beylkin2002numerical}, and has since been employed for various applications~\cite{harrison2003multiresolution,hackbusch2006efficient,beylkin2009multivariate}, as well as analyses of expressiveness in deep learning~\cite{cohen2017inductive,cohen2017analysis,levine2018benefits,levine2018deep,levine2020limits,wies2021transformer,levine2022inductive}.
It is also prevalent in quantum physics, where it serves as a measure of entanglement~\cite{levine2018deep}.

Consider the convolutional network equivalent to a hierarchical tensor factorization with mode tree $\htmodetree$.
It turns out (see formal derivation in \app~\ref{htf:app:ht_sep_rank}) that for functions realized by this network, separation ranks measuring dependencies between distinct regions of the input are precisely equal to entries of the hierarchical tensor rank with respect to $\htmodetree$ (recall that, as discussed in \sect~\ref{htf:sec:htf}, the hierarchical tensor rank is a tuple).  
Thus, low hierarchical tensor rank implies that the separation ranks are low, which in turn means that dependencies modeled between distinct input regions are weak, \ie~that only local dependencies are prominent.

\begin{definition}
	\label{htf:def:sep_rank}
	The \emph{separation rank} of $f : \times_{n = 1}^N \R^{D_n} \to \R$ with respect to $I \subset [N]$, denoted $\htfseprank (f; I)$, is the minimal $R \in \N \cup \{ 0 \}$ for which there exist $g_1, \ldots, g_R : \times_{i \in I} \R^{D_i} \to \R$ and $\bar{g}_1, \ldots, \bar{g}_R :\times_{j \in [N] \setminus I} \R^{D_j} \to \R$ such that:
	\[
	f \bigl ( \xbf^{(1)}, \ldots, \xbf^{(N)} \bigr )  =  \sum\nolimits{r = 1}^R g_r \big ( \big ( \xbf^{(i)} \big )_{i \in I} \big ) \cdot \bar{g}_r \bigl ( \bigl ( \xbf^{(j)} \bigr )_{j \in [N] \setminus I } \bigr )
	\text{.}
	\]
\end{definition}

\textbf{Interpretation.}~~The separation rank of~$f$ with respect to~$I$ is the minimal number of summands required to express~$f$, where each summand is a product of two functions~---~one that operates over variables indexed by~$I$, and another that operates over the remaining variables.
If $\htfseprank (f; I) = 1$, the function is separable, meaning it does not model any interaction between the sets of variables.
In a statistical setting, where $f$ is a probability density function, this would mean that $( \xbf^{(i)} )_{i \in I}$ and $( \xbf^{(j)} )_{j \in [N] \setminus I}$ are statistically independent.
The higher $\htfseprank (f; I)$ is, the farther $f$ is from separability, \ie~the stronger the dependencies it models between $( \xbf^{(i)} )_{i \in I}$ and~$( \xbf^{(j)} )_{j \in [N] \setminus I}$.

\section{Countering Locality of Convolutional Networks via Regularization} 
\label{htf:sec:countering_locality}

Convolutional networks often struggle or completely fail to learn tasks that entail strong dependence between spatially distant regions of the input (patches of pixels in image classification or tokens in natural language processing tasks)~---~see, \eg,~\cite{wang2016temporal,linsley2018learning,mlynarski2019convolutional,hong2020graph, kim2020disentangling}.
Conventional wisdom attributes this failure to the local nature of the architecture, \ie~to its inability to express long-range dependencies (see, \eg,~\cite{cohen2017inductive,linsley2018learning,kim2020disentangling}).
This suggests that addressing the problem requires modifying the architecture.
Our theory reveals that there is also an implicit regularization at play, giving rise to the possibility of countering the locality of convolutional networks via \emph{explicit} regularization, without modifying their architecture.
In the current section we affirm this possibility, demonstrating that carefully designed regularization can greatly improve the performance of contemporary convolutional networks on tasks involving long-range dependencies.
For brevity, we defer some implementation details and experiments to \app~\ref{htf:app:experiments}.

We conducted a series of experiments, using the ubiquitous ResNet18 and ResNet34 convolutional networks~\cite{he2016deep}, over two types of image classification datasets in which the distance between salient regions can be controlled.
The first type, referred to as "IsSameClass," comprises datasets we constructed, where the goal is to predict whether two randomly sampled CIFAR10~\cite{krizhevsky2009learning} images are of the same class.
Each input sample is a $32 \times 224$ image filled with zeros, in which the CIFAR10 images are placed (symmetrically around the center) at a predetermined distance from each other (to comply with ResNets, inputs were padded to have size $224 \times 224$).
By increasing the predetermined distance between CIFAR10 images, we produce datasets requiring stronger modeling of long-range dependencies.
The second type of datasets is taken from the Pathfinder challenge~\cite{linsley2018learning,kim2020disentangling,tay2021long}~---~a standard benchmark for modeling long-range dependencies.
In Pathfinder, each image contains two white circles and multiple dashed paths (curves) over a black background, and the goal is to predict whether the circles are connected by a path.  
The length of connecting paths is predetermined, allowing control over the (spatial) range of dependencies necessary to model.
Representative examples from IsSameClass and Pathfinder datasets are displayed in \fig~\ref{htf:fig:long_range_datasets_samples}.

\fig~\ref{htf:fig:long_range_and_reg_results} shows that when fitting IsSameClass and Pathfinder datasets, increasing the strength of long-range dependencies (\ie~the distance between images in IsSameClass, and the connecting path length in Pathfinder) leads to significant degradation in test accuracy, oftentimes resulting in performance no better than random guessing.
This phenomenon complies with existing evidence from~\cite{linsley2018learning,kim2020disentangling} showing failure of convolutional networks in learning tasks with long-range dependencies.
However, while \cite{linsley2018learning,kim2020disentangling} address the problem by modifying the architecture, we tackle it through explicit regularization (described in \subsect~\ref{htf:sec:countering_locality:reg} below) designed to promote high separation ranks (\defin~\ref{htf:def:sep_rank}), \ie~long-range dependencies between image regions.
As evident in \fig~\ref{htf:fig:long_range_and_reg_results}, our regularization significantly improves test accuracy.
This implies that the tendency towards locality of modern convolutional networks may in large part be due to implicit regularization, and not an inherent limitation of expressive power as often believed.
Our findings showcase that deep learning architectures considered suboptimal for certain tasks may be greatly improved through a right choice of explicit regularization.
Theoretical understanding of implicit regularization may be key to discovering such regularizers.

\begin{figure*}[t!]
	\begin{center}
		\hspace{0mm}
		\includegraphics[width=1\textwidth]{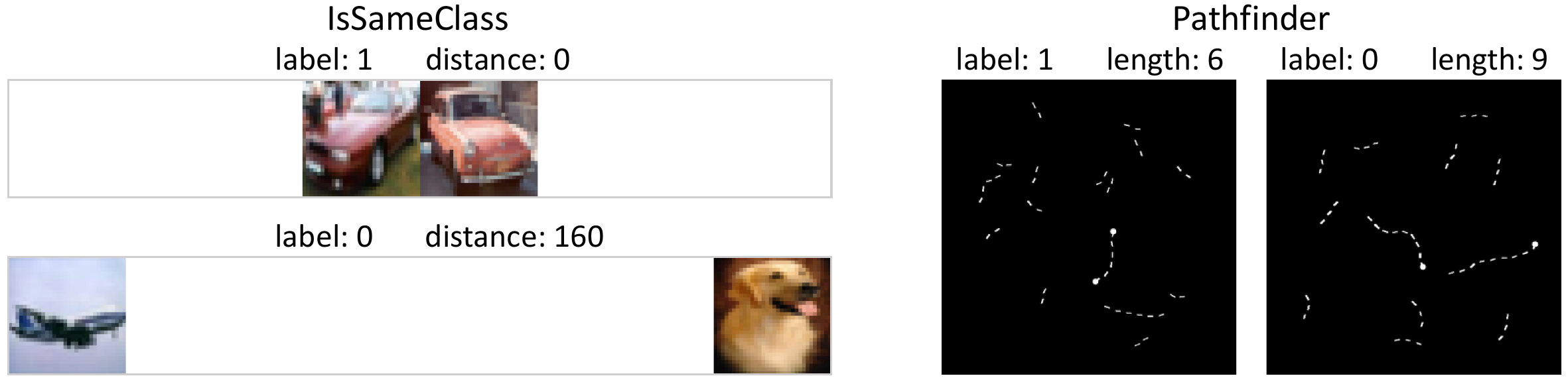}
	\end{center}
	\vspace{-2mm}
	\caption{
		Samples from IsSameClass and Pathfinder datasets.
		For further details on their creation process see \subapp~\ref{htf:app:experiments:details:conv}.
		\textbf{Left:} positive and negative samples from IsSameClass datasets with $0$ and $160$ pixels between images, respectively.
		The label is $1$ if the two CIFAR10 images are of the same class, and $0$ otherwise.
		For the sake of illustration, background is displayed as white instead of black, and padding is not shown (\ie~only the raw $32 \times 224$ input is presented).
		\textbf{Right:} positive and negative samples from Pathfinder challenge~\cite{linsley2018learning} datasets with connecting path lengths $6$ and $9$, respectively. 
		A connecting path is one that joins the two circles, and if present, its length is measured in the number of dashes.
		The label of a sample is $1$ if it includes a connecting path (\ie~if the two circles are connected), and $0$ otherwise.
	}
	\label{htf:fig:long_range_datasets_samples}
\end{figure*}

\begin{figure*}[t!]
	\begin{center}
		\vspace{1mm}
		\hspace{-3mm}
		\includegraphics[width=1\textwidth]{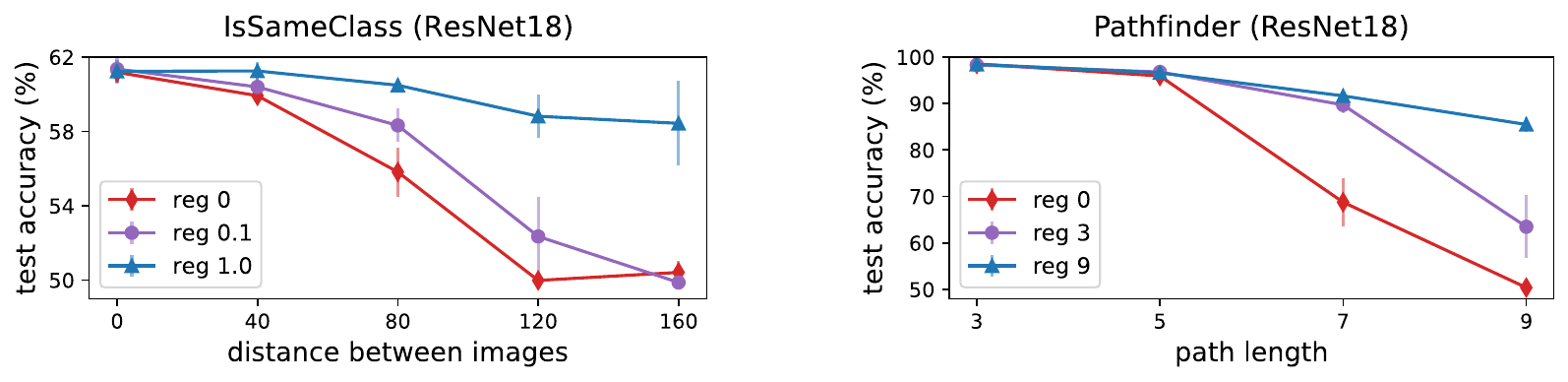}
	\end{center}
	\vspace{-3mm}
	\caption{
		Dedicated explicit regularization can counter the locality of convolutional networks, significantly improving performance on tasks with long-range dependencies.
		Plots present test accuracies achieved by a randomly initialized ResNet18 over IsSameClass (left) and Pathfinder (right) datasets, with varying spatial distances between salient regions of the input (CIFAR10 images in IsSameClass and connected circles in Pathfinder --- see \fig~\ref{htf:fig:long_range_datasets_samples}).
		For each dataset, the network was trained via stochastic gradient descent to minimize a regularized objective, consisting of the binary cross-entropy loss and the dedicated regularization described in \subsect~\ref{htf:sec:countering_locality:reg}.
		The legend specifies the regularization coefficients used.
		Markers and error bars report means and standard deviations, respectively, taken over five different runs for the corresponding combination of dataset and regularization coefficient.
		As expected, when increasing the (spatial) range of dependencies required to be modeled, the test accuracy obtained by an unregularized network (regularization coefficient zero) substantially deteriorates, reaching the vicinity of the trivial value $50\%$.
		Conventional wisdom attributes this failure to a limitation in the expressive capability of convolutional networks (\ie~to their inability to represent functions modeling long-range dependencies). 
		However, as can be seen, applying the dedicated regularization significantly improved performance, without any architectural modification.
		\app~\ref{htf:app:experiments} provides further implementation details, as well as additional experiments: \emph{(i)} using ResNet34; and \emph{(ii)} showing similar improvements when the baseline network (“reg $0$'') is already regularized via standard techniques (weight decay or dropout).
	}
	\label{htf:fig:long_range_and_reg_results}
\end{figure*}

\subsection{Explicit Regularization Promoting Long-Range Dependencies}
\label{htf:sec:countering_locality:reg}

We describe below the explicit regularization applied in our experiments to counter the locality of convolutional networks.
We emphasize that this regularization is based on our theory, and merely serves as an example to how the performance of convolutional networks on tasks involving long-range dependencies can be improved without modifying their architecture.
Further evaluation and improvement of our regularization are regarded as promising directions for future work.

Denote by $f_\Theta ( \Xbf )$ the output of a neural network, where $\Theta$ stands for its learnable weights, and $\Xbf := ( \xbf^{(1)}, \ldots, \xbf^{(N)} )$ represents an input image, with each $\xbf^{(n)}$ standing for a pixel.
Suppose we are given a subset of indices $I \subset [N]$, with complement $J := [N] \setminus I$, and we would like to encourage the network to learn a function $f_\Theta$ that models strong dependence between $\Xbf_I := ( \xbf^{(i)} )_{i \in I}$ (pixels indexed by $I$) and $\Xbf_J := ( \xbf^{(j)} )_{j \in J}$ (those indexed by~$J$).
As discussed in \subsect~\ref{htf:sec:low_htr_implies_locality:sep_rank}, a standard measure of such dependence is the separation rank, provided in \defin~\ref{htf:def:sep_rank}.
If the separation rank of $f_\Theta$ with respect to $I$ is one, meaning no dependence between $\Xbf_I$ and~$\Xbf_J$ is modeled, then we may write $f_\Theta ( \Xbf ) = g (\Xbf_I) \cdot \bar{g} (\Xbf_J)$ for some functions $g$ and $\bar{g}$.
This implies that $\nabla_{\Xbf_I} f_\Theta ( \Xbf ) = \bar{g} (\Xbf_{J}) \cdot \nabla g (\Xbf_I)$, meaning that a change in $\Xbf_J$ (with $\Xbf_I$ held fixed) does not affect the direction of $\nabla_{\Xbf_I} f_\Theta (\Xbf)$, only its magnitude (and possibly its sign).
This observation suggests that, in order to learn a function $f_\Theta$ modeling strong dependence between $\Xbf_I$ and $\Xbf_J$, one may add a regularization term that promotes a change in the direction of $\nabla_{\Xbf_I} f_\Theta (\Xbf)$ whenever $\Xbf_J$ is altered (with $\Xbf_I$ held fixed).

The regularization applied in our experiments is of the type outlined above, with $I$ and $J$ chosen to promote long-range dependencies.
Namely, at each iteration of stochastic gradient descent we randomly choose disjoint subsets of indices $I$ and $J$ corresponding to contiguous (distinct) image regions.
Then, for each image $\Xbf$ in the iteration's batch, we let $\Xbf'$ be the result of replacing the pixels in $\Xbf$ indexed by $J$ with alternative values taken from a different image in the training set.
Finally, we compute $\abs{ \inprod{ \nabla_{\Xbf_I} f_\Theta (\Xbf) }{ \nabla_{\Xbf_I} f_\Theta (\Xbf') } } \cdot  \normnoflex{ \nabla_{\Xbf_I} f_\Theta (\Xbf) }^{-1}  \normnoflex{ \nabla_{\Xbf_I} f_\Theta (\Xbf') }^{-1}$~---~(absolute value of) cosine of the angle between $\nabla_{\Xbf_I} f_\Theta (\Xbf)$ and $\nabla_{\Xbf_I} f_\Theta (\Xbf')$~---~average it across the batch, multiply the average by a constant coefficient, and add the result to the minimized objective.\footnote{
	Each artificially generated image $\Xbf'$ is used only to compute the regularization term, not as an additional training instance incurring its own loss.
	Our proposed regularization is therefore fundamentally different from data augmentation.
}
For further details see \subapp~\ref{htf:app:experiments:details:conv}.

%% file: Chapters/imp_reg_related.tex
\chapter{Related Work}
\label{chap:imp_reg_related}

\textbf{Implicit regularization.}~~A large and growing body of literature has theoretically investigated the implicit regularization brought forth by gradient-based optimization.
Works along this line have treated various models, including: linear predictors~\cite{soudry2018implicit,gunasekar2018implicit,nacson2019convergence,ji2019implicit,shachaf2021theoretical}; polynomially parameterized linear models with a single output~\cite{ji2019gradient,woodworth2020kernel,moroshko2020implicit,azulay2021implicit,haochen2021shape,pesme2021implicit,li2021happens,chou2021more}; shallow non-linear neural networks~\cite{hu2020surprising,vardi2021implicit,sarussi2021towards,mulayoff2021implicit,lyu2021gradient}; homogeneous networks~\cite{lyu2020gradient,vardi2021margin}; and ultra-wide networks~\cite{oymak2019overparameterized,chizat2020implicit}.
Arguably the most widely analyzed model is matrix factorization~\cite{gunasekar2017implicit,du2018algorithmic,li2018algorithmic,arora2019implicit,gidel2019implicit,mulayoff2020unique,blanc2020implicit,gissin2020implicit,razin2020implicit,chou2020gradient,eftekhari2021implicit,yun2021unifying,min2021explicit,li2021towards}, whose study we extended to tensor and hierarchical tensor factorizations in \cref{chap:imp_reg_tf,chap:imp_reg_htf}.
Among other contributions, our results generalize existing dynamical characterizations for matrix factorization (see \sect~\ref{htf:sec:prelim}) to tensor and hierarchical tensor factorizations~---~considerably richer and more complex models.

With regard to convolutional networks, theoretical investigations of their implicit regularization are scarce.
Existing works in this category treat linear~\cite{gunasekar2018implicit,jagadeesan2021inductive,kohn2021geometry} and homogeneous~\cite{nacson2019lexicographic,lyu2020gradient,ji2020directional} models.\footnote{
	There have also been works studying implicit effects of explicit regularizers for convolutional networks~\cite{ergen2021implicit}, but these are outside the scope of this thesis.
} 
None of these works have pointed out an implicit regularization towards local dependencies, as our theory does (\sects~\ref{htf:sec:inc_rank_lrn} and \ref{htf:sec:low_htr_implies_locality}).
Although the locality of convolutional networks is widely accepted, it is typically ascribed to expressive properties determined by their architecture (see, \eg,~\cite{cohen2017inductive,linsley2018learning,kim2020disentangling}).
Our work is the first to indicate that it also originates from implicit regularization.
As we demonstrate in \sect~\ref{htf:sec:countering_locality}, this observation can have far reaching implications to the performance of convolutional networks in practice.

\textbf{Matrix factorization.}~~The literature on matrix factorization for low-rank matrix recovery is far too broad to cover here~---~we refer to~\cite{chi2019nonconvex} for a recent survey, while mentioning that the technique is often attributed to~\cite{burer2003nonlinear}.
Notable works proving successful recovery of a low-rank matrix via matrix factorization trained by gradient descent with no explicit regularization are~\cite{tu2016low,ma2018implicit,li2018algorithmic}.
Of these, \cite{li2018algorithmic}~can be viewed as affirming Conjecture~\ref{conj:nuclear_norm} (from~\cite{gunasekar2017implicit}) for a certain special case.
\cite{belabbas2020implicit}~has affirmed Conjecture~\ref{conj:nuclear_norm} under different assumptions, but nonetheless argued empirically that it does not hold true in general, resonating with Conjecture~\ref{conj:no_norm} (from~\cite{arora2019implicit}).
To the best of our knowledge, no theoretical support for the latter was provided prior to its proof in \cite{razin2020implicit}, on which \cref{chap:imp_reg_not_norms} is based.

\textbf{Tensor factorizations.}~~Recovery of low rank tensors from incomplete observations via tensor factorizations is a setting of growing interest (\cf~\cite{acar2011scalable,narita2012tensor,anandkumar2014tensor,jain2014provable,yokota2016smooth,karlsson2016parallel,xia2017polynomial,zhou2017tensor,cai2019nonconvex} and the survey \cite{song2019tensor}).
In particular, hierarchical tensor factorization was recently introduced in~\cite{hackbusch2009new}, and used for recovery of low hierarchical tensor rank tensors~\cite{da2015optimization,steinlechner2016riemannian,rauhut2017low,kargas2020nonlinear,kargas2021supervised}. 
By virtue of its equivalence to different types of non-linear neural networks (with polynomial non-linearity), it has also been paramount to the study of expressiveness in deep learning~\cite{cohen2016expressive,sharir2016tensorial,cohen2016convolutional,cohen2017inductive,cohen2017analysis,sharir2018expressive,cohen2018boosting,levine2018benefits,levine2018deep,balda2018tensor,khrulkov2018expressive,khrulkov2019generalized,levine2019quantum}.
To the best of our knowledge, \cite{razin2021implicit,razin2022implicit}, on which \cref{chap:imp_reg_tf,chap:imp_reg_htf} are based, provided the first analysis of the implicit regularization induced by gradient-based optimization for tensor and hierarchical tensor factorizations.

%% file: Chapters/gnn_interactions.tex
\chapter{On the Ability of Graph Neural Networks to Model Interactions \\ Between Vertices}
\label{chap:gnn_interactions}

In \cref{part:gen} we employed a connection between neural networks (with polynomial non-linearity) and tensor factorizations for studying implicit regularization.
In this chapter, which is based on~\cite{razin2023ability}, we extend this connection to analyze the expressive power of certain \emph{graph neural networks (GNNs)}.
Specifically, we show that message-passing GNNs~\cite{hamilton2020graph} with product aggregation can be formulated via \emph{tensor networks}~---~a graphical language for expressing tensor factorizations.\footnote{
Tensor networks are widely used for constructing compact representations of quantum states in areas of physics (see,~\eg,~\cite{vidal2008class,orus2014practical}).
}
The formulation of GNNs as tensor networks then allows analyzing their ability to model interactions in an input graph.

\section{Background and Overview}
\label{gnn:sec:intro}

GNNs are a family of deep learning architectures, designed to model complex interactions between entities represented as vertices of a graph.
In recent years, GNNs have been successfully applied across a wide range of domains, including social networks, biochemistry, and recommender systems (see, \eg,~\cite{duvenaud2015convolutional,kipf2017semi,gilmer2017neural,hamilton2017inductive,velivckovic2018graph,ying2018graph,wu2020comprehensive,bronstein2021geometric}).
Consequently, significant interest in developing a mathematical theory behind GNNs has arisen.

One of the fundamental questions a theory of GNNs should address is \emph{expressiveness}, which concerns the class of functions a given architecture can realize.
Existing studies of expressiveness largely fall into three categories.
First, and most prominent, are characterizations of ability to distinguish non-isomorphic graphs~\citep{xu2019powerful,morris2019weisfeiler,maron2019provably,loukas2020hard,balcilar2021breaking,bodnar2021weisfeiler,barcelo2021graph,bouritsas2022improving,geerts2021let,geerts2022expressiveness,papp2022theoretical}, as measured by equivalence to classical Weisfeiler-Leman graph isomorphism tests~\citep{weisfeiler1968reduction}.
Second, are proofs for universal approximation of continuous permutation invariant or equivariant functions, possibly up to limitations in distinguishing some classes of graphs~\citep{maron2019universality,keriven2019universal,chen2019equivalence,loukas2020graph,azizian2021expressive,geerts2022expressiveness}.
Last, are works examining specific properties of GNNs such as frequency response~\citep{nt2019revisiting,balcilar2021analyzing} or computability of certain graph attributes, \eg~moments, shortest paths, and substructure multiplicity~\citep{dehmamy2019understanding,barcelo2020logical,chen2020can,garg2020generalization,loukas2020graph,chen2021graph,bouritsas2022improving,zhang2023rethinking}.

A major drawback of many existing approaches~---~in particular proofs of equivalence to Weisfeiler-Leman tests and those of universality~---~is that they operate in asymptotic regimes of unbounded network width or depth.
Moreover, to the best of our knowledge, none of the existing approaches formally characterize the strength of interactions GNNs can model between vertices, and how that depends on the structure of the input graph and the architecture of the neural network.

The current chapter addresses the foregoing gaps.
Namely, it theoretically quantifies the ability of fixed-size GNNs to model interactions between vertices, delineating the impact of the input graph structure and the neural network architecture (width and depth).
Strength of modeled interactions is formalized via \emph{separation rank}~\citep{beylkin2002numerical}~---~a commonly used measure for the interaction a function models between a subset of input variables and its complement (the rest of the input variables).\footnote{
Recall that, in \cref{htf:sec:low_htr_implies_locality,htf:sec:countering_locality}, we employed separation rank for designing an explicit regularization scheme that improves the performance of convolutional neural networks over tasks involving long-range interactions.
}
Given a function and a partition of its input variables, the higher the separation rank, the more interaction the function models between the sides of the partition.
Separation rank is prevalent in quantum physics, where it can be viewed as a measure of entanglement~\citep{levine2018deep}.
It was previously used for analyzing variants of convolutional, recurrent, and self-attention neural networks, yielding both theoretical insights and practical tools~\citep{cohen2017inductive,cohen2017analysis,levine2018benefits,levine2018deep,levine2020limits,wies2021transformer,levine2022inductive}.
We employ it for studying GNNs.

\begin{figure*}[t]
	\vspace{0mm}
	\begin{center}
		\includegraphics[width=1\textwidth]{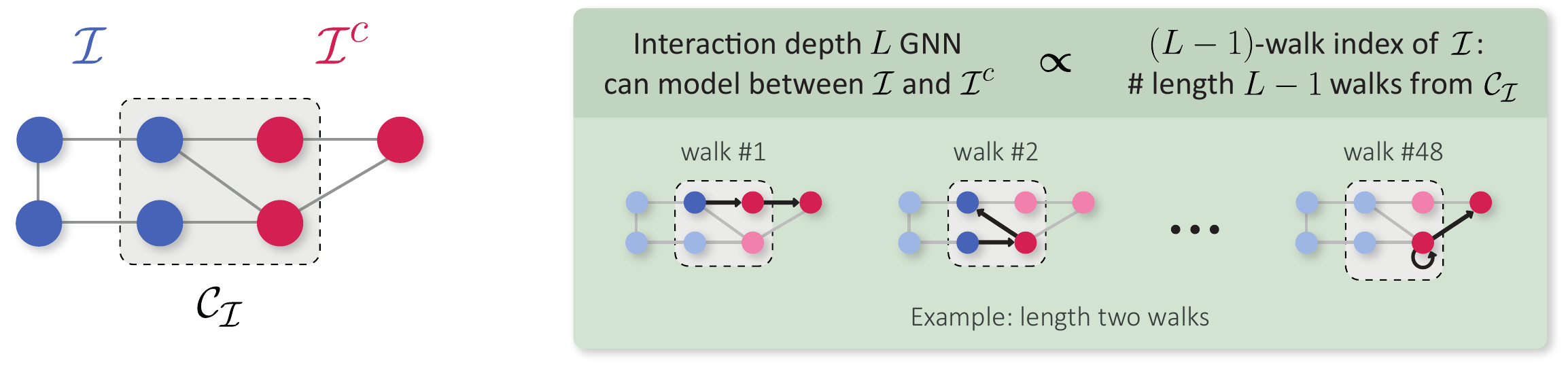}
	\end{center}
	\vspace{-1mm}
	\caption{
		Illustration of our main theoretical contribution: quantifying the ability of GNNs to model interactions between vertices of an input graph.
		Consider a partition of vertices $(\I, \I^c)$, illustrated on the left, and a depth~$L$ GNN with product aggregation (\cref{gnn:sec:gnns}).
		For graph prediction, as illustrated on the right, the strength of interaction the GNN can model between $\I$ and $\I^c$, measured via separation rank (\cref{gnn:sec:prelim:sep_rank}), is primarily determined by the partition's \emph{$(L - 1)$-walk index}~---~the number of length $L - 1$ walks emanating from $\cut_\I$, which is the set of vertices with an edge crossing the partition.
		The same holds for vertex prediction, except that there walk index is defined while only considering walks ending at the target vertex.
	}
	\label{gnn:fig:sep_rank_high_level}
\end{figure*}

We treat both graph prediction, where a single output is produced for an entire input graph, and vertex prediction, in which the network produces an output for every vertex.
For graph prediction, we prove that the separation rank of a depth $L$ GNN with respect to a partition of vertices is primarily determined by the partition's \emph{$(L - 1)$-walk index}~---~a graph-theoretical characteristic defined to be the number of length $L - 1$ walks originating from vertices with an edge crossing the partition.
The same holds for vertex prediction, except that there walk index is defined while only considering walks ending at the target vertex.
Our result, illustrated in \cref{gnn:fig:sep_rank_high_level}, implies that for a given input graph, the ability of GNNs to model interaction between a subset of vertices~$\I$ and its complement~$\I^c$, predominantly depends on the number of walks originating from the boundary between $\I$ and~$\I^c$.
We corroborate this proposition through experiments with standard GNN architectures, such as Graph Convolutional Network (GCN)~\citep{kipf2017semi} and Graph Isomorphism Network (GIN)~\citep{xu2019powerful}.

Our theory formalizes conventional wisdom by which GNNs can model stronger interaction between regions of the input graph that are more interconnected.
More importantly, we show that it facilitates an \emph{edge sparsification} algorithm that preserves the expressive power of GNNs (in terms of ability to model interactions).
Edge sparsification concerns removal of edges from a graph for reducing computational and/or memory costs, while attempting to maintain selected properties of the graph (\cf~\cite{baswana2007simple,spielman2011graph,hamann2016structure,chakeri2016spectral,sadhanala2016graph,voudigari2016rank,li2020sgcn,chen2021unified}).
In the context of GNNs, our interest lies in maintaining prediction accuracy as the number of edges removed from the input graph increases.

We propose an algorithm for removing edges, guided by our separation rank characterization.
The algorithm, named \emph{Walk Index Sparsification} (\emph{WIS}), is demonstrated to yield high predictive performance for GNNs (\eg~GCN and GIN) over standard benchmarks of various scales, even when removing a significant portion of edges.
WIS is simple, computationally efficient, and in our experiments has markedly outperformed alternative methods in terms of induced prediction accuracies across edge sparsity levels.
More broadly, WIS showcases the potential of improving GNNs by theoretically analyzing the interactions they can model, and we believe its further empirical investigation is a promising direction for future research.

\medskip

The remainder of the chapter is organized as follows.
\cref{gnn:sec:prelim} introduces notation and, for completeness, reintroduces the concept of separation rank (which was previously introduced in \cref{htf:sec:low_htr_implies_locality}).
\cref{gnn:sec:gnns} presents the theoretically analyzed GNN architecture.
\cref{gnn:sec:analysis} theoretically quantifies (via separation rank) its ability to model interactions between vertices of an input graph.
Finally, \cref{gnn:sec:sparsification} proposes and evaluates WIS~---~an edge sparsification algorithm for arbitrary GNNs, born from our theory.

\section{Preliminaries}
\label{gnn:sec:prelim}

\subsection{Notation}
\label{gnn:sec:prelim:notation}

For $N \in \N$, let $[N] := \brk[c]{1, \ldots, N}$.
We consider an undirected input graph $\graph = \brk{\vertices, \edges}$ with vertices $\vertices = [ \abs{\vertices} ]$ and edges $\edges \subseteq \{ \{ i, j \} : i, j \in \vertices \}$.
Vertices are equipped with features $\fmat := \brk{ \fvec{1}, \ldots, \fvec{\abs{\vertices}} } \in \R^{\indim \times \abs{\vertices}}$~---~one $\indim$-dimensional feature vector per vertex ($\indim \in \N$).
For $i \in \vertices$, we use $\neigh (i) := \brk[c]{j \in \vertices : \{ i, j \} \in \edges }$ to denote its set of neighbors, and, as customary in the context of GNNs, assume the existence of all self-loops, \ie~$i \in \neigh (i)$ for all $i \in \vertices$ (\cf~\cite{kipf2017semi,hamilton2020graph}).
Furthermore, for $\I \subseteq \vertices$ we let $\neigh (\I) := \cup_{i \in \I} \neigh (i)$ be the neighbors of vertices in $\I$, and $\I^c := \vertices \setminus \I$ be the complement of $\I$.
We use $\cut_\I$ to denote the boundary of the partition $(\I, \I^c)$, \ie~the set of vertices with an edge crossing the partition, defined by $\cut_\I := \{ i \in \I : \neigh (i) \cap \I^c \neq \emptyset \} \cup \{ j \in \I^c : \neigh (j) \cap \I \neq \emptyset \}$.\footnote{
	Due to the existence of self-loops, $\cut_\I$ is exactly the shared neighbors of $\I$ and $\I^c$, \ie~$\cut_\I = \neigh (\I) \cap \neigh (\I^c)$.
}
Lastly, we denote the number of length $l \in \N_{\geq 0}$ walks from any vertex in $\I \subseteq \vertices$ to any vertex in $\J \subseteq \vertices$ by $\nwalk{l}{\I}{\J}$.\footnote{
	For $l \in \N_{\geq 0}$, a sequence of vertices $i_{0}, \ldots, i_{l} \in \vertices$ is a length $l$ walk if $\{ i_{l' - 1} , i_{l'} \} \in \edges$ for all $l' \in [l]$.
}
In particular, $\nwalk{l}{\I}{\J} = \sum_{i \in \I, j \in \J} \nwalk{l}{ \{i\}}{\{j\}}$.

Note that we focus on undirected graphs for simplicity of presentation.
As discussed in~\cref{gnn:sec:analysis}, our results are extended to directed graphs in~\cref{gnn:app:extensions}.

\subsection{Separation Rank: A Measure of Modeled Interaction}
\label{gnn:sec:prelim:sep_rank}

\emph{Separation rank} is a prominent measure quantifying the interaction a multivariate function models between a subset of input variables and its complement (\ie~all other variables).
It was introduced in \cref{htf:sec:low_htr_implies_locality}, but for completeness we introduce it here once more.

The separation rank was originally defined in~\cite{beylkin2002numerical}, and has since been employed for various applications~\citep{harrison2003multiresolution,hackbusch2006efficient,beylkin2009multivariate}.
It is also a common measure of \emph{entanglement}, a profound concept in quantum physics quantifying interaction between particles~\citep{levine2018deep}.
In the context of deep learning, it enabled analyses of expressiveness and generalization in certain convolutional, recurrent, and self-attention neural networks, resulting in theoretical insights and practical methods (guidelines for neural architecture design, pretraining schemes, and regularizers~---~see~\cite{cohen2017inductive,cohen2017analysis,levine2018benefits,levine2018deep,levine2020limits,wies2021transformer,levine2022inductive} and \cref{chap:imp_reg_htf}).

Given a multivariate function $f : \brk{\R^{\indim}}^N \to \R$, its separation rank with respect to a subset of input variables $\I \subseteq [N]$ is the minimal number of summands required to express it, where each summand is a product of two functions~---~one that operates over variables indexed by~$\I$, and another that operates over the remaining variables.
Formally:
\begin{definition}
	\label{gnn:def:sep_rank}
	The \emph{separation rank} of $f : \brk{\R^{\indim}}^N \to \R$ with respect to $\I \subseteq [N]$ is the minimal $R \in \N \cup \{0\}$ for which there exist $g^{(1)}, \ldots, g^{(R)} : \brk{\R^{\indim}}^{\abs{\I}} \to \R$ and $\bar{g}^{(1)}, \ldots, \bar{g}^{(R)} : \brk{\R^{\indim}}^{ \abs{\I^c} } \to \R$ such that:
	\be
	\begin{split}
		f ( \fmat ) = \sum\nolimits_{r = 1}^R g^{(r)} (\fmat_\I) \cdot \bar{g}^{(r)} (\fmat_{\I^c}) 
		\text{\,,}
	\end{split}
	\label{gnn:eq:sep_rank}
	\ee
	where $\fmat := \brk{\fvec{1}, \ldots, \fvec{N}}$, $\fmat_\I := \brk{ \fvec{i} }_{i \in \I}$, and $\fmat_{\I^c} := \brk{ \fvec{j} }_{j \in \I^c}$.
	By convention, if $f$ is identically zero then $\sepranknoflex{f}{\I} = 0$, and if the set on the right hand side of~\cref{gnn:eq:sep_rank} is empty then $\seprank{f}{\I} = \infty$.
\end{definition}

\textbf{Interpretation.}~~If $\sepranknoflex{f}{\I} = 1$, the function is separable, meaning it does not model any interaction between $\fmat_\I$ and $\fmat_{\I^c}$, \ie~between the sides of the partition $(\I, \I^c)$.
Specifically, it can be represented as $f (\fmat) = g ( \fmat_\I ) \cdot \bar{g} ( \fmat_{\I^c} )$ for some functions $g$ and $\bar{g}$.
In a statistical setting, where $f$ is a probability density function, this would mean that $\fmat_\I$ and $\fmat_{\I^c}$ are statistically independent.
The higher $\sepranknoflex{f}{\I}$ is, the farther $f$ is from separability, implying stronger modeling of interaction between $\fmat_\I$ and $\fmat_{\I^c}$.

\section{Graph Neural Networks}
\label{gnn:sec:gnns}

GNN architectures predominantly follow the message-passing paradigm~\citep{gilmer2017neural,hamilton2020graph}, whereby each vertex is associated with a hidden embedding that is updated according to its neighbors.
The initial embedding of $i \in \vertices$ is taken to be its input features: $\hidvec{0}{i} := \fvec{i} \in \R^{\indim}$.
Then, in a depth~$L$ message-passing GNN, a common update scheme for the hidden embedding of $i \in \vertices$ at layer $l \in [L]$ is:
\be
\hidvec{l}{i} = \agg \brk2{ \multisetbig{ \weightmat{l} \hidvec{l - 1}{j} : j \in \neigh (i) } }
\text{\,,}
\label{gnn:eq:gnn_update}
\ee
where $\multiset{\cdot}$ denotes a multiset, $\weightmat{1} \in \R^{\hdim \times \indim}, \weightmat{2} \in \R^{\hdim \times \hdim}, \ldots, \weightmat{L} \in \R^{\hdim \times \hdim}$ are learnable weight matrices, with $\hdim \in \N$ being the network's width (\ie~hidden dimension), and $\agg$ is a function combining multiple input vectors into a single vector.
A notable special case is GCN~\citep{kipf2017semi}, in which $\agg$ performs a weighted average followed by a non-linear activation function (\eg~ReLU).\footnote{
	In GCN, $\agg$ also has access to the degrees of vertices, which are used for computing the averaging weights.
	We omit the dependence on vertex degrees in our notation for conciseness.
}
Other aggregation operators are also viable, \eg~element-wise sum, max, or product (\cf~\cite{hamilton2017inductive,hua2022high}).
We note that distinguishing self-loops from other edges, and more generally, treating multiple edge types, is possible through the use of different weight matrices for different edge types~\citep{hamilton2017inductive,schlichtkrull2018modeling}.
For conciseness, we hereinafter focus on the case of a single edge type, and treat multiple edge types in~\cref{gnn:app:extensions}. 

After $L$ layers, the GNN generates hidden embeddings $\hidvec{L}{1}, \ldots, \hidvec{L}{\abs{\vertices}} \in \R^{\hdim}$.
For graph prediction, where a single output is produced for the whole graph, the hidden embeddings are usually combined into a single vector through the $\agg$ function.
A final linear layer with weights $\weightmat{o} \in \R^{1 \times \hdim}$ is then applied to the resulting vector.\footnote{
	We treat the case of output dimension one merely for the sake of presentation.
	Extension of our theory (delivered in~\cref{gnn:sec:analysis}) to arbitrary output dimension is straightforward~---~the results hold as stated for each of the functions computing an output entry.
}
Overall, the function realized by a depth~$L$ graph prediction GNN receives an input graph $\graph$ with vertex features $\fmat := \brk{\fvec{1}, \ldots, \fvec{ \abs{\vertices} }} \in \R^{\indim \times \abs{\vertices}}$, and returns:
\be
\hspace{-7.5mm}\text{(graph prediction)} \quad \funcgraph{\params}{\graph} \brk{ \fmat } := \weightmat{o} \agg \brk1{ \multisetbig{ \hidvec{L}{i} : i \in \vertices } }
\text{\,,}
\label{gnn:eq:graph_pred_gnn}
\ee
with $\params := \brk{ \weightmat{1}, \ldots, \weightmat{L}, \weightmat{o} }$ denoting the network's learnable weights.
For vertex prediction tasks, where the network produces an output for every $t \in \vertices$, the final linear layer is applied to each $\hidvec{L}{t}$ separately.
That is, for a target vertex $t \in \vertices$, the function realized by a depth~$L$ vertex prediction GNN is given by:
\be
\hspace{-42.5mm}\text{(vertex prediction)} \quad
\funcvert{\params}{\graph}{t} \brk {\fmat} := \weightmat{o} \hidvec{L}{t}
\text{\,.}
\label{gnn:eq:vertex_pred_gnn}
\ee

Our aim is to investigate the ability of GNNs to model interactions between vertices.
Prior studies of interactions modeled by different deep learning architectures have focused on neural networks with polynomial non-linearity, building on their representation as tensor networks~\citep{cohen2016expressive,cohen2017inductive,cohen2018boosting,sharir2018expressive,levine2018benefits,levine2018deep,balda2018tensor,khrulkov2018expressive,levine2019quantum,levine2020limits,wies2021transformer,levine2022inductive}.
Although neural networks with polynomial non-linearity are less common in practice, they have demonstrated competitive performance~\citep{cohen2014simnets,cohen2016deep,sharir2016tensorial,stoudenmire2018learning,chrysos2020p,felser2021quantum,hua2022high}, and hold promise due to their compatibility with quantum computation~\citep{grant2018hierarchical,bhatia2019matrix} and fully homomorphic encryption~\citep{gilad2016cryptonets}.
More importantly, their analyses brought forth numerous insights that were demonstrated empirically and led to development of practical tools for widespread deep learning models (with non-linearities such as ReLU).

Following the above, in our theoretical analysis (\cref{gnn:sec:analysis}) we consider GNNs with (element-wise) product aggregation, which are polynomial functions of their inputs.
Namely, the $\agg$ operator from~\cref{gnn:eq:gnn_update,gnn:eq:graph_pred_gnn} is taken to be:
\be
\agg \brk{ \XX } := \hadmp_{ \xbf \in \XX } \xbf
\text{\,,}
\label{gnn:eq:prod_gnn_agg}
\ee
where $\hadmp$ stands for the Hadamard product and $\XX$ is a multiset of vectors.
The resulting architecture can be viewed as a variant of the GNN proposed in~\cite{hua2022high}, where it was shown to achieve competitive performance in practice.
Central to our proofs are tensor network representations of GNNs with product aggregation (formally established in~\cref{gnn:app:prod_gnn_as_tn}), analogous to those used for analyzing other types of neural networks.
We empirically demonstrate our theoretical findings on popular GNNs (\cref{gnn:sec:analysis:experiments}), such as GCN and GIN with ReLU non-linearity, and use them to derive a practical edge sparsification algorithm (\cref{gnn:sec:sparsification}).

We note that some of the aforementioned analyses of neural networks with polynomial non-linearity were extended to account for additional non-linearities, including~ReLU, through constructs known as \emph{generalized tensor networks}~\citep{cohen2016convolutional}.
We thus believe our theory may be similarly extended, and regard this as an interesting direction for future work.


\section{Theoretical Analysis: The Effect of Input Graph Structure and Neural Network Architecture on Modeled Interactions}
\label{gnn:sec:analysis}

In this section, we employ separation rank (\cref{gnn:def:sep_rank}) to theoretically quantify how the input graph structure and network architecture (width and depth) affect the ability of a GNN with product aggregation to model interactions between input vertices.
We begin with an overview of the main results and their implications (\cref{gnn:sec:analysis:overview}), after which we delve into the formal analysis (\cref{gnn:sec:analysis:formal}).
Experiments demonstrate our theory's implications on common GNNs, such as GCN and GIN with ReLU non-linearity (\cref{gnn:sec:analysis:experiments}).

\subsection{Overview and Implications}
\label{gnn:sec:analysis:overview}

Consider a depth~$L$ GNN with width $\hdim$ and product aggregation (\cref{gnn:sec:gnns}).
Given a graph $\graph$, any assignment to the weights of the network $\params$ induces a multivariate function~---~$\funcgraph{\params}{\graph}$ for graph prediction (\cref{gnn:eq:graph_pred_gnn}) and $\funcvert{\params}{\graph}{t}$ for prediction over a given vertex $t \in \vertices$ (\cref{gnn:eq:vertex_pred_gnn})~---~whose variables correspond to feature vectors of input vertices.
The separation rank of this function with respect to $\I \subseteq \vertices$ thus measures the interaction modeled across the partition $(\I, \I^c)$, \ie~between the vertices in $\I$ and those in $\I^c$.
The higher the separation rank is, the stronger the modeled interaction.

\begin{figure*}[t]
	\vspace{0mm}
	\begin{center}
		\includegraphics[width=1\textwidth]{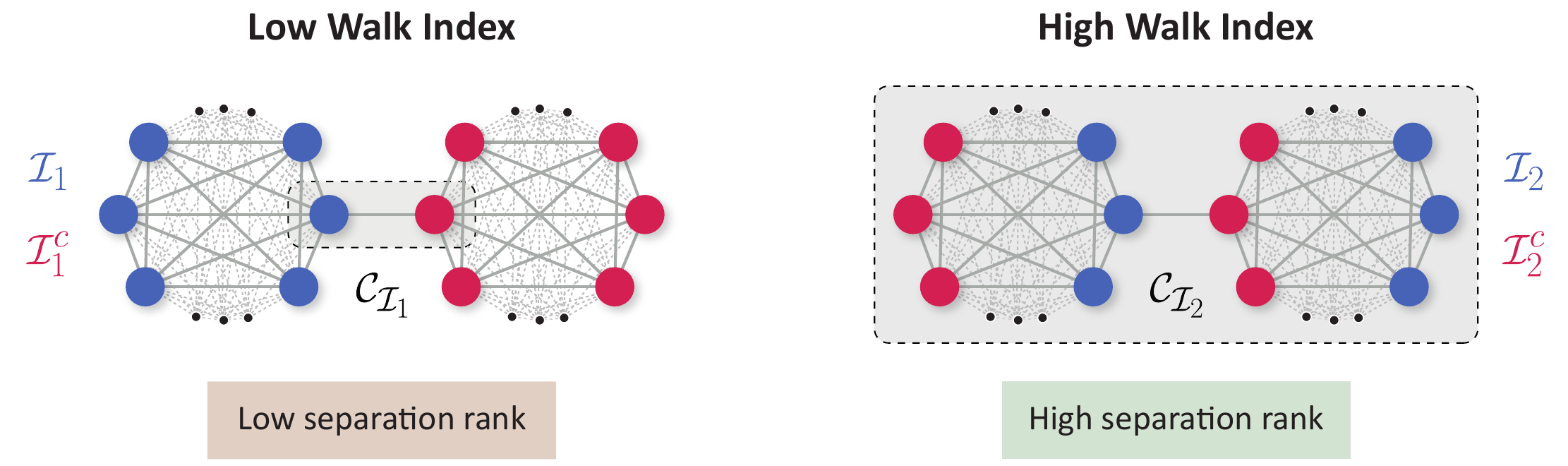}
	\end{center}
	\vspace{-1mm}
	\caption{
		Depth $L$ GNNs can model stronger interactions between sides of partitions that have a higher walk index (\cref{gnn:def:walk_index}).
		The partition $(\I_1, \I_1^c)$ (left) divides the vertices into two separate cliques, connected by a single edge.
		Only two vertices reside in $\cut_{\I_1}$~---~the set of vertices with an edge crossing the partition.
		Taking for example depth $L = 3$, the $2$-walk index of $\I_1$ is $\Theta ( | \vertices |^{2} )$ and its $(2, t)$-walk index is $\Theta ( | \vertices |)$, for $t \in \vertices$.
		In contrast, the partition $(\I_2, \I_2^c)$ (right) equally divides the vertices in each clique to different sides.
		All vertices reside in $\cut_{\I_2}$, meaning the $2$-walk index of $\I_2$ is $\Theta ( | \vertices |^{3} )$ and its $(2, t)$-walk index is $\Theta (| \vertices|^{2})$, for $t \in \vertices$.
		Hence, in both graph and vertex prediction scenarios, the walk index of $\I_1$ is relatively low compared to that of $\I_2$.
		Our analysis (\cref{gnn:sec:analysis:overview,gnn:sec:analysis:formal}) states that a higher separation rank can be attained with respect to $\I_2$, meaning stronger interaction can be modeled across $(\I_2, \I_2^c)$ than across $(\I_1, \I_1^c)$.
		We empirically confirm this prospect in~\cref{gnn:sec:analysis:experiments}.
	}
	\label{gnn:fig:sep_rank_analysis}
\end{figure*}

Key to our analysis are the following notions of \emph{walk index}, defined by the number of walks emanating from the boundary of the partition $(\I, \I^c)$, \ie~from vertices with an edge crossing the partition induced by~$\I$ (see~\cref{gnn:fig:sep_rank_high_level} for an illustration).

\begin{definition}
	\label{gnn:def:walk_index}
	Let $\I \subseteq \vertices$.
	Denote by $\cut_\I$ the set of vertices with an edge crossing the partition $(\I, \I^c)$, \ie~$\cut_\I := \{ i \in \I : \neigh (i) \cap \I^c \neq \emptyset \} \cup \{ j \in \I^c : \neigh (j) \cap \I \neq \emptyset \}$, and recall that $\nwalk{l}{\cut_\I}{\J}$ denotes the number of length $l \in \N_{\geq 0}$ walks from any vertex in $\cut_\I$ to any vertex in $\J \subseteq \vertices$.
	For~$L \in \N$:
	\begin{itemize}[leftmargin=2em]
		\vspace{-2mm}
		\item (graph prediction)~~we define the \emph{$(L - 1)$-walk index} of $\I$, denoted $\walkin{L -1}{\I}$, to be the number of length $L - 1$ walks originating from $\cut_\I$, \ie~$\walkin{L-1}{\I} := \nwalk{L - 1}{\cut_\I}{\vertices}$; and
		
		\item (vertex prediction)~~for $t \in \vertices$ we define the \emph{$(L -1 , t)$-walk index} of $\I$, denoted $\walkinvert{L-1}{t}{\I}$, to be the number of length $L - 1$ walks from $\cut_\I$ that end at $t$, \ie~$\walkinvert{L-1}{t}{\I} := \nwalk{L - 1}{\cut_\I}{\{t\}}$.
	\end{itemize}
\end{definition}

As our main theoretical contribution, we prove:
\begin{theorem}[informally stated]
	\label{gnn:thm:sep_rank_informal_bounds}
	For all weight assignments $\params$ and $t \in \vertices$:
	\begin{align*}
		\text{(graph prediction)} \quad &\log \brk1{ \seprankbig{ \funcgraph{\params}{\graph} }{\I} } 
		= 
		\OO \brk1{ \log \brk{\hdim} \cdot \walkin{L - 1}{\I} } \text{\,,} \\[0.5em]
		\text{(vertex prediction)} \quad &\log \brk1{ \seprankbig{ \funcvert{\params}{\graph}{ t } }{\I} }
		=
		\OO \brk1{ \log \brk{ \hdim } \cdot \walkinvert{L - 1}{ t }{ \I } } \text{\,.}
	\end{align*}
	Moreover, nearly matching lower bounds hold for almost all weight assignments.\footnote{
		Almost all in the sense that the lower bounds hold for all weight assignments but a set of Lebesgue measure zero.
	} 
\end{theorem}

The upper and lower bounds are formally established by~\cref{gnn:thm:sep_rank_upper_bound,gnn:thm:sep_rank_lower_bound} in~\cref{gnn:sec:analysis:formal}, respectively, and are generalized to input graphs with directed edges and multiple edge types in~\cref{gnn:app:extensions}.
\cref{gnn:thm:sep_rank_informal_bounds} implies that, the $(L - 1)$-walk index of $\I$ in graph prediction and its $(L - 1, t)$-walk index in vertex prediction control the separation rank with respect to $\I$, and are thus paramount for modeling interaction between $\I$ and~$\I^c$~---~see~\cref{gnn:fig:sep_rank_analysis} for an illustration.
It thereby formalizes the conventional wisdom by which GNNs can model stronger interaction between areas of the input graph that are more interconnected.
We support this finding empirically with common GNN architectures (\eg~GCN and GIN with ReLU non-linearity) in~\cref{gnn:sec:analysis:experiments}.

One may interpret~\cref{gnn:thm:sep_rank_informal_bounds} as encouraging addition of edges to an input graph.
Indeed, the theorem states that such addition can enhance the GNN's ability to model interactions between input vertices.  
This accords with existing evidence by which increasing connectivity can improve the performance of GNNs in practice (see,~\eg,~\cite{gasteiger2019diffusion,alon2021bottleneck}). 
However, special care needs to be taken when adding edges: it may distort the semantic meaning of the input graph, and may lead to plights known as over-smoothing and over-squashing~\citep{li2018deeper,oono2019graph,chen2020measuring,alon2021bottleneck,banerjee2022oversquashing}.
Rather than employing~\cref{gnn:thm:sep_rank_informal_bounds} for adding edges, we use it to select which edges to preserve in a setting where some must be removed.  
That is, we employ it for designing an edge sparsification algorithm.  
The algorithm, named \emph{Walk Index Sparsification} (\emph{WIS}), is simple, computationally efficient, and in our experiments has markedly outperformed alternative methods in terms of induced prediction accuracy.   
We present and evaluate it in~\cref{gnn:sec:sparsification}.

\vspace{-0.5mm}

\subsection{Formal Presentation}
\label{gnn:sec:analysis:formal}

We begin by upper bounding the separation ranks a GNN can achieve.

\begin{theorem}
	\label{gnn:thm:sep_rank_upper_bound}
	For an undirected graph $\graph$ and $t \in \vertices$, let $\funcgraph{\params}{\graph}$ and $\funcvert{\params}{\graph}{t}$ be the functions realized by depth $L$ graph and vertex prediction GNNs, respectively, with width~$\hdim$, learnable weights $\params$, and product aggregation (\cref{gnn:eq:gnn_update,gnn:eq:graph_pred_gnn,gnn:eq:vertex_pred_gnn,gnn:eq:prod_gnn_agg}).
	Then, for any $\I \subseteq \vertices$ and assignment of weights $\params$ it holds that:
	\begin{align}
		\text{(graph prediction)} \quad &\log \brk1{ \seprankbig{ \funcgraph{\params}{\graph} }{\I} }
		\leq
		\log \brk{ \hdim } \cdot \brk1{4 \underbrace{ \nwalk{L - 1}{\cut_\I}{ \vertices } }_{ \walkin{L -1}{\I} } + 1 } \text{\,,}
		\label{gnn:eq:sep_rank_upper_bound_graph_pred} \\[0.5em]
		\text{(vertex prediction)} \quad &\log \brk1{ \seprankbig{ \funcvert{\params}{\graph}{ t } }{\I} }
		\leq
		\log \brk{ \hdim } \cdot 4 \underbrace{ \nwalk{L - 1}{\cut_\I}{ \{ t \} } }_{ \walkinvert{L - 1}{t}{\I} } \text{\,.}
		\label{gnn:eq:sep_rank_upper_bound_vertex_pred}
	\end{align}
\end{theorem}

\begin{proof}[Proof sketch (proof in~\cref{gnn:app:proofs:sep_rank_upper_bound})]
	In~\cref{gnn:app:prod_gnn_as_tn}, we show that the computations performed by a GNN with product aggregation can be represented as a \emph{tensor network}.
	In brief, a tensor network is a weighted graph that describes a sequence of arithmetic operations known as tensor contractions (see~\cref{gnn:app:prod_gnn_as_tn:tensors,gnn:app:prod_gnn_as_tn:tensor_networks} for a self-contained introduction to tensor networks).
	The tensor network corresponding to a GNN with product aggregation adheres to a tree structure~---~its leaves are associated with input vertex features and interior nodes embody the operations performed by the GNN.
	Importing machinery from tensor analysis literature, we prove that $\sepranknoflex{ \funcgraph{\params}{\graph} }{\I}$ is upper bounded by a minimal cut weight in the corresponding tensor network, among cuts separating leaves associated with input vertices in~$\I$ from leaves associated with input vertices in~$\I^c$.
	\cref{gnn:eq:sep_rank_upper_bound_graph_pred} then follows by finding such a cut in the tensor network with sufficiently low weight.
	\cref{gnn:eq:sep_rank_upper_bound_vertex_pred} is established analogously.
\end{proof}

A natural question is whether the upper bounds in~\cref{gnn:thm:sep_rank_upper_bound} are tight, \ie~whether separation ranks close to them can be attained.
We show that nearly matching lower bounds hold for almost all assignments of weights $\params$.
To this end, we define \emph{admissible subsets of $\cut_\I$}, based on a notion of vertex subsets with \emph{no repeating shared neighbors}.
\begin{definition}
	\label{gnn:def:no_rep_neighbors}
	We say that $\I, \J \subseteq \vertices$ \emph{have no repeating shared neighbors} if every $k \in \neigh (\I) \cap \neigh (\J)$ has only a single neighbor in each of $\I$ and $\J$, \ie~$\abs{\neigh (k) \cap \I} = \abs{\neigh (k) \cap \J} = 1$.
\end{definition}

\begin{definition}
	\label{gnn:def:admissible_subsets}
	For $\I \subseteq \vertices$, we refer to $\cut \subseteq \cut_\I$ as an \emph{admissible subset of $\cut_\I$} if there exist $\I' \subseteq \I, \J' \subseteq \I^c$ with no repeating shared neighbors such that $\cut = \neigh (\I') \cap \neigh (\J')$.
	We use $\cutset (\I)$ to denote the set comprising all admissible subsets of $\cut_\I$:
	\[
	\cutset (\I) := \brk[c]1{ \cut \subseteq \cut_\I : \cut \text{ is an admissible subset of $\cut_\I$} }
	\text{\,.}
	\]
\end{definition}

\cref{gnn:thm:sep_rank_lower_bound} below establishes that almost all possible values for the network's weights lead the upper bounds in~\cref{gnn:thm:sep_rank_upper_bound} to be tight, up to logarithmic terms and to the number of walks from $\cut_\I$ being replaced with the number of walks from any single $\cut \in \cutset (\I)$.
The extent to which $\cut_\I$ can be covered by an admissible subset thus determines how tight the upper bounds are.
Trivially, at least the shared neighbors of any $i \in \I, j \in \I^c$ can be covered, since $\neigh (i) \cap \neigh (j) \in \cutset (\I)$.
\cref{gnn:app:sep_rank_examples} shows that for various canonical graphs all of $\cut_\I$, or a large part of it, can be covered by an admissible subset.

\begin{theorem}
	\label{gnn:thm:sep_rank_lower_bound}
	Consider the setting and notation of~\cref{gnn:thm:sep_rank_upper_bound}.
	Given $\I \subseteq \vertices$, for almost all assignments of weights $\params$, \ie~for all but a set of Lebesgue measure zero, it holds that:
	\begin{align}
		\text{(graph prediction)} \quad &\log \brk1{ \seprankbig{ \funcgraph{\params}{\graph} }{\I} }
		\geq
		\max_{ \cut \in \cutset (\I) } \log \brk{ \alpha_{\cut} } \cdot \nwalk{L - 1}{\cut}{\vertices}
		\text{\,,}
		\label{gnn:eq:sep_rank_lower_bound_graph_pred} \\[0.5em]
		\text{(vertex prediction)} \quad &\log \brk1{ \seprankbig{ \funcvert{\params}{\graph}{ t } }{\I} }
		\geq
		\max_{ \cut \in \cutset (\I) } \log \brk{ \alpha_{\cut, t} } \cdot \nwalk{L - 1}{\cut}{ \{ t \} }
		\text{\,,}
		\label{gnn:eq:sep_rank_lower_bound_vertex_pred}
	\end{align}
	where:
	\[
	\begin{split}
	\alpha_{\cut} & := \begin{cases}
		\mindim^{1 / \nwalk{0}{\cut}{\vertices} } & , \text{if } L = 1 \\
		\brk{ \mindim - 1 } \cdot \nwalk{L - 1}{\cut}{\vertices}^{-1} + 1 & , \text{if } L \geq 2
	\end{cases}
	\text{\,,} \\
	\alpha_{\cut, t} & := \begin{cases}
		\mindim & , \text{if } L = 1 \\
		\brk{ \mindim - 1 } \cdot \nwalk{L - 1}{\cut}{ \{ t \}}^{-1} + 1 & , \text{if } L \geq 2
	\end{cases}
	\text{\,,}
	\end{split}
	\]
	with $\mindim := \min \brk[c]{ \indim, \hdim }$.
	If $\nwalk{L - 1}{\cut}{\vertices} = 0$ or $\nwalk{L - 1}{\cut}{ \{ t \}} = 0$, the respective lower bound (right hand side of~\cref{gnn:eq:sep_rank_lower_bound_graph_pred} or~\cref{gnn:eq:sep_rank_lower_bound_vertex_pred}) is zero by convention.
\end{theorem}

\begin{proof}[Proof sketch (proof in~\cref{gnn:app:proofs:sep_rank_lower_bound})]
	Our proof follows a line similar to that used in~\cite{levine2020limits,wies2021transformer,levine2022inductive} for lower bounding the separation rank of self-attention neural networks.
	The separation rank of any $f : ( \R^{\indim})^{\abs{\vertices}} \to \R$ can be lower bounded by examining its outputs over a grid of inputs.
	Specifically, for $M \in \N$ \emph{template vectors} $\vbf^{(1)}, \ldots, \vbf^{(M)} \in \R^{\indim}$, we can create a \emph{grid tensor} for $f$ by evaluating it over each point in \smash{$\brk[c]{ \brk{ \vbf^{(d_1)} , \ldots, \vbf^{(d_{ \abs{ \vertices} } ) } } }_{d_1, \ldots, d_{ \abs{ \vertices }} = 1}^M$} and storing the outcomes in a tensor with $\abs{\vertices}$ axes of dimension $M$ each.
	Arranging the grid tensor as a matrix $\gridmatnoflex{f}$ where rows correspond to axes indexed by $\I$ and columns correspond to the remaining axes, we show that $\rank ( \gridmatnoflex{f} ) \leq \seprank{f}{\I}$.
	The proof proceeds by establishing that for almost every assignment of $\theta$, there exist template vectors with which $\log \brk{ \rank \brk{ \gridmatnoflex{ \funcgraph{\params}{\graph} } } }$ and $\log \brk{ \rank \brk{ \gridmatnoflex{ \funcvert{\params}{\graph}{t} } } }$ are greater than (or equal to) the right hand sides of~\cref{gnn:eq:sep_rank_lower_bound_graph_pred,gnn:eq:sep_rank_lower_bound_vertex_pred}, respectively.
\end{proof}

\textbf{Directed edges and multiple edge types.}~~ \cref{gnn:app:extensions}~generalizes~\cref{gnn:thm:sep_rank_upper_bound,gnn:thm:sep_rank_lower_bound} to the case of graphs with directed edges and an arbitrary number of edge types.

\subsection{Empirical Demonstration}
\label{gnn:sec:analysis:experiments}

Our theoretical analysis establishes that, the strength of interaction GNNs can model across a partition of input vertices is primarily determined by the partition’s walk index~---~a graph-theoretical characteristic defined by the number of walks originating from the boundary of the partition (see~\cref{gnn:def:walk_index}).  
The analysis formally applies to GNNs with product aggregation (see~\cref{gnn:sec:gnns}), yet we empirically demonstrate that its conclusions carry over to various other message-passing GNN architectures, namely GCN~\citep{kipf2017semi}, GAT~\citep{velivckovic2018graph}, and GIN~\citep{xu2019powerful} (with ReLU non-linearity).
Specifically, through controlled experiments, we show that such models perform better on tasks in which the partitions that require strong interaction are ones with higher walk index, given that all other aspects of the tasks are the same.
A description of these experiments follows.
For brevity, we defer some implementation details to~\cref{gnn:app:experiments:details}.

We constructed two graph prediction datasets, in which the vertex features of each input graph are patches of pixels from two randomly sampled Fashion-MNIST~\citep{xiao2017fashion} images, and the goal is to predict whether the two images are of the same class.\footnote{
	Images are sampled such that the amount of positive and negative examples are roughly balanced.
}
In both datasets, all input graphs have the same structure: two separate cliques with~$16$ vertices each, connected by a single edge.
The datasets differ in how the image patches are distributed among the vertices: in the first dataset each clique holds all the patches of a single image, whereas in the second dataset each clique holds half of the patches from the first image and half of the patches from the second image.
\cref{gnn:fig:sep_rank_analysis} illustrates how image patches are distributed in the first (left hand side of the figure) and second (right hand side of the figure) datasets, with blue and red marking assignment of vertices to images.

\begin{table*}
	\caption{
		In accordance with our theory (\cref{gnn:sec:analysis:overview,gnn:sec:analysis:formal}), GNNs can better fit datasets in which the partitions (of input vertices) that require strong interaction are ones with higher walk index (\cref{gnn:def:walk_index}).
		Table reports means and standard deviations, taken over five runs, of train and test accuracies obtained by GNNs of depth~$3$ and width~$16$ on two datasets: one in which the essential partition~---~\ie~the main partition requiring strong interaction~---~has low walk index, and another in which it has high walk index (see~\cref{gnn:sec:analysis:experiments} for a detailed description of the datasets).
		For all GNNs, the train accuracy attained over the second dataset is considerably higher than that attained over the first dataset.
		Moreover, the better train accuracy translates to better test accuracy.
		See~\cref{gnn:app:experiments:details} for further implementation details.
	}
	\begin{center}
		\small
		\begin{tabular}{llcc}
			\toprule
			& & \multicolumn{2}{c}{Essential Partition Walk Index} \\
			\cmidrule(lr){3-4}
			& & Low & High \\
			\midrule
			\multirow{2}{*}{GCN} & Train Acc. (\%) & $70.4~\pm$ \scriptsize{$1.7$} & $\mathbf{81.4}~\pm$ \scriptsize{$2.0$} \\
			& Test Acc. ~~(\%) & $52.7~\pm$ \scriptsize{$1.9$} & $\mathbf{66.2}~\pm$ \scriptsize{$1.1$} \\
			\midrule
			\multirow{2}{*}{GAT} & Train Acc. (\%) & $82.8~\pm$ \scriptsize{$2.6$} & $\mathbf{88.5}~\pm$ \scriptsize{$1.1$} \\
			& Test Acc. ~~(\%) & $69.6~\pm$ \scriptsize{$0.6$} & $\mathbf{72.1}~\pm$ \scriptsize{$1.2$} \\
			\midrule
			\multirow{2}{*}{GIN} & Train Acc. (\%) & $83.2~\pm$ \scriptsize{$0.8$} & $\mathbf{94.2}~\pm$ \scriptsize{$0.8$} \\
			& Test Acc. ~~(\%) & $53.7~\pm$ \scriptsize{$1.8$} & $\mathbf{64.8}~\pm$ \scriptsize{$1.4$} \\
			\bottomrule
		\end{tabular}
	\end{center}
	\label{tab:low_vs_high_walkindex}
\end{table*}

Each dataset requires modeling strong interaction across the partition separating the two images, referred to as the \emph{essential partition} of the dataset.
In the first dataset the essential partition separates the two cliques, thus it has low walk index.
In the second dataset each side of the essential partition contains half of the vertices from the first clique and half of the vertices from the second clique, thus the partition has high walk index.
For an example illustrating the gap between these walk indices see~\cref{gnn:fig:sep_rank_analysis}.

\cref{tab:low_vs_high_walkindex} reports the train and test accuracies achieved by GCN, GAT, and GIN (with ReLU non-linearity) over both datasets.
In compliance with our theory, the GNNs fit the dataset whose essential partition has high walk index significantly better than they fit the dataset whose essential partition has low walk index.
Furthermore, the improved train accuracy translates to improvements in test accuracy.

\section{Practical Application: Expressivity Preserving Edge Sparsification}
\label{gnn:sec:sparsification}

\cref{gnn:sec:analysis} theoretically characterizes the ability of a GNN to model interactions between input vertices. 
It reveals that this ability is controlled by a graph-theoretical property we call walk index (\cref{gnn:def:walk_index}). 
The current section derives a practical application of our theory, specifically, an \emph{edge sparsification} algorithm named \emph{Walk Index Sparsification} (\emph{WIS}), which preserves the ability of a GNN to model interactions when input edges are removed.  
We present WIS, and show that it yields high predictive performance for GNNs over standard vertex prediction benchmarks of various scales, even when removing a significant portion of edges.
In particular, we evaluate WIS using GCN~\citep{kipf2017semi}, GIN~\citep{xu2019powerful}, and ResGCN~\citep{li2020deepergcn} over multiple datasets, including: Cora~\citep{sen2008collective}, which contains thousands of edges, DBLP~\citep{bojchevski2018deep}, which contains tens of thousands of edges, and OGBN-ArXiv~\citep{hu2020open}, which contains more than a million edges.
WIS is simple, computationally efficient, and in our experiments has markedly outperformed alternative methods in terms of induced prediction accuracy across edge sparsity levels.  
We believe its further empirical investigation is a promising direction for future research.

\vspace{-1mm}
\subsection{Walk Index Sparsification (WIS)}
\label{gnn:sec:sparsification:wis}

Running GNNs over large-scale graphs can be prohibitively expensive in terms of runtime and memory.
A natural way to tackle this problem is edge sparsification~---~removing edges from an input graph while attempting to maintain prediction accuracy (\cf~\cite{li2020sgcn,chen2021unified}).\footnote{
	As opposed to \emph{edge rewiring} methods that add or remove only a few edges with the goal of improving prediction accuracy~(\eg,~\cite{zheng2020robust,luo2021learning,topping2022understanding,banerjee2022oversquashing}).
}\textsuperscript{,}\footnote{
	An alternative approach is to remove vertices from an input graph (see,~\eg,~\cite{leskovec2006sampling}).  
	However, this approach is unsuitable for vertex prediction tasks, so we limit our attention to edge sparsification.
}

Our theory (\cref{gnn:sec:analysis}) establishes that, the strength of interaction a depth~$L$ GNN can model across a partition of input vertices is determined by the partition's walk index, a quantity defined by the number of length $L - 1$ walks originating from the partition's boundary.
This brings forth a recipe for pruning edges.
First, choose partitions across which the ability to model interactions is to be preserved.
Then, for every input edge (excluding self-loops), compute a tuple holding what the walk indices of the chosen partitions will be if the edge is to be removed.
Lastly, remove the edge whose tuple is maximal according to a preselected order over tuples (\eg~an order based on the sum, min, or max of a tuple's entries).
This process repeats until the desired number of edges are removed.
The idea behind the above-described recipe, which we call \emph{General Walk Index Sparsification}, is that each iteration greedily prunes the edge whose removal takes the smallest toll in terms of ability to model interactions across chosen partitions~---~see~\cref{alg:walk_index_sparsification_general} in~\cref{gnn:app:walk_index_sparsification_general} for a formal outline.
Below we describe a specific instantiation of the recipe for vertex prediction tasks, which are particularly relevant with large-scale graphs, yielding our proposed algorithm~---~Walk Index Sparsification (WIS). 

In vertex prediction tasks, the interaction between an input vertex and the remainder of the input graph is of central importance.
Thus, it is natural to choose the partitions induced by singletons (\ie~the partitions $( \{t\}, \vertices \setminus \{t\})$, where $t \in \vertices$) as those across which the ability to model interactions is to be preserved.
We would like to remove edges while avoiding a significant deterioration in the ability to model interaction under any of the chosen partitions.  
To that end, we compare walk index tuples according to their minimal entries, breaking ties using the second smallest entries, and so forth.
This is equivalent to sorting (in ascending order) the entries of each tuple separately, and then ordering the tuples lexicographically.

\cref{alg:walk_index_sparsification_vertex} provides a self-contained description of the method attained by the foregoing choices.  
We refer to this method as $(L-1)$-Walk Index Sparsification (WIS), where the “$(L-1)$'' indicates that only walks of length $L-1$ take part in the walk indices.
Since $(L - 1)$-walk indices can be computed by taking the $(L - 1)$'th power of the graph's adjacency matrix, $(L - 1)$-WIS runs in $\OO ( N \abs{\edges} \abs{\vertices}^3 \log (L) )$ time and requires $\OO ( \abs{\edges} \abs{\vertices} + \abs{\vertices}^2 )$ memory, where $N$ is the number of edges to be removed.
For large graphs a runtime cubic in the number of vertices can be restrictive. 
Fortunately, $1$-WIS, which can be viewed as an approximation for $(L - 1)$-WIS with $L > 2$, facilitates a particularly simple and efficient implementation based solely on vertex degrees, requiring only linear time and memory~---~see~\cref{alg:walk_index_sparsification_vertex_one} (whose equivalence to $1$-WIS is explained in~\cref{gnn:app:one_wis_efficient}).
Specifically, $1$-WIS runs in $\OO (N \abs{\edges} + \abs{\vertices})$ time and requires $\OO (\abs{\edges} + \abs{\vertices})$ memory.

\begin{algorithm}[t!]
	\caption{$(L - 1)$-Walk Index Sparsification (WIS) ~ (instance of a general scheme described in~\cref{gnn:app:walk_index_sparsification_general})} 
	\label{alg:walk_index_sparsification_vertex}	
	\begin{algorithmic}
		\STATE \!\!\!\!\textbf{Input:} $\graph$~---~graph , $L \in \N$~---~GNN depth , $N \in \N$~---~number of edges to remove \\[0.2em]
		\STATE \!\!\!\!\textbf{Result:} Sparsified graph obtained by removing $N$ edges from $\graph$ \\[0.4em]
		\hrule
		\vspace{2mm}
		\FOR{$n = 1, \ldots, N$}
		\vspace{0.5mm}
		\STATE \darkgray{\# for every edge, compute walk indices of partitions induced by $\{t\}$, for $t \in \vertices$, after the edge's removal}
		\vspace{0.5mm}
		\FOR{$e \in \edges$ (excluding self-loops)}
		\vspace{0.5mm}
		\STATE initialize $\sbf^{(e)} = \brk{0, \ldots, 0} \in \R^{\abs{\vertices}}$
		\vspace{0.5mm}
		\STATE remove $e$ from $\graph$ (temporarily)
		\vspace{0.5mm}
		\STATE for every $t \in \vertices$, set $\sbf^{(e)}_{t} = \walkinvert{L - 1}{t}{ \{t\} }$
		\vspace{0.5mm}
		\STATE add $e$ back to $\graph$
		\vspace{0.5mm}
		\ENDFOR
		\vspace{0.5mm}
		\STATE \darkgray{\# prune edge whose removal harms walk indices the least according to an order over $\brk{ \sbf^{(e)} }_{e \in \edges}$}
		\vspace{0.5mm}
		\STATE for $e \in \edges$, sort the entries of $\sbf^{(e)}$ in ascending order
		\vspace{0.5mm}
		\STATE let $e' \in \argmax_{e \in \edges} \sbf^{(e)}$ according to lexicographic order over tuples
		\vspace{0.5mm}
		\STATE \textbf{remove} $e'$ from $\graph$ (permanently)
		\vspace{0.5mm}
		\ENDFOR
	\end{algorithmic}
\end{algorithm}

\begin{algorithm}[t!]
	\caption{$1$-Walk Index Sparsification (WIS) ~ (efficient implementation of \cref{alg:walk_index_sparsification_vertex} for $L = 2$)} 
	\label{alg:walk_index_sparsification_vertex_one}	
	\begin{algorithmic}
		\STATE \!\!\!\!\textbf{Input:} $\graph$~---~graph , $N \in \N$~---~number of edges to remove \\[0.2em]
		\STATE \!\!\!\!\textbf{Result:} Sparsified graph obtained by removing $N$ edges from $\graph$ \\[0.4em]
		\hrule
		\vspace{2mm}
		\STATE compute vertex degrees, \ie~$\text{deg} (i) = \abs{\neigh (i)}$ for $i \in \vertices$
		\vspace{0.5mm}
		\FOR{$n = 1, \ldots, N$}
		\vspace{0.5mm}
		\FOR{$\{i, j\} \in \edges$ (excluding self-loops)}
		\vspace{0.5mm}
		\STATE  let $\text{deg}_{min} (i, j) := \min \brk[c]{ \text{deg} (i) , \text{deg} (j) }$ 
		\vspace{0.5mm}
		\STATE let $\text{deg}_{max} (i, j) := \max \brk[c]{ \text{deg} (i) , \text{deg} (j) }$
		\vspace{0.5mm}		
		\ENDFOR
		\vspace{0.5mm}
		\STATE \darkgray{\# prune $\{i, j\} \in \edges$ with maximal $\text{deg}_{min} (i, j)$, breaking ties using $\text{deg}_{max} (i, j)$}
		\vspace{0.5mm}
		\STATE let $e' \in \argmax_{\brk[c]{i, j} \in \edges} \brk1{ \text{deg}_{min} (i, j) , \text{deg}_{max} (i, j) }$ according to lexicographic order over pairs
		\vspace{0.5mm}
		\STATE \textbf{remove} $e'$ from $\graph$ \\
		\vspace{0.5mm}
		\STATE decrease by one the degree of vertices connected by $e'$
		\vspace{0.5mm}		
		\ENDFOR
	\end{algorithmic}
\end{algorithm}

\subsection{Empirical Evaluation}
\label{gnn:sec:sparsification:eval}

\begin{figure*}[t!]
	\vspace{2mm}
	\begin{center}
		\hspace{-2mm}
		\includegraphics[width=1\textwidth]{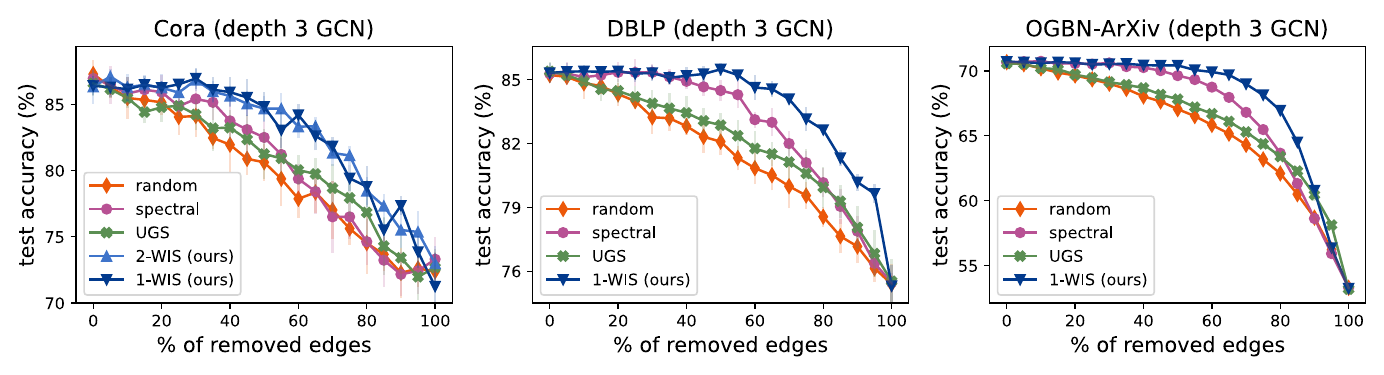}
	\end{center}
	\vspace{-3mm}
	\caption{
		Comparison of GNN accuracies following sparsification of input edges~---~WIS, the edge sparsification algorithm brought forth by our theory (\cref{alg:walk_index_sparsification_vertex}), markedly outperforms alternative methods.
		Plots present test accuracies achieved by a depth $L = 3$ GCN of width~$64$ over the Cora (left), DBLP (middle), and OGBN-ArXiv (right) vertex prediction datasets, with increasing percentage of removed edges (for each combination of dataset, edge sparsification algorithm, and percentage of removed edges, a separate GCN was trained and evaluated).
		WIS, designed to maintain the ability of a GNN to model interactions between input vertices, is compared against: \emph{(i)} removing edges uniformly at random; \emph{(ii)} a spectral sparsification method~\citep{spielman2011graph}; and \emph{(iii)}~an adaptation of UGS~\citep{chen2021unified}.
		For Cora, we run both $2$-WIS, which is compatible with the GNN's depth, and $1$-WIS, which can be viewed as an approximation that admits a particularly efficient implementation (\cref{alg:walk_index_sparsification_vertex_one}).
		For DBLP and OGBN-ArXiv, due to their larger scale only $1$-WIS is evaluated.
		Markers and error bars report means and standard deviations, respectively, taken over ten runs per configuration.
		Note that $1$-WIS achieves results similar to $2$-WIS, suggesting that the efficiency it brings does not come at a significant cost in performance.
		\cref{gnn:app:experiments} provides further implementation details and experiments with additional GNN architectures (GIN and ResGCN) and datasets (Chameleon, Squirrel, and Amazon Computers).		
		Code for reproducing the experiment is available at \url{https://github.com/noamrazin/gnn_interactions}.
	}
	\label{gnn:fig:sparse_gcn}
\end{figure*}

Below is an empirical evaluation of WIS.
For brevity, we defer to~\cref{gnn:app:experiments} some implementation details, as well as experiments with additional GNN architectures (GIN and ResGCN) and datasets (Chameleon~\citep{pei2020geom}, Squirrel~\citep{pei2020geom}, and Amazon Computers~\citep{shchur2018pitfalls}).
Overall, our evaluation includes six standard vertex prediction datasets in which we observed the graph structure to be crucial for accurate prediction, as measured by the difference between the test accuracy of a GCN trained and evaluated over the original graph and its test accuracy when trained and evaluated over the graph after all of the graph's edges were removed.
We also considered, but excluded, the following datasets in which the accuracy difference was insignificant (less than five percentage points): Citeseer~\citep{sen2008collective}, PubMed~\citep{namata2012query}, Coauthor CS and Physics~\citep{shchur2018pitfalls}, and Amazon Photo~\citep{shchur2018pitfalls}.

Using depth $L = 3$ GNNs (with ReLU non-linearity), we evaluate over the Cora dataset both $2$-WIS, which is compatible with the GNNs' depth, and $1$-WIS, which can be viewed as an efficient approximation.
Over the DBLP and OGBN-ArXiv datasets, due to their larger scale only $1$-WIS is evaluated.
\cref{gnn:fig:sparse_gcn} (and~\cref{gnn:fig:sparse_gin} in~\cref{gnn:app:experiments}) shows that WIS significantly outperforms the following alternative methods in terms of induced prediction accuracy: \emph{(i)} a baseline in which edges are removed uniformly at random; \emph{(ii)} a well-known spectral algorithm~\citep{spielman2011graph} designed to preserve the spectrum of the sparsified graph's Laplacian; and \emph{(iii)} an adaptation of UGS~\citep{chen2021unified}~---~a recent supervised approach for learning to prune edges.\footnote{
	UGS~\citep{chen2021unified} jointly prunes input graph edges and GNN weights.
	For fair comparison, we adapt it to only remove edges.
}
Both $2$-WIS and $1$-WIS lead to higher test accuracies, while (as opposed to UGS) avoiding the need for labels, and for training a GNN over the original (non-sparsified) graph~---~a procedure which in some settings is prohibitively expensive in terms of runtime and memory.
Interestingly, $1$-WIS performs similarly to $2$-WIS, indicating that the efficiency it brings does not come at a sizable cost in performance.

%% file: Chapters/gnn_related.tex
\chapter{Related Work}
\label{chap:gnn_related}

\textbf{Expressiveness of GNNs.}~~The expressiveness of GNNs has been predominantly evaluated through ability to distinguish non-isomorphic graphs, as measured by correspondence to Weisfeiler-Leman (WL) graph isomorphism tests (see~\cite{morris2021weisfeiler} for a recent survey).
\cite{xu2019powerful,morris2019weisfeiler} instigated this thread of research, establishing that message-passing GNNs are at most as powerful as the WL algorithm, and can match it under certain technical conditions.
Subsequently, architectures surpassing WL were proposed, with expressiveness measured via higher-order WL variants (see, \eg,~\cite{morris2019weisfeiler,maron2019provably,chen2019equivalence,geerts2020expressive,balcilar2021breaking,bodnar2021weisfeiler,barcelo2021graph,geerts2022expressiveness,bouritsas2022improving,papp2022theoretical}).
Another line of inquiry regards universality among continuous permutation invariant or equivariant functions~\cite{maron2019universality,keriven2019universal,loukas2020graph,azizian2021expressive,geerts2022expressiveness}.
\cite{chen2019equivalence} showed that distinguishing non-isomorphic graphs and universality are, in some sense, equivalent.
Lastly, there exist analyses of expressiveness focused on the frequency response of GNNs~\cite{nt2019revisiting,balcilar2021analyzing} and their capacity to compute specific graph functions, \eg~moments, shortest paths, and substructure counting~\cite{dehmamy2019understanding,barcelo2020logical,garg2020generalization,loukas2020graph,chen2020can,chen2021graph,bouritsas2022improving}. 

Although a primary purpose of GNNs is to model interactions between vertices, none of the past works formally characterize their ability to do so, as our theory does.
\cref{chap:gnn_interactions} thus provides a novel perspective on the expressive power of GNNs.
Furthermore, a major limitation of existing approaches~---~in particular, proofs of equivalence to WL tests and universality~---~is that they often operate in asymptotic regimes of unbounded network width or depth.
Consequently, they fall short of addressing which type of functions can be realized by GNNs of practical size.
In contrast, we characterize how the modeled interactions depend on both the input graph structure and the neural network architecture (width and depth).
As shown in~\cref{gnn:sec:sparsification}, this facilitates designing an efficient and effective edge sparsification algorithm.

\textbf{Measuring modeled interactions via separation rank.}~~Separation rank (\cref{gnn:sec:prelim:sep_rank}) has been paramount to the study of interactions modeled by certain convolutional, recurrent, and self-attention neural networks.
It enabled theoretically analyzing how different architectural parameters impact expressiveness~\cite{cohen2016expressive,cohen2016convolutional,cohen2017inductive,cohen2018boosting,balda2018tensor,sharir2018expressive,levine2018deep,levine2018benefits,khrulkov2018expressive,khrulkov2019generalized,levine2020limits,wies2021transformer,levine2022inductive} and implicit regularization (\cref{chap:imp_reg_htf}).\footnote{
	We note that, over a two-dimensional grid graph, a message-passing GNN can be viewed as a convolutional neural network with overlapping convolutional windows.
	Similarly, over a chain graph, it can be viewed as a bidirectional recurrent neural network.
	Thus, for these special cases, our separation rank bounds (delivered in~\cref{gnn:sec:analysis}) extend those of~\cite{cohen2017inductive,levine2018benefits,levine2018deep,khrulkov2018expressive}, which consider convolutional neural networks with non-overlapping convolutional windows and unidirectional recurrent neural networks.
}
On the practical side, insights brought forth by separation rank led to tools for improving performance, including: guidelines for architecture design~\cite{cohen2017inductive,levine2018deep,levine2020limits,wies2021transformer}, pretraining schemes~\cite{levine2022inductive}, and regularizers for countering locality in convolutional neural networks (\cref{chap:imp_reg_htf}).
We employ separation rank for studying the interactions GNNs model between vertices, and similarly provide both theoretical insights and a practical application~---~edge sparsification algorithm (\cref{gnn:sec:sparsification}).

\textbf{Edge sparsification.}~~Computations over large-scale graphs can be prohibitively expensive in terms of runtime and memory.
As a result, various methods were proposed for sparsifying graphs by removing edges while attempting to maintain structural properties, such as distances between vertices~\cite{baswana2007simple,hamann2016structure}, graph Laplacian spectrum~\cite{spielman2011graph,sadhanala2016graph}, and vertex degree distribution~\cite{voudigari2016rank}, or outcomes of graph analysis and clustering algorithms~\cite{chakeri2016spectral}.
Most relevant to our work are recent edge sparsification methods aiming to preserve the prediction accuracy of GNNs as the number of removed edges increases~\cite{li2020sgcn,chen2021unified}.
These methods require training a GNN over the original (non-sparsified) graph, hence only inference costs are reduced.
Guided by our theory, in~\cref{gnn:sec:sparsification} we propose \emph{Walk Index Sparsification} (\emph{WIS})~---~an edge sparsification algorithm that preserves expressive power in terms of ability to model interactions.
WIS improves efficiency for both training and inference.
Moreover, comparisons with the spectral algorithm of~\cite{spielman2011graph} and a recent method from~\cite{chen2021unified} demonstrate that WIS brings about higher prediction accuracies across edge sparsity levels.

%% file: Chapters/conclusion.tex
Two pillars on which the theory of deep learning rests are generalization and expressiveness.
Strengthening the formal understanding of these pillars can facilitate principled methods for improving the efficiency, reliability, and performance of neural networks.
A major challenge towards doing so is finding suitable complexity measures.
That is, measures with which it is possible to characterize the ability of neural networks to generalize over natural data (\eg, images, audio, and text) and express rich classes of functions.
This thesis puts forth notions of rank as promising measures of complexity for developing a theory of deep learning.

In \cref{part:gen}, we focused on the mystery of generalization in deep learning: why do neural networks generalize despite having far more learnable parameters than training examples?
Conventional wisdom suggests that this generalization stems from an implicit regularization induced by gradient-based training, \ie~its tendency to fit training examples with predictors of minimal complexity~\cite{neyshabur2017implicit}.
A widespread hope was that this tendency can be characterized as minimization of some norm~\cite{neyshabur2015search,gunasekar2017implicit}.
Contradicting prior belief~\cite{gunasekar2017implicit}, we proved that implicit regularization cannot be captured by norms in the context of matrix factorization~---~a model equivalent to linear neural networks.
Instead, we showed that it is more faithfully described as an implicit minimization of rank.
Then, capitalizing on this interpretation, we established that the tendency towards low rank extends from linear neural networks to more practical non-linear neural networks (with polynomial non-linearity), which are equivalent to tensor factorizations.

In \cref{part:gnns}, we employed the connection between neural networks and tensor factorizations to study the expressiveness of graph neural networks (GNNs).
Our analysis characterized the ability of certain GNNs to model interactions between vertices via an established measure known as separation rank~\cite{beylkin2009multivariate,cohen2017inductive}.
In particular, it formalized intuition by which GNNs can model stronger interactions between areas of the graph that are more interconnected.

In terms of practical impact, based on the presented theory we developed: \emph{(i)} a regularization scheme for improving the performance of convolutional neural networks over tasks involving non-local interactions; and \emph{(ii)} a state of the art edge sparsification algorithm, called Walk Index Sparsification (WIS), that preserves the ability of GNNs to model interactions.
Moreover, other research groups have also built upon our analyses of implicit rank minimization for designing practical deep learning systems~\cite{jing2020implicit,huh2021low}.

Overall, our work highlights that notions of rank may be key for explaining the remarkable performance of neural networks over natural data.

\subsection*{Future Work}

Our theoretical analysis considered neural networks with polynomial non-linearity, by employing their connection with tensor factorizations.
Such neural networks have demonstrated competitive performance in practice~\cite{cohen2014simnets,sharir2016tensorial,stoudenmire2018learning,chrysos2020p}, and we empirically demonstrated (in \cref{htf:sec:countering_locality,gnn:sec:analysis:experiments}) that conclusions from their analysis apply to neural networks with more popular non-linearities, such as ReLU.
Nonetheless, we view extending our theory to account for additional non-linearities as a promising direction for future research.
A possible approach is to build on the connection between \emph{generalized tensor factorizations}~\cite{cohen2016convolutional} and neural networks with non-polynomial non-linearities.

Our work also raises several interesting directions concerning WIS~---~the edge sparsification algorithm introduced in~\cref{gnn:sec:sparsification}.
A naive implementation of $(L - 1)$-WIS has runtime cubic in the number of vertices (\cf~\cref{gnn:sec:sparsification:wis}).
Since this can be restrictive for large-scale graphs, the evaluation in~\cref{gnn:sec:sparsification:eval} mostly focused on $1$-WIS, which can be viewed as an efficient approximation of $(L - 1)$-WIS (its runtime and memory requirements are linear~---~see~\cref{gnn:sec:sparsification:wis}).
Future work can develop efficient exact implementations of $(L - 1)$-WIS (\eg~using parallelization) and investigate regimes where it outperforms $1$-WIS in terms of induced prediction accuracy.
Additionally, $(L - 1)$-WIS is a specific instantiation of the general WIS scheme (given in~\cref{gnn:app:walk_index_sparsification_general}), tailored for preserving the ability to model interactions across certain partitions. 
Exploring other instantiations, as well as methods for automatically choosing the partitions across which the ability to model interactions is preserved, are valuable directions for further research.

%% file: Appendices/imp_reg_not_norms.tex
\chapter{Implicit Regularization in Deep Learning May Not Be Explainable \\ by Norms} 
\label{mf:app:imp_reg_not_norms}

\section{Extension to Different Matrix Dimensions} 
\label{mf:app:dimensions}

In this appendix we outline an extension of the construction and analysis given in \cref{mf:sec:analysis:setting,mf:sec:analysis:norms_up}, respectively, to completion of matrices with dimensions beyond~$2$-by-$2$.
The extension presented here is not unique, but rather one simple option out of many.
It is demonstrated empirically in Figure~\ref{mf:fig:experiment_dmf_diff_dimensions}.

Beginning with square matrices, for $2 \leq D \in \N$, consider completion of a $D$-by-$D$ matrix based on the following observations:
\bea
\Omega &=& \{1, \ldots, D\} \times \{1, \ldots , D\} \, \setminus \, \{ (1 , 1 ) \} 
\text{\,,} 
\nonumber
\\[1mm]
y_{i,j} &=& 
\begin{cases}
	1 & , \text{if } i = j \geq 3 \text{ or } (i,j) \in \{(1,2), (2,1)\}\\
	0 & , \text{otherwise} 
\end{cases} 
~ , ~ \text{for } ( i , j ) \in \Omega
\label{mf:eq:obs_dxd}
\text{\,,}
\eea
where, as in~\cref{mf:sec:model}, $\Omega$~represents the set of observed locations, and $\{ y_{i , j} \in \R \}_{( i , j ) \in \Omega}$ the~corresponding set of observed values.
The solution set for this problem (\ie~the set of matrices zeroing the loss in Equation~\eqref{mf:eq:loss}) is: 
\be
\S_D := \left \{
\begin{pmatrix}
	w_{1,1} \hspace{-2mm} & 1 & 0 & 0 & \hspace{-2mm} \cdots \hspace{-2mm} & 0 \\[0.5mm]
	1 \hspace{-2mm} & 0 & 0 & 0 & \hspace{-2mm} \cdots \hspace{-2mm} & 0 \\[0.5mm]
	0 \hspace{-2mm} & 0 & 1 & 0 & \hspace{-2mm} \cdots \hspace{-2mm} & 0 \\[0.5mm]
	0 \hspace{-2mm} & 0 & 0 & 1 &  & 0 \\[-1.5mm]
	\vdots \hspace{-2mm} & \vdots & \vdots &  & \hspace{-2mm} \ddots \hspace{-2mm} &  \\
	0 \hspace{-2mm} & 0 & 0 & 0 &  & 1
\end{pmatrix}
\in \R^{D \times D} 
: w_{1,1} \in \R \right \} 
\label{mf:eq:sol_set_dxd}
\text{\,.}
\ee
Observing~$\S_D$, while comparing to the solution set~$\S$ in our original construction (Equation~\eqref{mf:eq:sol_set}), we see that the former has a $2$-by-$2$ block diagonal structure, with the top-left block holding the latter, and the bottom-right block set to identity.
This implies that $D - 2$ of the singular values along~$\S_D$ are fixed to one, and the remaining two are identical to the singular values along~$\S$.
Results analogous to Propositions~\ref{mf:prop:sol_set_norms} and~\ref{mf:prop:sol_set_rank} can therefore easily be proven.
Since the determinant along~$\S_D$ is bounded below and away from zero (it is equal to~$-1$), approaching~$\S_D$ while having positive determinant necessarily means that absolute value of unobserved entry (\ie~of the entry in location~$( 1 , 1 )$) grows towards infinity.
Combining this with the fact that the end matrix (Equation~\eqref{mf:eq:prod_mat}) of a depth~$L \geq 2$ matrix factorization maintains the sign of its determinant (see Lemma~\ref{mf:lem:det_does_not_change_sign} in Appendix~\ref{mf:app:lemma:dmf}), results analogous to Theorem~\ref{mf:thm:norms_up_finite} and Corollary~\ref{mf:cor:norms_up_asymp} may readily be established.
That is, one may show that, with probability~$0.5$ or more over random near-zero initialization, gradient descent with small learning rate drives \emph{all} norms (and quasi-norms) \emph{towards infinity}, while essentially driving rank towards~its~minimum.

\begin{figure}[t]
	\begin{center}
	\includegraphics[width=0.31\textwidth]{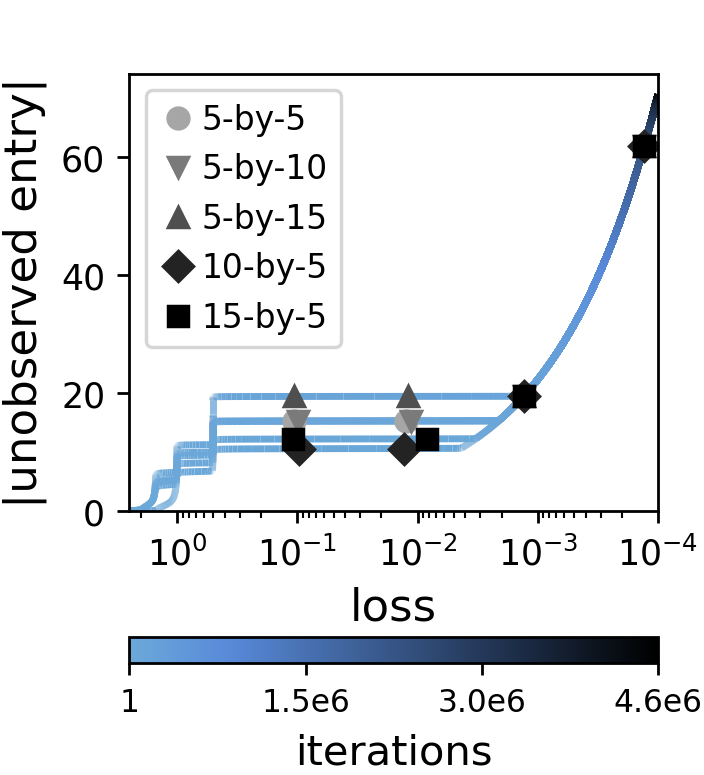}
	\end{center}
	\vspace{-2mm}
	\caption{
		Phenomenon of implicit regularization in matrix factorization driving \emph{all} norms (and quasi-norms) \emph{towards infinity} extends to arbitrary matrix dimensions.
		Appendix~\ref{mf:app:dimensions} outlines an extension of the construction and analysis given in \cref{mf:sec:analysis:setting,mf:sec:analysis:norms_up}, respectively, to completion of matrices with arbitrary dimensions.
		The extension implies that for any $2 \leq D, D' \in \N$, when applying matrix factorization to the specified $D$-by-$D'$ matrix completion problem, decreasing loss, \ie~fitting observations, can lead absolute value of unobserved entry to increase (which in turn means norms and quasi-norms increase).
		This is demonstrated in the plot above, which for representative runs corresponding to different choices of $D$ and~$D'$, shows absolute value of unobserved entry as a function of the loss (Equation~\ref{mf:eq:loss}), with iteration number encoded by color.
		Runs were obtained with a depth~$3$ matrix factorization initialized randomly by an unbalanced (layer-wise independent) distribution, with the latter's standard deviation and the learning rate for gradient descent set to the smallest values used for depth~$3$ in Figure~\ref{mf:fig:experiment_dmf} (other settings we evaluated produced similar results).
		For further implementation details see Appendix~\ref{mf:app:experiments:details:dmf}.
	}
	\label{mf:fig:experiment_dmf_diff_dimensions}
\end{figure}

Moving on to the rectangular case, for $2 \leq D , D' \in \N$, consider completion of a $D$-by-$D'$ matrix based on the same observations as in Equation~\eqref{mf:eq:obs_dxd}, but with additional zero observations such that only the entry in location~$( 1 , 1 )$ is unobserved.
The singular values along the solution set for this problem are the same as those along~$\S_D$ (Equation~\eqref{mf:eq:sol_set_dxd}).
Moreover, assuming without loss of generality that~$D \leq D'$, if a matrix factorization applied to this problem is initialized such that its end matrix holds zeros in columns $D + 1$ to~$D'$, then a dynamical characterization from~\cite{arora2018optimization} (restated as Lemma~\ref{mf:lem:prod_mat_dyn} in Appendix~\ref{mf:app:lemma:dmf}), along with the structure of the loss (Equation~\eqref{mf:eq:loss}), ensure the leftmost $D$-by-$D$~submatrix of the end matrix evolves precisely as in the square case discussed above, while the remaining columns ($D + 1$~to~$D'$) stay at zero.
Results thus carry over from the square to the rectangular case.

\section{Further Experiments and Implementation Details} \label{mf:app:experiments}

\subsection{Further Experiments} \label{mf:app:experiments:further}

Figures~\ref{mf:fig:experiment_dmf_diff_dimensions} and~\ref{mf:fig:experiment_dmf_perturb} supplement Figure~\ref{mf:fig:experiment_dmf} from \cref{mf:sec:experiments:analyzed}, by demonstrating empirically that the phenomenon of implicit regularization in matrix factorization driving all norms (and quasi-norms) towards infinity is, respectively:
\emph{(i)}~applicable to arbitrary matrix dimensions, as outlined in Appendix~\ref{mf:app:dimensions};
and
\emph{(ii)}~robust to perturbations, as proven in Section~\ref{mf:sec:analysis:robust}.
Figure~\ref{mf:fig:experiment_tf_r3} supplements Figure~\ref{mf:fig:experiment_tf} from \cref{mf:sec:experiments:tensor}, further demonstrating that gradient descent over tensor factorization exhibits an implicit regularization towards low tensor rank.

\subsection{Implementation Details} \label{mf:app:experiments:details}

\begin{figure}
	\begin{center}
		\hspace{-5mm}
		\subfloat{
			\includegraphics[width=0.31\textwidth]{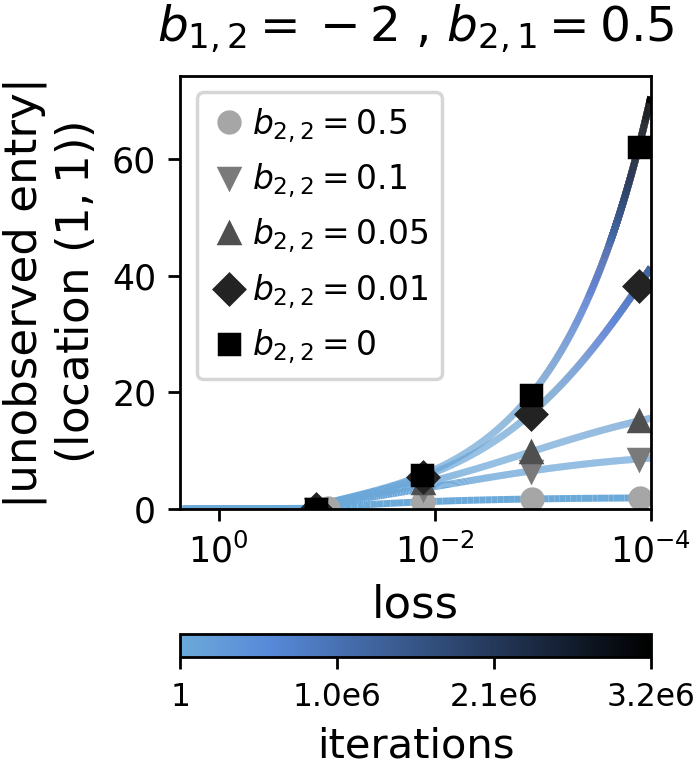}
		} ~~~~~~~~~~~~
		\subfloat{
			\includegraphics[width=0.31\textwidth]{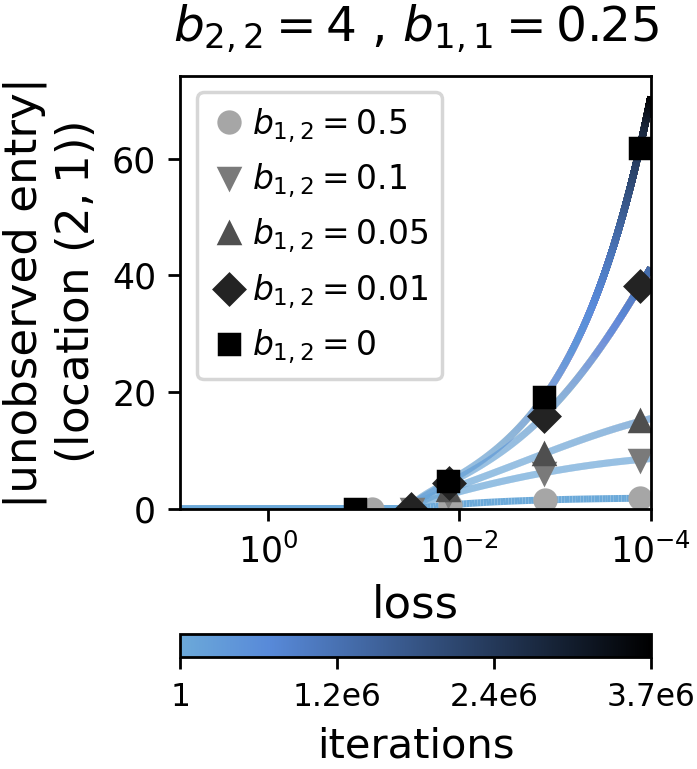}
		}
	\end{center}
	\vspace{-2mm}
	\caption{
		Phenomenon of implicit regularization in matrix factorization driving \emph{all} norms (and quasi-norms) \emph{towards infinity} is robust to perturbations.
		Our analysis (\cref{mf:sec:analysis:robust}) implies that, when applying matrix factorization to the matrix completion problem defined in \cref{mf:sec:analysis:setting}, even if observations are perturbed and repositioned, decreasing loss, \ie~fitting them, leads absolute value of unobserved entry to increase (which in turn means norms and quasi-norms increase).
		Specifically, with $( i , j ) \in \{ 1 , 2 \} \times \{ 1 , 2 \}$ representing the unobserved location and $\bar{i} := 3 - i$, $\bar{j} := 3 - j$, Theorem~\ref{mf:thm:norms_up_finite_perturb} implies that:
		\emph{(i)}~if the diagonally-opposite observation~$y_{\bar{i} , \bar{j}}$ is unperturbed (stays at zero), the adjacent ones $y_{i , \bar{j}} , y_{\bar{i} , j}$ can take on \emph{any} non-zero values, and as long as at initialization the sign of the end matrix's (Equation~\ref{mf:eq:prod_mat}) determinant accords with that of $y_{i , \bar{j}} \cdot y_{\bar{i} , j}$, the absolute value of unobserved entry will grow to infinity;
		and
		\emph{(ii)}~the extent to which absolute value of unobserved entry grows gracefully recedes as $y_{\bar{i} , \bar{j}}$ is perturbed away from zero.
		This is demonstrated in the plots above, which for representative runs, show absolute value of unobserved entry as a function of the loss (Equation~\ref{mf:eq:loss}), with iteration number encoded by color.
		Each plot corresponds to a different choice of~$( i , j )$ and a different assignment for~$y_{i , \bar{j}} , y_{\bar{i} , j}$, presenting runs with varying values for~$y_{\bar{i} , \bar{j}}$.
		Runs were obtained with a depth~$3$ matrix factorization initialized randomly by an unbalanced (layer-wise independent) distribution, with the latter's standard deviation and the learning rate for gradient descent set to the smallest values used for depth~$3$ in Figure~\ref{mf:fig:experiment_dmf} (other settings we evaluated produced similar results).
		For further implementation details see Appendix~\ref{mf:app:experiments:details:dmf}.
	}
	\label{mf:fig:experiment_dmf_perturb}
\end{figure}

\begin{figure}[t]
	\begin{center}
		\hspace*{-3.5mm}
		\subfloat{
			\includegraphics[width=0.48\textwidth]{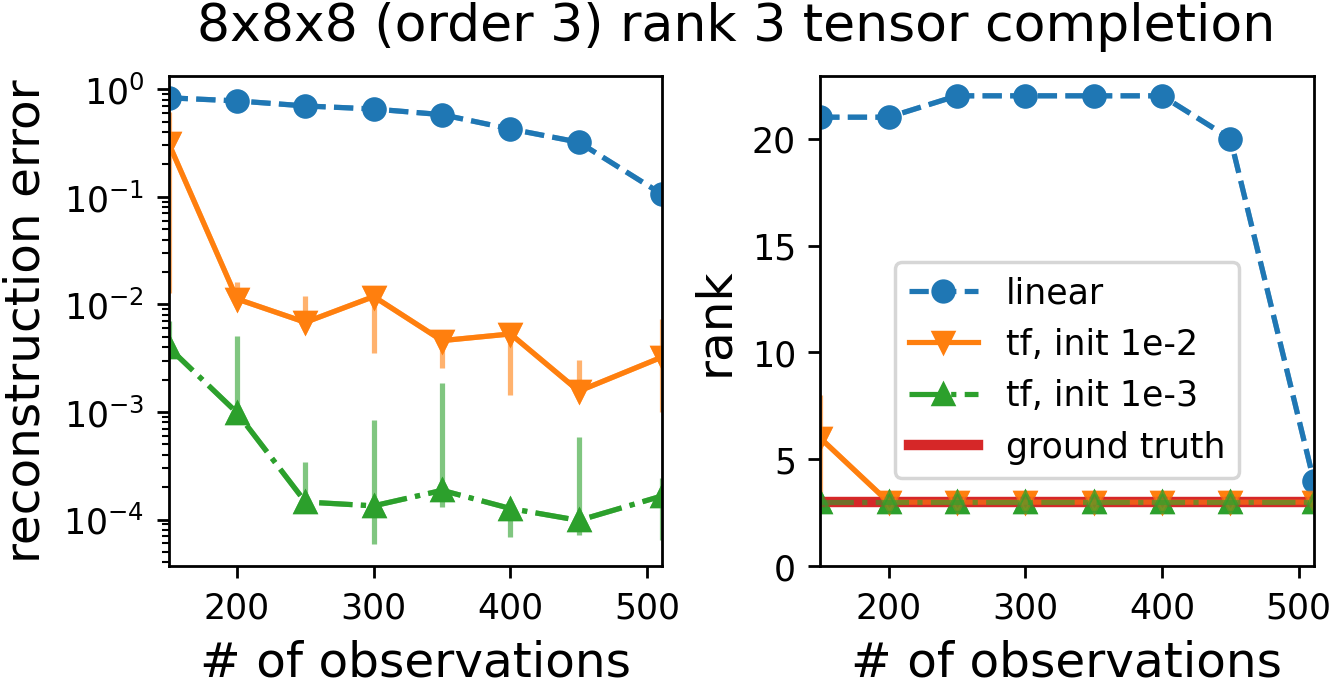}
		} ~~
		\subfloat{
			\includegraphics[width=0.48\textwidth]{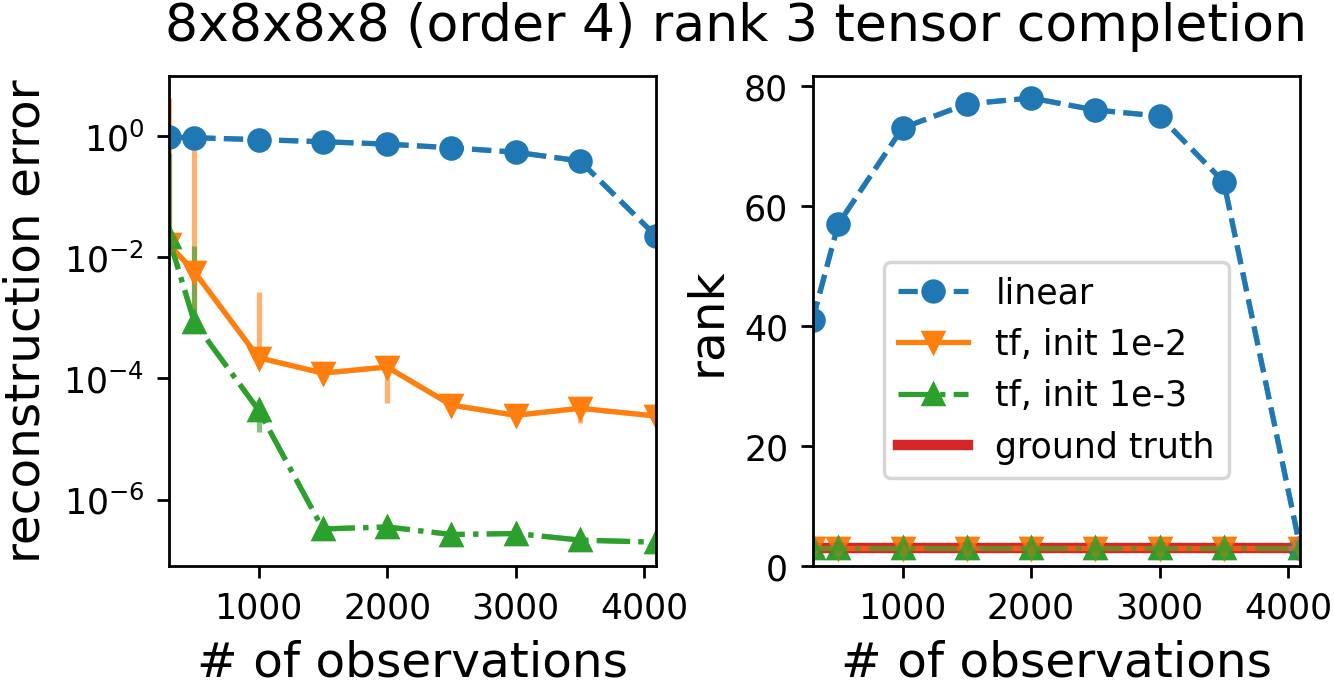}
		}
	\end{center}
	\vspace{-1mm}
	\caption{
		Gradient descent over tensor factorization exhibits an implicit regularization towards low tensor rank.
		This figure is identical to Figure~\ref{mf:fig:experiment_tf}, except that the experiments it portrays had ground truth tensors of rank~$3$ (instead of~$1$).
		For further details see caption of Figure~\ref{mf:fig:experiment_tf}, as well as Appendix~\ref{mf:app:experiments:details:tf}.
	}
	\label{mf:fig:experiment_tf_r3}
\end{figure}

Below we provide implementation details omitted from the experimental reports in Section~\ref{mf:sec:experiments} and \cref{mf:app:experiments:further}.
Source code for reproducing our results and figures, based on the PyTorch framework~\cite{paszke2019pytorch}, can be found at \url{https://github.com/noamrazin/imp_reg_dl_not_norms}.

\subsubsection{Deep Matrix Factorization (Figures~\ref{mf:fig:experiment_dmf},~\ref{mf:fig:experiment_dmf_diff_dimensions}, and~\ref{mf:fig:experiment_dmf_perturb})} \label{mf:app:experiments:details:dmf}

In all experiments with deep matrix factorization, hidden dimensions were set to the minimal value ensuring unconstrained search space, \ie~to the minimum between the number of rows and the number of columns in the matrix to complete.
Gradient descent was run with fixed learning rate until loss (Equation~\eqref{mf:eq:loss}) reached a value lower than~$10^{-4}$ or $5 \cdot 10^6$~iterations elapsed.
Both balanced (Equation~\eqref{mf:eq:balance}) and unbalanced (layer-wise independent) random initializations were calibrated according to a desired standard deviation~$\alpha > 0$ for the entries of the initial end matrix (Equation~\eqref{mf:eq:prod_mat}).
Namely:
\emph{(i)}~under unbalanced initialization, entries of all weight matrices were sampled independently from a Gaussian distribution with zero mean and standard deviation $(\alpha^2 / \bar{D}^{L - 1})^{1 / 2L}$, where~$L$ stands for the depth of the factorization, and~$\bar{D}$ for the size of its hidden dimensions;
and
\emph{(ii)}~under balanced initialization, we used Procedure~1 from~\cite{arora2019convergence}, based on a Gaussian distribution with independent entries, zero mean and standard deviation~$\alpha$.

In accordance with the description in Appendix~\ref{mf:app:dimensions}, if the matrix to complete was rectangular, we ensured that excess rows or columns of the initial end matrix held zeros, by clearing (setting to zero) corresponding rows or columns of the initial leftmost or rightmost (respectively) matrix in the factorization.\footnote{
	That is, if the matrix to complete had size $D$-by-$D'$ with $D \neq D'$, we cleared rows $D' + 1$ to~$D$ of~$\Wbf_L ( 0 )$ if~$D > D'$, and columns $D + 1$ to~$D'$ of~$\Wbf_1 ( 0 )$ if~$D' > D$.
}
Random initializations were repeated until the determinant of the initial end matrix (or of its top-left $\min \{ D, D' \}$-by-$\min \{ D, D' \}$ submatrix if its size was $D$-by-$D'$ with $D \neq D'$) was of the necessary sign,\footnote{
	Positive for the experiments reported by Figures~\ref{mf:fig:experiment_dmf} and~\ref{mf:fig:experiment_dmf_diff_dimensions}, and negative for those reported by Figure~\ref{mf:fig:experiment_dmf_perturb}.
}
taking two attempts on average.
In the experiment reported by Figure~\ref{mf:fig:experiment_dmf}, runs with matrix factorization depths~$2$ and~$3$ were carried out with learning rates $\{ 6 \cdot 10^{-2} , 3 \cdot 10^{-2}, 9 \cdot 10^{-3}, 6 \cdot 10^{-3}, 3 \cdot 10^{-3}, 9 \cdot 10^{-4} \}$ and corresponding standard deviations for initialization $\{ 10^{-2}, 10^{-3}, 10^{-4}, 10^{-5}, 10^{-6}, 10^{-7} \}$.
Factorizations of depth~$4$ were slightly more sensitive to changes in learning rate, thus we refined attempted values to $\{ 6 \cdot 10^{-3}, 4.5 \cdot 10^{-3}, 3 \cdot 10^{-3}, 1.5 \cdot 10^{-3}, 10^{-3} \}$, with corresponding standard deviations for initialization $\{ 10^{-1}, 10^{-2}, 10^{-3}, 10^{-4}, 10^{-5} \}$.

\subsubsection{Tensor Factorization (Figures~\ref{mf:fig:experiment_tf} and~\ref{mf:fig:experiment_tf_r3})} \label{mf:app:experiments:details:tf}

In all experiments with tensor factorization (Equation~\eqref{mf:eq:tf}), the number of terms~$R$ was set to ensure an unconstrained search space, \ie~it was set to $8^2$ and~$8^3$ for tensor sizes $8$-by-$8$-by-$8$ and $8$-by-$8$-by-$8$-by-$8$ respectively.\footnote{
	As shown in~\cite{hackbusch2012tensor}, for any $D_1 , \ldots , D_N \in \N$, using $R = ( \Pi_{n = 1}^N D_n ) / \max \{ D_n \}_{n = 1}^N$ suffices for~expressing all tensors in $\R^{D_1 \times \cdots \times D_N}$.
}
Horizontal axes in all plots begin from the smallest number of observations producing stable results, and end when all entries but one are observed.
Specifically:
\emph{(i)}~in the experiments with rank~$1$ ground truth tensors (Figure~\ref{mf:fig:experiment_tf}), the number of observations ranged over $\{ 50, 100, 150, \ldots , 400, 450, 511 \}$ and $\{ 100, 500, 1000, 1500, \ldots, 3000, 3500, 4095\}$ for orders~$3$ and~$4$ respectively;
and
\emph{(ii)}~for experiments with rank~$3$ ground truth tensors (Figure~\ref{mf:fig:experiment_tf_r3}), the minimal number of observations was increased threefold (\ie~ranges of $\{ 150, 200, 250, \ldots , 400, 450, 511 \}$ and $\{ 300, 500, 1000, 1500, \ldots, 3000, 3500, 4095\}$ were used for orders~$3$ and~$4$ respectively).

Gradient descent was run until the mean squared error over observations reached a value lower than~$10^{-6}$ or~$10^6$ iterations elapsed.
For initialization, weights were sampled independently from a Gaussian distribution with zero mean and varying standard deviation.
In particular, five trials (differing in random seed) were conducted for each standard deviation in the range $\{10^{-1}, 10^{-2}, 10^{-3}, 10^{-4} \}$.
To facilitate more efficient experimentation, we employed an adaptive learning rate scheme, where at each iteration a base learning rate of~$10^{-2}$ was divided by the square root of an exponential moving average of squared gradient norms.
That is, with base learning rate $\eta = 10^{-2}$ and weighted average coefficient $\beta = 0.99$, at iteration~$t$ the learning rate was set to $\eta_t = \eta / (\sqrt{\gamma_t / (1 - \beta^t)} + 10^{-6})$, where $\gamma_t = \beta \cdot \gamma_{t-1} + (1 - \beta) \cdot \sum \hspace{0mm}_{r = 1}^{R} \hspace{0mm}_{n = 1}^{N} \| \nicefrac{\partial}{\partial \wbf_r^{n}} \ell ( \{ \wbf^{n}_r ( t ) \}_{r , n} ) \|^2_{Fro}$, with $\gamma_0 = 0$ and $\ell ( \cdot )$ standing for the mean squared error over observations.
We emphasize that only the learning rate (step size) is affected by this scheme, not the direction of movement.
Comparisons between the scheme and a fixed (small) learning rate schedule have shown no noticeable impact on the end result, with significant difference in terms of run time.

While exact inference of tensor rank is in the worst case computationally hard (\cf~\cite{haastad1990tensor}), in practice, a standard way to estimate it is by the minimal number of terms ($R$~in Equation~\eqref{mf:eq:tf}) for which the Alternating Least Squares (ALS) algorithm achieves reconstruction (mean squared) error below a certain threshold (see~\cite{kolda2009tensor} for further details).
We follow this method with a threshold of~$10^{-6}$.
Generating a ground truth rank~$R^*$ tensor $\W^* \in \R^{D_1 \times \cdots \times D_N}$ was done by computing:
\[
\W^* = \sum\nolimits_{r = 1}^{R^*} \vbf^{1 }_r \otimes \cdots \otimes \vbf^{N }_r
\quad , ~
\vbf^{ n }_r \in \R^{D_n}
~ , ~
r = 1 ,  \ldots , R^*
~ , ~
n = 1 , \ldots , N
\text{\,,}
\]
with $\{ \vbf^{n }_r \}_{r = 1}^{R^*} \hspace{0mm}_{n = 1}^{N}$ drawn independently from the standard normal distribution.
After every such generation, we estimated the rank of the obtained tensor (its construction only ensures a rank of \emph{at most}~$R^*$), and repeated the process if it was smaller than~$R^*$.
For convenience, we subsequently normalized the ground truth tensor to be of unit Frobenius norm.

\section{Deferred Proofs} \label{mf:app:proofs}

\subsection{Notation} \label{mf:app:proofs:notation}

We define a few notational conventions that will be used throughout our proofs.
For $N \in \N$, let $[N]$ denote the set $\{ 1 , \ldots , N \}$.
Let $\{ \ebf_i \}_{i = 1}^D \subset \R^D$ be the standard basis vectors, \ie~$\ebf_i$ holds~$1$ in its $i$'th coordinate and $0$ elsewhere.
The singular values of a matrix $\Wbf \in \R^{D \times D'}$ are denoted by $\sigma_1 (\Wbf) \geq \ldots \geq \sigma_{\min \{D, D'\}}(\Wbf) \geq 0$, where by convention $\sigma_i (\Wbf) := 0$ for $i > \min \{D, D' \}$.
Similarly, the eigenvalues of a symmetric matrix $\Wbf \in \R^{D \times D}$ are denoted by $\lambda_1(\Wbf) \geq \ldots \geq \lambda_D (\Wbf)$. 
We let $\norm{\Wbf}_{S_p}$, with $p \in (0, \infty]$, stand for the Schatten-$p$ (quasi-)norm of a matrix $\Wbf \in \R^{D \times D'}$, and denote by $\norm{\Wbf}_{Fro}$ the special case $p = 2$, \ie~the Frobenius norm.
The Euclidean norm of a vector $\wbf \in \R^D$ is denoted by $\norm{\wbf}_2$.
Since norms are a special case of quasi-norms, when providing results applicable to both, only the latter is explicitly treated.
To admit a compact representation of matrix products, given $1\leq a \leq b \leq L$ and matrices $\Wbf_1, \ldots ,\Wbf_L$ for which the product $\Wbf_L \cdots \Wbf_1$ is defined, we denote:
\[
\begin{split}
	& \prod\nolimits_{a}^{r=b} \Wbf_r := \Wbf_b \cdots \Wbf_a ~, \\
	& \prod\nolimits_{r=a}^{b} \Wbf_r^\top  := \Wbf_a^\top  \cdots \Wbf_b^\top \text{\,.}
\end{split}
\]
By definition, if $a > b$, then both $\prod_{a}^{r=b} \Wbf_r$ and $\prod_{r=a}^{b} \Wbf_r^\top $ are identity matrices, with size to be inferred by context.
The $k$'th derivative of a function (from~$\R$ to~$\R$) $f(t)$ is denoted by $f^{(k)}(t)$, with $f^{(0)}(t) := f(t)$ by convention.
For consistency with differential equations literature, when the variable~$t$ is regarded as a time index, we also denote the first order derivative by~$\dot{f}(t)$.
Lastly, when clear from context, a time index~$t$ will often be omitted.

\subsection{Useful Lemmas} \label{mf:app:lemmas}

\subsubsection{Deep Matrix Factorization} \label{mf:app:lemma:dmf}

For completeness, we include the following result from~\cite{arora2018optimization}, which characterizes the evolution of the end matrix under gradient flow on a deep matrix factorization:
\begin{lemma}[adaptation of Theorem~1 in~\cite{arora2018optimization}] \label{mf:lem:prod_mat_dyn}
	Let $\mfendloss : \R^{D \times D'} \to \R_{\geq 0}$ be an analytic\footnote{
		An infinitely differentiable function $f: \D \rightarrow \R$ is \textit{analytic} if at every $\xbf \in \D$ its Taylor series converges to it on some neighborhood of $\xbf$ (see \cite{krantz2002primer} for further details).
		Specifically, the matrix completion loss considered (Equation~\eqref{mf:eq:loss}) is analytic.
		\label{note:analytic_func}
	}
	loss, overparameterized by a depth~$L$ matrix factorization:
	\[
	\mfobj (\Wbf_1, \ldots, \Wbf_L) = \mfendloss (\Wbf_L \cdots \Wbf_1)
	\text{\,.}
	\]
	Suppose we run gradient flow over the factorization:
	\[
	\dot{\Wbf}_l (t) := \tfrac{d}{dt} \Wbf_l(t) = -\tfrac{\partial}{\partial \Wbf_l} \mfobj(\Wbf_1(t), \ldots, \Wbf_L(t))
	\quad,~ t \geq 0 ~,~ l = 1, \ldots, L
	\text{\,,}
	\]
	with a balanced initialization,~\ie:
	\[
	\Wbf_{l + 1} (0)^\top \Wbf_{l + 1}(0) = \Wbf_l(0) \Wbf_l (0)^\top
	\quad,~ l = 1, \ldots, L - 1
	\text{\,.}
	\]
	Then, the end matrix $\matrixend (t) = \Wbf_L(t) \cdots \Wbf_1(t)$ obeys the following dynamics:
	\[
	\matrixenddot (t) = -\sum\nolimits_{l = 1}^L \left[ \matrixend(t) \matrixend (t)^\top \right]^\frac{l - 1}{L} \cdot \nabla \mfendloss\big(\matrixend(t)\big) \cdot \left[ \matrixend (t)^\top \matrixend(t) \right]^\frac{L - l}{L}
	\text{\,,}
	\]
	where~$[\,\cdot\,]^\beta$, $\beta \in \R_{\geq 0}$, stands for a power operator defined over positive semidefinite matrices (with $\beta = 0$ yielding identity by definition).
\end{lemma}

Additionally, recall from~\cite{arora2019implicit} the following characterization for the singular values of $\matrixend(t)$:
\begin{lemma}[adaptation of Lemma~1 and Theorem~3 in~\cite{arora2019implicit}] \label{mf:lem:prod_mat_sing_dyn}
	Consider the setting of Lemma~\ref{mf:lem:prod_mat_dyn} for depth $2 \leq L \in \N$.
	Then, there exist analytical functions $\{ \sigma_r : [0, \infty) \rightarrow \R_{\geq 0} \}_{r=1}^{\min \{D, D'\}}, \{ \ubf_r : [0,\infty) \rightarrow \R^D \}_{r=1}^{\min \{D, D'\}}$ and $\{ \vbf_r : [0, \infty) \rightarrow \R^{D'} \}_{r=1}^{\min \{D, D'\}}$ such that:
	\[
	\begin{split}
		&\ubf_r ( t )^\top \ubf_{r'} ( t ) = \vbf_r ( t )^\top \vbf_{r'} ( t ) = 
		\begin{cases}
			1 & , r = r' \\
			0 & , r \neq r' 
		\end{cases}
		\quad , ~
		t \geq 0 
		~ , ~ 
		r , r' \in [\min \{ D, D' \}] \text{\,,} \\[1mm]
		& \matrixend(t) = \sum\nolimits_{r = 1}^{\min \{D, D'\}} \sigma_r(t) \ubf_r (t) \vbf_r (t)^\top \text{\,,}
	\end{split}
	\]
	\ie~$\sigma_r (t) \geq 0$ are the singular values of $\matrixend (t)$, and $\ubf_r (t), \vbf_r (t)$ are corresponding left and right (respectively) singular vectors.
	Furthermore, the singular values~$\sigma_r (t)$ evolve by:
	\be
	\dot{\sigma}_r (t) = -L \cdot \left ( \sigma_r^2 (t) \right )^{1 - 1 / L} \cdot \left \langle \nabla \mfendloss \left ( \matrixend (t) \right ) , \ubf_r (t) \vbf_r (t)^\top \right \rangle \quad ,~r = 1, \ldots, \min \{ D, D' \}
	\text{\,.}
	\label{mf:eq:prod_mat_sing_value_ode}
	\ee
\end{lemma}

We rely on this result to establish that for square end matrices the sign of $\det (\matrixend(t))$ does not change throughout time. 
\begin{lemma} \label{mf:lem:det_does_not_change_sign}
	Consider the setting of Lemma~\ref{mf:lem:prod_mat_dyn} with depth $2 \leq L \in \N$ and $D = D'$.
	Then, the determinant of $\matrixend(t)$ has the same sign as its initial value $\det (\matrixend(0))$.
	That is, $\det (\matrixend(t))$ is identically zero if $\det (\matrixend(0)) = 0$, is positive if $\det (\matrixend(0)) > 0$, and is negative if $\det (\matrixend(0)) < 0$.
\end{lemma}

\begin{proof}
	We prove an analogous claim for the singular values of $\matrixend(t)$, from which the lemma readily follows.
	That is, for $r \in [D]$, the singular value $\sigma_r (t)$ is identically zero if $\sigma_r (0) = 0$, and is positive if $\sigma_r (0) > 0$.
	
	For conciseness, define $g(t) := -L \cdot \left \langle \nabla \mfendloss \left ( \matrixend (t) \right ) , \ubf_r (t) \vbf_r (t)^\top \right \rangle$.
	Invoking Lemma~\ref{mf:lem:prod_mat_sing_dyn}, let us solve the differential equation for  $\sigma_r (t)$.
	If $L = 2$, the solution to Equation~\eqref{mf:eq:prod_mat_sing_value_ode} is $\sigma_r (t) = \sigma_r (0) \cdot \exp \left ( \int\nolimits_{t' = 0}^t g(t') dt' \right )$.
	Clearly, $\sigma_r (t)$ is either identically zero or positive according to its initial value.
	If $L > 2$, Equation~\eqref{mf:eq:prod_mat_sing_value_ode} is solved by:
	\[
	\sigma_r (t) = 
	\begin{cases}
		~~ \left ( \sigma_r (0)^{\frac{2}{L} - 1} + \left ( \frac{2}{L} - 1 \right ) \int\nolimits_{t'=0}^t g(t') dt' \right )^{\frac{1}{\frac{2}{L} - 1}} & , \sigma_r (0) > 0 \\
		\qquad\qquad\qquad\qquad 0 & , \sigma_r (0) = 0
	\end{cases}
	\text{\,.}
	\]
	As before, if $\sigma_r (0) = 0$, then $\sigma_r (t) = 0$ for all $t \geq 0$.
	If $\sigma_r (0) > 0$, divergence in finite time of $\sigma_r(t)$ is possible, however, its positivity is preserved until that occurs nonetheless.
	
	Turning our attention to the determinant of $\matrixend (t)$, suppose $\det (\matrixend(0)) = 0$. Then, $\matrixend(0)$ has a singular value which is $0$, and for all $t$ that singular value and the determinant remain $0$.
	If $\det (\matrixend(0)) \neq 0$, the end matrix remains full rank for all $t$.
	The proof then immediately follows from the continuity of $\det (\matrixend(t))$.
\end{proof}

We will also make use of the following lemmas:
\begin{lemma}[adapted from~\cite{arora2019implicit}] \label{mf:lem:analytic_factors_and_prod_mat}
	Under the setting of Lemma~\ref{mf:lem:prod_mat_dyn}, $\Wbf_1 (t), \ldots, \Wbf_L (t)$, $\matrixend (t)$ ,and $\nabla \mfendloss (\matrixend (t))$ are analytic functions of $t$.
\end{lemma}

\begin{proof}
	Analytic functions are closed under summation, multiplication, and composition.
	The analyticity of $\mfendloss(\cdot)$ therefore implies that $\mfobj(\cdot)$ (Equation~\eqref{mf:eq:oprm_obj}) is analytic as well.
	From Theorem 1.1 in~\cite{ilyashenko2008lectures}, it then follows that under gradient flow (Equation~\eqref{mf:eq:gf}) $\Wbf_1(t), \ldots , \Wbf_L(t)$ are analytic functions of~$t$. 
	Lastly, the aforementioned closure properties imply that $\matrixend (t)$ and $\nabla \mfendloss ( \matrixend (t) )$ are also analytic in~$t$.
\end{proof}

\subsubsection{Technical} \label{mf:app:lemma:technical}

Included below are a few technical lemmas used in our analyses.
\begin{lemma} \label{mf:lem:entropy_sqrt_bound}
	Let $h: [0, 1] \rightarrow \R$ be the binary entropy function $h(p) := -p \cdot \ln (p) - (1-p) \ln (1-p)$, where by convention $0 \cdot \ln (0) = 0$.
	Then, for all $p \in [0,1]$:
	\[
	h(p) \leq 2 \sqrt{p} \text{\,.}
	\]
\end{lemma}

\begin{proof}
	We present a tighter inequality, $h(p) \leq 2 \sqrt{p (1-p)}$, from which the proof immediately follows since $2 \sqrt{p(1-p)} \leq 2 \sqrt{p}$ for $p \in [0,1]$.
	
	Define the function $f(p) := \frac{h(p)^2}{p(1-p)}$ over the open interval $(0, 1)$.
	Differentiating it with respect to $p$ we have:
	\[
	\frac{d}{dp} f(p) = \frac{ \left (-p \cdot \ln (p) \right )^2 - \left ( -(1-p) \cdot \ln (1-p) \right )^2}{p^2 (1-p)^2}.
	\]
	Introducing $g(p) := -p \cdot \ln (p)$, we show that $g(p)^2 > g(1-p)^2$ for all $p \in (0, \frac{1}{2})$.
	It is easily verified that $g(p) - g(1 - p)$ is concave on the interval $(0, \frac{1}{2})$ (second derivative is negative).
	Since for $p = 0$ and $p = 1 / 2$ we have exactly $g(p) - g(1-p) = 0$, it holds that $g(p) - g(1-p) \geq 0$ and $g(p)^2 \geq g(1-p)^2$ for all $p \in (0, \frac{1}{2})$.
	Noticing $\frac{d}{dp} f(p) = \left ( g(p)^2 - g(1-p)^2 \right ) / p^2 (1-p)^2$, it follows that $f (\cdot)$ is monotonically non-decreasing on $(0, \frac{1}{2})$.
	Due to the fact that $f(p) = f(1-p)$, it is non-increasing on $(\frac{1}{2}, 1)$, and attains its maximal value over $(0, 1)$ at $p = \frac{1}{2}$.
	Putting it all together, for $p \in (0, 1)$ we have:
	\[
	h(p) \leq \sqrt{p (1 - p)} \cdot \sqrt{f(1 / 2)} = 2 \ln (2) \cdot \sqrt{p (1 - p)} \leq 2 \sqrt{p (1 - p)} \text{\,,}
	\]
	and for $p = 0,1$ there is exact equality, completing the proof.
\end{proof}

\begin{lemma} \label{mf:lem:analytic_func_all_deriv_eq}
	Let $f,g: [0, \infty) \rightarrow \R$ be real analytic functions (see Footnote~\ref{note:analytic_func}) such that $f^{(k)}(0) = g^{(k)}(0)$ for all $k \in \N \cup \{ 0 \}$.
	Then, $f(t) = g(t)$ for all $t \geq 0$.
\end{lemma}

\begin{proof}
	Define the function $h(t) := f(t) - g(t)$.
	Since analytic functions are closed under subtraction, $h (\cdot)$ is analytic as well.
	An analytic function with all zero derivatives at a point is constant on the corresponding connected component.
	Noticing that $h^{(k)}(0) = 0$ for all $k \in \N \cup \{ 0 \}$, we may conclude that $h(t) = 0$ and $f(t) = g(t)$ for all $t \geq 0$.
\end{proof}

\begin{lemma} \label{mf:lem:A_TBA_trace_bound}
	Let $\Abf, \Bbf \in \R^{D \times D}$, and suppose $\Bbf$ is positive semidefinite.
	Then,
	\[
	\Tr (\Abf^\top \Bbf\Abf) \geq \lambda_1 (\Bbf) \cdot \sigma_D (\Abf)^2 \text{\,.}	
	\]
\end{lemma}

\begin{proof}
	The matrix $\Abf^\top \Bbf\Abf$ is positive semidefinite since for all $\ybf \in \R^D$ we have:
	\[
	\ybf^\top  \Abf^\top \Bbf\Abf \ybf = (\Abf \ybf )^\top \Bbf(\Abf \ybf ) \geq 0 \text{\,.}
	\]
	Therefore, $\Tr (\Abf^\top \Bbf\Abf) \geq \lambda_{1} (\Abf^\top \Bbf\Abf)$. 
	Let $\Bbf = \Obf \Dbf \Obf^\top$ be an orthogonal eigenvalue decomposition of $\Bbf$, \ie~$\Obf \in \R^{D \times D}$ is an orthogonal matrix with columns $\{\obf_i\}_{i=1}^D$ and $\Dbf \in \R^{D \times D}$ is diagonal holding the non-negative eigenvalues of $\Bbf$.
	Additionally, let $\Abf = \Ubf \Sigma \Vbf^\top $ be a singular value decomposition of $\Abf$, where $\Ubf, \Vbf \in \R^{D \times D}$ are orthogonal matrices, and $\Sigma \in \R^{D \times D}_{\geq 0}$ is diagonal holding the singular values of $\Abf$.
	For any unit vector (with respect to the Euclidean norm) $\ybf \in \R^D$ it holds that:
	\[
	\ybf^\top  \Abf^\top \Bbf\Abf \ybf = \sum\nolimits_{i=1}^D \lambda_i (\Bbf) (\obf_i^\top  \Abf \ybf )^2 \geq \lambda_{1} (\Bbf) (\obf_{1}^\top  \Abf \ybf )^2 \text{\,.}
	\]
	Replacing $\Abf$ with its singular value decomposition and choosing $\ybf = \Vbf \Ubf^\top  \obf_1$:
	\[
	\lambda_{1} (\Bbf) (\obf_{1}^\top  \Abf \ybf )^2  = \lambda_{1} (\Bbf) (\obf_{1}^\top  \Ubf \Sigma \Ubf^\top  \obf_1)^2 \text{\,.}
	\]
	Recalling that for any unit vector the quadratic form of a symmetric matrix is bounded by the maximal and minimal eigenvalues completes the proof:
	\[
	\Tr (\Abf^\top \Bbf \Abf) \geq \lambda_{1} (\Abf^\top \Bbf \Abf) \geq \lambda_{1} (\Bbf) (\obf_{1}^\top  \Ubf \Sigma \Ubf^\top  \obf_1)^2 \geq \lambda_{1} (\Bbf) \cdot \sigma_{d} (\Abf)^2 \text{\,.}
	\]
\end{proof}

\begin{lemma} \label{mf:lem:ode_bounded_in_interval}
	Let $g: [0, \infty) \rightarrow \R$ be a continuously differentiable function, and fix some $t > 0$.
	If $g(t) < g(0)$, then for any $a \in (g(t), g(0)]$ there exists $t_a \in [0, t)$ such that $g ( t_a ) = a$ and $\dot{g} ( t_a ) \leq 0$.
	Similarly, if $g(t) > g(0)$, then for any $a \in [g(0), g(t))$ there exists $t_a \in [0, t)$ such that $g ( t_a ) = a$ and $\dot{g} ( t_a ) \geq 0$.
\end{lemma}

\begin{proof}
	Let $t > 0$ be such that $g(t) < g(0)$, and fix some $a \in (g(t), g(0)]$.
	Define $t_a := \max \{t' : t'~\leq~t \text{ and } g(t') = a \}$.
	Continuity of~$g (\cdot)$, along with the intermediate value theorem, imply that $t_a$ is well defined (maximum of a closed non-empty set bounded from above).
	Assume by contradiction that $\dot{g}(t_a) > 0$.
	Then, $g (\cdot)$~is monotonically increasing on some neighborhood of $t_a$.
	Thus, by the intermediate value theorem, there exists $t' \in (t_a , t)$ such that $g ( t' ) = a$, in contradiction to the definition of $t_a$.
	An identical argument establishes the analogous result for the case $g(t) > g(0)$.
\end{proof}

\begin{lemma} \label{mf:lem:bounded_integral_and_der_convergence}
	Let $g: [0, \infty) \rightarrow \R$ be a non-negative differentiable function. 
	Assume there exist constants $a,b > 0$ such that $\int\nolimits_{t'=0}^{t} g(t') dt' \leq a$ and $\dot{g}(t) \leq b$ for all $t \geq 0$.
	Then, $\lim_{t \rightarrow \infty} g(t) = 0$.
\end{lemma}

\begin{proof}
	By way of contradiction let us assume that $g(t)$ does not converge to $0$.
	Let $\epsilon > 0$ be such that for all $M > 0$ there exists $t > M$ with $g(t) > \epsilon$.
	
	We claim that for all $M, \epsilon' > 0$ there exists $t > M$ such that $g(t) < \epsilon'$. 
	Otherwise, we have a contradiction to the bound on the integral of $g( \cdot )$. Combined with our assumption, this means that for all $M > 0$ we can find an interval $[t_1, t_2]$, with $t_1 > M$, where $g(t)$ transitions from $\frac{\epsilon}{2}$ to $\epsilon$.
	We now examine one such interval. Formally, for $t_0$ with $g(t_0) < \frac{\epsilon}{2}$, we define:
	\[
	t_2 := \min \left \{t | t \geq t_0 \text{ and } g(t) = \epsilon \right \} \quad , \quad t_1 := \max \left \{t | t \leq t_2 \text{ and } g(t) = \epsilon / 2 \right \} \text{\,.}
	\]
	Due to the fact that $g ( \cdot )$ is continuous, $t_2$ and $t_1$ are well defined as they are the minimum and maximum, respectively, of closed non-empty sets bounded from below and above, respectively.
	Furthermore, notice that $t_0 < t_1 < t_2$.
	From the mean value theorem and the bound on the derivative of $g ( \cdot )$ we have $t_2 - t_1 \geq \epsilon / 2b$.
	Since $g(t) \geq \epsilon / 2$ over the interval $[t_1, t_2]$, this gives us $\int\nolimits_{t'=t_1}^{t_2} g(t') dt' \geq \epsilon^2 / 4b$.
	Recall there are infinitely many such occurrences, implying that $\int\nolimits_{t'=0}^{\infty} g(t') dt' = \infty$, in contradiction to the bound on the integral.
\end{proof}

\subsection{Proof of Proposition~\ref{mf:prop:sol_set_norms}}
\label{mf:app:proofs:sol_set_norms}

For a quasi-norm $\norm{\cdot}$, the weakened triangle inequality (see Footnote~\ref{note:qnorm}) implies that there exists a constant $c_{\norm{\cdot}} \geq 1$ for which
\be
\begin{split}
	\norm{W} & \geq \frac{1}{c_{\norm{\cdot}}} \norm{(\Wbf)_{1,1} \ebf_1 \ebf_1^\top} - \norm{W - (\Wbf)_{1,1} \ebf_1 \ebf_1^\top} \\
	& = \abs{(\Wbf)_{1,1}} \frac{ \norm{\ebf_1 \ebf_1^\top}}{c_{\norm{\cdot}}} - \norm{\ebf_2 \ebf_1^\top + \ebf_1 \ebf_3^\top} 
	\text{\,,}
\end{split}	
\label{mf:eqn:triangle_eq_W}
\ee
for any $\Wbf \in \S$.
Fix some $\epsilon > 0$ and define $M_{\norm{\cdot} , \epsilon} := \{(\Wbf)_{1,1} \in \R : \norm{\Wbf} \leq \inf_{\Wbf' \in \S} \norm{\Wbf'} + \epsilon , ~~ \Wbf \in \S \}$, the set of $(\Wbf)_{1,1}$ values corresponding to $\epsilon$-minimizers of $\norm{\cdot}$.
The first part of the proposition thus boils down to showing $M_{\norm{\cdot} , \epsilon}$ is bounded.
By Equation~\eqref{mf:eqn:triangle_eq_W}, there exist a $C > 0$ such that $\abs{(\Wbf)_{1,1}} > C$ means $\norm{W} > \inf_{W' \in \S} \norm{W'} + \epsilon$.
Hence, $M_{\norm{\cdot} , \epsilon} \subset I_{\norm{\cdot}, \epsilon} := [-C, C]$.

If in addition $\norm{\cdot}$ is a Schatten-$p$ quasi-norm for $p \in (0,\infty]$, we now show that $W$ is its minimizer over $\S$ if and only if $(\Wbf)_{1,1} = 0$.
Let $\Wbf_x \in \S$ denote the solution matrix with $(\Wbf_x)_{1,1} = x$ for $x \in \R$.
The singular values of an arbitrary such $\Wbf_x$ are:
\be
\left \{ \sigma_1 (\Wbf_x), \sigma_2 (\Wbf_x) \right \} = \left \{ \abs*{ \big (x + \sqrt{x^2 + 4} \big ) / 2 } , \abs*{ \big ( x - \sqrt{x^2 + 4} \big ) / 2 } \right \}
\text{\,.}
\label{mf:eqn:sol_Wx_sing_values}
\ee
Starting with $p = \infty$, the corresponding norm is the spectral norm $\norm{\Wbf_x}_{S_\infty} := \sigma_1(\Wbf_x)$.
When $x = 0$, we have that $\sigma_1(\Wbf_0) = 1$. 
If $x > 0$, then $\sigma_1(\Wbf_x) = ( x + \sqrt{x^2 + 4} ) ~ / ~ 2 > 1$.
Similarly, if $x < 0$, then $\sigma_1(\Wbf_x) = ( -x + \sqrt{x^2 + 4} ) ~ / ~ 2 > 1$. Therefore, $\norm{\Wbf_x}_{S_\infty}$ attains its minimal value of $1$ if and only if $x = 0$.

Moving to the case of $p \in (0, \infty)$, the corresponding quasi-norm is $\norm{\Wbf_x}_{S_p} := (\sigma_1(\Wbf_x)^p + \sigma_2(\Wbf_x)^p)^{\frac{1}{p}}$.
We now examine $\norm{\Wbf_x}_{S_p}^p$ for $x > 0$:
\[
\norm{\Wbf_x}_{S_p}^p = \left (\frac{x + \sqrt{x^2 + 4}}{2} \right)^p + \left (\frac{-x + \sqrt{x^2 + 4}}{2} \right )^p
\text{\,.}
\]
Differentiating with respect to $x$, we arrive at:
\[
\begin{split}
	& \frac{p}{2^p} \left ( \left (x + \sqrt{x^2 + 4} \right )^{p-1} \left (1 + \frac{x}{\sqrt{x^2 + 4}} \right ) + \left (-x + \sqrt{x^2 + 4} \right )^{p-1} \left (-1 + \frac{x}{\sqrt{x^2 + 4}} \right ) \right ) \\
	& > \frac{p}{2^p} \left ( \left ( x + \sqrt{x^2 + 4} \right )^{p-1} - \left( -x + \sqrt{x^2 + 4} \right )^{p-1} \right ) \\
	& > 0
	\text{\,,}
\end{split}
\]
where in the first transition we used the fact that both $\big ( x + \sqrt{x^2 + 4} \big )^{p-1} > 0$ and $\big ( -x + \sqrt{x^2 + 4} \big )^{p-1} > 0$ (as well as $x > 0$). It then directly follows that $\norm{\Wbf_x}_{S_p}^p$ and thus $\norm{\Wbf_x}_{S_p}$ are monotonically increasing with respect to $x$ on $(0, \infty)$.

Similar arguments show that when $x < 0$ the Schatten-$p$ quasi-norm of $\Wbf_x$ is monotonically decreasing with respect to $x$, implying that $\norm{\Wbf_x}_{S_p}$ is minimized if and only if $x = 0$.\footnote{
	The claim relies on the fact that the Schatten-$p$ quasi-norm of $\Wbf_x$ is continuous with respect to $x$ for all $p \in (0, \infty)$.
	We note, however, that quasi-norms in general may be discontinuous.
}
\qed

\subsection{Proof of Proposition~\ref{mf:prop:sol_set_rank}}
\label{mf:app:proofs:sol_set_rank}

As in the proof of Proposition~\ref{mf:prop:sol_set_norms} (Appendix~\ref{mf:app:proofs:sol_set_norms}), we denote by $\Wbf_x \in \S$ the solution matrix with $(\Wbf_x)_{1,1} = x$.
We begin by analyzing the behavior of $\sigma_1(\Wbf_x)$ and $\sigma_2(\Wbf_x)$ with respect to $x$.
When $x = 0$ the singular values are simply $\sigma_1 (\Wbf_0) = \sigma_2 (\Wbf_0) = 1$.
When $x$ is positive, the singular values may be written as:
\[
\sigma_1(\Wbf_x) = \frac{x + \sqrt{x^2 + 4}}{2} \quad , \quad \sigma_2(\Wbf_x) = \frac{-x + \sqrt{x^2 + 4}}{2}
\text{\,.}
\]
Taking the derivative with respect to $x$, we arrive at:
\[
\frac{d}{dx} \sigma_1(\Wbf_x) = \frac{1}{2} +\frac{x}{2 \sqrt{x^2 + 4}}   \quad , \quad \frac{d}{dx} \sigma_2 (\Wbf_x) = -\frac{1}{2} +\frac{x}{2 \sqrt{x^2 + 4}}
\text{\,.}
\]
Since $x > 0$, we have that $d/dx~\sigma_1 (\Wbf_x) > 0$ and $d/dx~\sigma_2 (\Wbf_x) < 0$.
In other words, $\sigma_1(\Wbf_x)$ is monotonically increasing, while $\sigma_2(\Wbf_x)$ is monotonically decreasing, when $x > 0$.
It can easily be verified that $\sigma_1(\Wbf_x)$ and $\sigma_2(\Wbf_x)$ are even functions of $x$, \ie~$\sigma_1(\Wbf_x) = \sigma_1(\Wbf_{-x})$ and $\sigma_2(\Wbf_x) = \sigma_2(\Wbf_{-x})$.
It then follows that $\sigma_1(\Wbf_x)$ is monotonically decreasing (conversely $\sigma_2(\Wbf_x)$ is monotonically increasing) when $x < 0$. Noticing that $\lim_{x \rightarrow \infty} \sigma_1(\Wbf_x) = \infty$ and $\lim_{x \rightarrow \infty} \sigma_2(\Wbf_x) = 0$ (accordingly $\lim_{x \rightarrow -\infty} \sigma_1(\Wbf_x) = \infty$ and $\lim_{x \rightarrow -\infty} \sigma_2(\Wbf_x) = 0$ ), we have a characterization of the behavior of $\sigma_1(\Wbf_x)$ and $\sigma_2(\Wbf_x)$.

We are now in a position to obtain the desired results for effective and infimal ranks.
The effective rank (Definition~\ref{def:erank}) of $\Wbf_x$ can be written as
\[
\erank ( \Wbf_x ) = \exp \left ( H \left ( \frac{\sigma_1 (\Wbf_x)}{\sigma_1 (\Wbf_x) + \sigma_2 (\Wbf_x)}, \frac{\sigma_2 (\Wbf_x)}{\sigma_1 (\Wbf_x) + \sigma_2 (\Wbf_x)} \right ) \right )
\text{\,.}
\]
The binary entropy function is bounded by $\ln (2)$, hence, the effective rank over $\S$ is bounded by $2$.
This upper bound is attained at $x = 0$.
According to the singular values analysis, when $\abs{x} \rightarrow \infty$ we have that $\rho_1 (\Wbf_x)$ monotonically increases towards $1$, starting from the value $\rho_1 (\Wbf_0) = \frac{1}{2}$.
Noticing that this implies the entropy function and effective rank monotonically decrease towards $0$ and $1$, respectively, completes the effective rank analysis.

Next, we analyze the infimal rank of $\S$ and the distance of $\Wbf_x$ from that infimal rank.
The distance of $\Wbf_x$ from $\M_1$ is $\frodist ( \Wbf_x , \M_1 ) = \sigma_2(\Wbf_x)$.
Since $\lim_{x \rightarrow \infty} \sigma_2(\Wbf_x) = 0$, we have $\frodist (\S , \M_1) = 0$.
Clearly $\frodist (\S, \M_0) > 0$, leading to the conclusion that the infimal rank of $\S$ is $1$.
Finally, the analysis of $\sigma_2(\Wbf_x)$ directly implies that the distance of $\Wbf_x$ from the infimal rank of $\S$ is maximized when $x = 0$, monotonically tending to $0$ as $\abs{x} \rightarrow \infty$.
\qed

\subsection{Proof of Theorem~\ref{mf:thm:norms_up_finite}} \label{mf:app:proofs:norms_up_finite}

In the following, as stated in Appendix~\ref{mf:app:proofs:notation}, for results that hold for all $t \geq 0$ or when clear from the context, we omit the time index $t$.
Furthermore, we denote the entries of the end matrix $\matrixend$ by $\{w_{i , j}\}_{i,j \in [2]}$.

We begin by deriving loss-dependent bounds for $\abs{w_{1,1}}, \sigma_1 (\matrixend)$, and $\sigma_2 (\matrixend)$.
Writing the loss explicitly:
\[
\mfendloss (\matrixend) = \frac{1}{2} \left [ (w_{1,2} - 1)^2 + (w_{2,1} - 1)^2 + w_{2,2}^2 \right ]
\text{\,,}
\]
we can upper bound each of the non-negative terms separately. Multiplying by $2$ and taking the square root of both sides yields:
\be
\abs{w_{2,2}} \leq \sqrt{2 \mfendloss (\matrixend)}  \quad,\quad \abs{w_{1,2} - 1} \leq \sqrt{2 \mfendloss (\matrixend)} \quad,\quad \abs{w_{2,1} - 1} \leq \sqrt{2 \mfendloss (\matrixend)}
\text{\,.}
\label{mf:eqn:entries_bound}
\ee
The following lemma characterizes the relation between $|w_{1,1}|$ and the loss.
\begin{lemma} \label{mf:lem:w_11_lower_bound}
	Suppose $\mfendloss (\matrixend) < \frac{1}{2}$.
	Then:
	\[
	\abs{w_{1,1}} > \frac{(1 - \sqrt{2 \mfendloss (\matrixend)})^2}{\sqrt{2 \mfendloss (\matrixend)}} = \frac{1}{\sqrt{2 \mfendloss (\matrixend)}} - 2 + \sqrt{2 \mfendloss (\matrixend)}
	\text{\,.}
	\]
\end{lemma}

\begin{proof}
	From Lemma~\ref{mf:lem:det_does_not_change_sign}, the determinant of $\matrixend$ does not change signs and remains positive, \ie:
	\be
	\det (\matrixend) = w_{1,1}w_{2,2} - w_{1,2}w_{2,1} > 0
	\text{\,.}
	\label{mf:eqn:pos_det_inequality}
	\ee
	Under the assumption that $\mfendloss (\matrixend) < \frac{1}{2}$, both $w_{1,2}$ and $w_{2,1}$ are positive and lie inside the open interval $(0,2)$.
	Since the determinant is positive, $w_{2,2} \neq 0$ and $w_{1,1}w_{2,2} > 0$ must hold.
	Rearranging Equation~\eqref{mf:eqn:pos_det_inequality}, we may therefore write $\abs{w_{1,1} w_{2,2}} > w_{1,2}w_{2,1}$.
	Dividing both sides by $|w_{2,2}|$ and applying the bounds from Equation~\eqref{mf:eqn:entries_bound} completes the proof:
	\[
	\abs{w_{1,1}} > \frac{(1 - \sqrt{2 \mfendloss (\matrixend)})^2}{\sqrt{2 \mfendloss (\matrixend)}} = \frac{1}{\sqrt{2 \mfendloss (\matrixend)}} - 2 + \sqrt{2 \mfendloss (\matrixend)}
	\text{\,.}
	\]
\end{proof}
An immediate consequence of the lemma above is that decreasing the loss towards zero drives $|w_{1,1}|$ towards infinity.

With this bound in hand, Lemma~\ref{mf:lem:W_singular_values_bound} below establishes bounds on the singular values of $\matrixend$.
In turn, they will allow us to obtain the necessary results for effective rank (Definition~\ref{def:erank}) and distance from infimal rank of $\S$ (Definition~\ref{def:irank_dist}).
\begin{lemma} \label{mf:lem:W_singular_values_bound}
	The singular values of $\matrixend$ fulfill:
	\be
	\sigma_1(\matrixend) \geq \abs{w_{1,1}} - \sqrt{2 \mfendloss (\matrixend)} \quad , \quad \sigma_2(\matrixend) \leq 3\sqrt{2 \mfendloss (\matrixend)} \text{\,.}
	\label{mf:eqn:W_sing_val_bounds_no_assump}
	\ee
	Furthermore, if $\mfendloss (\matrixend) < \frac{1}{2}$, then:
	\be
	\sigma_1 (\matrixend) \geq \frac{1}{\sqrt{2 \mfendloss (\matrixend)}} - 2  \text{\,.}
	\label{mf:eqn:W_sing_val_bounds_assump}
	\ee
\end{lemma}

\begin{proof}
	Define $\Wbf_\S := \begin{pmatrix} w_{1,1} & 1 \\ 1 & 0 \end{pmatrix}$, the orthogonal projection of $\matrixend$ onto the solution set $\S$.
	By Corollary 8.6.2 in~\cite{golub2012matrix} we have that:
	\be
	|\sigma_i (\matrixend) - \sigma_i (\Wbf_\S)| \leq \norm{\matrixend - \Wbf_\S}_{Fro} = \sqrt{2 \mfendloss (\matrixend)} \quad ,~ i = 1,2 \text{\,.}
	\label{mf:eqn:W_singular_value_perturbation_bound}
	\ee
	One can easily verify that $\Wbf_\S$ is a symmetric indefinite matrix with eigenvalues
	\[
	\left \{\lambda_{1}(\Wbf_\S) , \lambda_{2}(\Wbf_\S) \right \} = \left \{\left ( w_{1,1} + \sqrt{w_{1,1}^2 + 4} \right ) / ~ 2 ~ , ~ \left ( w_{1,1} - \sqrt{w_{1,1}^2 + 4}\right ) / ~ 2 \right \}
	\text{\,.}
	\]
	Suppose that $w_{1,1} \geq 0$.
	We thus have:
	\[
	\sigma_1(\Wbf_\S) = \max_{i = 1,2} \abs{\lambda_i (\Wbf_\S)} = \frac{w_{1,1} + \sqrt{w_{1,1}^2 + 4}}{2} \geq |w_{1,1}|
	\text{\,,}
	\]
	and
	\[
	\begin{split}
		\sigma_2(\Wbf_\S) &  = \min_{i = 1,2} \abs{\lambda_i (\Wbf_\S)}  \\
		& = \frac{\sqrt{w_{1,1}^2 + 4} - w_{1,1}}{2} \\[1mm]
		& = \frac{2}{\sqrt{w_{1,1}^2 + 4} + w_{1,1}} \\
		& \leq \frac{2}{2 + w_{1,1}}
		\text{\,,}
	\end{split}
	\]
	where in the third transition we made use of the identity $a - b = \frac{a^2 - b^2}{a + b}$ for $a, b \in \R$ such that $a + b \neq 0$.
	If $\mfendloss (\matrixend) \geq \frac{1}{2}$, it holds that $\sigma_2 (\Wbf_\S) \leq 2 / (2 + w_{1,1}) \leq 1 \leq 2\sqrt{2 \mfendloss (\matrixend)}$.
	Otherwise, we may apply the lower bound on $w_{1,1}$ (Lemma~\ref{mf:lem:w_11_lower_bound}) and conclude that $\sigma_2(\Wbf_\S) \leq 2\sqrt{2 \mfendloss (\matrixend)}$ for any loss value.
	Having established that $\sigma_1(\Wbf_\S) \geq |w_{1,1}|$ and $\sigma_2 (\Wbf_\S) \leq 2\sqrt{2 \mfendloss (\matrixend)}$, Equation~\eqref{mf:eqn:W_singular_value_perturbation_bound} completes the proof of Equation~\eqref{mf:eqn:W_sing_val_bounds_no_assump}.
	It remains to see that if $\mfendloss (\matrixend) < \frac{1}{2}$, from the lower bound on $w_{1,1}$ (Lemma~\ref{mf:lem:w_11_lower_bound}), Equation~\eqref{mf:eqn:W_sing_val_bounds_assump} immediately follows.
	
	By similar arguments, Equations~\eqref{mf:eqn:W_sing_val_bounds_no_assump} and~\eqref{mf:eqn:W_sing_val_bounds_assump} hold for $w_{1,1} < 0$ as well.
\end{proof}

\subsubsection{Proof of Equation~\eqref{mf:eq:norms_lb} (Lower Bound for Quasi-Norm)} \label{mf:sec:norm_bound}

We turn to lower bound the quasi-norm of the end matrix.
It holds that:
\be
\begin{split}
	\norm{\matrixend} &\geq \frac{1}{c_{\norm{\cdot}}} \norm{ w_{1,1} \ebf_1 \ebf_1^\top } - \norm{ \matrixend - w_{1,1} \ebf_1 \ebf_1^\top }
	\text{\,,}
	\label{mf:eqn:triangle_inequality_norm_lower_bound}
\end{split}
\ee
where $c_{\norm{\cdot}} \geq 1$ is a constant for which $\norm{\cdot}$ satisfies the weakened triangle inequality (see Footnote~\ref{note:qnorm}). 
We now assume that $\mfendloss (\matrixend) < \frac{1}{2}$. 
Later this assumption will be lifted, providing a bound that holds for all loss values.
Subsequent applications of the weakened triangle inequality, together with homogeneity of $\norm{\cdot}$ and the bounds on the entries of $\matrixend$ (Equation~\eqref{mf:eqn:entries_bound}), give:
\[
\begin{split}
	\norm{\matrixend - w_{1,1} \ebf_1 \ebf_1^\top} & \leq c_{\norm{\cdot}} |w_{2,2}| \norm{\ebf_2 \ebf_2^\top} + c_{\norm{\cdot}}^2 \left ( |w_{2,1}| \norm{\ebf_2 \ebf_1^\top} + |w_{1,2}| \norm{\ebf_1 \ebf_2^\top} \right ) \\
	& \leq c_{\norm{\cdot}} \sqrt{2 \mfendloss( \matrixend)} \norm{\ebf_2 \ebf_2^\top} \\
	& \hspace{5mm} + c_{\norm{\cdot}}^2 \Big ( 1 + \sqrt{2 \mfendloss( \matrixend)} \Big ) \left (\norm{\ebf_2 \ebf_1^\top} + \norm{\ebf_1 \ebf_2^\top} \right ) \\
	& \leq c_{\norm{\cdot}} \norm{\ebf_2 \ebf_2^\top} + 2c_{\norm{\cdot}}^2 \left (\norm{\ebf_2 \ebf_1^\top} + \norm{\ebf_1 \ebf_2^\top} \right ) \\
	& \leq 2c_{\norm{\cdot}}^2 \left (\norm{\ebf_2 \ebf_2^\top} + \norm{\ebf_2 \ebf_1^\top} + \norm{\ebf_1 \ebf_2^\top} \right )
	\text{\,.}
\end{split}
\]
Plugging the inequality above and the lower bound on $|w_{1,1}|$ (Lemma~\ref{mf:lem:w_11_lower_bound}) into Equation~\eqref{mf:eqn:triangle_inequality_norm_lower_bound}, we have:
\[
\begin{split}
	\norm{\matrixend} & \geq \frac{\norm{ \ebf_1 \ebf_1^\top }}{c_{\norm{\cdot}}} \frac{1}{\sqrt{2 \mfendloss (\matrixend)}} - 2\frac{\norm{ \ebf_1 \ebf_1^\top }}{c_{\norm{\cdot}}} - 2c_{\norm{\cdot}}^2 \left (\norm{\ebf_2 \ebf_2^\top} + \norm{\ebf_2 \ebf_1^\top} + \norm{\ebf_1 \ebf_2^\top} \right ) \\
	& \geq \frac{\norm{ \ebf_1 \ebf_1^\top }}{c_{\norm{\cdot}}} \frac{1}{\sqrt{2 \mfendloss (\matrixend)}} - 2c_{\norm{\cdot}}^2 \left (\norm{\ebf_1 \ebf_1^\top} + \norm{\ebf_2 \ebf_2^\top} + \norm{\ebf_2 \ebf_1^\top} + \norm{\ebf_1 \ebf_2^\top} \right )
	\text{\,.}
\end{split}
\]
Since $\norm{\matrixend}$ is trivially lower bounded by zero, defining the constants
\[
a_{\norm{\cdot}} := \frac{ \norm{\ebf_1 \ebf_1^\top} }{ \sqrt{2} c_{\norm{\cdot}} }
~ , ~
b_{\norm{\cdot}} := \max \left \{ \sqrt{2} a_{\norm{\cdot}} , 8c_{\norm{\cdot}}^2 \max_{i,j \in \{1,2\}} \norm{\ebf_i \ebf_j^\top} \right \}
,
\]
allows us, on the one hand, to arrive at a bound of the form:
\[
\norm{\matrixend} \geq a_{\norm{\cdot}} \cdot \frac{1}{\sqrt{\mfendloss (\matrixend)}}  -  b_{\norm{ \cdot }}
\text{\,,}
\]
and on the other hand, to lift our previous assumption on the loss: when $\mfendloss (\matrixend) \geq \frac{1}{2}$ the bound is vacuous, \ie~non-positive and trivially holds.
Noticing this is exactly Equation~\eqref{mf:eq:norms_lb} (recall we omitted the time index $t$), concludes the first part of the proof.

\subsubsection{Proof of Equation~\eqref{mf:eq:erank_ub} (Upper Bound for Effective Rank)} \label{mf:sec:effective_rank_bound}

During the following effective rank (Definition~\ref{def:erank}) analysis we operate under the assumption of $\mfendloss (\matrixend) < \frac{1}{32}$.
We later remove this assumption, delivering a bound that holds for all loss values.
Making use of the obtained bounds on $\sigma_1(\matrixend)$ and $\sigma_2(\matrixend)$ (Lemma~\ref{mf:lem:W_singular_values_bound}) we arrive at:
\[
\begin{split}
	\rho_1 (\matrixend) & = \frac{\sigma_1 (\matrixend)}{\sigma_1 (\matrixend) + \sigma_2 (\matrixend)} \\
	& \geq \frac{\sigma_1 (\matrixend)}{\sigma_1 (\matrixend) + 3\sqrt{2 \mfendloss (\matrixend)}} \\
	& = 1 - \frac{3\sqrt{2 \mfendloss (\matrixend)}}{\sigma_1 (\matrixend) + 3\sqrt{2 \mfendloss (\matrixend)}} \\
	& \geq 1 - \frac{3\sqrt{2 \mfendloss (\matrixend)}}{ \frac{1}{ \sqrt{2 \mfendloss (\matrixend)}} - 2 + 3\sqrt{2 \mfendloss (\matrixend)}} \\
	& = 1 - \frac{6 \mfendloss (\matrixend)}{ 6 \mfendloss (\matrixend) - 2\sqrt{2 \mfendloss (\matrixend)} + 1 }
	\text{\,.}
\end{split}
\]
Given our assumption on the loss, we have $1 - 2\sqrt{2 \mfendloss ((\matrixend))} \geq \frac{1}{2}$ and thus
\be
\rho_2 (\matrixend) = 1 - \rho_1 (\matrixend) \leq \frac{6 \mfendloss (\matrixend)}{6 \mfendloss (\matrixend) + \frac{1}{2}} \leq 12 \mfendloss (\matrixend)
\text{\,.}
\label{mf:eqn:rho_2_upper_bound}
\ee
Let $h \left (\rho_2 (\matrixend) \right ) := - \rho_2 (\matrixend) \cdot \ln \left (\rho_2 (\matrixend) \right ) - (1 - \rho_2 (\matrixend)) \cdot \ln \left (1 - \rho_2 (\matrixend) \right )$ denote the binary entropy function, and recall that the effective rank of $\matrixend$ is defined to be $\erank (\matrixend) := \exp \{ h \left (\rho_2 (\matrixend) \right ) \}$.
The exponent function is convex and therefore upper bounded on the interval $[0, \ln (2)]$ by the linear function that intersects it at these points.
Formally, for $x \in [0, \ln (2)]$ it holds that $\exp (x) \leq 1 + \frac{1}{\ln (2)} x$, yielding the following bound:
\[
\erank (\matrixend) \leq 1 + \frac{1}{\ln (2)} \cdot  h \left (\rho_2 (\matrixend) \right )
\text{\,.}
\]
By Lemma~\ref{mf:lem:entropy_sqrt_bound} we have that $h \left (\rho_2 (\matrixend) \right ) \leq 2 \sqrt{\rho_2 (\matrixend)}$.
Combined with Equation~\eqref{mf:eqn:rho_2_upper_bound}, since $\inf_{\Wbf' \in \S} \erank (\Wbf') = 1$ (Proposition~\ref{mf:prop:sol_set_rank}), this leads to:
\[
\erank (\matrixend) \leq \inf_{\Wbf' \in \S} \erank (\Wbf') + \frac{2\sqrt{12}}{\ln (2)} \cdot \sqrt{\mfendloss (\matrixend)}
\text{\,.}
\]
Recall that the time index $t$ is omitted, and the result holds for all $t \geq 0$, \ie~this is exactly Equation~\eqref{mf:eq:erank_ub}.
To remove our assumption on the loss, notice that when $\mfendloss (\matrixend) \geq \frac{1}{32}$ the bound is trivial, as the right-hand side is greater than $2$, which is the maximal effective rank (for a $2$-by-$2$ matrix).

\subsubsection{Proof of Equation~\eqref{mf:eq:irank_dist_ub} (Upper Bound for Distance from Infimal Rank)} \label{mf:sec:distance_from_infimal_rank_upper_bound}

According to Proposition~\ref{mf:prop:sol_set_rank}, the infimal rank of $\S$ is $1$.
The quantity we seek to upper bound is therefore $\frodist ( \matrixend ( t ) , \M_1 ) = \sigma_2 (\matrixend (t) )$. By Equation~\eqref{mf:eqn:W_sing_val_bounds_no_assump} in Lemma~\ref{mf:lem:W_singular_values_bound}, for all $t \geq 0$ we have
\[
\frodist (\matrixend (t), \M_{1}) \leq 3\sqrt{2} \cdot \sqrt{\mfendloss (t)}
\text{\,,}
\]
completing the proof.
\qed

\subsection{Proof of Proposition~\ref{mf:prop:det_pos}}
\label{mf:app:proofs:det_pos}

Define $\Wbf_{-1}$ to be the matrix obtained from $W$ by multiplying its first row by $-1$.
On the one hand, symmetry around the origin implies that $\Wbf_{-1}$ and $\Wbf$ follow the same distribution.
On the other hand, $\det ( \Wbf_{-1} ) = - \det ( \Wbf ) $.
Due to the fact that the set of matrices with zero determinant has probability $0$ under continuous distributions (see, \eg,~Remark 2.5 in~\cite{hackbusch2012tensor}), we may conclude $\Pr ( \det (\Wbf) > 0 ) = \Pr ( \det (\Wbf) < 0 ) = 0.5$.

For $\Wbf_1, \ldots , \Wbf_L$ random matrices drawn independently, let $l \in [L]$ be the index such that $\Pr ( \det (\Wbf_l) > 0) = 0.5$.
Since $Pr ( \det (\Wbf_{l'}) = 0) = 0$ for any $l' \in [L]$, the proof readily follows from determinant multiplicativity and the law of total probability:
\[
\begin{split}
	\Pr \left ( \det ( \Wbf_L \cdots \Wbf_1 ) > 0 \right ) & = \Pr ( \det (\Wbf_l) > 0) \cdot \Pr \left ( \Pi_{i \neq l} \det (\Wbf_i) > 0 \right ) \\
	& + \Pr ( \det (\Wbf_l) < 0) \cdot \Pr \left ( \Pi_{i \neq l} \det (\Wbf_i) < 0 \right ) \\
	& = \frac{1}{2} \left [ \Pr \left ( \Pi_{i \neq l} \det (\Wbf_i) > 0 \right ) + \Pr \left ( \Pi_{i \neq l} \det (\Wbf_i) < 0 \right ) \right ] \\
	& = 0.5
	\text{\,.}
\end{split}
\]
An identical computation yields $\Pr \left ( \det ( \Wbf_L \cdots \Wbf_1 ) < 0 \right ) = 0.5$.
\qed

\subsection{Proof of Proposition~\ref{mf:prop:loss_down}} \label{mf:app:proofs:loss_down}

The proof makes use of the following lemma, proven in Appendix~\ref{mf:app:proofs:prod_mat_remains_pos_def}.

\begin{lemma} \label{mf:lem:balanced_scaled_id_remains_pos_def}
	Consider the setting of Theorem~\ref{mf:thm:norms_up_finite} (arbitrary depth $L \in \N$) in the special case of an initial end matrix $\matrixend (0) = \alpha \cdot \Ibf$, where $\Ibf$ stands for identity matrix and $\alpha \in (0, 1]$. Then, $\matrixend (t)$ is positive definite for all $t \geq 0$ .
\end{lemma}

With Lemma~\ref{mf:lem:balanced_scaled_id_remains_pos_def} in place, we may derive the exact differential equations governing the end matrix in our setting of depth $L = 2$.
Then, a detailed analysis of the dynamics will yield convergence of the loss to global minimum, \ie~ $\lim_{t \to \infty} \mfendloss ( t ) = 0$.
As usual, we omit the time index $t$ when stating results for all $t$ or when clear from the context.

According to Lemma~\ref{mf:lem:balanced_scaled_id_remains_pos_def}, the end matrix $\matrixend$ is symmetric and positive definite. 
Thus, we may write the loss and its gradient with respect to $\matrixend$ as:
\be
\mfendloss (\matrixend) = \frac{1}{2} \left [ w_{2,2}^2 + 2(w_{1,2} - 1)^2 \right ] \quad , \quad \nabla \mfendloss (\matrixend) = \begin{pmatrix}
	0 & w_{1,2} - 1 \\
	w_{1,2} - 1 & w_{2,2}
\end{pmatrix} \text{\,,}
\label{mf:eqn:loss_and_grad_W}
\ee
where $\{w_{i,j}\}_{i , j \in [2]}$ are the entries of $\matrixend$.
Since the factors $\Wbf_1$ and $\Wbf_2$ are balanced at initialization (Equation~\eqref{mf:eq:balance}), the differential equation governing the end matrix (Lemma~\ref{mf:lem:prod_mat_dyn}) for depth $L = 2$ gives:
\be
\begin{split}
	\matrixenddot & = - \left[ \matrixend \matrixend^\top \right]^\frac{1}{2} \cdot \nabla \mfendloss (\matrixend) -\nabla \mfendloss (\matrixend) \cdot \left[ \matrixend^\top \matrixend \right]^\frac{1}{2} \\
	& = -\matrixend \nabla \mfendloss (\matrixend) -\nabla \mfendloss (\matrixend) \matrixend
	\text{\,,}
\end{split}
\label{mf:eqn:W_derivative_depth_2}
\ee
where the transition is by positive definiteness of $\matrixend$.
Writing the differential equation of each entry separately, we have:
\be
\begin{split}
	\dot{w}_{1,1} & = 2w_{1,2} (1 - w_{1,2}) \text{\,,}\\
	\dot{w}_{2,2} & = 2w_{1,2} (1 - w_{1,2}) -2w_{2,2}^2 \text{\,,}\\
	\dot{w}_{1,2} & = w_{2,2} (1 - 2w_{1,2}) + w_{1,1} (1 - w_{1,2})
	\text{\,.}
	\label{mf:eqn:W_entries_diff_eqns}
\end{split}
\ee
Let us characterize the behavior of these entries throughout time.
\begin{lemma} \label{mf:lem:W_entries_opt_bounds}
	The following holds for all $t \geq 0$:
	\begin{enumerate} 
		\item $w_{1,1} > 0$ and is monotonically non-decreasing.
		\item $0 \leq w_{1,2} \leq 1$.
		\item $0 < w_{2,2} \leq 1$.
	\end{enumerate}
\end{lemma}

\begin{proof}
	Since $\matrixend$ is positive definite, it follows that $w_{1,1}$ and $w_{2,2}$ are positive.
	Examining the behavior of $w_{1,2}$ (Equation~\eqref{mf:eqn:W_entries_diff_eqns}): on the one hand, when $w_{1,2} = 0$ then $\dot{w}_{1,2} = w_{2,2} + w_{1,1} > 0$, and on the other hand, when $w_{1,2} = 1$ then $\dot{w}_{1,2} = -w_{2,2} < 0$.
	Because $w_{1,2}$ is initialized at $0$, it stays in the interval $[0, 1]$.
	Otherwise, by Lemma~\ref{mf:lem:ode_bounded_in_interval}, we have a contradiction to the positivity of $\dot{w}_{1,2}$ when $w_{1,2} = 0$ or its negativity when $w_{1,2} = 1$.
	Similarly, if $w_{2,2} > \frac{1}{2}$ we have $\dot{w}_{2,2} < 2w_{1,2} (1 - w_{1,2}) -\frac{1}{2} \leq 0$.
	Since at initialization $w_{2,2} (0) = \alpha \leq 1$, by Lemma~\ref{mf:lem:ode_bounded_in_interval}, it will not go above $1$.
	Lastly, since $w_{1,2}$ is in the interval $[0, 1]$, it holds that $\dot{w}_{1,1} \geq 0$, \ie~$w_{1,1}$ is monotonically non-decreasing.
\end{proof}

We turn our focus to the derivative of the loss with respect to $t$:
\[
\frac{d}{dt} \mfendloss (\matrixend) = \langle \nabla \mfendloss (\matrixend) , \matrixenddot \rangle
\text{\,.}
\]
Plugging in Equation~\eqref{mf:eqn:W_derivative_depth_2} and recalling the fact that $\langle \Abf, \Bbf \rangle = \Tr (\Abf^\top \Bbf)$ for matrices $\Abf, \Bbf$ of the same size:
\[
\frac{d}{dt} \mfendloss (\matrixend) = - \Tr (\nabla \mfendloss (\matrixend)^\top  \matrixend \nabla \mfendloss (\matrixend)) - \Tr (\nabla \mfendloss(\matrixend)^\top  \nabla \mfendloss(\matrixend) \matrixend)
\text{\,.}
\]
From the cyclic property of the trace operator and symmetry of $\nabla \mfendloss (\matrixend)$ (Equation~\eqref{mf:eqn:loss_and_grad_W}), we arrive at the following expression:
\[
\frac{d}{dt} \mfendloss (\matrixend) = -2\Tr (\nabla \mfendloss (\matrixend) \matrixend \nabla \mfendloss(\matrixend))
\text{\,.}
\]
Notice that since $\nabla\mfendloss (\matrixend) \matrixend \nabla \mfendloss (\matrixend)$ is positive semidefinite the trace is non-negative and $\frac{d}{dt} \mfendloss (\matrixend) \leq 0$.
That is, the loss is monotonically non-increasing throughout time.
Invoking Lemma~\ref{mf:lem:A_TBA_trace_bound}, we can upper bound the derivative by:
\be
\frac{d}{dt} \mfendloss (\matrixend) \leq -2 \lambda_{1}(\matrixend) \cdot \sigma_{2} (\nabla \mfendloss (\matrixend))^2 \text{\,,}
\label{mf:eqn:L_deriv_spectral_bound}
\ee
where $\lambda_{1} (\matrixend)$ is the maximal eigenvalue of $\matrixend$ and $\sigma_{2} (\nabla \mfendloss (\matrixend))$ is the minimal singular value of $\nabla \mfendloss (\matrixend)$.
The maximal eigenvalue of a symmetric matrix is greater than its diagonal entries. 
Therefore, $\lambda_{1}(\matrixend) \geq w_{1,1}$. 
Since $w_{1,1}$ is initialized at $\alpha > 0$, and by Lemma~\ref{mf:lem:W_entries_opt_bounds} is monotonically non-decreasing, we have $\lambda_{1}(\matrixend) \geq \alpha$.
Writing the eigenvalues of $\nabla \mfendloss(\matrixend)$ explicitly:
\[
\begin{split}
	\lambda_{1} (\nabla \mfendloss(\matrixend)) & = \frac{ w_{2,2} + \sqrt{w_{2,2}^2 + 4(1 - w_{1,2})^2} }{2} \text{\,,} \\
	\lambda_{2} (\nabla \mfendloss(\matrixend)) & = \frac{ w_{2,2} - \sqrt{w_{2,2}^2 + 4(1 - w_{1,2})^2} }{2}
	\text{\,,}
\end{split}
\]
we can see that, since $w_{2,2}$ is positive (Lemma~\ref{mf:lem:W_entries_opt_bounds}):
\[
\sigma_{2} (\nabla \mfendloss (\matrixend)) = \min\nolimits_{i=1,2} \abs{\lambda_i (\nabla \mfendloss(\matrixend))} = \frac{ \sqrt{w_{2,2}^2 + 4(1 - w_{1,2})^2} - w_{2,2} }{ 2 }
\text{\,.}
\]
Applying the identity $a - b = \frac{a^2 - b^2}{a + b}$ for $a, b \in \R$ such that $a + b \neq 0$, and the bounds on $w_{2,2}$ and $w_{1,2}$ (Lemma~\ref{mf:lem:W_entries_opt_bounds}):
\[
\begin{split}
	\sigma_2 (\nabla \mfendloss (\matrixend)) & = \frac{ 2(1 - w_{1,2})^2 }{  \sqrt{w_{2,2}^2 + 4(1 - w_{1,2})^2} + w_{2,2} } \\
	& \geq \frac{ 2(1 - w_{1,2})^2 }{ \sqrt{1 + 4(1 - w_{1,2})^2} + 1 } \\
	& \geq \frac{ 2(1 - w_{1,2})^2 }{ 2 \abs{(1 - w_{1,2})} + 2 } \\
	& \geq \frac{1}{2} (1 - w_{1,2})^2 \text{\,,}
\end{split}
\]
where in the penultimate transition we bounded the square root of a sum by the sum of square roots.
Returning to Equation~\eqref{mf:eqn:L_deriv_spectral_bound} we have:
\[
\frac{d}{dt} \mfendloss (\matrixend) \leq - b (1 - w_{1,2})^4 \text{\,,}
\]
for $b = \frac{1}{2} \alpha$. 
We are now in a position to prove that $w_{1,2} \rightarrow 1$ as $t$ tends to infinity.
Integrating both sides with respect to time:
\[
\mfendloss (\matrixend(t)) - \mfendloss (\matrixend(0)) \leq -b \int\nolimits_{t'=0}^{t} (1 - w_{1,2}(t'))^4 dt' \text{\,.}
\]
Since $\mfendloss (\matrixend(t)) \geq 0$, by rearranging the inequality we may write:
\[
\int\nolimits_{t'=0}^{t} (1 - w_{1,2}(t'))^4 dt' \leq \frac{\mfendloss (\matrixend(0))}{b} \text{\,.}
\]
Going back to the differential equation of $\dot{w}_{1,2}$ (Equation~\eqref{mf:eqn:W_entries_diff_eqns}), by applying the bounds on $w_{1,2}$ and $w_{2,2}$ (Lemma~\ref{mf:lem:W_entries_opt_bounds}) we have that $\dot{w}_{1,2} \geq -1$.
Defining $g(t) := (1 - w_{1,2}(t))^4$, it then holds that $\dot{g}(t) \leq 4$. Since $g(\cdot)$ is non-negative and has an upper bounded integral and derivative, from Lemma~\ref{mf:lem:bounded_integral_and_der_convergence}, we can conclude that $\lim_{t \rightarrow \infty} g(t) = 0$ and $\lim_{t \rightarrow \infty} w_{1,2} (t) = 1$.

Because $\mfendloss (\matrixend(t))$ is monotonically non-increasing, we need only show that for each $\epsilon > 0$ there exists a $t_{\epsilon} > 0$ such that $\mfendloss (\matrixend (t_{\epsilon})) < \epsilon$. 
Having already established that $w_{1,2} (t)$ converges to $1$, this amounts to finding a large enough $t_{\epsilon}$ for which $w_{2,2} (t_{\epsilon})$ is sufficiently close to $0$.
Fix some $\epsilon > 0$ and let $\hat{t} > 0$ be such that for all $t \geq \hat{t}$ the following holds:
\be
2(1 - w_{1,2}(t))^2 < \epsilon \quad , \quad 2w_{1,2}(t)(1 - w_{1,2}(t)) < \epsilon
\text{\,.}
\label{mf:eqn:w_12_conv_eps_bound}
\ee
Such $\hat{t}$ exists since all terms above converge to $0$.
Returning to the differential equation of $\dot{w}_{2,2}$ (Equation~\eqref{mf:eqn:W_entries_diff_eqns}):
\be
\dot{w}_{2,2}(t) < \epsilon -2w_{2,2}(t)^2
\text{\,.}
\label{mf:eqn:w_22_dot_upper_bound}
\ee
Recalling that $w_{2,2}(t) > 0$ (Lemma~\ref{mf:lem:W_entries_opt_bounds}), it follows that there exists $t_{\epsilon} \geq \hat{t}$ with $\dot{w}_{2,2}(t_{\epsilon}) > -\epsilon$ (otherwise $w_{2,2} (t)$ goes to $-\infty$ as $t \rightarrow \infty$, in contradiction to the positivity of $w_{2,2} (t)$).
For the above $t_{\epsilon}$, by rearranging the terms in Equation~\eqref{mf:eqn:w_22_dot_upper_bound} we achieve $w_{2,2}(t_{\epsilon}) < \sqrt{\epsilon}$.
Finally, combined with Equation~\eqref{mf:eqn:w_12_conv_eps_bound}, the result readily follows:
\[
\mfendloss (\matrixend(t_{\epsilon})) = \frac{1}{2} \left [ w_{2,2}(t_{\epsilon})^2 + 2(w_{1,2}(t_{\epsilon}) - 1)^2 \right ] < \epsilon
\text{\,,}
\]
concluding the proof.
\qed

\subsubsection{Proof of Lemma~\ref{mf:lem:balanced_scaled_id_remains_pos_def}} \label{mf:app:proofs:prod_mat_remains_pos_def}

The proof proceeds as follows. 
We initially consider initializations where the matrices $\Wbf_1 (0), \ldots , \Wbf_L (0)$ form a \emph{symmetric factorization} of $\matrixend (0)$ (Definition~\ref{def:sym_factorization}), and show that this ensures the end matrix stays symmetric.
Then, we establish that for every balanced initial factors (Equation~\eqref{mf:eq:balance}) with a positive definite end matrix there exist alternative balanced factors such that: \emph{(i)}~the initial end matrix is the same; and \emph{(ii)}~the factors form a symmetric factorization of the end matrix. 
Since the end matrices for the original and the constructed initializations obey the exact same dynamics (Lemma~\ref{mf:lem:prod_mat_dyn}), the proof concludes.

\begin{definition} \label{def:sym_factorization}
	We say that the matrices $\Wbf_1, \ldots ,\Wbf_L \in \R^{D \times D}$ form a \textit{symmetric factorization} of $\Wbf \in \R^{D \times D}$ if $\Wbf = \Wbf_L \cdots \Wbf_1$ and
	\[
	\Wbf_l = \Wbf_{L-l+1}^\top  \quad , l\in \{1, \ldots ,\lfloor L/2 \rfloor + 1\}
	\text{\,.}
	\]
\end{definition}
A straightforward result is that matrices with a symmetric factorization are symmetric themselves.
\begin{lemma} \label{mf:lem:symm_factorization_is_sym}
	If a matrix $\Wbf \in \R^{D \times D}$ has a symmetric factorization, then it is symmetric.
\end{lemma}

\begin{proof}
	Let $\Wbf_1, \ldots ,\Wbf_L \in \R^{D \times D}$ form a symmetric factorization of $\Wbf$. It directly follows that
	\[
	\Wbf = \Wbf_L \cdots \Wbf_1 = \Wbf_1^\top  \cdots \Wbf_L^\top  = \Wbf^\top \text{\,.}
	\]
\end{proof}

By Lemma~\ref{mf:lem:analytic_factors_and_prod_mat}, $\Wbf_1(t), \ldots , \Wbf_L(t), \matrixend(t)$ and $\nabla \mfendloss ( \matrixend (t) )$ are analytic, and hence infinitely differentiable, with respect to $t$.
Lemmas~\ref{mf:lem:induction_on_order_k} and~\ref{mf:lem:sym_factorization_stays_sym} below thus establish that if $\Wbf_1 (0), \ldots , \Wbf_L (0)$ form a symmetric factorization of $\matrixend (0)$, then the end matrix stays symmetric for all $t$.

\begin{lemma} \label{mf:lem:induction_on_order_k}
	Under the setting of Lemma~\ref{mf:lem:balanced_scaled_id_remains_pos_def}, assume that the matrices $\Wbf_1(0), \ldots ,\Wbf_L(0)$ form a symmetric factorization of $\matrixend (0)$ (Definition~\ref{def:sym_factorization}).
	Then, for all $k\in \N\cup \{0\}$:
	\be
	\matrixend^{(k)}(0)=\matrixend^{(k)}(0)^\top
	\text{\,,}
	\label{mf:eqn:induction_W_sym_k_deriv}
	\ee
	and
	\be 
	\Wbf_{l}^{(k)}(0) = \Wbf_{L-l+1}^{(k)}(0)^\top  \quad , l \in \{1, \ldots,\lfloor L/2 \rfloor + 1\}
	\text{\,.}
	\label{mf:eqn:induction_sym_fac_k_deriv}
	\ee
\end{lemma}

\begin{proof}
	The proof is by induction over $k$.
	For $k=0$, the claim holds directly from the initialization assumption and Lemma~\ref{mf:lem:symm_factorization_is_sym}.	
	For $k\in \N$, suppose the claim is true for all $m\in \N \cup \{0\}$ with $m < k$. 
	We begin by showing Equation~\eqref{mf:eqn:induction_sym_fac_k_deriv} holds for $k$. 
	In turn, this will lead to Equation~\eqref{mf:eqn:induction_W_sym_k_deriv} holding as well. 
	For $l \in [L]$, the dynamics of $\Wbf_l(t)$ under gradient flow are
	\[
	\Wbf_{l}^{(1)}(t) = - \tfrac{\partial}{\partial \Wbf_l} \mfobj ( \Wbf_1 ( t ) , \ldots , \Wbf_L ( t ) ) = -\prod\nolimits_{r=l+1}^{L} \Wbf_{r}(t)^\top  \cdot \Gbf (t) \cdot \prod\nolimits_{r=1}^{l-1} \Wbf_{r}(t)^\top \text{\,,}
	\]
	where $\Gbf (t) := \nabla \mfendloss (\matrixend (t))$ denotes the loss gradient with respect to $\matrixend$ at time $t$.
	We can explicitly write the $k$'th ($k \geq 1$) derivative with respect to $t$ of each $\Wbf_l(t)$ using the product rule for higher order derivatives:
	\[
	\Wbf_{l}^{(k)}(t) = -\sum_{i_1, \ldots ,i_L} \binom{k-1}{i_1, \ldots ,i_L} \prod\nolimits_{r=l+1}^{L} \Wbf_{r}^{(i_r)}(t)^\top  \cdot \Gbf^{(i_l)}(t) \cdot \prod\nolimits_{r=1}^{l-1} \Wbf_{r}^{(i_r)}(t)^\top \text{\,,}
	\]
	where $\sum_{l=1}^L i_l = k-1$ and $\binom{k-1}{i_1, \ldots ,i_L} = (k-1)! / \left ( i_1! \cdots i_L! \right )$ for $i_1, \ldots, i_L \in \{0, \ldots, k -1 \}$. 
	Taking the transpose of both sides we have:
	\be
	\Wbf_{l}^{(k)}(t)^\top  = -\sum_{i_1, \ldots ,i_L} \binom{k-1}{i_1, \ldots ,i_L} \prod\nolimits_{1}^{r=l-1} \Wbf_{r}^{(i_r)}(t) \cdot \Gbf^{(i_l)}(t)^\top  \cdot \prod\nolimits_{l+1}^{r=L} \Wbf_{r}^{(i_r)}(t) \text{\,.}
	\label{mf:eqn:W_j_transpose_k_deriv}
	\ee
	Turning our attention to $\Gbf (t)$, we may write it explicitly as:
	\[
	\Gbf (t) = \nabla \mfendloss (\matrixend(t)) = \begin{pmatrix}
		0 & w_{1,2} (t) - 1\\
		w_{2, 1} (t) - 1 & w_{2,2} (t)
	\end{pmatrix} \text{\,,}
	\]
	where $\{ w_{i , j} (t) \}_{i,j \in [2]}$ are the entries of $\matrixend (t)$.
	For $m < k$, note that when $\matrixend^{(m)} (t)$ is symmetric so is $\Gbf^{(m)} (t)$.
	With this in hand, the inductive assumption (Equation~\eqref{mf:eqn:induction_W_sym_k_deriv}) implies that $\Gbf^{(m)} (0)$ is symmetric (for all $m < k$).
	Combined with Equation~\eqref{mf:eqn:induction_sym_fac_k_deriv} (for $m < k$, from the inductive assumption), we may write Equation~\eqref{mf:eqn:W_j_transpose_k_deriv} for $t=0$ as:
	\[
	\Wbf_{l}^{(k)}(0)^\top  = -\sum_{i_1, \ldots ,i_L} \binom{k-1}{i_1, \ldots ,i_L} \prod_{r=L-l+2}^{L} \Wbf_{r}^{(i_{L-r+1})}(0)^\top  \cdot \Gbf^{(i_l)}(0) \cdot \prod_{r=1}^{L - l} \Wbf_{r}^{(i_{L-r+1})}(0)^\top \text{\,.}
	\]
	Reordering the sum according to $h_r:=i_{L-r+1}$ and noticing that $\binom{k-1}{h_1, \ldots ,h_L} = \binom{k-1}{i_1, \ldots ,i_L}$, we conclude:
	\[
	\Wbf_{l}^{(k)}(0)^\top  = -\sum_{h_1, \ldots ,h_L} \binom{k-1}{h_1, \ldots ,h_L} \prod_{r=L-l+2}^{L} \Wbf_{r}^{(h_r)}(0)^\top  \cdot \Gbf^{(h_{L-l+1})}(0) \cdot \prod_{r=1}^{L-l} \Wbf_{r}^{(h_r)}(0)^\top \text{\,.}
	\]
	That is,
	\[
	\Wbf_{l}^{(k)}(0)^\top  =  \Wbf_{L-l+1}^{(k)}(0) \text{\,,}
	\] 
	proving Equation~\eqref{mf:eqn:induction_sym_fac_k_deriv}.
	
	It remains to show that $\matrixend^{(k)}(0)$ is symmetric. Similarly to before, we take the $k$'th derivative of $\matrixend(t) := \Wbf_L(t) \cdots \Wbf_1(t)$ using the product rule:
	\[
	\matrixend^{(k)}(t) = \sum_{i_1, \ldots ,i_L} \binom{k}{i_1, \ldots ,i_L} \prod\nolimits_{1}^{l=L} \Wbf_{l}^{(i_l)}(t) \text{\,,}
	\]
	where $\sum_{l=1}^L i_l = k$ and $\binom{k}{i_1, \ldots ,i_L} = k! / (i_1! \cdots i_L!)$ for $i_1, \ldots, i_L \in \{0, \ldots, k \}$.
	For convenience, we denote $\Bbf_{i_1, \ldots ,i_L}(t) := \binom{k}{i_1, \ldots ,i_L} \prod\nolimits_{1}^{l = L} \Wbf_{l}^{(i_l)}(t)$.
	Pairing up elements in the sum with indices $(i_1, \ldots ,i_L)$ that are a reverse order of each other, \ie~$(i_1, \ldots ,i_L)$ is paired with $(i_L, \ldots ,i_1)$:
	\be
	\matrixend^{(k)}(t) = \sum_{i_1, \ldots ,i_L}  \frac{1}{2} \left [\Bbf_{i_1, \ldots ,i_L}(t)) + \Bbf_{i_L, \ldots ,i_1}(t) \right ]
	\text{\,.}
	\label{mf:eqn:W_k_deriv_paired}
	\ee
	With Equation~\eqref{mf:eqn:W_k_deriv_paired} in place, we can conclude the proof by showing $\matrixend^{(k)}(0)$ is a sum of symmetric matrices. 
	By the inductive assumption for Equation~\eqref{mf:eqn:induction_sym_fac_k_deriv}, which was established in the first part of the proof for $k$ as well, we have:
	\be
	\Bbf_{i_1, \ldots ,i_L}(0) =  \Bbf_{i_L, \ldots ,i_1}(0)^\top \text{\,,}
	\label{mf:eqn:B_0_eq_transpose}
	\ee
	for each $(i_1, \ldots ,i_L)$. 
	Therefore, the matrix $\Bbf_{i_1, \ldots ,i_L}(0) +  \Bbf_{i_L, \ldots ,i_1}(0)$ is symmetric.
	Plugging Equation~\eqref{mf:eqn:B_0_eq_transpose} into Equation~\eqref{mf:eqn:W_k_deriv_paired} with $t=0$, we arrive at a representation of $\matrixend^{(k)}(0)$ as a sum of symmetric matrices.
	Thus, $\matrixend^{(k)}(0)$ is symmetric, completing the proof.
\end{proof}	

\begin{lemma} \label{mf:lem:sym_factorization_stays_sym}
	Under the setting of Lemma~\ref{mf:lem:balanced_scaled_id_remains_pos_def}, assume that the matrices $\Wbf_1(0), \ldots ,\Wbf_L(0)$ form a symmetric factorization of $\matrixend (0)$ (Definition~\ref{def:sym_factorization}). Then, $\matrixend (t)$ is symmetric for all $t \geq 0$ .
\end{lemma}

\begin{proof}
	By Lemmas~\ref{mf:lem:induction_on_order_k} and~\ref{mf:lem:analytic_func_all_deriv_eq}, we may conclude that for all $t \geq 0$:
	\[
	\Wbf_{l}(t) = \Wbf_{L-l+1}(t)^\top  \quad , l \in \{1, \ldots,\lfloor L/2 \rfloor + 1\}
	\text{\,.}
	\]
	In words, $\Wbf_1(t), \ldots ,\Wbf_L(t)$ form a symmetric factorization of $\matrixend (t)$, and therefore $\matrixend (t)$ is symmetric (Lemma~\ref{mf:lem:symm_factorization_is_sym}).
\end{proof}

Going back to the setting of Lemma~\ref{mf:lem:balanced_scaled_id_remains_pos_def} --- initialization is balanced (Equation~\eqref{mf:eq:balance}), but does not necessarily comprise a symmetric factorization --- we show that here too the end matrix remains symmetric throughout optimization.
To do so, we first construct a factorization of $\matrixend (0)$ that is both balanced and symmetric, for which Lemma~\ref{mf:lem:sym_factorization_stays_sym} ensures the end matrix stays symmetric throughout optimization.
We then prove that the trajectories of the end matrix for the original and the modified initializations coincide.

Recall that $\matrixend(0) = \alpha \cdot \Ibf$ and define $\widebar{\Wbf}_l(0) : = \alpha^{\frac{1}{L}} \cdot  \Ibf$ for $l \in [L]$.
It is easily verified that:
\begin{itemize}
	\item $\matrixend (0) = \widebar{\Wbf}_L(0) \cdots \widebar{\Wbf}_1(0)$.
	\item $\widebar{\Wbf}_l(0) = \widebar{\Wbf}_{L-l+1}(0)^\top$ for $l \in [L]$.
	\item $\widebar{\Wbf}_{l+1}(0)^\top  \widebar{\Wbf}_{l+1}(0) = \widebar{\Wbf}_{l}(0) \widebar{\Wbf}_{l}(0)^\top $ for $l \in [L - 1]$.
\end{itemize}
Meaning, $\widebar{\Wbf}_1(0), \ldots , \widebar{\Wbf}_L(0)$ are balanced, and form a symmetric factorization of $\matrixend (0)$. 
Suppose the factors $\widebar{\Wbf}_1(t), \ldots , \widebar{\Wbf}_L(t)$ follow the gradient flow dynamics, with initial values $\widebar{\Wbf}_1(0), \ldots , \widebar{\Wbf}_L(0)$, and let $\matrixendbar(t) := \widebar{\Wbf}_L(t) \cdots \widebar{\Wbf}_1(t) $ be the induced end matrix. 
From Lemma~\ref{mf:lem:sym_factorization_stays_sym}, it follows that $\matrixendbar (t)$ is symmetric for all $t \geq 0$.

As characterized in \cite{arora2018optimization} (restated as Lemma~\ref{mf:lem:prod_mat_dyn}), if the initial factors are balanced, the end matrix trajectory depends only on its initial value $\matrixend(0)$.
Since both the original and modified initializations are balanced and have the same end matrix, they lead to the exact same trajectory.
Thus, $\matrixend(t) = \matrixendbar (t)$ for all $t \geq 0$, and specifically, $\matrixend (t)$ is symmetric.

The last step is to see that $\matrixend (t)$ is not only symmetric, but positive definite as well.
Since its initial value $\matrixend (0)$ is positive definite, it suffices to show that its eigenvalues do not change sign. 
By Lemma~\ref{mf:lem:det_does_not_change_sign}, the determinant of $\matrixend (t)$ is positive for all $t$.
Specifically, the end matrix does not have zero eigenvalues.
Recalling that $\matrixend (t)$ is an analytic function of~$t$ (Lemma~\ref{mf:lem:analytic_factors_and_prod_mat}), Theorem~6.1 in~\cite{kato2013perturbation} implies that its eigenvalues are continuous in~$t$.
Therefore, they can not change sign, as that would require them to pass through zero, concluding the proof.
\qed

\subsection{Proof of Theorem~\ref{mf:thm:norms_up_finite_perturb}} \label{mf:app:proofs:norms_up_finite_perturb}

The proof follows a similar line to that of Theorem~\ref{mf:thm:norms_up_finite} (Appendix~\ref{mf:app:proofs:norms_up_finite}), where the differences mostly stem from the fact that the solution set $\widetilde{\S}$ (Equation~\eqref{mf:eq:sol_set_perturb}) is not confined to symmetric matrices, as opposed to the original $\S$ (Equation~\eqref{mf:eq:sol_set}), slightly complicating the computation of singular values.
For the sake of the proof, as mentioned in Appendix~\ref{mf:app:proofs:notation}, we omit the time index $t$ when stating results for all $t \geq 0$ or when clear from context. 
We also let $\{w_{i , j}\}_{i,j \in [2]}$ denote the entries of the end matrix~$\matrixend$.

We begin by deriving loss-dependent bounds for $\abs{w_{1,1}}, \sigma_1 (\matrixend)$ and $\sigma_2 (\matrixend)$.
The entries of $\matrixend$ can be trivially bounded by the loss as follows:
\be
\abs{w_{2,2} - \epsilon}  \leq \sqrt{2 \mfendloss (\matrixend)} ~~~ , ~~~ |w_{1,2} - z| \leq \sqrt{2 \mfendloss (\matrixend)} ~~~,~~~ \abs{w_{2,1} - z'} \leq \sqrt{2 \mfendloss (\matrixend)}
\text{\,.}
\label{mf:eqn:perturbed_entries_loss_bound}
\ee
Lemma~\ref{mf:lem:perturbed_w11_bound} below, analogous to Lemma~\ref{mf:lem:w_11_lower_bound} from the proof of Theorem~\ref{mf:thm:norms_up_finite}, characterizes the relation between $\abs{w_{1,1}}$ and the loss.
\begin{lemma} \label{mf:lem:perturbed_w11_bound}
	Suppose $\mfendloss (\matrixend) < \min \{ z^2 / 2, z'^2 / 2 \}$. 
	Then:
	\[
	\abs{w_{1,1}} > \frac{ ( \abs{z} - \sqrt{2 \mfendloss (\matrixend) } ) ( |z'| - \sqrt{2 \mfendloss (\matrixend) } ) }{ \abs{\epsilon} + \sqrt{2 \mfendloss (\matrixend) } } \geq \frac{\abs{z} \cdot |z'|}{\abs{ \epsilon } + \sqrt{2 \mfendloss (\matrixend) } } - (\abs{z} + |z'|)
	\text{\,.}
	\]
\end{lemma}

\begin{proof}
	According to Lemma~\ref{mf:lem:det_does_not_change_sign}, the determinant of $\matrixend$ does not change sign, \ie~it remains equal to $\sign ( z \cdot z' )$ (the initial sign assumed).
	Under the assumption that $\mfendloss (\matrixend) < \min \{ z^2 / 2, z'^2 / 2 \}$, both $w_{1,2}$ and $w_{2,1}$ have the same signs as $z$ and $z'$, respectively, implying that $w_{2,2} \neq 0$ (otherwise we have a contradiction to the sign of the end matrix determinant).
	If $ z \cdot z' > 0$, the determinant is positive as well, and it holds that $w_{1,1} w_{2,2} > w_{1,2} w_{2,1} > 0$.
	Otherwise, if $ z \cdot z' < 0$ we have $w_{1,1} w_{2,2} < w_{1,2} w_{2,1} < 0$.
	Putting it together we may write $\abs{w_{1,1} w_{2,2}} > \abs{w_{1,2} w_{2,1}}$.
	Dividing by $\abs{w_{2,2}}$ and applying the bounds from Equation~\eqref{mf:eqn:perturbed_entries_loss_bound} then completes the proof:
	\[
	\abs{w_{1,1}} > \frac{ ( \abs{z} - \sqrt{2 \mfendloss (\matrixend) } ) ( |z'| - \sqrt{2 \mfendloss (\matrixend) } ) }{ \abs{\epsilon} + \sqrt{2 \mfendloss (\matrixend) } } \geq \frac{\abs{z} \cdot |z'|}{\abs{ \epsilon } + \sqrt{2 \mfendloss (\matrixend) } } - (\abs{z} + |z'|)
	\text{\,.}
	\]
\end{proof}
We are now able to see that, indeed, the smaller $\abs{\epsilon}$ is compared to $\abs{z \cdot z'}$, the higher $\abs{w_{1,1}}$ will be driven when the loss is minimized.
With Lemma~\ref{mf:lem:perturbed_w11_bound} in place, we are now able to bound the singular values of $\matrixend$.
\begin{lemma} \label{mf:lem:perturbed_W_singular_values_bound}
	The singular values of $\matrixend$ fulfill:
	\be
	\begin{split}
		& \sigma_1(\matrixend) \geq \frac{1}{\sqrt{2}} \cdot |w_{1,1}|  - \sqrt{2\mfendloss (\matrixend)} ~, \\ & \sigma_2(\matrixend) \leq 4 \abs{\epsilon} + \left (4 + \frac{ \sqrt{ \abs{z} \cdot |z'| } }{\min \{\abs{z} , |z'| \}} \right ) \sqrt{2 \mfendloss (\matrixend)}
		\text{\,.}
	\end{split}
	\label{mf:eqn:perturbed_sing_values_no_assump_bounds}
	\ee
	Furthermore, if $\mfendloss (\matrixend) < \min \left \{ z^2 / 2 ~ , ~ z'^2 / 2 \right \}$, the bound on $\sigma_2 (\matrixend)$ may be simplified:
	\be
	\sigma_2(\matrixend) \leq 4 \abs{\epsilon} + 4 \sqrt{2 \mfendloss (\matrixend)}
	\text{\,.}
	\label{mf:eqn:perturbed_sing_values_with_assump_bounds}
	\ee
\end{lemma}

\begin{proof}
	Define $\Wbf_{\widetilde{\S}} := \begin{pmatrix} w_{1,1} & z' \\ z & \epsilon \end{pmatrix}$, the orthogonal projection of $\matrixend$ onto the solution set $\widetilde{\S}$.
	From Corollary 8.6.2 in~\cite{golub2012matrix} we know that:
	\be
	|\sigma_i (\matrixend) - \sigma_i (\Wbf_{\widetilde{\S}})| \leq \norm{\matrixend - \Wbf_{\widetilde{\S}}}_{Fro} = \sqrt{2 \mfendloss (\matrixend)}  \quad ,~ i = 1,2 \text{\,.}
	\label{mf:eqn:perturbed_W_singular_value_perturbation_bound}
	\ee
	This means that any bound on the singular values of $\Wbf_{\widetilde{\S}}$ can be transferred to those of $\matrixend$ (up to an additive loss-dependent term).
	It is straightforwardly verified that the squared singular values of $\Wbf_{\widetilde{\S}}$ are
	\be
	\begin{split}
		\sigma_{1}^2 (\Wbf_{\widetilde{\S}}) & = \frac{1}{2} \left ( w_{1,1}^2 + z^2 + z'^2 + \epsilon^2 + \sqrt{\left ( w_{1,1}^2 + z^2 + z'^2 + \epsilon^2 \right )^2 -4 \left ( w_{1,1} \epsilon - z z' \right )^2} \right ) \text{\,,} \\
		\sigma_{2}^2 (\Wbf_{\widetilde{\S}}) & = \frac{1}{2} \left ( w_{1,1}^2 + z^2 + z'^2 + \epsilon^2 - \sqrt{\left ( w_{1,1}^2 + z^2 + z'^2 + \epsilon^2 \right )^2 -4 \left ( w_{1,1} \epsilon - z z' \right )^2} \right ) \text{\,.}
	\end{split}
	\label{mf:eqn:perturbed_W_s_singular_values}
	\ee
	Note that the term inside the square roots is non-negative for all $w_{1,1}, z, z', \epsilon$.
	Since all elements in the expression for $\sigma_1 ^2 (\Wbf_{\widetilde{\S}})$ are non-negative, we have $\sigma_{1} (\Wbf_{\widetilde{\S}}) \geq (1 / \sqrt{2}) \cdot \abs{w_{1,1}}$.
	Combining this with Equation~\eqref{mf:eqn:perturbed_W_singular_value_perturbation_bound} completes the lower bound for $\sigma_1(\matrixend)$.
	
	Next, let $\Wbf_{\widetilde{\S}_0} := \begin{pmatrix} w_{1,1} & z' \\ z & 0 \end{pmatrix}$ be the matrix obtained by replacing the bottom-right entry of $\Wbf_{\widetilde{\S}}$ by~$0$.
	Replacing $\epsilon$ with $0$ in Equation~\eqref{mf:eqn:perturbed_W_s_singular_values}, and applying the identity $a - b = \frac{a^2 - b^2}{a + b}$ for $a, b \in \R$ such that $a + b \neq 0$, we get:
	\be
	\begin{split}
		\sigma_{2}^2 (\Wbf_{\widetilde{\S}_0}) & = \frac{2 z^2 z'^2}{ w_{1,1}^2 + z^2 + z'^2 + \sqrt{\left ( w_{1,1}^2 + z^2 + z'^2 \right )^2 -4 z^2 z'^2 } } \\ 
		& \leq \frac{2 z^2 z'^2}{ w_{1,1}^2 + z^2 + z'^2 }
		\text{\,.}
	\end{split}
	\label{mf:eqn:sq_sing_2_W_s_0_upper_bound}
	\ee
	We initially prove Equation~\eqref{mf:eqn:perturbed_sing_values_with_assump_bounds} in the case where $\mfendloss (\matrixend) < \min \left \{ z^2 / 2 ~ , ~ z'^2 / 2 \right \}$.
	By lifting said assumption we then show that the bound on $\sigma_2 (\matrixend)$ in Equation~\eqref{mf:eqn:perturbed_sing_values_no_assump_bounds} holds for any loss value.
	Under the assumption that $\mfendloss (\matrixend) < \min \left \{ z^2 / 2 ~ , ~ z'^2 / 2 \right \}$, taking the square root of both sides in Equation~\eqref{mf:eqn:sq_sing_2_W_s_0_upper_bound}, we arrive at the following bound:
	\[
	\begin{split}
		\sigma_{2} (\Wbf_{\widetilde{\S}_0}) & \leq \sqrt{2} \cdot \frac{\abs{z} \cdot |z'|}{ \sqrt{w_{1,1}^2 + z^2 + z'^2} } \\
		& \leq \sqrt{6} \cdot \frac{\abs{z} \cdot |z'|}{ \abs{w_{1,1}} + \abs{z} + |z'| } \\
		& \leq \sqrt{6} \cdot \frac{\abs{z} \cdot |z'|}{ \frac{\abs{z} \cdot |z'|}{\abs{ \epsilon } + \sqrt{2 \mfendloss (\matrixend) } } } \\
		& \leq 3 \left ( \abs{ \epsilon } + \sqrt{2 \mfendloss (\matrixend) } \right ) 
		\text{\,,}
	\end{split}
	\]
	where in the second transition we applied the inequality $\sqrt{w_{1,1}^2 + z^2 + z'^2} \geq (\abs{w_{1,1}} + \abs{z} + |z'|) / \sqrt{3}$, and in the third made use of the bound on $\abs{w_{1,1}}$ (Lemma~\ref{mf:lem:perturbed_w11_bound}).
	Applying Corollary 8.6.2 from~\cite{golub2012matrix} twice, once for the matrices $\matrixend$ and $\Wbf_{\widetilde{\S}}$, and another for $\Wbf_{\widetilde{\S}}$ and $\Wbf_{\widetilde{\S}_0}$, we have:
	\[
	\sigma_2 (\matrixend) \leq 3 \left ( \abs{ \epsilon } + \sqrt{2 \mfendloss (\matrixend) } \right ) + \abs{\epsilon} + \sqrt{2 \mfendloss (\matrixend) } =  4 \left ( \abs{ \epsilon } + \sqrt{2 \mfendloss (\matrixend) } \right )
	\text{\,,}
	\]
	achieving the desired result from Equation~\eqref{mf:eqn:perturbed_sing_values_with_assump_bounds}.
	It remains to see that the bound on $\sigma_2(\matrixend)$ in Equation~\eqref{mf:eqn:perturbed_sing_values_no_assump_bounds} holds regardless of the loss value.
	When $\mfendloss (\matrixend) < \min \left \{ z^2 / 2 ~ , ~ z'^2 / 2 \right \}$ it obviously holds since it is only looser than the bound already obtained under this assumption.
	Otherwise, going back to Equation~\eqref{mf:eqn:sq_sing_2_W_s_0_upper_bound}, it can be seen that
	\[
	\sigma_{2}^2 (\Wbf_{\widetilde{\S}_0}) \leq \frac{2 z^2 z'^2}{ (z - z')^2 + 2 \abs{z} \cdot |z'| } \leq  \abs{z} \cdot |z'|
	\text{\,.}
	\]
	Thus, $\sigma_{2} (\Wbf_{\widetilde{\S}_0}) \leq \sqrt{ \abs{z} \cdot |z'| }$.
	Following the same procedure as before (applying Corollary 8.6.2 from~\cite{golub2012matrix}), combined with the fact that $\mfendloss (\matrixend) \geq \min \left \{ z^2 / 2 ~ , ~ z'^2 / 2 \right \}$ concludes the proof:
	\[
	\begin{split}
		\sigma_2 (\matrixend) & \leq  \sqrt{ \abs{z} \cdot |z'| } + \abs{\epsilon} + \sqrt{2 \mfendloss (\matrixend) } \\
		& \leq \frac{ \sqrt{ \abs{z} \cdot |z'| } }{ \min \{\abs{z} , |z'| \} } \cdot \sqrt{2 \mfendloss (\matrixend) } + \abs{\epsilon} + \sqrt{2 \mfendloss (\matrixend) } \\
		& \leq 4 \abs{\epsilon} + \left (4 + \frac{ \sqrt{ \abs{z} \cdot |z'| } }{ \min \{\abs{z} , |z'| \} } \right ) \sqrt{2 \mfendloss (\matrixend)}
		\text{\,.}
	\end{split}
	\]
	
\end{proof}

\subsubsection{Proof of Equation~\eqref{mf:eq:norms_lb_perturb} (Lower Bound for Quasi-Norm)} \label{mf:sec:norm_bound_perturb}

Turning our attention to $\norm{\matrixend}$, following the same steps as in the proof of Theorem~\ref{mf:thm:norms_up_finite} (Appendix~\ref{mf:sec:norm_bound}) will lead to a generalized bound.
By the triangle inequality:
\be
\norm{ \matrixend } \geq \frac{1}{c_{\norm{\cdot}}} \norm{w_{1,1} \ebf_1 \ebf_1^\top} - \norm{\matrixend - w_{1,1} \ebf_1 \ebf_1^\top}
\text{\,,}
\label{mf:eqn:perturb_norm_triangle_ineq}
\ee
where $c_{\norm{\cdot}} \geq 1$ is a constant with which $\norm{\cdot}$ satisfies the weakened triangle inequality (see Footnote~\ref{note:qnorm}).
Let us initially assume that $\mfendloss (\matrixend) < \min \{ z^2 / 2, z'^2 / 2 \}$.
We later lift this assumption, delivering a bound that holds for all loss values.
Invoking Equation~\eqref{mf:eqn:perturbed_entries_loss_bound} we may bound the negative term in Equation~\eqref{mf:eqn:perturb_norm_triangle_ineq} as follows:
\[
\begin{split}
	\norm{\matrixend - w_{1,1} \ebf_1 \ebf_1^\top} & \leq c_{\norm{\cdot}} |w_{2,2}| \norm{\ebf_2 \ebf_2^\top} + c_{\norm{\cdot}}^2 \left ( |w_{2,1}| \norm{\ebf_2 \ebf_1^\top} + |w_{1,2}| \norm{\ebf_1 \ebf_2^\top} \right ) \\
	& \leq 3 c_{\norm{\cdot}}^2 \left ( \max \{ \abs{z}, |z'|, \abs{\epsilon} \} + \sqrt{2 \mfendloss (\matrixend)} \right ) \max_{ \substack{ i,j \in \{1,2\} \\ (i,j) \neq (1,1) } } \norm{ \ebf_i \ebf_j^\top } \\
	& \leq 6 c_{\norm{\cdot}}^2 \max \{ \abs{z}, |z'|, \abs{\epsilon} \} \cdot \max_{ \substack{ i,j \in \{1,2\} \\ (i,j) \neq (1,1) } } \norm{ \ebf_i \ebf_j^\top } 
	\text{\,,}
\end{split}
\]
Returning to Equation~\eqref{mf:eqn:perturb_norm_triangle_ineq}, applying the inequality above and the bound on $\abs{w_{1,1}}$ (Lemma~\ref{mf:lem:perturbed_w11_bound}) we have:
\[
\begin{split}
	\norm{\matrixend} & \geq \frac{\norm{\ebf_1 \ebf_1^\top}}{c_{\norm{\cdot}}} \Big ( \frac{\abs{z} \cdot |z'|}{\abs{ \epsilon } + \sqrt{2 \mfendloss (\matrixend) } } - \abs{z} - |z'| \Big )  \\[1mm]
	& \hspace{5mm} - 6 c_{\norm{\cdot}}^2 \max \{ \abs{z}, |z'|, \abs{\epsilon} \} \max_{ \substack{ i,j \in \{1,2\} \\ (i,j) \neq (1,1) } } \norm{ \ebf_i \ebf_j^\top } \\
	& \geq \frac{\norm{\ebf_1 \ebf_1^\top}}{c_{\norm{\cdot}}} \cdot \frac{\abs{z} \cdot |z'|}{\abs{ \epsilon } + \sqrt{2 \mfendloss (\matrixend) } } - 8 c_{\norm{\cdot}}^2 \max \{ \abs{z}, |z'|, \abs{\epsilon} \} \cdot \max_{ i,j \in \{1,2\}  } \norm{ \ebf_i \ebf_j^\top }
	\text{\,.}
\end{split}
\]
Since $\norm{\matrixend}$ is trivially lower bounded by zero, defining the constants
\[
\begin{split}
a_{\norm{\cdot}} & := \frac{\norm{\ebf_1 \ebf_1^\top}}{c_{\norm{\cdot}}} \text{\,,} \\[1mm]
b_{\norm{\cdot}} & := \max \left \{  \frac{ a_{\norm{\cdot}} \cdot \abs{z} \cdot |z'|}{\abs{ \epsilon } + \min \{ \abs{z}, |z'| \}} , 8 c_{\norm{\cdot}}^2 \max \{ \abs{z}, |z'|, \abs{\epsilon} \} \max_{ i,j \in \{1,2\}  } \norm{ \ebf_i \ebf_j^\top }  \right \}
\text{\,,}
\end{split}
\]
allows us, on the one hand, to arrive at a bound of the form:
\[
\norm{\matrixend} \geq a_{\norm{\cdot}} \cdot \frac{\abs{z} \cdot |z'|}{\abs{ \epsilon } + \sqrt{2 \mfendloss (\matrixend) } } - b_{\norm{\cdot}}
\text{\,,}
\]
and on the other hand, to remove the previous assumption on the loss: in the case where $\mfendloss (\matrixend) \geq \min \{ z^2 / 2, z'^2 / 2 \}$, the bound is non-positive and trivially holds.
Noticing this is exactly Equation~\eqref{mf:eq:norms_lb_perturb} (recall we omitted the time index $t$), concludes this part of the proof.

\subsubsection{Proof of Equation~\eqref{mf:eq:erank_ub_perturb} (Upper Bound for Effective Rank)} \label{mf:sec:effective_rank_bound_perturb}

Derivation of the upper bound for effective rank (Definition~\ref{def:erank}) is initially done under the assumption that $\mfendloss (\matrixend) < \min \{ z^2 / 8, z'^2 / 8 \}$.
We then remove this assumption, establishing a bound that holds for all loss values.

The bounds on $\sigma_1 (\matrixend)$ and $\sigma_2 (\matrixend)$ in Lemma~\ref{mf:lem:perturbed_W_singular_values_bound} give:
\[
\begin{split}
	\rho_1 (\matrixend) & =  \frac{\sigma_1 (\matrixend)}{\sigma_1 (\matrixend) + \sigma_2(\matrixend)} \\
	& \geq \frac{\sigma_1 (\matrixend)}{\sigma_1 (\matrixend) + 4 \left ( \abs{\epsilon} + \sqrt{2\mfendloss (\matrixend)} \right )} \\
	& = 1 - \frac{4 \left ( \abs{\epsilon} + \sqrt{2\mfendloss (\matrixend)} \right )}{\sigma_1 (\matrixend) + 4 \left ( \abs{\epsilon} + \sqrt{2\mfendloss (\matrixend)} \right )} \\
	& \geq  1 - \frac{4 \left ( \abs{\epsilon} + \sqrt{2\mfendloss (\matrixend)} \right )}{ \frac{1}{\sqrt{2}} \cdot \abs{w_{1,1}} + 4 \abs{\epsilon} + 3 \sqrt{2\mfendloss (\matrixend)} } \\
	& \geq 1 - \frac{4 \sqrt{2} \left ( \abs{\epsilon} + \sqrt{2\mfendloss (\matrixend)} \right )}{\abs{w_{1,1}}}
	\text{\,.}
\end{split}
\]

Additionally, under our assumption that $\mfendloss (\matrixend) < \min \{ z^2 / 8, z'^2 / 8 \}$, the bound on $\abs{w_{1,1}}$ in Lemma~\ref{mf:lem:perturbed_w11_bound} can be simplified to:
\[
\abs{w_{1,1}} \geq \frac{ ( \abs{z} - \sqrt{2 \mfendloss (\matrixend) } ) ( |z'| - \sqrt{2 \mfendloss (\matrixend) } ) }{ \abs{\epsilon} + \sqrt{2 \mfendloss (\matrixend) } } \geq \frac{ \min \{ \abs{z} , |z'| \}^2 }{4 \left ( \abs{\epsilon} + \sqrt{2 \mfendloss (\matrixend)} \right )}
\text{\,.}
\]
Combining the last two inequalities we have:
\[
\rho_2 (\matrixend) = 1 - \rho_1 (\matrixend) \leq  \frac{16 \sqrt{2} \left ( \abs{\epsilon} + \sqrt{2\mfendloss (\matrixend)} \right )^2 }{ \min \{ \abs{z} , |z'| \}^2 }
\text{\,.}
\]
It is now possible to see that, in accordance with Section~\ref{mf:sec:analysis:robust}, the smaller $\abs{\epsilon}$ is compared to $\min \{ \abs{z} , |z'| \}$, the closer to zero $\rho_2 (\matrixend)$ becomes as the loss is minimized.
Let $h \left (\rho_2 (\matrixend) \right ) := - \rho_2 (\matrixend) \cdot \ln \left (\rho_2 (\matrixend) \right ) - (1 - \rho_2 (\matrixend)) \cdot \ln \left (1 - \rho_2 (\matrixend) \right )$ denote the binary entropy function, and recall that the effective rank of the end matrix defined to be $\erank (\matrixend) := \exp \{ h \left (\rho_2 (\matrixend) \right ) \}$.
As in the proof of Theorem~\ref{mf:thm:norms_up_finite} (Appendix~\ref{mf:sec:effective_rank_bound}), we may bound the exponent on the interval $[0, \ln (2)]$ by the linear function intersecting it at these points.
That is,
\[
\erank (\matrixend) \leq 1 + \frac{1}{\ln (2)} \cdot h \left (\rho_2 (\matrixend) \right )
\text{\,.}
\]
From Lemma~\ref{mf:lem:entropy_sqrt_bound} it holds that $h \left (\rho_2 (\matrixend) \right ) \leq 2 \sqrt{\rho_2 (\matrixend)}$.
Plugging this into the inequality above leads to:
\[
\begin{split}
	\erank (\matrixend) & \leq 1 + \frac{8 \cdot 2^{\frac{1}{4}}}{\ln (2) \cdot \min \{ \abs{z} , |z'| \} } \cdot \left ( \abs{\epsilon} + \sqrt{2\mfendloss (\matrixend)} \right ) \\
	& \leq 1 + \frac{16}{\min \{ \abs{z} , |z'| \} } \cdot \left ( \abs{\epsilon} + \sqrt{2\mfendloss (\matrixend)} \right )
	\text{\,,}
\end{split}
\]
where the second transition is a slight simplification of the constants ($2^{1 /4} / \ln (2) < 2$).
As will be shown below, $\inf_{\Wbf' \in \widetilde{S}} \erank(\Wbf') = 1$.
We may thus conclude:
\[
\erank (\matrixend) \leq \inf_{\Wbf' \in \widetilde{S}} \erank(\Wbf') + \frac{16}{\min \{ \abs{z} , |z'| \} } \cdot \left ( \abs{\epsilon} + \sqrt{2\mfendloss (\matrixend)} \right )
\text{\,.}
\]
Notice that when $\mfendloss (\matrixend) \geq \min \{ z^2 / 8, z'^2 / 8 \}$ the inequality trivially holds since the right-hand side is greater than $2$ (the maximal effective rank for a $2 \times 2$ matrix).
This establishes Equation~\eqref{mf:eq:erank_ub_perturb} (time index is omitted).

It remains to prove that $\inf_{\Wbf' \in \widetilde{S}} \erank(\Wbf') = 1$.
If $\epsilon \neq 0$, it is trivial since there exists $\Wbf' \in \widetilde{\S}$ with $\rank(\Wbf') = 1$, meaning $\sigma_2 (\Wbf') = 0$ and $\erank(\Wbf') = 1$.
If $\epsilon = 0$, examining the squared singular values of $\Wbf' \in \widetilde{\S}$ (Equation~\eqref{mf:eqn:perturbed_W_s_singular_values} with $(\Wbf')_{1,1}$ in place of $w_{1,1}$) reveals that $\lim_{(\Wbf')_{1,1} \rightarrow \infty} \sigma_2 (\Wbf') = 0$, while $\lim_{(\Wbf')_{1,1} \rightarrow \infty} \sigma_1 (\Wbf') = \infty$.
Thus, there exists a matrix in $\widetilde{\S}$ with effective rank arbitrarily close to $1$.
Since the effective rank of any matrix is at least $1$, this implies that $\inf_{\Wbf' \in \widetilde{S}} \erank(\Wbf') = 1$.

\subsubsection{Proof of Equation~\eqref{mf:eq:irank_dist_ub_perturb} (Upper Bound for Distance From Infimal Rank)} \label{mf:sec:distance_from_infimal_rank_upper_bound_perturb}

We claim that the infimal rank (Definition~\ref{def:irank_dist}) of $\widetilde{\S}$ is $1$.
Since $z, z' \neq 0$, it cannot be $0$.
If $\epsilon \neq 0$, our claim is trivial since there exists $\Wbf' \in \widetilde{\S}$ with $\rank(\Wbf') = 1$.
Otherwise, inspecting the squared singular values of a matrix $\Wbf' \in \widetilde{\S}$ (Equation~\eqref{mf:eqn:perturbed_W_s_singular_values} with $(\Wbf')_{1,1}$ in place of $w_{1,1}$), we can see that, when $\epsilon = 0$, taking $(\Wbf')_{1,1}$ to infinity drives the minimal singular value towards zero ($\lim_{(\Wbf')_{1,1} \rightarrow \infty} \sigma_2 (\Wbf') = 0$).
Hence, the distance of $\widetilde{\S}$ from the set of matrices with rank $1$ or less is $0$ in this case as well.

The distance of the end matrix from the infimal rank of $\widetilde{\S}$ is therefore given by $\frodist ( \matrixend ( t ) , \M_1 ) = \sigma_2 (\matrixend (t) )$. 
From Lemma~\ref{mf:lem:perturbed_W_singular_values_bound} we have
\[
\frodist(\matrixend (t), \M_{1}) \leq 4 \abs{\epsilon} + \left (4 + \frac{ \sqrt{ \abs{z} \cdot |z'| } }{\min \{\abs{z} , |z'| \}} \right ) \sqrt{2 \mfendloss (t)}
\text{\,,}
\]
for all $t \geq 0$.

\subsubsection{Robustness to Change in Observed Locations} \label{mf:sec:robustness_observed_locations}

Lastly, we prove that the established bounds (Equations~\eqref{mf:eq:norms_lb_perturb},~\eqref{mf:eq:erank_ub_perturb} and~\eqref{mf:eq:irank_dist_ub_perturb}) are robust to a change in observed locations.
Let $(i, j) \in [2] \times [2]$ be the unobserved entry's location.
Following proof steps analogous to those in Lemmas~\ref{mf:lem:perturbed_w11_bound} and~\ref{mf:lem:perturbed_W_singular_values_bound}~---~while recalling our assumption of $\det ( \matrixend ( 0 ) )$ having same sign as $z \cdot z'$ if $i = j$ and opposite sign otherwise~---~yields identical bounds on the unobserved entry and singular values of $\matrixend$. 
Since the derivations of Equations~\eqref{mf:eq:norms_lb_perturb}, \eqref{mf:eq:erank_ub_perturb}, and~\eqref{mf:eq:irank_dist_ub_perturb} in Appendices~\ref{mf:sec:norm_bound_perturb}, \ref{mf:sec:effective_rank_bound_perturb}, and~\ref{mf:sec:distance_from_infimal_rank_upper_bound_perturb}, respectively, rely solely on the aforementioned bounds, the proof concludes.
\qed

%% file: Appendices/imp_reg_tf.tex
\chapter{Implicit Regularization in Tensor Factorization}
\label{tf:app:imp_reg_tf}

\section{Extension to Tensor Sensing} \label{tf:app:sensing}

Our theoretical analyses (Sections \ref{tf:sec:dynamic} and~\ref{tf:sec:rank}) are presented in the context of tensor completion, but readily extend to the more general task of \emph{tensor sensing}~---~reconstruction of an unknown tensor from linear measurements (projections).
In this appendix we outline the extension.
Empirical demonstrations for tensor sensing are given in Appendix~\ref{tf:app:experiments:further} (Figure~\ref{tf:fig:ts_mse_ord4}).

For measurement tensors $\{ \A_i \in \R^{D_1 \times \cdots \times D_N } \}_{i = 1}^M$ and a ground truth tensor $\W^* \in \R^{D_1 \times \cdots \times D_N}$, the goal in tensor sensing is to reconstruct $\W^*$ based on $\{ \inprodnoflex{ \A_i }{ \W^* } \in \R \}_{i = 1}^{M}$, where $\inprod{\, \cdot \,}{\, \cdot \,}$ represents the standard inner product.
Similarly to tensor completion (\cf~Equation~\eqref{tf:eq:tc_loss}), a standard loss function for the task is:
\[
\L_{s} ( \W ) = \frac{1}{M} \sum\nolimits_{i = 1}^M \ell \left ( \inprod{ \A_i }{ \W }  - \inprod{ \A_i }{ \W^* }  \right )
\text{\,,}
\]
where $\L_{s} : \R^{D_1 \times \cdots \times D_N} \to \R_{\geq 0}$, and $\ell : \R \to \R_{\geq 0}$ is differentiable and locally smooth.
Note that tensor completion is a special case, in which the measurement tensors hold $1$ at a single entry and $0$ elsewhere.

Beginning with Section~\ref{tf:sec:dynamic}, its results (in particular Lemma~\ref{tf:lem:balancedness_conservation_body}, Theorem~\ref{tf:thm:dyn_fac_comp_norm_unbal} and Corollary~\ref{tf:cor:dyn_fac_comp_norm_balanced}) hold (and are proven in Appendix~\ref{tf:app:proofs}) for any differentiable and locally smooth~$\tfendloss (\cdot)$, thus they apply as is to tensor sensing.
Turning to Section~\ref{tf:sec:rank}, the extension of Theorem~\ref{tf:thm:approx_rank_1} and Corollary~\ref{corollary:converge_rank_1} to tensor sensing (with Huber loss) is straightforward.
Proofs rely on the specifics of tensor completion only in the preliminary Lemmas~\ref{tf:lem:huber_loss_const_grad_near_zero},~\ref{tf:lem:huber_loss_smooth} and~\ref{tf:lem:huber_cp_objective_is_smooth_over_bounded_domain} (Appendix~\ref{tf:app:proofs:approx_rank_1:prelim_lemmas}), for which analogous lemmas may readily be established.
Thus, up to slight changes in constants if $\max_{i = 1, \ldots, M} \normnoflex{ \A_i} > 1$, the results carry over.

\subsection{Stronger Results Under Restricted Isometry Property}
\label{tf:app:sensing:rip}

In the classic setting of \emph{matrix sensing} (tensor sensing with order~$N = 2$), a commonly studied condition on the measurement matrices is the \emph{restricted isometry property}.
This condition allows for efficient recovery when the ground truth matrix has low rank, and holds with high probability when the entries of the measurement matrices are drawn independently from a zero-mean sub-Gaussian distribution (\cf~\cite{recht2010guaranteed}).
The notion of restricted isometry property extends from matrix to tensor sensing (\ie~from order $N = 2$ to arbitrary $N \in \N_{\geq 2}$)~---~see~\cite{rauhut2017low,ibrahim2020recoverability}.
When it applies, the tensor sensing analogues of Theorem~\ref{tf:thm:approx_rank_1} and Corollary~\ref{corollary:converge_rank_1} can be strengthened as described below.

\medskip

In the context of tensor sensing, the restricted isometry property is defined as follows.
\begin{definition}
	\label{def:rip}
	We say that the measurement tensors $\{ \A_i \in \R^{D_1 \times \cdots \times D_N} \}_{i = 1}^M$ satisfy \emph{$r$-restricted isometry property} (\emph{$r$-RIP}) with parameter $\delta \in [0, 1)$ if:
	\[
	(1 - \delta) \norm{ \W }^2 \leq \sum\nolimits_{i = 1}^M \inprod{ \A_i }{ \W }^2 \leq (1 + \delta) \norm{ \W }^2
	\text{\,,}
	\]
	for all $\W \in \R^{D_1 \times \cdots \times D_N}$ of tensor rank $r$ or less.
\end{definition}
By~\cite{ibrahim2020recoverability}, given $m \in \OO ( \log (N) \cdot \sum_{n = 1}^N D_n )$ measurement tensors with entries drawn independently from a zero-mean sub-Gaussian distribution, $1$-RIP holds with high probability.
In this case, we may strengthen the tensor sensing analogue of Theorem~\ref{tf:thm:approx_rank_1}, such that it ensures that arbitrarily small initialization leads tensor factorization to follow a rank one trajectory for an arbitrary amount of time, regardless of the distance traveled.
That is, with the notations of Theorem~\ref{tf:thm:approx_rank_1}, for any time duration $T > 0$ and degree of approximation $\epsilon \in ( 0 , 1 )$, if initialization is sufficiently small, $\tftensorendbar (t)$~is within $\epsilon$~distance from a balanced rank one trajectory emanating from~$\S$ at least until time $t \geq T$.
To see it is so, notice that since the loss function during gradient flow is monotonically non-increasing, $\sum_{i = 1}^M \inprodnoflex{ \A_i }{ \W_1 (t) }^2$ is bounded through time for any rank one trajectory~$\W_1 (t)$.
In turn, since the measurement tensors satisfy $1$-RIP, all such trajectories emanating from $\S$ are confined to a ball of radius $B > 0$ about the origin, for some $B > 0$.
By the tensor sensing analogue of Theorem~\ref{tf:thm:approx_rank_1}, sufficiently small initialization ensures that there exists $\W_1 (t)$~---~a balanced rank one trajectory emanating from $\S$~---~such that $\tftensorendbar (t)$ is within $\epsilon$~distance from it at least until $t \geq T$ or $\normnoflex{\tftensorendbar (t)} \geq B + 1$.
However, we know that $\normnoflex{\W_1 (t)} \leq B$, and so $\tftensorendbar (t)$ cannot reach norm of $B + 1$ before time~$T$, as that would entail a contradiction~---~$\normnoflex{\W_1 (t)} > B$.
As a consequence of the above, in the tensor sensing analogue of Corollary~\ref{corollary:converge_rank_1}, when $1$-RIP is satisfied we need not assume all balanced rank one trajectories emanating from~$\S$ are jointly bounded.

\section{Further Experiments and Implementation Details} \label{tf:app:experiments}

\subsection{Further Experiments} \label{tf:app:experiments:further}

Figures~\ref{tf:fig:tc_huber_ord4},~\ref{tf:fig:tc_mse_ord3}, and~\ref{tf:fig:ts_mse_ord4} supplement Figure~\ref{tf:fig:tc_mse_ord4} from Section~\ref{tf:sec:experiments:dyn} by including, respectively: \emph{(i)} Huber loss (Equation~\eqref{tf:eq:huber_loss}) instead of $\ell_2$~loss; \emph{(ii)} ground truth tensors of different orders and (tensor) ranks; and \emph{(iii)} tensor sensing (see Appendix~\ref{tf:app:sensing}).
Table~\ref{table:mnist_fmnist_linear_errors} supplements Figure~\ref{tf:fig:mnist_fmnist_rank}, reporting mean squared errors of linear predictors fitted to the different datasets.

\subsection{Implementation Details} \label{tf:app:experiments:details}

Below are implementation details omitted from our experimental reports (Section~\ref{tf:sec:experiments} and Appendix~\ref{tf:app:experiments:further}).
Source code for reproducing our results and figures can be found at \url{https://github.com/noamrazin/imp_reg_in_tf} (based on the PyTorch framework~\cite{paszke2019pytorch}).

\begin{figure*}[h!]
	\begin{center}
		\subfloat{
			\includegraphics[width=0.24\textwidth]{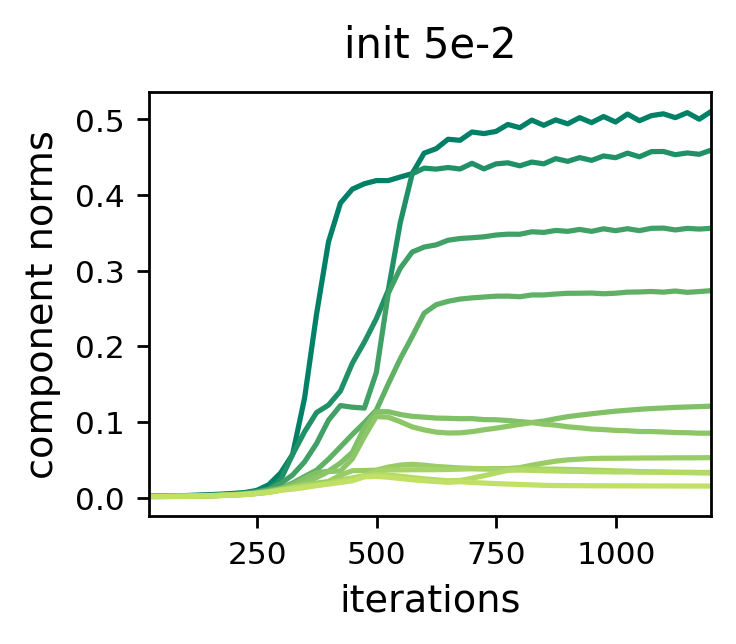}
		}
		\subfloat{
			\includegraphics[width=0.24\textwidth]{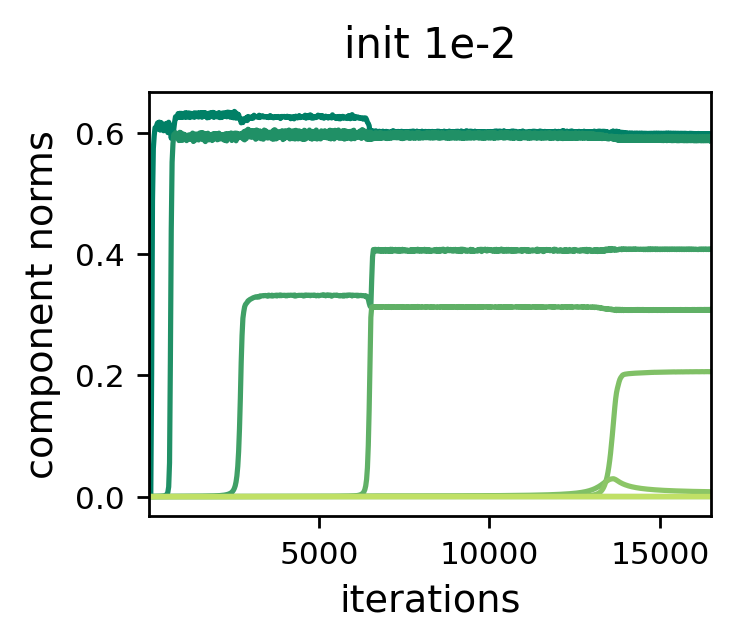}
		}
		\subfloat{
			\includegraphics[width=0.24\textwidth]{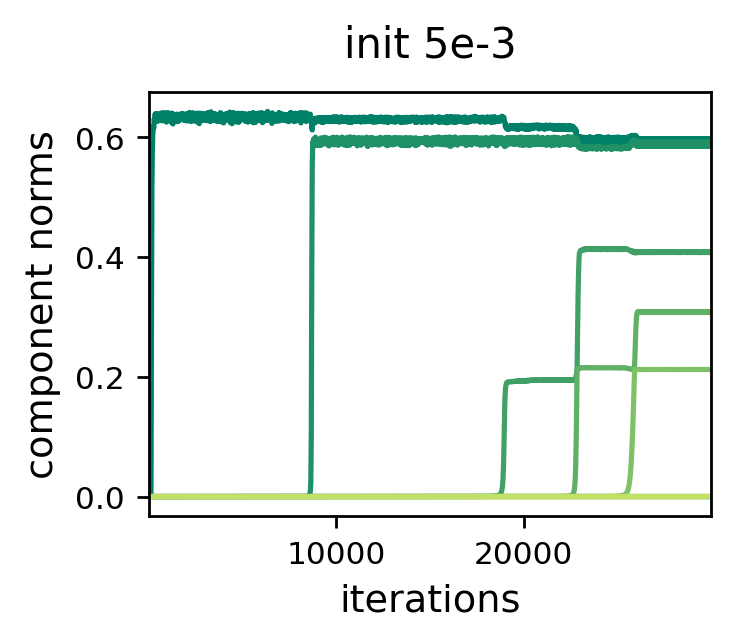}
		}
		\subfloat{
			\includegraphics[width=0.24\textwidth]{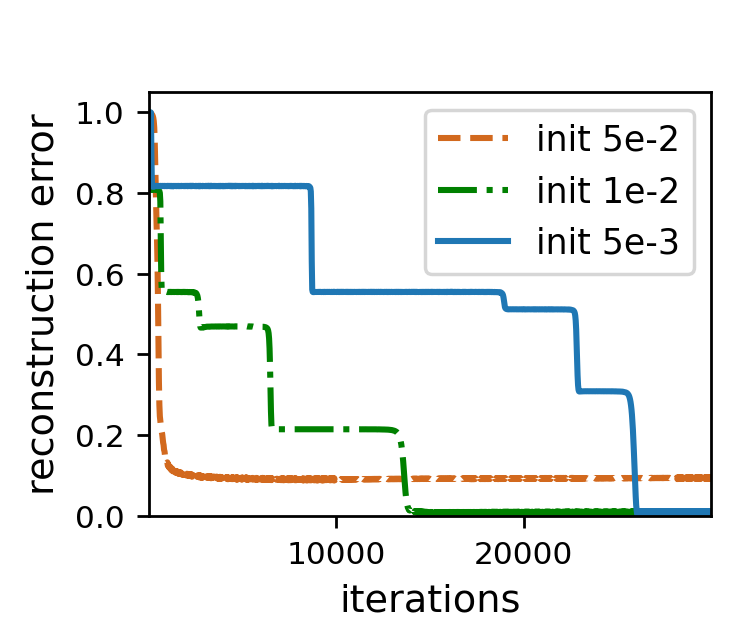}
		}
	\end{center}
	\vspace{-2mm}
	\caption{
		Dynamics of gradient descent over tensor factorization (with Huber loss) --- incremental learning of components yields low tensor rank solutions.
		This figure is identical to Figure~\ref{tf:fig:tc_mse_ord4}, except that the minimized objective (Equation~\eqref{tf:eq:tc_loss}) is based on Huber loss~($\ell_h ( \cdot )$ from Equation~\eqref{tf:eq:huber_loss}) instead of $\ell_2$~loss.
		In accordance with Assumption~\ref{tf:assump:delta_h}, the transition point~$\delta_h$ was set to $5 \cdot 10^{-7}$~---~smaller than the absolute value of observed entries (though larger $\delta_h$ led to similar results).
		For further details see caption of Figure~\ref{tf:fig:tc_mse_ord4}, as well as Appendix~\ref{tf:app:experiments:details:synth_ts}.
	}
	\vspace{-1mm}
	\label{tf:fig:tc_huber_ord4}
\end{figure*}

\begin{figure*}[h!]
	\begin{center}
		\subfloat{
			\includegraphics[width=0.24\textwidth]{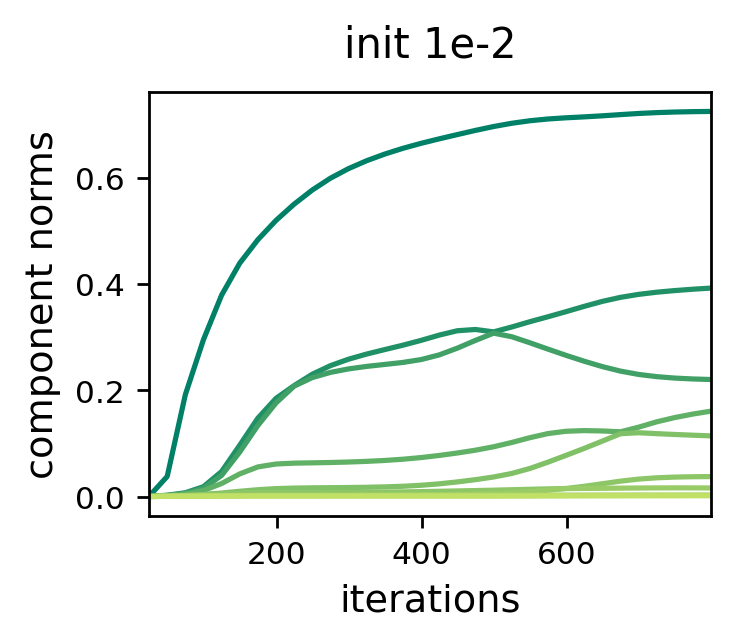}
		}
		\subfloat{
			\includegraphics[width=0.24\textwidth]{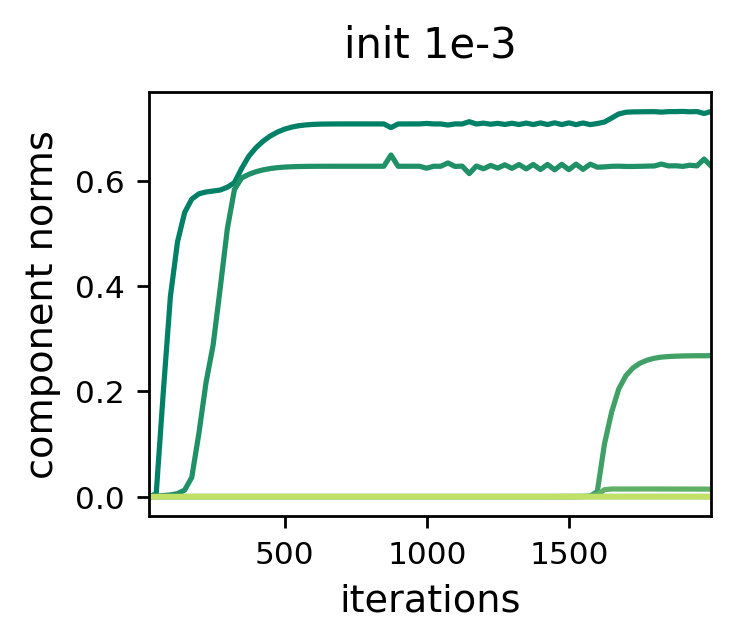}
		}
		\subfloat{
			\includegraphics[width=0.24\textwidth]{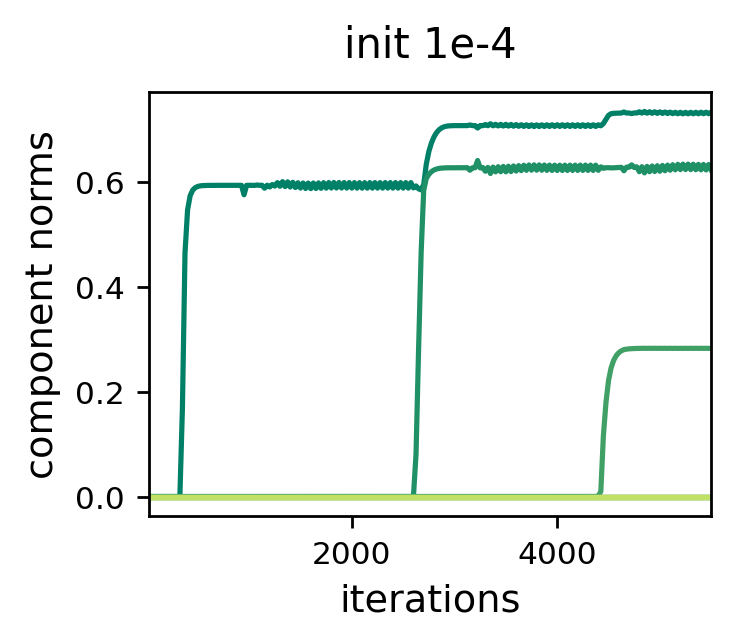}
		}
		\subfloat{
			\includegraphics[width=0.24\textwidth]{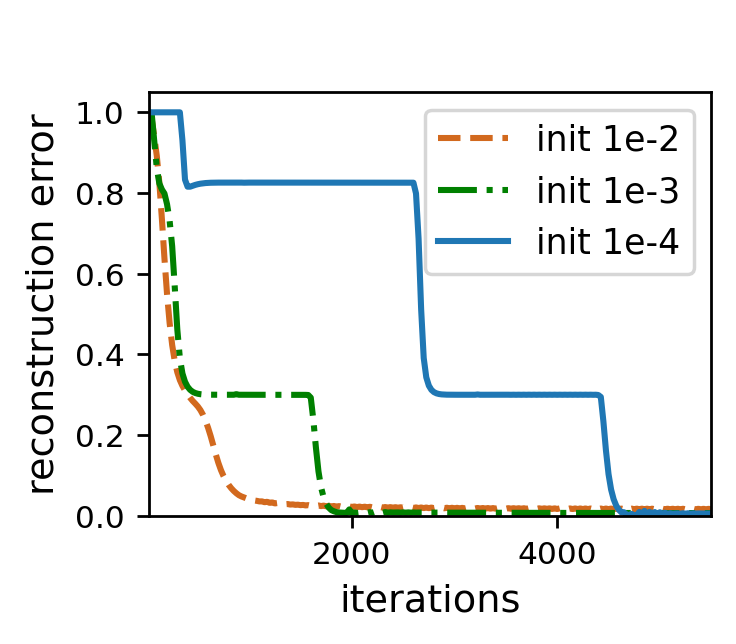}
		}
	\end{center}
	\vspace{-3mm}
	\caption{
		Dynamics of gradient descent over (order~$3$) tensor factorization --- incremental learning of components yields low tensor rank solutions.
		This figure is identical to Figure~\ref{tf:fig:tc_mse_ord4}, except that: \emph{(i)} the ground truth tensor is of (tensor) rank $3$ with size $10$-by-$10$-by-$10$ (order $3$), completed based on $300$ observed entries (smaller sample sizes led to solutions with tensor rank lower than that of the ground truth tensor); and \emph{(ii)} the employed tensor factorization consists of $R = 100$ components (large enough to express any tensor).
		For further details see caption of Figure~\ref{tf:fig:tc_mse_ord4}, as well as Appendix~\ref{tf:app:experiments:details:synth_ts}.
	}
	\vspace{-1mm}
	\label{tf:fig:tc_mse_ord3}
\end{figure*}

\begin{figure*}[h!]
	\begin{center}
		\subfloat{
			\includegraphics[width=0.24\textwidth]{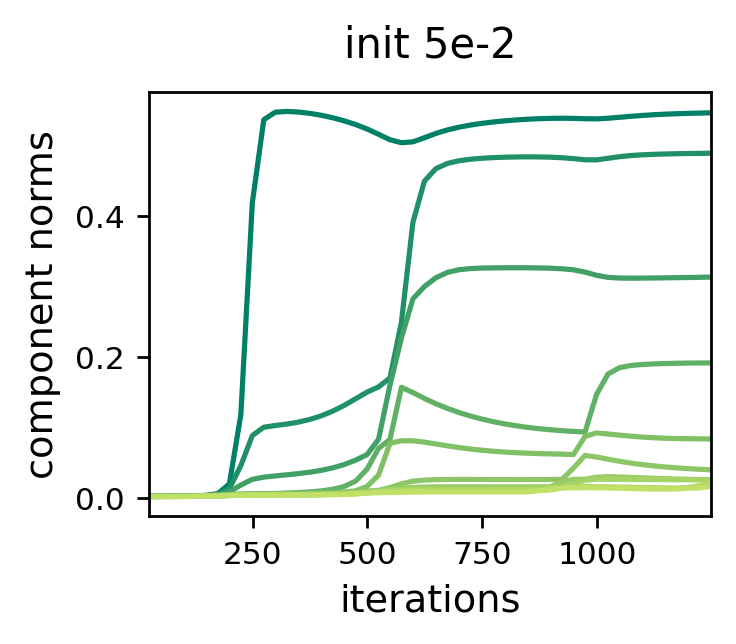}
		}
		\subfloat{
			\includegraphics[width=0.24\textwidth]{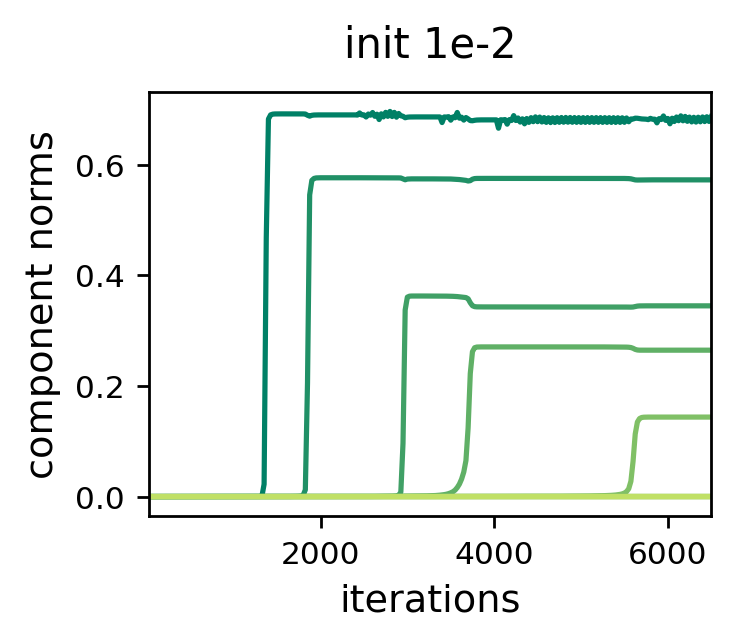}
		}
		\subfloat{
			\includegraphics[width=0.24\textwidth]{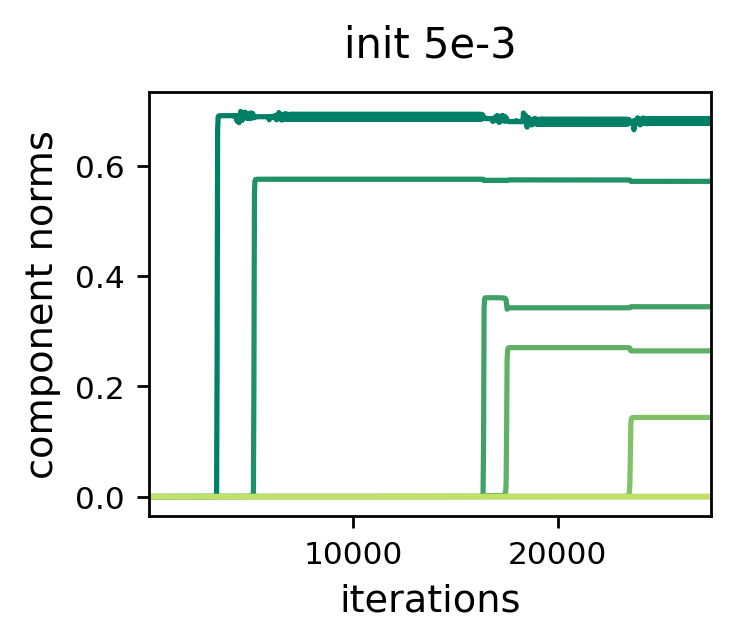}
		}
		\subfloat{
			\includegraphics[width=0.24\textwidth]{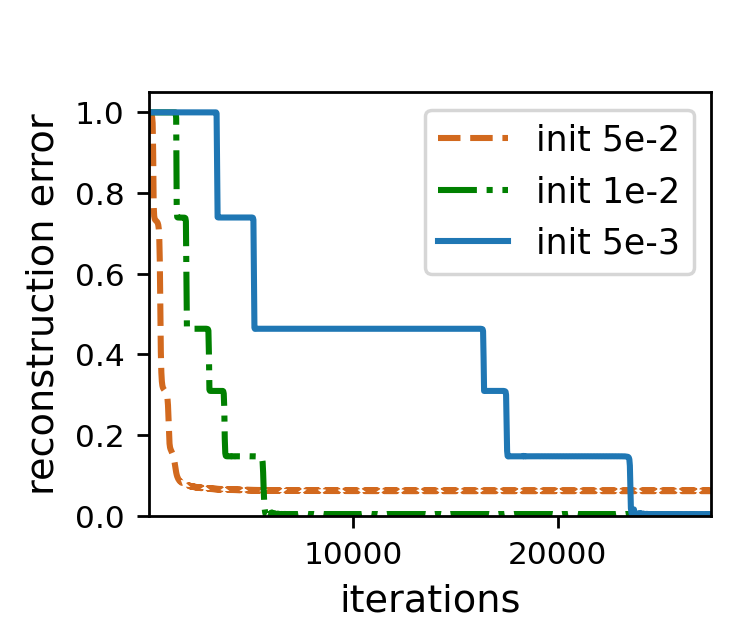}
		}
	\end{center}
	\vspace{-3mm}
	\caption{
		Dynamics of gradient descent over tensor factorization (on tensor sensing task) --- incremental learning of components yields low tensor rank solutions.
		This figure is identical to Figure~\ref{tf:fig:tc_mse_ord4}, except that reconstruction of the ground truth tensor is based on $2000$ linear measurements (instead of $2000$ randomly chosen entries), \ie~on $\{ \inprodnoflex{ \A_i }{ \W^* } \}_{i = 1}^{2000}$, where $\W^* \in \R^{D_1 \times \cdots \times D_N}$ is the ground truth tensor and \smash{$\A_1, \ldots, \A_{2000} \in \R^{D_1 \times \cdots \times D_N}$} are measurement tensors sampled independently from a zero-mean Gaussian distribution (see Appendix~\ref{tf:app:sensing} for a description of the tensor sensing task).
		For further details see caption of Figure~\ref{tf:fig:tc_mse_ord4}, as well as Appendix~\ref{tf:app:experiments:details:synth_ts}.
	}
	\label{tf:fig:ts_mse_ord4}
\end{figure*}

\begin{table}[t]
	\caption{
		Linear predictors are incapable of accurately fitting the datasets in the experiment reported by Figure~\ref{tf:fig:mnist_fmnist_rank}.
		Table presents mean squared errors (over train and test sets) attained by fitting linear predictors to the one-vs-all prediction tasks induced by MNIST and Fashion-MNIST datasets, as well as their random variants (in compliance with Figure~\ref{tf:fig:mnist_fmnist_rank}, to mitigate impact of outliers, large squared errors over test samples were clipped~---~see Appendix~\ref{tf:app:experiments:details:natural_data} for details).
		For each dataset, mean and standard deviation of train and test errors, taken over the different one-vs-all prediction tasks, are reported.
		Notice that all errors are not far from~$0.09$~---~the variance of the label~---~which is trivial to achieve.		
		For further details see caption of Figure~\ref{tf:fig:mnist_fmnist_rank}, as well as Appendix~\ref{tf:app:experiments:details:natural_data}.
	}
	\label{table:mnist_fmnist_linear_errors}
	\begin{center}
				\fontsize{7}{10}\selectfont
				\begin{tabular}{lcccc}
					\toprule
					& \multicolumn{2}{c}{MNIST} & \multicolumn{2}{c}{Fashion-MNIST}\\
					& Train& Test & Train& Test \\
					\midrule
					Original & $3.90\cdot{}10^{-2}$ $\pm$ $8.37\cdot{}10^{-3}$ & $3.92\cdot{}10^{-2}$ $\pm$ $8.04\cdot{}10^{-2}$ & $4.09\cdot{}10^{-2}$ $\pm$ $1.50\cdot{}10^{-2}$ & $4.24\cdot{}10^{-2}$ $\pm$ $1.58\cdot{}10^{-2}$ \\
					Rand Image & $8.88\cdot{}10^{-2}$ $\pm$ $4.24\cdot{}10^{-3}$ & $9.11\cdot{}10^{-2}$ $\pm$ $4.80\cdot{}10^{-3}$ & $8.88\cdot{}10^{-2}$ $\pm$ $3.11\cdot{}10^{-5}$ & $9.12\cdot{}10^{-2}$ $\pm$ $2.07\cdot{}10^{-4}$ \\
					Rand Label & $8.89\cdot{}10^{-2}$ $\pm$ $4.22\cdot{}10^{-3}$ & $9.09\cdot{}10^{-2}$ $\pm$ $4.77\cdot{}10^{-3}$ & $8.88\cdot{}10^{-2}$ $\pm$ $7.46\cdot{}10^{-5}$ & $9.11\cdot{}10^{-2}$ $\pm$ $2.23\cdot{}10^{-4}$ \\
					\bottomrule
				\end{tabular}
	\end{center}
\end{table}

\subsubsection{Dynamics of Learning (Figures~\ref{tf:fig:tc_mse_ord4},~\ref{tf:fig:tc_huber_ord4}, ~\ref{tf:fig:tc_mse_ord3}, and~\ref{tf:fig:ts_mse_ord4})} \label{tf:app:experiments:details:synth_ts}

The number of components $R$ was set to ensure an unconstrained search space, \ie~to $10^2$ and $10^3$ for tensor sizes $10$-by-$10$-by-$10$ and $10$-by-$10$-by-$10$-by-$10$ respectively.\footnote{
	For any $D_1 , \ldots , D_N \in \N$, setting $R = ( \Pi_{n = 1}^N D_n ) / \max \{ D_n \}_{n = 1}^N$ suffices for expressing all tensors in $\R^{D_1 \times \cdots \times D_N}$ (\cf~\cite{hackbusch2012tensor}).
}
Gradient descent was initialized randomly by sampling each weight independently from a zero-mean Gaussian distribution, and was run until the loss reached a value lower than $10^{-8}$ or~$10^6$ iterations elapsed.
For each figure, experiments were carried out with standard deviation of initialization varying over $\{ 0.05, 0.01, 0.005, 0.001, 0.0005, 0.0001, 0.00005 \}$.
Reported are representative runs illustrating the different types of dynamics encountered.
To facilitate more efficient experimentation, we employed the adaptive learning scheme described in Appendix~\ref{mf:app:experiments:details:tf}.

Generating a ground truth rank~$R^*$ tensor $\W^* \in \R^{D_1 \times \cdots \times D_N}$ was done as in the experiments of \cref{mf:sec:experiments:tensor}, \ie~by computing $\W^* = \sum\nolimits_{r = 1}^{R^*} \vbf^{1 }_r \tenp \cdots \tenp \vbf^{ N }_r$, with $\{ \vbf^{n }_r \in \R^{D_n} \}_{r = 1}^{R^*} \hspace{0mm}_{n = 1}^{N}$ drawn independently from the standard normal distribution.
For convenience, the ground truth tensor was normalized to be of unit Frobenius norm.
In tensor completion experiments (Figures~\ref{tf:fig:tc_mse_ord4},~\ref{tf:fig:tc_huber_ord4}, and~\ref{tf:fig:tc_mse_ord3}), the subset of observed entries was chosen uniformly at random.
For tensor sensing (Figure~\ref{tf:fig:ts_mse_ord4}), we sampled the entries of all measurement tensors independently from a zero-mean Gaussian distribution with standard deviation $10^{-2}$ (ensures measurement tensors have expected square Frobenius norm of $1$).

\subsubsection{Tensor Rank as Measure of Complexity (Figure~\ref{tf:fig:mnist_fmnist_rank} and Table~\ref{table:mnist_fmnist_linear_errors})} \label{tf:app:experiments:details:natural_data}

For both MNIST and Fashion-MNIST datasets, we quantized pixels to hold either $0$ or~$1$ by rounding grayscale values to the nearest integer.
Random input datasets were created by replacing all pixels in all images with random values ($0$ or~$1$) drawn independently from the uniform distribution.
Random label datasets were generated by shuffling labels according to a random permutation, separately for train and test sets.

Given a prediction task, fitting the corresponding tensor completion problem with a predictor of tensor rank~$k$ (or less) was done by minimizing the mean squared error over a $k$-component tensor factorization.
Stochastic gradient descent, using the Adam optimizer~\cite{kingma2015adam} with learning rate $5 \cdot 10^{-4}$, default $\beta_1, \beta_2$ coefficients, and a batch size of $5000$, was run until the loss reached a value lower than $10^{-8}$ or $10^4$ iterations elapsed.
For numerical stability, factorization weights were initialized near one.
Namely, their initial values were sampled independently from a Gaussian distribution with mean one and standard deviation~$10^{-3}$.
To accelerate convergence, label values ($0$ or~$1$) were scaled up by two during optimization (thereby ensuring symmetry about initialization), with predictions of resulting models scaled down by the same factor during evaluation.
Results reported in Table~\ref{table:mnist_fmnist_linear_errors} were obtained using the ridge regression implementation of scikit-learn~\cite{scikit-learn} with $\alpha = 0.5$ (setting $\alpha = 0$, \ie~using unregularized linear regression, led to numerical issues due to bad conditioning of the data).
Lastly, to mitigate impact of outliers, in both Figure~\ref{tf:fig:mnist_fmnist_rank} and Table~\ref{table:mnist_fmnist_linear_errors} squared errors over test samples were clipped at one, \ie~taken to be the minimum between one and the calculated error.

\section{Deferred Proofs}
\label{tf:app:proofs}

\subsection{Notation}
\label{tf:app:proofs:notations}

For $N \in \N$, let $[N] := \{ 1, \ldots, N \}$.
We use $\inprod{ \cdot }{ \cdot}$ to denote the standard Euclidean (Frobenius) inner product between two vectors, matrices, or tensors, and $\norm{ \cdot }$ to denote the norm induced by it.
Furthermore, we denote the outer and Kronecker products by $\tenp$ and $\kronp$, respectively.
For a tensor $\W \in \R^{D_1 \times \cdots \times D_N}$ and $n \in [N]$, we let $\tfmat{ \W }_{n}$ be the mode-$n$ matricization of $\W$, \ie~its arrangement as a matrix where the rows correspond to the $n$'th mode and the columns correspond to all other modes (see Section~2.4 in~\cite{kolda2009tensor}).

\subsection{Useful Lemmas}
\label{tf:app:proofs:useful_lemmas}

\subsubsection{Technical}
\label{tf:app:proofs:useful_lemmas:technical}

Following are several technical lemmas, which are used throughout the proofs.

\begin{lemma}
	\label{tf:lem:inp_with_tenp_to_mat_kronp}
	For any $\W \in \R^{D_1 \times \cdots \times D_N}$ and $\{ \w^{n} \in \R^{D_n} \}_{n = 1}^N$, where $D_1, \ldots, D_N \in \N$, it holds that:
	\[
	\inprod{ \W }{ \tenp_{n' = 1}^N \w^{n'} } = \inprod{ \tfmat{ \W }_n \cdot \kronp_{n' \neq n} \w^{n'} }{ \w^{n} } \quad , ~n = 1, \ldots, N
	\text{\,.}
	\]
\end{lemma}

\begin{proof}
	To simplify presentation, we prove the equality for $n = 1$.
	For $n = 2, \ldots, N$, an analogous computation yields the desired result.
	By opening up the inner product and applying straightforward computations, we conclude:
	\[
	\begin{split}
		\inprod{ \W }{ \tenp_{n' = 1}^N \w^{n'} } & = \sum_{i_1 = 1}^{D_1} \ldots \sum_{i_N = 1}^{D_N} \brk1{ \W }_{i_1, \ldots, i_N}  \cdot \prod_{n' = 1}^N \brk1{ \w^{n'} }_{i_{n'}} \\
		& = \sum_{i_1 = 1}^{D_1} \brk1{ \w^{1} }_{i_1} \sum_{i_2 = 1}^{D_2} \ldots \sum_{i_N = 1}^{D_N} \brk1{ \W }_{i_1, \ldots, i_N}  \cdot \prod_{n' = 2}^N \brk1{ \w^{n'} }_{i_{n'}} \\
		& = \inprod{ \tfmat{ \W }_1 \cdot \kronp_{n' = 2}^N \w^{n'} }{ \w^{1} }
		\text{\,.}
	\end{split}
	\]
\end{proof}

\begin{lemma}
	\label{tf:lem:outer_prod_distance_bound}
	For any $\{ \abf^{n} \in \R^{D_n} \}_{n = 1}^N, \{ \bbf^{n} \in \R^{D_n} \}_{n = 1}^N$, where $D_1, \ldots, D_N \in \N$, it holds that:
	\[
	\norm*{ \tenp_{n = 1}^N \abf^{n} - \tenp_{n = 1}^N \bbf^{n} } \leq \sum_{n = 1}^N \norm*{ \abf^{n} - \bbf^{n} } \cdot \prod_{n' \neq n} \max \left \{ \normnoflex{ \abf^{n'} }, \normnoflex{ \bbf^{n'} } \right \}
	\text{\,.}
	\]
\end{lemma}

\begin{proof}
	The proof is by induction over $N \in \N$.
	For $N = 1$, the claim is trivial.
	Assuming it holds for $N - 1 \geq 1$, we show that it holds for $N$ as well:
	\[
	\begin{split}
		\norm*{ \tenp_{n = 1}^N \abf^{n} - \tenp_{n = 1}^N \bbf^{n} } & =  \norm*{ \tenp_{n = 1}^N \abf^{n} - \left ( \tenp_{n = 1}^{N - 1} \abf^{n} \right ) \tenp \bbf^{N} +  \left ( \tenp_{n = 1}^{N - 1} \abf^{n} \right ) \tenp \bbf^{N} - \tenp_{n = 1}^N \bbf^{n} } \\
		& \leq \norm*{ \abf^{N} - \bbf^{N} } \cdot \norm*{ \tenp_{n = 1}^{N - 1} \abf^{n} } + \norm*{ \tenp_{n = 1}^{N - 1} \abf^{n} - \tenp_{n = 1}^{N - 1} \bbf^{n} } \cdot \norm*{ \bbf^{N} } \\
		& \leq  \norm*{ \abf^{N} - \bbf^{N} } \cdot \prod_{n = 1}^{N - 1} \max \left \{ \norm*{ \abf^{n} }, \norm*{ \bbf^{n} } \right \} \\
		& \hspace{5mm} + \norm*{ \tenp_{n = 1}^{N - 1} \abf^{n} - \tenp_{n = 1}^{N - 1} \bbf^{n} } \cdot \max \left \{ \norm1{ \abf^{N} }, \norm1{ \bbf^{N} } \right \}
		\text{\,.}
	\end{split}
	\]
	The proof concludes by the inductive assumption for $N - 1$.
\end{proof}

\begin{lemma}
	\label{tf:lem:param_dist_to_end_to_end_dist}
	Let $B_{\norm{\cdot}}, B_{dist} > 0$ and $\{ \abf_{r}^{n} \in \R^{D_n} \}_{r = 1}^R\hspace{0mm}_{n = 1}^N, \{ \bbf_{r}^{n} \in \R^{D_n} \}_{r = 1}^R\hspace{0mm}_{n = 1}^N$, where $D_1, \ldots, D_N \in \N$, such that $\max \{ \normnoflex{ \abf_r^{n} }, \normnoflex{ \bbf_r^{n} } \}_{r = 1}^R\hspace{0mm}_{n = 1}^N \leq B_{\norm{\cdot}}$ and $( \sum_{r = 1}^R \sum_{n = 1}^N \normnoflex{ \abf_r^{n} - \bbf_r^{n} }^2 )^{1 / 2} \leq B_{dist}$.
	Then:
	\[
	\norm*{ \sum_{r = 1}^R \tenp_{n = 1}^N \abf_r^{n} - \sum_{r = 1}^R \tenp_{n = 1}^N \bbf_r^{n} } \leq \sqrt{R N} B_{\norm{\cdot}}^{N - 1} B_{dist}
	\text{\,.}
	\]
\end{lemma}

\begin{proof}
	Applying the triangle inequality and Lemma~\ref{tf:lem:outer_prod_distance_bound}, we have that:
	\[
	\begin{split}
		\norm*{ \sum_{r = 1}^R \tenp_{n = 1}^N \abf_r^{n} - \sum_{r = 1}^R \tenp_{n = 1}^N \bbf_r^{n} } & \leq \sum_{r = 1}^R \norm*{ \tenp_{n = 1}^N \abf_r^{n} - \tenp_{n = 1}^N \bbf_r^{n} } \\
		& \leq \sum_{r = 1}^R \sum_{n = 1}^N \norm*{ \abf_r^{n} - \bbf_r^{n} } \cdot \prod_{n' \neq n} \max \left \{ \normnoflex{ \abf_r^{n'} }, \normnoflex{ \bbf_r^{n'} } \right \} \\
		& \leq B_{\norm{\cdot}}^{N - 1} \sum_{r = 1}^R \sum_{n = 1}^N \norm*{ \abf_r^{n} - \bbf_r^{n} }
		\text{\,.}
	\end{split}
	\]
	The desired result readily follows from the fact that $\normnoflex{ \x }_1 \leq \sqrt{D} \cdot \normnoflex{ \x }$ for any $\x \in \R^D$:
	\[
	\begin{split}
		\norm*{ \sum_{r = 1}^R \tenp_{n = 1}^N \abf_r^{n} - \sum_{r = 1}^R \tenp_{n = 1}^N \bbf_r^{n} } 
		& \leq B_{\norm{\cdot}}^{N - 1} \sum_{r = 1}^R \sum_{n = 1}^N \norm*{ \abf_r^{n} - \bbf_r^{n} } \\
		& \leq B_{\norm{\cdot}}^{N - 1} \sqrt{R N} \left ( \sum_{r = 1}^R \sum_{n = 1}^N \norm*{ \abf_r^{n} - \bbf_r^{n} }^2 \right )^{1 / 2} \\
		& \leq \sqrt{R N} B_{\norm{\cdot}}^{N - 1} B_{dist}
		\text{\,.}
	\end{split}
	\]
\end{proof}

\begin{lemma}
	\label{tf:lem:ivp_no_sign_change}
	Let $f: [0, T_2) \to \R$ and $g: [0, T_1) \to \R$ be continuous functions, where $T_1 < T_2$.
	Suppose that $g(t)$ is bounded, $f(0) > 0$, and:
	\be
	\frac{d}{dt} f(t) = f(t)^p \cdot g(t) \quad, ~t \in [0, T_1)
	\text{\,,}
	\label{tf:eq:ivp_no_sign_change_diff_eq}
	\ee
	for $1 < p \in \R$.
	Then, $f(t) > 0$ for all $t \in [0, T_1]$.
\end{lemma}

\begin{proof}
	Consider the initial value problem induced by Equation~\eqref{tf:eq:ivp_no_sign_change_diff_eq} over the interval $[0, T_1)$, with an initial value of $f(0)$.
	One can verify by differentiation that it is solved by:
	\[
	h(t) = \left ( f (0)^{1 - p} - (p - 1) \int_{t' = 0}^t g(t') dt' \right )^{- \frac{1}{p - 1}}
	\text{\,.}
	\]
	Since the problem has a unique solution (see, \eg, Theorem~2.2 in~\cite{teschl2012ordinary}), it follows that for any $t \in [0, T_1)$:\footnote{
		A technical subtlety is that, in principle, $h( \cdot )$ may asymptote at some $\widebar{T}_1 \in [0, T_1)$.
		However, since the initial value problem has a unique solution, $f(t) = h(t)$ until that time.
		This means $h( \cdot )$ cannot asymptote before $T_1$ as that would contradict continuity of $f (\cdot)$ over $[0, T_2)$.
	}
	\[
	\begin{split}
	f(t) & = h(t) \\
	& = \left ( f (0)^{1 - p} - (p - 1) \int_{t' = 0}^t g(t') dt' \right )^{- \frac{1}{p - 1}} \\
	& \geq \left ( f (0)^{1 - p} + (p - 1) \int_{t' = 0}^{t} \abs{ g(t') } dt'  \right )^{- \frac{1}{p - 1}} 
	\text{\,.}
	\end{split}
	\]
	Recall that $g(t)$ is bounded.
	Hence, from the inequality above and continuity of $f(\cdot)$ we conclude:
	\[
	f(t) \geq \left ( f (0)^{1 - p} + (p - 1) \cdot \sup\nolimits_{t' \in [0, T_1)} \abs{g(t')} \cdot T_1 \right )^{- \frac{1}{p - 1}} > 0 ~~,~t \in [0, T_1]
	\text{\,.}
	\] 
\end{proof}

\begin{lemma}
	\label{tf:lem:gf_smooth_dist_bound}
	Let $\theta, \theta' : [0, T] \to \R^D$, where $T > 0$, be two curves born from gradient flow over a continuously differentiable function $f: \R^{D} \to \R$:
	\beas
	\theta ( 0 ) = \theta_0 \in \R^D ~ &,& ~ \tfrac{d}{dt} \theta ( t ) = - \nabla f ( \theta ( t ) ) ~~ , ~ t \in [ 0 , T ] 
	\text{\,,}
	\\ [1mm]
	\theta' ( 0 ) = \theta' _0 \in \R^D ~ &,& ~ \tfrac{d}{dt} \theta' ( t ) = - \nabla f ( \theta' ( t ) ) ~~ , ~ t \in [ 0 , T ]
	\text{\,.}
	\eeas
	Let $B > 0$, and suppose that $f( \cdot )$ is $\beta$-smooth over $\D_{B+ 1}$ for some $\beta \geq 0$,\footnote{
		That is, for any $\theta_1, \theta_2 \in \D_{B + 1}$ it holds that $\norm{ \nabla f (\theta_1) - \nabla f (\theta_2) } \leq \beta \cdot \norm{ \theta_1 - \theta_2}$.
	} where $\D_{B + 1} := \{ \theta \in \R^d : \normnoflex{ \theta } \leq B + 1 \}$.
	Then, if $\normnoflex{ \theta (0) - \theta' (0) } < \exp ( -\beta \cdot T )$, it holds that:
	\be
	\norm{ \theta (t) - \theta' (t) } \leq \norm{ \theta (0) - \theta' (0) } \cdot \exp \left ( \beta \cdot t \right )
	\text{\,}
	\label{tf:eq:gf_smooth_dist_bound}
	\ee
	at least until $t \geq T$ or $\normnoflex{ \theta' (t) } \geq B$.
	That is, Equation~\eqref{tf:eq:gf_smooth_dist_bound} holds for all $t \in [0, \min \{ T, T_B \}]$, where $T_B := \inf \{ t \geq 0 : \normnoflex{ \theta' (t) } \geq B \}$.
	
\end{lemma}

\begin{proof}
	If $\normnoflex{  \theta' (0) } \geq B$, the claim trivially holds.
	Suppose $\normnoflex{ \theta' (0) } < B$, and notice that in this case $\normnoflex{ \theta (0) } < \normnoflex{ \theta' (0) } + \exp ( - \beta \cdot T) < B + 1$.
	We examine the initial time at which $\normnoflex{  \theta' (t) } \geq B$ or $\normnoflex{ \theta (t) } \geq B + 1$.
	That is, let:
	\[
	\widebar{T}_{B } := \inf \left \{ t \in [0, T ] :  \normnoflex{  \theta' (t) } \geq B \text{ or }  \normnoflex{ \theta (t) } \geq B + 1 \right \}
	\text{\,,}
	\]
	where we take $\widebar{T}_{B} := T$ if the set is empty.
	Since both $\norm{  \theta' (t) }$ and $\norm{ \theta (t) }$ are continuous in $t$, it must be that $\widebar{T} > 0$.
	Furthermore, $\norm{  \theta' (t) } \leq B$ and $\norm{ \theta (t) } \leq B + 1$ for all $t \in [0, \widebar{T}_{B}]$.
	
	Now, define the function $g: [0, T] \to \R_{\geq 0}$ by $g(t) := \norm{ \theta(t) -  \theta' (t) }^2$.
	For any $t \in [0, \widebar{T}_{B}]$ it holds that:
	\[
	\begin{split}
		\frac{d}{dt} g(t) & = 2 \inprod{ \theta (t) -  \theta' (t) }{ \tfrac{d}{dt} \theta (t) - \tfrac{d}{dt}  \theta' (t) } \\
		& = -2 \inprod{ \theta (t) -  \theta' (t) }{ \nabla f ( \theta (t) ) - \nabla f ( \theta' (t) )  }
		\text{\,.}
	\end{split}
	\]
	By the Cauchy-Schwarz inequality and $\beta$-smoothness of $f (\cdot)$ over $\D_{B + 1}$ we have:
	\be
	\begin{split}
		\frac{d}{dt} g(t) & \leq 2 \beta \cdot \norm{ \theta (t) -  \theta' (t) }^2 = 2 \beta \cdot g(t) 
		\text{\,.}
	\end{split}
	\label{tf:eq:g_time_deriv_smoothness_bound}
	\ee
	Thus, Gronwall's inequality leads to $g(t) \leq g(0) \cdot \exp ( 2 \beta \cdot t )$.
	Taking the square root of both sides then establishes Equation~\eqref{tf:eq:gf_smooth_dist_bound} for all $t \in [0 , \widebar{T}_{B}]$.
	
	If $\widebar{T}_B = T$, the proof concludes since Equation~\eqref{tf:eq:gf_smooth_dist_bound} holds over $[0, T]$.
	Otherwise, if $\widebar{T}_B < T$, then either $\normnoflex{ \theta' (\widebar{T}_B ) } = B$ or $\normnoflex{ \theta ( \widebar{T}_B ) } = B + 1$.
	It suffices to show that in both cases $T_B \leq \widebar{T}_B$.
	In case $\normnoflex{  \theta' ( \widebar{T}_B ) } = B$, the definition of $T_B$ implies $T_B = \widebar{T}_B$.
	On the other hand, suppose $\normnoflex{ \theta ( \widebar{T}_B) } = B+ 1$.
	Since $\norm{ \theta (0) -  \theta' (0) } < \exp ( -\beta \cdot T )$, the fact that Equation~\eqref{tf:eq:gf_smooth_dist_bound} holds for $\widebar{T}_B$ gives $\norm{ \theta (\widebar{T}_B) -  \theta' (\widebar{T}_B) } \leq 1$.
	Therefore, it must be that $\norm{  \theta' ( \widebar{T}_B ) } \geq B$, and so $T_B \leq \widebar{T}_B$, completing the proof.
\end{proof}

\subsubsection{Tensor Factorization}
\label{tf:app:proofs:useful_lemmas:tf}

Suppose that we minimize the objective $\tfobj(\cdot)$ (Equations~\eqref{tf:eq:cp_objective} and~\eqref{tf:eq:end_tensor}) via gradient flow over an $R$-component tensor factorization (Equation~\eqref{tf:eq:cp_gf}), where we allow the loss $\tfendloss (\cdot)$ in Equation~\eqref{tf:eq:cp_objective} to be any differentiable and locally smooth function.
Under this setting, the following lemmas establish several results which will be of use when proving the main theorems.

\begin{lemma}
	\label{tf:lem:cp_gradient}
	For any $\{ \w_{r}^{n}  \in \R^{D_n} \}_{r = 1}^R\hspace{0mm}_{n = 1}^N$:
	\[
	\frac{\partial}{\partial \w_r^{n}} \tfobj \left (\{ \w_{r'}^{n'} \}_{r' = 1}^R\hspace{0mm}_{n' = 1}^N \right ) = \tfmat{ \nabla \tfendloss \left ( \tftensorend \right ) }_{n} \cdot \kronp_{n' \neq n} \w_r^{n'} \quad , ~r = 1, \ldots, R ~ , ~ n = 1, \ldots, N
	\text{\,,}
	\]
	where $\tftensorend$ denotes the end tensor (Equation~\eqref{tf:eq:end_tensor}) induced by $\{ \w_{r}^{n} \}_{r = 1}^R\hspace{0mm}_{n = 1}^N$.
\end{lemma}

\begin{proof}
	For $r \in [R], n \in [N]$, we treat $\{ \w_{r'}^{n'} \}_{(r', n') \neq (r, n)}$ as fixed, and with slight abuse of notation consider:
	\[
	\phi_{r, n} \left ( \w_r^{n} \right ) := \tfobj \left (\{ \w_{r'}^{n'} \}_{r' = 1}^R\hspace{0mm}_{n' = 1}^N \right )
	\text{\,.}
	\]
	For $\Delta \in \R^{D_n}$, from the first order Taylor approximation of $\tfendloss (\cdot)$ we have that:
	\[
	\begin{split}
		\phi_{r, n} \left ( \w_r^{n} + \Delta \right ) & = \tfendloss \left ( \tftensorend + \left ( \tenp_{n' = 1}^{n - 1} \w_r^{n'} \right ) \tenp \Delta \tenp \left ( \tenp_{n' = n + 1}^N \w_r^{n'} \right ) \right ) \\
		& = \tfendloss \left (  \tftensorend \right ) + \inprod{ \nabla \tfendloss \left ( \tftensorend \right ) }{ \left ( \tenp_{n' = 1}^{n - 1} \w_r^{n'} \right ) \tenp \Delta \tenp \left ( \tenp_{n' = n + 1}^N \w_r^{n'} \right )  } \\
		& \hspace{5mm} + o \left ( \norm{ \Delta } \right )
		\text{\,.}
	\end{split}
	\]
	Since $\tfendloss \left ( \tftensorend \right ) = \phi_{r, n} ( \w_r^{n} )$, by applying Lemma~\ref{tf:lem:inp_with_tenp_to_mat_kronp} we arrive at:
	\[
	\phi_{r, n} \left ( \w_r^{n} + \Delta \right ) = \phi_{r, n} \left ( \w_r^{n} \right ) + \inprod{ \tfmat{ \nabla \tfendloss ( \tftensorend ) }_{n} \cdot \kronp_{n' \neq n} \w_r^{n'} }{ \Delta } + o(\norm{\Delta})
	\text{\,.}
	\]
	Uniqueness of the linear approximation of $\phi_{r, n} ( \cdot )$ at $\w_r^{n}$ then implies:
	\[
	\frac{\partial}{\partial \w_r^{n}} \tfobj \left ( \{ \w_{r'}^{n'} \}_{r' = 1}^R\hspace{0mm}_{n' = 1}^N \right ) = \frac{d}{d \w_r^{n}} \phi_{r, n} \left ( \w_r^{n} \right ) = \tfmat{ \nabla \tfendloss \left ( \tftensorend \right ) }_{n} \cdot \kronp_{n' \neq n} \w_r^{n'}
	\text{\,.}
	\]
\end{proof}

\begin{lemma}
	\label{tf:lem:dyn_parameter_vector_sq_norm}
	For any $r \in [R]$ and $n \in [N]$:
	\[
	\frac{d}{dt}  \normnoflex{ \w_r^{n} (t) }^2 = - 2 \inprod{ \nabla \tfendloss \left ( \tftensorend (t) \right ) }{ \tenp_{n' = 1}^N \w_r^{n'} (t) } 
	\text{\,.}
	\]
\end{lemma}

\begin{proof}
	Fix $r \in [R]$ and $n \in [N]$.
	Differentiating $\normnoflex{ \w_r^{n} (t) }^2$ with respect to time, we have:
	\[
	\frac{d}{dt} \normnoflex{ \w_r^{n} (t) }^2 = 2 \inprod{  \w_r^{n} (t) }{ \tfrac{d}{dt} \w_r^{n} (t) } = - 2 \inprod{  \w_r^{n} (t) }{ \frac{\partial}{\partial \w_r^{n}} \tfobj \left (\{ \w_{r'}^{n'} (t) \}_{r' = 1}^R\hspace{0mm}_{n' = 1}^N \right ) }
	\text{\,.}
	\]
	Applying Lemmas~\ref{tf:lem:cp_gradient} and~\ref{tf:lem:inp_with_tenp_to_mat_kronp} completes the proof.
\end{proof}

\begin{lemma}[Lemma~\ref{tf:lem:balancedness_conservation_body} restated]
	\label{tf:lem:balancedness_conservation}
	For all $r \in [R]$ and $n, \bar{n} \in [N]$:
	\[
	\norm{ \w_r^{n} (t) }^2 - \norm{ \w_r^{\bar{n}} (t) }^2 = \norm{ \w_r^{n} (0) }^2 - \norm{ \w_r^{\bar{n}} (0) }^2 \quad , ~t \geq 0
	\text{\,.}
	\]
\end{lemma}

\begin{proof}[Proof of Lemma~\ref{tf:lem:balancedness_conservation}]
	For any $r \in [R]$ and $n, \bar{n} \in [N]$, by Lemma~\ref{tf:lem:dyn_parameter_vector_sq_norm} it holds that:
	\[
	\frac{d}{dt}  \normnoflex{ \w_r^{n} (t) }^2 = - 2 \inprod{ \nabla \tfendloss \left ( \tftensorend (t) \right ) }{ \tenp_{n' = 1}^N \w_r^{n'} (t) }  = \frac{d}{dt}  \normnoflex{ \w_r^{\bar{n}} (t) }^2
	\text{\,.}
	\]
	Integrating both sides with respect to time gives:
	\[
	\norm{ \w_r^{n} (t) }^2 - \norm{ \w_r^{n} (0) }^2 = \norm{ \w_r^{\bar{n}} (t) }^2 - \norm{ \w_r^{\bar{n}} (0) }^2
	\text{\,.}
	\]
	Rearranging the equality above establishes the desired result.
\end{proof}

\begin{lemma}
	\label{tf:lem:width_R_equivalent_to_larger_width_with_zero_init}
	Let $\widetilde{R} > R$, and define:
	\be
	\widetilde{ \w }_r^n (t) := \begin{cases}
		\w_r^n (t)	& , r \in \{ 1, \ldots, R \} \\
		0 \in \R^{D_n}	& , r \in \{ R + 1, \ldots, \widetilde{R} \}
	\end{cases}
	\quad ,~t \geq 0 ~,~n = 1, \ldots, N
	\text{\,.}
	\label{tf:eq:width_R_equivalent_to_larger_width_with_zero_init}
	\ee
	Then, $\{ \widetilde{ \w }_r^n (t) \}_{r = 1}^{\widetilde{R}}\hspace{0mm}_{n = 1}^N$ follow a gradient flow path of an $\widetilde{R}$-component factorization.
\end{lemma}

\begin{proof}
	We verify that $\{ \widetilde{ \w }_r^n (t) \}_{r = 1}^{\widetilde{R}}\hspace{0mm}_{n = 1}^N$ satisfy the differential equations governing gradient flow.
	Fix $n \in [N]$.
	For any $r \in [R]$ and $t \geq 0$ we have:
	\[
	\begin{split}
		\frac{d}{dt} \widetilde{\w}_{r}^n (t) & = \frac{d}{dt} \w_{r}^n (t) = - \frac{\partial}{ \partial \w_{r}^{n} } \tfobj \left (\{ \w_{r'}^{n'} (t) \}_{r' = 1}^{ R }\hspace{0mm}_{n' = 1}^N \right )
		\text{\,.}
	\end{split}
	\]
	Noticing that $\tftensorend (t) =\sum_{r' = 1}^{R} \tenp_{n' = 1}^N \w_{r'}^{n'} (t) = \sum_{r' = 1}^{\widetilde{R}} \tenp_{n' = 1}^N \widetilde{\w}_{r'}^{n'} (t) := \widetilde{ \W }_{\mathrm{T}} (t)$, and invoking Lemma~\ref{tf:lem:cp_gradient}, we may write:
	\[
	\begin{split}
		\frac{d}{dt} \widetilde{\w}_{r}^n (t) & = - \tfmatflex{ \nabla \tfendloss \left ( \tftensorend (t) \right ) }_{n} \cdot \kronp_{n' \neq n} \w_{r}^{n'} (t) \\
		& = - \tfmatflex{ \nabla \tfendloss \left ( \widetilde{\W}_\mathrm{T} (t) \right ) }_{n} \cdot \kronp_{n' \neq n} \widetilde{\w}_{r}^{n'} (t)  \\
		& = - \frac{\partial}{ \partial \widetilde{\w}_{r}^{n} } \tfobj \left (\{ \widetilde{\w}_{r'}^{n'} (t) \}_{r' = 1}^{ \widetilde{R} }\hspace{0mm}_{n' = 1}^N \right )
		\text{\,.}
	\end{split}
	\]
	On the other hand, for any $r \in \{ R + 1, \ldots, \widetilde{R} \}$, recalling that $\widetilde{\w}_r^n (t)$ is identically zero:
	\[
	\begin{split}
		\frac{d}{dt} \widetilde{\w}_{r}^n (t) = 0 = - \tfmatflex{ \nabla \tfendloss \left ( \widetilde{\W}_\mathrm{T} (t) \right ) }_{n} \cdot \kronp_{n' \neq n} \widetilde{\w}_{r}^{n'} (t) = - \frac{\partial}{ \partial \widetilde{\w}_{r}^{n} } \tfobj \left (\{ \widetilde{\w}_{r'}^{n'} (t) \}_{r' = 1}^{ \widetilde{R} }\hspace{0mm}_{n' = 1}^N \right )
		\text{\,,}
	\end{split}
	\]
	for all $t \geq 0$, completing the proof.
\end{proof}

\begin{lemma}
	\label{tf:lem:balanced_param_vector_norm_no_sign_change}
	For any $r \in [R]$:
	\begin{itemize}
		\item If $\normnoflex{ \w_r^{1} (0) } = \cdots = \normnoflex{ \w_r^{N} (0) } = 0$, then:
		\be
		\norm{ \w_r^{1} (t) } = \cdots = \norm{ \w_r^{N} (t) } = 0 \quad ,~t \geq 0
		\text{\,.}
		\label{tf:eq:balanced_param_vector_norm_no_sign_change_stay_zero}
		\ee
		\item On the other hand, if $\normnoflex{ \w_r^{1} (0) } = \cdots = \normnoflex{ \w_r^{N} (0) } > 0$, then:
		\be
		\norm{ \w_r^{1} (t) } = \cdots = \norm{ \w_r^{N} (t) } > 0 \quad,~t \geq 0
		\text{\,.}
		\label{tf:eq:balanced_param_vector_norm_no_sign_change_stay_nonzero}
		\ee
	\end{itemize}
\end{lemma}

\begin{proof}
	The proof is divided into two separate parts, establishing Equations~\eqref{tf:eq:balanced_param_vector_norm_no_sign_change_stay_zero} and~\eqref{tf:eq:balanced_param_vector_norm_no_sign_change_stay_nonzero} under their respective conditions.
	
	\paragraph*{Proof of Equation~\eqref{tf:eq:balanced_param_vector_norm_no_sign_change_stay_zero}  (if $\normnoflex{ \w_r^{1} (0) } = \cdots = \normnoflex{ \w_r^{N} (0) } = 0$):}
	
	To simplify presentation, we assume without loss of generality that $r = R$.
	Consider the following initial value problem induced by gradient flow over $\tfobj (\cdot)$:
	\be
	\begin{split}
		& \widetilde{\w}_{\bar{r}}^{n} ( 0 ) = \w_{\bar{r}}^{n} ( 0 ) \quad ,  ~ \bar{r} = 1, \ldots, R ~,~ n = 1, \ldots, N \text{\,,}  \\
		& \frac{d}{dt} \widetilde{\w}_{\bar{r}}^{n} (t) = - \frac{\partial}{\partial \widetilde{\w}_{\bar{r}}^{n}} \tfobj \left (\{ \widetilde{\w}_{r'}^{n'} (t) \}_{r' = 1}^{R}\hspace{0mm}_{n' = 1}^N \right ) 	\quad , ~t \geq 0 ~,~ \bar{r} = 1, \ldots, R ~,~ n = 1, \ldots, N  \text{\,.}
	\end{split}
	\label{tf:eq:gf_w_tilde_ivp}
	\ee
	By definition, $\{\w_{\bar{r}}^{n} (t) \}_{\bar{r} = 1}^{R}\hspace{0mm}_{n = 1}^N$ is a solution to the initial value problem above.
	Since it has a unique solution (see, \eg, Theorem~2.2 in~\cite{teschl2012ordinary}), we need only show that there exist $\{ \widetilde{\w}_{\bar{r}}^{n} (t) \}_{\bar{r} = 1}^{R}\hspace{0mm}_{n = 1}^N$ satisfying Equation~\eqref{tf:eq:gf_w_tilde_ivp} such that $\widetilde{\w}_R^{1} (t) = \cdots = \widetilde{\w}_R^{N} (t) = 0$ for all $t \geq 0$.
	
	If $R = 1$, \ie~the factorization consists of a single component, by Lemma~\ref{tf:lem:cp_gradient}:
	\[
	- \frac{\partial}{\partial \widetilde{\w}_1^{n}} \tfobj \left (\{ \widetilde{\w}_{1}^{n'} \}_{n' = 1}^N \right ) = - \tfmatflex{ \nabla \tfendloss \left ( \tenp_{n' = 1}^{N} \widetilde{\w}_1^{n'} \right ) }_{n} \cdot \kronp_{n' \neq n} \widetilde{\w}_1^{n'} \quad , ~n = 1 ,\ldots, N
	\text{\,,}
	\]
	for any $\widetilde{\w}_1^{1} \in \R^{D_1}, \ldots, \widetilde{\w}_1^{N} \in \R^{D_N}$.
	Hence, $\widetilde{\w}_1^{1} (t) = \cdots = \widetilde{\w}_1^{N} (t) = 0$ for all $t \geq 0$ form a solution to the initial value problem in Equation~\eqref{tf:eq:gf_w_tilde_ivp}.
	To see it is so, notice that the initial conditions are met, and:
	\[
	\frac{d}{dt} \widetilde{\w}_1^{n} (t) = 0 = - \frac{\partial}{\partial \widetilde{\w}_1^{n}} \tfobj \left (\{ \widetilde{\w}_{1}^{n'} (t) \}_{n' = 1}^N \right ) \quad , ~t \geq 0 ~,~ n = 1 ,\ldots, N
	\text{\,.}
	\]
	
	If $R > 1$, with slight abuse of notation we let $\tfobj ( \{ \widetilde{\w}_{\bar{r}}^{n} \}_{\bar{r} = 1}^{R - 1}\hspace{0mm}_{n = 1}^N ) := \tfendloss ( \sum_{\bar{r} = 1}^{R - 1} \tenp_{n = 1}^N \widetilde{\w}_{\bar{r}}^{n} )$ be the objective over an $(R - 1)$-component tensor factorization.
	Let $\{ \widetilde{\w}_{\bar{r}}^{n} (t) \}_{\bar{r} = 1}^{R - 1}\hspace{0mm}_{n = 1}^N$ be curves obtained by running gradient flow on this objective, initialized such that:
	\[
	\widetilde{\w}_{\bar{r}}^{n} (0) := \w_{\bar{r}}^{n} (0) \quad , ~\bar{r} = 1, \ldots, R - 1 ~,~ n = 1, \ldots, N
	\text{\,.}
	\]
	Additionally, define $\widetilde{\w}_{R}^{1} (t) = \cdots = \widetilde{\w}_{R}^{N} (t) = 0$ for all $t \geq 0$. 
	According to Lemma~\ref{tf:lem:width_R_equivalent_to_larger_width_with_zero_init},  $\{ \widetilde{\w}_{\bar{r}}^{n} (t) \}_{\bar{r} = 1}^{R}\hspace{0mm}_{n = 1}^N$ form a valid solution to the original gradient flow over an $R$-component factorization, \ie~satisfy Equation~\eqref{tf:eq:gf_w_tilde_ivp}.
	Thus, uniqueness of the solution implies $\w_R^{1} (t) = \cdots = \w_R^{N} (t) = 0$ for all $t \geq 0$, completing the proof for Equation~\eqref{tf:eq:balanced_param_vector_norm_no_sign_change_stay_zero}.

	\paragraph*{Proof of Equation~\eqref{tf:eq:balanced_param_vector_norm_no_sign_change_stay_nonzero}  (if $\normnoflex{ \w_r^{1} (0) } = \cdots = \normnoflex{ \w_r^{N} (0) } > 0$):}
	
	From Lemma~\ref{tf:lem:balancedness_conservation_body} it follows that $\normnoflex{ \w_r^{1} ( t ) } = \cdots = \normnoflex{ \w_r^{N} ( t ) }$ for any $t \geq 0$.
	Hence, it suffices to show that $\normnoflex{ \w_r^{1} (t) }$ stays positive.
	Assume by way of contradiction that there exists $\bar{t} > 0$ for which $\normnoflex{\w_r^{1} (\bar{t}\,)} = 0$.
	Define:
	\[
	t_0 := \inf \left \{ t \geq 0: \normnoflex{ \w_r^{1} ( t ) } = 0 \right \}
	\text{\,,}
	\]
	the initial time at which $\normnoflex{  \w_r^{1} (t) }$ meets zero.
	Due to the fact that $\normnoflex{ \w_r^{1} (t) }$ is continuous in $t$, $\normnoflex{ \w_r^{1} (t_0) } = 0$ and $t_0 > 0$.
	Furthermore, $\normnoflex{ \w_r^{1} (t) } > 0$ for all $t \in [0, t_0)$.
	We may therefore differentiate $\normnoflex{\w_r^{1} (t) }$ with respect to time over the interval $[0, t_0)$ as follows:
	\[
	\begin{split}
		\frac{d}{dt} \norm{ \w_r^{1} (t) } & =  \left ( \tfrac{d}{dt} \norm{ \w_r^{1} (t) }^2 \right ) \cdot 2^{-1}  \norm{ \w_r^{1} (t) }^{-1} \\
		& = \norm{ \w_r^{1} (t) }^{-1}  \inprod{ - \nabla \tfendloss \left ( \tftensorend (t) \right ) }{ \tenp_{n = 1}^N \w_r^{n} (t) } \\
		& = \norm{ \w_r^{1} (t) }^{N - 1}  \inprod{ - \nabla \tfendloss \left ( \tftensorend (t) \right ) }{ \tenp_{n = 1}^N \widehat{\w}_r^{n} (t) }
		\text{\,,}
	\end{split}
	\]
	where the second transition is due to Lemma~\ref{tf:lem:dyn_parameter_vector_sq_norm}, and $\widehat{\w}_r^{n} (t) := \w_r^{n} (t) / \normnoflex{ \w_r^{n} (t) }$ for $n = 1, \ldots, N$.
	Define $g(t) := \inprodnoflex{ - \nabla \tfendloss \left ( \tftensorend (t) \right ) }{ \tenp_{n = 1}^N \widehat{\w}_r^{n} (t) }$.
	Since $\nabla \tfendloss ( \tftensorend (t) )$ is continuous with respect to time, $g(t)$ is bounded over $[0, t_0]$ and continuous over $[0, t_0)$.
	Thus, invoking Lemma~\ref{tf:lem:ivp_no_sign_change} with $g(t)$, $T_1 := t_0$ and $f(t) := \normnoflex{ \w_r^1 (t) }$, we get that $\normnoflex{ \w_r^{1} (t) } > 0$ for all $t \in [0, t_0]$, in contradiction to $\normnoflex{ \w_r^{1} (t_0) } = 0$.
	This means that $\normnoflex{ \w_r^{1} (t) } > 0$ for all $t \geq 0$, concluding the proof for Equation~\eqref{tf:eq:balanced_param_vector_norm_no_sign_change_stay_nonzero}.
\end{proof}

\subsection{Proof of Theorem~\ref{tf:thm:dyn_fac_comp_norm_unbal}}
\label{tf:app:proofs:dyn_fac_comp_norm_unbal}

Fix $r \in [R]$ and $t \geq 0$.
Since $\normnoflex{ \tenp_{n = 1}^N \w_r^{n} (t) } = \prod_{n = 1}^N \normnoflex { \w_r^{n} (t) }$, the product rule gives:
\[
\frac{d}{dt} \norm{ \tenp_{n = 1}^N \w_r^{n} (t) } = \sum_{n = 1}^N \frac{d}{dt} \norm{ \w_r^{n} (t) } \cdot \prod_{n' \neq n} \normnoflex{ \w_r^{n'} (t) }
\text{\,.}
\]
Notice that for any $n \in [N]$ we have $\normnoflex{ \w_r^{n} (t) } > 0$, as otherwise $\normnoflex{ \tenp_{n' = 1}^N \w_r^{n'} (t) }$ must be zero.
Thus, applying Lemma~\ref{tf:lem:dyn_parameter_vector_sq_norm} we get $\frac{d}{dt} \normnoflex{ \w_r^{n} (t) } = \frac{1}{2} \normnoflex{ \w_r^{n} (t) }^{-1} \frac{d}{dt} \normnoflex{ \w_r^n (t) }^2 = \normnoflex{ \w_r^{n} (t) }^{-1} \inprodnoflex{ - \nabla \tfendloss \left ( \tftensorend (t) \right ) }{ \tenp_{n' = 1}^N \w_r^{n'} (t) }$.
Combined with the equation above, we arrive at:
\be
\begin{split}
	\frac{d}{dt} \norm{ \tenp_{n = 1}^N \w_r^{n} (t) }  & = \sum_{n = 1}^N \norm{ \w_r^{n} (t) }^{-1} \inprod{ - \nabla \tfendloss \left ( \tftensorend (t) \right ) }{ \tenp_{n' = 1}^N \w_r^{n'} (t) }   \cdot \prod_{n' \neq n} \normnoflex{ \w_r^{n'} (t) } \\
	& =  \inprod{ - \nabla \tfendloss \left ( \tftensorend (t) \right ) }{ \tenp_{n' = 1}^N \widehat{\w}_r^{n'} (t) }  \cdot \sum_{n = 1}^N \prod_{n' \neq n} \normnoflex{ \w_r^{n'} (t) }^2 
	\text{\,.}
\end{split}
\label{tf:eq:unbal_comp_time_derive_intermid}
\ee
By Lemma~\ref{tf:lem:balancedness_conservation_body}, the differences between squared norms of vectors in the same component are constant through time.
In particular, the unbalancedness magnitude (Definition~\ref{def:unbalancedness_magnitude}) is conserved during gradient flow, implying that for any $n \in [N]$:
\be
\norm{ \w_r^n (t) }^2 \leq \min_{n' \in [N]} \normnoflex{ \w_r^{n'} (t) }^2 + \epsilon \leq \norm{ \tenp_{n' = 1}^N \w_r^{n'} (t) }^{\frac{2}{N}} + \epsilon
\text{\,.}
\label{tf:eq:sq_norm_unbal_upper_bound}
\ee
Now, suppose that $\gamma_r (t) := \inprodnoflex{ - \nabla \tfendloss ( \tftensorend (t) ) }{ \tenp_{n = 1}^N \widehat{\w}_r^{n} (t) } \geq 0$. 
Going back to Equation~\eqref{tf:eq:unbal_comp_time_derive_intermid}, applying the inequality in Equation~\eqref{tf:eq:sq_norm_unbal_upper_bound} for each $\normnoflex{ \w_r^{n'} (t) }^2$ yields the desired upper bound from Equation~\eqref{tf:eq:dyn_fac_comp_norm_unbal_pos}.
On the other hand, multiplying and dividing each summand in Equation~\eqref{tf:eq:unbal_comp_time_derive_intermid} by the corresponding $\normnoflex{ \w_r^n (t) }^2$, we may equivalently write:
\[
\begin{split}
	\frac{d}{dt} \norm{ \tenp_{n = 1}^N \w_r^{n} (t) } & =  \inprod{ - \nabla \tfendloss \left ( \tftensorend (t) \right ) }{ \tenp_{n' = 1}^N \widehat{\w}_r^{n'} (t) }  \cdot \sum_{n = 1}^N \normnoflex{ \w_r^n (t) }^{-2} \prod_{n' = 1}^N \normnoflex{ \w_r^{n'} (t) }^2 \\
	& = \inprod{ - \nabla \tfendloss \left ( \tftensorend (t) \right ) }{ \tenp_{n' = 1}^N \widehat{\w}_r^{n'} (t) } \norm{ \tenp_{n = 1}^N \w_r^n (t) }^2 \cdot \sum_{n = 1}^N \normnoflex{ \w_r^n (t) }^{-2}
	\text{\,.}
\end{split}
\]
Noticing that Equation~\eqref{tf:eq:sq_norm_unbal_upper_bound} implies $\normnoflex{ \w_r^{n} (t) }^{-2} \geq ( \normnoflex{ \tenp_{n' = 1}^N \w_r^{n'} (t) }^{\frac{2}{N}} + \epsilon )^{-1}$, the lower bound from Equation~\eqref{tf:eq:dyn_fac_comp_norm_unbal_pos} readily follows.

If $\gamma_r (t) < 0$, Equation~\eqref{tf:eq:dyn_fac_comp_norm_unbal_neg} is established by following the same computations, up to differences in the direction of inequalities due to the negativity of $\gamma_r (t)$.
\qed

\subsection{Proof of Corollary~\ref{tf:cor:dyn_fac_comp_norm_balanced}}
\label{tf:app:proofs:dyn_fac_comp_norm_balanced}

Fix $r \in [R]$ and $t \geq 0$.
The lower and upper bounds in Theorem~\ref{tf:thm:dyn_fac_comp_norm_unbal} are equal to $\smash{ N \gamma_r (t) \cdot \norm{ \tenp_{n = 1}^N \w_r^{n} (t) }^{2 - 2 / N} }$ for unbalancedness magnitude $\epsilon = 0$.
Therefore, if $\normnoflex{ \tenp_{n = 1}^N \w_r^{n} (t) } > 0$, Equation~\eqref{tf:eq:dyn_fac_comp_norm} immediately follows from Theorem~\ref{tf:thm:dyn_fac_comp_norm_unbal}.

If $\normnoflex{ \tenp_{n = 1}^N \w_r^{n} (t) } = 0$, we claim that necessarily $\normnoflex{ \tenp_{n = 1}^N \w_r^{n} (t') }  = 0$ for all $t' \geq 0$, in which case both sides of Equation~\eqref{tf:eq:dyn_fac_comp_norm} are zero.
Indeed, since the unbalancedness magnitude is zero at initialization and $\normnoflex{ \tenp_{n = 1}^N \w_r^{n} (t) } = \prod_{n = 1}^N \normnoflex { \w_r^{n} (t) }$, by Lemma~\ref{tf:lem:balanced_param_vector_norm_no_sign_change} we know that either $\normnoflex{ \tenp_{n = 1}^N \w_r^{n} (t') } = 0$  for all $t' \geq 0$, or $\normnoflex{ \tenp_{n = 1}^N \w_r^{n} (t') } > 0$ for all $t' \geq 0$.
Hence, given that $\normnoflex{ \tenp_{n = 1}^N \w_r^{n} (t) } = 0$, the norm of the component must be identically zero through time.
\qed

\subsection{Proof of Theorem~\ref{tf:thm:approx_rank_1}}
\label{tf:app:proofs:approx_rank_1}

For conciseness, we consider the case where the number of components $R \geq 2$.
For $R = 1$, existence of a time $T_0 > 0$ at which $\tftensorend (T_0) \in \S$ follows by analogous steps, disregarding parts pertaining to factorization components $2, \ldots, R$.
Furthermore, proximity to a balanced rank one trajectory becomes trivial as, by Assumption~\ref{tf:assump:a_balance} and Lemma~\ref{tf:lem:balancedness_conservation_body}, $\tftensorend (t)$ is in itself such a trajectory.

Assume without loss of generality that Assumption~\ref{tf:assump:a_lead_comp} holds for $\bar{r} = 1$.

Before delving into the proof details, let us introduce some notation and specify the exact requirement on the initialization scale $\alpha$.
We let $\L_h : \R^{D_1 \times \cdots \times D_N} \to \R_{\geq 0}$ be the tensor completion objective induced by the Huber loss (Equation~\eqref{tf:eq:tc_loss} with $\ell_h ( \cdot )$ in place of $\ell (\cdot)$), and $\phi_h ( \cdot )$ be the corresponding tensor factorization objective (Equation~\eqref{tf:eq:cp_objective} with $\L_h (\cdot)$ in place of $\tfendloss (\cdot)$). 
For reference sphere radius $\rho \in ( 0 , \min_{(i_1, \ldots, i_N) \in \Omega} \abs{ y_{i_1, \ldots, i_N} } - \delta_h )$, distance from origin $B > 0$, time duration $T > 0$, and degree of approximation $\epsilon \in (0, 1)$, let:
\be
\begin{split}
	& \normnoflex{ \abf_r } := \normnoflex{ \abf_r^1} = \cdots = \normnoflex{ \abf_r^N } \quad , ~r = 1, \ldots, R
	\text{\,,} \\
	&  A := \max\nolimits_{r \in [R]} \norm{ \abf_r } \text{\,,} \\ 
	& A_{-1} := \max\nolimits_{r \in \{ 2, \ldots, R\}} \norm{ \abf_r } \text{\,,} \\
	& \widetilde{B} := \sqrt{N} \left ( \max \{ B, \rho \} + 1 \right )^{\frac{1}{N} } \text{\,,} \\
	& \beta := R N \left ( (\widetilde{B} + 1)^{2(N - 1)} + \delta_h (\widetilde{B} + 1)^{N - 2} \right ) \text{\,,} \\
	& \hat{\epsilon} <  \min \left \{ 2^{- \frac{N}{2}} R^{-N} N^{-N} (\widetilde{B} + 1)^{N - N^2} \cdot \exp ( -N \beta T) \cdot \epsilon^N ~,~ \rho (R - 1)^{-1} \right \} \text{\,,} \\
	& \tilde{\epsilon} := \min \left \{ \hat{\epsilon} ~,~ ( R - 1 )^{-1} \left ( \rho - \left [ \rho^{\frac{1}{N}} - (R - 1)^{\frac{1}{N}} \cdot \hat{\epsilon}^{\frac{1}{N}} \right ]^{N} \right ) \right \}
	\text{\,.}
\end{split}
\label{tf:eq:approx_rank_1_consts}
\ee
With the constants above in place, for the results of the theorem to hold it suffices to require that:
\be
\alpha < \min \left \{  R^{- \frac{1}{N}} A^{-1} \rho^{\frac{1}{N}} ~,~ \left ( A_{-1}^{2 - N} - \norm{ \abf_1 }^{2 - N}  \tfrac{ \norm{ \nabla \tfendloss (0) } }{ \inprod{ -\nabla \tfendloss (0) }{ \tenp_{n = 1}^N \widehat{\abf}_{1}^n } } \right )^{\frac{1}{N - 2}} \cdot \tilde{\epsilon}^{\frac{1}{N}} \right \}
\text{\,.}
\label{tf:eq:alpha_size_requirement}
\ee

\medskip

The proof is sectioned into three parts.
We begin with several preliminary lemmas in Appendix~\ref{tf:app:proofs:approx_rank_1:prelim_lemmas}.
Then, Appendix~\ref{tf:app:proofs:approx_rank_1:const_grad} establishes the existence of a time $T_0 > 0$ at which $ \tftensorend ( t )$ initially reaches the reference sphere $\S$, \ie~$\normnoflex{ \tftensorend (T_0) } = \rho$, while $ \normnoflex{ \tenp_{n = 1}^N \w_2^n (T_0) }, \ldots, \normnoflex{ \tenp_{n = 1}^N \w_R^n (T_0) }$ are still $\OO ( \alpha^N )$.
Consequently, as shown in Appendix~\ref{tf:app:proofs:approx_rank_1:trajectory}, at that time the weight vectors of the $R$-component tensor factorization are close to weight vectors corresponding to a balanced rank one trajectory emanating from $\S$, denoted $\W_1 (t)$.
The proof concludes by showing that this implies the time-shifted trajectory $\tftensorendbar ( t )$ is within $\epsilon$ distance from $\W_1 (t)$ at least until $t \geq T$ or $\normnoflex{ \tftensorendbar ( t ) } \geq B$.

\subsubsection{Preliminary Lemmas}
\label{tf:app:proofs:approx_rank_1:prelim_lemmas}

\begin{lemma}
	\label{tf:lem:huber_loss_const_grad_near_zero}
	Let $\W \in \R^{D_1 \times \cdots \times D_N}$ be such that $\norm{ \W } \leq \rho$, where $\rho \in (0, \min_{(i_1, \ldots, i_N) \in \Omega} \abs{ y_{i_1, \ldots, i_N} } - \delta_h )$.
	Then:
	\[
	\nabla \L_h ( \W ) = \frac{\delta_h}{ \abs{\Omega} } \sum\nolimits_{(i_1, \ldots, i_N) \in \Omega} \sign( - y_{i_1, \ldots, i_N} ) \cdot \mathcal{E}_{i_1,...,i_N}
	\text{\,,}
	\]
	where $\mathcal{E}_{i_1,...,i_N} \in \R^{D_1 \times \cdots \times D_N}$ holds $1$ in its $(i_1, \ldots, i_N)$'th entry and $0$ elsewhere.
\end{lemma}

\begin{proof}
	Fix $I := (i_1, \ldots, i_N) \in \Omega$, and let $\ell_h' (\cdot)$ denote the derivative of $\ell_h ( \cdot )$.
	If $y_I > 0$, we have that $( \W )_I - y_I \leq \norm{ \W } - y_I \leq \min_{(i_1, \ldots, i_N) \in \Omega} \abs{ y_{i_1, \ldots, i_N} } - \delta_h - y_I \leq - \delta_h$.
	Therefore, $\ell_h' ( ( \W )_I - y_I ) = - \delta_h = \sign( - y_I ) \delta_h$.
	Similarly, if $y_I < 0$, we have that $( \W )_I - y_I \geq \delta_h$ and $\ell_h' ( ( \W )_I - y_I ) = \delta_h = \sign( - y_I ) \delta_h$.
	Note that $y_I$ cannot be exactly zero as, by Assumption~\ref{tf:assump:delta_h}, $\min_{(i_1, \ldots, i_N) \in \Omega} \abs{ y_{i_1, \ldots, i_N} } > \delta_h > 0$.
	The proof concludes by the chain rule:
	\[
	\begin{split}
		\nabla \L_h ( \W ) & = \frac{1}{ \abs{\Omega} } \sum\nolimits_{I \in \Omega} \ell_h' ( ( \W )_{I} - y_{I}) \cdot \mathcal{E}_{I} \\
		& = \frac{\delta_h}{ \abs{\Omega} } \sum\nolimits_{I \in \Omega} \sign( - y_{I} ) \cdot \mathcal{E}_{I}
		\text{\,.}
	\end{split}
	\]
\end{proof}

\begin{lemma}
	\label{tf:lem:huber_loss_smooth}
	The function $\L_h ( 
	\cdot )$ is $1$-smooth, \ie~for any $\W_1, \W_2 \in \R^{D_1 \times \cdots \times D_N}$:
	\[
	\norm{ \nabla \L_h ( \W_1 ) - \nabla \L_h ( \W_2 ) } \leq \norm{ \W_1 - \W_2 }
	\text{\,.}
	\]
\end{lemma}

\begin{proof}
	Let $\W_1, \W_2 \in \R^{D_1 \times \cdots \times D_N}$.
	Denote by $\ell_h' (\cdot)$ the derivative of $\ell_h (\cdot)$, \ie:
	\[
	\ell_h' (z) = \begin{cases}
		- \delta_h	& , z < - \delta_h \\
		z	& , \abs{ z } \leq \delta_h \\
		\delta_h	& , z > \delta_h , 
	\end{cases}
	\text{\,.}
	\]
	The result readily follows from the triangle inequality and the fact that $\ell_h' (\cdot)$ is $1$-Lipschitz:
	\[
	\begin{split}
		\norm{ \nabla \L_h ( \W_1 ) - \nabla \L_h ( \W_2 ) } & = \norm{ \frac{1}{ \abs{\Omega}} \sum_{I \in \Omega} \left [ \ell_h' ( ( \W_1 )_I - y_I ) \cdot \mathcal{E}_{I} - \ell_h' (( \W_2 )_I - y_I ) \cdot \mathcal{E}_{I} \right ] } \\
		& \leq \frac{1}{ \abs{\Omega}} \sum_{I \in \Omega} \abs{ \ell_h' ( ( \W_1 )_I - y_I ) - \ell_h' (( \W_2 )_I - y_I ) } \\
		& \leq \frac{1}{ \abs{\Omega}} \sum_{I \in \Omega} \abs{ ( \W_1 )_I - ( \W_2 )_I }  \\
		& \leq \norm{ \W_1 - \W_2 }
		\text{\,,}
	\end{split}
	\]	
	where $\mathcal{E}_{I} \in \R^{D_1 \times \cdots \times D_N}$ holds $1$ in its $I$'th entry and $0$ elsewhere, for $I = (i_1, \ldots, i_N) \in \Omega$.
\end{proof}

\begin{lemma}
	\label{tf:lem:huber_cp_objective_is_smooth_over_bounded_domain}
	Let $G \geq 0$, and denote $\D_{G} :=  \{  \{ \w_{r}^{n} \in \R^{D_n} \}_{r = 1}^R\hspace{0mm}_{n = 1}^N : ( \sum_{r = 1}^R \sum_{n = 1}^N \normnoflex{ \w_r^n }^2 )^{1/2} \leq G  \}$.
	Then, the objective $\phi_h (\cdot)$ is $R N (G^{2(N - 1)} + \delta_h G^{N - 2})$-smooth over $\D_G$, \ie:
	\[
	\begin{split}
	& \norm*{ \nabla \phi_h \big ( \{ \w_{r}^{n} \}_{r = 1}^R\hspace{0mm}_{n = 1}^N \big ) - \nabla \phi_h \big ( \{ \widetilde{\w}_{r}^{n} \}_{r = 1}^R\hspace{0mm}_{n = 1}^N \big ) } \\
	& \hspace{5mm} \leq R N ( G^{2(N - 1)} + \delta_h G^{N - 2}) \cdot \sqrt{ \sum\nolimits_{r = 1}^R \sum\nolimits_{n = 1}^N \norm{ \w_{r}^{n} -  \widetilde{\w}_{r}^{n} }^2 }
	\text{\,,}
	\end{split}
	\]
	for any $ \{ \w_{r}^{n} \}_{r = 1}^R\hspace{0mm}_{n = 1}^N,  \{ \widetilde{\w}_{r}^{n} \}_{r = 1}^R\hspace{0mm}_{n = 1}^N \in \D_G$.
\end{lemma}

\begin{proof}
	Let $ \{ \w_{r}^{n} \}_{r = 1}^R\hspace{0mm}_{n = 1}^N,  \{ \widetilde{\w}_{r}^{n} \}_{r = 1}^R\hspace{0mm}_{n = 1}^N \in \D_G$.
	By Lemma~\ref{tf:lem:cp_gradient} we may write:
	\be
	\begin{split}
		& \norm*{ \nabla \phi_h \big ( \{ \w_{r}^{n} \}_{r = 1}^R\hspace{0mm}_{n = 1}^N \big ) - \nabla \phi_h \big ( \{ \widetilde{\w}_{r}^{n} \}_{r = 1}^R\hspace{0mm}_{n = 1}^N \big ) }^2  \\
		& = \sum_{r = 1}^R \sum_{n = 1}^N \norm*{ \tfmat{ \nabla \L_h \left ( \tftensorend \right ) }_{n} \cdot \kronp_{n' \neq n} \w_r^{n'} - \tfmat{ \nabla \L_h \big ( \widetilde{\W}_{\mathrm{T}} \big ) }_{n} \cdot \kronp_{n' \neq n} \widetilde{\w}_r^{n'}  }^2
		\text{\,,}
	\end{split}
	\label{tf:eq:cp_huber_loss_sq_grad_dist}
	\ee
	where $\tftensorend$ and $\widetilde{\W}_{\mathrm{T}}$ are the end tensors (Equation~\eqref{tf:eq:end_tensor}) of $\{ \w_r^n \}_{r = 1}^R\hspace{0mm}_{n = 1}^N$ and $\{ \widetilde{\w}_r^n \}_{r = 1}^R\hspace{0mm}_{n = 1}^N$, respectively.
	We turn to bound the square root of each term in the sum.
	Fix $r \in [R], n \in [N]$.
	By the triangle inequality and sub-multiplicativity of the Frobenius norm, we have that:
	\[
	\begin{split}
		& \norm*{ \tfmat{ \nabla \L_h \left ( \tftensorend \right ) }_{n} \cdot \kronp_{n' \neq n} \w_r^{n'} - \tfmat{ \nabla \L_h  \big ( \widetilde{\W}_{\mathrm{T}} \big ) }_{n} \cdot \kronp_{n' \neq n} \widetilde{\w}_r^{n'}  } \\[1mm]
		& \hspace{5mm} \leq \underbrace{ \norm*{ \tfmat{ \nabla \L_h \left ( \tftensorend \right ) }_{n} - \tfmat{ \nabla \L_h  \big ( \widetilde{\W}_{\mathrm{T}} \big ) }_{n} } }_{(I)} \cdot \underbrace{ \norm*{ \kronp_{n' \neq n} \w_r^{n'} } }_{ (II) } \\
		& \hspace{6mm} + \underbrace{ \norm*{ \tfmat{ \nabla \L_h \big ( \widetilde{\W}_{\mathrm{T}} \big ) }_{n} } }_{(III)} \cdot \underbrace{ \norm*{ \kronp_{n' \neq n} \w_r^{n'} - \kronp_{n' \neq n} \widetilde{\w}_r^{n'} } }_{(IV)}
		\text{\,.}
	\end{split}
	\]
	Below, we derive upper bounds for $(I), (II), (III)$, and $(IV)$ separately.
	Starting with $(I)$, by Lemma~\ref{tf:lem:huber_loss_smooth}, the triangle inequality and Lemma~\ref{tf:lem:outer_prod_distance_bound}, it follows that:
	\[
	\begin{split}
		(I) & = \norm*{ \nabla \L_h \left ( \tftensorend \right ) - \nabla \L_h  \big ( \widetilde{\W}_{\mathrm{T}} \big ) } \\
		& \le \norm*{ \tftensorend - \widetilde{\W}_{\mathrm{T}} } \\
		& \leq \sum_{r' = 1}^R \norm*{ \tenp_{n' = 1}^N \w_{r'}^{n'} - \tenp_{n' = 1}^N \widetilde{\w}_{r'}^{n'} } \\
		& \leq G^{N - 1} \sum_{r' = 1}^R \sum_{n' = 1}^N  \norm*{ \w_{r'}^{n'} - \widetilde{\w}_{r'}^{n'} }
		\text{\,.}
	\end{split}
	\]
	Moving on to $(II)$, we have that $\normnoflex{ \kronp_{n' \neq n} \w_r^{n'} } = \prod_{n' \neq n} \normnoflex{ \w_r^{n'} } \leq G^{N - 1}$.
	For $(III)$, the triangle inequality and the fact that $\ell_h' ( \cdot )$, the derivative of $\ell_h (\cdot)$, is bounded (in absolute value) by $\delta_h$ yield:
	\[
	\begin{split}
		(III) & = \norm*{ \frac{1}{ \abs{\Omega} } \sum\nolimits_{I \in \Omega} \ell_h' \left ( ( \widetilde{\W}_{\mathrm{T}} )_I - y_I \right ) \cdot \mathcal{E}_{I} } \leq \delta_h
		\text{\,,}
	\end{split}
	\]
	where $\mathcal{E}_{I} \in \R^{D_1 \times \cdots \times D_N}$ holds $1$ in its $I$'th entry and $0$ elsewhere, for $I = (i_1, \ldots, i_N) \in \Omega$.
	Lastly, since $\normnoflex{ \kronp_{n' \neq n} \w_r^{n'} - \kronp_{n' \neq n} \widetilde{\w}_r^{n'} } = \normnoflex{ \tenp_{n' \neq n} \w_r^{n'} - \tenp_{n' \neq n} \widetilde{\w}_r^{n'} }$, by Lemma~\ref{tf:lem:outer_prod_distance_bound} we have that:
	\[
	(IV) \leq G^{N - 2} \sum_{n' \neq n} \norm*{ \w_r^{n'} - \widetilde{\w}_r^{n'} } \leq G^{N - 2} \sum_{n' =1}^N \norm*{ \w_r^{n'} - \widetilde{\w}_r^{n'} }
	\text{\,.}
	\]
	Putting it all together, we arrive at the following bound:
	\[
	\begin{split}
		& \norm*{ \tfmat{ \nabla \L_h \left ( \tftensorend \right ) }_{n} \cdot \kronp_{n' \neq n} \w_r^{n'} - \tfmat{ \nabla \L_h  \big ( \widetilde{\W}_{\mathrm{T}} \big ) }_{n} \cdot \kronp_{n' \neq n} \widetilde{\w}_r^{n'}  } \\
		& \leq G^{2(N - 1)} \sum_{r' = 1}^R \sum_{n' = 1}^N  \norm*{ \w_{r'}^{n'} - \widetilde{\w}_{r'}^{n'} } + \delta_h G^{N - 2} \sum_{n' =1}^N \norm*{ \w_r^{n'} - \widetilde{\w}_r^{n'} } \\
		& \leq ( G^{2(N - 1)} + \delta_h G^{N - 2}) \sum_{r' = 1}^R \sum_{n' = 1}^N  \norm*{ \w_{r'}^{n'} - \widetilde{\w}_{r'}^{n'} } 
		\text{\,.}
	\end{split}
	\]
	Applying the bound above to Equation~\eqref{tf:eq:cp_huber_loss_sq_grad_dist}, for all $r \in [R] , n \in [N]$, leads to:
	\[
	\begin{split}
		& \norm*{ \nabla \phi_h \left ( \{ \w_{r}^{n} \}_{r = 1}^R\hspace{0mm}_{n = 1}^N \right ) - \nabla \phi_h \left( \{ \widetilde{\w}_{r}^{n} \}_{r = 1}^R\hspace{0mm}_{n = 1}^N \right ) }^2  \\
		& \leq R N ( G^{2(N - 1)} + \delta_h G^{N - 2})^2 \left ( \sum\nolimits_{r = 1}^R \sum\nolimits_{n = 1}^N  \norm*{ \w_{r}^{n} - \widetilde{\w}_{r}^{n} } \right )^2 \\
		& \leq R^2 N^2 ( G^{2(N - 1)} + \delta_h G^{N - 2})^2 \sum\nolimits_{r = 1}^R \sum\nolimits_{n = 1}^N  \norm*{ \w_{r}^{n} - \widetilde{\w}_{r}^{n} }^2
		\text{\,,}
	\end{split}
	\]
	where the last transition is by the fact that $\normnoflex{ \x }_1 \leq \sqrt{D} \cdot \normnoflex{ \x }$ for any $\x \in \R^D$.
	Taking the square root of both sides concludes the proof.
\end{proof}

\begin{lemma}
	\label{tf:lem:const_grad_inner_prod_monotonicity_helper_lemma}
	Let $t' > 0$ and $r \in [R]$.
	Denote $\gamma_r (t) := \inprodnoflex{ - \nabla \L_h ( \tftensorend (t) ) }{ \tenp_{n = 1}^N \widehat{\w}_r^{n} (t) }$, where $\widehat{\w}_r^{n} (t) := \w_r^{n} (t) / \normnoflex{ \w_r^{n} (t) }$ if $\w_r^{n} (t) \neq 0$, and $\widehat{\w}_r^{n} (t) := 0$ otherwise, for $n = 1, \ldots, N$. Suppose that $\nabla \L_h ( \tftensorend (t)) = \nabla \L_h (0)$ for all $t \in [0, t')$.
	Then, $\gamma_r (t)$ is monotonically non-decreasing over the interval $[0, t')$.
\end{lemma}

\begin{proof}
	In the following, unless explicitly stated otherwise, $t$ is to be considered in the time interval $[0, t')$.
	
	Recall that by Assumption~\ref{tf:assump:a_balance} we have that $\normnoflex{ \w_r^{1} (0) } = \cdots = \normnoflex{ \w_r^{N} (0) }$.
	If $\normnoflex{ \w_r^{1} (0) } = \cdots = \normnoflex{ \w_r^{N} (0) } = 0$, then according to Lemma~\ref{tf:lem:balanced_param_vector_norm_no_sign_change} $\normnoflex{ \w_r^{1} (t) } = \cdots = \normnoflex{ \w_r^{N} (t) } = 0$ for all $t \geq 0$.
	In this case $\gamma_r (t) = 0$ over $[0, t')$, and is therefore non-decreasing.
	
	Otherwise, if $\normnoflex{ \w_r^{1} (0) } = \cdots = \normnoflex{ \w_r^{N} (0) } > 0$, from Lemma~\ref{tf:lem:balanced_param_vector_norm_no_sign_change}  we get that $\normnoflex{ \w_r^{1} (t) } = \cdots = \normnoflex{ \w_r^{N} (t) }> 0$ for all $t \geq 0$.
	Thus:
	\[
	\begin{split}
		\gamma_r (t) & = \norm{ \tenp_{n = 1}^N \w_r^{n} (t) }^{-1} \inprodnoflex{ - \nabla \L_h ( \tftensorend (t) ) }{ \tenp_{n = 1}^N \w_r^{n} (t) } \\
		& = \norm{ \tenp_{n = 1}^N \w_r^{n} (t) }^{-1} \inprodnoflex{ - \nabla \L_h ( 0 ) }{ \tenp_{n = 1}^N \w_r^{n} (t) }
		\text{\,,}
	\end{split}
	\] 
	where the second transition is due to $\nabla \L_h ( \tftensorend (t)) = \nabla \L_h (0)$.
	Differentiating with respect to time, we have that:
	\be
	\begin{split}
		\frac{d}{dt} \gamma_r (t) & = - \underbrace{\frac{d}{dt} \left [ \norm{ \tenp_{n = 1}^N \w_r^{n} (t) } \right ] \cdot \norm{ \tenp_{n = 1}^N \w_r^{n} (t) }^{-2} \inprod{ - \nabla \L_h ( 0 ) }{ \tenp_{n = 1}^N \w_r^{n} (t) }}_{(I)} \\
		& \hspace{5mm} + \norm{ \tenp_{n = 1}^N \w_r^{n} (t) }^{-1} \underbrace{ \inprod{ - \nabla \L_h ( 0 ) }{ \tfrac{d}{dt} \tenp_{n = 1}^N \w_r^{n} (t) } }_{ (II) }
		\text{\,.}
	\end{split}
	\label{tf:eq:inner_prod_time_deriv_two_parts}
	\ee
	We now treat $(I)$ and $(II)$ separately.
	Plugging the expression for $\frac{d}{dt} \normnoflex{ \tenp_{n = 1}^N \w_r^{n} (t) }$ from Corollary~\ref{tf:cor:dyn_fac_comp_norm_balanced} into $(I)$, and recalling that $\nabla \L_h ( \tftensorend (t) ) = \nabla \L_h (0)$, leads to:
	\[
	\begin{split}
		(I) = N \norm{ \tenp_{n = 1}^N \w_r^{n} (t) }^{-1 - 2 / N} \inprod{ - \nabla \L_h ( 0 ) }{ \tenp_{n = 1}^N \w_r^{n} (t) }^2
		\text{\,.}
	\end{split}
	\]
	Due to the fact that $\normnoflex{ \tenp_{n = 1}^N \w_r^{n} (t) }^{-2/N} = \normnoflex{ \w_r^{1} (t) }^{-2} = \cdots = \normnoflex{ \w_r^{N} (t) }^{-2}$, we may equivalently write:
	\be
	\begin{split}
		(I) = \norm{ \tenp_{n = 1}^N \w_r^{n} (t) }^{-1} \sum_{n = 1}^N \norm{ \w_r^{n} (t) }^{-2} \inprod{ - \nabla \L_h ( 0 ) }{ \tenp_{n' = 1}^N \w_r^{n'} (t) }^2
		\text{\,.}
	\end{split}
	\label{tf:eq:I_phrase_in_mon_inner_prod_lemma}
	\ee
	For any $n \in [N]$, by Lemma~\ref{tf:lem:dyn_parameter_vector_sq_norm} we know that $\frac{d}{dt} \normnoflex{ \w_r^{n} (t) }^2 = - 2 \inprodnoflex{ \nabla \L_h \left ( 0 \right ) }{ \tenp_{n' = 1}^N \w_r^{n'} (t) }$, which implies $\frac{d}{dt} \normnoflex{ \w_r^{n} (t) } = \normnoflex{ \w_r^{n} (t) }^{-1} \inprodnoflex{ - \nabla \L_h \left ( 0 \right ) }{ \tenp_{n' = 1}^N \w_r^{n'} (t) }$.
	Going back to Equation~\eqref{tf:eq:I_phrase_in_mon_inner_prod_lemma}, we can see that:
	\[
	(I) = \norm{ \tenp_{n = 1}^N \w_r^{n} (t) }^{-1} \sum_{n = 1}^N \left ( \tfrac{d}{dt} \norm{ \w_r^{n} (t) } \right )^2
	\text{\,.}
	\]
	Turning our attention to $(II)$, by Lemmas~\ref{tf:lem:inp_with_tenp_to_mat_kronp} and~\ref{tf:lem:cp_gradient} it follows that:
	\[
	\begin{split}
		(II) & =  \sum_{n = 1}^N \inprod{ - \nabla \L_h ( 0 ) }{ \left ( \tenp_{n' = 1}^{n - 1} \w_r^{n'} (t) \right ) \tenp \tfrac{d}{dt} \w_r^{n} (t) \tenp \left ( \tenp_{n' = n + 1}^N \w_r^{n'} (t) \right )  } \\
		& = \sum_{n = 1}^N \inprod{ \tfmat{ - \nabla \L_h \left ( 0 \right ) }_{n} \cdot \kronp_{n' \neq n} \w_r^{n'} (t)}{ \tfrac{d}{dt} \w_r^{n} (t) } \\
		& = \sum_{n = 1}^N \norm{  \tfrac{d}{dt} \w_r^{n} (t) }^2
		\text{\,.}
	\end{split}
	\]
	Plugging the expressions we derived for $(I)$ and $(II)$ into Equation~\eqref{tf:eq:inner_prod_time_deriv_two_parts} yields:
	\be
	\begin{split}
		& \frac{d}{dt} \gamma_r (t) = \norm{ \tenp_{n = 1}^N \w_r^{n} (t) }^{-1} \cdot \sum_{n = 1}^N  \left [ \norm{  \tfrac{d}{dt} \w_r^{n} (t) }^2 -  \left ( \tfrac{d}{dt} \norm{ \w_r^{n} (t) } \right )^2 \right ]
		\text{\,.}
	\end{split}
	\label{tf:eq:inner_prod_time_deriv_final}
	\ee
	Notice that for any $n \in [N]$:
	\[
	\begin{split}
		\norm{ \tfrac{d}{dt} \w_r^{n} (t) }^2 & \geq \norm{ \Pi_{\w_r^{n} (t)} \left ( \tfrac{d}{dt} \w_r^{n} (t) \right ) }^2 \\
		& = \norm{ \inprodnoflex{ \tfrac{d}{dt} \w_r^{n} (t) }{ \w_r^{n} (t) } \frac{ \w_r^{n} (t) }{  \normnoflex{ \w_r^{n} (t) }^2 } }^2 \\
		& = \left ( \norm{ \w_r^{n} (t) }^{-1} \inprod{ \tfrac{d}{dt} \w_r^{n} (t) }{ \w_r^{n} (t) } \right )^2 \\
		& =  \left ( \tfrac{d}{dt} \norm{ \w_r^{n} (t) } \right )^2
		\text{\,,}
	\end{split}
	\]
	where $\Pi_{\w_r^{n} (t)} (\cdot)$ denotes the orthogonal projection onto the subspace spanned by $\w_r^{n} (t)$.
	The right hand side in Equation~\eqref{tf:eq:inner_prod_time_deriv_final} is therefore non-negative, \ie~$\frac{d}{dt} \gamma_r (t) \geq 0$,
	concluding the proof.
\end{proof}

\subsubsection{Stage I: End Tensor Reaches Reference Sphere}
\label{tf:app:proofs:approx_rank_1:const_grad}

\begin{proposition}
	\label{tf:prop:end_tensor_reaches_ref_sphere}
	The end tensor initially reaches reference sphere $\S$ (Equation~\eqref{tf:eq:ref_sphere}) at some time $T_0 >0$, and:
	\begin{align}
		& \norm{ \tenp_{n = 1}^N \w_r^{n} (t) } \leq \tilde{\epsilon} \quad, ~t \in [0, T_0] ~,~ r = 2, \ldots, R 
		\text{\,,}	
		\label{tf:eq:approx_rank_1_stage_I_t_0_comp_bounds_upper_bound} \\[0.5ex]
		& \abs*{\norm{ \tenp_{n = 1}^N \w_1^{n} (T_0) } - \rho} \leq (R - 1) \cdot \tilde{\epsilon} \text{\,,} 
		\label{tf:eq:approx_rank_1_stage_I_t_0_first_comp_bound} 
	\end{align}
	where $\tilde{\epsilon}$ is as defined in Equation~\eqref{tf:eq:approx_rank_1_consts}.
\end{proposition}

\medskip

Towards proving Proposition~\ref{tf:prop:end_tensor_reaches_ref_sphere}, we establish the following key lemma.
\begin{lemma}
	\label{tf:lem:const_grad_comp_norm_bounds_helper_lemma}
	Let $t' \leq  \frac{\alpha^{2 - N} \normnoflex{ \abf_1 }^{2 - N} (N - 2)^{-1}}{ \inprodnoflex{ - \nabla \L_h (0) }{ \tenp_{n = 1}^N \widehat{\abf}_{ 1 }^{n} } }$, and suppose that $\nabla \L_h ( \tftensorend (t)) = \nabla \L_h (0)$ for all $t \in [0, t')$.
	Then:
	\begin{align}
		& \norm{ \tenp_{n = 1}^N \w_1^{n} (t) } \geq \left ( \alpha^{2-N} \normnoflex{ \abf_1 }^{2 - N}  - (N - 2) \inprod{ - \nabla \L_h (0) }{ \tenp_{n = 1}^N \widehat{\abf}_{ 1 }^{n} } \cdot t \right )^{- \frac{N}{N-2}} , t \in [0, t') \text{\,,}
		\label{tf:eq:comp_norm_lower_bound_with_time_helper_lemma} \\[0.5ex]
		& \norm{ \tenp_{n = 1}^N \w_r^{n} (t) } \leq \left ( \alpha^{2-N} \normnoflex{ \abf_r }^{2 - N}  - (N - 2) \norm{ \nabla \L_h ( 0 ) } \cdot t \right )^{- \frac{N}{N-2}} , t \in [0, t') , r = 2, \ldots, R \text{\,.}
		\label{tf:eq:comp_norms_upper_bound_with_time_helper_lemma}
	\end{align}
	In particular:
	\be
	\norm{ \tenp_{n = 1}^N \w_r^{n} (t) } \leq \alpha^N \Big ( \normnoflex{ \abf_r }^{2 - N}  - \normnoflex{ \abf_1 }^{2 - N} \tfrac{ \norm{ \nabla \L_h ( 0 ) } }{  \inprod{ - \nabla \L_h (0) }{ \tenp_{n = 1}^N \widehat{\abf}_{ 1 }^{n} } } \Big )^{- \frac{N}{N-2}} \! , t \in [0, t') , r = 2, \ldots, R \text{\,.}
	\label{tf:eq:comp_norms_upper_bound_helper_lemma}
	\ee
\end{lemma}

\begin{proof}
	For simplicity of notation we denote $\gamma_r (t) := \inprodnoflex{ - \nabla \L_h ( \tftensorend (t) ) }{ \tenp_{n = 1}^N \widehat{\w}_r^{n} (t) }$, where $\widehat{\w}_r^{n} (t) := \w_r^{n} (t) / \normnoflex{ \w_r^{n} (t) }$ if $\w_r^{n} (t) \neq 0$, and $\widehat{\w}_r^{n} (t) := 0$ otherwise, for $r = 1, \ldots, R, ~n = 1, \ldots, N$.
	In the following, unless explicitly stated otherwise, $t$ is to be considered in the time interval $[0, t')$.
	
	By Assumption~\ref{tf:assump:a_balance}, $\{ \abf_r^n \}_{r = 1}^R\hspace{0mm}_{n = 1}^N$ have unbalancedness magnitude zero, thus, so do $\{ \w_r^n (0) \}_{r = 1}^R\hspace{0mm}_{n = 1}^N$ (recall $\w_r^n (0) = \alpha \cdot \abf_r^n$ for $r = 1, \ldots, R, ~n = 1, \ldots, N$).
	According to Corollary~\ref{tf:cor:dyn_fac_comp_norm_balanced} the evolution of a component's norm is given by:
	\be
	\frac{d}{dt} \norm{ \tenp_{n = 1}^N \w_r^{n} (t) } = N \gamma_r (t) \cdot \norm{ \tenp_{n = 1}^N \w_r^{n} (t) }^{2 - \frac{2}{N}} \quad , ~r = 1, \ldots, R
	\text{\,.}
	\label{tf:eq:dyn_fac_comp_norm_const_grad}
	\ee
	
	\paragraph*{Proof of Equation~\eqref{tf:eq:comp_norm_lower_bound_with_time_helper_lemma}  (lower bound for $\normnoflex{ \tenp_{n = 1}^N \w_1^{n} (t) }$):}
	
	By Lemma~\ref{tf:lem:const_grad_inner_prod_monotonicity_helper_lemma}, $\gamma_1 (t)$ is monotonically non-decreasing.
	Thus, from Equation~\eqref{tf:eq:dyn_fac_comp_norm_const_grad} we have:
	\be
	\frac{d}{dt} \norm{ \tenp_{n = 1}^N \w_1^{n} (t) } \geq N \gamma_1 (0) \cdot \norm{ \tenp_{n = 1}^N \w_1^{n} (t) }^{2 -\frac{2}{N}}
	\text{\,.}
	\label{tf:eq:comp_1_norm_deriv_lower_bound}
	\ee
	Assumption~\ref{tf:assump:a_lead_comp} (second line in Equation~\eqref{tf:eq:assump_components_sep_at_init}) necessarily means that $\w_1^{n} (0) = \alpha \cdot \abf_1^{n} \neq 0$ for all $n \in [N]$.
	Recalling that the unbalancedness magnitude is zero at initialization, from Lemma~\ref{tf:lem:balanced_param_vector_norm_no_sign_change} we get that $\normnoflex{ \w_1^{1} (t) } = \cdots = \normnoflex{ \w_1^{N} (t) } > 0$, and so $\normnoflex{ \tenp_{n = 1}^N \w_1^{n} (t) }^{2 - 2 / N} > 0$, for all $t \in [0, t')$.
	Therefore, we may divide both sides of Equation~\eqref{tf:eq:comp_1_norm_deriv_lower_bound} by $\normnoflex{ \tenp_{n = 1}^N \w_1^{n} (t) }^{2 - 2 / N}$.
	Doing so, and integrating with respect to time, leads to:
	\be
	\begin{split}
		& \int_{\hat{t} = 0}^t \left [ \norm{ \tenp_{n = 1}^N \w_1^{n} (\hat{t}) }^{2 / N - 2} \frac{d}{d \hat{t}} \norm{ \tenp_{n = 1}^N \w_1^{n} (\hat{t}) } \right ] d \hat{t} \geq N \gamma_1 (0) \cdot t \\
		\implies & \frac{N}{2 - N} \left ( \norm{ \tenp_{n = 1}^N \w_1^{n} (t) }^{2 / N - 1} - \norm{ \tenp_{n = 1}^N \w_1^{n} (0) }^{2 / N - 1} \right ) \geq N \gamma_1 (0) \cdot t \\
		\implies & \norm{ \tenp_{n = 1}^N \w_1^{n} (t) }^{2 / N - 1} \leq  \norm{ \tenp_{n = 1}^N \w_1^{n} (0) }^{2 / N - 1} - (N - 2) \gamma_1 (0) \cdot t
		\text{\,.}
	\end{split}
	\label{tf:eq:comp_1_integration_lower_bound}
	\ee
	Notice that $\gamma_1 (0) =  \inprodnoflex{ - \nabla \L_h ( \tftensorend (0) ) }{ \tenp_{n = 1}^N \widehat{\w}_1^{n} (0) } = \inprodnoflex{ - \nabla \L_h ( 0 ) }{ \tenp_{n = 1}^N \widehat{\abf}_1^{n} }$. 
	Since $\normnoflex{ \tenp_{n = 1}^N \w_1^{n} (0) } = \prod_{n = 1}^N \normnoflex{ \w_1^{n} (0) } = \alpha^N \normnoflex{ \abf_1 }^N$ and $t < t' \leq \alpha^{2 - N} \normnoflex{ \abf_1 }^{2 - N} (N - 2)^{-1} \gamma_1 (0)^{-1}$, we can see that:
	\[
	\norm{ \tenp_{n = 1}^N \w_1^{n} (0) }^{2 / N - 1} - (N - 2) \gamma_1 (0) \cdot t = \alpha^{2 - N} \norm{ \abf_1 }^{2 - N} - (N - 2) \gamma_1 (0) \cdot t > 0
	\text{\,.}
	\]
	Therefore, Equation~\eqref{tf:eq:comp_norm_lower_bound_with_time_helper_lemma} readily follows by rearranging the last inequality in Equation~\eqref{tf:eq:comp_1_integration_lower_bound}:
	\[
	\begin{split}
		\norm{ \tenp_{n = 1}^N \w_1^{n} (t) } \geq \left (\alpha^{2 - N} \norm{ \abf_1}^{2 - N} - (N - 2) \gamma_1 (0) \cdot t  \right )^{ - \frac{N}{N - 2} }
		\text{\,.}
	\end{split}
	\]

	\paragraph*{Proof of Equations~\eqref{tf:eq:comp_norms_upper_bound_with_time_helper_lemma} and~\eqref{tf:eq:comp_norms_upper_bound_helper_lemma}  (upper bounds for $\normnoflex{ \tenp_{n = 1}^N \w_r^{n} (t) }$):}
	
	Fix some $r \in \{ 2, \ldots, R \}$.
	First, we deal with the case where $\normnoflex{ \w_r^{1} (0) } = \cdots = \normnoflex{ \w_r^{N} (0) } = 0$.
	If it is so, by Lemma~\ref{tf:lem:balanced_param_vector_norm_no_sign_change} we have that $\normnoflex{ \w_r^{1} (t) } = \cdots = \normnoflex{ \w_r^{N} (t) } = 0$ for all $t \in [0, t')$.
	Hence, $\normnoflex{ \tenp_{n = 1}^N \w_r^{n} (t) } = 0$ for all $t \in [0, t')$, \ie~Equations~\eqref{tf:eq:comp_norms_upper_bound_with_time_helper_lemma} and~\eqref{tf:eq:comp_norms_upper_bound_helper_lemma} trivially hold.
	
	Now we move to the case where $\normnoflex{ \w_r^{1} (0) } = \cdots = \normnoflex{ \w_r^{N} (0) } > 0$.
	From Lemma~\ref{tf:lem:balanced_param_vector_norm_no_sign_change} we know that $\normnoflex{ \w_r^{1} (t) } = \cdots = \normnoflex{ \w_r^{N} (t) } > 0$ for all $t \in [0, t')$.
	Since $\nabla \L_h ( \tftensorend (t)) = \nabla \L_h (0)$, by the Cauchy-Schwarz inequality we then have:
	\[
	\gamma_r (t) = \inprod{ - \nabla \L_h ( 0 ) }{ \tenp_{n = 1}^N \widehat{\w}_r^{n} (t) } \leq \norm{ \nabla \L_h (0) } \norm{ \tenp_{n = 1}^N \widehat{\w}_r^{n} (t) } = \norm{ \nabla \L_h (0) }
	\text{\,.}
	\]
	Combined with Equation~\eqref{tf:eq:dyn_fac_comp_norm_const_grad}, we arrive at the following upper bound:
	\[
	\frac{d}{dt} \norm{ \tenp_{n = 1}^N \w_r^{n} (t) } \leq N \norm{ \nabla \L_h (0) } \cdot \norm{ \tenp_{n = 1}^N \w_r^{n} (t) }^{2 - \frac{2}{N}}
	\text{\,.}
	\]
	Dividing both sides of the inequality by $\normnoflex{ \tenp_{n = 1}^N \w_r^{n} (t) }^{2 - 2 / N}$ (is positive since $\normnoflex{ \w_r^{1} (t) } = \cdots = \normnoflex{ \w_r^{N} (t) } > 0$), and integrating with respect to time, yields:
	\[
	\begin{split}
		& \int_{\hat{t} = 0}^t \left [ \norm{ \tenp_{n = 1}^N \w_r^{n} (\hat{t}) }^{2 / N - 2} \frac{d}{d \hat{t}} \norm{ \tenp_{n = 1}^N \w_r^{n} (\hat{t}) } \right ] d \hat{t} \leq N \norm{ \nabla \L_h (0) } \cdot t \\
		\implies & \frac{N}{2 - N} \left ( \norm{ \tenp_{n = 1}^N \w_r^{n} (t) }^{2 / N - 1} - \norm{ \tenp_{n = 1}^N \w_r^{n} (0) }^{2 / N - 1} \right ) \leq N \norm{ \nabla \L_h (0) } \cdot t
		\text{\,.}
	\end{split}
	\]
	Rearranging the inequality above, and making use of the fact that $\normnoflex{ \tenp_{n = 1}^N \w_r^{n} (0) } = \prod_{n = 1}^N \normnoflex{ \w_r^{n} (0) } = \alpha^N \normnoflex{ \abf_r }^N$, we arrive at:
	\be
	\begin{split}
		\norm{ \tenp_{n = 1}^N \w_r^{n} (t) }^{2 / N - 1} & \geq \norm{ \tenp_{n = 1}^N \w_r^{n} (0) }^{2 / N - 1} - (N - 2) \norm{ \nabla \L_h (0) } \cdot t \\
		& = \alpha^{2 - N} \norm{ \abf_r}^{2 - N} - (N - 2) \norm{ \nabla \L_h (0) } \cdot t 
		\text{\,.}
	\end{split}
	\label{tf:eq:comp_norm_derivation_intermediate_upper_bound}
	\ee
	Noticing $\gamma_1 (0) =  \inprodnoflex{ - \nabla \L_h ( \tftensorend (0) ) }{ \tenp_{n = 1}^N \widehat{\w}_1^{n} (0) } = \inprodnoflex{ - \nabla \L_h ( 0 ) }{ \tenp_{n = 1}^N \widehat{\abf}_1^{n} }$, by Assumption~\ref{tf:assump:a_lead_comp} we have that $\normnoflex{ \abf_{1} } > \normnoflex{ \abf_{r} } \norm{ \nabla \L_h ( 0 ) }^{1 / (N - 2)} \cdot \gamma_1 (0)^{- 1 / (N - 2)}$.
	Therefore:
	\[
	t' \leq \alpha^{2 - N} \normnoflex{ \abf_1 }^{2 - N} (N - 2)^{-1} \gamma_1 (0)^{-1} < \alpha^{2 - N} \normnoflex{ \abf_r }^{2 - N} (N - 2)^{-1} \normnoflex{ \nabla \L_h (0)}^{-1}
	\text{\,.}
	\]
	This implies that the right hand side in Equation~\eqref{tf:eq:comp_norm_derivation_intermediate_upper_bound} is positive for all $t \in [0, t')$.
	Thus, rearranging Equation~\eqref{tf:eq:comp_norm_derivation_intermediate_upper_bound} establishes
	Equation~\eqref{tf:eq:comp_norms_upper_bound_with_time_helper_lemma}:
	\[
	\norm{ \tenp_{n = 1}^N \w_r^{n} (t) } \leq \left ( \alpha^{2-N} \normnoflex{ \abf_r }^{2 - N}  - (N - 2) \norm{ \nabla \L_h ( 0 ) } \cdot t \right )^{- \frac{N}{N-2}} 
	\text{\,.}
	\]
	Equation~\eqref{tf:eq:comp_norms_upper_bound_helper_lemma} then directly follows:
	\[
	\begin{split}
		\norm{ \tenp_{n = 1}^N \w_r^{n} (t) } & \leq \left ( \alpha^{2-N} \normnoflex{ \abf_r }^{2 - N}  - (N - 2) \norm{ \nabla \L_h ( 0 ) } \cdot t' \right )^{- \frac{N}{N-2}} \\
		& \leq \left ( \alpha^{2-N} \normnoflex{ \abf_r }^{2 - N}  - \alpha^{2 - N} \normnoflex{ \abf_1 }^{2 - N} \norm{ \nabla \L_h ( 0 ) } \gamma_1 (0)^{-1}  \right )^{- \frac{N}{N-2}} \\
		& = \alpha^N \left (  \normnoflex{ \abf_r }^{2 - N} - \normnoflex{ \abf_1 }^{2 - N} \norm{ \nabla \L_h ( 0 ) } \gamma_1 (0)^{-1} \right )^{- \frac{N}{N-2}} 
		\text{\,.}
	\end{split}
	\]
\end{proof}

\medskip

\begin{proof}[Proof of Proposition~\ref{tf:prop:end_tensor_reaches_ref_sphere}]
	Notice that at initialization $\norm{ \tftensorend (0) } \leq  \sum_{r = 1}^R \normnoflex{ \tenp_{n = 1}^N \w_r^{n} (0) } \leq R \alpha^N A^N < \rho$.
	We can therefore examine the trajectory up until the time at which $\norm{ \tftensorend (t) } = \rho$, \ie~until it reaches the reference sphere $\S$.
	Formally, define:
	\[
	T_0 := \inf \left \{ t \geq 0 : \tftensorend (t) \in \S \right \}
	\text{\,,}
	\]
	where by convention $T_0 := \infty$ if the set on the right hand side is empty.
	For all $t \in [0, T_0)$, clearly, $\norm{ \tftensorend (t) } < \rho$, and so by Lemma~\ref{tf:lem:huber_loss_const_grad_near_zero} $\nabla \L_h ( \tftensorend (t) ) = \nabla \L_h (0)$.
	We claim that $T_0$ is finite.
	Assume by way of contradiction that $T_0 = \infty$.
	For $t' :=  \alpha^{2 - N} \normnoflex{ \abf_1 }^{2 - N} (N - 2)^{-1} \inprodnoflex{ - \nabla \L_h (0) }{ \tenp_{n = 1}^N \widehat{\abf}_{1}^{n} }^{-1}$, by Equation~\eqref{tf:eq:comp_norm_lower_bound_with_time_helper_lemma} from Lemma~\ref{tf:lem:const_grad_comp_norm_bounds_helper_lemma} we have that $ \normnoflex{ \tenp_{n = 1}^N \w_1^{n} (t) } $ is lower bounded by a quantity that goes to $\infty$ as $t \to t^{\prime -}$.
	On the other hand, by Equation~\eqref{tf:eq:comp_norms_upper_bound_helper_lemma} from Lemma~\ref{tf:lem:const_grad_comp_norm_bounds_helper_lemma}, $\normnoflex{ \tenp_{n = 1}^N \w_2^{n} (t) }, \ldots, \normnoflex{ \tenp_{n = 1}^N \w_R^{n} (t) }$ are bounded over $[0, t')$.
	Taken together, there must exist $\hat{t} \in [0, t')$ at which:
	\[
	\norm{ \tftensorend  ( \hat{t})} \geq \normnoflex{ \tenp_{n = 1}^N \w_1^{n} ( \hat{t} ) } - \sum_{r = 2}^R \normnoflex{ \tenp_{n = 1}^N \w_r^{n} ( \hat{t} ) } \geq \rho
	\text{\,.}
	\]
	Since $\norm{ \tftensorend (t) }$ is continuous in $t$, and $\norm{ \tftensorend (0) } < \rho$, this contradicts our assumption that $T_0 = \infty$.
	Hence, $T_0 < \infty$, and in particular $T_0 < t'$.
	Notice that continuity of $\norm{ \tftensorend (t) }$ further implies that $\norm{ \tftensorend (T_0) } = \rho$, \ie~ $T_0$ is the initial time at which $\tftensorend (t)$ reaches the reference sphere $\S$.
	Applying our assumption on the size of $\alpha$ (Equation~\eqref{tf:eq:alpha_size_requirement}) to Equation~\eqref{tf:eq:comp_norms_upper_bound_helper_lemma} from Lemma~\ref{tf:lem:const_grad_comp_norm_bounds_helper_lemma} establishes Equation~\eqref{tf:eq:approx_rank_1_stage_I_t_0_comp_bounds_upper_bound}.
	Equation~\eqref{tf:eq:approx_rank_1_stage_I_t_0_first_comp_bound} then readily follows by the triangle inequality:
	\[
	\begin{split}
		\abs*{ \norm{ \tenp_{n = 1}^N \w_1^{n} (T_0) } - \rho } & =  \abs*{ \norm{ \tenp_{n = 1}^N \w_1^{n} (T_0) } - \norm{ \tftensorend (T_0) } } \\
		& \leq \norm{ \tenp_{n = 1}^N \w_1^{n} (T_0) - \tftensorend (T_0) } \\
		& = \norm*{ \sum\nolimits_{r = 2}^R \tenp_{n = 1}^N \w_r^n (T_0) }  \\
		& \leq (R - 1) \cdot \tilde{\epsilon}
		\text{\,.}
	\end{split}
	\]
\end{proof}

\subsubsection{Stage II: End Tensor Follows Rank One Trajectory}
\label{tf:app:proofs:approx_rank_1:trajectory}

As shown in Proposition~\ref{tf:prop:end_tensor_reaches_ref_sphere} (Appendix~\ref{tf:app:proofs:approx_rank_1:const_grad}), the end tensor initially reaches reference sphere $\S$ at some time $T_0 > 0$, for which Equations~\eqref{tf:eq:approx_rank_1_stage_I_t_0_comp_bounds_upper_bound} and~\eqref{tf:eq:approx_rank_1_stage_I_t_0_first_comp_bound} hold.
Therefore, the time-shifted trajectory is given by $\tftensorendbar (t) = \tftensorend (t + T_0)$ for all $t \geq 0$.
Denote the corresponding time-shifted factorization weight vectors by:
\[
\widebar{\w}_r^n (t) := \w_r^n (t + T_0) \quad , ~ t \geq 0 ~,~ r = 1, \ldots, R ~,~ n = 1, \ldots, N
\text{\,.}
\]
We are now at a position to define the approximating rank one trajectory $\W_1 ( t )$ emanating from~$\S$.
Let $\{ \widetilde{\w}^n (t) \}_{n = 1}^N$ be a curve born from gradient flow when minimizing $\phi_h (\cdot)$ with a one-component tensor factorization, initialized at:
\[
\widetilde{\w}^n (0) := \frac{ \rho^{1 / N} }{\normnoflex{ \widebar{\w}_1^n (0) }} \cdot \widebar{\w}_{1}^n (0) \quad , ~n = 1, \ldots, N
\text{\,.}
\]
Notice that by definition $\normnoflex{ \widetilde{\w}^1 (0) } = \cdots = \normnoflex{ \widetilde{\w}^N (0) } = \rho^{1 / N}$.
Therefore, $\{ \widetilde{\w}^n (0) \}_{n = 1}^N$ have unbalancedness magnitude zero (Definition~\ref{def:unbalancedness_magnitude}).
Denoting $\W_1 (t) := \tenp_{n = 1}^N \widetilde{\w}^n (t)$, for $t \geq 0$, we can see that $\W_1 (t)$ is a balanced rank one trajectory.
Furthermore, $\normnoflex{ \W_1 (0) } = \normnoflex{ \tenp_{n = 1}^N \widetilde{\w}^n (0) } = \prod_{n = 1}^N \normnoflex{ \widetilde{\w}^n (0) } = \rho$, meaning $\W_1 (0) \in \S$.
It will be convenient to treat $\{ \widetilde{\w}^n (t) \}_{n = 1}^N$ as an $R$-component factorization with components $2, \ldots, R$ being zero.
To this end, denote $\widetilde{\w}_1^n (t) := \widetilde{\w}^n (t)$, and define $\widetilde{\w}_r^n (t) := 0$ for all $t \geq 0$, $r \in \{ 2, \ldots, R \}$ and $n \in [N]$.
Notice that, according to Lemma~\ref{tf:lem:width_R_equivalent_to_larger_width_with_zero_init}, $\{ \widetilde{\w}_r^n (t) \}_{r = 1}^R\hspace{0mm}_{n = 1}^N$ indeed follow a gradient flow path of an $R$-component factorization.

Next, we turn to bound the distance between $\{ \widebar{\w}_r^n (0) \}_{r = 1}^R\hspace{0mm}_{n = 1}^N$ and $\{ \widetilde{\w}_r^n (0) \}_{r = 1}^R\hspace{0mm}_{n = 1}^N$.
From Equation~\eqref{tf:eq:approx_rank_1_stage_I_t_0_comp_bounds_upper_bound} in Proposition~\ref{tf:prop:end_tensor_reaches_ref_sphere}, recalling $\tilde{\epsilon} \leq \hat{\epsilon}$ (by their definition in Equation~\eqref{tf:eq:approx_rank_1_consts}), we obtain:
\be
\normnoflex{ \widebar{\w}_r^n (0)} = \normnoflex{ \w_r^n (T_0)} = \normnoflex{ \tenp_{n' = 1}^N \w_r^{n'} (T_0) }^{ \frac{1}{N} } \leq \tilde{\epsilon}^{ \frac{1}{N} } \leq \hat{\epsilon}^{ \frac{1}{N} } ~~ , ~r = 2, \ldots, R ~,~ n = 1, \ldots, N
\text{\,.}
\label{tf:eq:approx_rank_1_t_0_param_vec_bounds_other_compcs}
\ee
As for the first component, for any $n \in [N]$, the fact that $\normnoflex{\widebar{\w}_1^n (0)} = \normnoflex{ \w_1^n (T_0) } =  \normnoflex{ \tenp_{n' = 1}^N \w_1^{n'} (T_0 ) }^{1 / N}$ and Equation~\eqref{tf:eq:approx_rank_1_stage_I_t_0_first_comp_bound} from Proposition~\ref{tf:prop:end_tensor_reaches_ref_sphere} yield the following bound:
\[
\left ( \rho - (R - 1) \cdot \tilde{\epsilon} \right )^{\frac{1}{N}} \leq \norm{\widebar{\w}_1^n (0)} \leq \left ( \rho + (R - 1) \cdot \tilde{\epsilon} \right )^{\frac{1}{N}}
\text{\,.}
\]
On the one hand, since the $\ell_1$ norm is no greater than the $\ell_p$ norm for $p < 1$, we have that $( \rho + (R - 1) \cdot \tilde{\epsilon} )^{1 / N} \leq \rho^{1 / N} + (R - 1)^{1 / N} \cdot \tilde{\epsilon}^{1 / N} \leq \rho^{1 / N} + (R - 1)^{1 / N} \cdot \hat{\epsilon}^{1 / N}$.
On the other hand, since by definition $\tilde{\epsilon} \leq ( R - 1 )^{-1} ( \rho - [ \rho^{1 / N} - (R - 1)^{1 / N} \cdot \hat{\epsilon}^{1 / N} ]^{N} )$, it is straightforward to verify that $( \rho - (R - 1) \cdot \tilde{\epsilon}  )^{1 / N} \geq \rho^{1 / N} - (R - 1)^{1 / N} \cdot \hat{\epsilon}^{1 / N}$.
Put together, while noticing that $\normnoflex{ \widebar{\w}_1^n (0) - \widetilde{\w}_1^n (0) } = \abs{ \normnoflex{ \widebar{\w}_1^n (0) } - \rho^{1 / N} }$, we arrive at:
\be
\norm{ \widebar{\w}_1^n (0) - \widetilde{\w}_1^n (0) } = \abs{ \norm{\widebar{\w}_1^n (0)} - \rho^{ \frac{1}{N} } } \leq (R - 1)^{\frac{1}{N}} \cdot \hat{\epsilon}^{\frac{1}{N}} \quad , ~ n \in [N]
\text{\,.}
\label{tf:eq:approx_rank_1_t_0_param_vec_bounds_first_comp}
\ee
Equations~\eqref{tf:eq:approx_rank_1_t_0_param_vec_bounds_other_compcs} and~\eqref{tf:eq:approx_rank_1_t_0_param_vec_bounds_first_comp} lead to the following bound on the distance between $\{ \widebar{\w}_r^n (0) \}_{r = 1}^R\hspace{0mm}_{n = 1}^N$ and $\{ \widetilde{\w}_r^n (0) \}_{r = 1}^R\hspace{0mm}_{n = 1}^N$:
\[
\begin{split}
	\sum_{r = 1}^R \sum_{n = 1}^N \norm{ \widebar{\w}_r^n (0) - \widetilde{\w}_r^n (0) }^2 & = \sum_{n = 1}^N \norm{ \widebar{\w}_1^n (0) - \widetilde{\w}_1^n (0) }^2  + \sum_{r = 2}^R \sum_{n = 1}^N \norm{ \widebar{\w}_r^n (0)}^2 \\
	& \leq (R - 1)^{\frac{2}{N}} N \cdot \hat{\epsilon}^{\frac{2}{N}} + (R - 1) N \cdot \hat{\epsilon}^{ \frac{2}{N} } \\
	& \leq 2 (R - 1) N \cdot \hat{\epsilon}^{ \frac{2}{N} }
	\text{\,,}
\end{split}
\]
where the last transition is by $(R - 1)^{2 / N} \leq (R - 1)$.
Let $\widetilde{B} := \sqrt{N} \left ( \max \{ B, \rho \} + 1 \right )^{\frac{1}{N}}$ and $\beta := R N ( (\widetilde{B} + 1)^{2(N - 1)} + \delta_h (\widetilde{B} + 1)^{N - 2} )$ (as defined in Equation~\eqref{tf:eq:approx_rank_1_consts}).
According to Lemma~\ref{tf:lem:huber_cp_objective_is_smooth_over_bounded_domain}, the objective $\phi_h (\cdot)$ is $\beta$-smooth over the closed ball of radius $\widetilde{B} + 1$ around the origin.
Furthermore, seeing that $2 (R - 1) N \cdot \hat{\epsilon}^{2 / N} < \exp ( - 2 \beta \cdot T )$ (by the definition of $\hat{\epsilon}$ in Equation~\eqref{tf:eq:approx_rank_1_consts}), we obtain:
\[
\begin{split}
	\sum_{r = 1}^R \sum_{n = 1}^N \norm{ \widebar{\w}_r^n (0) - \widetilde{\w}_r^n (0) }^2 \leq 2 (R - 1) N \cdot \hat{\epsilon}^{ \frac{2}{N} } < \exp ( - 2 \beta \cdot T )
	\text{\,.}
\end{split}
\]
Thus, Lemma~\ref{tf:lem:gf_smooth_dist_bound} implies that at least until $t \geq T$ or $( \sum_{r = 1}^R \sum_{n = 1}^N \normnoflex{ \widetilde{\w} _r^n ( t)}^2 )^{1/2} \geq \widetilde{B}$ the following holds:
\be
\begin{split}
	\sum_{r = 1}^R \sum_{n = 1}^N \norm{ \widebar{\w}_r^n (t) - \widetilde{\w}_r^n (t) }^2 & \leq \sum_{r = 1}^R \sum_{n = 1}^N \norm{ \widebar{\w}_r^n (0) - \widetilde{\w}_r^n (0) }^2 \cdot \exp \left ( 2 \beta \cdot t \right ) \\
	& \leq 2 (R - 1) N \cdot \hat{\epsilon}^{\frac{2}{N}} \cdot \exp \left ( 2 \beta \cdot t \right )
	\text{\,.}
\end{split}
\label{tf:eq:approx_rank_1_t_0_sq_param_dist_at_end}
\ee

Suppose that $( \sum_{r = 1}^R  \sum_{n = 1}^N \normnoflex{ \widetilde{\w} _r^n ( t )}^2 )^{1/2} < \widetilde{B}$ for all $t \in [0, T]$.
In this case, Equation~\eqref{tf:eq:approx_rank_1_t_0_sq_param_dist_at_end} holds for all $t \in [0, T]$.
Seeing that $2 (R - 1) N \cdot \hat{\epsilon}^{2 / N} \cdot \exp \left ( 2 \beta \cdot T \right ) < 1$, Equation~\eqref{tf:eq:approx_rank_1_t_0_sq_param_dist_at_end} gives $( \sum_{r = 1}^R \sum_{n = 1}^N \normnoflex{ \widebar{\w}_r^n ( t )}^2 )^{1/2} < \widetilde{B} + 1$.
Then, Equation~\eqref{tf:eq:approx_rank_1_t_0_sq_param_dist_at_end}, the fact that $\W_1 (t) = \tenp_{n = 1}^N \widetilde{\w}_1^n (t) = \sum_{r = 1}^R \tenp_{n = 1}^N \widetilde{\w}_r^n (t)$, and Lemma~\ref{tf:lem:param_dist_to_end_to_end_dist} yield:
\[
\norm{ \tftensorendbar ( t ) - \W_1 (t) } \leq \sqrt{2} R N (\widetilde{B} + 1)^{N - 1} \cdot \exp \left ( \beta \cdot T \right ) \cdot \hat{\epsilon}^{ \frac{1}{N} } \quad , ~ t \in [0, T]
\text{\,.}
\]
Recalling that $\hat{\epsilon} \leq 2^{- \frac{N}{2}} R^{-N} N^{-N} (\widetilde{B} + 1)^{N - N^2} \cdot \exp ( -N \beta T) \cdot \epsilon^N$, we conclude:
\be
\norm{ \tftensorendbar ( t ) - \W_1 (t) } \leq \epsilon
\text{\,,}
\label{tf:eq:end_tensor_rank_1_epsilon_dist}
\ee
for all $t \in [0, T]$.

It remains to treat the case where $( \sum_{r = 1}^R  \sum_{n = 1}^N \normnoflex{ \widetilde{\w} _r^n ( t )}^2 )^{1/2} \geq \widetilde{B}$ for some $t \in [0, T]$.
Let $t' \in [ 0 , T ]$ be the initial such time.
The desired result readily follows by showing that: \emph{(i)} Equation~\eqref{tf:eq:end_tensor_rank_1_epsilon_dist} holds for $t \in [0, t']$; and \emph{(ii)} $\normnoflex{ \tftensorendbar ( t' )} \geq B$.

We start by proving that $\norm{ \W_1 ( t' ) } \geq \max \{ B, \rho \} + 1$ and $t' > 0$.
Recalling that $\widetilde{\w}_r^1 (t), \ldots, \widetilde{\w}_r^N (t)$ are identically zero for all $r \in \{ 2, \ldots, R \}$, we have that:
\[
\sum_{n = 1}^N \normnoflex{ \widetilde{\w} _1^n ( t' )}^2  =  \sum_{r = 1}^R  \sum_{n = 1}^N \normnoflex{ \widetilde{\w} _r^n ( t' )}^2 \geq \widetilde{B}^2
\text{\,.}
\]
Since $\normnoflex{ \widetilde{\w}_1^1 ( 0 ) } = \cdots = \normnoflex{ \widetilde{\w}_1^N ( 0 ) }$, Lemma~\ref{tf:lem:balancedness_conservation_body} implies $\normnoflex{ \widetilde{\w}_1^1 ( t' ) } = \cdots = \normnoflex{ \widetilde{\w}_1^N ( t' ) }$.
Thus, for any $n \in [N]$:
\[
N \normnoflex{ \widetilde{\w}_{1}^{n} (t') }^2 = \sum_{n' = 1}^N \normnoflex{ \widetilde{\w}_{1}^{n'} ( t' ) }^2 \geq \widetilde{B}^2
\text{\,,}
\]
which leads to $\normnoflex{ \widetilde{\w}_{1}^{n} ( t' ) } \geq \widetilde{B} N^{-1 / 2}$.
In turn this yields $\normnoflex{ \W_1 (t') } = \norm{ \tenp_{n = 1}^N \widetilde{\w}_{1}^{n} ( t' ) } = \prod_{n = 1}^N \norm{ \widetilde{\w}_{1}^{n} ( t' ) } \geq \widetilde{B}^N N^{- \frac{N}{2}}$.
Plugging in $\widetilde{B} := \sqrt{N} (\max \{ B, \rho \}  + 1)^{\frac{1}{N}}$, we conclude: 
\be
\normnoflex{ \W_1 ( t' ) } \geq \max \{ B, \rho \} + 1
\text{\,.}
\label{tf:eq:rank_1_traj_norm_lower_bound}
\ee
Note that this necessarily means $t' > 0$ as $\W_1 (0) \in \S$, \ie~$\normnoflex{ \W_1 ( 0 ) } = \rho < \max \{ B, \rho \} + 1$.

Now, we focus on the time interval $[0, t')$, over which Equation~\eqref{tf:eq:approx_rank_1_t_0_sq_param_dist_at_end} holds and $( \sum_{r = 1}^R  \sum_{n = 1}^N \normnoflex{ \widetilde{\w} _r^n ( t )}^2 )^{1/2} < \widetilde{B}$.
By arguments analogous to those used in the case where $( \sum_{r = 1}^R  \sum_{n = 1}^N \normnoflex{ \widetilde{\w} _r^n ( t )}^2 )^{1/2} < \widetilde{B}$ for all $t \in [0, T]$, we obtain that Equation~\eqref{tf:eq:end_tensor_rank_1_epsilon_dist} holds for all $t \in [0, t')$.
Continuity with respect to time then implies $\normnoflex{ \tftensorendbar ( t' ) - \W_1 (t') } \leq \epsilon < 1$.
Lastly, together with Equation~\eqref{tf:eq:rank_1_traj_norm_lower_bound} this leads to $\normnoflex{ \tftensorendbar ( t' )} \geq \normnoflex{ \W_1 (t') } - 1 \geq B$.

Overall, we have shown that $\normnoflex{\tftensorendbar ( t ) - \W_1 (t) } \leq \epsilon$ at least until time $T$ or time $t'$ at which $\normnoflex{\tftensorendbar ( t' )} \geq B$, establishing the desired result.
\qed

\subsection{Proof of Corollary~\ref{corollary:converge_rank_1}}
\label{tf:app:proofs:converge_rank_1}

For $\epsilon > 0$, there exists a time $T' > 0$ at which all balanced rank one trajectories emanating from $\S$ are within distance $\epsilon / 2$ from $\W^*$.
Moreover, these trajectories are confined to a ball of radius $B$ around the origin, for some $B > 0$.
According to Theorem~\ref{tf:thm:approx_rank_1}, if initialization scale $\alpha$ is sufficiently small, $\normnoflex{ \tftensorendbar (t) - \W_1 (t) } \leq \min \{ \epsilon / 2, 1 / 2 \}$ at least until $t \geq T'$ or $\normnoflex{ \tftensorendbar (t) } \geq B + 1$, where $\tftensorendbar (t)$ is the time-shifted trajectory of $\tftensorend (t)$, and $ \W_1 (t)$ is a balanced rank one trajectory emanating from $\S$.
We claim that the latter cannot hold, \ie~$\normnoflex{ \tftensorendbar (t) } < B + 1$ for all $t \in [0, T']$.
To see it is so, assume by way of contradiction otherwise, and let $t' \in [0, T']$ be the initial time at which $\normnoflex{ \tftensorendbar (t') } \geq B + 1$.
Since $\normnoflex{ \tftensorendbar (t') - \W_1 (t') } < 1$, we have that $\normnoflex{  \W_1 (t') } > B$, in contradiction to $\W_1 (t)$ being confined to a ball of radius $B$ around the origin.
Thus, $\normnoflex{ \tftensorendbar ( T' ) - \W_1 (T') } \leq \epsilon / 2$.
The proof concludes by the triangle inequality:
\[
\norm{ \tftensorendbar (T') - \W^* } \leq \norm{ \tftensorendbar (T') - \W_1 (T') } + \norm{ \W_1 (T') - \W^* } \leq \epsilon
\text{\,.}
\]
\qed

%% file: Appendices/imp_reg_htf.tex
\chapter{Implicit Regularization in Hierarchical Tensor Factorization \\ and Deep Convolutional Neural Networks} 
\label{htf:app:imp_reg_htf}

\section{Hierarchical Tensor Factorization as Deep Non-Linear Convolutional Network}
\label{htf:app:htf_cnn}

\begingroup
\setlength{\columnsep}{17pt}
\begin{wrapfigure}{r}{0.36\textwidth}
	\vspace{-4.6mm}
	\hspace*{-0.5mm}
	\includegraphics[width=0.31\textwidth]{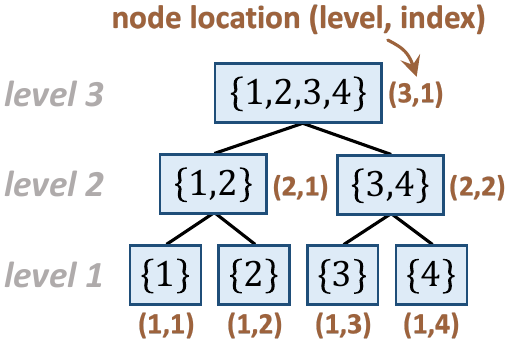}
	\vspace{-1mm}
	\caption{
		Perfect $P$-ary mode tree that combines adjacent indices, for order $N = 4$ and $P = 2$.
	}
	\vspace{-2mm}
	\label{htf:fig:pary_mode_tree_with_locs}
	\vspace{-3mm}
\end{wrapfigure}
In this appendix, we formally state and prove a known correspondence between hierarchical tensor factorization and certain deep non-linear convolutional networks (\cf~\cite{cohen2016expressive}).
For conciseness, we assume the tensor order $N$ is a power of $P \in \N_{\geq 2}$ and the mode dimensions $D_1, \ldots, D_N$ are equal, and focus on the factorization induced by a perfect $P$-ary mode tree (\defin~\ref{htf:def:mode_tree}) that combines nodes with adjacent indices.

Let $L := \log_P N$ denote the height of the mode tree, and associate each of its nodes with a respective location $(l, n)$, where $l \in [L + 1]$ is the level in the tree (numbered from leaves to root in ascending order), and $n \in [ N / P^{l - 1} ]$ is the index inside the level (see \fig~\ref{htf:fig:pary_mode_tree_with_locs} for an illustration).
Adapting \eq~\eqref{htf:eq:ht_end_tensor} to the current setting, the end tensor is computed as follows:
\be
\begin{split}
	&\text{for all $n \in [N]$ and $r \in [ R_{ 1 } ]$:} \\
	& \quad\underbrace{ \htftensorpart{1, n}{r} }_{\text{order $1$}} := \weightmat{1, n}_{:, r} \text{\,,} \\[1mm]
	&\text{for all $l \in \{ 2, \ldots, L \}, n \in [ N / P^{l - 1}]$, and $r \in [ R_{ l } ]$ (traverse interior nodes of $\htmodetree$ from} \\[-1.2mm]
	& \text{leaves to root, non-inclusive):} \\
	&\quad \underbrace{\htftensorpart{l, n}{r}}_{\text{order $P^{l - 1}$}} := \sum\nolimits_{r' = 1}^{R_{l - 1}} \weightmat{l, n}_{r', r} \left [ \tenp_{p = (n - 1) \cdot P + 1}^{n \cdot P} \htftensorpart{l - 1, p }{r'} \right ] \text{\,,} \\[0mm]
	&\underbrace{\tensorend}_{\text{order $N$}} := \sum\nolimits_{r' = 1}^{R_{L}} \weightmat{L + 1, 1}_{r', 1} \left [ \tenp_{p = 1}^{P} \htftensorpart{L, p }{r'} \right ]
	\text{\,,}
\end{split}
\label{htf:eq:ht_pary_end_tensor}
\ee
where $\big ( \weightmat{l, n} \in \R^{R_{l - 1} \times R_{l}} \big )_{l \in [L + 1], n \in [N / P^{l - 1}]}$ are the factorization's weight matrices, $R_{L + 1} = 1$, and $R_{0} := D_1 = \cdots =~D_N$.

The deep non-linear convolutional network corresponding to the above factorization (illustrated in \fig~\ref{htf:fig:tf_htf_as_convnet} (bottom)) has $L$ hidden layers, the $l$'th one comprising a locally connected linear operator with $R_{l}$ channels followed by channel-wise product pooling with window size $P$ (multiplicative non-linearity).
Denoting by $\bigl ( \hbf^{(l - 1, 1)}, \ldots,  \hbf^{(l - 1, N / P^{l - 1})} \bigr ) \in \R^{R_{l - 1}} \times \cdots \times \R^{R_{l - 1}}$ the output of the $l - 1$'th hidden layer, where $\bigl ( \hbf^{ (0, 1) }, \ldots,  \hbf^{ (0, N) } \bigr ) := \bigl ( \xbf^{ (1) }, \ldots, \xbf^{ (N) } \bigr )$ is the network's input, the locally connected operator of the $l$'th layer computes $\big ( \weightmat{l, n} \big )^\top \hbf^{(l - 1, n)}$ for each index $n \in [N / P^{l - 1}]$.
We refer to this operator as “$1 \times 1$ conv'' in appeal to the case of weight sharing, where $\weightmat{l, 1} = \cdots = \weightmat{l, N / P^{l - 1}}$.
Following the locally connected operator, for each $n \in [N / P^l]$ and $r \in [R_{l}]$, the pooling operator computes $\prod_{p = (n - 1) \cdot P + 1}^{n \cdot P}  \big [ \big ( \weightmat{l, p} \big )^\top \hbf^{(l - 1, p)} \big ]_r$, thereby producing $\big ( \hbf^{(l, 1)}, \ldots, \hbf^{(l, N / P^l)} \big )$.
After passing the input through all hidden layers, a final linear layer, whose weights are $\weightmat{L + 1, 1}$, yields the scalar output of the network $\bigl ( \weightmat{L + 1, 1} \bigr )^\top \hbf^{(L, 1)}$.
Notice that the weight matrices of the hierarchical tensor factorization are exactly the learnable weights of the network, and $R_{l - 1}$~---~the number of local components (\defin~\ref{htf:def:local_comp}) at nodes in level $l$ of the factorization~---~is the width of the network's $l - 1$'th hidden layer.

The above formulation of the network supports not only sequential inputs (\eg~audio and text), but also inputs arranged as multi-dimensional arrays (\eg~two-dimensional images).  
The choice of how to assign the indices $1, \ldots, N$ to input elements determines the geometry of pooling windows throughout the network~\cite{cohen2017inductive}.

\prop~\ref{htf:prop:htf_cnn} below implies that we may view solution of a prediction task using the deep convolutional network described above as a hierarchical tensor factorization problem, and vice versa.
For example, solving tensor completion and certain sensing problems using hierarchical tensor factorization amounts to applying the corresponding network to a regression task.
\endgroup
\begin{proposition}[adapted from~\cite{cohen2016expressive}]
	\label{htf:prop:htf_cnn}
	Let $f_\Theta : \times_{n = 1} \R^{D_n} \to \R$ be the function realized by the deep non-linear convolutional network described above, where $\Theta$ stands for the network's weights, \ie~$\Theta := \big ( \weightmat{l, n} \big )_{l \in [L + 1], n \in [N / P^{l - 1}]}$.
	Denote by $\tensorend$ the end tensor of the hierarchical tensor factorization specified in \eq~\eqref{htf:eq:ht_pary_end_tensor}.
	Then, for all $\xbf^{(1)} \in \R^{D_1}, \ldots, \xbf^{(N)} \in \R^{D_N}$:
	\[
	f_\Theta \big ( \xbf^{(1)}, \ldots, \xbf^{(N)} \big ) = \inprodbig{ \tenp_{n = 1}^N \xbf^{(n)}  }{ \tensorend }
	\text{\,.}
	\]
\end{proposition}
\begin{proof}[Proof sketch (proof in \subapp~\ref{htf:app:proofs:htf_cnn})]
	By induction over the layers of the network, we show that the output of the $l$'th convolutional layer (linear output layer for $l = L + 1$) at index $n$ and channel $r$ is $\inprodbig{ \tenp_{p = (n -1 ) \cdot P^{l - 1} + 1}^{n \cdot P^{l - 1}} \xbf^{(p)} }{ \htftensorpart{l, n}{r} }$, where \smash{$\htftensorpart{L + 1, 1}{1} := \tensorend$}, and all other $\htftensorpart{l,n}{r}$ are the intermediate tensors formed when computing $\tensorend$ according to \eq~\eqref{htf:eq:ht_pary_end_tensor}.
	Since $f_\Theta \big ( \xbf^{(1)}, \ldots, \xbf^{(N)} \big )$ is the output of the $L + 1$'th layer at index $1$ and channel $1$, applying the inductive claim for $l = L + 1, n = 1$, and $r = 1$ concludes the proof.
\end{proof}

We conclude this appendix by noting that in the special case where $P = N$, if the weight matrix of the root node holds ones, the hierarchical tensor factorization reduces to a tensor factorization, and the corresponding convolutional network has a single hidden layer (with global product pooling) followed by a final summation layer. 
We thus obtain the equivalence between tensor factorization and a shallow non-linear convolutional network as a corollary of \prop~\ref{htf:prop:htf_cnn}.


\section{Evolution of Local Component Norms Under Arbitrary Initialization}
\label{htf:app:dyn_arbitrary}

\thm~\ref{htf:thm:loc_comp_norm_bal_dyn} in \subsect~\ref{htf:sec:inc_rank_lrn:evolution} characterizes the evolution of local component norms in a hierarchical tensor factorization, under the assumption of unbalancedness magnitude zero at initialization.
\thm~\ref{htf:thm:loc_comp_norm_unbal_dyn} below extends the characterization to account for arbitrary initialization. 
It establishes that if the unbalancedness magnitude at initialization is small --- as is the case under any near-zero initialization~---~local component norms approximately evolve per \thm~\ref{htf:thm:loc_comp_norm_bal_dyn}.

\begin{theorem}
	\label{htf:thm:loc_comp_norm_unbal_dyn}
	With the context and notations of \thm~\ref{htf:thm:loc_comp_norm_bal_dyn}, assume unbalancedness magnitude $\epsilon \geq 0$ at initialization.
	Then, for any $\nu \in \interior (\htmodetree)$, $r \in [R_\nu]$, and time $t \geq 0$ at which $\htfcompnorm{\nu}{r} (t) > 0$:\footnote{
		Since norms are not differentiable at the origin, when $\htfcompnorm{\nu}{r} (t)$ is equal to zero it may not be differentiable with respect to time.
	}
	\begin{itemize}
		\item If $\inprodbig{- \nabla \htfendloss ( \tensorend (t) ) }{ \htfcomp{\nu}{r} (t) } \geq 0$, then:
		\be
		\begin{split}
			&\frac{d}{dt} \htfcompnorm{\nu}{r} (t) \leq \left ( \htfcompnorm{\nu}{r} (t)^{\frac{2}{  L_{\nu} } } + \epsilon \right )^{L_{\nu} - 1} \cdot L_{\nu} \inprodbig{- \nabla \htfendloss ( \tensorend (t) ) }{ \htfcomp{\nu}{r} (t) }
			\text{\,,} \\[1mm]
			&\frac{d}{dt} \htfcompnorm{\nu}{r} (t) \geq \frac{ \htfcompnorm{\nu}{r} (t)^2 }{ \htfcompnorm{\nu}{r} (t)^{ \frac{2}{  L_{\nu} } }  + \epsilon } \cdot  L_{\nu} \inprodbig{- \nabla \htfendloss ( \tensorend (t) ) }{ \htfcomp{\nu}{r} (t) }
			\text{\,;}
		\end{split}
		\label{htf:eq:loc_comp_norm_unbal_pos_bound}
		\ee
		\item otherwise, if $\inprodbig{- \nabla \htfendloss ( \tensorend (t) ) }{ \htfcomp{\nu}{r} (t) } < 0$, then:
		\be
		\begin{split}
			&\frac{d}{dt} \htfcompnorm{\nu}{r} (t) \geq \left ( \htfcompnorm{\nu}{r} (t)^{\frac{2}{ L_{\nu} } } + \epsilon \right )^{  L_{\nu} - 1 } \cdot  L_{\nu} \inprodbig{- \nabla \htfendloss ( \tensorend (t) ) }{ \htfcomp{\nu}{r} (t) }
			\text{\,,} \\[1mm]
			&\frac{d}{dt} \htfcompnorm{\nu}{r} (t) \leq \frac{ \htfcompnorm{\nu}{r} (t)^2 }{ \htfcompnorm{\nu}{r} (t)^{ \frac{2}{ L_{\nu} } }  + \epsilon } \cdot  L_{\nu} \inprodbig{- \nabla \htfendloss ( \tensorend (t) ) }{ \htfcomp{\nu}{r} (t) }
			\text{\,.}
		\end{split}
		\label{htf:eq:loc_comp_norm_unbal_neg_bound}
		\ee
	\end{itemize}
\end{theorem}
\begin{proof}[Proof sketch (proof in \subapp~\ref{htf:app:proofs:loc_comp_norm_unbal_dyn})]
	The proof follows a line similar to that of \thm~\ref{htf:thm:loc_comp_norm_bal_dyn}, except that here conservation of unbalancedness magnitude leads to $\normnoflex{ \wbf (t) }^2 \leq \htfcompnorm{\nu}{r} (t)^{ \frac{2}{ L_\nu } } + \epsilon$ for all $\wbf \in \localcomp (\nu, r)$.
	Applying this inequality to:
	\[
	\frac{d}{dt} \htfcompnorm{\nu}{r} (t) = \inprodbig{ - \nabla \htfendloss ( \tensorend (t)) }{ \htfcomp{\nu}{r} (t) } \sum\nolimits_{ \wbf \in \localcomp (\nu, r) } \prod\nolimits_{ \wbf' \in \localcomp (\nu, r) \setminus \{ \wbf \} } \norm{ \wbf' (t) }^2
	\text{\,,}
	\]
	yields \eqs~\eqref{htf:eq:loc_comp_norm_unbal_pos_bound} and~\eqref{htf:eq:loc_comp_norm_unbal_neg_bound}.
\end{proof}

\section{Hierarchical Tensor Rank as Measure of Long-Range Dependencies}
\label{htf:app:ht_sep_rank}

\subsect~\ref{htf:sec:low_htr_implies_locality} discusses the known fact by which the hierarchical tensor rank (\defin~\ref{htf:def:ht_rank}) of a hierarchical tensor factorization measures the strength of long-range dependencies modeled by the equivalent convolutional network (see~\cite{cohen2017inductive,levine2018benefits,levine2018deep}).
For the convenience of the reader, the current appendix formally explains this fact.

Consider a hierarchical tensor factorization with mode tree $\htmodetree$ (\defin~\ref{htf:def:mode_tree}), weight matrices \smash{$\Theta := \big ( \weightmat{\nu} \big )_{\nu \in \htmodetree}$}, and an equivalent convolutional network realizing a parametric input-output function $f_\Theta$.
As claimed in \sect~\ref{htf:sec:htf} (and formally justified in \app~\ref{htf:app:htf_cnn}), the function realized by the convolutional network takes the form $f_\Theta \big ( \xbf^{(1)}, \ldots, \xbf^{(N)} \big ) = \inprodbig{ \tenp_{n = 1}^N \xbf^{(n)}  }{ \tensorend }$, where $\tensorend$ stands for the end tensor of the factorization (\eq~\eqref{htf:eq:ht_end_tensor}).
\prop~\ref{htf:prop:matrank_eq_seprank} below establishes that for any subset of indices $I \subset [N]$, the matrix rank of~$\tensorend$'s matricization according to $I$ is equal to the separation rank (\defin~\ref{htf:def:sep_rank}) of $f_\Theta$ with respect to $I$, \ie~$\rank \mat{\tensorend}{I} = \htfseprank ( f_\Theta ; I)$.
In particular, the hierarchical tensor rank of $\tensorend$ with respect to $\htmodetree$~---~$\big ( \rank \mat{\tensorend}{\nu} \big )_{\nu \in \htmodetree  \setminus \{ [N] \} }$~---~amounts to $\big ( \htfseprank ( f_\Theta ; \nu )\big )_{\nu \in \htmodetree \setminus \{ [N] \} }$.
In the canonical case where nodes in $\htmodetree$ hold adjacent indices, the separation ranks of  $f_\Theta$ with respect to them measure the dependencies modeled between distinct areas of the input, \ie~the non-local (long-range) dependencies.

\begin{proposition}[adaptation of Claim 1 in~\cite{cohen2017inductive}]
	\label{htf:prop:matrank_eq_seprank}
	Consider a hierarchical tensor factorization with mode tree $\htmodetree$ (\defin~\ref{htf:def:mode_tree}) and weight matrices $\Theta := \big ( \weightmat{\nu} \big )_{\nu \in \htmodetree}$, and denote its end tensor by $\tensorend$ (\eq~\eqref{htf:eq:ht_end_tensor}). Let $f_\Theta : \times_{n = 1} \R^{D_n} \to \R$ be defined by $f_\Theta \big ( \xbf^{(1)}, \ldots, \xbf^{(N)} \big ) := \inprodbig{ \tenp_{n = 1}^N \xbf^{(n)}  }{ \tensorend }$.
	Then, for all $I \subset [ N ]$:
	\[
	\rank \mat{\tensorend}{ I} = \htfseprank (f_\Theta ; I)
	\text{\,.}
	\]
\end{proposition}
\vspace{-3mm}
\begin{proof}[Proof sketch (proof in \subapp~\ref{htf:app:proofs:matrank_eq_seprank})]
	To prove that $\rank \mat{\tensorend }{ I } \geq \htfseprank (f_\Theta ; I)$, we derive a representation of $f_\Theta$ as a sum of $\rank \mat{\tensorend}{I}$ terms, each being a product between a function that operates over $( \xbf^{(i)} )_{i \in I}$ and another that operates over the remaining input variables.
	For the converse, $\rank \mat{\tensorend}{I} \leq \htfseprank (f_\Theta ; I)$, we prove that for any grid tensor $\W$ of a function $f$, \ie~tensor holding the outputs of $f$ over a grid of inputs, it holds that $\rank \mat{\W}{I} \leq \htfseprank (f; I)$.
	We conclude by showing that $\tensorend$ is a grid tensor of~$f_\Theta$.
\end{proof}

\section{Further Experiments and Implementation Details}
\label{htf:app:experiments}

\subsection{Further Experiments} \label{htf:app:experiments:further}

\figs~\ref{htf:fig:tc_o4_p2},~\ref{htf:fig:ts_o4_p2}, and~\ref{htf:fig:tc_o9_p3} supplement \fig~\ref{htf:fig:mf_tf_htf_dynamics} by including, respectively: \emph{(i)} plots of additional local component norms and singular values during optimization in the experiment presented by \fig~\ref{htf:fig:mf_tf_htf_dynamics} (right); \emph{(ii)} experiments with tensor sensing loss; and \emph{(iii)} experiments with different hierarchical tensor factorization orders and mode trees, as well as different ground truth hierarchical tensor ranks.
\fig~\ref{htf:fig:long_range_and_reg_results_resnet34} portrays an experiment identical to that of \fig~\ref{htf:fig:long_range_and_reg_results}, but with ResNet34 in place of ResNet18.  
\figs~\ref{htf:fig:long_range_other_reg_results_resnet18} and~\ref{htf:fig:long_range_other_reg_results_resnet34} extend \figs~\ref{htf:fig:long_range_and_reg_results} and~\ref{htf:fig:long_range_and_reg_results_resnet34}, respectively, by presenting results obtained with baseline networks that are already regularized using standard techniques (weight decay and dropout).

\subsection{Implementation Details} \label{htf:app:experiments:details}

In this subappendix we provide implementation details omitted from our experimental reports (\fig~\ref{htf:fig:mf_tf_htf_dynamics}, \sect~\ref{htf:sec:countering_locality}, and \subapp~\ref{htf:app:experiments:further}).
Source code for reproducing our results and figures, based on the PyTorch framework~\cite{paszke2019pytorch}, can be found at \url{https://github.com/asafmaman101/imp_reg_htf}.
All experiments were run on a single Nvidia RTX 2080 Ti GPU.

\subsubsection{Incremental Hierarchical Tensor Rank Learning (\figs~\ref{htf:fig:mf_tf_htf_dynamics},~\ref{htf:fig:tc_o4_p2},~\ref{htf:fig:ts_o4_p2}, and \ref{htf:fig:tc_o9_p3})}
\label{htf:app:experiments:details:inc_rank_lrn}

\textbf{\fig~\ref{htf:fig:mf_tf_htf_dynamics} (left):} the minimized matrix completion loss was
\[
\mfendloss ( \matrixend ) = \frac{1}{\abs{\Omega}} \sum\nolimits_{(i, j) \in \Omega} ( (\matrixend)_{i, j} - \Wbf^*_{i, j} )^2
\text{\,,}
\]
where $\Omega$ denotes a set of $2048$ observed entries chosen uniformly at random (without repetition) from a matrix rank $5$ ground truth $\Wbf^* \in \R^{64 \times 64}$.
We generated $\Wbf^*$ by computing $\Wbf^{* (1) } \Wbf^{* (2) }$, with each entry of $\Wbf^{* (1) } \in \R^{64 \times 5}$ and $\Wbf^{* (2)} \in \R^{5 \times 64}$ drawn independently from the standard normal distribution, and subsequently normalizing the result to be of Frobenius norm $64$ (square root of its number of entries).
Reconstruction error with respect to $\Wbf^*$ is based on normalized Frobenius distance, \ie~for a solution $\matrixend$ it is $\norm{ \matrixend - \Wbf^* } / \norm{ \Wbf^* }$.
The matrix factorization applied to the task was of depth $3$ and had hidden dimensions $64$ between its layers so that its rank was unconstrained.
Standard deviation for initialization was set to $0.001$.

\textbf{\fig~\ref{htf:fig:mf_tf_htf_dynamics} (middle):} the minimized tensor completion loss was:
\[
\tfendloss ( \tftensorend ) = \frac{1}{\abs{\Omega}} \sum\nolimits_{(d_1, d_2, d_3) \in \Omega} ( (\tftensorend)_{d_1, d_2, d_3} - \W^*_{d_1,  d_2, d_3} )^2
\text{\,,}
\]
where $\Omega$ denotes a set of $2048$ observed entries chosen uniformly at random (without repetition) from a tensor rank $5$ ground truth $\W^* \in \R^{16 \times 16 \times 16}$.
We generated $\W^*$ by computing $\sum_{r = 1}^5 \Wbf^{* (1)}_{:, r} \tenp \Wbf^{* (2)}_{:, r} \tenp \Wbf^{* (3)}_{:, r}$, with each entry of $\Wbf^{* (1)}, \Wbf^{* (2)},$ and $\Wbf^{* (3)} \in \R^{16 \times5}$ drawn independently from the standard normal distribution, and subsequently normalizing the result to be of Frobenius norm $64$ (square root of its number of entries).
Reconstruction error with respect to $\W^*$ is based on normalized Frobenius distance, \ie~for a solution $\tftensorend$ it is $\norm{ \tftensorend - \W^* } / \norm{ \W^* }$.
The tensor factorization applied to the task had $R = 256$ components so that its tensor rank was unconstrained.\footnote{
	For any $D_1, \ldots, D_N \in \N$, setting $R = ( \prod_{n = 1}^N D_n) / \max \{ D_n \}_{n = 1}^N$ suffices for expressing all tensors in $\R^{D_1 \times \cdots \times D_N}$ (Lemma~3.41 in~\cite{hackbusch2012tensor}).
} Standard deviation for initialization was set to~$0.001$.

\begin{figure*}[t!]
	\begin{center}
		\hspace{-3.8mm}
		\includegraphics[width=1\textwidth]{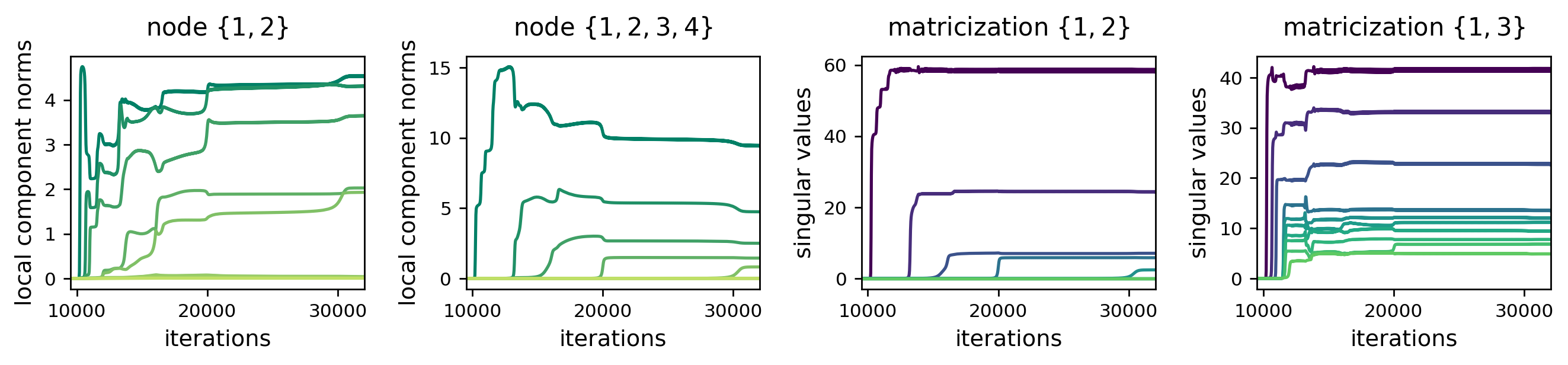}
	\end{center}
	\vspace{-2mm}
	\caption{
		Dynamics of gradient descent over order four hierarchical tensor factorization with a perfect binary mode tree (on tensor completion task)~---~incremental learning leads to low hierarchical tensor rank. 
		For the hierarchical tensor factorization experiment in \fig~\ref{htf:fig:mf_tf_htf_dynamics} (right), plots present the evolution of additional quantities during optimization.
		\textbf{Left and second to left:} top $10$ local component norms at nodes $\{ 1, 2 \}$ and $\{ 1, 2, 3, 4\}$ (respectively) in the mode tree (the latter also appears in \fig~\ref{htf:fig:mf_tf_htf_dynamics} (right)).
		\textbf{Second to right and right:} top $10$ singular values of the end tensor's matricizations according to $\{ 1, 2\}$ and $\{ 1, 3 \}$ (respectively).
		The former corresponds to a node in the mode tree, meaning its rank is part of the end tensor's hierarchical tensor rank, whereas the latter does not.
		\textbf{All:} notice that, in line with our analysis (\sect~\ref{htf:sec:inc_rank_lrn}), local component norms move slower when small and faster when large, creating an incremental process that leads to low hierarchical tensor rank solutions.
		Moreover, the singular values of the end tensor's matricizations according to nodes in the mode tree exhibit a similar behavior, whereas those of matricizations according to index sets outside the mode tree do not.
		The rank of a matricization lower bounds the (non-hierarchical) tensor rank (Remark~6.21 in~\cite{hackbusch2012tensor}). Thus, while the hierarchical tensor rank of the obtained solution is low, its tensor rank is high.
		For further implementation details, such as loss definition and factorization size, see \subapp~\ref{htf:app:experiments:details}.
	}
	\label{htf:fig:tc_o4_p2}
\end{figure*}

\begin{figure*}[t!]
	\vspace{3mm}
	\begin{center}
		\hspace{-3.8mm}
		\includegraphics[width=1\textwidth]{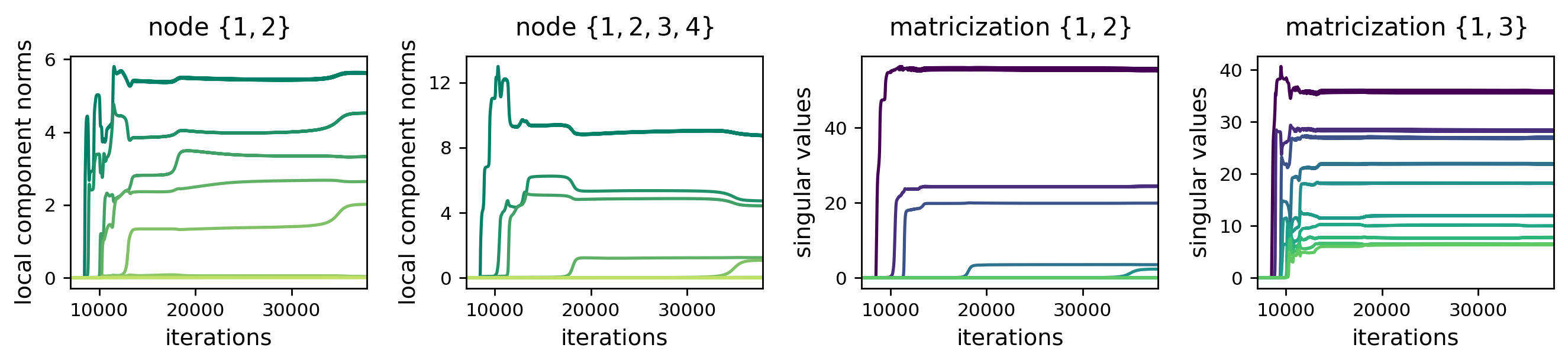}
	\end{center}
	\vspace{-2mm}
	\caption{
		Dynamics of gradient descent over order four hierarchical tensor factorization with a perfect binary mode tree (on tensor sensing task)~---~incremental learning leads to low hierarchical tensor rank.
		This figure is identical to \fig~\ref{htf:fig:tc_o4_p2}, except that the minimized mean squared error was based on random linear measurements (instead of randomly chosen entries).
		For further implementation details, such as loss definition and factorization size, see \subapp~\ref{htf:app:experiments:details}.
	}
	\label{htf:fig:ts_o4_p2}
\end{figure*}

\begin{figure*}[t!]
	\vspace{3mm}
	\begin{center}
		\hspace{-3.8mm}
		\includegraphics[width=1\textwidth]{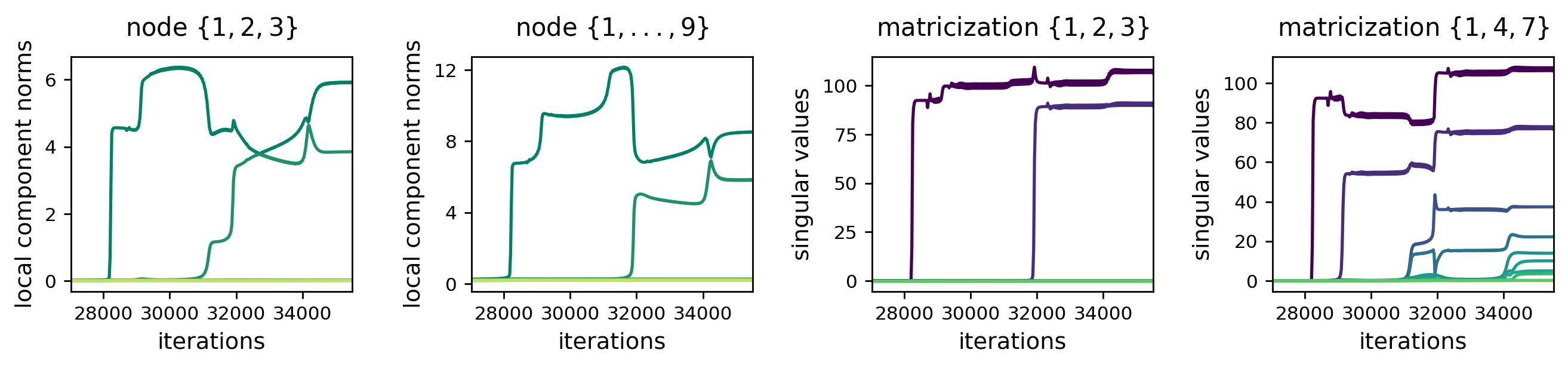}
	\end{center}
	\vspace{-2mm}
	\caption{
		Dynamics of gradient descent over order nine hierarchical tensor factorization with a perfect ternary mode tree~---~incremental learning leads to low hierarchical tensor rank.
		This figure is identical to \fig~\ref{htf:fig:tc_o4_p2}, except that: \emph{(i)} the hierarchical tensor factorization employed had order nine and complied with a perfect ternary mode tree; and \emph{(ii)} the ground truth tensor was of hierarchical tensor rank $(2, \ldots, 2)$ (\defin~\ref{htf:def:ht_rank}).
		For further implementation details, such as loss definition and factorization size, see \subapp~\ref{htf:app:experiments:details}.
	}
	\label{htf:fig:tc_o9_p3}
\end{figure*}

\textbf{\fig~\ref{htf:fig:mf_tf_htf_dynamics} (right):} the minimized tensor completion loss was:
\[
\htfendloss ( \tensorend ) = \frac{1}{\abs{\Omega}} \sum\nolimits_{(d_1, \ldots, d_4) \in \Omega} ( (\tensorend)_{d_1, \ldots, d_4} - \W^*_{d_1,  \ldots, d_4} )^2
\text{\,,}
\]
where $\Omega$ denotes a set of $2048$ observed entries chosen uniformly at random (without repetition) from a hierarchical tensor rank $(5, 5, 5, 5, 5, 5)$ ground truth $\W^* \in \R^{8 \times 8 \times 8 \times 8}$.
We generated $\W^*$ according to \eq~\eqref{htf:eq:ht_end_tensor} using a perfect binary mode tree $\htmodetree$ over $[4]$ and weight matrices $\big ( \Wbf^{*(\nu)} \big )_{\nu \in \htmodetree}$, where $\Wbf^{* (\nu) } \in \R^{8 \times 5}$ for $\nu \in \{ \{1\}, \ldots, \{4\} \}$, $\Wbf^{* (\nu) } \in \R^{5 \times 5}$ for $\nu \in \interior (\htmodetree) \setminus \{ [4] \}$, and $\Wbf^{* ( [4] ) }  \in \R^{5 \times 1}$.
We sampled the entries of $\big ( \Wbf^{* (\nu) } \big )_{\nu \in \htmodetree}$ independently from the standard normal distribution, and subsequently normalized the ground truth to be of Frobenius norm $64$ (square root of its number of entries).
Reconstruction error with respect to $\W^*$ is based on normalized Frobenius distance, \ie~for a solution $\tensorend$ it is $\norm{ \tensorend - \W^* } / \norm{ \W^* }$.
The hierarchical tensor factorization applied to the task had $512$ local components at all interior nodes due to computational and memory considerations (increasing the number of local components had no substantial impact on the dynamics).
Standard deviation for initialization was set to~$0.01$.

\textbf{\fig~\ref{htf:fig:tc_o4_p2}:} plots correspond to the same experiment presented in \fig~\ref{htf:fig:mf_tf_htf_dynamics} (right).

\textbf{\fig~\ref{htf:fig:ts_o4_p2}:} implementation details are identical to those of \fig~\ref{htf:fig:mf_tf_htf_dynamics} (right), except that the following tensor sensing loss was minimized:
\[
\htfendloss ( \tensorend ) = \sum\nolimits_{i = 1}^{2048} \big ( \inprodbig{ \tenp_{n = 1}^4 \xbf^{(i, n)} }{ \tensorend} - \inprodbig{ \tenp_{n = 1}^4 \xbf^{(i, n)} }{\W^*} \big )^2
\text{\,,}
\]
where the entries of $\big ( ( \xbf^{(i, 1)}, \ldots, \xbf^{(i, 4)} ) \in \R^8 \times \cdots \times \R^8 \big )_{i = 1}^{2048}$ were sampled independently from a zero-mean Gaussian distribution with standard deviation $4096^{- 1 / 8}$ (ensures each measurement tensor $\tenp_{n = 1}^4 \xbf^{(i, n)}$ has expected square Frobenius norm $1$).

\textbf{\fig~\ref{htf:fig:tc_o9_p3}:} implementation details are identical to those of \fig~\ref{htf:fig:mf_tf_htf_dynamics} (right), except that: \emph{(i)} the ground truth tensor was of order $9$ with modes of dimension $3$, Frobenius norm $\sqrt{19683}$ (square root of its number of entries), hierarchical tensor rank $(2, \ldots, 2)$, and was generated according to a perfect ternary mode tree; \emph{(ii)} reconstruction was based on $9840$ entries chosen uniformly at random; \emph{(iii)} the hierarchical tensor factorization applied to the task had $100$ local components at all  interior nodes; and \emph{(iv)} standard deviation for initialization was set to $0.1$.

\textbf{All:} using sample sizes smaller than those specified above led to similar results, up until a point where solutions found had fewer non-zero singular values, components, or local components (at all nodes) than the ground truths.
Gradient descent was initialized randomly by sampling each weight in the factorization independently from a zero-mean Gaussian distribution, and was run until the loss remained under $5 \cdot 10^{-5}$ for $100$ iterations in a row.
For each figure, experiments were carried out with initialization standard deviations $0.1, 0.05, 0.01, 0.005, 0.001,$ and $0.0005$.
Reported are representative runs striking a balance between the potency of the incremental learning effect and run time.
Reducing standard deviations further did not yield a significant change in the dynamics, yet resulted in longer optimization times due to vanishing gradients around the origin.
To facilitate more efficient experimentation, we employed the adaptive learning scheme described in Appendix~\ref{mf:app:experiments:details:tf}.

\subsubsection{Countering Locality of Convolutional Networks via Regularization (\figs~\ref{htf:fig:long_range_and_reg_results},~\ref{htf:fig:long_range_and_reg_results_resnet34},~\ref{htf:fig:long_range_other_reg_results_resnet18}, and~\ref{htf:fig:long_range_other_reg_results_resnet34})} 
\label{htf:app:experiments:details:conv}

In all experiments, we randomly initialized the ResNet18 and ResNet34 networks according to the default PyTorch~\cite{paszke2019pytorch} implementation.
The (regularized) binary cross-entropy loss was minimized via stochastic gradient descent with learning rate $0.01$, momentum coefficient $0.9$, and batch size $64$ (for ResNet34 we used a batch size of $32$ and accumulated gradients over two batches due to GPU memory considerations).
Optimization proceeded until perfect training accuracy was attained for $20$ consecutive epochs or $150$ epochs elapsed (runs without regularization always reached perfect training accuracy).
For each dataset and model combination, runs were carried out using the regularization described in \subsect~\ref{htf:sec:countering_locality:reg} with coefficients $0, 0.1, 0.5, 1, 3, 6, 9,$ and $10$.
Values lower than those reported in \figs~\ref{htf:fig:long_range_and_reg_results} and~\ref{htf:fig:long_range_and_reg_results_resnet34} had no noticeable impact, whereas higher values typically did not allow fitting the training data.
\tab~\ref{tab:other_reg_hyperparams} specifies the hyperparameters used for the different regularizations in the experiments of~\figs~\ref{htf:fig:long_range_other_reg_results_resnet18} and~\ref{htf:fig:long_range_other_reg_results_resnet34}.
Dropout layers shared the same probability hyperparameter, and were inserted before blocks expanding the number of channels, \ie~before the first convolutional layers with $128$, $256$, and $512$ output channels (the default ResNet18 and ResNet34 implementations do not include dropout).

\begin{figure*}[t]
	\begin{center}
		\hspace{-3mm}
		\includegraphics[width=1\textwidth]{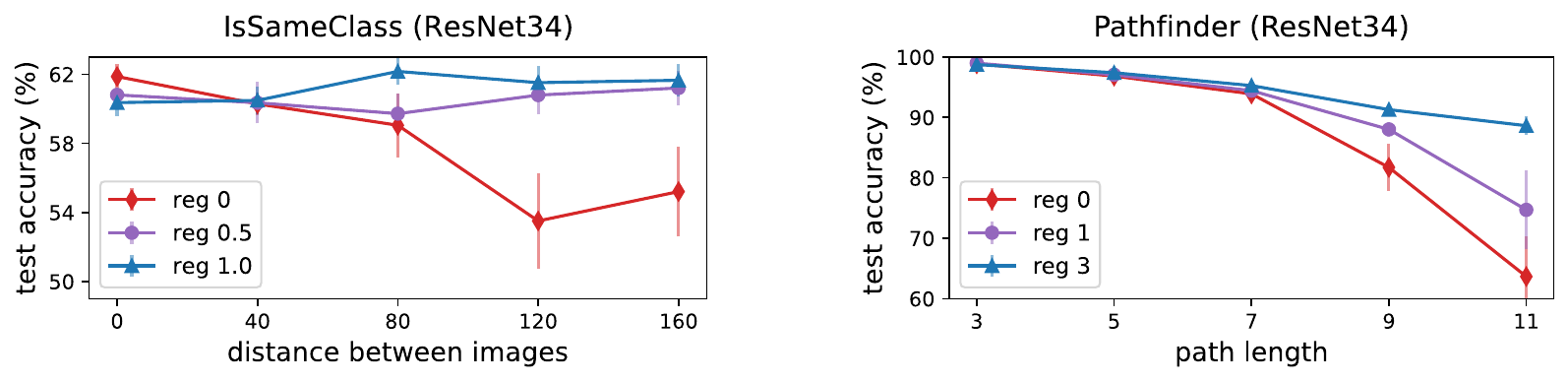}
	\end{center}
	\vspace{-3mm}
	\caption{	
		Dedicated explicit regularization can counter the locality of convolutional networks, significantly improving performance on tasks with long-range dependencies.
		This figure is identical to \fig~\ref{htf:fig:long_range_and_reg_results}, except that: \emph{(i)} experiments were carried out using a randomly initialized ResNet34 (as opposed to ResNet18); and \emph{(ii)} it includes evaluation over a Pathfinder dataset with path length $11$, since up until path length $9$ an unregularized network still obtained non-trivial performance.
		For further details see \subapp~\ref{htf:app:experiments:details:conv}.
	}
	\label{htf:fig:long_range_and_reg_results_resnet34}
\end{figure*}

\begin{figure*}[t!]
	\vspace{1mm}
	\begin{center}
		\hspace{-3mm}
		\includegraphics[width=1\textwidth]{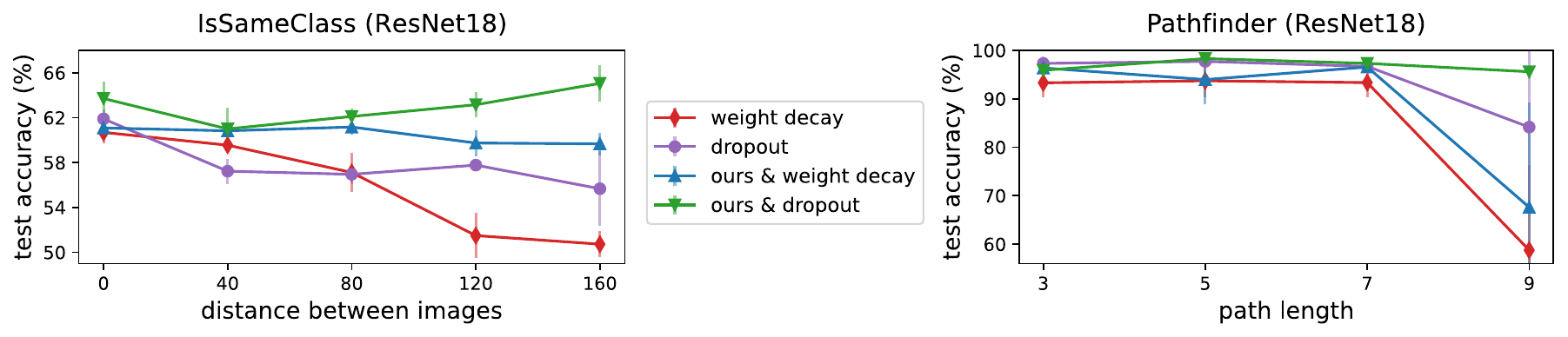}
	\end{center}
	\vspace{-3mm}
	\caption{	
		Dedicated explicit regularization can counter the locality of convolutional networks (regularized via standard techniques), significantly improving performance on tasks with long-range dependencies.
		This figure is identical to \fig~\ref{htf:fig:long_range_and_reg_results}, except that instead of applying our regularizer (\subsect~\ref{htf:sec:countering_locality:reg}) to a baseline unregularized network, the baseline networks here were regularized using either weight decay or dropout, and are compared to the results obtained when applying our regularization in addition to them.
		\fig~\ref{htf:fig:long_range_and_reg_results} shows that the test accuracy obtained by an unregularized network substantially deteriorates when increasing the (spatial) range of dependencies required to be modeled.
		From the plots above it is evident that, even when employing standard regularization techniques such as weight decay or dropout, a similar degradation in performance occurs.
		As was the case for unregularized networks, applying our dedicated regularization, in addition to these techniques, significantly improved performance.
		In particular, for the combination of our regularization and dropout, the test accuracy was high across all datasets.
		For further details such as regularization hyperparameters, see \subapp~\ref{htf:app:experiments:details:conv}.
	}
	\label{htf:fig:long_range_other_reg_results_resnet18}
\end{figure*}

\begin{figure*}[t!]
	\vspace{0mm}
	\begin{center}
		\hspace{-3mm}
		\includegraphics[width=1\textwidth]{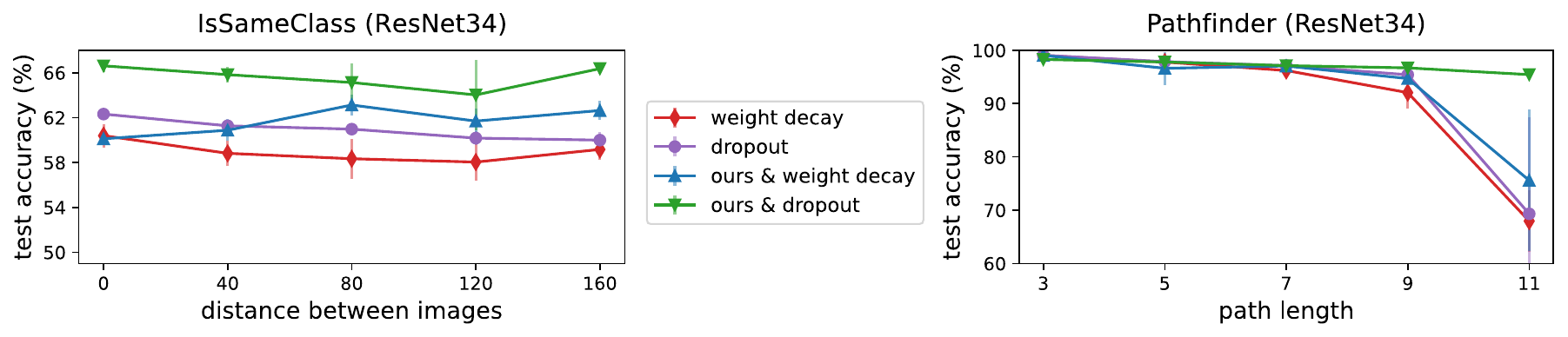}
	\end{center}
	\vspace{-3mm}
	\caption{	
		Dedicated explicit regularization can counter the locality of convolutional networks (regularized via standard techniques), significantly improving performance on tasks with long-range dependencies.
		This figure is identical to \fig~\ref{htf:fig:long_range_other_reg_results_resnet18},  except that: \emph{(i)} experiments were carried out using a randomly initialized ResNet34 (as opposed to ResNet18); and \emph{(ii)} it includes evaluation over a Pathfinder dataset with path length $11$, since up until path length $9$ networks regularized using weight decay or dropout still obtained non-trivial performance.
		For further details such as regularization hyperparameters, see \subapp~\ref{htf:app:experiments:details:conv}.	
	}
	\label{htf:fig:long_range_other_reg_results_resnet34}
\end{figure*}

At each stochastic gradient descent iteration, the subset of indices $I$ and $J$ used for computing the regularized objective were sampled as follows.
For IsSameClass datasets, we set $I$ to be the indices marking either the left or right CIFAR10 image uniformly at random, and then let $J$ be the indices corresponding to the remaining CIFAR10 image.
For Pathfinder datasets, $I$ and $J$ were set to non-overlapping $2 \times 2$ patches chosen uniformly across the input.
In order to prevent additional computational overhead, alternative values for pixels indexed by $J$ were taken from other images in the batch (as opposed to from the whole training set).
Specifically, we used a permutation without fixed points to shuffle the pixel patches indexed by~$J$ across the batch.

IsSameClass datasets consisted of $5000$ training and $10000$ test samples.
Each sample was generated by first drawing uniformly at random a label from $\{0, 1\}$ and an image from CIFAR10.
Then, depending on the chosen label, another image was sampled either from the same class (for label $1$) or from all other classes (for label $0$).
Lastly, the CIFAR10 images were placed at a predetermined horizontal distance from each other around the center of a $224 \times 224$ image filled with zeros.
For example, when the horizontal distance is $0$, the CIFAR10 images are adjacent, and when it is $160$, they reside in opposite borders of the $224 \times 224$ input.
Pathfinder datasets consisted of $10000$ training and $10000$ test samples.
Given a path length, the corresponding dataset was generated according to the protocol of~\cite{linsley2018learning}, with hyperparameters: circle radius $3$, paddle length $5$, paddle thickness $2$, inner paddle margin $3$, and continuity $1.8$.
See~\cite{linsley2018learning} for additional information regarding the data generation process.
As to be expected, when running a subset of all experiments using larger training set sizes (for both IsSameClass and Pathfinder datasets), we observed improved generalization across the board.
Nevertheless, the addition of training samples did not alleviate the degradation in test accuracy observed for larger horizontal distances and path lengths, nor did it affect the beneficial impact of our regularization.
That is, the trends observed in \figs~\ref{htf:fig:long_range_and_reg_results},~\ref{htf:fig:long_range_and_reg_results_resnet34},~\ref{htf:fig:long_range_other_reg_results_resnet18}, and~\ref{htf:fig:long_range_other_reg_results_resnet34} remained intact up to a certain shift upwards.

\begin{table*}[t]
	\caption{
		Hyperparameters for the regularizations employed in the experiments of~\figs~\ref{htf:fig:long_range_other_reg_results_resnet18} and~\ref{htf:fig:long_range_other_reg_results_resnet34}.
		For every model and dataset type combination, table reports the weight decay coefficient and dropout probability used when applied individually, as well as when combined with our regularization (described in \subsect~\ref{htf:sec:countering_locality:reg}), whose coefficients are also specified.
		These hyperparameters were tuned on the datasets with largest spatial range between salient regions of the input.
		That is, for each model separately, their values on IsSameClass datasets were set to those achieving the best test accuracy over a dataset with $160$ pixels between CIFAR10 images.
		Similarly, their values on Pathfinder datasets were set to those achieving the best test accuracy over a dataset with connecting path length $9$ for ResNet18 and path length $11$ for ResNet34.
		For further details see the captions of \figs~\ref{htf:fig:long_range_other_reg_results_resnet18} and~\ref{htf:fig:long_range_other_reg_results_resnet34}, as well as \subapp~\ref{htf:app:experiments:details:conv}.
	}
	\label{tab:other_reg_hyperparams}
	\vspace{-2mm}
	\begin{center}
		\begin{small}
			\begin{tabular}{lcccc}
				\toprule
				& \multicolumn{2}{c}{ResNet18} & \multicolumn{2}{c}{ResNet34} \\
				\cmidrule(lr){2-3}\cmidrule(lr){4-5}
				& IsSameClass & Pathfinder & IsSameClass & Pathfinder \\
				\midrule
				Weight Decay & $0.001$ & $0.01$ & $0.01$ & $0.001$ \\
				Dropout & $0.6$ & $0.5$ & $0.3$ & $0.2$ \\
				Ours \& Weight Decay & $1$ ~\&~ $0.001$ & $0.1$ ~\&~ $0.01$ & $1$ ~\&~ $0.0001$ & $0.1$ ~\&~ $0.001$ \\
				Ours \& Dropout & $1$ ~\&~ $0.5$ & $0.1$ ~\&~ $0.4$ & $1$ ~\&~ $0.5$ & $0.5$ ~\&~ $0.3$ \\
				\bottomrule
			\end{tabular}
		\end{small}
	\end{center}
\end{table*}

\section{Deferred Proofs}
\label{htf:app:proofs}

\subsection{Additional Notation}
\label{htf:app:proofs:notation}

Before delving into the proofs, we introduce the following notation.

\vspace{-2mm}

\textbf{General.}~~A colon is used to indicate a range of entries in a mode, \eg~$\Wbf_{i, :} \in \R^{D'}$ and $\Wbf_{:, j} \in \R^D$ are the $i$'th row and $j$'th column of $\Wbf \in \R^{D \times D'}$, respectively, and $\Wbf_{:i, :j} \in \R^{i \times j}$ is the sub-matrix of $\Wbf$ consisting of its first $i$ rows and $j$ columns.
For $\W \in \R^{D_1 \times \cdots \times D_N}$, we let $\vectnoflex{\W} \in \R^{\prod_{n = 1}^N D_n}$ be its arrangement as a vector.
The tensor and Kronecker products are denoted by $\tenp$ and $\kronp$, respectively.

\textbf{Hierarchical tensor factorization.}~~For a mode tree $\htmodetree$ over $[N]$ (\defin~\ref{htf:def:mode_tree}), we denote the set of nodes in the sub-tree of~$\htmodetree$ whose root is $\nu \in \htmodetree$ by $\subtree (\nu) \subset \htmodetree$.
The sets of left and right siblings of $\nu \in \htmodetree$ are denoted by $\lsib (\nu)$ and $\rsib (\nu)$, respectively.
For $\nu \in \htmodetree$, we let $\htftensorpart{\nu}{:}$ be the tensor obtained by stacking $\big ( \htftensorpart{\nu}{r} \big )_{r = 1}^{R_{ \parent (\nu) } }$ into a single tensor, \ie~$\htftensorpart{\nu}{:}_{:, \ldots, :, r} = \htftensorpart{\nu}{r}$ for all $r \in [ R_{\parent (\nu)} ]$.
Given weight matrices $\big ( \weightmat{\nu} \in \R^{R_\nu \times R_{ \parent (\nu) } } \big )_{\nu \in \htmodetree}$, the function mapping them to the end tensor they produce according to \eq~\eqref{htf:eq:ht_end_tensor} is denoted by $\tensorendmap \big ( \big ( \weightmat{ \nu } \big )_{ \nu \in \htmodetree } \big )$.
For $\nu \in \htmodetree$, with slight abuse of notation we let $\tensorendmap \big ( \big ( \weightmat{ \nu' } \big )_{ \nu' \in \htmodetree \setminus \subtree (\nu)}, \htftensorpart{\nu}{:} \big )$ be the function mapping $\big ( \htftensorpart{\nu}{r} \big )_{r = 1}^{ R_{ \parent(\nu) } }$ and weight matrices outside of $\subtree (\nu)$ to the end tensor they produce.

\subsection{Useful Lemmas}
\label{htf:app:proofs:useful_lemmas}

\subsubsection{Technical}
\label{htf:app:proofs:useful_lemmas:technical}

\begin{lemma}
	\label{htf:lem:tenp_eq_kronp}
	For any $\U \in \R^{D_1 \times \cdots \times D_N}, \V \in \R^{ H_1 \times \cdots \times H_K }$, and $I \subset [N + K]$:
	\[
	\mat{ \U \tenp \V }{ I}= \mat{ \U }{ I \cap [N] } \kronp \mat{ \V }{ I - N \cap [K]}
	\text{\,,}
	\]
	where $I - N := \{ i - N : i \in I \}$.
\end{lemma}
\begin{proof}
	The identity follows directly from the definitions of the tensor and Kronecker products.
\end{proof}

\begin{lemma}
	\label{htf:lem:kronp_mixed_prod_row_vec}
	For any $\Ubf \in \R^{D_1 \times D_2}, \Vbf \in \R^{D_2 \times  D_3}$, and $\wbf \in \R^{D_4}$, the following holds:
	\[
	\left ( \Ubf \Vbf \right ) \kronp \wbf^\top = \Ubf \left ( \Vbf \kronp \wbf^\top \right ) \quad , \quad \wbf^\top \kronp \left ( \Ubf \Vbf \right ) = \Ubf \left ( \wbf^\top \kronp \Vbf \right )
	\text{\,.}
	\]
\end{lemma}
\begin{proof}
	According to the mixed-product property of the Kronecker product, for any matrices $\Abf, \Abf', \Bbf, \Bbf'$ for which $\Abf \Abf'$ and $\Bbf \Bbf'$ are defined, it holds that $( \Abf \Abf') \kronp ( \Bbf \Bbf') = ( \Abf \kronp \Bbf)(\Abf' \kronp \Bbf')$.
	Thus:
	\[
	\big ( \Ubf \Vbf \big ) \kronp \wbf^\top = \big ( \Ubf \Vbf \big ) \kronp \big ( 1 \cdot \wbf^\top \big) = \big ( \Ubf \kronp 1)(\Vbf \kronp \wbf^\top \big ) = \Ubf \big ( \Vbf \kronp \wbf^\top \big )
	\text{\,,}
	\]
	where $1$ is treated as the $1$-by-$1$ identity matrix.
	Similarly:
	\[
	\wbf^\top \kronp \big ( \Ubf \Vbf \big ) = \big (1 \cdot \wbf^\top \big ) \kronp \big ( \Ubf \Vbf \big ) = \big (1 \kronp \Ubf \big ) \big ( \wbf^\top \kronp \Vbf  \big ) = \Ubf \big ( \wbf^\top \kronp \Vbf \big)
	\text{\,.}
	\]
\end{proof}

\subsubsection{Hierarchical Tensor Factorization}
\label{htf:app:proofs:useful_lemmas:htf}

Suppose that we use a hierarchical tensor factorization with mode tree $\htmodetree$, weight matrices $\big ( \weightmat{\nu} \in \R^{R_\nu \times R_{ \parent (\nu) } } \big )_{\nu \in \htmodetree}$, and end tensor $\tensorend \in \R^{D_1 \times \cdots \times D_N}$ (\eqs~\eqref{htf:eq:ht_end_tensor}) to minimize $\htfobj$ (\eq~\eqref{htf:eq:htf_obj}) via gradient flow (\eq~\eqref{htf:eq:gf_htf}).
Under this setting, we prove the following technical lemmas.

\begin{lemma}
	\label{htf:lem:ht_multilinear}
	The functions $\tensorendmap \big ( \big ( \weightmat{ \nu } \big )_{ \nu \in \htmodetree } \big )$ and $\tensorendmap \big ( \big ( \weightmat{ \nu' } \big )_{ \nu' \in \htmodetree \setminus \subtree (\nu)}, \htftensorpart{\nu}{:} \big )$, for $\nu \in \htmodetree$, defined in \subapp~\ref{htf:app:proofs:notation}, are multilinear.
\end{lemma}
\begin{proof}
	We begin by proving that $\tensorendmap \big ( \big ( \weightmat{ \nu } \big )_{ \nu \in \htmodetree } \big )$ is multilinear. 
	Fix $\nu \in \htmodetree$, and let $\weightmat{\nu}, \Ubf^{(\nu)} \in \R^{ R_\nu \times R_{\parent (\nu)}}$, and $\alpha > 0$.
	
	\textbf{Homogeneity.}~~Denote by $\big ( \U^{ (\nu', r) }_\alpha \big )_{\nu' \in \htmodetree, r \in [ R_{ \parent (\nu') } ] }$ the intermediate tensors produced when computing the end tensor $\tensorendmap \big ( \big ( \weightmat{ \nu' } \big )_{ \nu' \in \htmodetree \setminus \{ \nu \} }, \alpha \cdot \weightmat{\nu} \big )$ according to \eq~\eqref{htf:eq:ht_end_tensor} (there denoted $\big (\htftensorpart{\nu'}{r} \big )_{\nu', r }$).
	If $\nu$ is a leaf node, then $\U^{(\nu, r)}_\alpha = \alpha \cdot \weightmat{\nu}_{:, r} = \alpha \cdot \htftensorpart{\nu}{r}$ for all $r \in [ R_{ \parent(\nu) } ]$.
	Otherwise, if $\nu$ is an interior node, a straightforward computation leads to the same conclusion, \ie~for all $r \in [ R_{\parent (\nu) } ]$:
	\[
	\begin{split}
		\U^{(\nu, r)}_\alpha & = \pi_\nu \left ( \sum\nolimits_{r' = 1}^{R_\nu} \alpha \cdot \weightmat{\nu}_{r', r} \left [ \tenp_{\nu_c \in \children ( \nu )} \htftensorpart{ \nu_c }{ r' } \right ] \right ) \\
		& =  \alpha \cdot \pi_\nu \left ( \sum\nolimits_{r' = 1}^{R_\nu} \weightmat{\nu}_{r', r} \left [ \tenp_{\nu_c \in \children ( \nu )} \htftensorpart{ \nu_c }{ r' } \right ] \right )  \\
		& = \alpha \cdot \htftensorpart{\nu}{r}
		\text{\,,}
	\end{split}
	\]
	where the second equality is by the linearity of $\pi_v$ (recall it is merely a reordering of the tensor entries).
	Moving on to the parent of $\nu$, multilinearity of the tensor product implies that for all $r \in [ R_{ \parent (\parent ( \nu ))} ]$:
	{\fontsize{9}{9}\selectfont
	\[
	\begin{split}	
		\U^{( \parent ( \nu ), r )}_\alpha & =  \pi_{ \parent (\nu) } \left ( \sum\nolimits_{r' = 1}^{R_{ \parent (\nu) } } \weightmat{ \parent (\nu)  }_{r', r} \left [ \left ( \tenp_{\nu_c \in \lsib (\nu)} \htftensorpart{ \nu_c }{ r' } \right ) \tenp \U^{(\nu, r')}_{\alpha} \tenp \left ( \tenp_{\nu_c \in \rsib (\nu)} \htftensorpart{\nu_c}{r'} \right ) \right ] \right ) \\
		& =  \pi_{ \parent (\nu) } \left ( \sum\nolimits_{r' = 1}^{R_{ \parent (\nu) } } \weightmat{ \parent (\nu)  }_{r', r} \left [ \left ( \tenp_{\nu_c \in \lsib (\nu)} \htftensorpart{ \nu_c }{ r' } \right ) \tenp \big ( \alpha \cdot \htftensorpart{\nu}{r'} \big ) \tenp \left ( \tenp_{\nu_c \in \rsib (\nu)} \htftensorpart{\nu_c}{r'} \right ) \right ] \right ) \\
		& = \alpha \cdot \pi_{ \parent (\nu) } \left ( \sum\nolimits_{r' = 1}^{R_{ \parent (\nu) } } \weightmat{ \parent (\nu)  }_{r', r} \left [ \tenp_{\nu_c \in \children (\parent (\nu))} \htftensorpart{ \nu_c }{ r' } \right ] \right ) \\
		& = \alpha \cdot \htftensorpart{\parent ( \nu )}{r}
		\text{\,.}
	\end{split}
	\]
	}
	An inductive claim over the path from $\nu$ to the root $[N]$ therefore yields:
	\be
	\tensorendmap \big  ( \big ( \weightmat{ \nu' } \big )_{ \nu' \in \htmodetree \setminus \{ \nu \} }, \alpha \cdot \weightmat{\nu} \big ) = \alpha \cdot \tensorendmap \big  ( \big ( \weightmat{ \nu' } \big )_{ \nu' \in \htmodetree } \big )
	\text{\,.}
	\label{htf:eq:H_multi_hom}
	\ee
	
	\textbf{Additivity.}~~We let $\big ( \U^{ (\nu', r) } \big )_{\nu' \in \htmodetree, r \in [ R_{ \parent (\nu') } ] }$ and $\big ( \U^{ (\nu', r) }_+ \big )_{\nu' \in \htmodetree, r \in [ R_{ \parent (\nu') } ] }$ denote the intermediate tensors produced when computing
	\[
	\tensorendmap \big  ( \big ( \weightmat{ \nu' } \big )_{ \nu' \in \htmodetree \setminus \{ \nu \} }, \Ubf^{ (\nu) } \big )
	\]
	and 
	\[
	\tensorendmap \big  ( \big ( \weightmat{ \nu' } \big )_{ \nu' \in \htmodetree \setminus \{ \nu \} }, \weightmat{\nu} + \Ubf^{ (\nu) } \big )
	\]
	according to \eq~\eqref{htf:eq:ht_end_tensor}, respectively.
	If $\nu$ is a leaf node, we have that $\U^{(\nu, r)}_+ = \weightmat{\nu}_{:, r} + \Ubf^{(\nu)}_{:, r} =  \htftensorpart{\nu}{r} + \U^{ ( \nu, r ) }$ for all $r \in [ R_{\parent (\nu) } ]$.
	Otherwise, if $\nu$ is an interior node, we arrive at the same conclusion, \ie~for all $ r \in [ R_{\parent (\nu) } ]$:
	\[
	\begin{split}
		\U^{(\nu, r)}_+ & =  \pi_\nu \left ( \sum\nolimits_{r' = 1}^{R_\nu} \left ( \weightmat{\nu}_{r', r} + \Ubf^{(\nu)}_{r', r} \right ) \left [ \tenp_{\nu_c \in \children ( \nu )} \htftensorpart{ \nu_c }{ r' } \right ] \right ) \\
		& = \pi_\nu \left ( \sum\nolimits_{r' = 1}^{R_\nu} \weightmat{\nu}_{r', r} \left [ \tenp_{\nu_c \in \children ( \nu )} \htftensorpart{ \nu_c }{ r' } \right ] \right )  + \pi_\nu \left ( \sum\nolimits_{r' = 1}^{R_\nu} \Ubf^{(\nu)}_{r', r} \left [ \tenp_{\nu_c \in \children ( \nu )} \htftensorpart{ \nu_c }{ r' } \right ] \right ) \\
		& = \htftensorpart{\nu}{r} + \U^{ ( \nu, r ) }
		\text{\,,}
	\end{split}
	\]
	where the second equality is by the linearity of $\pi_\nu$.
	Then, for any $r \in [ R_{ \parent ( \parent ( \nu ) ) } ]$:
	{\fontsize{9}{9}\selectfont
	\[
	\begin{split}
		\U^{ ( \parent (\nu) , r)}_+ & = \pi_{ \parent (\nu) } \left ( \sum\nolimits_{r' = 1}^{R_{ \parent (\nu) } } \weightmat{ \parent (\nu)  }_{r', r} \left [ \left ( \tenp_{\nu_c \in \lsib (\nu)} \htftensorpart{ \nu_c }{ r' } \right ) \tenp \U^{(\nu, r')}_{+} \tenp \left ( \tenp_{\nu_c \in \rsib (\nu)} \htftensorpart{\nu_c}{r'} \right ) \right ] \right ) \\
		& =  \pi_{ \parent (\nu) } \left ( \sum\nolimits_{r' = 1}^{R_{ \parent (\nu) } } \weightmat{ \parent (\nu)  }_{r', r} \left [ \left ( \tenp_{\nu_c \in \lsib (\nu)} \htftensorpart{ \nu_c }{ r' } \right ) \tenp \big ( \htftensorpart{\nu}{r'} + \U^{(\nu, r')} \big ) \tenp \left ( \tenp_{\nu_c \in \rsib (\nu)} \htftensorpart{\nu_c}{r'} \right ) \right ] \right ) \\
		& = \pi_{ \parent (\nu) } \left ( \sum\nolimits_{r' = 1}^{R_{ \parent (\nu) } } \weightmat{ \parent (\nu)  }_{r', r} \left [ \tenp_{\nu_c \in \children (\parent (\nu))} \htftensorpart{ \nu_c }{ r' } \right ] \right ) + \\
		& \hspace{4.5mm} \pi_{ \parent (\nu) } \left ( \sum\nolimits_{r' = 1}^{R_{ \parent (\nu) } } \weightmat{ \parent (\nu)  }_{r', r} \left [ \left ( \tenp_{\nu_c \in \lsib (\nu)} \htftensorpart{ \nu_c }{ r' } \right ) \tenp \U^{(\nu, r')} \tenp \left ( \tenp_{\nu_c \in \rsib (\nu)} \htftensorpart{\nu_c}{r'} \right ) \right ] \right ) \\
		& = \htftensorpart{ \parent (\nu) }{ r } + \U^{ ( \parent (\nu) , r)}	
		\text{\,,}
	\end{split}
	\]
	}
	where the penultimate equality is by multilinearity of the tensor product as well as linearity of $\pi_{\parent (\nu)}$.
	An induction over the path from $\nu$ to the root thus leads to:
	\be
	\tensorendmap \big  ( \big ( \weightmat{ \nu' } \big )_{ \nu' \in \htmodetree \setminus \{ \nu \} }, \weightmat{\nu} + \Ubf^{ (\nu) } \big ) = \tensorendmap \big  ( \big ( \weightmat{ \nu' } \big )_{ \nu' \in \htmodetree } \big ) + \tensorendmap \big  ( \big ( \weightmat{ \nu' } \big )_{ \nu' \in \htmodetree \setminus \{ \nu \} }, \Ubf^{ (\nu) } \big )
	\text{\,.}
	\label{htf:eq:H_multi_add}
	\ee
	
	\medskip
	
	\eqs~\eqref{htf:eq:H_multi_hom} and~\eqref{htf:eq:H_multi_add} establish that $\tensorendmap \big ( \big ( \weightmat{ \nu } \big )_{ \nu \in \htmodetree } \big )$ is multilinear.
	The proof for $\tensorendmap \big ( \big ( \weightmat{ \nu' } \big )_{ \nu' \in \htmodetree \setminus \subtree (\nu)}, \htftensorpart{\nu}{:} \big )$ follows by analogous derivations.
	
\end{proof}

\begin{lemma}
	\label{htf:lem:zero_inter_tensor}
	Suppose there exists $\nu \in \htmodetree$ such that $\htftensorpart{\nu}{r} = 0$ for all $r \in [ R_{ \parent (\nu) } ]$, where $\htftensorpart{\nu}{r}$ is as defined in \eq~\eqref{htf:eq:ht_end_tensor}.
	Then, $\tensorend = 0$.
\end{lemma}
\begin{proof}
	Since $\htftensorpart{\nu}{r} = 0$ for all $r \in [R_{ \parent (\nu ) }]$, for any $r' \in [R_{ \parent ( \parent (\nu ) )} ]$ we have that:
	{\fontsize{9}{9}\selectfont
	\[
	\begin{split}
		\htftensorpart{\parent (\nu )}{r'} & =  \pi_{ \parent (\nu) } \left ( \sum\nolimits_{r = 1}^{R_{ \parent (\nu) } } \weightmat{ \parent (\nu)  }_{r, r'} \left [ \left ( \tenp_{\nu_c \in \lsib (\nu)} \htftensorpart{ \nu_c }{ r } \right ) \tenp \htftensorpart{\nu}{r} \tenp \left ( \tenp_{\nu_c \in \rsib (\nu)} \htftensorpart{\nu_c}{r} \right ) \right ] \right ) \\
		& =  \pi_{ \parent (\nu) } \left ( \sum\nolimits_{r = 1}^{R_{ \parent (\nu) } } \weightmat{ \parent (\nu)  }_{r, r'} \left [ \left ( \tenp_{\nu_c \in \lsib (\nu)} \htftensorpart{ \nu_c }{ r } \right ) \tenp 0 \tenp \left ( \tenp_{\nu_c \in \rsib (\nu)} \htftensorpart{\nu_c}{r} \right ) \right ] \right ) \\
		&  = 0
		\text{\,.}
	\end{split}
	\]
	}
	Thus, the claim readily follows by an induction up the path from $\nu$ to the root $[N]$.
\end{proof}

\begin{lemma}
	\label{htf:lem:tensorpart_norm_bound}
	For any $\nu \in \interior (\htmodetree)$ and $r \in [ R_{ \parent (\nu) } ]$:
	\[
	\norm1{ \htftensorpart{\nu}{r} } \leq \norm1{ \weightmat{\nu}_{:, r} } \cdot \prod\nolimits_{\nu_c \in \children (\nu) } \norm1{ \htftensorpart{\nu_c}{:} }
	\text{\,,}
	\]
	where $\htftensorpart{\nu_c}{:}$, for $\nu_c \in \children (\nu)$, is the tensor obtained by stacking $\big ( \htftensorpart{\nu_c}{r'} \big )_{r' = 1}^{R_{\nu} }$ into a single tensor, \ie~$\htftensorpart{\nu_c}{:}_{:, \ldots, :, r'} = \htftensorpart{\nu_c}{r'}$ for all $r' \in [ R_{\nu} ]$.
\end{lemma}
\begin{proof}
	By the definition of $\htftensorpart{\nu}{r}$ (\eq~\eqref{htf:eq:ht_end_tensor}) we have that:
	\[
	\begin{split}
		\norm1{ \htftensorpart{\nu}{r} } &= \norm*{ \pi_\nu \left ( \sum\nolimits_{r' = 1}^{R_\nu} \weightmat{\nu}_{r', r} \left [ \tenp_{\nu_c \in \children ( \nu )} \htftensorpart{ \nu_c }{ r' } \right ] \right ) } \\[1mm]
		& =  \norm*{ \sum\nolimits_{r' = 1}^{R_\nu} \weightmat{\nu}_{r', r} \left [ \tenp_{\nu_c \in \children ( \nu )} \htftensorpart{ \nu_c }{ r' } \right ] }
		\text{\,,}
	\end{split}
	\]
	where the second equality is due to the fact that $\pi_\nu$ merely reorders entries of a tensor, and therefore does not alter its Frobenius norm.
	Vectorizing each $\tenp_{\nu_c \in \children (\nu)} \htftensorpart{\nu_c}{r'}$, we may write $\norm{ \htftensorpart{\nu}{r} }$ as the Frobenius norm of a matrix-vector product:
	\[
	\norm1{ \htftensorpart{\nu}{r} } = \norm*{ \left ( \vectbig{ \tenp_{\nu_c \in \children (\nu)} \htftensorpart{\nu_c}{1} } , \ldots,\vectbig{ \tenp_{\nu_c \in \children (\nu)} \htftensorpart{\nu_c}{R_\nu} }  \right ) \weightmat{\nu}_{:, r} }
	\text{\,.}
	\]
	Hence, sub-multiplicativity of the Frobenius norm gives:
	\be
	\norm1{ \htftensorpart{\nu}{r} } \leq \norm1{ \weightmat{\nu}_{:, r} } \cdot \norm*{  \left ( \vectbig{ \tenp_{\nu_c \in \children (\nu)} \htftensorpart{\nu_c}{1} } , \ldots,\vectbig{ \tenp_{\nu_c \in \children (\nu)} \htftensorpart{\nu_c}{R_\nu} }  \right ) }
	\text{\,.}
	\label{htf:eq:tensorpart_norm_interm_bound}
	\ee
	Notice that:
	\[
	\begin{split}
		\norm*{ \! \left ( \vectbig{ \tenp_{\nu_c \in \children (\nu)} \htftensorpart{\nu_c}{1} } , \ldots,\vectbig{ \tenp_{\nu_c \in \children (\nu)} \htftensorpart{\nu_c}{R_\nu} }  \right ) \! }^2 & \! = \sum\nolimits_{r' = 1}^{R_\nu} \norm{ \tenp_{\nu_c \in \children (\nu)} \htftensorpart{\nu_c}{r'} }^2 \\
		& = \sum\nolimits_{r' = 1}^{R_\nu} \prod\nolimits_{\nu_c \in \children (\nu)} \norm{ \htftensorpart{\nu_c}{r'} }^2 \\
		& \leq \prod\nolimits_{\nu_c \in \children (\nu)} \left ( \sum\nolimits_{r' = 1}^{R_\nu} \norm{ \htftensorpart{\nu_c}{r'} }^2 \right ) \\
		& =  \prod\nolimits_{\nu_c \in \children (\nu)} \norm{ \htftensorpart{\nu_c}{:} }^2
		\text{\,,}
	\end{split}
	\]
	where the second transition is by the fact that the norm of a tensor product is equal to the product of the norms,
	and the inequality is due to $\prod\nolimits_{\nu_c \in \children (\nu)} \big ( \sum\nolimits_{r' = 1}^{R_\nu} \normbig{ \htftensorpart{\nu_c}{r'} }^2 \big )$ being a sum of non-negative elements which includes $\sum\nolimits_{r' = 1}^{R_\nu} \prod\nolimits_{\nu_c \in \children (\nu)} \normbig{ \htftensorpart{\nu_c}{r'} }^2$.
	Taking the square root of both sides in the equation above and plugging it into \eq~\eqref{htf:eq:tensorpart_norm_interm_bound} completes the proof.
\end{proof}

\begin{lemma}
	\label{htf:lem:stacked_tensorpart_norm_bound}
	For any $\nu \in \interior (\htmodetree)$:
	\[
	\norm1{ \htftensorpart{\nu}{:} } \leq \norm1{ \weightmat{\nu} } \cdot  \prod\nolimits_{\nu_c \in \children (\nu) } \norm1{ \htftensorpart{\nu_c}{:} }
	\text{\,,}
	\]
	where $\htftensorpart{\nu_c}{:}$, for $\nu_c \in \children (\nu)$, is the tensor obtained by stacking $\big ( \htftensorpart{\nu_c}{r} \big )_{r = 1}^{R_{ \nu } }$ into a single tensor, \ie~$\htftensorpart{\nu_c}{:}_{:, \ldots, :, r} = \htftensorpart{\nu_c}{r}$ for all $r \in [ R_{ \nu } ]$.
\end{lemma}
\begin{proof}
	We may explicitly write $\norm{ \htftensorpart{\nu}{:} }^2$ as follows:
	\be
	\norm1{ \htftensorpart{\nu}{:} }^2 = \sum\nolimits_{r = 1}^{ R_{\parent (\nu) } } \norm1{ \htftensorpart{\nu}{r} }^2
	\text{\,.}
	\label{htf:eq:stacked_tensorpart_norm}
	\ee
	For each $r \in [ R_{ \parent (\nu) } ]$, by \lem~\ref{htf:lem:tensorpart_norm_bound} we know that:
	\[
	\norm1{ \htftensorpart{\nu}{r} }^2 \leq \norm1{ \weightmat{\nu}_{:, r} }^2 \cdot \prod\nolimits_{\nu_c \in \children (\nu)} \norm1{ \htftensorpart{\nu_c}{:} }^2
	\text{\,.}
	\]
	Thus, going back to \eq~\eqref{htf:eq:stacked_tensorpart_norm} we arrive at:
	\[
	\norm1{ \htftensorpart{\nu}{:} }^2 \leq  \sum\nolimits_{ r = 1 }^{R_{ \parent (\nu) } } \norm1{ \weightmat{\nu}_{:, r} }^2 \cdot \prod\nolimits_{\nu_c \in \children (\nu)} \norm1{ \htftensorpart{\nu_c}{:} }^2 = \norm1{ \weightmat{\nu} }^2 \cdot \prod\nolimits_{\nu_c \in \children (\nu)} \norm1{ \htftensorpart{\nu_c}{:} }^2
	\text{\,.}
	\]
	Taking the square root of both sides concludes the proof.
\end{proof}

\begin{lemma}
	\label{htf:lem:ht_weightmat_grad}
	For any  $\nu \in \htmodetree$ and $\Delta \in \R^{R_\nu \times R_{\parent (\nu)}}$:
	\[
	\inprod{ \frac{ \partial }{ \partial \weightmat{\nu} } \htfobj \big ( \big ( \weightmat{\nu'} \big )_{\nu' \in \htmodetree} \big ) }{ \Delta } = \inprod{ \nabla \htfendloss ( \tensorend ) }{ \tensorendmap \Big  ( \big ( \weightmat{ \nu' } \big )_{ \nu' \in \htmodetree \setminus \{ \nu \} }, \Delta \Big ) }
	\text{\,.}
	\]
\end{lemma}
\begin{proof}
	We treat $\big ( \weightmat{\nu'} \big )_{\nu '\in \htmodetree \setminus{ \{ \nu \}} }$ as fixed, and with slight abuse of notation consider:
	\[
	\htfobj^{ ( \nu ) } \big ( \weightmat{\nu}  \big ) := \htfobj \big ( \big ( \weightmat{\nu'} \big )_{\nu' \in \htmodetree} \big )
	\text{\,.}
	\]
	For $\Delta \in \R^{R_\nu \times R_{ \parent (\nu) } }$, by multilinearity of $\tensorendmap$ (\lem~\ref{htf:lem:ht_multilinear}) we have that:
	\[
	\begin{split}
		\htfobj^{ ( \nu ) } \big ( \weightmat{\nu}  + \Delta \big ) & = \htfendloss \left ( \tensorendmap \Big  ( \big ( \weightmat{ \nu' } \big )_{ \nu' \in \htmodetree \setminus \{ \nu \} },  \weightmat{\nu}  + \Delta \Big ) \right ) \\
		& =  \htfendloss \left ( \tensorend + \tensorendmap \Big  ( \big ( \weightmat{ \nu' } \big )_{ \nu' \in \htmodetree \setminus \{ \nu \} }, \Delta \Big ) \right ) 
		\text{\,.}
	\end{split}
	\]
	According to the  first order Taylor approximation of $\htfendloss$ we may write:
	\[
	\begin{split}
		\htfobj^{ ( \nu ) } \big ( \weightmat{\nu}  + \Delta \big ) & =  \htfendloss \left ( \tensorend  \right ) + \inprod{ \nabla \htfendloss \left ( \tensorend \right )  }{ \tensorendmap \Big  ( \big ( \weightmat{ \nu' } \big )_{ \nu' \in \htmodetree \setminus \{ \nu \} }, \Delta \Big ) } + o \left ( \norm{ \Delta } \right ) \\
		& = \htfobj^{ ( \nu ) } \big ( \weightmat{\nu}  \big ) + \inprod{ \nabla \htfendloss \left ( \tensorend \right )  }{ \tensorendmap \Big  ( \big ( \weightmat{ \nu' } \big )_{ \nu' \in \htmodetree \setminus \{ \nu \} }, \Delta \Big ) } + o \left ( \norm{ \Delta } \right ) 
		\text{\,.}
	\end{split}
	\]
	The term $\inprodbig{ \nabla \htfendloss \left ( \tensorend \right )  }{ \tensorendmap \big  ( ( \weightmat{ \nu' } )_{ \nu' \in \htmodetree \setminus \{ \nu \} }, \Delta \big ) }$ is a linear function of $\Delta$.
	Therefore, uniqueness of the linear approximation of $\htfobj^{ (\nu) }$ at $\weightmat{\nu}$ implies:
	\[
	\inprod{ \frac{ d }{ d \weightmat{\nu} } \htfobj^{(\nu)} \big ( \weightmat{\nu} \big ) }{ \Delta } = \inprod{ \nabla \htfendloss ( \tensorend ) }{ \tensorendmap \Big  ( \big ( \weightmat{ \nu' } \big )_{ \nu' \in \htmodetree \setminus \{ \nu \} }, \Delta \Big ) }
	\text{\,.}
	\]	
	Noticing that $\frac{ \partial }{ \partial \weightmat{\nu} } \htfobj \big ( \big ( \weightmat{\nu'} \big )_{\nu' \in \htmodetree} \big ) =  \frac{ d }{ d \weightmat{\nu} } \htfobj^{(\nu)} \big ( \weightmat{\nu} \big )$ completes the proof.
\end{proof}

\begin{lemma}
	\label{htf:lem:ht_weightvec_grad}
	For any $\nu \in \interior(\htmodetree), r \in [R_{ \nu }]$, and $ \Delta \in \R^{ R_{ \parent ( \nu )} }$:
	\be
	\inprodbigg{ \frac{ \partial }{ \partial \weightmat{\nu}_{r, :} } \htfobj \big ( \big ( \weightmat{\nu'} \big )_{\nu' \in \htmodetree} \big ) }{ \Delta^\top } = \inprod{ \nabla \htfendloss ( \tensorend ) }{ \tensorendmap \Big  ( \big ( \weightmat{ \nu' } \big )_{ \nu' \in \htmodetree \setminus \{ \nu  \} }, \padrow_r \left ( \Delta^\top \right ) \Big ) }
	\text{\,,}
	\label{htf:eq:ht_row_weightvec_grad}
	\ee
	where $\padrow_r \left ( \Delta^\top \right ) \in \R^{ R_{ \nu  } \times R_{ \parent ( \nu ) }}$ is the matrix whose $r$'th row is $\Delta^\top$, and all the rest are zero.
	Furthermore, for any $\nu_c \in \children (\nu)$ and $\Delta \in \R^{R_{\nu_c}}$:
	\be
	\inprodbigg{ \frac{ \partial }{ \partial \weightmat{\nu_c}_{:, r} } \htfobj \big ( \big ( \weightmat{\nu'} \big )_{\nu' \in \htmodetree} \big ) }{ \Delta } = \inprod{ \nabla \htfendloss ( \tensorend ) }{ \tensorendmap \Big ( \big ( \weightmat{ \nu' } \big )_{ \nu' \in \htmodetree \setminus \{ \nu_c \} }, \padcol_r ( \Delta ) \Big ) }
	\text{\,,}
	\label{htf:eq:ht_col_weightvec_grad}
	\ee
	where $\padcol_r ( \Delta ) \in \R^{ R_{\nu_c} \times R_{ \nu }}$ is the matrix whose $r$'th column is $\Delta$, and all the rest are zero.
\end{lemma}
\begin{proof}
	\eqs~\eqref{htf:eq:ht_row_weightvec_grad} and~\eqref{htf:eq:ht_col_weightvec_grad} are direct implications of Lemma~\ref{htf:lem:ht_weightmat_grad} since:
	\[
	\begin{split}
		\inprodbigg{ \frac{ \partial }{ \partial \weightmat{\nu}_{r, :} } \htfobj \big ( \big ( \weightmat{\nu'} \big )_{\nu' \in \htmodetree} \big ) }{ \Delta^\top } & = \inprodbigg{ \frac{ \partial }{ \partial \weightmat{\nu} } \htfobj \Big ( \big ( \weightmat{\nu'} \big )_{\nu' \in \htmodetree} \Big ) }{  \padrow_r \left (\Delta^\top \right )}
	\end{split}
	\text{\,,}
	\]
	and
	\[
	\begin{split}
		\inprodbigg{ \frac{ \partial }{ \partial \weightmat{\nu_c}_{:, r} } \htfobj \big ( \big ( \weightmat{\nu'} \big )_{\nu' \in \htmodetree} \big ) }{ \Delta } & = \inprodbigg{ \frac{ \partial }{ \partial \weightmat{\nu_c} } \htfobj \Big ( \big ( \weightmat{\nu'} \big )_{\nu' \in \htmodetree} \Big ) }{  \padcol_r ( \Delta ) }
		\text{\,.}
	\end{split}
	\]
\end{proof}

\begin{lemma}
	\label{htf:lem:ht_weightvec_grad_zero_comp}
	Let $\nu \in \interior(\htmodetree)$, $\nu_c \in \children (\nu)$, and $r \in [R_{ \nu }]$.
	If both $\weightmat{\nu}_{r, :} = 0$ and $\weightmat{\nu_c}_{:, r} = 0$, then:
	\be
	\frac{ \partial }{ \partial \weightmat{\nu}_{r, :} } \htfobj \big ( \big ( \weightmat{\nu'} \big )_{\nu' \in \htmodetree} \big ) = 0
	\text{\,,}
	\label{htf:eq:ht_row_weightvec_grad_zero}
	\ee
	and
	\be
	\frac{ \partial }{ \partial \weightmat{\nu_c}_{:, r} } \htfobj \big ( \big ( \weightmat{\nu'} \big )_{\nu' \in \htmodetree} \big ) = 0
	\text{\,.}
	\label{htf:eq:ht_col_weightvec_grad_zero}
	\ee
\end{lemma}
\begin{proof}
	We show that $\tensorendmap \big  ( \big ( \weightmat{ \nu' } \big )_{ \nu' \in \htmodetree \setminus \{ \nu  \} }, \padrow_r \left ( \Delta^\top \right ) \big ) = 0$ for all $\Delta \in \R^{\parent(\nu)}$.
	\eq~\eqref{htf:eq:ht_row_weightvec_grad_zero} then follows from \eq~\eqref{htf:eq:ht_row_weightvec_grad} in \lem~\ref{htf:lem:ht_weightvec_grad}.
	Fix some $\Delta \in \R^{\parent(\nu)}$ and let $\big ( \U^{ (\nu', r' ) } \big )_{\nu' \in \htmodetree, r' \in [ R_{ \parent (\nu') } ] }$ be the intermediate tensors produced when computing 
	\[
	\tensorendmap \big  ( \big ( \weightmat{ \nu' } \big )_{ \nu' \in \htmodetree \setminus \{ \nu  \} }, \padrow_r \left ( \Delta^\top \right ) \big )
	\]
	according to \eq~\eqref{htf:eq:ht_end_tensor} (there denoted $\big ( \htftensorpart{\nu'}{ r' } \big )_{\nu', r'}$).
	For any $\bar{r} \in [ R_{\parent (\nu)} ]$ we have that:
	\[
	\U^{(\nu, \bar{r})} = \pi_{\nu} \left ( \sum\nolimits_{r' = 1}^{R_\nu} \padrow_r \big ( \Delta^\top \big )_{r', \bar{r}} \left [  \tenp_{\nu' \in \children (\nu)} \htftensorpart{\nu'}{ r' }  \right ] \right ) =  \pi_\nu \left ( \Delta_{ \bar{r} } \left [  \tenp_{\nu' \in \children (\nu)} \htftensorpart{\nu'}{r}  \right ] \right )
	\text{\,.}
	\]
	The fact that $\weightmat{\nu_c}_{:, r} = 0$ implies that $\htftensorpart{\nu_c}{r} := \pi_{\nu_c} \big ( \sum\nolimits_{r' = 1}^{R_{\nu_c}} \weightmat{\nu_c}_{r', r} \big [ \tenp_{\nu' \in \children ( \nu )} \htftensorpart{ \nu' }{ r' } \big ] \big ) = 0$, and so for every $r' \in [ R_{\parent(\nu)} ]$:
	\[
	\U^{(\nu, \bar{r})} = \pi_\nu \left ( \Delta_{ \bar{r} } \left [ \left ( \tenp_{\nu' \in \lsib (\nu_c)} \htftensorpart{\nu'}{r} \right ) \tenp 0 \tenp \left ( \tenp_{\nu' \in \rsib (\nu_c)} \htftensorpart{\nu'}{r}  \right ) \right ] \right ) = 0
	\text{\,.}
	\]
	\lem~\ref{htf:lem:zero_inter_tensor} then gives $\tensorendmap \big  ( ( \weightmat{ \nu' } )_{ \nu' \in \htmodetree \setminus \{ \nu  \} }, \padrow_r \left ( \Delta^\top \right ) \big ) = 0$, completing this part of the proof.
	
	\medskip
	
	Next, we show that $\tensorendmap \big ( \big ( \weightmat{ \nu' } \big )_{ \nu' \in \htmodetree \setminus \{ \nu_c \} }, \padcol_r ( \Delta ) \big ) = 0$ for all $\Delta \in \R^{\nu_c}$.
	\eq~\eqref{htf:eq:ht_col_weightvec_grad} in \lem~\ref{htf:lem:ht_weightvec_grad} then yields \eq~\eqref{htf:eq:ht_col_weightvec_grad_zero}.
	Fix some $\Delta \in \R^{\nu_c}$ and let $\big ( \V^{ (\nu', r') } \big )_{\nu' \in \htmodetree, r' \in [ R_{ \parent (\nu') } ] }$ be the intermediate tensors produced when computing 
	\[
	\tensorendmap \big ( \big ( \weightmat{ \nu' } \big )_{ \nu' \in \htmodetree \setminus \{ \nu_c \} }, \padcol_r ( \Delta ) \big )
	\]
	according to \eq~\eqref{htf:eq:ht_end_tensor}.
	For any $\bar{r} \in [ R_{\nu} ] \setminus \{ r\}$:
	\[
	\begin{split}
	\U^{(\nu_c, \bar{r})} & = \pi_{\nu_c} \left ( \sum\nolimits_{ r' = 1}^{R_{\nu_c}}  \padcol_r \big ( \Delta \big )_{r' , \bar{r}}  \left [  \tenp_{\nu' \in \children (\nu_c)} \htftensorpart{\nu'}{ r' }  \right ] \right ) \\
	& = \pi_{\nu_c} \left ( \sum\nolimits_{ r' = 1}^{R_{\nu_c}} 0 \cdot  \left [  \tenp_{\nu' \in \children (\nu_c)} \htftensorpart{\nu'}{ r' }  \right ] \right ) \\
	& = 0
	\text{\,.}
	\end{split}
	\]
	Thus, for any $\hat{r} \in [ R_{ \parent( \nu) } ]$ we may write $\U^{(\nu, \hat{r})} = \pi_{\nu} \big ( \weightmat{\nu}_{r, \hat{r}} \big [  \tenp_{\nu' \in \children (\nu)} \U^{(\nu', r)}  \big ] \big )$.
	Since $\weightmat{\nu}_{r, :} = 0$, we get that $\U^{(\nu, \hat{r})} = 0$ for all $\hat{r} \in [ R_{ \parent( \nu) } ]$, which by \lem~\ref{htf:lem:zero_inter_tensor} leads to $\tensorendmap \big ( \big ( \weightmat{ \nu' } \big )_{ \nu' \in \htmodetree \setminus \{ \nu_c \} }, \padcol_r ( \Delta ) \big ) = 0$.
\end{proof}

\begin{lemma}
	\label{htf:lem:tensorend_comp_eq_zeroing_cols_or_rows}
	For any $\nu \in \interior (\htmodetree)$, $\nu_c \in \children (\nu)$, and $r \in [ R_{\nu}]$, the following hold:
	\be
	\tensorendmap \big  ( \big ( \weightmat{ \nu' } \big )_{ \nu' \in \htmodetree \setminus \{ \nu  \} }, \padrow_r \big ( \weightmat{\nu}_{r, :} \big ) \big ) = \htfcompnorm{\nu}{r} \cdot \htfcomp{ \nu }{ r }
	\text{\,,}
	\label{htf:eq:tensorend_zero_rows_loc_comp}
	\ee
	and
	\be
	\tensorendmap \big ( \big ( \weightmat{ \nu' } \big )_{ \nu' \in \htmodetree \setminus \{ \nu_c \} }, \padcol_r \big (  \weightmat{ \nu_c }_{:, r} \big ) \big ) = \htfcompnorm{\nu}{r} \cdot \htfcomp{ \nu }{ r }
	\text{\,,}
	\label{htf:eq:tensorend_zero_cols_loc_comp}
	\ee
	where $\padrow_r \left ( \Delta^\top \right ) \in \R^{ R_{ \nu  } \times R_{ \parent ( \nu ) }}$ is the matrix whose $r$'th row is $\Delta^\top$, and all the rest are zero,
	$\padcol_r ( \Delta ) \in \R^{ R_{\nu_c} \times R_{ \nu }}$ is the matrix whose $r$'th column is $\Delta$, and all the rest are zero, and $\htfcomp{\nu}{r}$ is as defined in \thm~\ref{htf:thm:loc_comp_norm_bal_dyn}.
\end{lemma}

\begin{proof}
	Starting with \eq~\eqref{htf:eq:tensorend_zero_rows_loc_comp}, let $\big ( \U^{ (\nu', r') } \big )_{\nu' \in \htmodetree, r' \in [ R_{ \parent (\nu') } ] }$ be the intermediate tensors formed when computing $\tensorendmap \big ( \big ( \weightmat{ \nu' } \big )_{ \nu' \in \htmodetree \setminus \{ \nu  \} }, \padrow_r \big ( \weightmat{\nu}_{r, :} \big ) \big )$ according to \eq~\eqref{htf:eq:ht_end_tensor} (there denoted $\big ( \htftensorpart{\nu'}{ r' } \big )_{\nu', r' }$).
	Clearly, for any $\nu' \in \subtree (\nu) \setminus \{ \nu \}$~---~a node in the subtree of $\nu$ which is not $\nu$~---~it holds that $\U^{(\nu', r')} = \htftensorpart{\nu'}{r'}$ for all $r' \in [ R_{\parent (\nu')} ]$.
	Thus, for all $r' \in [ R_{ \parent(\nu) } ]$ we have that:
	\be
	\U^{(\nu, r')} = \pi_{\nu} \left ( \sum\nolimits_{\bar{r} = 1}^{R_\nu} \padrow_r \big ( \weightmat{\nu}_{r, :} \big )_{\bar{r}, r'} \left [  \tenp_{\nu' \in \children (\nu)} \htftensorpart{\nu'}{\bar{r}}  \right ] \right ) =  \pi_{\nu} \left ( \weightmat{\nu}_{r, r'} \left [  \tenp_{\nu' \in \children (\nu)} \htftensorpart{\nu'}{r}  \right ] \right )
	\text{\,.}
	\label{htf:eq:tensorpart_U}
	\ee
	If $\htfcompnorm{\nu}{r} = \normbig{ \weightmat{\nu}_{r, :} \tenp \big ( \tenp_{\nu' \in \children (\nu)} \weightmat{\nu'}_{:, r} \big ) } = \normbig{\weightmat{\nu}_{r, :}} \prod_{\nu' \in \children (\nu)} \normbig{\weightmat{\nu'}_{:, r}} = 0$, then either $\weightmat{\nu}_{r, :} = 0$ or $\weightmat{\nu'}_{:, r} = 0$ for some $\nu' \in \children (\nu)$.
	We claim that in both cases $\U^{(\nu, r')} = 0$ for all $r' \in [R_{\parent(\nu)}]$.
	Indeed, if $\weightmat{\nu}_{r, :} = 0$ this immediately follows from \eq~\eqref{htf:eq:tensorpart_U}.
	On the other hand, if $\weightmat{\nu'}_{:, r} = 0$ for some $\nu' \in \children (\nu)$, then $\htftensorpart{\nu'}{r} = 0$, which combined with \eq~\eqref{htf:eq:tensorpart_U} also implies that $\U^{(\nu, r')} = 0$ for all $r' \in [R_{\parent(\nu)}]$.
	Hence, \lem~\ref{htf:lem:zero_inter_tensor} establishes \eq~\eqref{htf:eq:tensorend_zero_cols_loc_comp} for the case of $\htfcompnorm{\nu}{r} = 0$:
	\[
	\tensorendmap \Big  ( \big ( \weightmat{ \nu' } \big )_{ \nu' \in \htmodetree \setminus \{ \nu  \} }, \padrow_r \big ( \weightmat{\nu}_{r, :} \big ) \Big ) = 0 = \htfcompnorm{\nu}{r} \cdot \htfcomp{\nu}{r}
	\text{\,.}
	\]
	Now, suppose that $\htfcompnorm{\nu}{r} \neq 0$ and let $\U^{(\nu, :)}$ be the tensor obtained by stacking $\big ( \U^{(\nu, r') } \big )_{r' = 1}^{R_{ \parent (\nu) }}$ into a single tensor, \ie~$\U^{(\nu, :)}_{:, \ldots, :, r'} = \U^{(\nu, r')}$ for all $r' \in [R_{\parent(\nu)}]$.
	Multilinearity of $\tensorendmap \big ( \big ( \weightmat{ \nu' } \big )_{ \nu' \in \htmodetree \setminus \subtree (\nu) }, \U^{(\nu, :)} \big )$ (\lem~\ref{htf:lem:ht_multilinear}) leads to:
	\be
	\begin{split}
		\tensorendmap \Big  ( \big ( \weightmat{ \nu' } \big )_{ \nu' \in \htmodetree \setminus \{ \nu  \} }, \padrow_r \big ( \weightmat{\nu}_{r, :} \big ) \Big ) & = \tensorendmap \Big ( \big ( \weightmat{ \nu' } \big )_{ \nu' \in \htmodetree \setminus \subtree (\nu) }, \U^{(\nu, :)} \Big ) \\
		& = \htfcompnorm{\nu}{r} \cdot \tensorendmap \Big ( \big ( \weightmat{ \nu' } \big )_{ \nu' \in \htmodetree \setminus \subtree (\nu) }, \big ( \htfcompnorm{\nu}{r} \big )^{-1} \U^{(\nu, :)} \Big ) 
		\text{\,.}
	\end{split}
	\label{htf:eq:tensorend_zero_cols_loc_comp_intermid_I}
	\ee
	From \eq~\eqref{htf:eq:tensorpart_U} we know that $\big ( \htfcompnorm{\nu}{r} \big )^{-1} \U^{(\nu, :)}_{:,\ldots, :, r'} = \pi_\nu \big ( \big ( \htfcompnorm{\nu}{r} \big )^{-1} \weightmat{\nu}_{r, r'} \big [  \tenp_{\nu' \in \children (\nu)} \htftensorpart{\nu'}{r}  \big ] \big )$ for all $r' \in [R_{\parent(\nu)}]$.
	Thus, by the definition of $\htfcomp{\nu}{r}$ we may conclude that:
	\be
	\tensorendmap \Big ( \big ( \weightmat{ \nu' } \big )_{ \nu' \in \htmodetree \setminus  \subtree (\nu) }, \big ( \htfcompnorm{\nu}{r} \big )^{-1} \U^{(\nu, :)} \Big ) = \htfcomp{\nu}{r}
	\text{\,.}
	\label{htf:eq:tensorend_zero_cols_loc_comp_intermid_II}
	\ee
	Combining \eqs~\eqref{htf:eq:tensorend_zero_cols_loc_comp_intermid_I} and~\eqref{htf:eq:tensorend_zero_cols_loc_comp_intermid_II} yields \eq~\eqref{htf:eq:tensorend_zero_cols_loc_comp}, completing this part of the proof.
	
	\medskip
	
	Turning our attention to \eq~\eqref{htf:eq:tensorend_zero_cols_loc_comp}, let $\big ( \V^{ (\nu', r') } \big )_{\nu' \in \htmodetree, r' \in [ R_{ \parent (\nu') } ] }$ be the intermediate tensors produced when computing $\tensorendmap \big ( \big ( \weightmat{ \nu' } \big )_{ \nu' \in \htmodetree \setminus \{ \nu_c \} }, \padcol_r \big (  \weightmat{ \nu_c }_{:, r} \big ) \big )$ according to \eq~\eqref{htf:eq:ht_end_tensor} (there denoted $\big ( \htftensorpart{\nu'}{ r' } \big )_{\nu', r' }$).
	Clearly, for any $\nu' \in \subtree ( \nu ) \setminus \{ \nu, \nu_c \}$~---~a node in the subtree of $\nu$ which is not $\nu$ nor $\nu_c$~---~it holds that $\V^{(\nu', r')} = \htftensorpart{\nu'}{r'}$ for all $r' \in [R_{\parent (\nu')}]$.
	Thus, $\V^{ (\nu_c, r) } = \pi_{\nu_c} \big ( \sum\nolimits_{\bar{r} = 1}^{R_{\nu_c}} \padcol_r \big ( \weightmat{\nu_c}_{:, r}\big )_{\bar{r}, r} \big [  \tenp_{\nu' \in \children (\nu_c)} \htftensorpart{\nu'}{\bar{r}}  \big ] \big ) = \htftensorpart{\nu_c}{r}$, whereas for any $r' \in [ R_{\nu} ] \setminus \{r\}$:
	\[
	\begin{split}
		\V^{ (\nu_c, r') } & = 
		\pi_{\nu_c} \left ( \sum\nolimits_{\bar{r} = 1}^{R_{\nu_c}} \padcol_r \big ( \weightmat{\nu_c}_{:, r}\big )_{\bar{r}, r'} \left [  \tenp_{\nu' \in \children (\nu_c)} \htftensorpart{\nu'}{\bar{r}}  \right ] \right ) \\
		& = \pi_{\nu} \left ( \sum\nolimits_{\bar{r} = 1}^{R_{\nu_c}} 0 \cdot \left [  \tenp_{\nu' \in \children (\nu_c)} \htftensorpart{\nu'}{\bar{r}}  \right ] \right ) \\
		& = 0
		\text{\,.}
	\end{split}
	\]
	Putting it all together, for any $r' \in [R_{\parent(\nu)}]$ we may write:
	\[
	\V^{(\nu, r')} = \pi_{\nu} \left ( \sum\nolimits_{\bar{r} = 1}^{R_\nu} \weightmat{\nu}_{\bar{r}, r'} \left [  \tenp_{\nu' \in \children (\nu)} \U^{(\nu', \bar{r})}  \big ] \right ) = \pi_{\nu} \left ( \weightmat{\nu}_{r, r'} \right [  \tenp_{\nu' \in \children (\nu)} \htftensorpart{\nu'}{r}  \big ] \right )
	\text{\,.}
	\]
	From this point, following steps analogous to those used for proving \eq~\eqref{htf:eq:tensorend_zero_rows_loc_comp} based on \eq~\eqref{htf:eq:tensorpart_U} yields \eq~\eqref{htf:eq:tensorend_zero_cols_loc_comp}.
\end{proof}

\begin{lemma}
	\label{htf:lem:weightvec_sq_norm_time_deriv}
	For any $\nu \in \interior (\htmodetree)$, $\nu_c \in \children (\nu)$, and $r \in [ R_{\nu}]$:
	\[
	\frac{d}{dt} \norm1{ \weightmat{\nu}_{r, :} (t) }^2 = 2 \htfcompnorm{\nu}{r} (t) \inprod{ - \nabla \htfendloss ( \tensorend (t)) }{ \htfcomp{\nu}{r} (t) } = \frac{d}{dt} \norm{ \weightmat{\nu_c}_{:,r} (t) }^2 
	\text{\,,}
	\]
	where $\htfcomp{\nu}{r} (t)$ is as defined in \thm~\ref{htf:thm:loc_comp_norm_bal_dyn}.
\end{lemma}
\begin{proof}
	Differentiating $\normbig{ \weightmat{\nu_c}_{:, r} (t) }^2$ with respect to time we get:
	\[
	\begin{split}
	\frac{ d }{ dt } \norm1{\weightmat{ \nu_c}_{:, r} (t) }^2 & = 2 \inprod{ \weightmat{ \nu_c }_{:, r} (t) }{ \frac{d}{dt} \weightmat{\nu_c}_{:, r} (t) } \\
	& = - 2 \inprod{  \weightmat{ \nu_c }_{:, r} (t) }{ \frac{ \partial }{ \partial \weightmat{\nu_c}_{:, r} } \htfobj \big ( \big ( \weightmat{\nu'} (t) \big )_{\nu' \in \htmodetree} \big ) }
	\text{\,.}
	\end{split}
	\]
	By \eq~\eqref{htf:eq:ht_col_weightvec_grad} from \lem~\ref{htf:lem:ht_weightvec_grad} we have that:
	\[
	\frac{ d }{ dt } \norm1{\weightmat{ \nu_c}_{:, r} (t) }^2 = - 2 \inprod{ \nabla \htfendloss ( \tensorend (t) ) }{ \tensorendmap \Big ( \big ( \weightmat{ \nu' } (t) \big )_{ \nu' \in \htmodetree \setminus \{ \nu_c \} }, \padcol_r \big (  \weightmat{ \nu_c }_{:, r} (t) \big ) \Big ) }
	\text{\,.}
	\]
	Then, applying \eq~\eqref{htf:eq:tensorend_zero_cols_loc_comp} from \lem~\ref{htf:lem:tensorend_comp_eq_zeroing_cols_or_rows} concludes:
	\[
	\frac{d}{dt} \norm1{ \weightmat{\nu_c}_{:,r} (t) }^2 = 2 \htfcompnorm{\nu}{r} (t) \inprod{ - \nabla \htfendloss ( \tensorend (t)) }{ \htfcomp{\nu}{r} (t) }
	\text{\,.}
	\]
	
	A similar argument yields the desired result for $\normbig{ \weightmat{\nu}_{r, :} (t) }^2$.
	Differentiating with respect to time we obtain:
	\[
	\frac{d}{dt} \norm1{ \weightmat{\nu}_{r, :} (t) }^2 = 2 \inprod{ \weightmat{ \nu }_{r, :} (t) }{ \frac{d}{dt} \weightmat{\nu}_{r, :} (t) } = - 2 \inprod{  \weightmat{ \nu }_{r, :} (t) }{ \frac{ \partial }{ \partial \weightmat{\nu}_{r, :} } \htfobj \big ( \big ( \weightmat{\nu'} (t) \big )_{\nu' \in \htmodetree} \big ) }
	\text{\,.}
	\]
	By \eq~\eqref{htf:eq:ht_row_weightvec_grad} from \lem~\ref{htf:lem:ht_weightvec_grad} we may write:
	\[
	\frac{d}{dt} \norm1{ \weightmat{\nu}_{r, :} (t) }^2 = - 2 \inprod{ \nabla \htfendloss ( \tensorend (t) ) }{ \tensorendmap \Big ( \big ( \weightmat{ \nu' } (t) \big )_{ \nu' \in \htmodetree \setminus \{ \nu \} }, \padrow_r \big (  \weightmat{ \nu }_{r, :} (t) \big ) \Big ) }
	\text{\,.}
	\]
	Lastly, applying \eq~\eqref{htf:eq:tensorend_zero_rows_loc_comp} from \lem~\ref{htf:lem:tensorend_comp_eq_zeroing_cols_or_rows} completes the proof:
	\[
	\frac{d}{dt} \norm1{ \weightmat{\nu}_{r, :} (t) }^2 = 2 \htfcompnorm{\nu}{r} (t) \inprod{ - \nabla \htfendloss ( \tensorend (t)) }{ \htfcomp{\nu}{r} (t) }
	\text{\,.}
	\]
\end{proof}

\begin{lemma}
	\label{htf:lem:bal_zero_stays_zero}
	Let $\nu \in \interior (\htmodetree)$ and $r \in [ R_\nu ]$.
	If there exists a time $t_0 \geq 0$ at which $\wbf (t_0) = 0$ for all $\wbf \in \localcomp (\nu, r)$, then:
	\[
	\wbf (t) = 0 \quad , t \geq 0 ~,~ \wbf \in \localcomp (\nu, r)
	\text{\,,}
	\]
	\ie~$\wbf (t)$ is identically zero for all $\wbf \in \localcomp( \nu, r)$.
\end{lemma}
\begin{proof}
	Standard existence and uniqueness theorems (\eg~Theorem 2.2 in~\cite{teschl2012ordinary}) imply that the system of differential equations governing gradient flow over $\htfobj$ (\eq~\eqref{htf:eq:gf_htf}) has a unique solution that passes through $\big ( \weightmat{\nu'} (t_0) \big )_{\nu' \in \htmodetree}$ at time $t_0$.
	It therefore suffices to show that there exist $\big (\widebar{\Wbf}^{(\nu')} (t) \big )_{\nu' \in \htmodetree}$ satisfying \eq~\eqref{htf:eq:gf_htf} such that $\widebar{\Wbf}^{(\nu')} (t_0) = \weightmat{\nu'} (t_0)$ for all $\nu' \in \htmodetree$, for which $\widebar{\Wbf}^{ (\nu) }_{r, :} (t)$ and $\big ( \widebar{\Wbf}^{(\nu_c)}_{:, r} (t) \big )_{\nu_c \in \children (\nu)}$ are zero for all $t \geq 0$ (recall that $\localcomp (\nu, r)$ consists of $\weightmat{\nu}_{r, :}$ and $\big ( \weightmat{\nu_c}_{:, r} \big )_{\nu_c \in \children (\nu)}$).
	
	We denote by $\Theta_{\nu, r} (t)$ all factorization weights at time $t \geq 0$, except for those in $\localcomp (\nu, r)$, \ie:
	\[
	\Theta_{\nu , r } (t) := \big ( \Wbf^{(\nu')} (t)  \big )_{\nu' \in \htmodetree \setminus \left ( \{ \nu \} \cup \children (\nu) \right ) } \cup \big ( \Wbf^{(\nu_c)}_{:, r'} (t) \big )_{\nu_c \in \children (\nu), r' \in [R_\nu] \setminus \{ r\} } \cup \big ( \Wbf^{(\nu)}_{r', :} (t) \big )_{r' \in [R_\nu] \setminus \{ r \}}
	\text{\,.}
	\]
	We construct $\big (\widebar{\Wbf}^{(\nu')} (t) \big )_{\nu' \in \htmodetree}$ as follows.
	First, let $\widebar{\Wbf}^{ (\nu) }_{r, :} (t) := 0$ and $\widebar{\Wbf}^{(\nu_c)}_{:, r} (t) := 0$ for all $\nu_c \in \children (\nu)$ and $t \geq 0$.
	Then, considering $\Wbf^{ (\nu) }_{r, :} (t)$ and $\big ( \Wbf^{(\nu_c)}_{:, r} (t) \big )_{\nu_c \in \children (\nu)}$ as fixed to zero, we denote by $\widebar{\phi}_{\mathrm{H}} ( \Theta_{\nu, r} (t) )$ the induced objective over all other weights, and let
	\[
	\widebar{\Theta}_{\nu , r } (t) := \big ( \widebar{\Wbf}^{(\nu')} (t)  \big )_{\nu' \in \htmodetree \setminus \left ( \{ \nu \} \cup \children (\nu) \right ) } \cup \big ( \widebar{\Wbf}^{(\nu_c)}_{:, r'} (t) \big )_{\nu_c \in \children (\nu) , r' \in [R_\nu] \setminus \{ r\}} \cup \big ( \widebar{\Wbf}^{(\nu)}_{r', :} (t) \big )_{r' \in [R_\nu] \setminus \{ r \}}
	\text{\,}
	\]
	be a gradient flow path over $\widebar{\phi}_{\mathrm{H}}$ satisfying $\widebar{\Theta}_{\nu, r} (t_0) =\Theta_{\nu, r} (t_0)$.
	By definition, it holds that $\widebar{ \Wbf }^{(\nu')} (t_0) = \weightmat{\nu'} (t_0)$ for all $\nu' \in \htmodetree$.
	Thus, it remains to show that $\big ( \widebar{\Wbf}^{(\nu')} (t) \big )_{\nu' \in \htmodetree}$ obey the differential equations defining gradient flow over $\htfobj$ (\eq~\eqref{htf:eq:gf_htf}).
	To see it is so, notice that since $\widebar{\Wbf}^{ (\nu) }_{r, :} (t)$ and $\big ( \widebar{\Wbf}^{(\nu_c)}_{:, r} (t) \big )_{\nu_c \in \children (\nu)}$ are identically zero, by the definition of $\widebar{\phi}_{\mathrm{H}}$ we have that:
	\be
	\frac{d}{dt} \widebar{\Theta}_{\nu, r} (t) = - \frac{ d }{ d \Theta_{\nu, r} } \widebar{\phi}_{\mathrm{H}} \left ( \widebar{\Theta}_{\nu, r} (t) \right ) = - \frac{ \partial }{ \partial \Theta_{\nu, r} } \htfobj \big (  \big ( \widebar{\Wbf}^{(\nu')} (t) \big )_{\nu' \in \htmodetree} \big ) 
	\text{\,.}
	\label{htf:eq:bar_all_but_local_comp_time_deriv}
	\ee
	Furthermore, by \lem~\ref{htf:lem:ht_weightvec_grad_zero_comp} we obtain:
	\be
	\frac{d}{dt} \widebar{\Wbf}^{ (\nu) }_{r, :} (t) = 0 = - \frac{ \partial }{ \partial \weightmat{\nu}_{r, :} } \htfobj \big ( \big ( \widebar{\Wbf}^{(\nu')} (t) \big )_{\nu' \in \htmodetree} \big )
	\text{\,,}
	\label{htf:eq:bar_row_zero_time_deriv}
	\ee
	and for all $\nu_c \in \children (\nu)$:
	\be
	\frac{d}{dt} \widebar{\Wbf}^{ (\nu_c) }_{:, r} (t) = 0 = - \frac{ \partial }{ \partial \weightmat{\nu_c}_{:, r} } \htfobj \big ( \big ( \widebar{\Wbf}^{(\nu')} (t) \big )_{\nu' \in \htmodetree} \big )
	\text{\,.}
	\label{htf:eq:bar_col_zero_time_deriv}
	\ee
	Combining \eqs~\eqref{htf:eq:bar_all_but_local_comp_time_deriv},~\eqref{htf:eq:bar_row_zero_time_deriv}, and~\eqref{htf:eq:bar_col_zero_time_deriv}, completes the proof:
	\[
	\frac{d}{dt} \widebar{ \Wbf }^{(\nu')} (t) = - \frac{ \partial }{ \partial \weightmat{\nu'} } \htfobj \big ( \big ( \widebar{\Wbf}^{(\bar{\nu})} (t) \big )_{\bar{\nu} \in \htmodetree} \big ) \quad , \nu' \in \htmodetree
	\text{\,.}
	\]
\end{proof}

\subsection{Proof of Lemma~\ref{htf:lem:loc_comp_ht_rank_bound}}
\label{htf:app:proofs:loc_comp_ht_rank_bound}
The proof follows a line similar to that of \thm~7 in~\cite{cohen2018boosting}, extending from binary to arbitrary trees its upper bound on $( \rank \mat{\tensorend }{ \nu} )_{\nu \in \htmodetree}$.

Towards deriving a matricized form of \eq~\eqref{htf:eq:ht_end_tensor}, we define the notion of \emph{index set reduction}.
The reduction of $\nu \in \htmodetree$ onto $\nu' \in \htmodetree$, whose elements are denoted by $i_1 < \cdots < i_{ \abs{\nu'} }$, is defined by:
\[
\nu |_{\nu'} := \left \{ n \in [ \abs{ \nu' } ] : i_n \in \nu \cap \nu' \right \}
\text{\,.}
\]

Now, fix $\nu \in \interior (\htmodetree)$ and $\nu_c \in \children (\nu)$.
By \lem~\ref{htf:lem:tenp_eq_kronp} and the linearity of the matricization operator, we may write the computation of $\mat{ \tensorend }{ \nu_c}$ based on \eq~\eqref{htf:eq:ht_end_tensor} as follows:
\[
\begin{split}
	&\text{For $\bar{\nu} \in \left \{ \{1\}, \ldots, \{ N\} \right \}$ (traverses leaves of $\htmodetree$):} \\
	&\quad \htftensorpart{\bar{\nu}}{r} := \weightmat{\bar{\nu}}_{:, r} \quad , r \in [ R_{ \parent ( \bar{\nu} ) } ] \text{\,,} \\[2.5mm]
	&\text{for $\bar{\nu} \in \interior (\htmodetree) \setminus \left \{ [N] \right \}$ (traverses interior nodes of $\htmodetree$ from leaves to root):} \\
	&\quad \mat{ \htftensorpart{\bar{\nu}}{r} }{ \nu_c |_{ \bar{\nu} }} :=  \Qbf^{(\bar{\nu})} \left ( \sum\nolimits_{r' = 1}^{R_{\bar{\nu}}} \weightmat{\bar{\nu}}_{r', r} \left [ \kronp_{\nu' \in \children ( \bar{\nu} )} \mat{ \htftensorpart{ \nu' }{ r' } }{ \nu_c |_{\nu'} } \right ] \right ) \widebar{\Qbf}^{(\bar{\nu})} \quad ,r \in [ R_{\parent (\bar{\nu} )} ] \text{\,,} \\[2.5mm]
	& \mat{ \tensorend }{ \nu_c} = \Qbf^{([N])} \left ( \sum\nolimits_{r' = 1}^{R_{ [N] } } \weightmat{ [N] }_{r', 1} \left [ \kronp_{ \nu' \in \children ( [N] )} \mat{ \htftensorpart{\nu'}{r'} }{ \nu_c |_{\nu'}  } \right ] \right ) \widebar{\Qbf}^{([N])}
	\text{\,,}
\end{split}
\]
where $\Qbf^{(\bar{\nu})}$ and $\widebar{\Qbf}^{ (\bar{\nu})}$, for $\bar{\nu} \in \interior (\htmodetree)$, are permutation matrices rearranging the rows and columns, respectively, to accord with an ascending order of $\bar{\nu}$, \ie~they fulfill the role of $\pi_{\bar{\nu}}$ in \eq~\eqref{htf:eq:ht_end_tensor}.
For $r \in [R_{ \parent (\nu) }]$, let us focus on $\matbig{\htftensorpart{\nu}{r} }{ \nu_c |_\nu}$.
Since $\nu_c |_{\nu_c} = [\abs{ \nu_c } ]$ and $\nu_c |_{\nu'} = \emptyset$ for all $\nu' \in \children (\nu) \setminus \{ \nu_c \}$, we have that:
{\fontsize{9}{9}\selectfont
\[
\begin{split}
	& \mat{\htftensorpart{\nu}{r} }{ \nu_c |_\nu} \\
	& \hspace{0mm} = \Qbf^{(\nu)} \left ( \sum_{r' = 1}^{R_{\nu}} \weightmat{\nu}_{r', r} \left [ \left ( \kronp_{\nu' \in \lsib ( \nu_c )} \mat{ \htftensorpart{ \nu' }{ r' } }{ \emptyset} \right ) \kronp \mat{ \htftensorpart{\nu_c}{r'} }{ [\abs{ \nu_c } ]} \kronp \left ( \kronp_{\nu' \in \rsib ( \nu_c )} \mat{ \htftensorpart{ \nu' }{ r' } }{ \emptyset} \right ) \right ] \right ) \widebar{\Qbf}^{ (\nu)} 
	\text{\,.}
\end{split}
\]
}
Notice that $\matbig{ \htftensorpart{ \nu' }{ r' } }{ \emptyset}$ is a row vector, whereas $\matbig{ \htftensorpart{\nu_c}{r'} }{ [\abs{ \nu_c } ]}$ is a column vector, for each $\nu' \in \children (\nu) \setminus \{ \nu_c\}$ and $r' \in [R_\nu]$.
Commutativity of the Kronecker product between a row and column vectors therefore leads to:
{\fontsize{9}{9}\selectfont
\[
\begin{split}
	\mat{\htftensorpart{\nu}{r} }{ \nu_c |_\nu} & = \Qbf^{(\nu)} \left ( \sum\nolimits_{r' = 1}^{R_{\nu}} \weightmat{\nu}_{r', r} \left [ \mat{ \htftensorpart{\nu_c}{r'} }{ [\abs{ \nu_c } ]} \kronp \left ( \kronp_{\nu' \in \children ( \nu ) \setminus \{ \nu_c \}} \mat{ \htftensorpart{ \nu' }{ r' } }{ \emptyset} \right ) \right ] \right ) \widebar{\Qbf}^{ (\nu)} \\
	& = \Qbf^{(\nu)} \left ( \sum\nolimits_{r' = 1}^{R_{\nu}} \weightmat{\nu}_{r', r} \left [ \mat{ \htftensorpart{\nu_c}{r'} }{ [\abs{ \nu_c } ]} \left ( \kronp_{\nu' \in \children ( \nu ) \setminus \{ \nu_c \}} \mat{ \htftensorpart{ \nu' }{ r' } }{ \emptyset} \right ) \right ] \right ) \widebar{\Qbf}^{ (\nu)}
	\text{\,,}
\end{split}
\]
}
where the second equality is by the fact that for any column vector $\ubf$ and row vector $\vbf$ it holds that $\ubf \kronp \vbf = \ubf \vbf$.
Defining $\Bbf^{(\nu_c)}$ to be the matrix whose column vectors are $\matbig{ \htftensorpart{\nu_c}{1} }{ [\abs{ \nu_c } ]}, \ldots, \matbig{ \htftensorpart{\nu_c}{ R_\nu } }{ [\abs{ \nu_c } ]}$, we can express the term between $\Qbf^{(\nu)}$ and $\widebar{\Qbf}^{ (\nu)}$ in the equation above as a $\Bbf^{(\nu_c)} \Abf^{(\nu, r)}$, where $\Abf^{(\nu, r)}$ is defined to be the matrix whose rows are $\weightmat{\nu}_{1, r} \big ( \kronp_{\nu' \in \children ( \nu ) \setminus \{ \nu_c \}} \matbig{ \htftensorpart{ \nu' }{ 1 } }{ \emptyset} \big ), \ldots, \weightmat{\nu}_{R_\nu, r} \big ( \kronp_{\nu' \in \children ( \nu ) \setminus \{ \nu_c \}} \matbig{ \htftensorpart{ \nu' }{ R_\nu } }{ \emptyset} \big )$.
That is:
\be
\mat{\htftensorpart{\nu}{r}}{\nu_c |_\nu} = \Qbf^{(\nu)} \Bbf^{(\nu_c)} \Abf^{(\nu, r)} \widebar{\Qbf}^{ (\nu)}
\text{\,.}
\label{htf:eq:tensorpart_nu_mat_B}
\ee
The proof proceeds by propagating  $\Bbf^{(\nu_c)}$ and the left permutation matrices up the tree, until reaching a representation of $\mat{\tensorend }{ \nu_c}$ as a product of matrices that includes $\Bbf^{(\nu_c)}$.
Since $\Bbf^{(\nu_c)}$ has $R_\nu$ columns, this will imply that the rank of $\mat{\tensorend}{\nu_c}$ is at most $R_\nu$, as required.

We begin with the propagation step from $\nu$ to $\parent (\nu)$.
For $r \in [ R_{ \parent ( \parent ( \nu ) )} ]$, we examine:
{
\fontsize{9}{9}\selectfont
\[
\begin{split}
	& \big ( \Qbf^{ (\parent (\nu))} \big )^{-1} \matbig{ \htftensorpart{ \parent (\nu) }{r} }{\nu_c |_{ \parent (\nu) } } \big ( \widebar{\Qbf}^{  ( \parent (\nu) ) } \big )^{-1} \\
	& \hspace{1mm} = \sum\nolimits_{r' = 1}^{ R_{ \parent (\nu) } }\weightmat{\parent (\nu)}_{r', r} \left [ \left ( \kronp_{\nu' \in \lsib ( \nu )} \mat{ \htftensorpart{ \nu' }{ r' } }{\nu_c |_{\nu'}} \right ) \kronp \mat{ \htftensorpart{\nu}{r'} }{\nu_c |_{\nu}} \kronp \left ( \kronp_{\nu' \in \rsib ( \nu )} \mat{ \htftensorpart{ \nu' }{ r' } }{\nu_c |_{\nu'}} \right )\right ]
	\text{\,.}
\end{split}
\]
}
Plugging in \eq~\eqref{htf:eq:tensorpart_nu_mat_B} while noticing that $\nu_c |_{\nu'} = \emptyset$ for any $\nu'$ which is not an ancestor of $\nu_c$, we arrive~at:
\be
\fontsize{9}{9}\selectfont
\sum\nolimits_{r' = 1}^{ R_{ \parent (\nu) } }\weightmat{\parent (\nu)}_{r', r} \left [ \left ( \kronp_{\nu' \in \lsib ( \nu )} \mat{ \htftensorpart{ \nu' }{ r' } }{\emptyset} \right ) \kronp \left ( \Qbf^{(\nu)} \Bbf^{(\nu_c)} \Abf^{(\nu, r')} \widebar{\Qbf}^{ (\nu)} \right ) \kronp \left ( \kronp_{\nu' \in \rsib ( \nu )} \mat{ \htftensorpart{ \nu' }{ r' } }{\emptyset} \right ) \right ]
\text{\,.}
\label{htf:eq:tensorpart_pa_nu_mat_no_Q}
\ee
Let $r' \in [ R_{ \parent (\nu) } ]$.
Since $\matbig{ \htftensorpart{ \nu' }{ r' } }{\emptyset}$ is a row vector for any $\nu' \in \children ( \parent (\nu) ) \setminus \{ \nu \}$, so are $\kronp_{\nu' \in \lsib ( \nu )} \matbig{ \htftensorpart{ \nu' }{ r' } }{\emptyset}$ and $\kronp_{\nu' \in \rsib ( \nu )} \matbig{ \htftensorpart{ \nu' }{ r' } }{\emptyset}$.
Applying \lem~\ref{htf:lem:kronp_mixed_prod_row_vec} twice we therefore have that:
{\fontsize{9}{9}\selectfont
\[
\begin{split}
	& \left ( \kronp_{\nu' \in \lsib ( \nu )} \mat{ \htftensorpart{ \nu' }{ r' } }{\emptyset} \right ) \kronp \left ( \Qbf^{(\nu)} \Bbf^{(\nu_c)} \Abf^{(\nu, r')} \widebar{\Qbf}^{ (\nu)} \right ) \kronp \left ( \kronp_{\nu' \in \rsib ( \nu )} \mat{ \htftensorpart{ \nu' }{ r' } }{\emptyset} \right ) \\[1mm]
	& \quad\quad = \left ( \Qbf^{(\nu)} \Bbf^{(\nu_c)} \left [ \left ( \kronp_{\nu' \in \lsib ( \nu )} \mat{ \htftensorpart{ \nu' }{ r' } }{\emptyset} \right ) \kronp \left ( \Abf^{(\nu, r')} \widebar{\Qbf}^{ (\nu)} \right ) \right ] \right ) \kronp \left ( \kronp_{\nu' \in \rsib ( \nu )} \mat{ \htftensorpart{ \nu' }{ r' } }{\emptyset} \right ) \\[1mm]
	& \quad\quad = \Qbf^{(\nu)} \Bbf^{(\nu_c)} \left [ \left ( \kronp_{\nu' \in \lsib ( \nu )} \mat{ \htftensorpart{ \nu' }{ r' } }{\emptyset} \right ) \kronp \left ( \Abf^{(\nu, r')} \widebar{\Qbf}^{ (\nu)} \right )  \kronp \left ( \kronp_{\nu' \in \rsib ( \nu )} \mat{ \htftensorpart{ \nu' }{ r' } }{\emptyset} \right ) \right ]
	\text{\,.}
\end{split}
\]
}
Going back to \eq~\eqref{htf:eq:tensorpart_pa_nu_mat_no_Q}, we obtain:
\be
\fontsize{9}{9}\selectfont
\begin{split}
	& \sum_{r' = 1}^{ R_{ \parent (\nu) } }\weightmat{\parent (\nu)}_{r', r} \left [ \left (  \kronp_{\nu' \in \lsib ( \nu )} \mat{ \htftensorpart{ \nu' }{ r' } }{\emptyset} \right ) \kronp \left ( \Qbf^{(\nu)} \Bbf^{(\nu_c)} \Abf^{(\nu, r')} \widebar{\Qbf}^{ (\nu)} \right ) \kronp \left ( \kronp_{\nu' \in \rsib ( \nu )} \mat{ \htftensorpart{ \nu' }{ r' } }{\emptyset} \right ) \right ] \\[1mm]
	& =  \sum_{r' = 1}^{ R_{ \parent (\nu) } }\weightmat{\parent (\nu)}_{r', r}  \left ( \Qbf^{(\nu)} \Bbf^{(\nu_c)} \left [ \left ( \kronp_{\nu' \in \lsib ( \nu )} \mat{ \htftensorpart{ \nu' }{ r' } }{\emptyset} \right ) \kronp \left ( \Abf^{(\nu, r')} \widebar{\Qbf}^{ (\nu)} \right )  \kronp \left ( \kronp_{\nu' \in \rsib ( \nu )} \mat{ \htftensorpart{ \nu' }{ r' } }{\emptyset} \right ) \right ] \right ) \\[1mm]
	& =  \Qbf^{(\nu)} \Bbf^{(\nu_c)} \left ( \sum\nolimits_{r' = 1}^{ R_{ \parent (\nu) } }\weightmat{\parent (\nu)}_{r', r} \left [ \left ( \kronp_{\nu' \in \lsib ( \nu )} \mat{ \htftensorpart{ \nu' }{ r' } }{\emptyset} \right ) \kronp \left ( \Abf^{(\nu, r')} \widebar{\Qbf}^{ (\nu)} \right )\kronp \left ( \kronp_{\nu' \in \rsib ( \nu )} \mat{ \htftensorpart{ \nu' }{ r' } }{\emptyset} \right ) \right ]  \right )
	\text{.}
\end{split}
\label{htf:eq:tensorpart_pa_nu_mat_no_Q_B_out}
\ee
For brevity, we denote the matrix multiplying $\Qbf^{(\nu)} \Bbf^{(\nu_c)}$ from the right in the equation above by $\Abf^{(\parent (\nu), r)}$, \ie:
{\fontsize{9}{9}\selectfont
\[
\Abf^{(\parent (\nu), r)} := \sum_{r' = 1}^{ R_{ \parent (\nu) } }\weightmat{\parent (\nu)}_{r', r} \left [ \left ( \kronp_{\nu' \in \lsib ( \nu )} \mat{ \htftensorpart{ \nu' }{ r' } }{\emptyset} \right ) \kronp \left ( \Abf^{(\nu, r')} \widebar{\Qbf}^{ (\nu)} \right )\kronp \left ( \kronp_{\nu' \in \rsib ( \nu )} \mat{ \htftensorpart{ \nu' }{ r' } }{\emptyset} \right ) \right ]
\text{\,.}
\]
}
Recalling that the expression in \eq~\eqref{htf:eq:tensorpart_pa_nu_mat_no_Q_B_out} is of 
\[
\big ( \Qbf^{ (\parent (\nu))} \big )^{-1} \matbig{ \htftensorpart{ \parent (\nu) }{r} }{\nu_c |_{ \parent (\nu) } } \big ( \widebar{\Qbf}^{  ( \parent (\nu) ) } \big )^{-1}
\]
completes the propagation step:
\[
\mat{ \htftensorpart{ \parent (\nu) }{r} }{\nu_c |_{ \parent (\nu) } } = \Qbf^{( \parent (\nu) ) }  \Qbf^{(\nu)} \Bbf^{(\nu_c)} \Abf^{ ( \parent (\nu), r)} \widebar{\Qbf}^{ ( \parent (\nu) )}
\text{\,.}
\]

Continuing this process, we propagate $\Bbf^{(\nu_c)}$, along with the left permutation matrices, upwards in the tree until reaching the root.
This brings forth the following representation of $\mat{ \tensorend }{\nu_c}$:
\[
\mat{ \tensorend }{\nu_c} =  \Qbf^{([N])} \Qbf \Bbf^{ (\nu_c )} \Abf^{ ([N])} \widebar{\Qbf}^{ ( [N] )}
\text{\,,}
\]
for appropriate $\Qbf$ and $\Abf^{([N])}$ encompassing the propagated permutation matrices and the “remainder'' of the decomposition, respectively.
Since $\Bbf^{(\nu_c)}$ has $R_\nu$ columns, we may conclude:
\[
\rank \mat{\tensorend}{\nu_c} \leq \rank \, \Bbf^{(\nu_c)} \leq R_\nu
\text{\,.}
\]
\qed

\subsection{Proof of Lemma~\ref{htf:lem:loc_comp_sq_norm_diff_invariant}}
\label{htf:app:proofs:loc_comp_sq_norm_diff_invariant}
For any $\nu \in \interior (\htmodetree)$, $r \in [R_\nu]$, and $\wbf, \wbf' \in \localcomp (\nu, r)$, \lem~\ref{htf:lem:weightvec_sq_norm_time_deriv} implies that:
\[
\frac{d}{dt} \normbig{ \wbf (t) }^2 = 2 \htfcompnorm{\nu}{r} (t) \inprod{ - \nabla \htfendloss ( \tensorend (t)) }{ \htfcomp{\nu}{r} (t) } = \frac{d}{dt} \normbig{ \wbf' (t) }^2 
\text{\,,}
\]
where $\htfcomp{\nu}{r} (t)$ is as defined in \thm~\ref{htf:thm:loc_comp_norm_bal_dyn}.
For $t \geq 0$, integrating both sides with respect to time leads to:
\[
\normbig{ \wbf (t) }^2 - \normbig{\wbf (0) }^2 = \normbig{ \wbf' (t) }^2 - \normbig{ \wbf' (0) }^2
\text{\,.}
\]
Rearranging the equality above yields the desired result.
\qed

\subsection{Proof of Theorem~\ref{htf:thm:loc_comp_norm_bal_dyn}}
\label{htf:app:proofs:loc_comp_norm_bal_dyn}
Let $t \geq 0$.

First, suppose that $\htfcompnorm{\nu}{r} (t) = 0$.
Since the unbalancedness magnitude at initialization is zero, from \lem~\ref{htf:lem:loc_comp_sq_norm_diff_invariant} we know that $\norm{ \wbf (t) } = \norm{ \wbf' (t) }$ for any $\wbf, \wbf' \in \localcomp (\nu, r)$.
Hence, the fact that $\htfcompnorm{\nu}{r} (t) = 0$ implies that $\norm{ \wbf (t) } = 0$ for all $\wbf \in \localcomp (\nu, r)$.
\lem~\ref{htf:lem:bal_zero_stays_zero} then establishes that $\htfcompnorm{\nu}{r} (t')$ is identically zero through time, in which case both sides of \eq~\eqref{htf:eq:loc_comp_norm_bal_dyn} are equal to zero.

We now move to the case where $\htfcompnorm{\nu}{r} (t) > 0$.
Since $\htfcompnorm{\nu}{r} (t) = \normnoflex{\tenp_{\wbf \in \localcomp (\nu, r)} \wbf (t) } = \prod\nolimits_{ \wbf \in \localcomp (\nu, r) } \norm{ \wbf (t) }$ (the norm of a tensor product is equal to the product of the norms), by the product rule of differentiation we have that:
\[
\frac{d}{dt} \htfcompnorm{\nu}{r} (t)^2 = \sum\nolimits_{ \wbf \in \localcomp (\nu, r) } \frac{d}{dt} \norm{ \wbf (t) }^2 \cdot \prod\nolimits_{ \wbf' \in \localcomp (\nu, r) \setminus \{ \wbf \} } \norm{ \wbf' (t) }^2
\text{\,.}
\]
Applying \lem~\ref{htf:lem:weightvec_sq_norm_time_deriv} then leads to:
\[
\begin{split}
	& \frac{d}{dt} \htfcompnorm{\nu}{r} (t)^2 \\
	& = \sum\nolimits_{ \wbf \in \localcomp (\nu, r) } 2 \htfcompnorm{\nu}{r} (t) \inprod{ - \nabla \htfendloss ( \tensorend (t)) }{ \htfcomp{\nu}{r} (t) } \cdot \prod\nolimits_{ \wbf' \in \localcomp (\nu, r) \setminus \{ \wbf \} } \norm{ \wbf' (t) }^2 \\
	& = 2 \htfcompnorm{\nu}{r} (t) \inprod{ - \nabla \htfendloss ( \tensorend (t)) }{ \htfcomp{\nu}{r} (t) } \sum\nolimits_{ \wbf \in \localcomp (\nu, r) } \prod\nolimits_{ \wbf' \in \localcomp (\nu, r) \setminus \{ \wbf \} } \norm{ \wbf' (t) }^2
	\text{\,.}
\end{split}
\]
From the chain rule we know that $\frac{d}{dt} \htfcompnorm{\nu}{r} (t)^2 = 2 \htfcompnorm{\nu}{r} (t) \cdot \frac{d}{dt} \htfcompnorm{\nu}{r} (t)$ (note that $\frac{d}{dt} \htfcompnorm{\nu}{r} (t)$ surely exists because $\htfcompnorm{\nu}{r} (t) > 0$).
Thus:
\be
\begin{split}
	\frac{d}{dt} \htfcompnorm{\nu}{r} (t) & = \frac{1}{2} \htfcompnorm{\nu}{r} (t)^{-1} \frac{d}{dt} \htfcompnorm{\nu}{r} (t)^2 \\
	& = \inprod{ - \nabla \htfendloss ( \tensorend (t)) }{ \htfcomp{\nu}{r} (t) } \sum\nolimits_{ \wbf \in \localcomp (\nu, r) } \prod\nolimits_{ \wbf' \in \localcomp (\nu, r) \setminus \{ \wbf \} } \norm{ \wbf' (t) }^2
	\text{\,.}
\end{split}
\label{htf:eq:localcompnorm_time_deriv_interm_bal}
\ee
According to \lem~\ref{htf:lem:loc_comp_sq_norm_diff_invariant}, the unbalancedness magnitude remains zero through time, and so $\norm{ \wbf (t) } = \norm{ \wbf' (t) }$ for any $\wbf, \wbf' \in \localcomp (\nu, r)$.
Recalling that $L_\nu := \abs{\children (\nu)} + 1$ is the number of weight vectors in a local component at $\nu$, this implies that for each $\wbf \in \localcomp (\nu, r)$:
\be
\norm{ \wbf (t) }^2 =  \norm{ \wbf (t) }^{L_\nu \cdot \frac{2}{L_\nu} } = \left ( \prod\nolimits_{\wbf' \in \localcomp (\nu, r)} \norm{ \wbf' (t) } \right )^{\frac{2}{L_\nu}} =  \htfcompnorm{\nu}{r} (t)^{ \frac{2}{ L_\nu } }
\text{\,.}
\label{htf:eq:sq_weightvec_norm_eq_local_comp_norm}
\ee
Plugging \eq~\eqref{htf:eq:sq_weightvec_norm_eq_local_comp_norm} into \eq~\eqref{htf:eq:localcompnorm_time_deriv_interm_bal} completes the proof.
\qed

\subsection{Proof of Proposition~\ref{htf:prop:low_rank_dist_bound}}
\label{htf:app:proofs:low_rank_dist_bound}
We begin by establishing the following key lemma, which upper bounds the distance between the end tensor $\tensorend$ and the one obtained after setting a local component to zero.

\begin{lemma}
	\label{htf:lem:dist_from_pruned_local_comp_end_tensor}
	Let $\nu \in \interior ( \htmodetree )$ and $r \in [R_\nu]$.
	Denote by $\widebar{\W}_{\mathrm{H}}^{(\nu, r)}$ the end tensor obtained by pruning the $(\nu, r)$'th local component, \ie~by setting the $r$'th row of $\weightmat{\nu}$ and the $r$'th columns of $\big ( \weightmat{\nu_c} \big )_{\nu_c \in \children (\nu)}$ to zero.
	Then:
	\[
	\norm1{ \tensorend - \widebar{\W}_{\mathrm{H}}^{(\nu, r)} } \leq \htfcompnorm{\nu}{r} \cdot \prod\nolimits_{\nu' \in \htmodetree \setminus ( \{ \nu \} \cup \children (\nu) )} \norm1{ \weightmat{\nu'} }
	\text{\,.}
	\]
\end{lemma}
\begin{proof}
	Let $\big ( \widebar{ \Wbf }^{(\nu')} \big )_{\nu' \in \htmodetree}$ be the weight matrices corresponding to $\widebar{\W}_{\mathrm{H}}^{(\nu, r)}$, \ie~$\widebar{\Wbf}^{(\nu)}$ is the weight matrix obtained by setting the $r$'th row of $\weightmat{\nu}$ to zero, $\big ( \widebar{\Wbf}^{(\nu_c)} \big )_{\nu_c \in \children (\nu)}$ are the weight matrices obtained by setting the $r$'th columns of $\big ( \weightmat{\nu_c} \big )_{\nu_c \in \children (\nu) }$ to zero, and $\widebar{ \Wbf }^{(\nu')} = \weightmat{\nu'}$ for all $\nu' \in \htmodetree \setminus ( \{ \nu \} \cup \children (\nu))$.
	Accordingly, we denote by $\big ( \widebar{ \W }^{(\nu', r')} \big )_{\nu' \in \htmodetree, r' \in [R_{ \parent (\nu') }]}$ the intermediate tensors produced when computing $\widebar{\W}_{\mathrm{H}}^{(\nu, r)}$ according to \eq~\eqref{htf:eq:ht_end_tensor} (there denoted $\big ( \htftensorpart{\nu'}{r'} \big )_{\nu', r' }$).
	
	By definition, $\tensorendmap \big ( \big ( \weightmat{ \nu' } \big )_{ \nu' \in \htmodetree \setminus \subtree (\nu)}, \htftensorpart{\nu}{:} \big ) = \tensorend$ and $\tensorendmap \big ( \big ( \widebar{\Wbf}^{ (\nu') } \big )_{ \nu' \in \htmodetree \setminus \subtree (\nu)}, \widebar{\W}^{(\nu, :)} \big ) = \widebar{\W}_{\mathrm{H}}^{(\nu, r)}$.
	Since $\tensorendmap$ is multilinear (\lem~\ref{htf:lem:ht_multilinear}) and $\widebar{ \Wbf }^{(\nu')} = \weightmat{\nu'}$ for all $\nu' \in \htmodetree \setminus \subtree (\nu)$, we have that:
	\[
	\begin{split}
		\norm*{ \tensorend - \widebar{\W}_{\mathrm{H}}^{(\nu, r)} } & = \norm*{ \tensorendmap \big ( ( \weightmat{ \nu' } )_{ \nu' \in \htmodetree \setminus \subtree (\nu)}, \htftensorpart{\nu}{:} \big ) - \tensorendmap \big ( ( \widebar{\Wbf}^{ (\nu') } )_{ \nu' \in \htmodetree \setminus \subtree (\nu)}, \widebar{\W}^{(\nu, :)} \big ) } \\
		& = \norm*{ \tensorendmap \big ( ( \weightmat{ \nu' } )_{ \nu' \in \htmodetree \setminus \subtree (\nu)}, \htftensorpart{\nu}{:} -  \widebar{\W}^{(\nu, :)} \big ) }
		\text{\,.}
	\end{split}
	\]
	Heading from the root downwards, subsequent applications of \lem~\ref{htf:lem:stacked_tensorpart_norm_bound} over all nodes in the mode tree, except those belonging to the sub-tree whose root is $\nu$, then yield:
	\be
	\norm{ \tensorend - \widebar{\W}_{\mathrm{H}}^{(\nu, r)} } \leq \norm{  \htftensorpart{\nu}{:} -  \widebar{\W}^{(\nu, :)} } \cdot \prod\nolimits_{\nu' \in \htmodetree \setminus \subtree (\nu)} \norm{ \weightmat{\nu'} }
	\text{\,.}
	\label{htf:eq:tensorend_trimmed_comp_upper_bound_intermid}
	\ee
	Notice that for any $r' \in [R_{ \parent (\nu) } ]$:
	\[
	\begin{split}
		\left ( \htftensorpart{\nu}{:} -  \widebar{\W}^{(\nu, :)} \right )_{:, \ldots, :, r'} & = \sum_{\bar{r} \in [R_\nu]} \weightmat{\nu}_{\bar{r}, r'} \tenp_{ \nu_c \in \children (\nu) } \htftensorpart{\nu_c}{\bar{r}} - \sum_{\bar{r} \in [R_\nu] \setminus \{ r \}} \weightmat{\nu}_{\bar{r}, r'} \tenp_{ \nu_c \in \children (\nu) } \htftensorpart{\nu_c}{\bar{r}} \\
		& = \weightmat{\nu}_{r, r'} \tenp_{ \nu_c \in \children (\nu) } \htftensorpart{\nu_c}{r}
		\text{\,.}
	\end{split}
	\]
	Thus, a straightforward computation shows:
	\[
	\begin{split}
		\norm*{  \htftensorpart{\nu}{:} -  \widebar{\W}^{(\nu, :)} }^2 & = \sum\nolimits_{r' = 1}^{ R_{ \parent (\nu') } } \norm{  \weightmat{\nu}_{r, r'} \tenp_{ \nu_c \in \children (\nu) } \htftensorpart{\nu_c}{r} }^2 \\
		& = \sum\nolimits_{r' = 1}^{ R_{ \parent (\nu') } } \left ( \weightmat{\nu}_{r, r'} \right )^2 \cdot \prod\nolimits_{ \nu_c \in \children (\nu) } \norm*{ \htftensorpart{\nu_c}{r} }^2 \\
		& = \norm*{ \weightmat{\nu}_{r, :} }^2 \cdot \prod\nolimits_{ \nu_c \in \children (\nu) } \norm*{ \htftensorpart{\nu_c}{r} }^2
		\text{\,,}
	\end{split}
	\]
	where the second equality is by the fact that the norm of a tensor product is equal to the product of the norms.
	From \lem~\ref{htf:lem:tensorpart_norm_bound} we get that $\normbig{ \htftensorpart{\nu_c}{r} } \leq \normbig{ \weightmat{ \nu_c }_{:, r} } \cdot \prod\nolimits_{ \nu' \in \children (\nu_c) } \normbig{ \htftensorpart{\nu'}{:} }$ for all $\nu_c \in \children (\nu)$, which leads to:
	\[
	\begin{split}
		\norm*{  \htftensorpart{\nu}{:} -  \widebar{\W}^{(\nu, :)} }^2 & \leq \norm*{ \weightmat{\nu}_{r, :} }^2 \cdot \prod\nolimits_{ \nu_c \in \children (\nu) } \left ( \norm*{ \weightmat{ \nu_c }_{:, r} }^2 \cdot \prod\nolimits_{ \nu' \in \children (\nu_c) } \norm{ \htftensorpart{\nu'}{:} }^2 \right ) \\
		& = \left ( \htfcompnorm{\nu}{r} \right )^2 \cdot \prod\nolimits_{ \nu_c \in \children (\nu), \nu' \in \children (\nu_c) } \norm*{ \htftensorpart{\nu'}{:} }^2
		\text{\,.}
	\end{split}
	\]
	Taking the square root of both sides and plugging the inequality above into \eq~\eqref{htf:eq:tensorend_trimmed_comp_upper_bound_intermid}, we arrive at:
	\[
	\norm*{ \tensorend - \widebar{\W}_{\mathrm{H}}^{(\nu, r)} } \leq \htfcompnorm{\nu}{r} \cdot \prod\nolimits_{\nu_c \in \children (\nu), \nu' \in \children (\nu_c) } \norm*{ \htftensorpart{\nu'}{:} } \cdot \prod\nolimits_{\nu' \in \htmodetree \setminus \subtree (\nu)} \norm*{ \weightmat{\nu'} }
	\text{\,.}
	\]
	Applying \lem~\ref{htf:lem:stacked_tensorpart_norm_bound} iteratively over the sub-trees whose roots are $\children (\nu_c)$ gives:
	\[
	\prod\nolimits_{\nu_c \in \children (\nu), \nu' \in \children (\nu_c) } \norm*{ \htftensorpart{\nu'}{:} } \leq \prod\nolimits_{\nu' \in \subtree (\nu) \setminus ( \{ \nu \} \cup \children (\nu) )} \norm*{ \weightmat{\nu' } }
	\text{\,,}
	\]
	concluding the proof.
\end{proof}

\medskip

With \lem~\ref{htf:lem:dist_from_pruned_local_comp_end_tensor} in hand, we are now in a position to prove \prop~\ref{htf:prop:low_rank_dist_bound}.
Let
\[
\S := \left \{ (\nu, r) : \nu \in \interior (\htmodetree) , r \in \{ R'_{\nu} + 1, \ldots, R_\nu \right \}
\text{\,,}
\]
and denote by $\widebar{\W}_{\mathrm{H}}^{\S}$ the end tensor obtained by pruning all local components in $\S$, \ie~by setting to zero the $r$'th row of $\weightmat{\nu}$ and the $r$'th column of $\weightmat{\nu_c}$ for all $(\nu, r) \in \S$ and $\nu_c \in \children (\nu)$.
As can be seen from \eq~\eqref{htf:eq:ht_end_tensor}, we may equivalently discard these weight vectors instead of setting them to zero.
Doing so, we arrive at a representation of $\widebar{\W}_{\mathrm{H}}^{\S}$ as the end tensor of $\big ( \widebar{ \Wbf }^{(\nu)} \in \R^{R'_\nu \times R'_{ \parent (\nu) } } \big )_{\nu \in \htmodetree}$, where $R'_{ \parent ([N])} = 1$, $R'_{ \{n\}} = D_n$ for $n \in [N]$, and $\widebar{ \Wbf }^{(\nu)} = \weightmat{\nu}_{:R'_\nu, :R'_{\parent (\nu)}}$ for all $\nu \in \htmodetree$.
Hence, \lem~\ref{htf:lem:loc_comp_ht_rank_bound} implies that for any $\nu \in \htmodetree$ the rank of $\mat{ \widebar{\W}_{\mathrm{H}}^{\S} }{\nu}$ is at most $R'_{ \parent (\nu) }$.
This means that it suffices to show that:
\be
\norm*{ \tensorend - \widebar{\W}_{\mathrm{H}}^{\S} } \leq \epsilon
\text{\,.}
\label{htf:eq:tensorend_prunedtensorend_dist}
\ee
For $i \in [ \abs{\S} ]$, let $\S_i \subset \S$ be the set comprising the first $i$ local components in $\S$ according to an arbitrary order.
Adding and subtracting $\widebar{\W}_{\mathrm{H}}^{\S_i}$ for all $i \in [ \abs{\S} - 1]$, and applying the triangle inequality, we have:
\[
\norm*{ \tensorend - \widebar{\W}_{\mathrm{H}}^{\S} } \leq \sum\nolimits_{i = 0}^{\abs{\S} - 1} \norm*{ \widebar{\W}_{\mathrm{H}}^{\S_i} - \widebar{\W}_{\mathrm{H}}^{\S_{i + 1}} }
\text{\,,}
\]
where $\widebar{\W}_{\mathrm{H}}^{\S_0} := \tensorend$.
Upper bounding each term in the sum according to \lem~\ref{htf:lem:dist_from_pruned_local_comp_end_tensor}, while noticing that pruning a local component can only decrease the norms of weight matrices and other local components in the factorization, we obtain:
\[
\begin{split}
	\norm1{ \tensorend - \widebar{\W}_{\mathrm{H}}^{\S} } & \leq \sum\nolimits_{\nu \in \interior (\htmodetree) } \sum\nolimits_{r = R'_\nu + 1}^{R_\nu} \htfcompnorm{\nu}{r} \cdot \prod\nolimits_{\nu' \in \htmodetree \setminus \left ( \{ \nu \} \cup \children (\nu) \right )}  \normbig{ \weightmat{\nu'} } \\
	& \leq \sum\nolimits_{\nu \in \interior (\htmodetree) } B^{\abs{\htmodetree} - 1 - \abs{\children (\nu)}} \cdot \sum\nolimits_{r = R'_\nu + 1}^{R_\nu} \htfcompnorm{\nu}{r}
	\text{\,,}
\end{split}
\]
where the latter inequality is by recalling that $B = \max_{\nu \in \htmodetree} \normnoflex{ \weightmat{\nu}}$.
Since for all $\nu \in \interior (\htmodetree)$ we have that $\sum\nolimits_{r = R'_\nu + 1}^{R_\nu} \htfcompnorm{\nu}{r} \leq \epsilon \cdot (\abs{\htmodetree} - N)^{-1} B^{\abs{\children (\nu)} + 1 - \abs{\htmodetree} }$, \eq~\eqref{htf:eq:tensorend_prunedtensorend_dist} readily follows.
\qed

\subsection{Proof of Proposition~\ref{htf:prop:htr_multiple_minima}}
\label{htf:app:proofs:htr_multiple_minima}
Consider the tensor completion problem defined by the set of observed entries
\[
\Omega = \left \{ (1, \ldots, 1, 1, 1, 1), (1, \ldots, 1, 1, 1, 2), (1, \ldots, 1, 2, 2, 1), (1, \ldots, 1, 2, 2, 2) \right \}
\]
and ground truth $\W^* \in \R^{D_1 \times \cdots \times D_N}$, whose values at those locations are:
\be
\W^*_{1, \ldots, 1, 1, :2, :2} = \begin{bmatrix} 1 & 0 \\ ? & ? \end{bmatrix} ~~ , ~~ \W^*_{1, \ldots, 1, 2, :2, :2} = \begin{bmatrix} ? & ? \\ 0 & 1 \end{bmatrix} 
\text{\,,}
\label{htf:eq:ht_multiple_minima_observations}
\ee
where $?$ stands for an unobserved entry.
We define two solutions for the tensor completion problem, $\W$ and $\W'$ in $\R^{D_1 \times \cdots \times D_N}$, as follows:
\[
\begin{split}
& \W_{1, \ldots, 1, 1, :2, :2} := \begin{bmatrix} 1 & 0 \\ 1 & 0 \end{bmatrix} ~~ , ~~ \W_{1, \ldots, 1, 2, :2, :2} := \begin{bmatrix} 0 & 1 \\ 0 & 1 \end{bmatrix}  \text{\,,} \\[1mm]
& \W'_{1, \ldots, 1, 1, :2, :2} := \begin{bmatrix} 1 & 0 \\ 0 & 1 \end{bmatrix} ~~ , ~~ \W'_{1, \ldots, 1, 2, :2, :2} := \begin{bmatrix} 1 & 0 \\ 0 & 1 \end{bmatrix} 
\text{\,,}
\end{split}
\]
and the remaining entries of $\W$ and $\W'$ hold zero.
Clearly, $\L (\W) = \L (\W') = 0$.

Fix a mode tree $\htmodetree$ over $[N]$.
Since $\W$ and $\W'$ fit the observed entries their hierarchical tensor ranks with respect to $\htmodetree$, $ ( \rank \mat{ \W }{\nu} )_{\nu \in \htmodetree \setminus \{ [N] \}}$ and $ ( \rank \mat{ \W' }{\nu} )_{\nu \in \htmodetree \setminus \{ [N] \}}$, are in $\RR_\htmodetree$.
We prove that neither $ ( \rank \mat{ \W }{\nu} )_{\nu \in \htmodetree \setminus \{ [N] \}} \leq  ( \rank \mat{\W'}{\nu} )_{\nu \in \htmodetree \setminus \{ [N] \}}$ nor $( \rank \mat{ \W' }{\nu} )_{\nu \in \htmodetree \setminus \{ [N] \}} \leq  ( \rank \mat{\W}{\nu} )_{\nu \in \htmodetree \setminus \{ [N] \}}$ (with respect to the standard product partial order), by examining the matrix ranks of the matricizations of $\W$ and $\W'$ according to $\{ N - 2 \} \in \htmodetree$ and $\{ N - 1\} \in \htmodetree$ (recall that any mode tree has leaves $\{1\}, \ldots, \{N\}$).
For $\{ N - 2\}$, we have that $\rank \mat{\W}{\{N - 2\}} = 2$ whereas $\rank \mat{\W'}{\{ N - 2 \}} = 1$.
To see it is so, notice that:
\[
\mat{\W}{ \{N - 2\} }_{:2, :4} = \begin{bmatrix}
	1 & 0 & 1 & 0 \\
	0 & 1 & 0 & 1
\end{bmatrix} 
\quad , \quad
\mat{\W'}{\{N - 2\}}_{:2, :4} = \begin{bmatrix}
	1 & 0 & 0 & 1 \\
	1 & 0 & 0 & 1
\end{bmatrix} 
\text{\,,}
\]
and all other entries of $\mat{\W}{ \{N - 2 \} }$ and $\mat{\W'}{ \{N - 2\} }$ hold zero.
This means that $( \rank \mat{ \W }{\nu} )_{\nu \in \htmodetree \setminus \{ [N] \}} \leq  ( \rank \mat{\W'}{\nu})_{\nu \in \htmodetree \setminus \{ [N] \}}$ does not hold.
On the other hand, for $\{N - 1\}$ we have that $\rank \mat{\W}{\{N - 1\}} = 1$ while $\rank \mat{\W'}{\{N - 1\}} = 2$, because:
\[
\mat{\W}{ \{N - 1\}}_{:2, :4} = \begin{bmatrix}
	1 & 0 & 0 & 1 \\
	1 & 0 & 0 & 1
\end{bmatrix} 
\quad , \quad
\mat{\W'}{\{ N - 1\}}_{:2, :4} = \begin{bmatrix}
	1 & 0 & 1 & 0 \\
	0 & 1 & 0 & 1
\end{bmatrix} 
\text{\,,}
\]
and the remaining entries of $\mat{\W}{\{ N - 1\}}$ and $\mat{\W'}{\{ N - 1 \}}$ hold zero.
This implies that $( \rank \mat{ \W' }{\nu} )_{\nu \in \htmodetree \setminus \{ [N] \}} \leq  ( \rank \mat{\W}{\nu} )_{\nu \in \htmodetree \setminus \{ [N] \}}$ does not hold, and so the hierarchical tensor ranks of $\W$ and $\W'$ are incomparable, \ie~neither is smaller than or equal to the other.

It remains to show that there exists no 
\[
( R''_\nu )_{ \nu \in \htmodetree \setminus \{ [N] \} } \in \RR_\htmodetree \setminus \left \{ ( \rank \mat{\W}{\nu} )_{\nu \in \htmodetree \setminus \{ [N] \}},  ( \rank \mat{\W'}{\nu} )_{\nu \in \htmodetree \setminus \{ [N] \}} \right \}
\]
satisfying
\[
( R''_\nu )_{ \nu \in \htmodetree \setminus \{ [N] \} } \leq ( \rank \mat{ \W }{\nu} )_{\nu \in \htmodetree \setminus \{ [N] \}}
\]
or
\[
( R''_\nu )_{ \nu \in \htmodetree \setminus \{ [N] \} } \leq ( \rank \mat{ \W' }{\nu} )_{\nu \in \htmodetree \setminus \{ [N] \}}
\text{\,.}
\]
Assume by way of contradiction that there exists such $( R''_\nu )_{ \nu \in \htmodetree \setminus \{ [N] \} }$, and let $\W'' \in \R^{D_1 \times \cdots \times D_N}$ be a solution of this hierarchical tensor rank.
We now prove that \[
( R''_\nu )_{ \nu \in \htmodetree \setminus \{ [N] \} } \leq ( \rank \mat{ \W }{\nu} )_{\nu \in \htmodetree \setminus \{ [N] \}}
\]
entails a contradiction.
Since $( R''_\nu )_{ \nu \in \htmodetree \setminus \{ [N] \} }$ is not equal to the hierarchical tensor rank of $\W$, there exists $\nu \in \htmodetree \setminus \{[N]\}$ for which $\rank \mat{\W''}{\nu} = R''_\nu < \rank \mat{\W}{\nu}$.
Let us examine the possible cases:
\begin{itemize}
	\item If $\nu$ does not contain $N - 2, N - 1,$ and $N$, then $\rank \mat{\W}{\nu} = 1$ as all rows but the first of this matricization are zero.
	In this case $\rank \mat{\W''}{\nu}  = R''_\nu = 0$, implying that $\W''$ is the zero tensor, in contradiction to it fitting the (non-zero) observed entries from \eq~\eqref{htf:eq:ht_multiple_minima_observations}.
	
	\item If $\nu$ contains $N$ but not $N-2$ and $N-1$, then $\rank \mat{\W}{\nu} = 2$ since:
	\[
	\mat{\W}{\nu}_{:2, :4} = \begin{bmatrix}
		1 & 1 & 0 & 0 \\
		0 & 0 & 1 & 1
	\end{bmatrix} 
	\text{\,,}
	\]
	and all other entries of $\mat{\W}{\nu}$ hold zero.
	In this case $\rank \mat{\W''}{\nu}  = R''_\nu < 2$.
	However, the fact that $\W''$ fits the observed entries from \eq~\eqref{htf:eq:ht_multiple_minima_observations} leads to a contradiction, as $\mat{\W''}{\nu}$ must contain at least two linearly independent columns.
	To see it is so, notice that:
	\[
	\mat{\W^*}{\nu}_{:2, :4} = \begin{bmatrix}
		1 & ? & ? & 0 \\
		0 & ? & ? & 1
	\end{bmatrix} 
	\text{\,,}
	\]
	where recall that $?$ stands for an unobserved entry.
	
	\item If $\nu$ contains $N - 1$ but not $N - 2$ and $N$, then $\rank \mat{\W}{\nu} = 1$ since:
	\[
	\mat{\W}{\nu}_{:2, :4} = \begin{bmatrix}
		1 & 0 & 0 & 1 \\
		1 & 0 & 0 & 1
	\end{bmatrix} 
	\text{\,,}
	\]
	and all other entries of $\mat{\W}{\nu}$ hold zero.
	In this case $\rank \mat{\W''}{\nu}  = R''_\nu = 0$, which means that $\W''$ is the zero tensor, in contradiction to it fitting the (non-zero) observed entries from \eq~\eqref{htf:eq:ht_multiple_minima_observations}.
	
	\item If $\nu$ contains $N - 2$ but not $N - 1$ and $N$, then $\rank \mat{\W}{\nu} = 2$ since:
	\[
	\mat{\W}{\nu}_{:2, :4} = \begin{bmatrix}
		1 & 0 & 1 & 0 \\
		0 & 1 & 0 & 1
	\end{bmatrix} 
	\text{\,,} 
	\]
	and all other entries of $\mat{\W}{\nu}$ hold zero.
	In this case $\rank \mat{\W''}{\nu}  = R''_\nu  < 2$.
	Noticing that $\mat{\W''}{\{N - 2\}}_{:2, :4} = \mat{\W''}{\nu}_{:2, :4}$, and that entries of $\mat{\W''}{\{N - 2\}}$ outside its top $2$-by-$4$ submatrix hold zero, we get that $\rank \mat{\W''}{\{N - 2\}} = \rank \mat{\W''}{\nu} < 2$.
	Furthermore, from the assumption that 
	\[
	( R''_\nu )_{ \nu \in \htmodetree \setminus \{ [N] \} } \leq ( \rank \mat{ \W }{\nu} )_{\nu \in \htmodetree \setminus \{ [N] \}}
	\]
	and the previous three cases, we know that $R''_{\{n\}} = \rank \mat{\W}{\{n\}} = 1$ for all $n \in [N - 3]$, $R''_{\{N\}} = \rank \mat{\W}{\{N\}} = 2$, and $R''_{\{N - 1\}} = \rank \mat{\W}{\{N - 1\}} = 1$.
	Any tensor $\V \in \R^{D_1 \times \cdots \times D_N}$ that satisfies $\rank \mat{\V}{\{n\}} \leq R_{\{n\}} \in \N$ for all $n \in [N]$ can be represented as:
	\[
	\V = \sum\nolimits_{r_1 = 1}^{R_{\{1\}}} \cdots \sum\nolimits_{r_N = 1}^{R_{\{N\}}} \mathcal{C}_{r_1, \ldots, r_N} \tenp_{n = 1}^N \Ubf^{(n)}_{:, r_n}
	\text{\,,}
	\]
	where $\mathcal{C} \in \R^{R_{\{1\}} \times \cdots \times R_{\{N\}}}$ and $\big ( \Ubf^{(n)} \in \R^{D_n \times R_{\{n\}}} \big )_{n = 1}^N$ (see, \eg,~Section 4 in~\cite{kolda2009tensor}).
	Thus, there exist $c_1, c_2 \in \R$, $\big ( \Ubf^{(n)} \in \R^{D_n \times 1} \big )_{n = 1}^{N - 1}$, and $\Ubf^{(N)} \in \R^{D_N \times 2}$ such that:
	\[
	\W'' = c_1 \cdot \bigl ( \tenp_{n = 1}^{N - 1} \Ubf^{(n)}_{:, 1} \bigr ) \tenp \Ubf^{(N)}_{:, 1} + c_2 \cdot \bigl ( \tenp_{n = 1}^{N - 1} \Ubf^{(n)}_{:, 1} \bigr ) \tenp \Ubf^{(N)}_{:, 2} 
	\text{\,.}
	\]
	By multilinearity of the tensor product, we may write: $\W'' = \bigl ( \tenp_{n = 1}^{N - 1} \Ubf^{(n)}_{:, 1} \bigr ) \tenp \big ( c_1 \cdot \Ubf^{(N)}_{:, 1} + c_2 \cdot \Ubf^{(N)}_{:, 2} \big )$, and so $\W''$ has tensor rank one (it can be represented as a single non-zero tensor product between vectors).
	Since the tensor rank of a given tensor upper bounds the ranks of its matricizations (Remark 6.21 in~\cite{hackbusch2012tensor}), $R''_{\{ n \}} = \rank \mat{ \W'' }{\{ n \}} = 1$ for all $n \in [N]$ (the matrix ranks of these matricizations cannot be zero as $\W''$ is not the zero tensor).
	Hence, we have arrived at a contradiction~---~$2 = R''_{ \{N\}} \leq 1$.
	
	\item Contradictions in the remaining cases, where $\nu$ contains $N - 2, N - 1,$ and $N$, or any two of them, readily follow from the previous cases due to the fact that $\mat{\V}{\nu} = \mat{\V}{[N] \setminus \nu}^\top$ for any tensor $\V \in \R^{D_1 \times \cdots \times D_N}$, and that the matrix rank of a matrix is equal to the matrix rank of its transpose.
	In particular, for any such $\nu$, it holds that $\rank \mat{ \W'' }{\nu} = \rank \mat{\W''}{[N] \setminus \nu}$ and $\rank \mat{\W}{\nu} = \rank \mat{\W}{[N] \setminus \nu }$.
	Therefore, if $\rank \mat{\W''}{\nu} = R''_\nu < \rank \mat{\W}{\nu}$, then $\rank \mat{\W''}{[N] \setminus \nu} < \rank \mat{\W}{[N] \setminus \nu}$.
	Since $\nu$ contains $N - 2, N - 1,$ and $N$, or any two of them, its complement $[N] \setminus \nu$ contains none or just one of them.
	Each of these scenarios was already covered in previous cases, which imply that $\rank \mat{\W''}{[N] \setminus \nu} < \rank \mat{\W}{[N] \setminus \nu}$ entails a contradiction.
\end{itemize}

In all cases, we have established that the existence of $( R''_\nu )_{ \nu \in \htmodetree \setminus \{ [N] \} } \in \RR_\htmodetree$, different from $( \rank \mat{\W}{\nu} )_{\nu \in \htmodetree \setminus \{ [N] \}}$ and $( \rank \mat{\W'}{\nu} )_{\nu \in \htmodetree \setminus \{ [N] \}}$, satisfying $( R''_\nu )_{ \nu \in \htmodetree \setminus \{ [N] \} } \leq ( \rank \mat{ \W }{\nu} )_{\nu \in \htmodetree \setminus \{ [N] \}}$ leads to a contradiction.
The claim for $\W'$, \ie~that there exists no such $( R''_\nu )_{ \nu \in \htmodetree \setminus \{ [N] \} }$ satisfying $( R''_\nu )_{ \nu \in \htmodetree \setminus \{ [N] \} } \leq ( \rank \mat{ \W' }{\nu} )_{\nu \in \htmodetree \setminus \{ [N] \}}$, is proven analogously.
Combined with the previous part of the proof, in which we established that neither $( \rank \mat{ \W }{\nu} )_{\nu \in \htmodetree \setminus \{ [N] \}}$ nor $( \rank \mat{ \W' }{\nu} )_{\nu \in \htmodetree \setminus \{ [N] \}}$ is smaller than or equal to the other, we conclude that $( \rank \mat{ \W }{\nu} )_{\nu \in \htmodetree \setminus \{ [N] \}}$ and $( \rank \mat{ \W' }{\nu} )_{\nu \in \htmodetree \setminus \{ [N] \}}$ are two different minimal elements of $\RR_\htmodetree$.
\qed

\subsection{Proof of Proposition~\ref{htf:prop:htf_cnn}}
\label{htf:app:proofs:htf_cnn}
For $l \in [L]$, the output of the $l$'th convolutional layer at index $n \in [N / P^{l - 1}]$ and channel $r \in [R_{l}]$ depends solely on inputs $\xbf^{( (n - 1) \cdot P^{l - 1} + 1 )}, \ldots ,\xbf^{(n \cdot P^{l - 1})}$.
Hence, we denote it by $conv_{l, n, r} \big ( \xbf^{( (n - 1) \cdot P^{l - 1} + 1 )}, \ldots ,\xbf^{(n \cdot P^{l - 1})} \big )$.
We may view the output linear layer as a $1 \times 1$ convolutional layer with a single output channel.
Accordingly, let $conv_{L + 1, 1, 1} \big ( \xbf^{(1)}, \ldots, \xbf^{(N)} \big ) := f_\theta \big ( \xbf^{(1)}, \ldots, \xbf^{(N)} \big )$ and $\htftensorpart{L +1, 1}{1} := \tensorend$.

We show by induction over the layer $l \in [L + 1]$ that for any $n \in [N / P^{l - 1}]$ and $r \in [R_{l}]$:
\be
conv_{l, n, r} \big ( \xbf^{( (n - 1) \cdot P^{l - 1} + 1 )}, \ldots ,\xbf^{(n \cdot P^{l - 1})} \big ) = \inprod{ \tenp_{p = (n -1 ) \cdot P^{l - 1} + 1}^{n \cdot P^{l - 1}} \xbf^{(p)} }{ \htftensorpart{l, n}{r} }
\text{\,.}
\label{htf:eq:htf_cnn_inductive_claim}
\ee
For $l = 1$, let $n \in [N]$ and $r \in [R_1]$.
From the definition of $\htftensorpart{1,n}{r}$ (\eq~\eqref{htf:eq:ht_pary_end_tensor}) we can see that:
\[
conv_{1, n, r} \big ( \xbf^{(n)} \big ) = \inprod{\xbf^{ (n)} }{ \weightmat{1, n}_{:, r} } = \inprod{ \xbf^{(n)} }{ \htftensorpart{1, n}{r} }
\text{\,.}
\]
Now, assuming that the inductive claim holds for $l - 1 \geq 1$, we prove that it holds for $l$.
Fix some $n \in [N / P^{l - 1}]$ and $r \in [R_{l}]$.
The $l$'th convolutional layer is applied to the output of the $l - 1$'th hidden layer, denoted $\big ( \hbf^{(l - 1, 1)}, \ldots, \hbf^{(l - 1, N / P^{l - 1})} \big ) \in \R^{R_{l - 1}} \times \cdots \times \R^{R_{l - 1}}$.
Each $\hbf^{(l-1, n)}$, for $n \in [N / P^{l - 1}]$, is a result of the product pooling operation (with window size $P$) applied to the output of the $l-1$'th convolutional layer. 
Thus:
\[
\begin{split}
	& conv_{l, n, r} \big ( \xbf^{( (n - 1) \cdot P^{l - 1} + 1 )}, \ldots ,\xbf^{(n \cdot P^{l - 1})} \big ) \\
	& \hspace{4.5mm} = \sum_{r' = 1}^{R_{l - 1}} \weightmat{l, n}_{r', r} \cdot \hbf^{(l - 1, n)}_{r'} \\
	& \hspace{4.5mm} = \sum_{r' = 1}^{R_{l - 1}} \weightmat{l, n}_{r', r}  \cdot \,\,  \prod_{\mathclap{p = (n - 1) \cdot P + 1}}^{n \cdot P} \,\,\, conv_{l - 1, p. r'} \big ( \xbf^{ ((p - 1) \cdot P^{l - 2} + 1) }, \ldots, \xbf^{ (p \cdot P^{l - 2}) } \big )
	\text{\,.}
\end{split}
\]
The inductive assumption for $l - 1$ then implies that:
\[
\begin{split}
& conv_{l, n, r} \big ( \xbf^{( (n - 1) \cdot P^{l - 1} + 1 )}, \ldots ,\xbf^{(n \cdot P^{l - 1})} \big ) \\
& \hspace{4.5mm} = \sum_{r' = 1}^{R_{l - 1}} \weightmat{l, n}_{r', r} \cdot \,\, \prod_{\mathclap{p = (n - 1) \cdot P + 1}}^{n \cdot P} \,\,\, \inprod{ \tenp_{n' = (p - 1) \cdot P^{l - 2} + 1}^{p \cdot P^{l - 2}} \xbf^{ (n') } }{ \htftensorpart{l - 1, p }{r'} }
\text{\,.}
\end{split}
\]
For any tensors $\A, \A', \B, \B'$ such that $\A$ is of the same dimensions as $\A'$ and $\B$ is of the same dimensions as $\B'$, it holds that $\inprod{ \A \tenp \B}{ \A' \tenp \B' } = \inprod{ \A }{ \A' } \cdot \inprod{ \B }{ \B' }$.
We may therefore write:
\[
\begin{split}
	& conv_{l, n, r} \big ( \xbf^{( (n - 1) \cdot P^{l - 1} + 1 )}, \ldots ,\xbf^{(n \cdot P^{l - 1})} \big ) \\[1mm]
	&\quad\quad = \sum\nolimits_{r' = 1}^{R_{l - 1}} \weightmat{l, n}_{r', r} \cdot \inprod{ \tenp_{p = (n - 1) \cdot P + 1}^{n \cdot P} \left ( \tenp_{n' = (p - 1) \cdot P^{l - 2} + 1}^{p \cdot P^{l - 2}} \xbf^{ (n') } \right ) }{ \tenp_{p = (n - 1) \cdot P + 1}^{n \cdot P} \htftensorpart{l - 1, p }{r'} } \\[1mm]
	&\quad\quad = \sum\nolimits_{r' = 1}^{R_{l - 1}} \weightmat{l, n}_{r', r} \cdot \inprod{ \tenp_{p = (n - 1) \cdot P^{l - 1} + 1}^{n \cdot P^{l - 1}} \xbf^{ (p) } }{ \tenp_{p = (n - 1) \cdot P + 1}^{n \cdot P} \htftensorpart{l - 1, p }{r'} } \\[1mm]\
	&\quad\quad = \inprod{ \tenp_{p = (n - 1) \cdot P^{l - 1} + 1}^{n \cdot P^{l - 1}} \xbf^{ (p) } }{ \sum\nolimits_{r' = 1}^{R_{l - 1}} \weightmat{l, n}_{r', r}  \left [ \tenp_{p = (n - 1) \cdot P + 1}^{n \cdot P} \htftensorpart{l - 1, p }{r'} \right ] }
	\text{\,.}
\end{split}
\]
Noticing that $\htftensorpart{l,n}{r} = \sum\nolimits_{r' = 1}^{R_{l - 1}} \weightmat{l, n}_{r', r} \big [ \tenp_{p = (n - 1) \cdot P + 1}^{n \cdot P} \htftensorpart{l - 1, p }{r'} \big ]$ (\eq~\eqref{htf:eq:ht_pary_end_tensor}) establishes \eq~\eqref{htf:eq:htf_cnn_inductive_claim}.

Applying the inductive claim for $l = L + 1, n = 1,$ and $r = 1$, while recalling that $L = \log_P N$, yields:
\[
\begin{split}
f_\theta \big ( \xbf^{(1)}, \ldots, \xbf^{(N)} \big ) & = conv_{L + 1, 1, 1} \big ( \xbf^{(1)}, \ldots, \xbf^{(N)} \big ) \\
& = \inprod{ \tenp_{n = 1}^N \xbf^{(n)} }{ \htftensorpart{L+1, 1}{1} } \\
& = \inprod{ \tenp_{n = 1}^N \xbf^{(n)} }{ \tensorend }
\text{\,.}
\end{split}
\]
\qed

\subsection{Proof of Theorem~\ref{htf:thm:loc_comp_norm_unbal_dyn}}
\label{htf:app:proofs:loc_comp_norm_unbal_dyn}
Let $t \geq 0$ be a time at which $\htfcompnorm{\nu}{r} (t) := \normnoflex{\tenp_{\wbf \in \localcomp (\nu, r)} \wbf (t) } = \prod\nolimits_{ \wbf \in \localcomp (\nu, r) } \norm{ \wbf (t) } > 0$.
We differentiate $\htfcompnorm{\nu}{r} (t)^2$ with respect to time as done in the proof of \thm~\ref{htf:thm:loc_comp_norm_bal_dyn} (\subapp~\ref{htf:app:proofs:loc_comp_norm_bal_dyn}).
From the product rule and \lem~\ref{htf:lem:weightvec_sq_norm_time_deriv} we get that:
\[
\frac{d}{dt} \htfcompnorm{\nu}{r} (t)^2 = 2 \htfcompnorm{\nu}{r} (t) \inprod{ - \nabla \htfendloss ( \tensorend (t)) }{ \htfcomp{\nu}{r} (t) } \sum\nolimits_{ \wbf \in \localcomp (\nu, r) } \prod\nolimits_{ \wbf' \in \localcomp (\nu, r) \setminus \{ \wbf \} } \norm{ \wbf' (t) }^2
\text{\,.}
\]
Since according to the chain rule $\frac{d}{dt} \htfcompnorm{\nu}{r} (t)^2 = 2 \htfcompnorm{\nu}{r} (t) \cdot \frac{d}{dt} \htfcompnorm{\nu}{r} (t)$, the equation above leads to:
\be
\frac{d}{dt} \htfcompnorm{\nu}{r} (t) = \inprod{ - \nabla \htfendloss ( \tensorend (t)) }{ \htfcomp{\nu}{r} (t) } \sum\nolimits_{ \wbf \in \localcomp (\nu, r) } \prod\nolimits_{ \wbf' \in \localcomp (\nu, r) \setminus \{ \wbf \} } \norm{ \wbf' (t) }^2
\text{\,.}
\label{htf:eq:localcompnorm_time_deriv_interm_unbal}
\ee
By \lem~\ref{htf:lem:loc_comp_sq_norm_diff_invariant}, the unbalancedness magnitude is constant through time, and so it remains equal to $\epsilon$~---~its value at initialization.
Hence, for any $\wbf \in \localcomp (\nu , r)$:
\be
\norm{ \wbf (t) }^2 \leq \min_{\wbf' \in \localcomp (\nu, r) } \norm{ \wbf' (t) }^2 + \epsilon = \Big ( \min_{\wbf' \in \localcomp (\nu, r) } \norm{ \wbf' (t) } \Big )^{L_\nu \cdot \frac{2}{L_\nu} } + \epsilon \leq \htfcompnorm{\nu}{r} (t)^{ \frac{2}{ L_\nu } } + \epsilon
\text{\,.}
\label{htf:eq:sq_weightvec_norm_up_local_comp_norm}
\ee

If $\inprodbig{- \nabla \htfendloss ( \tensorend (t) ) }{ \htfcomp{\nu}{r} (t) } \geq 0$,
applying the inequality above to each $\norm{ \wbf' (t) }^2$ in \eq~\eqref{htf:eq:localcompnorm_time_deriv_interm_unbal} yields the upper bound from \eq~\eqref{htf:eq:loc_comp_norm_unbal_pos_bound}:
\[
\begin{split}
	\frac{d}{dt} \htfcompnorm{\nu}{r} (t) & \leq \inprod{ - \nabla \htfendloss ( \tensorend (t)) }{ \htfcomp{\nu}{r} (t) } \sum\nolimits_{ \wbf \in \localcomp (\nu, r) } \prod\nolimits_{ \wbf' \in \localcomp (\nu, r) \setminus \{ \wbf \} } \left ( \htfcompnorm{\nu}{r} (t)^{ \frac{2}{ L_\nu } } + \epsilon \right ) \\
	& =  \left ( \htfcompnorm{\nu}{r} (t)^{\frac{2}{ L_\nu } } + \epsilon \right )^{ L_\nu - 1 } \cdot L_\nu \inprodbig{- \nabla \htfendloss ( \tensorend (t) ) }{ \htfcomp{\nu}{r} (t) }
	\text{\,.}
\end{split}
\]
To prove the lower bound from \eq~\eqref{htf:eq:loc_comp_norm_unbal_pos_bound}, we multiply and divide each summand on the right hand side of \eq~\eqref{htf:eq:localcompnorm_time_deriv_interm_unbal} by the corresponding $\norm{\wbf (t) }^2$ (non-zero because $\htfcompnorm{\nu}{r} (t) > 0$), \ie:
\[
\begin{split}
	\frac{d}{dt} \htfcompnorm{\nu}{r} (t)& = \inprod{ - \nabla \htfendloss ( \tensorend (t)) }{ \htfcomp{\nu}{r} (t) } \sum\nolimits_{ \wbf \in \localcomp (\nu, r) } \norm{ \wbf (t) }^{-2} \cdot \prod\nolimits_{ \wbf' \in \localcomp (\nu, r) } \norm{ \wbf' (t) }^2 \\
	& = \inprod{ - \nabla \htfendloss ( \tensorend (t)) }{ \htfcomp{\nu}{r} (t) } \htfcompnorm{\nu}{r} (t) \cdot \sum\nolimits_{ \wbf \in \localcomp (\nu, r) } \norm{ \wbf (t) }^{-2}
	\text{\,.}
\end{split}
\]
By \eq~\eqref{htf:eq:sq_weightvec_norm_up_local_comp_norm} we know that $\norm{ \wbf (t) }^{-2} \geq \big ( \htfcompnorm{\nu}{r} (t)^{ \frac{2}{ L_\nu} } + \epsilon \big )^{-1}$.
Thus, applying this inequality to the equation above establishes the desired lower bound.

If $\inprodbig{- \nabla \htfendloss ( \tensorend (t) ) }{ \htfcomp{\nu}{r} (t) } <  0$, the upper and lower bounds in \eq~\eqref{htf:eq:loc_comp_norm_unbal_neg_bound} readily follow by similar derivations, where the difference in the direction of inequalities is due to the negativity of $\inprodbig{- \nabla \htfendloss ( \tensorend (t) ) }{ \htfcomp{\nu}{r} (t) }$.
\qed

\subsection{Proof of Proposition~\ref{htf:prop:matrank_eq_seprank}}
\label{htf:app:proofs:matrank_eq_seprank}
We partition the proof into two parts: the first shows that $\rank \mat{\tensorend}{I} \geq \htfseprank (f_\Theta ; I)$, and the second establishes the converse.

\textbf{Proof of lower bound ($\rank \mat{\tensorend}{I} \geq \htfseprank (f_\Theta ; I)$).}~~Denote $R := \rank \mat{ \tensorend}{I}$, and assume without loss of generality that $I = [\abs{I}]$.
Since $\mat {\tensorend}{I}$ is a rank $R$ matrix, there exist $\vbf^{(1)}, \ldots, \vbf^{(R)} \in \R^{\prod_{n = 1}^{\abs{I}} D_n}$ and $\bar{\vbf}^{(1)}, \ldots, \bar{\vbf}^{(R)} \in \R^{ \prod_{n = \abs{I} + 1 }^{ N } D_n }$ such that:
\[
\mat{ \tensorend }{I} = \sum\nolimits_{r = 1}^R \vbf^{(r)} \big ( \bar{\vbf}^{(r)}\big )^\top
\text{\,.}
\]
For each $r \in [R]$, let $\V^{(r)} \in \R^{D_1 \times \cdots \times D_{ \abs{I} } }$ be the tensor whose arrangement as a column vector is equal to $\vbf^{(r)}$, \ie~$\matbig{ \V^{(r)} }{I} = \vbf^{(r)}$.
Similarly, for every $r \in [R]$ let $\widebar{\V}^{(r)} \in \R^{D_{ \abs{I} + 1} \times \cdots \times D_N }$ be the tensor whose arrangement as a row vector is equal to $( \bar{\vbf}^{(r)} )^\top$, \ie~$\matbig{ \widebar{\V}^{(r)} }{\emptyset} = ( \bar{\vbf}^{(r)} )^\top$.
Then:
\[
\begin{split}
	\mat{ \tensorend }{I} &= \sum\nolimits_{r = 1}^R \vbf^{(r)}  \big ( \bar{\vbf}^{(r)} \big )^\top \\
	& = \sum\nolimits_{r = 1}^R \mat{ \V^{(r)} }{I} \kronp \mat{ \widebar{\V}^{(r)} }{\emptyset} \\
	& = \sum\nolimits_{r = 1}^R \mat { \V^{(r)} \tenp \widebar{\V}^{(r)} }{I} \\
	& = \mat{ \sum\nolimits_{r = 1}^R \V^{(r)} \tenp \widebar{\V}^{(r)} }{I}
	\text{\,,}
\end{split}
\]
where the third equality makes use of \lem~\ref{htf:lem:tenp_eq_kronp}, and the last equality is by linearity of the matricization operator.
Since matricizations merely reorder the entries of tensors, the equation above implies that $\tensorend = \sum_{r = 1}^R \V^{(r)} \tenp \widebar{\V}^{(r)}$.
We therefore have that:
\[
\begin{split}
	f_\Theta \big ( \xbf^{(1)}, \ldots, \xbf^{(N)} \big ) & = \inprod{ \tenp_{n = 1}^N \xbf^{(n)}  }{ \tensorend } \\
	& = \inprodBig{ \tenp_{n = 1}^N \xbf^{(n)}  }{ \sum\nolimits_{r = 1}^R \V^{(r)} \tenp \widebar{\V}^{(r)} } \\
	& = \sum\nolimits_{r = 1}^R \inprod{ \tenp_{n = 1}^N \xbf^{(n)}  }{ \V^{(r)} \tenp \widebar{\V}^{(r)} }
	\text{\,.}
\end{split}
\]
For any $\A, \A' \in \R^{D_1 \times \cdots \times D_{ \abs{I} } }$ and $\B, \B' \in \R^{ D_{ \abs{I} + 1 } \times \cdots \times D_N }$ it holds that $\inprod{ \A \tenp \B}{ \A' \tenp \B' } = \inprod{ \A }{ \A' } \cdot \inprod{ \B }{ \B' }$.
Thus:
\[
\begin{split}
f_\Theta \big ( \xbf^{(1)}, \ldots, \xbf^{(N)} \big ) & = \sum\nolimits_{r = 1}^R \inprod{ \tenp_{n = 1}^N \xbf^{(n)}  }{ \V^{(r)} \tenp \widebar{\V}^{(r)} } \\
& = \sum\nolimits_{r = 1}^R \inprod{ \tenp_{n = 1}^{ \abs{I} } \xbf^{ (n) } }{ \V^{(r)} } \cdot \inprod{ \tenp_{n = \abs{I}  + 1}^{ N } \xbf^{ (n) } }{ \widebar{\V}^{(r)} }
\text{\,.}
\end{split}
\]
By defining $g_r : \times_{n =1}^{ \abs{I} } \R^{D_n} \to \R$ and $\bar{g}_r :\times_{n = \abs{I} + 1}^{N} \R^{D_n} \to \R$, for $r \in [R]$, as:
\[
g_r \big ( \xbf^{(1)}, \ldots, \xbf^{(\abs{I} )} \big ) = \inprod{ \tenp_{n = 1}^{ \abs{I} } \xbf^{ (n) } }{ \V^{(r)} } ~~ , ~~ \bar{g}_r \big ( \xbf^{ ( \abs{I} + 1 ) }, \ldots, \xbf^{(N)} \big ) = \inprod{ \tenp_{n = \abs{I}  + 1}^{ N } \xbf^{ (n) } }{ \widebar{\V}^{(r)} }
\text{\,,}
\]
we arrive at the following representation of $f_\Theta$ as a sum, where each summand is a product of two functions --- one that operates over inputs indexed by $I$ and another that operates over inputs indexed by $[N] \setminus I$:
\[
f_\Theta \big ( \xbf^{(1)}, \ldots, \xbf^{(N)} \big ) = \sum\nolimits_{r = 1}^R g_r \big ( \xbf^{(1)}, \ldots, \xbf^{ ( \abs{I} ) } \big ) \cdot \bar{g}_r \big ( \xbf^{( \abs{I} + 1 )}, \ldots, \xbf^{ ( N ) } \big )
\text{\,.}
\]
Since the separation rank of $f_\Theta$ is the minimal number of summands required to express it in such a manner, we conclude that $\rank \mat{\tensorend}{I} = R \geq \htfseprank (f_\Theta ; I)$.

\textbf{Proof of upper bound ($\rank \mat{\tensorend}{I} \leq \htfseprank (f_\Theta ; I)$).}~~Towards proving the upper bound, we establish the following lemma.
\begin{lemma}
	\label{htf:lem:grid_tensor_mat_rank_ub_by_sep_rank}
	Given $f : \times_{n = 1}^N \R^{D_n} \to \R$ and any $\big ( \xbf^{( 1, h_1 )} \in \R^{D_1} \big )_{h_1 = 1}^{H_1}, \ldots, \big ( \xbf^{ (N, h_N) } \in \R^{D_N} \big )_{h_N = 1}^{H_N}$, let $\W \in \R^{H_1 \times \cdots \times H_N}$ be the tensor defined by $\W_{h_1, \ldots, h_N} := f \big ( \xbf^{ (1, h_1) }, \ldots, \xbf^{ (N, h_N) }\big )$ for all $(h_1, \ldots, h_N) \in [H_1] \times \cdots \times [H_N]$.
	Then, for any $I \subset [N]$:
	\[
	\rank \mat{\W}{I} \leq \htfseprank (f ; I)
	\text{\,.}
	\]
	In words, for any tensor holding the outputs of $f$ over a grid of inputs, the rank of its matricization according to $I$ is upper bounded by the separation rank of $f$ with respect to $I$.
\end{lemma}
\begin{proof}
	If $\htfseprank (f ; I)$ is $\infty$ or zero, \ie~$f$ cannot be represented as a finite sum of separable functions (with respect to $I$) or is identically zero, then the claim is trivial. 
	Otherwise, denote $R := \htfseprank (f; I)$, and assume without loss of generality that $I = [\abs{I}]$.
	Let $g_1, \ldots, g_R : \times_{n =1}^{ \abs{I} } \R^{D_n} \to \R$ and $\bar{g}_1, \ldots, \bar{g}_R :\times_{n = \abs{I} + 1}^{N} \R^{D_n} \to \R$ such that:
	\be
	f \big ( \xbf^{(1)}, \ldots, \xbf^{(N)} \big ) = \sum\nolimits_{r = 1}^R g_r \big ( \xbf^{(1)}, \ldots, \xbf^{ ( \abs{I} ) } \big ) \cdot \bar{g}_r \big ( \xbf^{( \abs{I} + 1 )}, \ldots, \xbf^{ ( N ) } \big )
	\text{\,.}
	\label{htf:eq:grid_tensor_mat_ub_by_sep_rank:sep_rank}
	\ee
	We define $\big ( \V^{(r)} \in \R^{D_1 \times \cdots \times D_{\abs{I}} } \big )_{r = 1}^R$ to be the tensors holding the outputs of $( g_r )_{r = 1}^R$ over the grid of inputs
	\[
	\big ( \xbf^{( 1, h_1 )} \big )_{h_1 = 1}^{H_1}, \ldots, \big ( \xbf^{ (\abs{I} , h_{ \abs{I} } ) } \big )_{h_{ \abs{I} } = 1}^{ H_{ \abs{I} } }
	\text{\,,}
	\]
	\ie~for all $h_1, \ldots, h_{\abs{I}} \in [H_1] \times \cdots \times [ H_{ \abs{I} } ]$ and $r \in [R]$ it holds that $\V^{(r)}_{ h_1, \ldots, h_{\abs{I}} } = g_r \big ( \xbf^{(1, h_1)}, \ldots, \xbf^{(\abs{I}, h_{ \abs{I} } )} \big )$.
	Similarly, we let $\big ( \widebar{\V}^{(r)} \in \R^{D_{\abs{I} + 1} \times \cdots \times D_{N} } \big )_{r = 1}^R$ be the tensors holding the outputs of $( \bar{g}_r )_{r = 1}^R$ over their respective grid of inputs, \ie~for all $h_{\abs{I} + 1}, \ldots, h_N \in [H_{ \abs{I} + 1}] \times \cdots \times [H_N]$ and $r \in [R]$ it holds that $\widebar{\V}^{(r)}_{ h_{\abs{I} + 1}, \ldots, h_N  } = \bar{g}_r \big ( \xbf^{ ( \abs{I} + 1, h_{\abs{I} + 1} )}, \ldots, \xbf^{ ( N, h_N ) } \big )$.
	
	By \eq~\eqref{htf:eq:grid_tensor_mat_ub_by_sep_rank:sep_rank} and the definitions of $\W, ( \V^{(r)} )_{r = 1}^R$, and $( \widebar{\V}^{(r)} )_{r = 1}^R$, we have that for any $h_1, \ldots, h_N \in [H_1] \times \cdots \times [H_N]$:
	\[
	\begin{split}
		\W_{h_1 ,\ldots, h_N} & = f \big ( \xbf^{ (1, h_1) }, \ldots, \xbf^{ (N, h_N) }\big ) \\
		& = \sum\nolimits_{r = 1}^R g_r \big ( \xbf^{(1, h_1)}, \ldots, \xbf^{ ( \abs{I}, h_{ \abs{I} } ) } \big ) \cdot \bar{g}_r \big ( \xbf^{( \abs{I} + 1, h_{\abs{I} + 1} )}, \ldots, \xbf^{ ( N, h_N ) } \big )\\
		& =  \sum\nolimits_{r = 1}^R \V^{(r)}_{h_1, \ldots, h_{ \abs{I} }} \cdot \widebar{\V}^{(r)}_{ h_{ \abs{I} +  1}, \ldots, h_N }
		\text{\,,}
	\end{split}
	\]
	which means that $\W = \sum_{r = 1}^R \V^{(r)} \tenp \widebar{\V}^{(r)}$.
	From the linearity of the matricization operator and \lem~\ref{htf:lem:tenp_eq_kronp} we then get that $\mat{\W}{I} = \sum_{ r = 1}^R \mat{\V^{(r)}}{I} \kronp \mat{ \widebar{\V}^{(r)} }{\emptyset}$.
	Since $\mat{\V^{(r)}}{I}$ is a column vector and $\mat{ \widebar{\V}^{(r)} }{\emptyset}$ is a row vector for all $r \in [R]$, we have arrived at a representation of $\mat{\W}{I}$ as a sum of $R$ tensor products between two vectors.
	A tensor product of two vectors is a rank one matrix, and so, due to the sub-additivity of rank we conclude: $\rank \mat{\W}{I} \leq R = \htfseprank (f; I)$.
\end{proof}

\medskip 

Now, consider the grid of inputs defined by the standard bases of $\R^{D_1}, \ldots, R^{D_N}$, \ie~by:
\[
\big ( \ebf^{( 1, d_1 )} \in \R^{D_1} \big )_{d_1 = 1}^{D_1}, \ldots, \big ( \ebf^{ (N, d_N) } \in \R^{D_N} \big )_{d_N = 1}^{D_N} 
\text{\,,}
\]
where $\ebf^{(n, d_n)}$ is the vector holding one at its $d_n$'th entry and zero elsewhere for $n \in [N]$ and $d_n \in [D_n]$.
With Lemma~\ref{htf:lem:grid_tensor_mat_rank_ub_by_sep_rank} in hand, $\rank \mat{\tensorend}{I} \leq \htfseprank (f_\Theta ; I)$ follows by showing that $\tensorend$ is the tensor holding the outputs of $f_\Theta$ over this grid of inputs.
Indeed, for all $d_1, \ldots, d_N \in [D_1] \times \cdots \times [D_N]$:
\[
f_\Theta \big ( \ebf^{(1, d_1)}, \ldots, \ebf^{(N, d_N)} \big ) = \inprodbig{ \tenp_{n = 1}^N \ebf^{(n, d_n)}  }{ \tensorend } = (\tensorend)_{d_1, \ldots, d_N}
\text{\,.}
\]
\qed

%% file: Appendices/gnn_interactions.tex
\chapter{On the Ability of Graph Neural Networks to Model Interactions \\ Between Vertices} 
\label{app:gnn_interactions}

\section{Tightness of Upper Bounds for Separation Rank}
\label{gnn:app:sep_rank_examples}

\cref{gnn:thm:sep_rank_upper_bound} upper bounds the separation rank with respect to $\I \subseteq \vertices$ of a depth $L$ GNN with product aggregation.
According to it, under the setting of graph prediction, the separation rank is largely capped by the $(L-1)$-walk index of $\I$, \ie~the number of length $L - 1$ walks from $\cut_\I$~---~the set of vertices with an edge crossing the partition $(\I, \I^c)$.
Similarly, for prediction over $t \in \vertices$, separation rank is largely capped by the $(L - 1, t)$-walk index of $\I$, which takes into account only length $L - 1$ walks from $\cut_\I$ ending at~$t$.
\cref{gnn:thm:sep_rank_lower_bound} provides matching lower bounds, up to logarithmic terms and to the number of walks from $\cut_\I$ being replaced with the number of walks from any single admissible subset $\cut \in \cutset (\I)$ (\cref{gnn:def:admissible_subsets}).
Hence, the match between the upper and lower bounds is determined by the portion of $\cut_\I$ that can be covered by an admissible subset.

In this appendix, to shed light on the tightness of the upper bounds, we present several concrete examples on which a significant portion of $\cut_\I$ can be covered by an admissible subset.

\textbf{Complete graph.}~~Suppose that every two vertices are connected by an edge, \ie~$\edges = \brk[c]{ \{i, j\} : i,j \in \vertices }$.
For any non-empty $\I \subsetneq \vertices$, clearly $\cut_\I = \neigh (\I) \cap \neigh (\I^c) = \vertices$ .
In this case, $\cut_\I = \vertices \in \cutset (\I)$, meaning $\cut_\I$ is an admissible subset of itself.
To see it is so, notice that for any $i \in \I, j \in \I^c$, all vertices are neighbors of both $\I' := \{ i \}$ and $\J' := \{ j\}$, which trivially have no repeating shared neighbors (\cref{gnn:def:no_rep_neighbors}).
Thus, up to a logarithmic factor, the upper and lower bounds from~\cref{gnn:thm:sep_rank_upper_bound,gnn:thm:sep_rank_lower_bound} coincide.

\textbf{Chain graph.}~~Suppose that $\edges = \brk[c]{ \{ i, i + 1\} : i \in [\abs{\vertices} - 1] } \cup \{ \{ i, i \} : i \in \vertices \}$.
For any non-empty $\I \subsetneq \vertices$, at least half of the vertices in $\cut_\I$ can be covered by an admissible subset.
That is, there exists $\cut \in \cutset (\I)$ satisfying $\abs{ \cut } \geq 2^{-1} \cdot \abs{ \cut_\I}$.
For example, such $\cut$ can be constructed algorithmically as follows.
Let $\I', \J' = \emptyset$.
Starting from $k = 1$, if $\{k, k + 1\} \subseteq \cut_\I$ and one of $\{ k , k+ 1\}$ is in $\I$ while the other is in $\I^c$, then assign $\I' \leftarrow \I' \cup \brk{ \{k, k + 1\} \cap \I }$, $\J' \leftarrow \J' \cup \brk{ \{k, k + 1\} \cap \I^c }$, and $k \leftarrow k + 3$.
That is, add each of $\{ k, k + 1\}$ to either $\I'$ if it is in $\I$ or $\J'$ if it is in $\I^c$, and skip vertex $k + 2$.
Otherwise, set $k \leftarrow k + 1$.
The process terminates once $k > \abs{\vertices} - 1$.
By construction, $\I' \subseteq \I$ and $\J' \subseteq \I^c$, implying that $\neigh (\I') \cap \neigh (\J') \subseteq \cut_\I$.
Due to the chain graph structure, $\I' \cup \J' \subseteq \neigh (\I') \cap \neigh (\J')$ and $\I'$ and $\J'$ have no repeating shared neighbors (\cref{gnn:def:no_rep_neighbors}).
Furthermore, for every pair of vertices from $\cut_\I$ added to $\I'$ and $\J'$, we can miss at most two other vertices from $\cut_\I$.
Thus, $\cut := \neigh (\I') \cap \neigh (\J')$ is an admissible subset of $\cut_\I$ satisfying $\abs{\cut} \geq 2^{-1} \cdot \abs{\cut_\I}$.

\textbf{General graph.}~~For an arbitrary graph and non-empty $\I \subsetneq \vertices$, an admissible subset of $\cut_\I$ can be obtained by taking any sequence of pairs $(i_1, j_1), \ldots, (i_M, j_M) \in \I \times \I^c$ with no shared neighbors, in the sense that $\brk[s]{ \neigh (i_m) \cup \neigh (j_m) } \cap \brk[s] { \neigh (i_{m'}) \cup \neigh (j_{m'}) } = \emptyset$ for all $m \neq m' \in [M]$.
Defining $\I' := \brk[c]{ i_1, \ldots, i_M }$ and $\J' := \brk[c]{ j_1, \ldots, j_M}$, by construction they do not have repeating shared neighbors (\cref{gnn:def:no_rep_neighbors}), and so $\neigh (\I') \cap \neigh (\J') \in \cutset (\I)$.
In particular, the shared neighbors of each pair are covered by $\neigh (\I') \cap \neigh (\J')$, \ie~$\cup_{m = 1}^M \neigh (i_m) \cap \neigh (j_m) \subseteq \neigh (\I') \cap \neigh (\J')$.

\section{Extension of Analysis to Directed Graphs With Multiple Edge Types}
\label{gnn:app:extensions}

In this appendix, we generalize the separation rank bounds from~\cref{gnn:thm:sep_rank_upper_bound,gnn:thm:sep_rank_lower_bound} to directed graphs with multiple edge types.

Let $\graph = (\vertices, \edges, \edgetypemap)$ be a directed graph with vertices $\vertices = [\abs{\vertices}]$, edges $\edges \subseteq \{ (i, j) : i, j \in \vertices\}$, and a map $\edgetypemap : \edges \to [Q]$ from edges to one of $Q \in \N$ edge types.
For $i \in \vertices$, let $\neighin (i) := \brk[c]{ j \in \vertices : (j, i) \in \edges }$ be its \emph{incoming neighbors} and $\neighout (i) := \brk[c]{ j \in \vertices : (i, j) \in \edges }$ be its \emph{outgoing neighbors}.
For $\I \subseteq \vertices$, we denote $\neighin (\I) := \cup_{i \in \I} \neighin (i)$ and $\neighout (\I) := \cup_{i \in \I} \neighout (i)$.
As customary in the context of GNNs, we assume the existence of all self-loops (\cf~\cref{gnn:sec:prelim:notation}).

Message-passing GNNs (\cref{gnn:sec:gnns}) operate identically over directed and undirected graphs, except that in directed graphs the hidden embedding of a vertex is updated only according to its incoming neighbors.
For handling multiple edge types, common practice is to use different weight matrices per type in the GNN's update rule (\cf~\cite{hamilton2017inductive,schlichtkrull2018modeling}).
Hence, we consider the following update rule for directed graphs with multiple edge types, replacing that from~\cref{gnn:eq:gnn_update}:
\be
\hidvec{l}{i} = \agg \brk2{ \multisetbig{ \weightmat{l, \edgetypemap ( j , i ) } \hidvec{l - 1}{j} : j \in \neighin (i) } }
\text{\,,}
\label{gnn:eq:gnn_update_directed}
\ee
where $\brk{ \weightmat{1, q} \in \R^{\hdim \times \indim}}_{ q \in [Q]}$ and $\brk{ \weightmat{l, q} \in \R^{\hdim \times \hdim} }_{l \in \{2, \ldots, L\} , q \in [Q]}$ are learnable weight matrices.

In our analysis for undirected graphs (\cref{gnn:sec:analysis:formal}), a central concept is $\cut_\I$~---~the set of vertices with an edge crossing the partition induced by $\I \subseteq \vertices$.
Due to the existence of self-loops it is equal to the shared neighbors of $\I$ and $\I^c$, \ie~$\cut_\I = \neigh (\I) \cap \neigh (\I^c)$.
We generalize this concept to directed graphs, defining $\cutdir_\I$ to be the set of vertices with an incoming edge from the other side of the partition induced by $\I$, \ie~$\cutdir_\I := \{ i \in \I : \neighin (i) \cap \I^c \neq \emptyset \} \cup \{ j \in \I^c : \neighin (j) \cap \I \neq \emptyset \}$.
Due to the existence of self-loops it is given by $\cutdir_\I = \neighout (\I) \cap \neighout( \I^c)$.
Indeed, for undirected graphs $\cutdir_\I = \cut_\I$.

With the definition of $\cutdir_\I$ in place,~\cref{gnn:thm:directed_sep_rank_upper_bound} upper bounds the separation ranks a GNN can achieve over directed graphs with multiple edge types.
A technical subtlety is that the bounds depend on walks of lengths $l = L - 1, L - 2, \ldots, 0$, while those in~\cref{gnn:thm:sep_rank_upper_bound} for undirected graphs depend only on walks of length $L - 1$.
As shown in the proof of~\cref{gnn:thm:sep_rank_upper_bound}, this dependence exists in undirected graphs as well.
Though, in undirected graphs with self-loops, the number of length $l \in \N$ walks from $\cut_\I$ decays exponentially as $l$ decreases.
One can therefore replace the sum over walk lengths with walks of length $L - 1$ (up to a multiplicative constant).
By contrast, in directed graphs this is not true in general, \eg,~when $\cutdir_\I$ contains only vertices with no outgoing edges (besides self-loops).

\begin{theorem}
	\label{gnn:thm:directed_sep_rank_upper_bound}
	For a directed graph with multiple edge types $\graph$ and $t \in \vertices$, let $\funcgraph{\params}{\graph}$ and $\funcvert{\params}{\graph}{t}$ be the functions realized by depth $L$ graph and vertex prediction GNNs, respectively, with width~$\hdim$, learnable weights $\params$, and product aggregation (\cref{gnn:eq:gnn_update_directed,gnn:eq:graph_pred_gnn,gnn:eq:vertex_pred_gnn,gnn:eq:prod_gnn_agg}).
	Then, for any $\I \subseteq \vertices$ and assignment of weights $\params$ it holds that:
	\begin{align}
		\text{(graph prediction)} \quad &\log \brk1{ \seprankbig{ \funcgraph{\params}{\graph} }{\I} }
		\leq
		\log \brk{ \hdim } \cdot \brk2{ \sum\nolimits_{l = 1}^{L} \nwalk{L - l}{\cutdir_\I}{ \vertices } + 1 } \text{\,,}
		\label{gnn:eq:directed_sep_rank_upper_bound_graph_pred} \\[0.5em]
		\text{(vertex prediction)} \quad &\log \brk1{ \seprankbig{ \funcvert{\params}{\graph}{ t } }{\I} }
		\leq
		\log \brk{ \hdim } \cdot \sum\nolimits_{l = 1}^{L} \nwalk{L - l}{\cutdir_\I}{ \{ t \} } \text{\,.}
		\label{gnn:eq:directed_sep_rank_upper_bound_vertex_pred}
	\end{align}
\end{theorem}
\begin{proof}[Proof sketch (proof in~\cref{gnn:app:proofs:directed_sep_rank_upper_bound})]
	The proof follows a line identical to that of~\cref{gnn:thm:sep_rank_upper_bound}, only requiring adjusting definitions from undirected graphs to directed graphs with multiple edge types.
\end{proof}

Towards lower bounding separation ranks, we generalize the definitions of vertex subsets with no repeating shared neighbors (\cref{gnn:def:no_rep_neighbors}) and admissible subsets of $\cut_\I$ (\cref{gnn:def:admissible_subsets}) to directed graphs.

\begin{definition}
	\label{gnn:def:directed_no_rep_neighbors}
	We say that $\I, \J \subseteq \vertices$ \emph{have no outgoing repeating shared neighbors} if every $k \in \neighout (\I) \cap \neighout (\J)$ has only a single incoming neighbor in each of $\I$ and $\J$, \ie~$\abs{\neighin (k) \cap \I} = \abs{\neighin (k) \cap \J} = 1$.
\end{definition}

\begin{definition}
	\label{gnn:def:directed_admissible_subsets}
	For $\I \subseteq \vertices$, we refer to $\cut \subseteq \cutdir_\I$ as an \emph{admissible subset of $\cutdir_\I$} if there exist $\I' \subseteq \I, \J' \subseteq \I^c$ with no outgoing repeating shared neighbors such that $\cut = \neighout (\I') \cap \neighout (\J')$.
	We use $\cutsetdir (\I)$ to denote the set comprising all admissible subsets of $\cutdir_\I$:
	\[
	\cutsetdir (\I) := \brk[c]1{ \cut \subseteq \cutdir_\I : \cut \text{ is an admissible subset of $\cutdir_\I$} }
	\text{\,.}
	\]
\end{definition}

\cref{gnn:thm:directed_sep_rank_lower_bound} generalizes the lower bounds from~\cref{gnn:thm:sep_rank_lower_bound} to directed graphs with multiple edge types.

\begin{theorem}
	\label{gnn:thm:directed_sep_rank_lower_bound}
	Consider the setting and notation of~\cref{gnn:thm:directed_sep_rank_upper_bound}.
	Given $I \subseteq \vertices$, for almost all assignments of weights $\params$, \ie~for all but a set of Lebesgue measure zero, it holds that:
	\begin{align}
		\text{(graph prediction)} \quad &\log \brk1{ \seprankbig{ \funcgraph{\params}{\graph} }{\I} }
		\geq
		\max_{ \cut \in \cutsetdir (\I) } \log \brk{ \alpha_{\cut} } \cdot \nwalk{L - 1}{\cut}{\vertices}
		\text{\,,}
		\label{gnn:eq:directed_sep_rank_lower_bound_graph_pred} \\[0.5em]
		\text{(vertex prediction)} \quad &\log \brk1{ \seprankbig{ \funcvert{\params}{\graph}{ t } }{\I} }
		\geq
		\max_{ \cut \in \cutsetdir (\I) } \log \brk{ \alpha_{\cut, t} } \cdot \nwalk{L - 1}{\cut}{ \{ t \} }
		\text{\,,}
		\label{gnn:eq:directed_sep_rank_lower_bound_vertex_pred}
	\end{align}
	where:
	\[
	\begin{split}
	\alpha_{\cut} & := \begin{cases}
		\mindim^{1 / \nwalk{0}{\cut}{\vertices} } & , \text{if } L = 1 \\
		\brk{ \mindim - 1 } \cdot \nwalk{L - 1}{\cut}{\vertices}^{-1} + 1 & , \text{if } L \geq 2
	\end{cases}
	\text{\,,} \\
	\alpha_{\cut, t} & := \begin{cases}
		\mindim & , \text{if } L = 1 \\
		\brk{ \mindim - 1 } \cdot \nwalk{L - 1}{\cut}{ \{ t \}}^{-1} + 1 & , \text{if } L \geq 2
	\end{cases}
	\text{\,,}
	\end{split}
	\]
	with $\mindim := \min \brk[c]{ \indim, \hdim }$.
	If $\nwalk{L - 1}{\cut}{\vertices} = 0$ or $\nwalk{L - 1}{\cut}{ \{ t \}} = 0$, the respective lower bound (right hand side of~\cref{gnn:eq:directed_sep_rank_lower_bound_graph_pred} or~\cref{gnn:eq:directed_sep_rank_lower_bound_vertex_pred}) is zero by convention.
\end{theorem}
\begin{proof}[Proof sketch (proof in~\cref{gnn:app:proofs:directed_sep_rank_lower_bound})]
	The proof follows a line identical to that of~\cref{gnn:thm:sep_rank_lower_bound}, only requiring adjusting definitions from undirected graphs to directed graphs with multiple edge types.
\end{proof}

\section{Representing Graph Neural Networks With Product Aggregation as Tensor Networks}
\label{gnn:app:prod_gnn_as_tn}

In this appendix, we prove that GNNs with product aggregation (\cref{gnn:sec:gnns}) can be represented through tensor networks~---~a graphical language for expressing tensor contractions, widely used in quantum physics literature for modeling quantum states (\cf~\cite{vidal2008class}).
This representation facilitates upper bounding the separation ranks of a GNN with product aggregation (proofs for \cref{gnn:thm:sep_rank_upper_bound} and its extension in~\cref{gnn:app:extensions}), and is delivered in~\cref{gnn:app:prod_gnn_as_tn:correspondence}.
We note that analogous tensor network representations were shown for variants of recurrent and convolutional neural networks~\cite{levine2018benefits,levine2018deep}.
For the convenience of the reader, we lay out basic concepts from the field of tensor analysis in~\cref{gnn:app:prod_gnn_as_tn:tensors} and provide a self-contained introduction to tensor networks in~\cref{gnn:app:prod_gnn_as_tn:tensor_networks} (see~\cite{orus2014practical} for a more in-depth treatment).

\subsection{Primer on Tensor Analysis}
\label{gnn:app:prod_gnn_as_tn:tensors}

For our purposes, a \emph{tensor} is simply a multi-dimensional array.
The \emph{order} of a tensor is its number of axes, which are typically called \emph{modes} (\eg~a vector is an order one tensor and a matrix is an order two tensor).
The \emph{dimension} of a mode refers to its length, \ie~the number of values it can be indexed with.
For an order $N \in \N$ tensor $\Atensor \in \R^{D_1 \times \cdots \times D_N}$ with modes of dimensions $D_1, \ldots, D_N \in \N$, we will denote by $\Atensor_{d_1, \ldots, d_N}$ its $(d_1, \ldots, d_N)$'th entry, where $\brk{d_1, \ldots, d_N} \in [D_1] \times \cdots \times [D_N]$.

It is possible to rearrange tensors into matrices~---~a process known as \emph{matricization}.
The matricization of $\Atensor$ with respect to $\I \subseteq [N]$, denoted $\mat{\Atensor}{\I} \in \R^{\prod_{i \in \I} D_i \times \prod_{j \in \I^c} D_j}$ is its arrangement as a matrix where rows correspond to modes indexed by $\I$ and columns correspond to the remaining modes.
Specifically, denoting the elements in $\I$ by $i_1 < \cdots < i_{\abs{\I}}$ and those in $\I^c$ by $j_1 < \cdots < j_{ \abs{\I^c} }$, the matricization $\mat{\Atensor}{\I}$ holds the entries of $\Atensor$ such that $\Atensor_{d_1, \ldots, d_N}$ is placed in row index $1 + \sum_{l = 1}^{\abs{\I}} (d_{i_{l}} - 1) \prod_{l' = l + 1}^{\abs{\I}} D_{i_{l'}}$ and column index $1 + \sum_{l = 1}^{ \abs{\I^c} } ( d_{ j_{l} } - 1 ) \prod_{l' = l + 1}^{ \abs{\I^c} } D_{ j_{l'} }$.

Tensors with modes of the same dimension can be combined via \emph{contraction}~---~a generalization of matrix multiplication.
It will suffice to consider contractions where one of the modes being contracted is the last mode of its tensor.

\begin{definition}
	\label{gnn:def:tensor_contraction}
	Let $\Atensor \in \R^{D_1 \times \cdots \times D_N}, \Btensor \in \R^{D'_1 \times \cdots \times D'_{N'}}$ for orders $N, N' \in \N$ and mode dimensions $D_1, \ldots, D_N, D'_1, \ldots, D'_{N'} \in \N$ satisfying $D_n = D'_{N'}$ for some $n \in [N]$.
	The \emph{mode-$n$ contraction} of $\Atensor$ with $\Btensor$, denoted $\Atensor \contract{n} \Btensor \in \R^{D_1 \times \cdots \times D_{n - 1} \times D'_1 \times \cdots \times D'_{N' - 1} \times D_{n + 1} \times \cdots \times D_N}$, is given element-wise by:
	\[
	\brk*{ \Atensor \contract{n} \Btensor }_{d_1, \ldots, d_{n - 1}, d'_1, \ldots, d'_{N' - 1}, d_{n + 1}, \ldots, d_N} = \sum\nolimits_{d_n = 1}^{D_n} \Atensor_{d_1, \ldots, d_N} \cdot \Btensor_{d'_1, \ldots, d'_{N' - 1}, d_n}
	\text{\,,}
	\]
	for all indices $d_1 \in [D_1], \ldots, d_{n - 1} \in [D_{n - 1}], d'_1 \in [D'_1], \ldots, d'_{N' - 1} \in [D'_{N' - 1}], d_{n + 1} \in [D_{n + 1}], \ldots, d_N \in [D_N]$.
\end{definition}
For example, the mode-$2$ contraction of $\Abf \in \R^{D_1 \times D_2}$ with $\Bbf \in \R^{D'_1 \times D_2}$ boils down to multiplying $\Abf$ with $\Bbf^\top$ from the right, \ie~$\Abf \contract{2} \Bbf = \Abf \Bbf^\top$.
It is oftentimes convenient to jointly contract multiple tensors.
Given an order $N$ tensor $\Atensor$ and $M \in \N_{\leq N}$ tensors $\Btensor^{(1)}, \ldots, \Btensor^{(M)}$, we use $\Atensor \contract{i \in [M]} \Btensor^{(i)}$ to denote the contraction of $\Atensor$ with $\Btensor^{(1)}, \ldots, \Btensor^{(M)}$ in modes $1, \ldots, M$, respectively (assuming mode dimensions are such that the contractions are well-defined).

\subsection{Tensor Networks}
\label{gnn:app:prod_gnn_as_tn:tensor_networks}

\begin{figure*}[t]
	\vspace{0mm}
	\begin{center}
		\includegraphics[width=\textwidth]{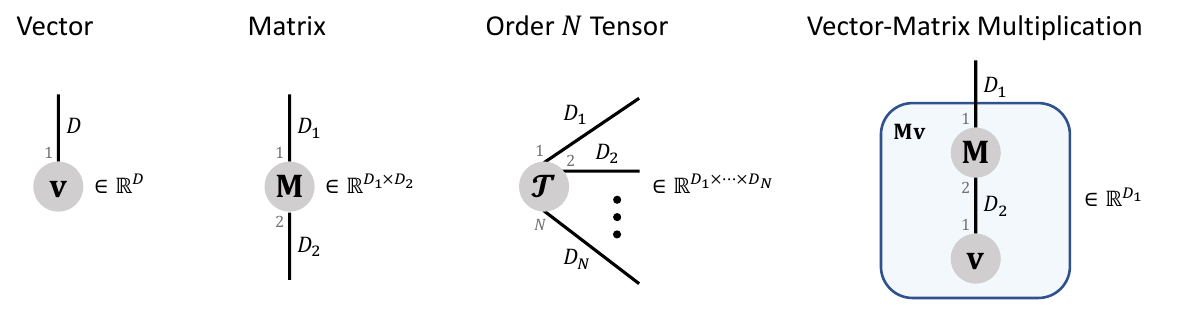}
	\end{center}
	\vspace{-1mm}
	\caption{
		Tensor network diagrams of (from left to right): a vector $\vbf \in \R^D$, matrix $\Mbf \in \R^{D_1 \times D_2}$, order $N \in \N$ tensor $\Ttensor \in \R^{D_1 \times \cdots \times D_N}$, and vector-matrix multiplication $\Mbf \vbf \in \R^{D_1}$.
		The mode index associated with a leg's end point is specified in gray, and the weight of the leg, specified in black, determines the mode dimension.
	}
	\label{gnn:fig:tensor_networks_examples}
\end{figure*}

A \emph{tensor network} is an undirected weighted graph $\tngraph = \brk{ \tnvertices{\tngraph}, \tnedges{\tngraph}, \tnedgeweights{\tngraph} }$ that describes a sequence of tensor contractions (\cref{gnn:def:tensor_contraction}), with vertices $\tnvertices{\TT}$, edges $\tnedges{\TT}$, and a function mapping edges to natural weights $\tnedgeweights{\TT}: \tnedges{\TT} \to \N$.
We will only consider tensor networks that are connected.
To avoid confusion with vertices and edges of a GNN's input graph, and in accordance with tensor network terminology, we refer by \emph{nodes} and \emph{legs} to the vertices and edges of a tensor network, respectively.

Every node in a tensor network is associated with a tensor, whose order is equal to the number of legs emanating from the node.
Each end point of a leg is associated with a mode index, and the leg's weight determines the dimension of the corresponding tensor mode.
That is, an end point of $e \in \tnedges{\TT}$ is a pair $(\Atensor, n) \in \tnvertices{\TT} \times \N$, with $n$ ranging from one to the order of $\Atensor$, and $\tnedgeweights{\TT} (e)$ is the dimension of $\Atensor$ in mode $n$.
A leg can either connect two nodes or be connected to a node on one end and be loose on the other end.
If two nodes are connected by a leg, their associated tensors are contracted together in the modes specified by the leg.
Legs with a loose end are called \emph{open legs}.
The number of open legs is exactly the order of the tensor produced by executing all contractions in the tensor network, \ie~by contracting the tensor network.
\cref{gnn:fig:tensor_networks_examples} presents exemplar tensor network diagrams of a vector, matrix, order $N \in \N$ tensor, and vector-matrix multiplication.

\subsection{Tensor Networks Corresponding to Graph Neural Networks With Product Aggregation}
\label{gnn:app:prod_gnn_as_tn:correspondence}

Fix some undirected graph $\graph$ and learnable weights $\params = \brk{ \weightmat{1}, \ldots, \weightmat{L}, \weightmat{o} }$.
Let $\funcgraph{\params}{\graph}$ and $\funcvert{\params}{\graph}{t}$, for $t \in \vertices$, be the functions realized by depth $L$ graph and vertex prediction GNNs, respectively, with width~$\hdim$ and product aggregation (\cref{gnn:eq:gnn_update,gnn:eq:graph_pred_gnn,gnn:eq:vertex_pred_gnn,gnn:eq:prod_gnn_agg}).
For $\fmat = \brk{ \fvec{1}, \ldots, \fvec{ \abs{\vertices} } }\in \R^{\indim \times \abs{\vertices} }$, we construct tensor networks $\tngraph (\fmat)$ and $\tngraph^{(t)} (\fmat)$ whose contraction yields $\funcgraph{\params}{\graph} (\fmat)$ and $\funcvert{\params}{\graph}{t} (\fmat)$, respectively. 
Both $\tngraph (\fmat)$ and $\tngraph^{(t)} (\fmat)$ adhere to a tree structure, where each leaf node is associated with a vertex feature vector, \ie~one of $\fvec{1}, \ldots, \fvec{ \abs{\vertices} }$, and each interior node is associated with a weight matrix from $\weightmat{1}, \ldots, \weightmat{L}, \weightmat{o}$ or a \emph{$\delta$-tensor} with modes of dimension $\hdim$, holding ones on its hyper-diagonal and zeros elsewhere.
We denote an order $N \in \N$ tensor of the latter type by $\deltatensor{N} \in \R^{\hdim \times \cdots \times \hdim}$, \ie~$\deltatensor{N}_{d_1, \ldots, d_N} = 1$ if $d_1 = \cdots = d_N$ and $\deltatensor{N}_{d_1, \ldots, d_N} = 0$ otherwise for all $d_1, \ldots, d_N \in [\hdim]$.

\begin{figure*}[t]
	\vspace{0mm}
	\begin{center}
		\includegraphics[width=1\textwidth]{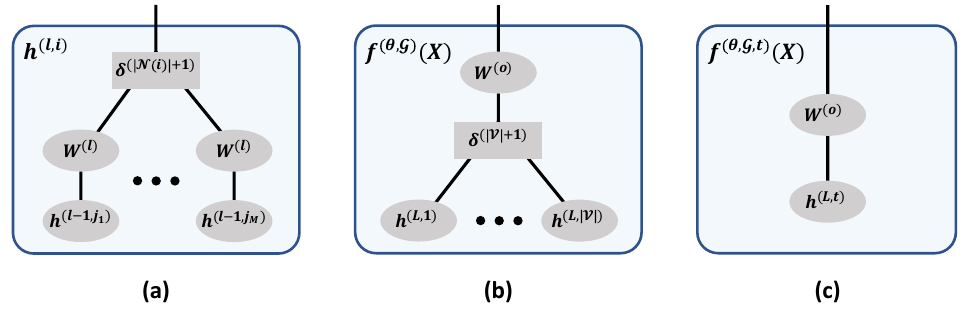}
	\end{center}
	\vspace{-2mm}
	\caption{
		Tensor network diagrams of the operations performed by GNNs with product aggregation (\cref{gnn:sec:gnns}).
		\textbf{(a)}~Hidden embedding update (\cf~\cref{gnn:eq:gnn_update,gnn:eq:prod_gnn_agg}): $\hidvec{l}{i} = \brk{ \weightmat{l} \hidvec{l - 1}{j_1} } \hadmp \cdots \hadmp \brk{ \weightmat{l} \hidvec{l - 1}{j_M} }$, where $\neigh (i) = \brk[c]{ j_1, \ldots, j_M}$, for $l \in [L], i \in \vertices$.
		\textbf{(b)}~Output layer for graph prediction (\cf~\cref{gnn:eq:graph_pred_gnn,gnn:eq:prod_gnn_agg}): $\funcgraph{\params}{\graph} (\fmat) = \weightmat{o} \brk{ \hidvec{L}{1} \hadmp \cdots \hadmp \hidvec{L}{ | \vertices | } }$. 
		\textbf{(c)}~Output layer for vertex prediction over $t \in \vertices$ (\cf~\cref{gnn:eq:vertex_pred_gnn}): $\funcvert{\params}{\graph}{t} (\fmat) = \weightmat{o} \hidvec{L}{t}$.
		We draw nodes associated with $\delta$-tensors as rectangles to signify their special (hyper-diagonal) structure, and omit leg weights to avoid clutter (legs connected to $\hidvec{0}{i} = \fvec{i}$, for $i \in \vertices$, have weight $\indim$ while all other legs have weight $\hdim$).
	}
	\label{gnn:fig:gnn_tensor_network_parts}
\end{figure*}

Intuitively, $\tngraph (\fmat)$ and $\tngraph^{(t)} (\fmat)$ embody unrolled computation trees, describing the operations performed by the respective GNNs through tensor contractions.
Let $\hidvec{l}{i} = \hadmp_{ j \in \neigh (i) } \brk{ \weightmat{l} \hidvec{l - 1}{j} }$ be the hidden embedding of $i \in \vertices$ at layer $l \in [L]$ (recall $\hidvec{0}{j} = \fvec{j}$ for $j \in \vertices$), and denote $\neigh (i) = \brk[c]{ j_1, \ldots, j_M}$.
We can describe $\hidvec{l}{i}$ as the outcome of contracting each $\hidvec{l - 1}{ j_1 }, \ldots, \hidvec{ l - 1}{j_M}$ with $\weightmat{l}$, \ie~computing $\weightmat{l} \hidvec{l - 1}{j_1}, \ldots, \weightmat{l} \hidvec{l - 1}{j_M}$, followed by contracting the resulting vectors with $\deltatensor{ \abs{\neigh (i)} +1}$, which induces product aggregation (see~\cref{gnn:fig:gnn_tensor_network_parts}(a)).
Furthermore, in graph prediction, the output layer producing $\funcgraph{\params}{\graph} (\fmat) = \weightmat{o} \brk{ \hadmp_{ i \in \vertices} \hidvec{L}{i} }$ amounts to contracting $\hidvec{L}{1}, \ldots, \hidvec{L}{ \abs{\vertices}}$ with $\deltatensor{ \abs{\vertices} + 1}$, and subsequently contracting the resulting vector with $\weightmat{o}$ (see~\cref{gnn:fig:gnn_tensor_network_parts}(b)); while for vertex prediction, $\funcvert{\params}{\graph}{t} (\fmat) = \weightmat{o} \hidvec{L}{t}$ is a contraction of $\hidvec{L}{t}$ with $\weightmat{o}$ (see~\cref{gnn:fig:gnn_tensor_network_parts}(c)).

Overall, every layer in a GNN with product aggregation admits a tensor network formulation given the outputs of the previous layer.
Thus, we can construct a tree tensor network for the whole GNN by starting from the output layer~---~\cref{gnn:fig:gnn_tensor_network_parts}(b) for graph prediction or \cref{gnn:fig:gnn_tensor_network_parts}(c) for vertex prediction~---~and recursively expanding nodes associated with $\hidvec{l}{i}$ according to~\cref{gnn:fig:gnn_tensor_network_parts}(a), for $l = L, \ldots, 1$ and $i \in \vertices$.
A technical subtlety is that each $\hidvec{l}{i}$ can appear multiple times during this procedure.
In the language of tensor networks this translate to duplication of nodes.
Namely, there are multiple copies of the sub-tree representing $\hidvec{l}{i}$ in the tensor network~---~one copy per appearance when unraveling the recursion.
\cref{gnn:fig:prod_gnn_tensor_networks_example} displays examples for tensor network diagrams of $\tngraph (\fmat)$ and $\tngraph^{(t)} (\fmat)$.

We note that, due to the node duplication mentioned above, the explicit definitions of $\tngraph (\fmat)$ and $\tngraph^{(t)} (\fmat)$ entail cumbersome notation.
Nevertheless, we provide them in Appendix~\ref{gnn:app:prod_gnn_as_tn:correspondence:explicit_tn} for the interested reader.

\begin{figure*}[t]
	\vspace{0mm}
	\begin{center}
		\includegraphics[width=1\textwidth]{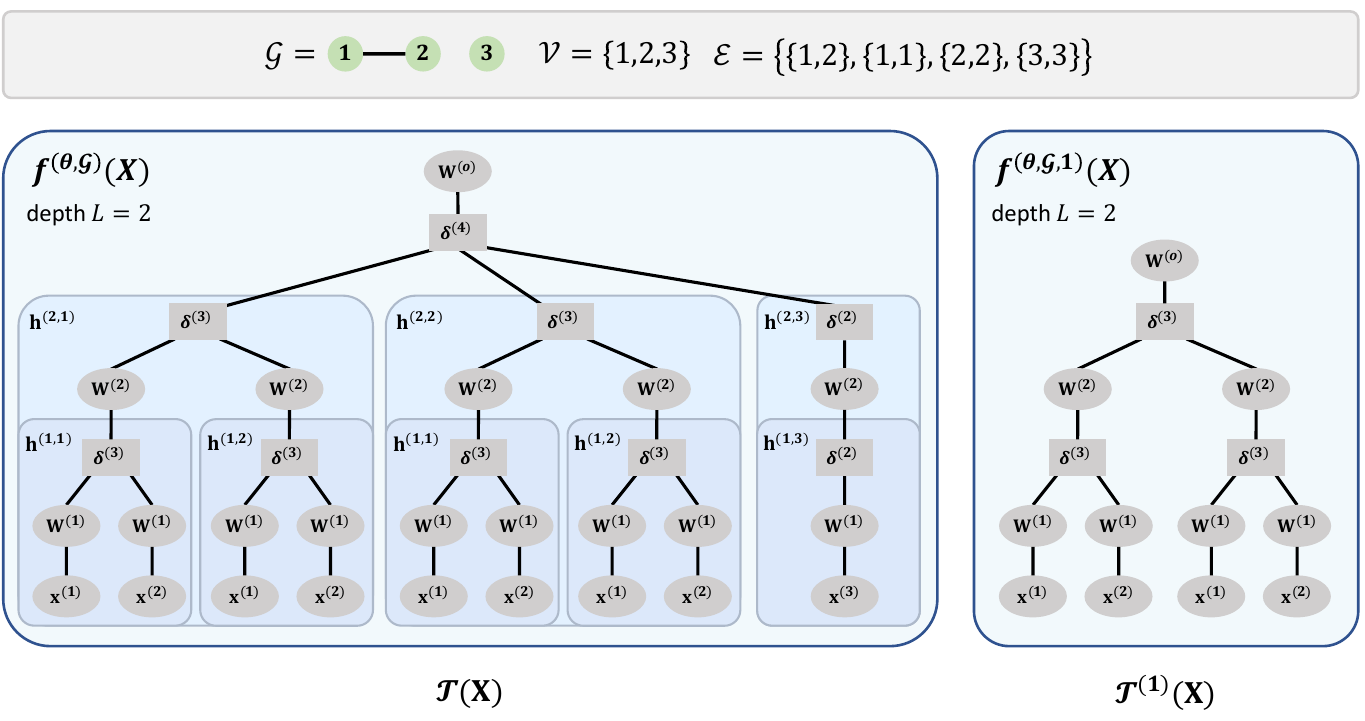}
	\end{center}
	\vspace{-2mm}
	\caption{
		Tensor network diagrams of $\tngraph (\fmat)$ (left) and $\tngraph^{(t)} (\fmat)$ (right) representing $\funcgraph{\params}{\graph} ( \fmat )$ and $\funcvert{\params}{\graph}{t} (\fmat)$, respectively, for $t = 1 \in \vertices$, vertex features $\fmat = \brk{ \fvec{1}, \ldots, \fvec{ | \vertices | } }$, and depth $L = 2$ GNNs with product aggregation (\cref{gnn:sec:gnns}).
		The underlying input graph $\graph$, over which the GNNs operate, is depicted at the top.
		We draw nodes associated with $\delta$-tensors as rectangles to signify their special (hyper-diagonal) structure, and omit leg weights to avoid clutter (legs connected to $\fvec{1}, \fvec{2}, \fvec{3}$ have weight $\indim$ while all other legs have weight $\hdim$).
		See~\cref{gnn:app:prod_gnn_as_tn:correspondence} for further details on the construction of $\tngraph (\fmat)$ and $\tngraph^{(t)} (\fmat)$, and Appendix~\ref{gnn:app:prod_gnn_as_tn:correspondence:explicit_tn} for explicit formulations.
	}
	\label{gnn:fig:prod_gnn_tensor_networks_example}
\end{figure*}

\subsubsection{Explicit Tensor Network Definitions}
\label{gnn:app:prod_gnn_as_tn:correspondence:explicit_tn}

The tree tensor network representing $\funcgraph{\params}{\graph} (\fmat)$ consists of an initial input level~---~the leaves of the tree~---~comprising $\nwalk{L}{ \{i\} }{ \vertices }$ copies of $\fvec{i}$ for each $i \in \vertices$.
We will use $\fvecdup{i}{\gamma}$ to denote the copies of $\fvec{i}$ for $i \in \vertices$ and $\gamma \in [\nwalk{L}{ \{i\} }{ \vertices }]$.
In accordance with the GNN inducing $\funcgraph{\params}{\graph}$, following the initial input level are $L + 1$ \emph{layers}.
Each layer $l \in [L]$ includes two levels: one comprising $\nwalk{L - l + 1}{\vertices}{\vertices}$ nodes standing for copies of $\weightmat{l}$, and another containing $\delta$-tensors~---~$\nwalk{L - l}{ \{i\} }{\vertices}$ copies of $\deltatensor{ \abs{\neigh(i)} + 1 }$ per $i \in \vertices$.
We associate each node in these layers with its layer index and a vertex of the input graph $i \in \vertices$.
Specifically, we will use $\weightmatdup{l}{i}{\gamma}$ to denote copies of $\weightmat{l}$ and $\deltatensordup{l}{i}{\gamma}$ to denote copies of $\deltatensor{ \abs{ \neigh (i) } + 1}$, for $l \in [L], i \in \vertices$, and $\gamma \in \N$.
In terms of connectivity, every leaf $\fvecdup{i}{\gamma}$ has a leg to $\weightmatdup{1}{i}{\gamma}$.
The rest of the connections between nodes are such that each sub-tree whose root is $\deltatensordup{l}{i}{\gamma}$ represents $\hidvec{l}{i}$, \ie~contracting the sub-tree results in the hidden embedding for $i \in \vertices$ at layer $l \in [L]$ of the GNN inducing $\funcgraph{\params}{\graph}$.
Last, is an output layer consisting of two connected nodes: a $\deltatensor{\abs{\vertices} + 1}$ node, which has a leg to every $\delta$-tensor from layer $L$, and a $\weightmat{o}$ node.
See~\cref{gnn:fig:prod_gnn_tensor_networks_example_formal} (left) for an example of a tensor network diagram representing~$\funcgraph{\params}{\graph} (\fmat)$ with this notation.

\begin{figure*}[t]
	\vspace{0mm}
	\begin{center}
		\includegraphics[width=1\textwidth]{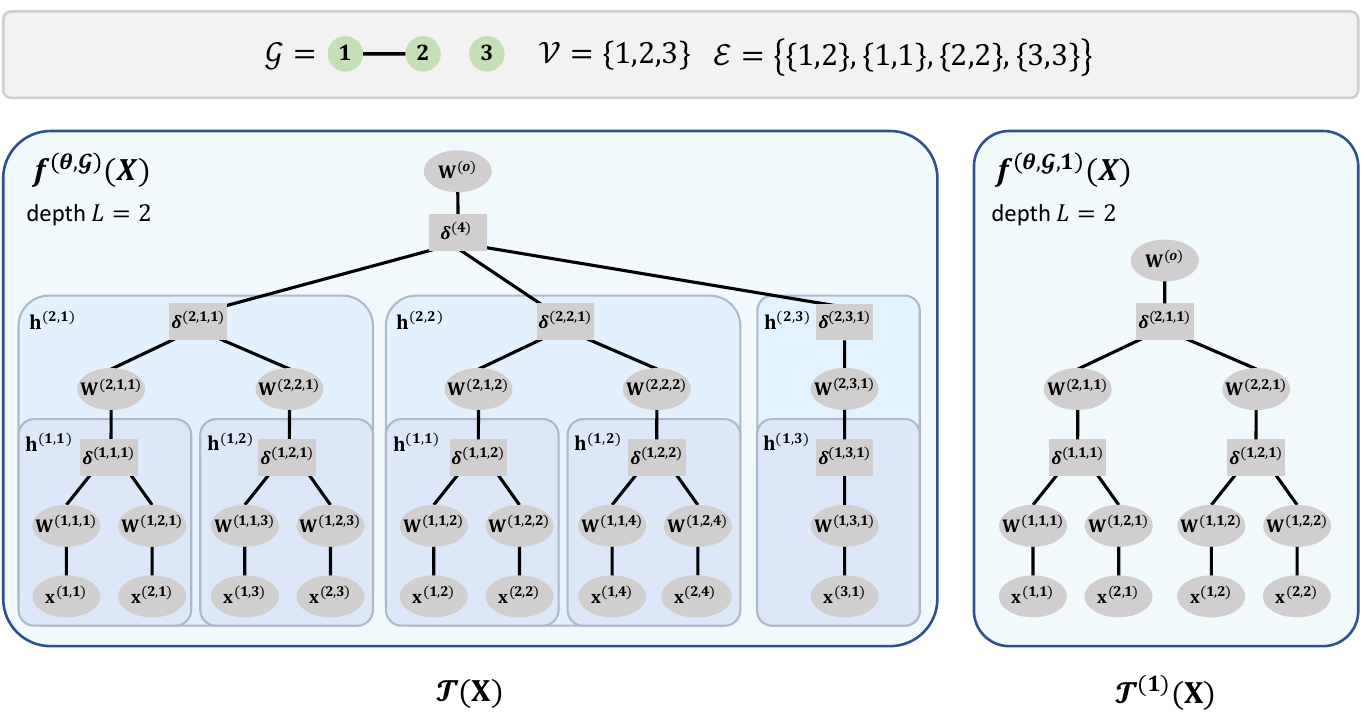}
	\end{center}
	\vspace{-2mm}
	\caption{
		Tensor network diagrams (with explicit node duplication notation) of $\tngraph (\fmat)$ (left) and $\tngraph^{(t)} (\fmat)$ (right) representing $\funcgraph{\params}{\graph} ( \fmat )$ and $\funcvert{\params}{\graph}{t} (\fmat)$, respectively, for $t = 1 \in \vertices$, vertex features $\fmat = \brk{ \fvec{1}, \ldots, \fvec{ | \vertices | } }$, and depth $L = 2$ GNNs with product aggregation (\cref{gnn:sec:gnns}).
		This figure is identical to~\cref{gnn:fig:prod_gnn_tensor_networks_example}, except that it uses the explicit notation for node duplication detailed in Appendix~\ref{gnn:app:prod_gnn_as_tn:correspondence:explicit_tn}.
		Specifically, each feature vector, weight matrix, and $\delta$-tensor is attached with an index specifying which copy it is (rightmost index in the superscript).
		Additionally, weight matrices and $\delta$-tensors are associated with a layer index and vertex in~$\vertices$ (except for the output layer $\delta$-tensor in $\tngraph (\fmat)$ and $\weightmat{o}$).
		See~\cref{gnn:eq:graphtensor_tn,gnn:eq:vertextensor_tn} for the explicit definitions of these tensor networks.
	}
	\label{gnn:fig:prod_gnn_tensor_networks_example_formal}
\end{figure*}

The tensor network construction for $\funcvert{\params}{\graph}{t} (\fmat)$ is analogous to that for $\funcgraph{\params}{\graph} (\fmat)$, comprising an initial input level followed by $L + 1$ layers.
Its input level and first $L$ layers are structured the same, up to differences in the number of copies for each node.
Specifically, the number of copies of $\fvec{i}$ is $\nwalk{L}{\{ i\}}{\{t\}}$ instead of $\nwalk{L}{\{i\}}{\vertices}$, the number of copies of $\weightmat{l}$ is $\nwalk{L - l + 1}{\vertices}{ \{t\}}$ instead of $\nwalk{L -l + 1}{\vertices}{ \vertices }$, and the number of copies of $\deltatensor{ \abs{ \neigh (i) } + 1 }$ is $\nwalk{L - l}{ \{i\} }{ \{t\} }$ instead of $\nwalk{L - l}{ \{i\} }{ \vertices }$, for $i \in \vertices$ and $l \in [L]$.
The output layer consists only of a $\weightmat{o}$ node, which is connected to the $\delta$-tensor in layer $L$ corresponding to vertex $t$.
See~\cref{gnn:fig:prod_gnn_tensor_networks_example_formal} (right) for an example of a tensor network diagram representing $\funcvert{\params}{\graph}{t} (\fmat)$ with this notation.

Formally, the tensor network producing $\funcgraph{\params}{\graph} (\fmat)$, denoted 
\[
\tngraph (\fmat) = \brk{ \tnvertices{\tngraph (\fmat)}, \tnedges{\tngraph (\fmat)}, \tnedgeweights{ \tngraph (\fmat) } }
\text{\,,}
\]
is defined by:
\be
\fontsize{9}{9}\selectfont
\begin{split}
	\tnvertices{\tngraph (\fmat)} := &~\brk[c]2{ \fvecdup{i}{\gamma} : i \in \vertices, \gamma \in [ \nwalk{L}{ \{i\} }{\vertices} ] } \cup \\
	&~\brk[c]2{ \weightmatdup{l}{i}{\gamma} : l \in [L], i \in \vertices, \gamma \in [ \nwalk{L - l + 1}{ \{i\} }{\vertices} ] } \cup \\
	&~\brk[c]2{ \deltatensordup{l}{i}{\gamma} : l \in [L], i \in \vertices, \gamma \in [ \nwalk{L - l}{ \{i\} }{ \vertices } ] } \cup \\
	&~\brk[c]2{ \deltatensor{\abs{\vertices} + 1 }, \weightmat{o} } \text{\,,} \\[1em]
	\tnedges{\tngraph (\fmat)} :=&~\brk[c]2{ \brk[c]1{ ( \fvecdup{i}{\gamma}, 1) , (\weightmatdup{1}{i}{\gamma}, 2 ) } : i \in \vertices , \gamma \in [ \nwalk{L}{ \{i\} }{\vertices} ] }\cup \\
	&~\brk[c]2{ \brk[c]1{ ( \deltatensordup{l}{i}{\gamma}, j ) , ( \weightmatdup{l}{ \neigh(i)_j }{ \dupmap{l}{i}{j} (\gamma) } , 1 ) } : l \in [L], i \in \vertices, j \in [ \abs{ \neigh (i) } ] , \gamma \in [ \nwalk{L - l}{ \{i\} }{ \vertices } ] } \cup \\
	&~\brk[c]2{ \brk[c]1{ ( \deltatensordup{l}{i}{\gamma}, \abs{\neigh (i)} + 1 ) , ( \weightmatdup{l + 1}{i}{\gamma} , 2 ) } : l \in [L - 1] , i \in \vertices , \gamma \in [ \nwalk{L - l}{ \{i\} }{ \vertices } ] } \cup \\
	&~\brk[c]2{ \brk[c]1{ ( \deltatensor{ \abs{\vertices} + 1 } , i ) , ( \deltatensordup{L}{i}{1} , \abs{ \neigh (i) } + 1 )} : i \in \vertices } \cup \brk[c]2{ \brk[c]1{ ( \deltatensor{ \abs{\vertices} + 1} , \abs{\vertices} + 1 ) , ( \weightmat{o} , 2 ) } } \text{\,,} \\[1em]
	\tnedgeweights{\tngraph (\fmat)} (e) :=& \begin{cases}
		\indim	& , \text{ if $(\fvecdup{i}{\gamma}, 1)$ is an endpoint of $e \in \tnedges{\tngraph}$ for some $i \in \vertices , \gamma \in [ \nwalk{L}{ \{i\} }{ \vertices } ]$} \\
		\hdim	& , \text{ otherwise}
	\end{cases} 
	\text{\,,}
\end{split}
\label{gnn:eq:graphtensor_tn}
\ee
where $\dupmap{l}{i}{j} (\gamma) := \gamma + \sum\nolimits_{k < i \text{ s.t. } k \in \neigh (j)} \nwalk{L - l}{ \{ k \} }{\vertices}$, for $l \in [L], i \in \vertices,$ and $\gamma \in [ \nwalk{L - l}{ \{i\} }{ \vertices } ]$, is used to map a $\delta$-tensor copy corresponding to $i$ in layer $l$ to a $\weightmat{l}$ copy, and $\neigh (i)_j$, for $i \in \vertices$ and $j \in [\abs{\neigh (i)}]$, denotes the $j$'th neighbor of $i$ according to an ascending order (recall vertices are represented by indices from $1$ to $\abs{\vertices}$).

Similarly, the tensor network producing $\funcvert{\params}{\graph}{t} (\fmat)$, denoted 
\[
\tngraph^{(t)} (\fmat) = \brk{ \tnvertices{\tngraph^{(t)} (\fmat)}, \tnedges{\tngraph^{(t)} (\fmat)}, \tnedgeweights{\tngraph^{(t)} (\fmat)} }
\text{\,,}
\]
is defined by:
\be
\fontsize{9}{9}\selectfont
\begin{split}
	\tnvertices{\tngraph^{(t)} (\fmat)} := &~\brk[c]2{ \fvecdup{i}{\gamma} : i \in \vertices, \gamma \in [ \nwalk{L}{ \{i\} }{\{t\}} ] } \cup \\
	&~\brk[c]2{ \weightmatdup{l}{i}{\gamma} : l \in [L], i \in \vertices, \gamma \in [ \nwalk{L - l + 1}{ \{i\} }{ \{t\} } ] } \cup \\
	&~\brk[c]2{ \deltatensordup{l}{i}{\gamma} : l \in [L], i \in \vertices, \gamma \in [ \nwalk{L - l}{ \{i\} }{ \{t\} } ] } \cup \\
	&~\brk[c]2{ \weightmat{o} } \text{\,,} \\[1em]
	\tnedges{\tngraph^{(t)} (\fmat)} :=&~\brk[c]2{ \brk[c]1{ ( \fvecdup{i}{\gamma} , 1 ) , ( \weightmatdup{1}{i}{\gamma}, 2 ) } : i \in \vertices , \gamma \in [ \nwalk{L}{ \{i\} }{ \{t\} } ] } \cup \\
	&~\brk[c]2{ \brk[c]1{ ( \deltatensordup{l}{i}{\gamma}, j ) , ( \weightmatdup{l}{ \neigh(i)_j }{ \dupmapvertex{l}{i}{j} (\gamma) } , 1 ) } : l \in [L], i \in \vertices, j \in [ \abs{ \neigh (i) } ] , \gamma \in [ \nwalk{L - l}{ \{i\} }{ \{t\} } ] } \cup \\
	&~\brk[c]2{ \brk[c]1{ ( \deltatensordup{l}{i}{\gamma}, \abs{\neigh (i)} + 1 ) , ( \weightmatdup{l + 1}{i}{\gamma} , 2 ) } : l \in [L - 1] , i \in \vertices , \gamma \in [ \nwalk{L - l}{ \{i\} }{ \{t\} } ] } \cup \\
	&~\brk[c]2{ \brk[c]1{ ( \deltatensordup{L}{t}{1} , \abs{ \neigh (t) } + 1 ) , ( \weightmat{o} , 2 ) } } \text{\,,} \\[1em]
	\tnedgeweights{\tngraph^{(t)} (\fmat)} (e) :=& \begin{cases}
		\indim	& , \text{ if $(\fvecdup{i}{\gamma}, 1)$ is an endpoint of $e \in \tnedges{\tngraph}$ for some $i \in \vertices , \gamma \in [ \nwalk{L}{ \{i\} }{ \{t\} } ]$ } \\
		\hdim	& , \text{ otherwise}
	\end{cases} 
	\text{\,,}
\end{split}
\label{gnn:eq:vertextensor_tn}
\ee
where $\dupmapvertex{l}{i}{j} (\gamma) := \gamma + \sum\nolimits_{k < i \text{ s.t. } k \in \neigh (j)} \nwalk{L - l}{ \{ k \} }{ \{t\} }$, for $l \in [L], i \in \vertices,$ and $\gamma \in [ \nwalk{L - l}{ \{i\} }{ \{t\} } ]$, is used to map a $\delta$-tensor copy corresponding to $i$ in layer $l$ to a $\weightmat{l}$ copy.

\cref{gnn:prop:tn_constructions} verifies that contracting $\tngraph (\fmat)$ and $\tngraph^{(t)} (\fmat)$ results in $\funcgraph{\params}{\graph} (\fmat)$ and $\funcvert{\params}{\graph}{t} (\fmat)$, respectively.

\begin{proposition}
	\label{gnn:prop:tn_constructions}
	For an undirected graph $\graph$ and $t \in \vertices$, let $\funcgraph{\params}{\graph}$ and $\funcvert{\params}{\graph}{t}$ be the functions realized by depth $L$ graph and vertex prediction GNNs, respectively, with width~$\hdim$, learnable weights $\params$, and product aggregation (\cref{gnn:eq:gnn_update,gnn:eq:graph_pred_gnn,gnn:eq:vertex_pred_gnn,gnn:eq:prod_gnn_agg}).
	For vertex features $\fmat = \brk{ \fvec{1}, \ldots, \fvec{\abs{\vertices}} } \in \R^{\indim \times \abs{\vertices}}$, let the tensor networks $\tngraph (\fmat) = \brk{ \tnvertices{\tngraph (\fmat)}, \tnedges{\tngraph (\fmat)}, \tnedgeweights{\tngraph (\fmat)} }$ and $\tngraph^{(t)} (\fmat) = \brk{ \tnvertices{\tngraph^{(t)} (\fmat)}, \tnedges{\tngraph^{(t)} (\fmat)}, \tnedgeweights{\tngraph^{(t)} (\fmat)} }$ be as defined in~\cref{gnn:eq:graphtensor_tn,gnn:eq:vertextensor_tn}, respectively.
	Then, performing the contractions described by $\tngraph (\fmat)$ produces~$\funcgraph{\params}{\graph} (\fmat)$, and performing the contractions described by $\tngraph^{(t)} (\fmat)$ produces~$\funcvert{\params}{\graph}{t} (\fmat)$.
\end{proposition}
\begin{proof}[Proof sketch (proof in~\cref{gnn:app:proofs:tn_constructions})]
	For both $\tngraph (\fmat)$ and $\tngraph^{(t)} (\fmat)$, a straightforward induction over the layer $l \in [L]$ establishes that contracting the sub-tree whose root is $\deltatensordup{l}{i}{\gamma}$ results in $\hidvec{l}{i}$ for all $i \in \vertices$ and $\gamma$, where $\hidvec{l}{i}$ is the hidden embedding for $i$ at layer $l$ of the GNNs inducing $\funcgraph{\params}{\graph}$ and $\funcvert{\params}{\graph}{t}$, given vertex features $\fvec{1}, \ldots, \fvec{\abs{\vertices}}$.
	The proof concludes by showing that the contractions in the output layer of $\tngraph (\fmat)$ and $\tngraph^{(t)} (\fmat)$ reproduce the operations defining $\funcgraph{\params}{\graph} (\fmat)$ and $\funcvert{\params}{\graph}{t} (\fmat)$ in~\cref{gnn:eq:graph_pred_gnn,gnn:eq:vertex_pred_gnn}, respectively.
\end{proof}

\section{General Walk Index Sparsification}
\label{gnn:app:walk_index_sparsification_general}

Our edge sparsification algorithm~---~Walk Index Sparsification (WIS)~---~was obtained as an instance of the General Walk Index Sparsification (GWIS) scheme described in~\cref{gnn:sec:sparsification}.
\cref{alg:walk_index_sparsification_general} formally outlines this general scheme.

\begin{algorithm}[H]
	\caption{$(L - 1)$-General Walk Index Sparsification (GWIS)} 
	\label{alg:walk_index_sparsification_general}
	\begin{algorithmic}
		\STATE \!\!\!\!\textbf{Input:} {
			\begin{itemize}[leftmargin=2em]
				\vspace{-1mm}
				\item $\graph$~---~graph
				\vspace{-1mm}
				\item $L \in \N$~---~GNN depth
				\vspace{-1mm}
				\item $N \in \N$~---~number of edges to remove
				\vspace{-1mm}
				\item $\I_{1}, \ldots, \I_{M} \subseteq \vertices$~---~vertex subsets specifying walk indices to maintain for graph prediction
				\vspace{-1mm}
				\item $\J_1, \ldots, \J_{M'} \subseteq \vertices$ and $t_1, \ldots, t_{M'} \in \vertices$~---~vertex subsets specifying walk indices to maintain with respect to target vertices, for vertex prediction
				\vspace{-1mm}
				\item {\sc argmax }~---~operator over tuples $\brk{ \sbf^{(e)} \in \R^{M + M'}}_{e \in \edges}$ that returns the edge whose tuple is maximal according to some order
			\end{itemize}
		}
		\STATE \!\!\!\!\textbf{Result:} Sparsified graph obtained by removing $N$ edges from $\graph$ \\[0.4em]
		\hrule
		\vspace{2mm}
		\FOR{$n = 1, \ldots, N$}
		\vspace{0.5mm}
		\STATE \darkgray{\# for every edge, compute walk indices of partitions after the edge's removal}
		\vspace{0.5mm}
		\FOR{$e \in \edges$ (excluding self-loops)}
		\vspace{0.5mm}
		\STATE initialize $\sbf^{(e)} = \brk{0, \ldots, 0} \in \R^{M + M'}$
		\vspace{0.5mm}
		\STATE remove $e$ from $\graph$ (temporarily)
		\vspace{0.5mm}
		\STATE for every $m \in [M]$, set $\sbf^{(e)}_{m} = \walkin{L - 1}{ \I_{m} }$ ~\darkgray{\# $= \nwalk{L - 1}{ \cut_{ \I_{m} } }{ \vertices }$}		
		\vspace{0.5mm}
		\STATE for every $m \in [M']$, set $\sbf^{(e)}_{M + m} = \walkinvert{L - 1}{t_m}{ \J_m }$ ~\darkgray{\# $= \nwalk{L - 1}{ \cut_{\J_m} }{ \{ t_m\}}$}		
		\vspace{0.5mm}
		\STATE add $e$ back to $\graph$
		\vspace{0.5mm}
		\ENDFOR
		\vspace{0.5mm}
		\STATE \darkgray{\# prune edge whose removal harms walk indices the least according to the {\sc argmax} operator}
		\vspace{0.5mm}
		\STATE let $e' \in \text{\sc{argmax}}_{e \in \edges} \sbf^{(e)}$
		\vspace{0.5mm}
		\STATE \textbf{remove} $e'$ from $\graph$ (permanently)
		\vspace{0.5mm}
		\ENDFOR
	\end{algorithmic}
\end{algorithm}

\section{Efficient Implementation of $1$-Walk Index Sparsification}
\label{gnn:app:one_wis_efficient}

\cref{alg:walk_index_sparsification_vertex_one} (\cref{gnn:sec:sparsification}) provides an efficient implementation for $1$-WIS, \ie~\cref{alg:walk_index_sparsification_vertex} with $L = 2$.
In this appendix, we formalize the equivalence between the two algorithms, meaning, we establish that~\cref{alg:walk_index_sparsification_vertex_one} indeed implements $1$-WIS.

Examining some iteration $n \in [N]$ of $1$-WIS, let $\sbf \in \R^{ \abs{\vertices} }$ be the tuple defined by $\sbf_t = \walkinvert{1}{t}{ \{ t\} } = \nwalk{1}{ \cut_{ \{t\} } }{ \{t\} }$ for $t \in \vertices$.
Recall that $\cut_{\{t\}}$ is the set of vertices with an edge crossing the partition induced by $\{t\}$.
Thus, if $t$ is not isolated, then $\cut_{\{t\}} = \neigh (t)$ and $\sbf_t = \walkinvert{1}{t}{ \{t\}} = \abs{ \neigh (t) }$.
Otherwise, if $t$ is isolated, then $\cut_{ \{t\} } = \emptyset$ and $\sbf_t = \walkinvert{1}{t}{ \{t\}} = 0$.
$1$-WIS computes for each $e \in \edges$ (excluding self-loops) a tuple $\sbf^{(e)} \in \R^{\abs{\vertices}}$ holding in its $t$'th entry what the value of $\walkinvert{1}{t}{ \{t\}}$ would be if $e$ is to be removed, for all $t \in \vertices$.
Notice that $\sbf^{(e)}$ and $\sbf$ agree on all entries except for $i, j \in e$, since removing $e$ from the graph only affects the degrees of $i$ and $j$.
Specifically, for $i \in e$, either $\sbf^{(e)}_i = \sbf_i - 1 = \abs{\neigh (i)} - 1$ if the removal of $e$ did not isolate $i$, or $\sbf^{(e)}_i = \sbf_i - 2 = 0$ if it did (due to self-loops, if a vertex has a single edge to another then $\abs{ \neigh (i) } = 2$, so removing that edge changes $\walkinvert{1}{i}{\{i\}}$ from two to zero).
As a result, for any $e = \brk[c]{i, j} , e' = \brk[c]{i', j'} \in \edges$, after sorting the entries of $\sbf^{(e)}$ and $\sbf^{(e')}$ in ascending order we have that $\sbf^{(e')}$ is greater in lexicographic order than $\sbf^{(e)}$ if and only if the pair $\brk{ \min \brk[c]{ \abs{\neigh (i') }, \abs{ \neigh (j') } } , \max \brk[c]{ \abs{\neigh (i') }, \abs{ \neigh (j') } }}$ is greater in lexicographic order than $\brk{ \min \brk[c]{ \abs{\neigh (i) }, \abs{ \neigh (j) } } , \max \brk[c]{ \abs{\neigh (i) }, \abs{ \neigh (j) } }}$.
Therefore, at every iteration $n \in [N]$~\cref{alg:walk_index_sparsification_vertex_one} and $1$-WIS (\cref{alg:walk_index_sparsification_vertex} with $L = 2$) remove the same edge.

\section{Further Experiments and Implementation Details}
\label{gnn:app:experiments}

\subsection{Further Experiments}
\label{gnn:app:experiments:further}

\begin{figure*}[t!]
	\begin{center}
		\hspace{-2mm}
		\includegraphics[width=1\textwidth]{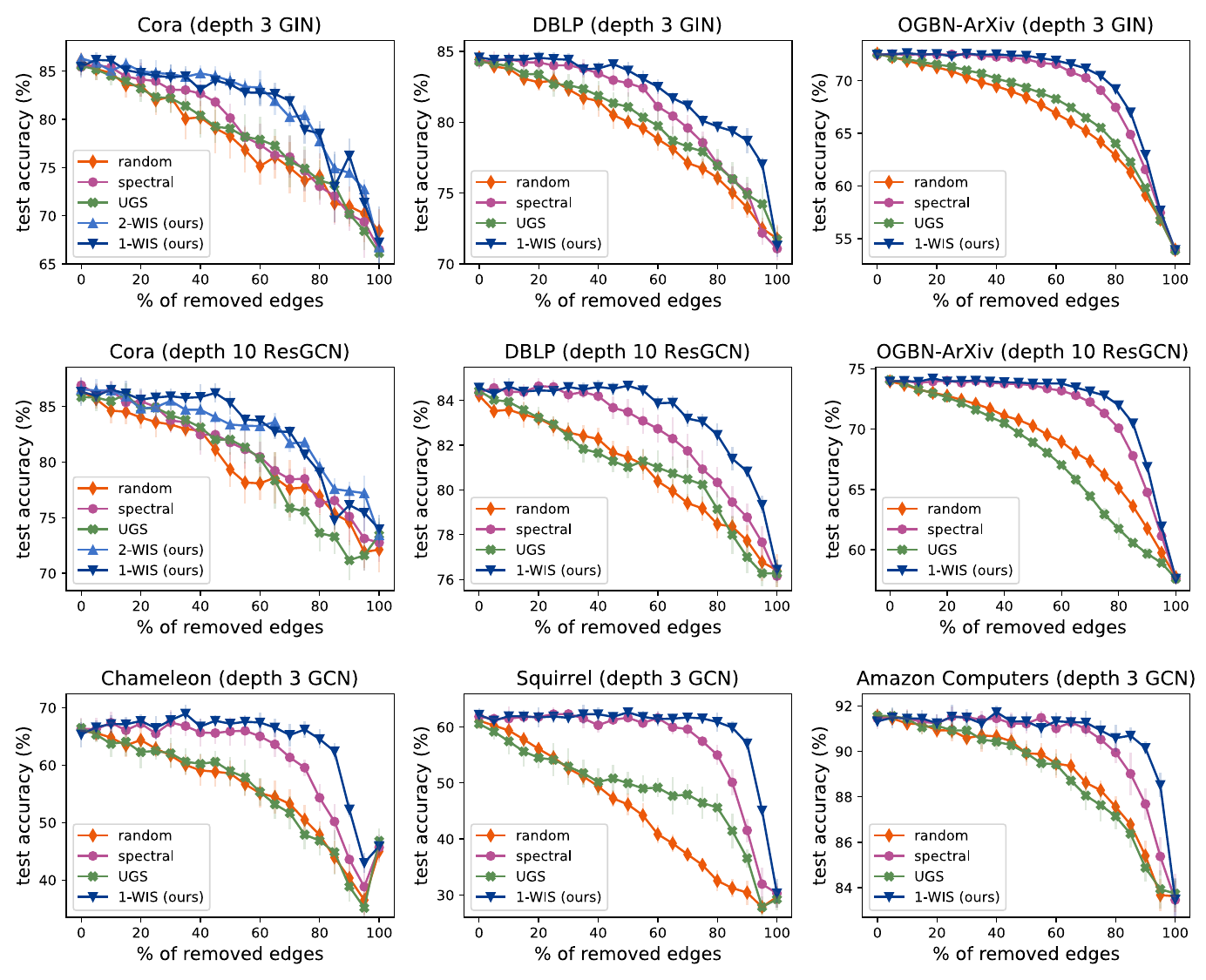}
	\end{center}
	\vspace{-2mm}
	\caption{
		Comparison of GNN accuracies following sparsification of input edges~---~WIS, the edge sparsification algorithm brought forth by our theory (\cref{alg:walk_index_sparsification_vertex}), markedly outperforms alternative methods.
		This figure supplements~\cref{gnn:fig:sparse_gcn} from~\cref{gnn:sec:sparsification:eval} by including experiments with: \emph{(i)} a depth $L = 3$ GIN over the Cora, DBLP, and OGBN-ArXiv datasets; \emph{(ii)} a depth $L = 10$ ResGCN over the Cora, DBLP, and OGBN-ArXiv datasets; and \emph{(iii)} a depth $L = 3$ GCN over the Chameleon, Squirrel, and Amazon Computers datasets.
		Markers and error bars report means and standard deviations, respectively, taken over ten runs per configuration for GCN and GIN, and over five runs per configuration for ResGCN (we use fewer runs due to the larger size of ResGCN).
		For further details see caption of~\cref{gnn:fig:sparse_gcn} as well as~\cref{gnn:app:experiments:details}.
	}
	\label{gnn:fig:sparse_gin}
\end{figure*}

\cref{gnn:fig:sparse_gin} supplements~\cref{gnn:fig:sparse_gcn} from~\cref{gnn:sec:sparsification:eval} by including experiments with additional: \emph{(i)} GNN architectures~---~GIN and ResGCN; and \emph{(ii)} datasets~---~Chameleon, Squirrel, and Amazon Computers.

\subsection{Further Implementation Details}
\label{gnn:app:experiments:details}

We provide implementation details omitted from our experimental reports (\cref{gnn:sec:analysis:experiments},~\cref{gnn:sec:sparsification}, and~\cref{gnn:app:experiments:further}).
Source code for reproducing our results and figures, based on the PyTorch~\cite{paszke2019pytorch} and PyTorch Geometric~\cite{fey2019fast} frameworks, can be found at \url{https://github.com/noamrazin/gnn_interactions}.
All experiments were run either on a single Nvidia RTX 2080 Ti GPU or a single Nvidia RTX A6000 GPU.

\subsubsection{Empirical Demonstration of Theoretical Analysis (\cref{tab:low_vs_high_walkindex})}
\label{gnn:app:experiments:details:analysis}

\textbf{Models.}~~All models used, \ie~GCN, GAT, and GIN, had three layers of width~$16$ with ReLU non-linearity.
To ease optimization, we added layer normalization~\cite{ba2016layer} after each one.
Mean aggregation and a linear output layer were applied over the last hidden embeddings for prediction.
As in the synthetic experiments of~\cite{alon2021bottleneck}, each GAT layer consisted of four attention heads.
Each GIN layer had its $\epsilon$ parameter fixed to zero and contained a two-layer feed-forward network, whose layers comprised a linear layer, batch normalization~\cite{ioffe2015batch}, and ReLU non-linearity.

\textbf{Data.}~~The datasets consisted of $10000$ train and $2000$ test graphs.
For every graph, we drew uniformly at random a label from $\{0, 1\}$ and an image from Fashion-MNIST.
Then, depending on the chosen label, another image was sampled either from the same class (for label $1$) or from all other classes (for label $0$).
We extracted patches of pixels from each image by flattening it into a vector and splitting the vector to $16$ equally sized segments.

\textbf{Optimization.}~~The binary cross-entropy loss was minimized via the Adam optimizer~\cite{kingma2015adam} with default $\beta_1, \beta_2$ coefficients and full-batches (\ie~every batch contained the whole training set).
Optimization proceeded until the train accuracy did not improve by at least $0.01$ over $1000$ consecutive epochs or $10000$ epochs elapsed.
The learning rates used for GCN, GAT, and GIN were $5 \cdot 10^{-3}, 5 \cdot 10^{-3},$ and $10^{-2}$, respectively.

\textbf{Hyperparameter tuning.}~~For each model separately, to tune the learning rate we carried out five runs (differing in random seed) with every value in the range $\{ 10^{-1}, 5 \cdot 10^{-2}, 10^{-2}, 5 \cdot 10^{-3}, 10^{-3} \}$ over the dataset whose essential partition has low walk index.
Since our interest resides in expressiveness, which manifests in ability to fit the training set, for every model we chose the learning rate that led to the highest mean train accuracy.

\subsubsection{Edge Sparsification (\cref{gnn:fig:sparse_gcn,gnn:fig:sparse_gin})}
\label{gnn:app:experiments:details:sparsification}

\textbf{Adaptations to UGS~\cite{chen2021unified}.}~~\cite{chen2021unified} proposed UGS as a framework for jointly pruning input graph edges and weights of a GNN.
At a high-level, UGS trains two differentiable masks, $\boldsymbol{m}_g$ and $\boldsymbol{m}_\theta$, that are multiplied with the graph adjacency matrix and the GNN's weights, respectively.
Then, after a certain number of optimization steps, a predefined percentage $p_g$ of graph edges are removed according to the magnitudes of entries in~$\boldsymbol{m}_g$, and similarly, $p_\theta$ percent of the GNN's weights are fixed to zero according to the magnitudes of entries in~$\boldsymbol{m}_\theta$.
This procedure continues in iterations, where each time the remaining GNN weights are rewinded to their initial values, until the desired sparsity levels are attained~---~see Algorithms~1 and~2 in~\cite{chen2021unified}.
To facilitate a fair comparison of our $(L - 1)$-WIS edge sparsification algorithm with UGS, we make the following adaptations to UGS.
\begin{itemize}[leftmargin=2em]
	\vspace{-2mm}
	\item We adapt UGS to only remove edges, which is equivalent to fixing the entries in the weight mask $\boldsymbol{m}_\theta$ to one and setting $p_\theta = 0$ in Algorithm~1 of~\cite{chen2021unified}.
	
	\item For comparing performance across a wider range of sparsity levels, the number of edges removed at each iteration is changed from $5\%$ of the current number of edges to $5\%$ of the original number of edges.
	
	\item Since our evaluation focuses on undirected graphs, we enforce the adjacency matrix mask $\boldsymbol{m}_g$ to be symmetric.
\end{itemize}

\textbf{Spectral sparsification~\cite{spielman2011graph}.}~~For Cora and DBLP, we used a Python implementation of the spectral sparsification algorithm from~\cite{spielman2011graph}, based on the PyGSP library implementation.\footnote{
	See \url{https://github.com/epfl-lts2/pygsp/}.
}
To enable more efficient experimentation over the larger scale OGBN-ArXiv dataset, we used a Julia implementation based on that from the Laplacians library.\footnote{
	See \url{https://github.com/danspielman/Laplacians.jl}.
}

\textbf{Models.}~~The GCN and GIN models had three layers of width~$64$ with ReLU non-linearity.
As in the experiments of~\cref{gnn:sec:analysis:experiments}, we added layer normalization~\cite{ba2016layer} after each one.
Every GIN layer had a trainable $\epsilon$ parameter and contained a two-layer feed-forward network, whose layers comprised a linear layer, batch normalization~\cite{ioffe2015batch}, and ReLU non-linearity.
For ResGCN, we used the implementation from~\cite{chen2021unified} with ten layers of width~$64$.
In all models, a linear output layer was applied over the last hidden embeddings for prediction.

\textbf{Data.}~~All datasets in our evaluation are multi-class vertex prediction tasks, each consisting of a single graph.
In Cora, DBLP, and OGBN-ArXiv, vertices represent scientific publications and edges stand for citation links.
In Chameleon and Squirrel, vertices represent web pages on Wikipedia and edges stand for mutual links between pages.
In Amazon Computers, vertices represent products and edges indicate that two products are frequently bought together.
For simplicity, we treat all graphs as undirected.
\cref{tab:dataset_sizes}~reports the number of vertices and undirected edges in each dataset.
For all datasets, except OGBN-ArXiv, we randomly split the labels of vertices into train, validation, and test sets comprising $80\%, 10\%,$ and $10\%$ of all labels, respectively.
For OGBN-ArXiv, we used the default split from~\cite{hu2020open}.

\begin{table}[H]
	\caption{
		Graph size of each dataset used for comparing edge sparsification algorithms in~\cref{gnn:fig:sparse_gcn,gnn:fig:sparse_gin}.
	}
	\label{tab:dataset_sizes}
	\vspace{-2mm}
	\begin{center}
		\fontsize{8}{8}\selectfont
			\begin{tabular}{lcc}
				\toprule
				& \# of Vertices & \# of Undirected Edges \\
				\midrule
				Cora & $2,\!708$ & $5,\!278$ \\
				DBLP & $17,\!716$ & $52,\!867$ \\
				OGBN-ArXiv & $169,\!343$ & $1,\!157,\!799$ \\
				Chameleon & $2,\!277$ & $31,\!396$ \\
				Squirrel & $5,\!201$ & $198,\!423$ \\
				Amazon Computers & $13,\!381$ & $245,\!861$ \\
				\bottomrule
			\end{tabular}
	\end{center}
\end{table}

\textbf{Optimization.}~~The cross-entropy loss was minimized via the Adam optimizer~\cite{kingma2015adam} with default $\beta_1, \beta_2$ coefficients and full-batches (\ie~every batch contained the whole training set).
Optimization proceeded until the validation accuracy did not improve by at least $0.01$ over $1000$ consecutive epochs or $10000$ epochs elapsed.
The test accuracies reported in~\cref{gnn:fig:sparse_gcn} are those achieved during the epochs with highest validation accuracies. 
\cref{tab:edge_sparse_hyperparam} specifies additional optimization hyperparameters.

\begin{table}[t]
	\caption{
		Optimization hyperparameters used in the experiments of~\cref{gnn:fig:sparse_gcn,gnn:fig:sparse_gin} per model and dataset.
	}
	\label{tab:edge_sparse_hyperparam}
	\vspace{-2mm}
	\begin{center}
		\fontsize{8}{8}\selectfont
			\begin{tabular}{llccc}
				\toprule
				& & Learning Rate & Weight Decay & Edge Mask $\ell_1$ Regularization of UGS \\
				\midrule
				\multirow{6}{*}{GCN}& Cora & $5 \cdot 10^{-4}$ & $10^{-3}$ & $10^{-2}$ \\
				& DBLP & $10^{-3}$ & $10^{-4}$ & $10^{-2}$ \\
				& OGBN-ArXiv & $10^{-3}$ & $0$ & $10^{-2}$ \\
				& Chameleon & $10^{-3}$ & $10^{-4}$  & $10^{-2}$ \\
				& Squirrel & $5 \cdot 10^{-4}$ & $0$  & $10^{-4}$ \\
				& Amazon Computers & $10^{-3}$ & $10^{-4}$  & $10^{-2}$ \\
				\midrule
				\multirow{3}{*}{GIN}& Cora & $10^{-3}$ & $10^{-3}$ & $10^{-2}$ \\
				& DBLP & $10^{-3}$ & $10^{-3}$ & $10^{-2}$ \\
				& OGBN-ArXiv & $10^{-4}$ & $0$ & $10^{-2}$ \\		
				\midrule
				\multirow{3}{*}{ResGCN}& Cora & $5 \cdot 10^{-4}$ & $10^{-3}$  & $10^{-4}$  \\
				& DBLP & $5 \cdot 10^{-4}$ & $10^{-4}$  & $10^{-4}$ \\
				& OGBN-ArXiv & $10^{-3}$ & $0$ & $10^{-2}$ \\
				\bottomrule
			\end{tabular}
	\end{center}
\end{table}

\textbf{Hyperparameter tuning.}~~For each combination of model and dataset separately, we tuned the learning rate, weight decay coefficient, and edge mask $\ell_1$ regularization coefficient for UGS, and applied the chosen values for evaluating all methods without further tuning (note that the edge mask $\ell_1$ regularization coefficient is relevant only for UGS).
In particular, we carried out a grid search over learning rates $\{ 10^{-3}, 5 \cdot 10^{-4}, 10^{-4} \}$, weight decay coefficients $\{ 10^{-3}, 10^{-4}, 0\}$, and edge mask $\ell_1$ regularization coefficients $\{ 10^{-2}, 10^{-3}, 10^{-4} \}$.
Per hyperparameter configuration, we ran ten repetitions of UGS (differing in random seed), each until all of the input graph's edges were removed.
At every edge sparsity level ($0\%, 5\%, 10\%, \ldots, 100\%$), in accordance with~\cite{chen2021unified}, we trained a new model with identical hyperparameters, but a fixed edge mask, over each of the ten graphs.
We chose the hyperparameters that led to the highest mean validation accuracy, taken over the sparsity levels and ten runs.

Due to the size of the ResGCN model, tuning its hyperparameters entails significant computational costs.
Thus, over the Cora and DBLP datasets, per hyperparameter configuration we ran five repetitions of UGS with ResGCN instead of ten.
For the large-scale OGBN-ArXiv dataset, we adopted the same hyperparameters used for GCN.

\textbf{Other.}~~To allow more efficient experimentation, we compute the edge removal order of $2$-WIS (\cref{alg:walk_index_sparsification_vertex}) in batches of size $100$.
Specifically, at each iteration of $2$-WIS, instead of removing the edge $e'$ with maximal walk index tuple $\sbf^{(e')}$, the $100$ edges with largest walk index tuples are removed.
For randomized edge sparsification algorithms~---~random pruning, the spectral sparsification method of~\cite{spielman2011graph}, and the adaptation of UGS~\cite{chen2021unified}~---~the evaluation runs for a given dataset and percentage of removed edges were carried over sparsified graphs obtained using different random seeds.

\section{Deferred Proofs}
\label{gnn:app:proofs}

\subsection{Additional Notation}
\label{gnn:app:proofs:notation}

For vectors, matrices, or tensors, parenthesized superscripts denote elements in a collection, \eg~$\brk{\abf^{(i)} \in \R^D }_{n = 1}^N$, while subscripts refer to entries, \eg~$\Abf_{d_1, d_2} \in \R$ is the $(d_1, d_2)$'th entry of $\Abf \in \R^{D_1 \times D_2}$.
A colon is used to indicate a range of entries, \eg~$\abf_{:d}$ is the first $d$ entries of $\abf \in \R^D$.
We use $\contractlast$ to denote tensor contractions (\cref{gnn:def:tensor_contraction}),~$\kronp$ to denote the Kronecker product, and~$\hadmp$ to denote the Hadamard product.
For $P \in \N_{\geq 0}$ , the $P$'th Hadamard power operator is denoted by $\hadmp^P$, \ie~$\brk[s]{ \hadmp^P \Abf }_{d_1, d_2} = \Abf_{d_1, d_2}^P$ for $\Abf \in \R^{D_1 \times D_2}$.
Lastly, when enumerating over sets of indices an ascending order is assumed.

\subsection{Proof of~\cref{gnn:thm:sep_rank_upper_bound}}
\label{gnn:app:proofs:sep_rank_upper_bound}

We assume familiarity with the basic concepts from tensor analysis introduced in~\cref{gnn:app:prod_gnn_as_tn:tensors}, and rely on the tensor network representations established for GNNs with product aggregation in~\cref{gnn:app:prod_gnn_as_tn}.
Specifically, we use the fact that for any $\fmat = \brk{\fvec{1}, \ldots, \fvec{\abs{\vertices}} } \in \R^{\indim \times \abs{\vertices} }$ there exist tree tensor networks $\tngraph (\fmat)$ and $\tngraph^{(t)} (\fmat)$ (described in~\cref{gnn:app:prod_gnn_as_tn:correspondence} and formally defined in~\cref{gnn:eq:graphtensor_tn,gnn:eq:vertextensor_tn}) such that: \emph{(i)}~their contraction yields $\funcgraph{\params}{\graph} (\fmat)$ and $\funcvert{\params}{\graph}{t} (\fmat)$, respectively (\cref{gnn:prop:tn_constructions}); and \emph{(ii)}~each of their leaves is associated with a vertex feature vector, \ie~one of $\fvec{1}, \ldots, \fvec{ \abs{\vertices} }$, whereas all other aspects of the tensor networks do not depend on $\fvec{1}, \ldots, \fvec{ \abs{\vertices} }$.

The proof proceeds as follows.
In Appendix~\ref{gnn:app:proofs:sep_rank_upper_bound:min_cut_tn_ub}, by importing machinery from tensor analysis literature (in particular, adapting Claim~7 from~\cite{levine2018deep}), we show that the separation ranks of $\funcgraph{\params}{\graph}$ and $\funcvert{\params}{\graph}{t}$ can be upper bounded via cuts in their corresponding tensor networks.
Namely, $\sepranknoflex{\funcgraph{\params}{\graph} }{\I}$ is at most the minimal multiplicative cut weight in $\tngraph (\fmat)$, among cuts separating leaves associated with vertices of the input graph in $\I$ from leaves associated with vertices of the input graph in $\I^c$, where multiplicative cut weight refers to the product of weights belonging to legs crossing the cut.
Similarly, $\sepranknoflex{ \funcvert{\params}{\graph}{t} }{\I}$ is at most the minimal multiplicative cut weight in $\tngraph^{(t)} (\fmat)$, among cuts of the same form.
We conclude in Appendices~\ref{gnn:app:proofs:sep_rank_upper_bound:graph} and~\ref{gnn:app:proofs:sep_rank_upper_bound:vertex} by applying this technique for upper bounding $\sepranknoflex{ \funcgraph{\params}{\graph} }{\I}$ and $\sepranknoflex{ \funcvert{\params}{\graph}{t} }{ \I}$, respectively, \ie~by finding cuts in the respective tensor networks with sufficiently low multiplicative weights.

\subsubsection{Upper Bounding Separation Rank via Multiplicative Cut Weight in Tensor Network}
\label{gnn:app:proofs:sep_rank_upper_bound:min_cut_tn_ub}

In a tensor network $\tngraph = \brk{ \tnvertices{\tngraph}, \tnedges{\tngraph}, \tnedgeweights{\tngraph} }$, every $\J_{\tngraph} \subseteq \tnvertices{\tngraph}$ induces a \emph{cut} $(\J_{\tngraph} , \J_{\tngraph}^c )$, \ie~a partition of the nodes into two sets.
We denote by 
\[
\tnedges{\tngraph} (\J_\tngraph) := \brk[c]{ \brk[c]{ u, v} \in \tnedges{\tngraph} : u \in \J_\tngraph , v \in \J_\tngraph^c }
\]
the set of legs crossing the cut, and define the \emph{multiplicative cut weight} of $\J_{\tngraph}$ to be the product of weights belonging to legs in $\tnedges{\tngraph} (\J_\tngraph)$,~\ie:
\[
\mulcutweight{\tngraph} ( \J_{\tngraph} ) := \prod\nolimits_{ e \in \tnedges{\tngraph} (\J_{\tngraph}) } \tnedgeweights{\tngraph} (e)
\text{\,.}
\]
For $\fmat = \brk{\fvec{1}, \ldots, \fvec{\abs{\vertices}} } \in \R^{\indim \times \abs{\vertices} }$, let $\tngraph (\fmat)$ and $\tngraph^{(t)} (\fmat)$ be the tensor networks corresponding to $\funcgraph{\params}{\graph} (\fmat)$ and $\funcvert{\params}{\graph}{t} (\fmat)$ (detailed in~\cref{gnn:app:prod_gnn_as_tn:correspondence}), respectively.
Both $\tngraph (\fmat)$ and $\tngraph^{(t)} (\fmat)$ adhere to a tree structure.
Each leaf node is associated with a vertex feature vector (\ie~one of $\fvec{1}, \ldots, \fvec{\abs{\vertices}}$), while interior nodes are associated with weight matrices or $\delta$-tensors.
The latter are tensors with modes of equal dimension holding ones on their hyper-diagonal and zeros elsewhere.
The restrictions imposed by $\delta$-tensors induce a modified notion of multiplicative cut weight, where legs incident to the same $\delta$-tensor only contribute once to the weight product (note that weights of legs connected to the same $\delta$-tensor are equal since they stand for mode dimensions).

\begin{definition}
	\label{gnn:def:modified_mul_cut_weight}
	For a tensor network $\tngraph = \brk{ \tnvertices{\tngraph}, \tnedges{\tngraph}, \tnedgeweights{\tngraph} }$ and subset of nodes $\J_{\tngraph} \subseteq \tnvertices{\tngraph}$, let $\tnedges{\tngraph} (\J_{\tngraph})$ be the set of edges crossing the cut $(\J_{\tngraph}, \J_{\tngraph}^c)$.
	Denote by $\modcutedges{\tngraph} (\J_{\tngraph}) \subseteq \tnedges{\tngraph} (\J_{\tngraph})$ a subset of legs containing for each $\delta$-tensor in $\tnvertices{\tngraph}$ only a single leg from $\tnedges{\tngraph} (\J_{\tngraph})$ incident to it, along with all legs in $\tnedges{\tngraph} (\J_{\tngraph}) $ not connected to $\delta$-tensors.
	Then, the \emph{modified multiplicative cut weight} of $\J_{\tngraph}$ is:
	\[
	\modcutweight{\tngraph} (\J_{\tngraph}) := \prod\nolimits_{ e \in \modcutedges{\tngraph} (\J_{\tngraph}) } \tnedgeweights{\tngraph} (e)
	\text{\,.}
	\]
\end{definition}

\cref{gnn:lem:min_cut_tn_sep_rank_ub} establishes that $\sepranknoflex{ \funcgraph{\params}{\graph} }{\I}$ and $\sepranknoflex{ \funcvert{\params}{\graph}{t} }{ \I}$ are upper bounded by the minimal modified multiplicative cut weights in $\tngraph (\fmat)$ and $\tngraph^{(t)} (\fmat)$, respectively, among cuts separating leaves associated with vertices in $\I$ from leaves associated vertices in $\I^c$.

\begin{lemma}
	\label{gnn:lem:min_cut_tn_sep_rank_ub}
	For any $\fmat = \brk{\fvec{1}, \ldots, \fvec{\abs{\vertices}} } \in \R^{\indim \times \abs{\vertices} }$, let $\tngraph (\fmat) = \brk{ \tnvertices{\tngraph (\fmat)}, \tnedges{\tngraph (\fmat)}, \tnedgeweights{\tngraph (\fmat)} }$ and $\tngraph^{(t)} (\fmat) = \brk{ \tnvertices{\tngraph^{(t)} (\fmat)}, \tnedges{\tngraph^{(t)} (\fmat)}, \tnedgeweights{\tngraph^{(t)} (\fmat)} }$ be the tensor network representations of $\funcgraph{\params}{\graph} (\fmat)$ and $\funcvert{\params}{\graph}{t} (\fmat)$ (described in~\cref{gnn:app:prod_gnn_as_tn:correspondence} and formally defined in~\cref{gnn:eq:graphtensor_tn,gnn:eq:vertextensor_tn}), respectively.
	Denote by $\tnverticesleaf{ \tngraph (\fmat) }{\I} \subseteq \tnvertices{ \tngraph (\fmat) }$ and $\tnverticesleaf{ \tngraph^{(t)} (\fmat) }{\I} \subseteq \tnvertices{ \tngraph^{(t)} (\fmat) }$ the sets of leaf nodes in $\tngraph (\fmat)$ and $\tngraph^{(t)} (\fmat)$, respectively, associated with vertices in $\I$ from the input graph $\graph$.
	Formally:
	\[
	\begin{split}
		\tnverticesleaf{ \tngraph (\fmat) }{\I} & := \brk[c]1{ \fvecdup{i}{\gamma} \in \tnvertices{ \tngraph (\fmat) } : i \in \I , \gamma \in [ \nwalk{L}{ \{i\} }{ \vertices} ] } \text{\,,} \\[0.5em]
		\tnverticesleaf{ \tngraph^{(t)} (\fmat) }{\I} & := \brk[c]1{ \fvecdup{i}{\gamma} \in \tnvertices{ \tngraph^{(t)} (\fmat) } : i \in \I , \gamma \in [ \nwalk{L}{ \{i\} }{ \{t\}} ] } \text{\,.}
	\end{split}
	\]
	Similarly, denote by $\tnverticesleaf{ \tngraph (\fmat) }{\I^c} \subseteq \tnvertices{ \tngraph (\fmat) }$ and $\tnverticesleaf{ \tngraph^{(t)} (\fmat) }{\I^c} \subseteq \tnvertices{ \tngraph^{(t)} (\fmat) }$ the sets of leaf nodes in $\tngraph (\fmat)$ and $\tngraph^{(t)} (\fmat)$, respectively, associated with vertices in $\I^c$.
	Then, the following hold:
	\begin{align}
		\text{(graph prediction)} \quad & \seprankbig{ \funcgraph{\params}{\graph} }{\I}
		\leq \!\!\!\!\!\!
		\min_{ \substack{ \J_{ \tngraph (\fmat) } \subseteq \tnvertices{ \tngraph (\fmat) } \\[0.2em] \text{s.t. } \tnverticesleaf{ \tngraph (\fmat) }{\I} \subseteq \J_{ \tngraph (\fmat) } \text{ and } \tnverticesleaf{ \tngraph (\fmat) }{\I^c} \subseteq \J^c_{\tngraph (\fmat)} } } \!\!\!\!\!\!\!\!\!\!\!\!\!\!\! \modcutweight{\tngraph (\fmat)} ( \J_{ \tngraph (\fmat) } )
		\text{\,,}
		\label{gnn:eq:sep_rank_min_cut_bound_graph} \\[0.5em]
		\text{(vertex prediction)} \quad & \seprankbig{ \funcvert{\params}{\graph}{ t } }{\I} 
		\leq \!\!\!\!\!\!
		\min_{ \substack{ \J_{ \tngraph^{(t)} (\fmat) } \subseteq \tnvertices{ \tngraph^{(t)} (\fmat) } \\[0.2em] \text{s.t. } \tnverticesleaf{ \tngraph^{(t)} (\fmat) }{\I} \subseteq \J_{ \tngraph^{(t)} (\fmat) } \text{ and } \tnverticesleaf{ \tngraph^{(t)} (\fmat) }{\I^c} \subseteq \J^c_{\tngraph^{(t)} (\fmat)} } } \!\!\!\!\!\!\!\!\!\!\!\!\!\!\! \modcutweight{\tngraph^{(t)} (\fmat)} ( \J_{ \tngraph^{(t)} (\fmat) } )	
		\text{\,,}
		\label{gnn:eq:sep_rank_min_cut_bound_vertex}
	\end{align}
	where $\modcutweight{\tngraph (\fmat)} ( \J_{ \tngraph (\fmat) } )$ is the modified multiplicative cut weight of $\J_{ \tngraph (\fmat) }$ in $\tngraph (\fmat)$ and $\modcutweight{\tngraph^{(t)} (\fmat)} ( \J_{ \tngraph^{(t)} (\fmat) } )$ is the modified multiplicative cut weight of $\J_{ \tngraph^{(t)} (\fmat) }$ in $\tngraph^{(t)} (\fmat)$ (\cref{gnn:def:modified_mul_cut_weight}).
\end{lemma}
\begin{proof}
	We first prove~\cref{gnn:eq:sep_rank_min_cut_bound_graph}.
	Examining $\tngraph (\fmat)$, notice that: \emph{(i)}~by~\cref{gnn:prop:tn_constructions} its contraction yields $\funcgraph{\params}{\graph} (\fmat)$;~\emph{(ii)} it has a tree structure; and~\emph{(iii)}~each of its leaves is associated with a vertex feature vector, \ie~one of $\fvec{1}, \ldots, \fvec{ \abs{\vertices} }$, whereas all other aspects of the tensor network do not depend on $\fvec{1}, \ldots, \fvec{ \abs{\vertices} }$.
	Specifically, for any $\fmat$ and $\fmat'$ the nodes, legs, and leg weights of $\tngraph (\fmat)$ and $\tngraph (\fmat')$ are identical, up to the assignment of features in the leaf nodes.
	Let $\tensorendgraph \in \R^{ \indim \times \cdots \times \indim }$ be the order $\nwalk{L}{ \vertices }{ \vertices }$ tensor obtained by contracting all interior nodes in $\tngraph (\fmat)$.
	The above implies that we may write $\funcgraph{\params}{\graph} (\fmat)$ as a contraction of $\tensorendgraph$ with $\fvec{1}, \ldots, \fvec{ \abs{ \vertices } }$.
	Specifically, it holds that:
	\be
	\funcgraph{\params}{\graph} (\fmat) = \tensorendgraph \contract{n \in [\nwalk{L}{\vertices}{\vertices}]} \fvec{ \modemapgraph (n) }
	\text{\,,}
	\label{gnn:eq:funcgraph_as_contraction_sep_rank}
	\ee
	for any $\fmat = \brk{\fvec{1}, \ldots, \fvec{\abs{\vertices}} } \in \R^{\indim \times \abs{\vertices} }$, where $\modemapgraph : [ \nwalk{L}{\vertices}{\vertices} ] \to \vertices$ maps a mode index of $\tensorendgraph$ to the appropriate vertex of $\graph$ according to $\tngraph (\fmat)$.
	Let $\mu^{-1} ( \I ) := \brk[c]{ n \in [ \nwalk{L}{\vertices}{\vertices} ] : \mu (n) \in \I }$ be the mode indices of $\tensorendgraph$ corresponding to vertices in~$\I$.
	Invoking~\cref{gnn:lem:contraction_matricization} leads to the following matricized form of~\cref{gnn:eq:funcgraph_as_contraction_sep_rank}:
	\[
	\funcgraph{\params}{\graph} (\fmat) = \brk1{ \kronp_{ n \in \mu^{-1} (\I) } \fvec{\mu (n)} }^\top \mat{ \tensorendgraph }{ \mu^{-1} (\I)} \brk1{ \kronp_{n \in \mu^{-1} ( \I^{c} ) } \fvec{ \mu (n) } }
	\text{\,,}
	\]
	where $\kronp$ denotes the Kronecker product.
	
	We claim that $\sepranknoflex{ \funcgraph{\params}{\graph} }{\I} \leq \rank \mat{ \tensorendgraph }{ \mu^{-1} (\I)}$.
	To see it is so, denote
	\[
	R := \rank \mat{ \tensorendgraph }{ \mu^{-1} (\I)}
	\]
	and let $\ubf^{(1)}, \ldots, \ubf^{(R)} \in \R^{ \indim^{ \nwalk{L}{\I}{\vertices} } }$ and $\bar{\ubf}^{(1)}, \ldots, \bar{\ubf}^{(R)} \in \R^{ \indim^{ \nwalk{L}{\I^c}{\vertices} } }$ be such that 
	\[
	\mat{ \tensorendgraph }{ \mu^{-1} (\I)} = \sum\nolimits_{r = 1}^R \ubf^{(r)} \brk{ \bar{\ubf}^{(r)} }^\top
	\text{\,.}
	\]
	Then, defining $g^{(r)} : ( \R^{\indim} )^{ \abs{\I } } \to \R$ and $\bar{g}^{(r)} : ( \R^{\indim} )^{ \abs{ \I^c } } \to \R$, for $r \in [R]$, as:
	\[
	g^{(r)} ( \fmat_{\I} ) := \inprod{ \kronp_{ n \in \mu^{-1} (\I) } \fvec{\mu (n)} }{ \ubf^{(r)} } \quad , \quad \bar{g}^{(r)} (\fmat_{ \I^c} ) := \inprod{ \kronp_{n \in \mu^{-1} ( \I^{c} ) } \fvec{ \mu (n) } }{ \bar{\ubf}^{(r)} }
	\text{\,,}
	\]
	where $\fmat_\I := \brk{ \fvec{i} }_{i \in \I}$ and $\fmat_{\I^c} := \brk{ \fvec{j} }_{j \in \I^c}$, we have that:
	\[
	\begin{split}
		\funcgraph{\params}{\graph} (\fmat) & = \brk1{ \kronp_{ n \in \mu^{-1} (\I) } \fvec{\mu (n)} }^\top \brk2{ \sum\nolimits_{r = 1}^R \ubf^{(r)} \brk1{ \bar{\ubf}^{(r)} }^\top } \brk1{ \kronp_{n \in \mu^{-1} ( \I^{c} ) } \fvec{ \mu (n) } } \\
		& = \sum\nolimits_{r = 1}^R \inprod{ \kronp_{ n \in \mu^{-1} (\I) } \fvec{\mu (n)} }{ \ubf^{(r)} } \cdot \inprod{ \kronp_{n \in \mu^{-1} ( \I^{c} ) } \fvec{ \mu (n) } }{ \bar{\ubf}^{(r)} } \\
		& = \sum\nolimits_{r = 1}^R g^{(r)} ( \fmat_{\I} ) \cdot \bar{g}^{(r)} (\fmat_{ \I^c} )
		\text{\,.}
	\end{split}
	\]
	Since $\sepranknoflex{ \funcgraph{\params}{\graph} }{\I}$ is the minimal number of summands in a representation of this form of $\funcgraph{\params}{\graph}$, indeed, $\sepranknoflex{ \funcgraph{\params}{\graph} }{\I} \leq R = \rank \mat{ \tensorendgraph }{ \mu^{-1} (\I)}$.
	
	What remains is to apply~Claim~7 from~\cite{levine2018deep}, which upper bounds the rank of a tensor's matricization with multiplicative cut weights in a tree tensor network.
	In particular, consider an order $N \in \N$ tensor $\Atensor$ produced by contracting a tree tensor network $\tngraph$.
	Then, for any $\K \subseteq [N]$ we have that $\rank \mat{\Atensor}{\K}$ is at most the minimal modified multiplicative cut weight in $\tngraph$, among cuts separating leaves corresponding to modes $\K$ from leaves corresponding to modes $\K^c$.
	Thus, invoking Claim~7 from~\cite{levine2018deep} establishes~\cref{gnn:eq:sep_rank_min_cut_bound_graph}:
	\[
	\seprankbig{ \funcgraph{\params}{\graph} }{\I}
	\leq
	\rank \mat{ \tensorendgraph }{ \mu^{-1} (\I)} \leq 
	\min_{ \substack{ \J_{ \tngraph (\fmat) } \subseteq \tnvertices{ \tngraph (\fmat) } \\[0.2em] \text{s.t. } \tnverticesleaf{ \tngraph (\fmat) }{\I} \subseteq \J_{ \tngraph (\fmat) } \text{ and } \tnverticesleaf{ \tngraph (\fmat) }{\I^c} \subseteq \J^c_{\tngraph (\fmat)} } } \!\!\!\!\!\!\!\!\!\!\!\!\!\!\! \modcutweight{\tngraph (\fmat)} ( \J_{ \tngraph (\fmat) } )
	\text{\,.}
	\]
	
	\medskip
	
	\cref{gnn:eq:sep_rank_min_cut_bound_vertex} readily follows by steps analogous to those used above for proving~\cref{gnn:eq:sep_rank_min_cut_bound_graph}.
\end{proof}

\subsubsection{Cut in Tensor Network for Graph Prediction (Proof of~\cref{gnn:eq:sep_rank_upper_bound_graph_pred})}
\label{gnn:app:proofs:sep_rank_upper_bound:graph}

For $\fmat = \brk{\fvec{1}, \ldots, \fvec{\abs{\vertices}} } \in \R^{\indim \times \abs{\vertices} }$, let $\tngraph (\fmat) = \brk{ \tnvertices{\tngraph (\fmat)}, \tnedges{\tngraph (\fmat)}, \tnedgeweights{\tngraph (\fmat)} }$ be the tensor network corresponding to $\funcgraph{\params}{\graph} (\fmat)$ (detailed in~\cref{gnn:app:prod_gnn_as_tn:correspondence} and formally defined in~\cref{gnn:eq:graphtensor_tn}).
By~\cref{gnn:lem:min_cut_tn_sep_rank_ub}, to prove that
\[
\seprankbig{ \funcgraph{\params}{\graph} }{\I} \leq \hdim^{4 \nwalk{L - 1}{\cut_\I}{\vertices} + 1 }
\text{\,,}
\]
it suffices to find $\J_{ \tngraph (\fmat) } \subseteq \tnvertices{\tngraph (\fmat) }$ satisfying: \emph{(i)}~leaves of $\tngraph (\fmat)$ associated with vertices in~$\I$ are in $\J_{ \tngraph (\fmat) }$, whereas leaves associated with vertices in~$\I^c$ are not in $\J_{ \tngraph (\fmat) }$; and \emph{(ii)}~$\modcutweight{\tngraph (\fmat)} ( \J_{ \tngraph (\fmat) } ) \leq \hdim^{4 \nwalk{L - 1}{\cut_\I}{\vertices} + 1}$, where $\modcutweight{\tngraph (\fmat)} ( \J_{ \tngraph (\fmat) } )$ is the modified multiplicative cut weight of $\J_{ \tngraph (\fmat) }$ (\cref{gnn:def:modified_mul_cut_weight}).
To this end, define $\J_{ \tngraph (\fmat) }$ to hold all nodes in $\tnvertices{\tngraph (\fmat) }$ corresponding to vertices in~$\I$.
Formally:
\[
\begin{split}
	\J_{ \tngraph (\fmat) } := &~\brk[c]2{ \fvecdup{i}{\gamma} : i \in \I, \gamma \in [ \nwalk{L}{ \{i\} }{\vertices} ] } \cup \\
	&~\brk[c]2{ \weightmatdup{l}{i}{\gamma} : l \in [L], i \in \I, \gamma \in [ \nwalk{L - l + 1}{ \{i\} }{\vertices} ] } \cup \\
	&~\brk[c]2{ \deltatensordup{l}{i}{\gamma} : l \in [L], i \in \I, \gamma \in [ \nwalk{L - l}{ \{i\} }{ \vertices } ] }
	\text{\,.}
\end{split}
\]
Clearly, $\J_{ \tngraph (\fmat) }$ upholds \emph{(i)}.

As for \emph{(ii)}, there are two types of legs crossing the cut induced by $\J_{ \tngraph (\fmat) }$ in $\tngraph (\fmat)$.
First, are those connecting a $\delta$-tensor with a weight matrix in the same layer, where one is associated with a vertex in~$\I$ and the other with a vertex in~$\I^c$.
That is, legs connecting $\deltatensordup{l}{i}{\gamma}$ with $\weightmatdup{l}{ \neigh (i)_j }{ \dupmap{l}{i}{j} (\gamma) }$, where $i \in \vertices$ and $\neigh (i)_j \in \vertices$ are on different sides of the partition $(\I, \I^c)$ in the input graph, for $j \in [\abs{ \neigh (i) } ], l \in [L] , \gamma \in [ \nwalk{L - l}{ \{i\}}{\vertices} ]$.
The $\delta$-tensors participating in these legs are exactly those associated with some $i \in \cut_\I$ (recall $\cut_\I$ is the set of vertices with an edge crossing the partition $(\I, \I^c)$).
So, for every $l \in [L]$ and $i \in \cut_\I$ there are $\nwalk{L - l}{\{i\}}{\vertices}$ such $\delta$-tensors.
Second, are legs from $\delta$-tensors associated with $i \in \I$ in the $L$'th layer to the $\delta$-tensor in the output layer of $\tngraph (\fmat)$.
That is, legs connecting $\deltatensordup{L}{i}{1}$ with $\deltatensor{ \abs{\vertices} + 1 }$, for $i \in \I$.
Legs incident to the same $\delta$-tensor only contribute once to $\modcutweight{\tngraph (\fmat)} ( \J_{ \tngraph (\fmat) } )$.
Thus, since the weights of all legs connected to $\delta$-tensors are equal to $\hdim$, we have that:
\[
\modcutweight{\tngraph (\fmat)} (\J_{\tngraph (\fmat)}) \leq \hdim^{ 1 + \sum\nolimits_{ l = 1 }^L \sum\nolimits_{i \in \cut_\I} \nwalk{L - l}{\{i\}}{\vertices} } = \hdim^{ 1 + \sum\nolimits_{ l = 1 }^L \nwalk{L - l}{ \cut_\I }{\vertices} }
\text{\,.}
\]
Lastly, it remains to show $\sum\nolimits_{ l = 1 }^L \nwalk{L - l}{ \cut_\I }{\vertices} \leq 4 \nwalk{L - 1}{ \cut_\I }{\vertices}$, since in that case~\cref{gnn:lem:min_cut_tn_sep_rank_ub} implies:
\[
\seprankbig{ \funcgraph{\params}{\graph} }{\I} \leq \modcutweight{\tngraph (\fmat)} (\J_{\tngraph (\fmat)}) \leq \hdim^{ 4 \nwalk{L - 1}{ \cut_\I }{\vertices} + 1 }
\text{\,,}
\]
which yields~\cref{gnn:eq:sep_rank_upper_bound_graph_pred} by taking the log of both sides.

The main idea is that, in an undirected graph with self-loops, the number of length $l \in \N$ walks from vertices with at least one neighbor decays exponentially when $l$ decreases.
Observe that $\nwalk{l}{ \cut_\I }{\vertices} \leq \nwalk{l + 1}{ \cut_\I }{\vertices}$ for all $l \in \N$.
Hence:
\be
\sum\nolimits_{ l = 1 }^L \nwalk{L - l}{ \cut_\I }{\vertices} \leq 2 \sum\nolimits_{l \in \{1, 3, \ldots, L - 1\} } \nwalk{L - l}{ \cut_\I }{\vertices}
\text{\,.}
\label{gnn:eq:nwalk_increasing_undirected_graph}
\ee
Furthermore, any length $l \in \N_{\geq 0}$ walk $i_0, i_1, \ldots, i_l \in \vertices$ from $\cut_\I$ induces at least two walks of length $l + 2$ from $\cut_\I$, distinct from those induced by other length $l$ walks~---~one which goes twice through the self-loop of $i_0$ and then proceeds according to the length $l$ walk, \ie~$i_0, i_0, i_0, i_1, \ldots, i_l$, and another that goes to a neighboring vertex (exists since $i_0 \in \cut_\I$), returns to $i_0$, and then proceeds according to the length $l$ walk.
This means that $\nwalk{L - l}{ \cut_\I}{\vertices} \leq 2^{-1} \cdot \nwalk{L - l + 2}{ \cut_\I}{\vertices} \leq \cdots \leq 2^{ - \lfloor l / 2 \rfloor } \cdot \nwalk{L -1 }{\cut_\I}{\vertices}$ for all $l \in \{ 3, 5, \ldots, L - 1\}$.
Going back to~\cref{gnn:eq:nwalk_increasing_undirected_graph}, this leads to:
\[
\begin{split}
	\sum\nolimits_{ l = 1 }^L \nwalk{L - l}{ \cut_\I }{\vertices}
	& \leq 
	2 \sum\nolimits_{l \in \{1, 3, \ldots, L - 1\} } 2^{ \lfloor l / 2 \rfloor } \cdot \nwalk{L - 1}{ \cut_\I }{\vertices} \\
	& \leq 
	2 \sum\nolimits_{l = 0}^{\infty} 2^{ - l} \cdot \nwalk{L - 1}{ \cut_\I }{\vertices} \\
	& = 4 \nwalk{L - 1}{ \cut_\I }{ \vertices }
	\text{\,,}
\end{split}
\]
completing the proof of~\cref{gnn:eq:sep_rank_upper_bound_graph_pred}.

\subsubsection{Cut in Tensor Network for Vertex Prediction (Proof of~\cref{gnn:eq:sep_rank_upper_bound_vertex_pred})}
\label{gnn:app:proofs:sep_rank_upper_bound:vertex}

This part of the proof follows a line similar to that of Appendix~\ref{gnn:app:proofs:sep_rank_upper_bound:graph}, with differences stemming from the distinction between the operation of a GNN over graph and vertex prediction tasks.

For $\fmat = \brk{\fvec{1}, \ldots, \fvec{\abs{\vertices}} } \in \R^{\indim \times \abs{\vertices} }$, let $\tngraph^{(t)} (\fmat) = \brk{ \tnvertices{\tngraph^{(t)} (\fmat)}, \tnedges{\tngraph^{(t)} (\fmat)}, \tnedgeweights{\tngraph^{(t)} (\fmat)} }$ be the tensor network corresponding to $\funcvert{\params}{\graph}{t} (\fmat)$ (detailed in~\cref{gnn:app:prod_gnn_as_tn:correspondence} and formally defined in~\cref{gnn:eq:vertextensor_tn}).
By~\cref{gnn:lem:min_cut_tn_sep_rank_ub}, to prove that
\[
\seprankbig{ \funcvert{\params}{\graph}{t} }{\I} \leq \hdim^{4 \nwalk{L - 1}{\cut_\I}{ \{t\} } }
\text{\,,}
\]
it suffices to find $\J_{ \tngraph^{(t)} (\fmat) } \subseteq \tnvertices{\tngraph^{(t)} (\fmat) }$ satisfying: \emph{(i)}~leaves of $\tngraph^{(t)} (\fmat)$ associated with vertices in~$\I$ are in $\J_{ \tngraph^{(t)} (\fmat) }$, whereas leaves associated with vertices in~$\I^c$ are not in $\J_{ \tngraph^{(t)} (\fmat) }$; and \emph{(ii)}~$\modcutweight{\tngraph^{(t)} (\fmat)} ( \J_{ \tngraph^{(t)} (\fmat) } ) \leq \hdim^{4 \nwalk{L - 1}{\cut_\I}{ \{t\}}}$, where $\modcutweight{\tngraph^{(t)} (\fmat)} ( \J_{ \tngraph^{(t)} (\fmat) } )$ is the modified multiplicative cut weight of $\J_{ \tngraph^{(t)} (\fmat) }$ (\cref{gnn:def:modified_mul_cut_weight}).
To this end, define $\J_{ \tngraph^{(t)} (\fmat) }$ to hold all nodes in $\tnvertices{\tngraph^{(t)} (\fmat) }$ corresponding to vertices in~$\I$.
Formally:
\[
\begin{split}
	\J_{ \tngraph^{(t)} (\fmat) } := &~\brk[c]2{ \fvecdup{i}{\gamma} : i \in \I, \gamma \in [ \nwalk{L}{ \{i\} }{ \{t\}} ] } \cup \\
	&~\brk[c]2{ \weightmatdup{l}{i}{\gamma} : l \in [L], i \in \I, \gamma \in [ \nwalk{L - l + 1}{ \{i\} }{\{t\}} ] } \cup \\
	&~\brk[c]2{ \deltatensordup{l}{i}{\gamma} : l \in [L], i \in \I, \gamma \in [ \nwalk{L - l}{ \{i\} }{ \{t\} } ] } \cup \\
	&~\WW^{(o)}
	\text{\,,}
\end{split}
\]
where $\WW^{(o)} := \brk[c]{ \weightmat{o} }$ if $t \in \I$ and $\WW^{(o)} := \emptyset$ otherwise.
Clearly, $\J_{ \tngraph^{(t)} (\fmat) }$ upholds \emph{(i)}.

As for \emph{(ii)}, the legs crossing the cut induced by $\J_{ \tngraph^{(t)} (\fmat) }$ in $\tngraph^{(t)} (\fmat)$ are those connecting a $\delta$-tensor with a weight matrix in the same layer, where one is associated with a vertex in~$\I$ and the other with a vertex in~$\I^c$.
That is, legs connecting $\deltatensordup{l}{i}{\gamma}$ with $\weightmatdup{l}{ \neigh (i)_j }{ \dupmapvertex{l}{i}{j} (\gamma) }$, where $i \in \vertices$ and $\neigh (i)_j \in \vertices$ are on different sides of the partition $(\I, \I^c)$ in the input graph, for $j \in [\abs{ \neigh (i) } ], l \in [L] , \gamma \in [ \nwalk{L - l}{ \{i\}}{ \{ t \} } ]$.
The $\delta$-tensors participating in these legs are exactly those associated with some $i \in \cut_\I$ (recall $\cut_\I$ is the set of vertices with an edge crossing the partition $(\I, \I^c)$).
Hence, for every $l \in [L]$ and $i \in \cut_\I$ there are $\nwalk{L - l}{\{i\}}{ \{ t\} }$ such $\delta$-tensors.
Legs connected to the same $\delta$-tensor only contribute once to $\modcutweight{\tngraph^{(t)} (\fmat)} ( \J_{ \tngraph^{(t)} (\fmat) } )$.
Thus, since the weights of all legs connected to $\delta$-tensors are equal to $\hdim$, we have that:
\[
\modcutweight{\tngraph^{(t)} (\fmat)} (\J_{\tngraph^{(t)} (\fmat)}) = \hdim^{ \sum\nolimits_{ l = 1 }^L \sum\nolimits_{i \in \cut_\I} \nwalk{L - l}{\{i\}}{\{t\}} } = \hdim^{ \sum\nolimits_{ l = 1 }^L \nwalk{L - l}{ \cut_\I }{\{t\}} }
\text{\,.}
\]
Lastly, it remains to show $\sum\nolimits_{ l = 1 }^L \nwalk{L - l}{ \cut_\I }{\{t\}} \leq 4 \nwalk{L - 1}{ \cut_\I }{\{t\}}$, as in that case~\cref{gnn:lem:min_cut_tn_sep_rank_ub} implies:
\[
\seprankbig{ \funcvert{\params}{\graph}{t} }{\I} \leq \modcutweight{\tngraph^{(t)} (\fmat)} (\J_{\tngraph^{(t)} (\fmat)}) \leq \hdim^{ 4 \nwalk{L - 1}{ \cut_\I }{ \{ t \}} }
\text{\,,}
\]
which leads to~\cref{gnn:eq:sep_rank_upper_bound_vertex_pred} by taking the log of both sides.

The main idea is that, in an undirected graph with self-loops, the number of length $l \in \N$ walks ending at $t$ that originate from vertices with at least one neighbor decays exponentially when $l$ decreases.
First, clearly $\nwalk{l}{ \cut_\I }{ \{t\} } \leq \nwalk{l + 1}{ \cut_\I }{\{t\}}$ for all $l \in \N$.
Therefore:
\be
\sum\nolimits_{ l = 1 }^L \nwalk{L - l}{ \cut_\I }{\{t\}} \leq 2 \sum\nolimits_{l \in \{1, 3, \ldots, L - 1\} } \nwalk{L - l}{ \cut_\I }{\{t\}}
\text{\,.}
\label{gnn:eq:nwalk_increasing_undirected_vertex}
\ee
Furthermore, any length $l \in \N_{\geq 0}$ walk $i_0, i_1, \ldots, i_{l - 1}, t \in \vertices$ from $\cut_\I$ to $t$ induces at least two walks of length $l + 2$ from $\cut_\I$ to $t$, distinct from those induced by other length $l$ walks~---~one which goes twice through the self-loop of $i_0$ and then proceeds according to the length $l$ walk, \ie~$i_0, i_0, i_0, i_1, \ldots, i_{l - 1}, t$, and another that goes to a neighboring vertex (exists since $i_0 \in \cut_\I$), returns to $i_0$, and then proceeds according to the length $l$ walk.
This means that $\nwalk{L - l}{ \cut_\I}{\{t\}} \leq 2^{-1} \cdot \nwalk{L - l + 2}{ \cut_\I}{\{t\}} \leq \cdots \leq 2^{ - \lfloor l / 2 \rfloor } \cdot \nwalk{L -1 }{\cut_\I}{\{t\}}$ for all $l \in \{ 3, 5, \ldots, L - 1\}$.
Going back to~\cref{gnn:eq:nwalk_increasing_undirected_vertex}, we have that:
\[
\begin{split}
	\sum\nolimits_{ l = 1 }^L \nwalk{L - l}{ \cut_\I }{\{t\}}
	& \leq 
	2 \sum\nolimits_{l \in \{1, 3, \ldots, L - 1\} } 2^{ \lfloor l / 2 \rfloor } \cdot \nwalk{L - 1}{ \cut_\I }{\{t\}} \\
	& \leq 
	2 \sum\nolimits_{l = 0}^{\infty} 2^{ - l} \cdot \nwalk{L - 1}{ \cut_\I }{\{t\}} \\
	& = 4 \nwalk{L - 1}{ \cut_\I }{ \{t\} }
	\text{\,,}
\end{split}
\]
concluding the proof of~\cref{gnn:eq:sep_rank_upper_bound_vertex_pred}.
\qed

\subsubsection{Technical Lemma}
\label{gnn:app:proofs:sep_rank_upper_bound:lemmas}

\begin{lemma}
	\label{gnn:lem:contraction_matricization}
	For any order $N \in \N$ tensor $\Atensor \in \R^{D \times \cdots \times D}$, vectors $\xbf^{(1)}, \ldots, \xbf^{(N)} \in \R^D$, and subset of mode indices $\I \subseteq [N]$, it holds that $\Atensor \contract{i \in [N]} \xbf^{(i)} = \brk1{ \kronp_{i \in I} \xbf^{(i)} }^\top \mat{ \Atensor }{ \I } \brk1{ \kronp_{j \in \I^c} \xbf^{(j)} } \in \R$.
\end{lemma}
\begin{proof}
	The identity follows directly from the definitions of tensor contraction, matricization, and Kronecker product (\cref{gnn:app:proofs:notation}):
	\[
	\Atensor \contract{i \in [N]} \xbf^{(i)} = \sum\nolimits_{d_1, \ldots, d_N = 1}^D \Atensor_{d_1, \ldots, d_N} \cdot \prod\nolimits_{i \in [N]} \xbf^{(i)}_{d_i} = \brk1{ \kronp_{i \in I} \xbf^{(i)} }^\top \mat{ \Atensor }{ \I } \brk1{ \kronp_{j \in \I^c} \xbf^{(j)} }
	\text{\,.}
	\]
\end{proof}

\subsection{Proof of~\cref{gnn:thm:sep_rank_lower_bound}}
\label{gnn:app:proofs:sep_rank_lower_bound}

We assume familiarity with the basic concepts from tensor analysis introduced in~\cref{gnn:app:prod_gnn_as_tn:tensors}.

We begin by establishing a general technique for lower bounding the separation rank of a function through \emph{grid tensors}, also used in~\cite{levine2020limits,wies2021transformer,levine2022inductive,razin2022implicit}.
For any $f : (\R^{\indim})^N \to \R$ and $M \in \N$ \emph{template vectors} $\vbf^{(1)}, \ldots, \vbf^{(M)} \in \R^{\indim}$, we can create a grid tensor of $f$, which is a form of function discretization, by evaluating it over each point in \smash{$\brk[c]{ \brk{ \vbf^{(d_1)} , \ldots, \vbf^{(d_{ N } ) } } }_{d_1, \ldots, d_{ N } = 1}^M$} and storing the outcomes in an order $N$ tensor with modes of dimension $M$.
That is, the grid tensor of $f$ for templates $\vbf^{(1)}, \ldots, \vbf^{(M)}$, denoted $\gridtensor{f} \in \R^{M \times \cdots \times M}$, is defined by $\gridtensor{f}_{d_1, \ldots, d_N} = f \brk{ \vbf^{(d_1)}, \ldots, \vbf^{(d_N)} }$ for all $d_1, \ldots, d_N \in [M]$.\footnote{
	The template vectors of a grid tensor $\gridtensor{f}$ will be clear from context, thus we omit them from the notation.
}
\cref{gnn:lem:grid_tensor_sep_rank_lb} shows that $\sepranknoflex{f}{\I}$ is lower bounded by the rank of $\gridtensor{f}$'s matricization with respect to $\I$.

\begin{lemma}
	\label{gnn:lem:grid_tensor_sep_rank_lb}
	For $f : ( \R^{\indim} )^N \to \R$ and $M \in \N$ template vectors $\vbf^{(1)}, \ldots, \vbf^{(M)} \in \R^{\indim}$, let $\gridtensor{f} \in \R^{M \times \cdots \times M}$ be the corresponding order $N$ grid tensor of $f$.
	Then, for any $\I \subseteq [N]$:
	\[
	\rank \mat{\gridtensor{f}}{\I} \leq \seprank{f}{\I}
	\text{\,.}
	\]
\end{lemma}
\begin{proof}
	If $\sepranknoflex{f}{\I}$ is $\infty$ or zero, \ie~$f$ cannot be represented as a finite sum of separable functions (with respect to $\I$) or is identically zero, then the claim is trivial. 
	Otherwise, denote $R := \seprank{f}{\I}$, and let $g^{(1)}, \ldots, g^{(R)} : (\R^{\indim})^{ \abs{\I} } \to \R$ and $\bar{g}^{(1)}, \ldots, \bar{g}^{(R)} : (\R^{\indim})^{ \abs{\I^c} } \to \R$ such that:
	\be
	f \brk{ \fmat } = \sum\nolimits_{r = 1}^R g^{(r)} \brk{ \fmat_{\I} } \cdot \bar{g}^{(r)} \brk{ \fmat_{\I^c} }
	\text{\,,}
	\label{gnn:eq:grid_tensor_mat_ub_by_sep_rank:sep_rank}
	\ee
	where $\fmat := \brk{\fvec{1}, \ldots, \fvec{N}}$, $\fmat_\I := \brk{ \fvec{i} }_{i \in \I}$, and $\fmat_{\I^c} := \brk{ \fvec{j} }_{j \in \I^c}$.
	For $r \in [R]$, let $\gridtensornoflex{g^{(r)}}$ and $\gridtensornoflex{\bar{g}^{(r)}}$ be the grid tensors of $g^{(r)}$ and $\bar{g}^{(r)}$ over templates $\vbf^{(1)}, \ldots, \vbf^{(M)}$, respectively.
	That is, $\gridtensornoflex{g^{(r)}}_{d_i : i \in \I } = g^{(r)} \brk{ \brk{ \vbf^{ (d_i) } }_{i \in \I} }$ and $\gridtensornoflex{ \bar{g}^{(r)} }_{ d_j : j \in \I^c } = \bar{g}^{(r)} \brk{ \brk{ \vbf^{(d_j)} }_{j \in \I^c} }$ for all $d_1, \ldots, d_{ N } \in [M]$.
	By~\cref{gnn:eq:grid_tensor_mat_ub_by_sep_rank:sep_rank} we have that for any $d_1, \ldots, d_N \in [M]$:
	\[
	\begin{split}
		\gridtensor{f}_{d_1, \ldots, d_N} & = f \brk1{ \vbf^{(d_1)}, \ldots, \vbf^{(d_N)} } \\
		& = \sum\nolimits_{r = 1}^R g^{(r)} \brk1{ \brk{ \vbf^{(d_i)} }_{i \in \I} } \cdot \bar{g}^{(r)} \brk1{ \brk{ \vbf^{(d_j)} }_{j \in \I^c} } \\
		& = \sum\nolimits_{r = 1}^R \gridtensorbig{g^{(r)}}_{d_i : i \in \I } \cdot \gridtensorbig{ \bar{g}^{(r)} }_{ d_j : j \in \I^c }
		\text{\,.}
	\end{split}
	\]
	Denoting by $\ubf^{(r)} \in \R^{M^{\abs{\I}}}$ and $\bar{\ubf}^{(r)} \in \R^{M^{ \abs{\I^c} }}$ the arrangements of $\gridtensornoflex{g^{(r)}}$ and $\gridtensornoflex{ \bar{g}^{(r)} }$ as vectors, respectively for $r \in [R]$, this implies that the matricization of $\gridtensor{f}$ with respect to $\I$ can be written as:
	\[
	\mat{ \gridtensor{f} }{\I} = \sum\nolimits_{r = 1}^R \ubf^{(r)} \brk1{ \bar{\ubf}^{(r)} }^{\top}
	\text{\,.}
	\]
	We have arrived at a representation of $\mat{ \gridtensor{f} }{\I}$ as a sum of $R$ outer products between two vectors.
	An outer product of two vectors is a matrix of rank at most one.
	Consequently, by sub-additivity of rank we conclude: $\rank \mat{ \gridtensor{f} }{\I} \leq R = \seprank{f}{\I}$.
\end{proof}

\medskip

In the context of graph prediction, let $\cut^* \in \argmax_{ \cut \in \cutset (\I) } \log \brk{ \alpha_{\cut} } \cdot \nwalk{L - 1}{\cut}{\vertices}$.
Then, by \cref{gnn:lem:grid_tensor_sep_rank_lb}, to prove that~\cref{gnn:eq:sep_rank_lower_bound_graph_pred} holds for weights~$\params$, it suffices to find template vectors for which $\log \brk{ \rank \mat{ \gridtensor{ \funcgraph{\params}{\graph} } }{\I} } \geq \log \brk{ \alpha_{\cut^*} } \cdot \nwalk{L - 1}{\cut^*}{\vertices}$.
Notice that, since the outputs of $\funcgraph{\params}{\graph}$ vary polynomially with the weights $\params$, so do the entries of $\mat{ \gridtensor{ \funcgraph{\params}{\graph} } }{\I}$ for 
any choice of template vectors.
Thus, according to~\cref{gnn:lem:poly_mat_max_rank}, by constructing weights~$\params$ and template vectors satisfying $\log \brk{ \rank \mat{ \gridtensor{ \funcgraph{\params}{\graph} } }{\I} } \geq \log \brk{ \alpha_{\cut^*} } \cdot \nwalk{L - 1}{\cut^*}{\vertices}$, we may conclude that this is the case for almost all assignments of weights, meaning~\cref{gnn:eq:sep_rank_lower_bound_graph_pred} holds for almost all assignments of weights.
In Appendix~\ref{gnn:app:proofs:sep_rank_lower_bound:graph} we construct such weights and template vectors.

In the context of vertex prediction, let $\cut_t^* \in \argmax_{ \cut \in \cutset (\I) } \log \brk{ \alpha_{\cut, t} } \cdot \nwalk{L - 1}{\cut}{ \{ t \} }$.
Due to arguments analogous to those above, to prove that~\cref{gnn:eq:sep_rank_lower_bound_vertex_pred} holds for almost all assignments of weights, we need only find weights $\params$ and template vectors satisfying $\log \brk{ \rank \mat{ \gridtensor{ \funcvert{\params}{\graph}{t} } }{\I} } \geq \log \brk{ \alpha_{\cut^*_t, t} } \cdot \nwalk{L - 1}{\cut^*_t}{ \{ t \} }$.
In Appendix~\ref{gnn:app:proofs:sep_rank_lower_bound:vertex} we do so.

Lastly, recalling that a finite union of measure zero sets has measure zero as well establishes that~\cref{gnn:eq:sep_rank_lower_bound_graph_pred,gnn:eq:sep_rank_lower_bound_vertex_pred} jointly hold for almost all assignments of weights.
\qed

\subsubsection{Weights and Template Vectors Assignment for Graph Prediction (Proof of~\cref{gnn:eq:sep_rank_lower_bound_graph_pred})}
\label{gnn:app:proofs:sep_rank_lower_bound:graph}

We construct weights~$\params$ and template vectors satisfying $\log \brk{ \rank \mat{ \gridtensor{ \funcgraph{\params}{\graph} } }{\I} } \geq \log \brk{ \alpha_{\cut^*} } \cdot \nwalk{L - 1}{\cut^*}{\vertices}$, where $\cut^* \in \argmax_{ \cut \in \cutset (\I) } \log \brk{ \alpha_{\cut} } \cdot \nwalk{L - 1}{\cut}{\vertices}$.

If $\nwalk{L - 1}{\cut^*}{\vertices} = 0$, then the claim is trivial since there exist weights and template vectors for which $\mat{ \gridtensor{ \funcgraph{\params}{\graph} } }{\I}$ is not the zero matrix (\eg~taking all weight matrices to be zero-padded identity matrices and choosing a single template vector holding one in its first entry and zeros elsewhere).

Now, assuming that $\nwalk{L - 1}{\cut^*}{\vertices} > 0$, which in particular implies that $\I \neq \emptyset, \I \neq \vertices,$ and $\cut^* \neq \emptyset$, we begin with the case of GNN depth $L = 1$, after which we treat the more general $L \geq 2$ case.

\textbf{Case of $L = 1$:}
Consider the weights $\params = \brk{ \weightmat{1}, \weightmat{o} }$ given by $\weightmat{1} := \Ibf \in \R^{\hdim \times \indim}$ and $\weightmat{o} := (1, \ldots, 1) \in \R^{1 \times \hdim}$, where $\Ibf$ is a zero padded identity matrix, \ie~it holds ones on its diagonal and zeros elsewhere.
We choose template vectors $\vbf^{(1)}, \ldots, \vbf^{(\mindim)} \in \R^{\indim}$ such that $\vbf^{(m)}$ holds the $m$'th standard basis vector of $\R^{\mindim}$ in its first $\mindim$ coordinates and zeros in the remaining entries, for $m \in [\mindim]$ (recall $\mindim := \min \{ \indim, \hdim \}$).
Namely, denote by $\ebf^{(1)}, \ldots, \ebf^{(\mindim)} \in \R^{\mindim}$ the standard basis vectors of $\R^{\mindim}$, \ie~$\ebf^{(m)}_d = 1$ if $d = m$ and $\ebf^{(m)}_d = 0$ otherwise for all $m, d \in [\mindim]$. We let $\vbf^{(m)}_{:\mindim} := \ebf^{(m)}$ and $\vbf^{(m)}_{\mindim + 1:} := 0$ for all $m \in [\mindim]$.

We prove that for this choice of weights and template vectors, for all $d_1, \ldots, d_{\abs{\vertices}} \in [\mindim]$:
\be
\funcgraph{\params}{\graph} \brk1{ \vbf^{(d_1)}, \ldots, \vbf^{ ( d_{\abs{\vertices}} )} } = \begin{cases}
	1	& , \text{if } d_1 = \cdots = d_{\abs{\vertices}} \\
	0	& , \text{otherwise} 
\end{cases}
\text{\,.}
\label{gnn:eq:standard_basis_grid_graph}
\ee
To see it is so, notice that:
\[
\funcgraph{\params}{\graph} \brk1{ \vbf^{(d_1)}, \ldots, \vbf^{ ( d_{\abs{\vertices}} )} } = \weightmat{o} \brk1{ \hadmp_{i \in \vertices} \hidvec{1}{i} } = \sum\nolimits_{d = 1}^{\hdim} \prod\nolimits_{i \in \vertices} \hidvec{1}{i}_d 
\text{\,,}
\]
with $\hidvec{1}{i} = \hadmp_{j \in \neigh (i)} \brk{\weightmat{1} \vbf^{(d_j)} } = \hadmp_{j \in \neigh (i)} \brk{ \Ibf \vbf^{(d_j)} }$ for all $i \in \vertices$.
Since $ \vbf^{(d_j)}_{:\mindim} = \ebf^{(d_j)}$ for all $j \in \neigh (i)$ and $\Ibf$ is a zero-padded $\mindim \times \mindim$ identity matrix, it holds that: 
\[
\funcgraph{\params}{\graph} \brk1{ \vbf^{(d_1)}, \ldots, \vbf^{ ( d_{\abs{\vertices}} )} } = \sum\nolimits_{d = 1}^{\mindim} \prod\nolimits_{i \in \vertices, j \in \neigh (i)} \ebf^{(d_j)}_d
\text{\,.}
\]
Due to the existence of self-loops (\ie~$i \in \neigh (i)$ for all $i \in \vertices$), for every $d \in [\mindim]$ the product $\prod\nolimits_{i \in \vertices, j \in \neigh (i)} \ebf^{(d_j)}_d$ includes each of $\ebf^{(d_1)}_d, \ldots, \ebf^{(d_{\abs{\vertices}})}_d$ at least once.
Consequently, $\prod\nolimits_{i \in \vertices, j \in \neigh (i)} \ebf^{(d_j)}_d = 1$ if $d_1 = \cdots = d_{\abs{\vertices}} = d$ and $\prod\nolimits_{i \in \vertices, j \in \neigh (i)} \ebf^{(d_j)}_d = 0$ otherwise.
This implies that $\funcgraph{\params}{\graph} \brk{ \vbf^{(d_1)}, \ldots, \vbf^{ ( d_{\abs{\vertices}} )} } = 1$ if $d_1 = \cdots = d_{\abs{\vertices}}$ and $\funcgraph{\params}{\graph} \brk{ \vbf^{(d_1)}, \ldots, \vbf^{ ( d_{\abs{\vertices}} )} } = 0$ otherwise, for all $d_1, \ldots, d_{\abs{\vertices}} \in [\mindim]$.

\cref{gnn:eq:standard_basis_grid_graph} implies that $\mat{ \gridtensor{ \funcgraph{\params}{\graph} } }{\I}$ has exactly $\mindim$ non-zero entries, each in a different row and column.
Thus, $\rank \mat{ \gridtensor{ \funcgraph{\params}{\graph} } }{\I} = \mindim$.
Recalling that $\alpha_{\cut^*} := \mindim^{1 / \nwalk{0}{\cut^*}{\vertices}}$ for $L = 1$, we conclude: 
\[
\log \brk*{ \rank \mat{ \gridtensor{ \funcgraph{\params}{\graph} } }{\I} } =  \log (\mindim) = \log \brk{ \alpha_{\cut^*} } \cdot \nwalk{0}{\cut^*}{\vertices}
\text{\,.}
\]

\textbf{Case of $L \geq 2$:}
Let $M := \multisetcoeffnoflex{\mindim}{ \nwalk{L - 1}{\cut^*}{\vertices}} = \binom{\mindim + \nwalk{L - 1}{\cut^*}{\vertices} - 1}{ \nwalk{L - 1}{\cut^*}{\vertices}}$ be the multiset coefficient of $\mindim$ and $ \nwalk{L - 1}{\cut^*}{\vertices}$ (recall $\mindim := \min \brk[c]{ \indim, \hdim}$).
By~\cref{gnn:lem:hadm_of_gram_full_rank}, there exists $\Zbf \in \R_{ > 0}^{M \times \mindim}$ for which
\[
\rank \brk*{ \hadmp^{ \nwalk{L - 1}{\cut^*}{\vertices} } \brk*{ \Zbf \Zbf^\top } } = \multisetcoeff{ \mindim }{ \nwalk{L - 1}{ \cut^* }{ \vertices } }
\text{\,,}
\]
with $\hadmp^{ \nwalk{L - 1}{\cut^*}{\vertices} } \brk{ \Zbf \Zbf^\top }$ standing for the $\nwalk{L - 1}{\cut^*}{\vertices}$'th Hadamard power of $\Zbf \Zbf^\top$.
For this $\Zbf$, by~\cref{gnn:lem:grid_submatrix_hadm_of_gram_graph} below we know that there exist weights $\params$ and template vectors such that $\mat{ \gridtensornoflex{ \funcgraph{\params}{\graph} } }{ \I }$ has an $M \times M$ sub-matrix of the form $\Sbf \brk{ \hadmp^{ \nwalk{L - 1}{\cut^*}{\vertices} } \brk{ \Zbf \Zbf^\top } } \Qbf$, where $\Sbf, \Qbf \in \R^{M \times M}$ are full-rank diagonal matrices.
Since the rank of a matrix is at least the rank of any of its sub-matrices:
\[
\begin{split}
	\rank \brk*{ \mat{ \gridtensorbig{ \funcgraph{\params}{\graph} } }{ \I } } & \geq \rank \brk*{ \Sbf \brk*{ \hadmp^{ \nwalk{L - 1}{\cut^*}{\vertices} } \brk*{ \Zbf \Zbf^\top } } \Qbf } \\
	& = \rank \brk*{ \hadmp^{ \nwalk{L - 1}{\cut^*}{\vertices} } \brk*{ \Zbf \Zbf^\top } } \\
	& = \multisetcoeff{ \mindim }{ \nwalk{L - 1}{ \cut^* }{ \vertices } }
	\text{\,,}
\end{split}
\]
where the second transition stems from $\Sbf$ and $\Qbf$ being full-rank.
Applying~\cref{gnn:lem:multiset_coeff_lower_bound} to lower bound the multiset coefficient, we have that:
\[
\rank \brk*{ \mat{ \gridtensorbig{ \funcgraph{\params}{\graph} } }{ \I } } \geq \multisetcoeff{ \mindim }{ \nwalk{L - 1}{ \cut^* }{ \vertices } } \geq \brk*{ \frac{\mindim - 1}{ \nwalk{L - 1}{ \cut^* }{ \vertices }} + 1 }^{ \nwalk{L - 1}{ \cut^* }{ \vertices } }
\text{\,.}
\]
Taking the log of both sides while recalling that $\alpha_{\cut^*} := \brk{ \mindim - 1 } \cdot \nwalk{L - 1}{\cut^*}{\vertices}^{-1} + 1$, we conclude that:
\[
\log \brk{ \rank \mat{ \gridtensor{ \funcgraph{\params}{\graph} } }{\I} } \geq \log \brk{ \alpha_{\cut^*} } \cdot \nwalk{L - 1}{\cut^*}{\vertices}
\text{\,.}
\]

\begin{lemma}
	\label{gnn:lem:grid_submatrix_hadm_of_gram_graph}
	Suppose that the GNN inducing $\funcgraph{\params}{\graph}$ is of depth $L \geq 2$ and that $\nwalk{L - 1}{\cut^*}{\vertices} > 0$.
	For any $M \in \N$ and matrix with positive entries $\Zbf \in \R_{> 0}^{M \times \mindim}$, there exist weights $\params$ and $M + 1$ template vectors $\vbf^{(1)}, \ldots, \vbf^{(M + 1)} \in \R^{\indim}$~such that $\mat{ \gridtensornoflex{ \funcgraph{\params}{\graph} } }{ \I }$ has an $M \times M$ sub-matrix $\Sbf \brk{ \hadmp^{ \nwalk{L - 1}{\cut^*}{\vertices} } \brk{ \Zbf \Zbf^\top } } \Qbf$, where $\Sbf, \Qbf \in \R^{M \times M}$ are full-rank diagonal matrices and $\hadmp^{ \nwalk{L - 1}{\cut^*}{\vertices} } \brk{ \Zbf \Zbf^\top }$ is the $\nwalk{L - 1}{\cut^*}{\vertices}$'th Hadamard power of $\Zbf \Zbf^\top$.
\end{lemma}

\begin{proof}
	Consider the weights $\params = \brk{ \weightmat{1}, \ldots, \weightmat{L}, \weightmat{o} }$ given by:
	\[
	\begin{split}
		& \weightmat{1} := \Ibf \in \R^{\hdim \times \indim} \text{\,,} \\[0.2em]
		& \weightmat{2} := \begin{pmatrix} 1 & 1 & \cdots & 1 \\ 0 & 0 & \cdots & 0 \\ \vdots & \vdots & \cdots & \vdots \\ 0 & 0 & \cdots & 0\end{pmatrix} \in \R^{\hdim \times \hdim} \text{\,,} \\[0.2em]
		\forall l \in \{ 3, \ldots, L\} : ~&\weightmat{l} := \begin{pmatrix} 1 & 0 & \cdots & 0 \\ 0 & 0 & \cdots & 0 \\ \vdots & \vdots & \cdots & \vdots \\ 0 & 0 & \cdots & 0\end{pmatrix} \in \R^{\hdim \times \hdim} \text{\,,} \\[0.2em]
		& \weightmat{o} := \begin{pmatrix} 1 & 0 & \cdots & 0 \end{pmatrix} \in \R^{1 \times \hdim} \text{\,,}
	\end{split}
	\]
	where $\Ibf$ is a zero padded identity matrix, \ie~it holds ones on its diagonal and zeros elsewhere.
	We define the templates $\vbf^{(1)}, \ldots, \vbf^{(M)} \in \R^{\indim}$ to be the vectors holding the respective rows of $\Zbf$ in their first $\mindim$ coordinates and zeros in the remaining entries (recall $\mindim := \min \brk[c]{ \indim, \hdim}$).
	That is, denoting the rows of $\Zbf$ by $\zbf^{(1)}, \ldots, \zbf^{(M)} \in \R_{> 0}^{\mindim}$, we let $\vbf^{(m)}_{: \mindim} := \zbf^{(m)}$ and $\vbf^{(m)}_{\mindim + 1: } := 0$ for all $m \in [M]$.
	We set all entries of the last template vector to one, \ie~$\vbf^{(M + 1)} := (1, \ldots, 1) \in \R^{\indim}$.
	
	Since $\cut^* \in \cutset (\I)$, \ie~it is an admissible subset of $\cut_\I$ (\cref{gnn:def:admissible_subsets}), there exist $\I' \subseteq \I, \J' \subseteq \I^c$ with no repeating shared neighbors (\cref{gnn:def:no_rep_neighbors}) such that $\cut^* = \neigh (\I') \cap \neigh ( \J' )$.
	Notice that $\I'$ and $\J'$ are non-empty as $\cut^* \neq \emptyset$ (this is implied by $\nwalk{L - 1}{\cut^*}{\vertices} > 0$).
	We focus on the $M \times M$ sub-matrix of $\mat{ \gridtensornoflex{ \funcgraph{\params}{\graph} } }{ \I }$ that includes only rows and columns corresponding to evaluations of $\funcgraph{\params}{\graph}$ where all variables indexed by $\I'$ are assigned the same template vector from $\vbf^{(1)}, \ldots, \vbf^{(M)}$, all variables indexed by $\J'$ are assigned the same template vector from $\vbf^{(1)}, \ldots, \vbf^{(M)}$, and all remaining variables are assigned the all-ones template vector $\vbf^{(M + 1)}$.
	Denoting this sub-matrix by $\Ubf \in \R^{M \times M}$, it therefore upholds:
	\[
	\Ubf_{m, n} = \funcgraph{\params}{\graph} \brk*{ \brk1{ \fvec{i} \leftarrow \vbf^{(m)} }_{i \in \I'} , \brk1{ \fvec{j} \leftarrow \vbf^{(n)}}_{j \in \J'} , \brk1{ \fvec{k} \leftarrow \vbf^{(M + 1)} }_{k \in \vertices \setminus \brk{ \I' \cup \J' } } }
	\text{\,,}
	\]
	for all $m, n \in [M]$, where we use $\brk{ \fvec{i} \leftarrow \vbf^{(m)} }_{i \in \I'}$ to denote that input variables indexed by $\I'$ are assigned the value $\vbf^{(m)}$.
	To show that $\Ubf$ obeys the form 
	\[
	\Sbf \brk{ \hadmp^{ \nwalk{L - 1}{\cut^*}{\vertices} } \brk{ \Zbf \Zbf^\top } } \Qbf
	\]
	for full-rank diagonal $\Sbf, \Qbf \in \R^{M \times M}$, we prove there exist $\phi, \psi : \R^{\indim} \to \R_{> 0}$ such that $\Ubf_{m, n} = \phi \brk{ \vbf^{(m)} } \inprodnoflex{ \zbf^{(m)} }{ \zbf^{(n)} }^{ \nwalk{L - 1}{ \cut^* }{ \vertices } } \psi \brk{ \vbf^{(n)} }$ for all $m, n \in [M]$.
	Indeed, defining $\Sbf$ to hold $\phi \brk{ \vbf^{(1)} }, \ldots, \phi \brk{ \vbf^{(M)} }$ on its diagonal and $\Qbf$ to hold $\psi \brk{ \vbf^{(1)} }, \ldots, \psi \brk{ \vbf^{(M)} }$ on its diagonal, we have that $\Ubf = \Sbf \brk{ \hadmp^{ \nwalk{L - 1}{\cut^*}{\vertices} } \brk{ \Zbf \Zbf^\top } } \Qbf$.
	Since $\Sbf$ and $\Qbf$ are clearly full-rank (diagonal matrices with non-zero entries on their diagonal), the proof concludes.
	
	\medskip
	
	For $m, n \in [M]$, let $\hidvec{l}{i} \in \R^{\hdim}$ be the hidden embedding for $i \in \vertices$ at layer $l \in [L]$ of the GNN inducing $\funcgraph{\params}{\graph}$, over the following assignment to its input variables (\ie~vertex features):
	\[
	\brk1{ \fvec{i} \leftarrow \vbf^{(m)} }_{i \in \I'} , \brk1{ \fvec{j} \leftarrow \vbf^{(n)}}_{j \in \J'} , \brk1{ \fvec{k} \leftarrow \vbf^{(M + 1)} }_{k \in \vertices \setminus \brk{ \I' \cup \J' } }
	\text{\,.}
	\]
	Invoking~\cref{gnn:lem:grid_tensor_entry_hadm_induction} with $\vbf^{(m)}, \vbf^{(n)}, \I',$ and $\J'$, for all $i \in \vertices$ it holds that:
	\[
	\hidvec{L}{i}_{1} = \phi^{(L, i)} \brk1{ \vbf^{(m)} } \inprodbig{ \zbf^{(m)} }{ \zbf^{(n)} }^{ \nwalk{L - 1}{ \cut^* }{ \{i\} } } \psi^{(L, i)} \brk1{ \vbf^{(n)} } ~~~ , ~~~ \forall d \in \{ 2, \ldots, \hdim \} :~\hidvec{L}{i}_d = 0
	\text{\,,}
	\]
	for some $\phi^{(L, i)}, \psi^{(L, i)} : \R^{\indim} \to \R_{ > 0}$.
	Since
	\[
	\begin{split}
		\Ubf_{m, n} & = \funcgraph{\params}{\graph} \brk*{ \brk1{ \fvec{i} \leftarrow \vbf^{(m)} }_{i \in \I'} , \brk1{ \fvec{j} \leftarrow \vbf^{(n)}}_{j \in \J'} , \brk1{ \fvec{k} \leftarrow \vbf^{(M + 1)} }_{k \in \vertices \setminus \brk{ \I' \cup \J' } } } \\
		& = \weightmat{o} \brk1{ \hadmp_{i \in \vertices} \hidvec{L}{i} }
	\end{split}
	\]
	and $\weightmat{o} = \brk{1, 0, \ldots, 0}$, this implies that:
	\[
	\begin{split}
		\Ubf_{m, n} & = \prod\nolimits_{ i \in \vertices } \hidvec{L}{i}_1 \\
		& = \prod\nolimits_{ i \in \vertices } \phi^{(L, i)} \brk1{ \vbf^{(m)} } \inprodbig{ \zbf^{(m)} }{ \zbf^{(n)} }^{ \nwalk{L - 1}{ \cut^* }{ \{i\} } } \psi^{(L, i)} \brk1{ \vbf^{(n)} } 
		\text{\,.}
	\end{split}
	\]
	Rearranging the last term leads to:
	\[
	\Ubf_{m, n} = \brk*{ \prod\nolimits_{ i \in \vertices } \phi^{(L, i)} \brk1{ \vbf^{(m)} } } \cdot \inprodbig{ \zbf^{(m)} }{ \zbf^{(n)} }^{ \sum\nolimits_{i \in \vertices} \nwalk{L - 1}{ \cut^* }{ \{i\} } } \cdot \brk*{ \prod\nolimits_{ i \in \vertices } \psi^{(L, i)} \brk1{ \vbf^{(n)} } }
	\text{\,.}
	\]
	Let $\phi : \vbf \mapsto \prod\nolimits_{i \in \vertices } \phi^{(L, i)} \brk{ \vbf }$ and $\psi : \vbf \mapsto \prod\nolimits_{ i \in \vertices } \psi^{(L, i)} \brk{ \vbf }$.
	Noticing that their range is indeed $\R_{ > 0}$ and that $\sum\nolimits_{i \in \vertices} \nwalk{L - 1}{ \cut^* }{ \{i\} } = \nwalk{L- 1}{ \cut^*}{ \vertices }$ yields the sought-after expression for $\Ubf_{m, n}$:
	\[
	\Ubf_{m, n} = \phi \brk1{ \vbf^{(m)} } \inprodbig{ \zbf^{(m)} }{ \zbf^{(n)} }^{ \nwalk{L - 1}{ \cut^* }{ \vertices } } \psi \brk1{ \vbf^{(n)} }
	\text{\,.}
	\]
\end{proof}

\subsubsection{Weights and Template Vectors Assignment for Vertex Prediction (Proof of~\cref{gnn:eq:sep_rank_lower_bound_vertex_pred})}
\label{gnn:app:proofs:sep_rank_lower_bound:vertex}

This part of the proof follows a line similar to that of Appendix~\ref{gnn:app:proofs:sep_rank_lower_bound:graph}, with differences stemming from the distinction between the operation of a GNN over graph and vertex prediction.
Namely, we construct weights~$\params$ and template vectors satisfying $\log \brk{ \rank \mat{ \gridtensor{ \funcvert{\params}{\graph}{t} } }{\I} } \geq \log \brk{ \alpha_{\cut^*_t, t} } \cdot \nwalk{L - 1}{\cut^*_t}{ \{ t \} }$, where: 
\[
\cut^*_t \in \argmax\nolimits_{ \cut \in \cutset (\I) } \log \brk{ \alpha_{\cut, t} } \cdot \nwalk{L - 1}{\cut}{\{t\}}
\text{\,.}
\]

If $\nwalk{L - 1}{\cut^*_t}{\{t\}} = 0$, then the claim is trivial since there exist weights and template vectors for which  $\mat{ \gridtensor{ \funcvert{\params}{\graph}{t} } }{\I}$ is not the zero matrix (\eg~taking all weight matrices to be zero-padded identity matrices and choosing a single template vector holding one in its first entry and zeros elsewhere).

Now, assuming that $\nwalk{L - 1}{\cut^*_t}{\{t\}} > 0$, which in particular implies that $\I \neq \emptyset, \I \neq \vertices,$ and $\cut^*_t \neq \emptyset$, we begin with the case of GNN depth $L = 1$, after which we treat the more general $L \geq 2$ case.

\textbf{Case of $L = 1$:}
Consider the weights $\params = \brk{ \weightmat{1}, \weightmat{o} }$ given by $\weightmat{1} := \Ibf \in \R^{\hdim \times \indim}$ and $\weightmat{o} := (1, \ldots, 1) \in \R^{1 \times \hdim}$, where $\Ibf$ is a zero padded identity matrix, \ie~it holds ones on its diagonal and zeros elsewhere.
We choose template vectors $\vbf^{(1)}, \ldots, \vbf^{(\mindim)} \in \R^{\indim}$ such that $\vbf^{(m)}$ holds the $m$'th standard basis vector of $\R^{\mindim}$ in its first $\mindim$ coordinates and zeros in the remaining entries, for $m \in [\mindim]$ (recall $\mindim := \min \{ \indim, \hdim \}$).
Namely, denote by $\ebf^{(1)}, \ldots, \ebf^{(\mindim)} \in \R^{\mindim}$ the standard basis vectors of $\R^{\mindim}$, \ie~$\ebf^{(m)}_d = 1$ if $d = m$ and $\ebf^{(m)}_d = 0$ otherwise for all $m, d \in [\mindim]$. We let $\vbf^{(m)}_{:\mindim} := \ebf^{(m)}$ and $\vbf^{(m)}_{\mindim + 1:} := 0$ for all $m \in [\mindim]$.

We prove that for this choice of weights and template vectors, for all $d_1, \ldots, d_{\abs{\vertices}} \in [\mindim]$:
\be
\funcvert{\params}{\graph}{t} \brk1{ \vbf^{(d_1)}, \ldots, \vbf^{ ( d_{\abs{\vertices}} )} } = \begin{cases}
	1	& , \text{if } d_{j} = d_{j'} \text{ for all } j, j' \in \neigh (t) \\
	0	& , \text{otherwise} 
\end{cases}
\text{\,.}
\label{gnn:eq:standard_basis_grid_vertex}
\ee
To see it is so, notice that:
\[
\funcvert{\params}{\graph}{t} \brk1{ \vbf^{(d_1)}, \ldots, \vbf^{ ( d_{\abs{\vertices}} )} } = \weightmat{o} \hidvec{1}{t} = \sum\nolimits_{d = 1}^{\hdim} \hidvec{1}{t}_d
\text{\,,}
\]
with $\hidvec{1}{t} = \hadmp_{j \in \neigh (t)} \brk{\weightmat{1} \vbf^{(d_j)} } = \hadmp_{j \in \neigh (t)} \brk{ \Ibf \vbf^{(d_j)} }$.
Since $ \vbf^{(d_j)}_{:\mindim} = \ebf^{(d_j)}$ for all $j \in \neigh (t)$ and $\Ibf$ is a zero-padded $\mindim \times \mindim$ identity matrix, it holds that: 
\[
\funcvert{\params}{\graph}{t} \brk1{ \vbf^{(d_1)}, \ldots, \vbf^{ ( d_{\abs{\vertices}} )} } = \sum\nolimits_{d = 1}^{\mindim} \prod\nolimits_{j \in \neigh (t)} \ebf^{(d_j)}_d
\text{\,.}
\]
For every $d \in [\mindim]$ we have that $\prod\nolimits_{j \in \neigh (t)} \ebf^{(d_j)}_d = 1$ if $d_{j} = d$ for all $j \in \neigh (t)$ and $\prod\nolimits_{j \in \neigh (t)} \ebf^{(d_j)}_d = 0$ otherwise.
This implies that $\funcvert{\params}{\graph}{t} \brk{ \vbf^{(d_1)}, \ldots, \vbf^{ ( d_{\abs{\vertices}} )} } = 1$ if $d_{j} = d_{j'}$ for all $j, j' \in \neigh (t)$ and $\funcvert{\params}{\graph}{t} \brk{ \vbf^{(d_1)}, \ldots, \vbf^{ ( d_{\abs{\vertices}} )} } = 0$ otherwise, for all $d_1, \ldots, d_{\abs{\vertices}} \in [\mindim]$.

\cref{gnn:eq:standard_basis_grid_vertex} implies that $\mat{ \gridtensor{ \funcvert{\params}{\graph}{t} } }{\I}$ has a sub-matrix of rank $\mindim$.
Specifically, such a sub-matrix can be obtained by examining all rows and columns of $\mat{ \gridtensor{ \funcvert{\params}{\graph}{t} } }{\I}$ corresponding to some fixed indices $\brk{ d_i \in [\mindim] }_{i \in \vertices \setminus \neigh (t)}$ for the vertices that are not neighbors of $t$.
Thus, $\rank \mat{ \gridtensor{ \funcvert{\params}{\graph}{t} } }{\I} \geq \mindim$.
Notice that necessarily $\nwalk{0}{\cut^*_t}{\{t\}} = 1$, as it is not zero and there can only be one length zero walk to $t$ (the trivial walk that starts and ends at $t$).
Recalling that $\alpha_{\cut^*_t, t} := \mindim$ for $L = 1$, we therefore conclude: 
\[
\log \brk*{ \rank \mat{ \gridtensor{ \funcvert{\params}{\graph}{t} } }{\I} } \geq \log (\mindim) = \log \brk{ \alpha_{\cut^*_t, t} } \cdot \nwalk{0}{\cut^*_t}{\{t\}}
\text{\,.}
\]

\textbf{Case of $L \geq 2$:}
Let $M :=\multisetcoeffnoflex{\mindim}{ \nwalk{L - 1}{\cut^*_t}{\{t\}}}= \binom{\mindim + \nwalk{L - 1}{\cut^*_t}{\{t\}} - 1}{ \nwalk{L - 1}{\cut^*_t}{\{t\}} }$ be the multiset coefficient of $\mindim$ and $ \nwalk{L - 1}{\cut^*_t}{\{t\}}$ (recall $\mindim := \min \brk[c]{ \indim, \hdim}$).
By~\cref{gnn:lem:hadm_of_gram_full_rank}, there exists $\Zbf \in \R_{ > 0}^{M \times \mindim}$ for which
\[
\rank \brk*{ \hadmp^{ \nwalk{L - 1}{ \cut^*_t }{ \{t\} } } \brk*{ \Zbf \Zbf^\top } } = \multisetcoeff{ \mindim }{ \nwalk{L - 1}{ \cut^*_t }{ \{t\} } }
\text{\,,}
\]
with $\hadmp^{ \nwalk{L - 1}{\cut^*_t}{\{t\}} } \brk{ \Zbf \Zbf^\top }$ standing for the $\nwalk{L - 1}{\cut^*_t}{\{t\}}$'th Hadamard power of $\Zbf \Zbf^\top$.
For this $\Zbf$, by~\cref{gnn:lem:grid_submatrix_hadm_of_gram_vertex} below we know that there exist weights $\params$ and template vectors such that $\mat{ \gridtensornoflex{ \funcvert{\params}{\graph}{t} } }{ \I }$ has an $M \times M$ sub-matrix of the form $\Sbf \brk{ \hadmp^{ \nwalk{L - 1}{\cut^*_t}{\{t\}} } \brk{ \Zbf \Zbf^\top } } \Qbf$, where $\Sbf, \Qbf \in \R^{M \times M}$ are full-rank diagonal matrices.
Since the rank of a matrix is at least the rank of any of its sub-matrices:
\[
\begin{split}
	\rank \brk*{ \mat{ \gridtensorbig{ \funcvert{\params}{\graph}{t} } }{ \I } } & \geq \rank \brk*{ \Sbf \brk*{ \hadmp^{ \nwalk{L - 1}{\cut^*_t}{\{t\}} } \brk*{ \Zbf \Zbf^\top } } \Qbf } \\
	& = \rank \brk*{ \hadmp^{ \nwalk{L - 1}{\cut^*_t}{\{t\}} } \brk*{ \Zbf \Zbf^\top } } \\
	& = \multisetcoeff{ \mindim }{ \nwalk{L - 1}{ \cut^*_t }{ \{t\}} }
	\text{\,,}
\end{split}
\]
where the second transition is due to $\Sbf$ and $\Qbf$ being full-rank.
Applying~\cref{gnn:lem:multiset_coeff_lower_bound} to lower bound the multiset coefficient, we have that:
\[
\rank \brk*{ \mat{ \gridtensorbig{ \funcvert{\params}{\graph}{t} } }{ \I } } \geq \multisetcoeff{ \mindim }{ \nwalk{L - 1}{ \cut^*_t }{ \{t\} } } \geq \brk*{ \frac{\mindim - 1}{ \nwalk{L - 1}{ \cut^*_t }{ \{t\} }} + 1 }^{ \nwalk{L - 1}{ \cut^*_t }{ \{t\} } }
\text{\,.}
\]
Taking the log of both sides while recalling that $\alpha_{\cut^*_t, t} := \brk{ \mindim - 1 } \cdot \nwalk{L - 1}{\cut^*_t}{\{t\}}^{-1} + 1$, we conclude that:
\[
\log \brk{ \rank \mat{ \gridtensor{ \funcvert{\params}{\graph}{t} } }{\I} } \geq \log \brk{ \alpha_{\cut^*_t, t} } \cdot \nwalk{L - 1}{\cut^*_t}{\{t\}}
\text{\,.}
\]

\begin{lemma}
	\label{gnn:lem:grid_submatrix_hadm_of_gram_vertex}
	Suppose that the GNN inducing $\funcvert{\params}{\graph}{t}$ is of depth $L \geq 2$ and $\nwalk{L - 1}{\cut^*_t}{\{t\}} > 0$.
	For any $M \in \N$ and matrix with positive entries $\Zbf \in \R_{> 0}^{M \times D}$, there exist weights $\params$ and $M + 1$ template vectors $\vbf^{(1)}, \ldots, \vbf^{(M + 1)} \in \R^{\indim}$~such that $\mat{ \gridtensornoflex{ \funcvert{\params}{\graph}{t} } }{ \I }$ has an $M \times M$ sub-matrix $\Sbf \brk{ \hadmp^{ \nwalk{L - 1}{\cut^*_t}{\{t\}} } \brk{ \Zbf \Zbf^\top } } \Qbf$, where $\Sbf, \Qbf \in \R^{M \times M}$ are full-rank diagonal matrices and $\hadmp^{ \nwalk{L - 1}{\cut^*_t}{\{t\}} } \brk{ \Zbf \Zbf^\top }$ is the $\nwalk{L - 1}{\cut^*_t}{\{t\}}$'th Hadamard power of $\Zbf \Zbf^\top$.
\end{lemma}
\begin{proof}
	Consider the weights $\params = \brk{ \weightmat{1}, \ldots, \weightmat{L}, \weightmat{o} }$ defined by:
	\[
	\begin{split}
		& \weightmat{1} := \Ibf \in \R^{\hdim \times \indim} \text{\,,} \\[0.2em]
		& \weightmat{2} := \begin{pmatrix} 1 & 1 & \cdots & 1 \\ 0 & 0 & \cdots & 0 \\ \vdots & \vdots & \cdots & \vdots \\ 0 & 0 & \cdots & 0\end{pmatrix} \in \R^{\hdim \times \hdim} \text{\,,} \\[0.2em]
		\forall l \in \{ 3, \ldots, L\} : ~&\weightmat{l} := \begin{pmatrix} 1 & 0 & \cdots & 0 \\ 0 & 0 & \cdots & 0 \\ \vdots & \vdots & \cdots & \vdots \\ 0 & 0 & \cdots & 0\end{pmatrix} \in \R^{\hdim \times \hdim} \text{\,,} \\[0.2em]
		& \weightmat{o} := \begin{pmatrix} 1 & 0 & \cdots & 0 \end{pmatrix} \in \R^{1 \times \hdim} \text{\,,}
	\end{split}
	\]
	where $\Ibf$ is a zero padded identity matrix, \ie~it holds ones on its diagonal and zeros elsewhere.
	We let the templates $\vbf^{(1)}, \ldots, \vbf^{(M)} \in \R^{\indim}$ be the vectors holding the respective rows of $\Zbf$ in their first $\mindim$ coordinates and zeros in the remaining entries (recall $\mindim := \min \brk[c]{ \indim, \hdim}$).
	That is, denoting the rows of $\Zbf$ by $\zbf^{(1)}, \ldots, \zbf^{(M)} \in \R_{> 0}^{\mindim}$, we let $\vbf^{(m)}_{: \mindim} := \zbf^{(m)}$ and $\vbf^{(m)}_{\mindim + 1: } := 0$ for all $m \in [M]$.
	We set all entries of the last template vector to one, \ie~$\vbf^{(M + 1)} := (1, \ldots, 1) \in \R^{\indim}$.
	
	Since $\cut^*_t \in \cutset (\I)$, \ie~it is an admissible subset of $\cut_\I$ (\cref{gnn:def:admissible_subsets}), there exist $\I' \subseteq \I, \J' \subseteq \I^c$ with no repeating shared neighbors (\cref{gnn:def:no_rep_neighbors}) such that $\cut^*_t = \neigh (\I') \cap \neigh ( \J' )$.
	Notice that $\I'$ and $\J'$ are non-empty as $\cut^*_t  \neq \emptyset$ (this is implied by $\nwalk{L - 1}{\cut^*_t}{\{t\}} > 0$).
	We focus on the $M \times M$ sub-matrix of $\mat{ \gridtensornoflex{ \funcvert{\params}{\graph}{t} } }{ \I }$ that includes only rows and columns corresponding to evaluations of $\funcvert{\params}{\graph}{t}$ where all variables indexed by $\I'$ are assigned the same template vector from $\vbf^{(1)}, \ldots, \vbf^{(M)}$, all variables indexed by $\J'$ are assigned the same template vector from $\vbf^{(1)}, \ldots, \vbf^{(M)}$, and all remaining variables are assigned the all-ones template vector $\vbf^{(M + 1)}$.
	Denoting this sub-matrix by $\Ubf \in \R^{M \times M}$, it therefore upholds:
	\[
	\Ubf_{m, n} = \funcvert{\params}{\graph}{t} \brk*{ \brk1{ \fvec{i} \leftarrow \vbf^{(m)} }_{i \in \I'} , \brk1{ \fvec{j} \leftarrow \vbf^{(n)}}_{j \in \J'} , \brk1{ \fvec{k} \leftarrow \vbf^{(M + 1)} }_{k \in \vertices \setminus \brk{ \I' \cup \J' } } }
	\text{\,,}
	\]
	for all $m, n \in [M]$, where we use $\brk{ \fvec{i} \leftarrow \vbf^{(m)} }_{i \in \I'}$ to denote that input variables indexed by $\I'$ are assigned the value $\vbf^{(m)}$.
	To show that $\Ubf$ obeys the form
	\[
	\Sbf \brk{ \hadmp^{ \nwalk{L - 1}{\cut^*_t}{ \{t\} } } \brk{ \Zbf \Zbf^\top } } \Qbf
	\]
	for full-rank diagonal $\Sbf, \Qbf \in \R^{M \times M}$, we prove there exist $\phi, \psi : \R^{\indim} \to \R_{> 0}$ such that $\Ubf_{m, n} = \phi \brk{ \vbf^{(m)} } \inprodnoflex{ \zbf^{(m)} }{ \zbf^{(n)} }^{ \nwalk{L - 1}{ \cut^*_t }{ \{t\}} } \psi \brk{ \vbf^{(n)} }$ for all $m, n \in [M]$.
	Indeed, defining $\Sbf$ to hold $\phi \brk{ \vbf^{(1)} }, \ldots, \phi \brk{ \vbf^{(M)} }$ on its diagonal and $\Qbf$ to hold $\psi \brk{ \vbf^{(1)} }, \ldots, \psi \brk{ \vbf^{(M)} }$ on its diagonal, we have that $\Ubf = \Sbf \brk{ \hadmp^{ \nwalk{L - 1}{\cut^*_t}{\{t\}} } \brk{ \Zbf \Zbf^\top } } \Qbf$.
	Since $\Sbf$ and $\Qbf$ are clearly full-rank (diagonal matrices with non-zero entries on their diagonal), the proof concludes.
	
	\medskip
	
	For $m, n \in [M]$, let $\hidvec{l}{i} \in \R^{\hdim}$ be the hidden embedding for $i \in \vertices$ at layer $l \in [L]$ of the GNN inducing $\funcvert{\params}{\graph}{t}$, over the following assignment to its input variables (\ie~vertex features):
	\[
	\brk1{ \fvec{i} \leftarrow \vbf^{(m)} }_{i \in \I'} , \brk1{ \fvec{j} \leftarrow \vbf^{(n)}}_{j \in \J'} , \brk1{ \fvec{k} \leftarrow \vbf^{(M + 1)} }_{k \in \vertices \setminus \brk{ \I' \cup \J' } }
	\text{\,.}
	\]
	Invoking~\cref{gnn:lem:grid_tensor_entry_hadm_induction} with $\vbf^{(m)}, \vbf^{(n)}, \I',$ and $\J'$, it holds that:
	\[
	\hidvec{L}{t}_{1} = \phi^{(L, t)} \brk1{ \vbf^{(m)} } \inprodbig{ \zbf^{(m)} }{ \zbf^{(n)} }^{ \nwalk{L - 1}{ \cut^*_t }{ \{t\} } } \psi^{(L, t)} \brk1{ \vbf^{(n)} } ~~~ , ~~~ \forall d \in \{ 2, \ldots, \hdim \} :~\hidvec{L}{t}_d = 0
	\text{\,,}
	\]
	for some $\phi^{(L, t)}, \psi^{(L, t)} : \R^{\indim} \to \R_{ > 0}$.
	Since
	\[
	\begin{split}
		\Ubf_{m, n} & = \funcvert{\params}{\graph}{t} \brk*{ \brk1{ \fvec{i} \leftarrow \vbf^{(m)} }_{i \in \I'} , \brk1{ \fvec{j} \leftarrow \vbf^{(n)}}_{j \in \J'} , \brk1{ \fvec{k} \leftarrow \vbf^{(M + 1)} }_{k \in \vertices \setminus \brk{ \I' \cup \J' } } } \\
		& = \weightmat{o} \hidvec{L}{t}
	\end{split}
	\]
	and $\weightmat{o} = \brk{1, 0, \ldots, 0}$, this implies that:
	\[
	\Ubf_{m, n} = \hidvec{L}{t}_1 = \phi^{(L, t)} \brk1{ \vbf^{(m)} } \inprodbig{ \zbf^{(m)} }{ \zbf^{(n)} }^{ \nwalk{L - 1}{ \cut^*_t }{ \{t\} } } \psi^{(L, t)} \brk1{ \vbf^{(n)} } 
	\text{\,.}
	\]
	Defining $\phi := \phi^{(L, t)}$ and $\psi := \psi^{(L, t)}$ leads to the sought-after expression for $\Ubf_{m, n}$:
	\[
	\Ubf_{m, n} = \phi \brk1{ \vbf^{(m)} } \inprodbig{ \zbf^{(m)} }{ \zbf^{(n)} }^{ \nwalk{L - 1}{ \cut^*_t }{ \{t\} } } \psi \brk1{ \vbf^{(n)} } 
	\text{\,.}
	\]
\end{proof}

\subsubsection{Technical Lemmas}
\label{gnn:app:proofs:sep_rank_lower_bound:lemmas}

For completeness, we include the \emph{vector rearrangement inequality} from~\cite{levine2018benefits}, which we employ for proving the subsequent~\cref{gnn:lem:hadm_of_gram_full_rank}.

\begin{lemma}[Lemma~1 from~\cite{levine2018benefits}]
	\label{gnn:lem:vector_rearrangement}
	Let $\abf^{(1)}, \ldots, \abf^{(M)} \in \R_{\geq 0}^D$ be $M \in \N$ different vectors with non-negative entries.
	Then, for any permutation $\sigma : [M] \to [M]$ besides the identity permutation it holds that:
	\[
	\sum\nolimits_{m = 1}^M \inprod{ \abf^{(m)} }{ \abf^{(\sigma(m))} } < \sum\nolimits_{m = 1}^M \normbig{ \abf^{(m)} }^2
	\text{\,.}
	\]
\end{lemma}

Taking the $P$'th Hadamard power of a rank at most $D$ matrix results in a matrix whose rank is at most the multiset coefficient $\multisetcoeffnoflex{D}{P} := \binom{D + P - 1}{P}$ (see, \eg, Theorem~1 in~\cite{amini2011low}).
\cref{gnn:lem:hadm_of_gram_full_rank}, adapted from Appendix~B.2 in~\cite{levine2020limits}, guarantees that we can always find a $\multisetcoeffnoflex{D}{P} \times D$ matrix $\Zbf$ with positive entries such that $\rank \brk{ \hadmp^P \brk*{ \Zbf \Zbf^\top } }$ is maximal, \ie~equal to $\multisetcoeffnoflex{D}{P}$.

\begin{lemma}[adapted from Appendix~B.2 in~\cite{levine2020limits}]
	\label{gnn:lem:hadm_of_gram_full_rank}
	For any $D \in \N$ and $P \in \N_{\geq 0}$, there exists a matrix with positive entries $\Zbf \in \R_{> 0}^{\multisetcoeff{D}{P} \times D}$ for which:
	\[
	\rank \brk*{ \hadmp^P \brk*{ \Zbf \Zbf^\top } } = \multisetcoeff{D}{P}
	\text{\,,}
	\]
	where $\hadmp^{ P } \brk{ \Zbf \Zbf^\top }$ is the $P$'th Hadamard power of $\Zbf \Zbf^\top$.
\end{lemma}
\begin{proof}
	We let $M := \multisetcoeffnoflex{D}{P}$ for notational convenience.
	Denote by $\zbf^{(1)}, \ldots, \zbf^{( M )} \in \R^D$ the row vectors of $\Zbf \in \R_{> 0}^{M \times D}$.
	Observing the $(m, n)$'th entry of $\hadmp^P \brk*{ \Zbf \Zbf^\top }$:
	\[
	\brk[s]*{ \hadmp^P \brk*{ \Zbf \Zbf^\top } }_{m, n} = \inprod{ \zbf^{(m)} }{ \zbf^{(n)} }^P = \brk*{ \sum\nolimits_{d = 1}^D \zbf^{(m)}_d \cdot \zbf^{(n)}_d }^P
	\text{\,,}
	\]
	by expanding the power using the multinomial identity we have that:
	\be
	\begin{split}
		\brk[s]*{ \hadmp^P \brk*{ \Zbf \Zbf^\top } }_{m, n} & = \sum_{ \substack{ q_1, \ldots, q_D \in \N_{\geq 0} \\[0.1em] \text{s.t. } \sum\nolimits_{d = 1}^D q_d = P } } \binom{ P }{ q_1, \ldots, q_D } \prod_{d = 1}^D \brk*{ \zbf^{(m)}_d \cdot \zbf^{(n)}_d }^{q_d} \\
		& = \sum_{ \substack{ q_1, \ldots, q_D \in \N_{\geq 0} \\[0.1em] \text{s.t. } \sum\nolimits_{d = 1}^D q_d = P } } \binom{ P }{ q_1, \ldots, q_D } \brk*{ \prod_{d = 1}^D \brk*{ \zbf^{(m)}_d }^{q_d} } \cdot \brk*{ \prod_{d = 1}^D \brk*{ \zbf^{(n)}_d }^{q_d} }
		\text{\,,}
	\end{split}
	\label{gnn:eq:gram_hadmp_entry}
	\ee
	where in the last equality we separated terms depending on $m$ from those depending on $n$.
	
	Let $\brk{ \abf^{ (q_1, \ldots, q_D) } \in \R^M }_{q_1, \ldots, q_D \in \N_{\geq 0} \text{ s.t } \sum\nolimits_{d = 1}^D q_d = P }$ be $M$ vectors that, for all $q_1, \ldots, q_D \in \N_{\geq 0}$ satisfying $\sum\nolimits_{d = 1}^D q_d = P$ and $m \in [M]$, are defined by $\abf^{(q_1, \ldots, q_D)}_{m} = \prod_{d = 1}^D \brk1{ \zbf^{(m)}_d }^{q_d}$.
	As can be seen from~\cref{gnn:eq:gram_hadmp_entry}, we can write:
	\[\
	\hadmp^P \brk*{ \Zbf \Zbf^\top } = \Abf \Sbf \Abf^\top
	\text{\,,}
	\]
	where $\Abf \in \R^{M \times M}$ is the matrix whose columns are $\brk{ \abf^{ (q_1, \ldots, q_D) } }_{q_1, \ldots, q_D \in \N_{\geq 0} \text{ s.t } \sum\nolimits_{d = 1}^D q_d = P }$ and $\Sbf \in \R^{M \times M}$ is the diagonal matrix holding $\binom{P}{q_1, \ldots, q_D}$ for every $q_1, \ldots, q_D \in \N_{\geq 0}$ satisfying $\sum\nolimits_{d = 1}^D q_d = P$ on its diagonal.
	Since all entries on the diagonal of $\Sbf$ are positive, it is of full-rank, \ie~$\rank (\Sbf) = M$.
	Thus, to prove that there exists $\Zbf \in \R_{> 0}^{M \times D}$ for which $\rank \brk{ \hadmp^P \brk{ \Zbf \Zbf^\top } } = M$, it suffices to show that we can choose $\zbf^{(1)}, \ldots, \zbf^{(M)}$ with positive entries inducing $\rank \brk{ \Abf } = M$, for $\Abf$ as defined above.
	Below, we complete the proof by constructing such $\zbf^{(1)}, \ldots, \zbf^{(M)}$.
	
	We associate each of $\zbf^{(1)}, \ldots, \zbf^{(M)}$ with a different configuration from the set:
	\[
	\brk[c]*{ \qbf = \brk*{ q_1, \ldots, q_D } : q_1, \ldots, q_D \in \N_{\geq 0}  ~,~ \sum\nolimits_{d = 1}^D q_d = P }
	\text{\,,}
	\]
	where note that this set contains $M = \multisetcoeffnoflex{D}{P}$ elements.
	For $m \in [M]$, denote by $\qbf^{(m)}$ the configuration associated with $\zbf^{(m)}$.
	For a variable $\gamma \in \R$, to be determined later on, and every $m \in [M]$ and $d \in [D]$, we set:
	\[
	\zbf^{(m)}_d = \gamma^{\qbf_d^{(m)}}
	\text{\,.}
	\]
	Given these $\zbf^{(1)}, \ldots, \zbf^{(M)}$, the entries of $\Abf$ have the following form:
	\[
	\Abf_{m, n} = \prod\nolimits_{d = 1}^D \brk2{ \zbf^{(m)}_d }^{\qbf^{(n)}_d} = \prod\nolimits_{d = 1}^D \brk2{ \gamma^{\qbf_d^{(m)}} }^{\qbf^{(n)}_d} = \gamma^{ \sum\nolimits_{d = 1}^D \qbf_d^{(m)} \cdot \qbf_d^{(n)} } = \gamma^{ \inprod{\qbf^{(m)} }{ \qbf^{(n)} } } 
	\text{\,,}
	\]
	for all $m, n \in [M]$.
	Thus, $\det \brk{ \Abf } = \sum\nolimits_{\text{permutation } \sigma : [M] \to [M]} \sign (\sigma) \cdot \gamma^{ \sum\nolimits_{m = 1}^M \inprod{ \qbf^{(m)} }{ \qbf^{( \sigma (m) )} } }$ is polynomial in~$\gamma$.
	By~\cref{gnn:lem:vector_rearrangement},~$\sum\nolimits_{m = 1}^M \inprod{ \qbf^{(m)} }{ \qbf^{( \sigma (m) )} } < \sum\nolimits_{m = 1}^M \norm{ \qbf^{(m)} }^2$ for all $\sigma$ which is not the identity permutation.
	This implies that $\sum\nolimits_{m = 1}^M \norm{ \qbf^{(m)} }^2$ is the maximal degree of a monomial in $\det (\Abf)$, and it is attained by a single element in $\sum\nolimits_{\text{permutation } \sigma : [M] \to [M]} \sign (\sigma) \cdot \gamma^{ \sum\nolimits_{m = 1}^M \inprod{ \qbf^{(m)} }{ \qbf^{( \sigma (m) )} } }$~---~that corresponding to the identity permutation.
	Consequently, $\det (\Abf)$ cannot be the zero polynomial with respect to $\gamma$, and so it vanishes only on a finite set of values for $\gamma$.
	In particular, there exists $\gamma > 0$ such that $\det (\Abf) \neq 0$, meaning $\rank \brk{ \Abf } = M$.
	The proof concludes by noticing that for a positive $\gamma$ the entries of the chosen $\zbf^{(1)}, \ldots, \zbf^{(M)}$ are positive as well.
\end{proof}

Additionally, we make use of the following lemmas.

\begin{lemma}
	\label{gnn:lem:multiset_coeff_lower_bound}
	For any $D, P \in \N$, let $\multisetcoeff{D}{P} := \binom{D + P - 1}{P}$ be the multiset coefficient.
	Then:
	\[
	\multisetcoeff{D}{P} \geq \brk*{ \frac{D - 1}{P} + 1 }^{P}
	\text{\,.}
	\]
\end{lemma}
\begin{proof}
	For any $N \geq K \in \N$, a known lower bound on the binomial coefficient is $\binom{N}{K} \geq \brk*{ \frac{N}{K} }^K$.
	Hence:
	\[
	\multisetcoeff{D}{P} = \binom{D + P - 1}{P} \geq \brk*{ \frac{D + P - 1}{P} }^P = \brk*{ \frac{D - 1}{P} + 1 }^{P}
	\text{\,.}
	\] 
\end{proof}

\begin{lemma}
	\label{gnn:lem:poly_mat_max_rank}
	For $D_1, D_2, K \in \N$, consider a polynomial function mapping variables $\params \in \R^K$ to matrices $\Abf (\params) \in \R^{D_1 \times D_2}$, \ie~the entries of $\Abf (\params)$ are polynomial in $\params$.
	If there exists a point $\params^* \in \R^K$ such that $\rank \brk{ \Abf (\params^*) } \geq R$, for $R \in [ \min \brk[c]{ D_1, D_2 } ]$, then the set $\brk[c]{ \params \in \R^K : \rank \brk{ \Abf (\params) } < R }$ has Lebesgue measure zero.
\end{lemma}
\begin{proof}
	A matrix is of rank at least $R$ if and only if it has a $R \times R$ sub-matrix whose determinant is non-zero.
	The determinant of any sub-matrix of $\Abf (\params)$ is polynomial in the entries of $\Abf (\params)$, and so it is polynomial in $\params$ as well.
	Since the zero set of a polynomial is either the entire space or a set of Lebesgue measure zero~\cite{caron2005zero}, the fact that $\rank \brk{ \Abf (\params^*) } \geq R$ implies that $\brk[c]{ \params \in \R^K : \rank \brk{ \Abf (\params) } < R }$ has Lebesgue measure zero.
\end{proof}

\begin{lemma}
	\label{gnn:lem:grid_tensor_entry_hadm_induction}
	Let $\vbf, \vbf' \in \R_{\geq 0}^{\indim}$ whose first $\mindim := \min \{ \indim, \hdim\}$ entries are positive, and disjoint $\I', \J' \subseteq \vertices$ with no repeating shared neighbors (\cref{gnn:def:no_rep_neighbors}).
	Denote by $\hidvec{l}{i} \in \R^{\hdim}$ the hidden embedding for $i \in \vertices$ at layer $l \in [L]$ of a GNN with depth $L \geq 2$ and product aggregation (\cref{gnn:eq:gnn_update,gnn:eq:prod_gnn_agg}), given the following assignment to its input variables (\ie~vertex features):
	\[
	\brk1{ \fvec{i} \leftarrow \vbf }_{i \in \I'} , \brk1{ \fvec{j} \leftarrow \vbf' }_{j \in \J'} , \brk1{ \fvec{k} \leftarrow \1 }_{k \in \vertices \setminus \brk{ \I' \cup \J' } }
	\text{\,,}
	\]
	where $\1 \in \R^{\indim}$ is the vector holding one in all entries.
	Suppose that the weights $\weightmat{1}, \ldots, \weightmat{L}$ of the GNN are given by:
	\[
	\begin{split}
		& \weightmat{1} := \Ibf \in \R^{\hdim \times \indim} \text{\,,} \\[0.2em]
		& \weightmat{2} := \begin{pmatrix} 1 & 1 & \cdots & 1 \\ 0 & 0 & \cdots & 0 \\ \vdots & \vdots & \cdots & \vdots \\ 0 & 0 & \cdots & 0\end{pmatrix} \in \R^{\hdim \times \hdim} \text{\,,} \\[0.2em]
		\forall l \in \{ 3, \ldots, L\} : ~&\weightmat{l} := \begin{pmatrix} 1 & 0 & \cdots & 0 \\ 0 & 0 & \cdots & 0 \\ \vdots & \vdots & \cdots & \vdots \\ 0 & 0 & \cdots & 0\end{pmatrix} \in \R^{\hdim \times \hdim} \text{\,,}
	\end{split}
	\]
	where $\Ibf$ is a zero padded identity matrix, \ie~it holds ones on its diagonal and zeros elsewhere.
	Then, for all $l \in \{ 2, \ldots, L\}$ and $i \in \vertices$, there exist $\phi^{(l, i)}, \psi^{(l, i)} : \R^{\indim} \to \R_{ > 0}$ such that:
	\[
	\hidvec{l}{i}_{1} = \phi^{(l, i)} \brk{ \vbf } \inprod{ \vbf_{:\mindim} }{ \vbf'_{:D} }^{ \nwalk{l - 1}{ \cut }{ \{i\} } } \psi^{(l, i)} \brk{ \vbf' } \quad , \quad \forall d \in \{ 2, \ldots, \hdim \} :~\hidvec{l}{i}_d = 0
	\text{\,,}
	\]
	where $\cut := \neigh (\I') \cap \neigh (\J')$.
\end{lemma}
\begin{proof}
	The proof is by induction over the layer $l \in \{2, \ldots, L\}$.
	For $l = 2$, fix $i \in \vertices$.
	By the update rule of a GNN with product aggregation:
	\[
	\hidvec{2}{i} = \hadmp_{j \in \neigh (i) } \brk1{ \weightmat{2} \hidvec{1}{j} }
	\text{\,.}
	\]
	Plugging in the value of $\weightmat{2}$ we get:
	\be
	\hidvec{2}{i}_{1} = \prod\nolimits_{ j \in \neigh (i) } \brk2{ \sum\nolimits_{d = 1}^{\hdim} \hidvec{1}{j}_d } \quad , \quad \forall d \in \{ 2, \ldots, \hdim \} :~\hidvec{2}{i}_d = 0
	\text{\,.}
	\label{gnn:eq:hid_vec_l2_entries_inter}
	\ee
	Let $\bar{\vbf}, \bar{\vbf}' \in \R^{\hdim}$ be the vectors holding $\vbf_{:\mindim}$ and $\vbf'_{:\mindim}$ in their first $\mindim$ coordinates and zero in the remaining entries, respectively.
	Similarly, we use $\bar{\1} \in \R^{\hdim}$ to denote the vector whose first $\mindim$ entries are one and the remaining are zero.
	Examining $\hidvec{1}{j}$ for $j \in \neigh (i)$, by the assignment of input variables and the fact that $\weightmat{1}$ is a zero padded identity matrix we have that:
	\[
	\begin{split}
		\hidvec{1}{j} & = \hadmp_{ k \in \neigh (j) } \brk1{ \weightmat{1} \fvec{k} } = \brk1{ \hadmp^{ \abs{ \neigh (j) \cap \I' } } \bar{\vbf} } \hadmp \brk1{ \hadmp^{ \abs{ \neigh (j) \cap \J' } } \bar{\vbf}' } \hadmp \brk1{ \hadmp^{ \abs{ \neigh (j) \setminus \brk{ \I' \cup \J' } } } \bar{\1} } \\
		& = \brk1{ \hadmp^{ \abs{ \neigh (j) \cap \I' } } \bar{\vbf} } \hadmp \brk1{ \hadmp^{ \abs{ \neigh (j) \cap \J' } } \bar{\vbf}' }
		\text{\,.}
	\end{split}
	\]
	Since the first $\mindim$ entries of $\bar{\vbf}$ and $\bar{\vbf}'$ are positive while the rest are zero, the same holds for $\hidvec{1}{j}$.
	Additionally, recall that $\I'$ and $\J'$ have no repeating shared neighbors.
	Thus, if $j \in \neigh (\I') \cap \neigh (\J') = \cut$, then $j$ has a single neighbor in $\I'$ and a single neighbor in $\J'$, implying $\hidvec{1}{j} = \bar{\vbf} \hadmp \bar{\vbf}'$.
	Otherwise, if $j \notin \cut$, then $\neigh (j) \cap \I' = \emptyset$ or $\neigh (j) \cap \J' = \emptyset$ must hold.
	In the former $\hidvec{1}{j}$ does not depend on $\vbf$, whereas in the latter $\hidvec{1}{j}$ does not depend on $\vbf'$.
	
	Going back to~\cref{gnn:eq:hid_vec_l2_entries_inter}, while noticing that $\abs{ \neigh (i) \cap \cut } = \nwalk{1}{\cut}{ \{i\} }$, we arrive at:
	\[
	\begin{split}
		\hidvec{2}{i}_{1} & = \prod\nolimits_{ j \in \neigh (i) \cap \cut } \brk2{ \sum\nolimits_{d = 1}^{\hdim} \hidvec{1}{j}_d } \cdot \prod\nolimits_{ j \in \neigh (i) \setminus \cut } \brk2{ \sum\nolimits_{d = 1}^{\hdim} \hidvec{1}{j}_d } \\
		& = \prod\nolimits_{ j \in \neigh (i) \cap \cut } \brk2{ \sum\nolimits_{d = 1}^{\hdim} \brk[s]1{ \bar{\vbf} \hadmp \bar{\vbf}' }_d } \cdot \prod\nolimits_{ j \in \neigh (i) \setminus \cut } \brk2{ \sum\nolimits_{d = 1}^{\hdim} \hidvec{1}{j}_d } \\
		& = \inprod{ \vbf_{:D} }{ \vbf'_{:D} }^{\nwalk{1}{\cut}{\{i\}}} \cdot \prod\nolimits_{ j \in \neigh (i) \setminus \cut } \brk2{ \sum\nolimits_{d = 1}^{\hdim} \hidvec{1}{j}_d }
		\text{\,.}
	\end{split}
	\]
	As discussed above, for each $j \in \neigh (i) \setminus \cut$ the hidden embedding $\hidvec{1}{j}$ does not depend on $\vbf$ or it does not depend on $\vbf'$.
	Furthermore, $\sum\nolimits_{d = 1}^{\hdim} \hidvec{1}{j}_d > 0$ for all $j \in \neigh (i)$.
	Hence, there exist $\phi^{(2, i)}, \psi^{(2, i)} : \R^{\indim} \to \R_{> 0}$ such that:
	\[
	\hidvec{2}{i}_{1} = \phi^{(2, i)} \brk{ \vbf } \inprod{ \vbf_{:\mindim} }{ \vbf'_{:D} }^{ \nwalk{1}{ \cut }{ \{i\} } } \psi^{(2, i)} \brk{ \vbf' }
	\text{\,,}
	\]
	completing the base case.
	
	Now, assuming that the inductive claim holds for $l - 1 \geq 2$, we prove that it holds for $l$.
	Let $i \in \vertices$.
	By the update rule of a GNN with product aggregation $\hidvec{l}{i} = \hadmp_{j \in \neigh (i)} \brk{ \weightmat{l} \hidvec{l -1}{j} }$.
	Plugging in the value of $\weightmat{l}$ we get:
	\[
	\hidvec{l}{i}_{1} = \prod\nolimits_{ j \in \neigh (i) } \hidvec{l - 1}{j}_1 \quad , \quad \forall d \in \{ 2, \ldots, \hdim \} :~\hidvec{l}{i}_d = 0
	\text{\,.}
	\]
	By the inductive assumption $\hidvec{l - 1}{j}_1 = \phi^{(l - 1, j)} \brk{ \vbf } \inprod{ \vbf_{:\mindim} }{ \vbf'_{:D} }^{ \nwalk{l - 2}{ \cut }{ \{j\} } } \psi^{(l - 1, j)} \brk{ \vbf' }$ for all $j \in \neigh (i)$, where $\phi^{(l - 1, j)}, \psi^{(l - 1, j)} : \R^{\indim} \to \R_{> 0}$.
	Thus:
	\[
	\begin{split}
		\hidvec{l}{i}_{1} & = \prod\nolimits_{ j \in \neigh (i) } \hidvec{l - 1}{j}_1 \\
		& = \prod\nolimits_{j \in \neigh (i) } \phi^{(l - 1, j)} \brk{ \vbf } \inprod{ \vbf_{:\mindim} }{ \vbf'_{:D} }^{ \nwalk{l - 2}{ \cut }{ \{j\} } } \psi^{(l - 1, j)} \brk{ \vbf' } \\
		& = \brk*{ \prod\nolimits_{j \in \neigh (i) } \phi^{(l - 1, j)} \brk{ \vbf } } \cdot \inprod{ \vbf_{:\mindim} }{ \vbf'_{:D} }^{ \sum\nolimits_{j \in \neigh (i) } \nwalk{l - 2}{\cut}{ \{j\} } } \cdot \brk*{ \prod\nolimits_{j \in \neigh (i) } \psi^{(l - 1, j)} \brk{ \vbf' } }
		\text{\,.}
	\end{split}
	\]
	Define $\phi^{(l, i)} : \vbf \mapsto \prod\nolimits_{j \in \neigh (i) } \phi^{(l - 1, j)} \brk{ \vbf }$ and $\psi^{(l, i)} : \vbf' \mapsto \prod\nolimits_{j \in \neigh (i) } \psi^{(l - 1, j)} \brk{ \vbf' }$.
	Since the range of $\phi^{(l - 1, j)}$ and $\psi^{(l - 1, j)}$ is $\R_{> 0}$ for all $j \in \neigh (i)$, so is the range of $\phi^{(l, i)}$ and $\psi^{(l, i)}$.
	The desired result thus readily follows by noticing that $\sum\nolimits_{j \in \neigh (i) } \nwalk{l - 2}{\cut}{ \{j\} } = \nwalk{l - 1}{ \cut}{ \{i\}}$:
	\[
	\hidvec{l}{i}_{1} = \phi^{(l, i)} \brk{ \vbf } \inprod{ \vbf_{:\mindim} }{ \vbf'_{:D} }^{ \nwalk{l - 1}{ \cut }{ \{i\} } } \psi^{(l, i)} \brk{ \vbf' }
	\text{\,.}
	\]
\end{proof}

\subsection{Proof of~\cref{gnn:thm:directed_sep_rank_upper_bound}}
\label{gnn:app:proofs:directed_sep_rank_upper_bound}

The proof follows a line identical to that of~\cref{gnn:thm:sep_rank_upper_bound} (\cref{gnn:app:proofs:sep_rank_upper_bound}), requiring only slight adjustments.
We outline the necessary changes.

Extending the tensor network representations of GNNs with product aggregation to directed graphs and multiple edge types is straightforward.
Nodes, legs, and leg weights are as described in~\cref{gnn:app:prod_gnn_as_tn} for undirected graphs with a single edge type, except that:
\begin{itemize}[leftmargin=2em]
	\vspace{-1mm}
	\item Legs connecting $\delta$-tensors with weight matrices in the same layer are adapted such that only incoming neighbors are considered.
	Formally, in~\cref{gnn:eq:graphtensor_tn,gnn:eq:vertextensor_tn}, $\neigh (i)$ is replaced by $\neighin (i)$ in the leg definitions, for $i \in \vertices$.
	
	\item Weight matrices $\brk{ \weightmat{l, q} }_{l \in [L] , q \in [Q]}$ are assigned to nodes in accordance with edge types.
	Namely, if at layer $l \in [L]$ a $\delta$-tensor associated with $i \in \vertices$ is connected to a weight matrix associated with $j \in \neighin (i)$, then $\weightmat{l, \edgetypemap (j, i)}$ is assigned to the weight matrix node, as opposed to $\weightmat{l}$ in the single edge type setting.
	Formally, let $\weightmatdup{l}{j}{\gamma}$ be a node at layer $l \in [L]$ connected to $\deltatensordup{l}{i}{\gamma'}$, for $i \in \vertices, j \in \neighin (i)$, and some $\gamma, \gamma' \in \N$.
	Then, $\weightmatdup{l}{j}{\gamma}$ stands for a copy of $\weightmat{l, \edgetypemap (j, i)}$.
\end{itemize}

For $\fmat = \brk{\fvec{1}, \ldots, \fvec{\abs{\vertices}} } \in \R^{\indim \times \abs{\vertices} }$, let $\tngraph (\fmat)$ and $\tngraph^{(t)} (\fmat)$ be the tensor networks corresponding to $\funcgraph{\params}{\graph} (\fmat)$ and $\funcvert{\params}{\graph}{t} (\fmat)$, respectively, whose construction is outlined above.
Then,~\cref{gnn:lem:min_cut_tn_sep_rank_ub} (from Appendix~\ref{gnn:app:proofs:sep_rank_upper_bound:min_cut_tn_ub}) and its proof apply as stated.
Meaning, $\sepranknoflex{\funcgraph{\params}{\graph} }{\I}$ and $\sepranknoflex{ \funcvert{\params}{\graph}{t} }{\I}$ are upper bounded by the minimal modified multiplicative cut weights in $\tngraph (\fmat)$ and $\tngraph^{(t)} (\fmat)$, respectively, among cuts separating leaves associated with vertices of the input graph in~$\I$ from leaves associated with vertices of the input graph in~$\I^c$.
Therefore, to establish~\cref{gnn:eq:directed_sep_rank_upper_bound_graph_pred,gnn:eq:directed_sep_rank_upper_bound_vertex_pred}, it suffices to find cuts in the respective tensor networks with sufficiently low modified multiplicative weights.
As is the case for undirected graphs with a single edge type (see Appendices~\ref{gnn:app:proofs:sep_rank_upper_bound:graph} and \ref{gnn:app:proofs:sep_rank_upper_bound:vertex}), the cuts separating nodes corresponding to vertices in~$\I$ from all other nodes yield the desired upper bounds. 
\qed

\subsection{Proof of~\cref{gnn:thm:directed_sep_rank_lower_bound}}
\label{gnn:app:proofs:directed_sep_rank_lower_bound}

The proof follows a line identical to that of~\cref{gnn:thm:sep_rank_lower_bound} (\cref{gnn:app:proofs:sep_rank_lower_bound}), requiring only slight adjustments.
We outline the necessary changes.

In the context of graph prediction, let $\cut^* \in \argmax_{ \cut \in \cutsetdir (\I) } \log \brk{ \alpha_{\cut} } \cdot \nwalk{L - 1}{\cut}{\vertices}$.
By~\cref{gnn:lem:grid_tensor_sep_rank_lb} (from~\cref{gnn:app:proofs:sep_rank_lower_bound}), to prove that~\cref{gnn:eq:directed_sep_rank_lower_bound_graph_pred} holds for weights~$\params$, it suffices to find template vectors for which $\log \brk{ \rank \mat{ \gridtensor{ \funcgraph{\params}{\graph} } }{\I} } \geq \log \brk{ \alpha_{\cut^*} } \cdot \nwalk{L - 1}{\cut^*}{\vertices}$.
Notice that, since the outputs of $\funcgraph{\params}{\graph}$ vary polynomially with~$\params$, so do the entries of~$\mat{ \gridtensor{ \funcgraph{\params}{\graph} } }{\I}$ for any choice of template vectors.
Thus, according to~\cref{gnn:lem:poly_mat_max_rank} (from Appendix~\ref{gnn:app:proofs:sep_rank_lower_bound:lemmas}), by constructing weights~$\params$ and template vectors satisfying $\log \brk{ \rank \mat{ \gridtensor{ \funcgraph{\params}{\graph} } }{\I} } \geq \log \brk{ \alpha_{\cut^*} } \cdot \nwalk{L - 1}{\cut^*}{\vertices}$, we may conclude that this is the case for almost all assignments of weights, meaning~\cref{gnn:eq:directed_sep_rank_lower_bound_graph_pred} holds for almost all assignments of weights.
For undirected graphs with a single edge type, Appendix~\ref{gnn:app:proofs:sep_rank_lower_bound:graph} provides such weights $\weightmat{1}, \ldots, \weightmat{L}, \weightmat{o}$ and template vectors.
The proof in the case of directed graphs with multiple edge types is analogous, requiring only a couple adaptations: \emph{(i)}~weight matrices of all edge types at layer $l \in [L]$ are set to the $\weightmat{l}$ chosen in Appendix~\ref{gnn:app:proofs:sep_rank_lower_bound:graph}; and \emph{(ii)}~$\cut_\I$ and $\cutset (\I)$ are replaced with their directed counterparts $\cutdir_\I$ and $\cutsetdir (\I)$, respectively.

In the context of vertex prediction, let $\cut_t^* \in \argmax_{ \cut \in \cutsetdir (\I) } \log \brk{ \alpha_{\cut, t} } \cdot \nwalk{L - 1}{\cut}{ \{ t \} }$.
Due to arguments similar to those above, to prove that~\cref{gnn:eq:directed_sep_rank_lower_bound_vertex_pred} holds for almost all assignments of weights, we need only find weights $\params$ and template vectors satisfying $\log \brk{ \rank \mat{ \gridtensor{ \funcvert{\params}{\graph}{t} } }{\I} } \geq \log \brk{ \alpha_{\cut^*_t, t} } \cdot \nwalk{L - 1}{\cut^*_t}{ \{ t \} }$.
For undirected graphs with a single edge type, Appendix~\ref{gnn:app:proofs:sep_rank_lower_bound:vertex} provides such weights and template vectors.
The adaptations necessary to extend Appendix~\ref{gnn:app:proofs:sep_rank_lower_bound:vertex} to directed graphs with multiple edge types are identical to those specified above for extending Appendix~\ref{gnn:app:proofs:sep_rank_lower_bound:graph} in the context of graph prediction.

Lastly, recalling that a finite union of measure zero sets has measure zero as well establishes that~\cref{gnn:eq:directed_sep_rank_lower_bound_graph_pred,gnn:eq:directed_sep_rank_lower_bound_vertex_pred} jointly hold for almost all assignments of weights.
\qed

\subsection{Proof of~\cref{gnn:prop:tn_constructions}}
\label{gnn:app:proofs:tn_constructions}

We first prove that the contractions described by $\tngraph (\fmat)$ produce $\funcgraph{\params}{\graph} (\fmat)$.
Through an induction over the layer $l \in [L]$, for all $i \in \vertices$ and $\gamma \in [ \nwalk{L - l}{ \{i\} }{\vertices} ]$ we show that contracting the sub-tree whose root is $\deltatensordup{l}{i}{\gamma}$ yields $\hidvec{l}{i}$~---~the hidden embedding for $i$ at layer $l$ of the GNN inducing $\funcgraph{\params}{\graph}$, given vertex features $\fvec{1}, \ldots, \fvec{\abs{\vertices}}$.

For $l = 1$, fix some $i \in \vertices$ and $\gamma \in [ \nwalk{L - 1}{ \{i\} }{\vertices} ]$.
The sub-tree whose root is $\deltatensordup{1}{i}{\gamma}$ comprises $\abs{ \neigh (i ) }$ copies of $\weightmat{1}$, each associated with some $j \in \neigh (i)$ and contracted in its second mode with a copy of $\fvec{ j }$.
Additionally, $\deltatensordup{1}{i}{\gamma}$, which is a copy of $\deltatensor{ \abs{ \neigh (i) } +1 }$, is contracted with the copies of $\weightmat{1}$ in their first mode.
Overall, the execution of all contractions in the sub-tree can be written as $\deltatensor{ \abs{ \neigh (i) } + 1} \contract{ j \in [ \abs{ \neigh (i) } ] } \brk{ \weightmat{ 1 } \fvec{ \neigh (i)_j }}$, where $\neigh (i)_j$, for $j \in [\abs{\neigh (i)}]$, denotes the $j$'th neighbor of $i$ according to an ascending order (recall vertices are represented by indices from $1$ to $\abs{\vertices}$).
The base case concludes by~\cref{gnn:lem:delta_contract}: 
\[
\deltatensor{ \abs{ \neigh (i) } + 1} \contract{ j \in [ \abs{ \neigh (i) } ] } \brk*{ \weightmat{ 1 } \fvec{ \neigh (i)_j }} = \hadmp_{ j \in [ \abs{ \neigh (i) } ] } \brk*{ \weightmat{ 1 } \fvec{ \neigh (i)_j }} = \hidvec{1}{i}
\text{\,.}
\]

Assuming that the inductive claim holds for $l - 1 \geq 1$, we prove that it holds for $l$.
Let $i \in \vertices$ and $\gamma \in [ \nwalk{L - l}{\{i\}}{\vertices} ]$.
The children of $\deltatensordup{l}{i}{\gamma}$ in the tensor network are of the form $\weightmatdup{l}{ \neigh (i)_j }{ \dupmap{l}{i}{j} (\gamma) }$, for $j \in [ \abs{ \neigh (i) } ]$, and each $\weightmatdup{l}{ \neigh (i)_j }{ \dupmap{l}{i}{j} (\gamma) }$ is connected in its other mode to $\deltatensordup{l - 1}{ \neigh (i)_j }{ \dupmap{l}{i}{j} (\gamma) }$.
By the inductive assumption for~$l - 1$, we know that performing all contractions in the sub-tree whose root is $\deltatensordup{l - 1}{ \neigh (i)_j }{ \dupmap{l}{i}{j} (\gamma) }$ produces $\hidvec{l-1}{ \neigh (i)_j }$, for all $j \in [ \abs{ \neigh (i) } ]$.
Since $\deltatensordup{l}{i}{\gamma}$ is a copy of $\deltatensor{ \abs{ \neigh (i) } + 1 }$, and each $\weightmatdup{l}{ \neigh (i)_j }{ \dupmap{l}{i}{j} (\gamma) }$ is a copy of $\weightmat{l}$, the remaining contractions in the sub-tree of $\deltatensordup{l}{i}{\gamma}$ thus give:
\[
\deltatensor{ \abs{ \neigh (i) } + 1 } \contract{ j \in [ \abs{ \neigh (i) } ] } \brk*{ \weightmat{l} \hidvec{l - 1}{ \neigh (i)_{j } } }
\text{\,,}
\]
which according to~\cref{gnn:lem:delta_contract} amounts to:
\[
\deltatensor{ \abs{ \neigh (i) } + 1 } \contract{ j \in [ \abs{ \neigh (i) } ] } \brk*{ \weightmat{l} \hidvec{l - 1}{ \neigh (i)_{j } } } = \hadmp_{ j \in [ \abs{ \neigh (i) } ] } \brk*{ \weightmat{l} \hidvec{l - 1}{ \neigh (i)_{j } } } = \hidvec{l}{i}
\text{\,,}
\]
establishing the induction step.

\medskip

With the inductive claim at hand, we show that contracting $\tngraph (\fmat)$ produces $\funcgraph{\params}{\graph} (\fmat)$.
Applying the inductive claim for $l = L$, we have that $\hidvec{L}{1}, \ldots, \hidvec{L}{ \abs{ \vertices} }$ are the vectors produced by executing all contractions in the sub-trees whose roots are $\deltatensordup{L}{1}{1}, \ldots, \deltatensordup{L}{ \abs{\vertices} }{1}$, respectively.
Performing the remaining contractions, defined by the legs of $\deltatensor{ \abs{\vertices} + 1}$, therefore yields $ \weightmat{o} \brk1{ \deltatensor{ \abs{\vertices} + 1} \contract{ i \in [\abs{\vertices}] } \hidvec{L}{ i } }$. By~\cref{gnn:lem:delta_contract}:
\[
\deltatensor{ \abs{\vertices} + 1} \contract{ i \in [\abs{\vertices}] } \hidvec{L}{ i } = \hadmp_{ i \in [\abs{\vertices}] } \hidvec{L}{ i }
\text{\,.}
\]
Hence, $ \weightmat{o} \brk1{ \deltatensor{ \abs{\vertices} + 1} \contract{ i \in [\abs{\vertices}] } \hidvec{L}{ i } } = \weightmat{o} \brk{ \hadmp_{ i \in [\abs{\vertices}] } \hidvec{L}{ i } } = \funcgraph{\params}{\graph} (\fmat)$, meaning contracting $\tngraph (\fmat)$ results in $\funcgraph{\params}{\graph} (\fmat)$.

\medskip

An analogous proof establishes that the contractions described by $\tngraph^{(t)} (\fmat)$ yield $\funcvert{\params}{\graph}{t} (\fmat)$.
Specifically, the inductive claim and its proof are the same, up to $\gamma$ taking values in $[ \nwalk{L - l}{ \{i\} }{ \{t\}} ]$ instead of $[ \nwalk{L - l}{ \{i\} }{ \vertices } ]$, for $l \in [L]$.
This implies that $\hidvec{L}{t}$ is the vector produced by contracting the sub-tree whose root is $\deltatensordup{L}{t}{1}$.
Performing the only remaining contraction, defined by the leg connecting $\deltatensordup{L}{t}{1}$ with $\weightmat{o}$, thus results in $\weightmat{o} \hidvec{L}{t} = \funcvert{\params}{\graph}{t} (\fmat)$.
\qed

\subsubsection{Technical Lemma}
\label{gnn:app:proofs:tn_constructions:lemmas}

\begin{lemma}
	\label{gnn:lem:delta_contract}
	Let $\deltatensor{N + 1} \in \R^{D \times \cdots \times D}$ be an order $N + 1 \in \N$ tensor that has ones on its hyper-diagonal and zeros elsewhere, \ie~$\deltatensor{N + 1}_{d_1, \ldots, d_{N + 1}} = 1$ if $d_1 = \cdots = d_{N + 1}$ and $\deltatensor{N + 1}_{d_1, \ldots, d_{N + 1}} = 0$ otherwise, for all $d_1, \ldots, d_{N + 1} \in [D]$.
	Then, for any $\xbf^{(1)}, \ldots, \xbf^{(N)} \in \R^{D}$ it holds that $\deltatensor{N + 1} \contract{i \in [N]} \xbf^{(i)} = \hadmp_{i \in [N]} \xbf^{(i)} \in \R^D$.
\end{lemma}
\begin{proof}
	By the definition of tensor contraction (\cref{gnn:def:tensor_contraction}), for all $d \in [D]$ we have that:
	\[
	\brk1{ \deltatensor{N + 1} \contract{i \in [N]} \xbf^{(i)} }_{d} = \sum\nolimits_{d_1, \ldots, d_N = 1}^{D} \deltatensor{N + 1}_{d_1, \ldots, d_N, d} \cdot \prod\nolimits_{i \in [N]} \xbf^{(i)}_{d_i} = \prod\nolimits_{i \in [N]} \xbf^{(i)}_d = \brk1{ \hadmp_{i \in [N]} \xbf^{(i)} }_d
	\text{\,.}
	\]
\end{proof}